\documentclass[twoside,11pt]{article}

\usepackage{jmlr2e}

\usepackage[utf8]{inputenc} 
\usepackage{hyperref}       
\hypersetup{hidelinks}

\usepackage{url}            
\usepackage{booktabs}       
\usepackage{amsfonts}       
\usepackage{nicefrac}       
\usepackage{microtype}      
\usepackage{float}
\usepackage{enumerate}

\usepackage{textcomp}
\usepackage{xspace}
\usepackage{amsmath}
\usepackage[linesnumbered, inoutnumbered]{algorithm2e}
\usepackage{xcolor}
\usepackage{dsfont}
\usepackage{amssymb}
\usepackage{graphicx}
\usepackage{subfig}
\usepackage{blkarray}
\usepackage{pifont}
\usepackage{physics}
\usepackage{gensymb}
\usepackage{comment}
\usepackage{mathalfa}

\usepackage{mathrsfs}

\usepackage {tikz}
\usetikzlibrary {positioning}
\definecolor {processblue}{cmyk}{0.96,0,0,0}
\definecolor {graphgreen}{cmyk}{0.86,0,0.63,0.13}
\definecolor {graphdarkgreen}{cmyk}{0.5,0,0.23,0.9}

\DeclareMathAlphabet{\mathpzc}{OT1}{pzc}{m}{it}

\newcommand{\RR}{\mathbb{R}} 
\newcommand{\AAA}{\mathbb{A}}

\newcommand{\KKK}{\mathbb{K}}

\newcommand{\ZZ}{\mathbb{Z}} 
\newcommand{\CC}{\mathbb{C}} 
\newcommand{\PP}{\mathbb{P}} 
\newcommand{\QQ}{\mathbb{Q}}
\newcommand{\FF}{\mathcal{F}}
\newcommand{\II}{\mathcal{I}}
\newcommand{\JJ}{\mathcal{J}}
\newcommand{\HH}{\mathcal{H}}
\newcommand{\DD}{\mathcal{D}}
\newcommand{\NN}{\mathcal{N}}
\newcommand*{\rom}[1]{\uppercase\expandafter{\romannumeral #1\relax}}
\newcommand{\spp}{\mathbb{S}} 
\newcommand{\half}[0]{\frac{1}{2}}
\newcommand{\pihalf}[0]{\frac{\pi}{2}}
\newcommand{\prob}[1]{\ensuremath{\mathbb{P}({#1})}}

\newcommand{\probi}[2]{\ensuremath{\mathbb{P}^{#1}({#2})}}
\newcommand{\probs}[1]{\ensuremath{\mathbb{P}^{#1}}}
\newcommand{\probii}[3]{\ensuremath{\mathbb{P}^{#1}_{#2}({#3})}}
\newcommand{\expectv}{\mathbb{E}}
\newcommand{\hess}{\mathpzc{H}}
\DeclareMathOperator*{\E}{\mathbb{E}}
\DeclareMathOperator*{\cov}{\text{Cov}}
\DeclareMathOperator*{\variance}{\text{Var}}
\DeclareMathOperator*{\atantwo}{\text{atan2}}

\newcommand{\psoset}{\mathds{C}}
\newcommand{\pdfestset}{\mathds{A}}
\newcommand{\logpdfestset}{\mathds{B}}

\newcommand{\RRpos}{\RR_{> 0}}
\newcommand{\RRnonneg}{\RR_{\geq 0}}
\newcommand{\RRneg}{\RR_{< 0}}

\newcommand{\usuff}{{\scriptscriptstyle U}}
\newcommand{\dsuff}{{\scriptscriptstyle D}}
\newcommand{\tsuff}{}
\newcommand{\ttsuff}{{\scriptscriptstyle T}}
\newcommand{\udcapsuff}{{\scriptscriptstyle U \cap D}}
\newcommand{\udcupsuff}{{\scriptscriptstyle U \cup D}}
\newcommand{\udbacksuff}{{\scriptscriptstyle U \backslash D}}
\newcommand{\dubacksuff}{{\scriptscriptstyle D \backslash U}}

\newcommand{\down}[0]{\emph{down}\xspace}
\newcommand{\up}[0]{\emph{up}\xspace}
\newcommand{\bp}[0]{\emph{balance state}\xspace}

\newcommand{\df}[0]{\emph{differential}\xspace}
\newcommand{\dfs}[0]{\emph{differentials}\xspace}
\newcommand{\mf}[0]{\emph{magnitude} function\xspace}
\newcommand{\mfs}[0]{\emph{magnitude} functions\xspace}
\newcommand{\ms}[0]{\emph{magnitudes}\xspace}
\newcommand{\mgn}[0]{\emph{magnitude}\xspace}

\newcommand{\pfs}[0]{\emph{primitive} functions\xspace}
\newcommand{\ps}[0]{\emph{primitives}\xspace}

\newcommand{\gs}[0]{\emph{gradient similarity}\xspace}
\newcommand{\bnd}[0]{\emph{bandwidth}\xspace}

\newcommand{\infh}[0]{\emph{infinite height}\xspace}
\newcommand{\pair}[0]{$\{ M^{\usuff}, M^{\dsuff} \}$\xspace}
\newcommand{\widepair}[0]{$\{ \widetilde{M}^{\usuff}, \widetilde{M}^{\dsuff} \}$\xspace}
\newcommand{\psocomp}[0]{$\{ \probs{\usuff}, \probs{\dsuff}, M^{\usuff}, M^{\dsuff} \}$\xspace}
\newcommand{\psofunc}[0]{\emph{PSO functional}\xspace}
\newcommand{\psodiv}[0]{\emph{PSO divergence}\xspace}
\newcommand{\overfit}[0]{\emph{overfitting}\xspace}
\newcommand{\underfit}[0]{\emph{underfitting}\xspace}
\newcommand{\pconv}[0]{\overset{p}{\to}}
\newcommand{\dconv}[0]{\overset{d}{\to}}

\newcommand{\psorepo}[0]{https://bit.ly/2AMwyJT}

\newcommand{\cmark}{{\color[rgb]{0,0.4,0} \checkmark}\xspace}%
\newcommand{\xmark}{{\color{red} \ding{55}}\xspace}%

\DeclareMathOperator*{\argmin}{arg\,min}
\DeclareMathOperator*{\argmax}{arg\,max}
\DeclareMathOperator*{\arginf}{arg\,inf}

\makeatletter
\let\save@mathaccent\mathaccent
\newcommand*\if@single[3]{%
	\setbox0\hbox{${\mathaccent"0362{#1}}^H$}%
	\setbox2\hbox{${\mathaccent"0362{\kern0pt#1}}^H$}%
	\ifdim\ht0=\ht2 #3\else #2\fi
}
\newcommand*\rel@kern[1]{\kern#1\dimexpr\macc@kerna}
\newcommand*\widebar[1]{\@ifnextchar^{{\wide@bar{#1}{0}}}{\wide@bar{#1}{1}}}
\newcommand*\wide@bar[2]{\if@single{#1}{\wide@bar@{#1}{#2}{1}}{\wide@bar@{#1}{#2}{2}}}
\newcommand*\wide@bar@[3]{%
	\begingroup
	\def\mathaccent##1##2{%
		\let\mathaccent\save@mathaccent
		\if#32 \let\macc@nucleus\first@char \fi
		\setbox\z@\hbox{$\macc@style{\macc@nucleus}_{}$}%
		\setbox\tw@\hbox{$\macc@style{\macc@nucleus}{}_{}$}%
		\dimen@\wd\tw@
		\advance\dimen@-\wd\z@
		\divide\dimen@ 3
		\@tempdima\wd\tw@
		\advance\@tempdima-\scriptspace
		\divide\@tempdima 10
		\advance\dimen@-\@tempdima
		\ifdim\dimen@>\z@ \dimen@0pt\fi
		\rel@kern{0.6}\kern-\dimen@
		\if#31
		\overline{\rel@kern{-0.6}\kern\dimen@\macc@nucleus\rel@kern{0.4}\kern\dimen@}%
		\advance\dimen@0.4\dimexpr\macc@kerna
		\let\final@kern#2%
		\ifdim\dimen@<\z@ \let\final@kern1\fi
		\if\final@kern1 \kern-\dimen@\fi
		\else
		\overline{\rel@kern{-0.6}\kern\dimen@#1}%
		\fi
	}%
	\macc@depth\@ne
	\let\math@bgroup\@empty \let\math@egroup\macc@set@skewchar
	\mathsurround\z@ \frozen@everymath{\mathgroup\macc@group\relax}%
	\macc@set@skewchar\relax
	\let\mathaccentV\macc@nested@a
	\if#31
	\macc@nested@a\relax111{#1}%
	\else
	\def\gobble@till@marker##1\endmarker{}%
	\futurelet\first@char\gobble@till@marker#1\endmarker
	\ifcat\noexpand\first@char A\else
	\def\first@char{}%
	\fi
	\macc@nested@a\relax111{\first@char}%
	\fi
	\endgroup
}
\newcommand*\widetilda[1]{\@ifnextchar^{{\wide@tilda{#1}{0}}}{\wide@tilda{#1}{1}}}
\newcommand*\wide@tilda[2]{\if@single{#1}{\wide@tilda@{#1}{#2}{1}}{\wide@tilda@{#1}{#2}{2}}}
\newcommand*\wide@tilda@[3]{%
	\begingroup
	\def\mathaccent##1##2{%
		\let\mathaccent\save@mathaccent
		\if#32 \let\macc@nucleus\first@char \fi
		\setbox\z@\hbox{$\macc@style{\macc@nucleus}_{}$}%
		\setbox\tw@\hbox{$\macc@style{\macc@nucleus}{}_{}$}%
		\dimen@\wd\tw@
		\advance\dimen@-\wd\z@
		\divide\dimen@ 3
		\@tempdima\wd\tw@
		\advance\@tempdima-\scriptspace
		\divide\@tempdima 10
		\advance\dimen@-\@tempdima
		\ifdim\dimen@>\z@ \dimen@0pt\fi
		\rel@kern{0.6}\kern-\dimen@
		\if#31
		\tilde{\rel@kern{-0.6}\kern\dimen@\macc@nucleus\rel@kern{0.4}\kern\dimen@}%
		\advance\dimen@0.4\dimexpr\macc@kerna
		\let\final@kern#2%
		\ifdim\dimen@<\z@ \let\final@kern1\fi
		\if\final@kern1 \kern-\dimen@\fi
		\else
		\tilde{\rel@kern{-0.6}\kern\dimen@#1}%
		\fi
	}%
	\macc@depth\@ne
	\let\math@bgroup\@empty \let\math@egroup\macc@set@skewchar
	\mathsurround\z@ \frozen@everymath{\mathgroup\macc@group\relax}%
	\macc@set@skewchar\relax
	\let\mathaccentV\macc@nested@a
	\if#31
	\macc@nested@a\relax111{#1}%
	\else
	\def\gobble@till@marker##1\endmarker{}%
	\futurelet\first@char\gobble@till@marker#1\endmarker
	\ifcat\noexpand\first@char A\else
	\def\first@char{}%
	\fi
	\macc@nested@a\relax111{\first@char}%
	\fi
	\endgroup
}
\makeatother

\newtheorem{claim}[theorem]{Claim}

\usepackage{caption}
\usepackage{autonum}

\jmlrheading{1}{2019}{1-141}{1/00}{1/00}{tttt}{Dmitry Kopitkov and Vadim Indelman}

\ShortHeadings{General PSO and Log Density Estimation}{Dmitry Kopitkov and Vadim Indelman}
\firstpageno{1}

\begin{document}

\title{General Probabilistic Surface Optimization and\\ Log Density Estimation}

\author{\name Dmitry Kopitkov \email dimkak@technion.ac.il \\
	\addr Technion Autonomous Systems Program (TASP)\\
	Technion - Israel Institute of Technology\\
	Haifa 32000, Israel
	\AND
	\name Vadim Indelman \email vadim.indelman@technion.ac.il \\
	\addr Department of Aerospace Engineering\\
    Technion - Israel Institute of Technology\\
    Haifa 32000, Israel}

\editor{XXX}

\maketitle

\begin{abstract}

Probabilistic inference, such as density (ratio) estimation, is a fundamental and highly important problem that needs to be solved in many different domains. Recently, a lot of research was done to solve it by producing various objective functions optimized over neural network (NN) models.
Such Deep Learning (DL) based approaches include \emph{unnormalized} and \emph{energy} models, as well as critics of Generative Adversarial Networks, where DL has shown top approximation performance. In this paper we contribute a novel algorithm family, which generalizes all above, and allows us to infer different statistical modalities (e.g.~data likelihood and ratio between densities) from data samples. The proposed \emph{unsupervised} technique, named \emph{Probabilistic Surface Optimization} (PSO), views a model as a flexible surface which can be pushed according to loss-specific virtual stochastic forces, where a dynamical equilibrium is achieved when the pointwise forces on the surface become equal. Concretely, the surface is pushed \up and \down at points sampled from two different distributions. The averaged \up and \down forces become functions of these two distribution densities and of force \emph{magnitudes} defined by the loss of a particular PSO instance. Upon convergence, the force equilibrium imposes an optimized model to be equal to various statistical functions depending on the used \mfs. Furthermore, this dynamical-statistical equilibrium is extremely intuitive and useful, providing many implications and possible usages in probabilistic inference. We connect PSO to numerous existing statistical works which are also PSO instances, and derive new PSO-based inference methods as a demonstration of PSO exceptional usability. 
Likewise, based on the insights coming from the virtual-force perspective we analyze PSO stability and propose new ways to improve it.
Finally, we present new instances of PSO, termed PSO-LDE,  for data log-density estimation and also provide a new NN block-diagonal architecture for increased surface flexibility, which significantly improves estimation accuracy. Both PSO-LDE and the new architecture are combined together as a new density estimation technique. In our experiments we demonstrate this technique to be superior over state-of-the-art baselines in density estimation tasks for multi-modal 20D data.


\end{abstract}

\pagebreak
\tableofcontents
\vfil
\pagebreak

%

\section{Introduction}
\label{sec:Intro}

Probabilistic inference is the wide domain of extremely important statistical problems including density (ratio) estimation, distribution transformation, density sampling and many more. Solutions to these problems are extensively used in domains of robotics, computer image, economics, and other scientific/industrial data mining cases. Particularly, in robotics we require to manually/automatically infer a measurement model between sensor measurements and the hidden state of the robot, which can further be used to estimate robot state during an on-line scenario. 
Considering the above, solutions to probabilistic inference and their applications to real-world problems are highly important for many scientific fields.

The universal approximation theory \citep{Hornik91nn} states that an artificial neural network with fully-connected layers can approximate any continuous function on compact subsets of $\RR^n$, making it an universal approximation tool.
Moreover, in the last decade  methods based on Deep Learning (DL) provided outstanding performance in areas of computer vision and reinforcement learning. 
Furthermore, recently strong  frameworks (e.g.~\citet{Tensorflow_url, Pytorch_url, Caffe_url}) were developed that allow fast and sophisticated training of neural networks (NNs) using GPUs.

With the above motivation, 
in this paper we contribute a novel unified paradigm,
\emph{Probabilistic Surface Optimization} (PSO),
that allows to solve various probabilistic inference problems
using DL, where we exploit the approximation power of NNs in full.
PSO expresses the probabilistic inference as a virtual physical system 
where the surface, represented by a function from the optimized function space (e.g. NN), is pushed by forces that are outcomes of a Gradient Descent (GD) optimization. We show that this surface is pushed during the optimization to the target surface for which the averaged pointwise forces cancel each other. Further, by using different virtual forces we can enforce the surface to converge to different probabilistic functions of data, such as a data density, various density ratios, conditional densities and many other useful statistical modalities.

We show that many existing probabilistic inference approaches, like \emph{unnormalized} models, GAN critics, \emph{energy} models and cross-entropy based methods, already apply such PSO principles implicitly, even though their underlying dynamics were not explored before through the prism of virtual forces. Additionally, many novel and original methods can be forged in a simple way by following the same fundamental rules of the virtual surface and the force balance. 
Moreover, PSO framework permits the proposal and the practical usage of new objective functions that can not be expressed in closed-form, by instead defining their Euler-Lagrange equation. This allows introduction of new estimators that were not considered before.
Furthermore, motivated by usefulness and intuitiveness of the proposed PSO paradigm, we derive sufficient conditions for its optimization stability and further analyze its convergence, relating it to the model kernel also known in DL community as Neural Tangent Kernel (NTK) \citep{Jacot18nips}.

Importantly, we emphasize that PSO is not only a new interpretation that allows for simplified and intuitive understanding of statistical learning. Instead, in this paper we show that optimization dynamics that the inferred model undergoes are indeed matching the picture of a physical surface with particular forces applied on it. Moreover, such match allowed us to understand the optimization stability of PSO instances in more detail and to suggest new ways to improve it.

Further, we apply PSO framework to solve the density estimation task - a fundamental statistical problem essential in many scientific fields and application domains.
We analyze PSO sub-family with the corresponding equilibrium, proposing a novel PSO log-density estimators (PSO-LDE). 
These techniques, as also other PSO-based density estimation approaches presented in this paper, do not impose any explicit constraint over a total integral of the learned model, allowing it to be entirely unnormalized. Yet, 
the implicit PSO force balance produces at the convergence density approximations that are highly accurate and \emph{almost} normalized, with total integral being very close to 1.

Finally, we examine several NN architectures for a better estimation performance, and propose new block-diagonal layers that led us to significantly improved accuracy. PSO-LDE approach combined with new NN architecture allowed us to learn multi-modal densities of 20D continuous data with superior precision compared to other state-of-the-art methods, including Noise Contrastive Estimation (NCE) \citep{Smith05acl,Gutmann10aistats},
which we demonstrate in our experiments.

To summarize, our main contributions in this paper are as follows: 
\begin{enumerate}[(a)]
	\item We develop a \emph{Probabilistic Surface Optimization} (PSO) that enforces any approximator function to converge to a target statistical function which nullifies a point-wise virtual force.
	
	\item We derive sufficient optimality conditions under which the functional implied by PSO is stable during the optimization.

	\item We show that many existing probabilistic and (un-)supervised learning techniques can be seen as instances of PSO.
	
	\item We show how new probabilistic techniques can be derived in a simple way by using PSO principles, and also propose several such new methods.
	
	\item We provide analysis of PSO convergence where we relate its performance towards properties of an model kernel implicitly defined by the optimized function space.

	\item We use PSO to approximate a logarithm of the target density, proposing for this purpose several hyper-parametric PSO subgroups and analyzing their properties.

	\item We present a new NN architecture with block-diagonal layers that allows for lower side-influence (a smaller bandwidth of the corresponding model kernel) between various regions of the input space and that leads to a higher NN flexibility and to more accurate density estimation.

	\item We experiment with different continuous 20D densities, and accurately infer all of them using the proposed PSO instances, thus demonstrating these instances' robustness and top performance. Further, we compare our methods with  state-of-the-art baselines, showing the superiority of former over latter.
		
\end{enumerate}

The paper is structured as follows. In Section \ref{sec:RelW} we describe the related work, in Section \ref{sec:PromOptim} we formulate PSO algorithm family, in Section \ref{sec:PSOFuncObj} we derive sufficient optimality conditions, in Section \ref{sec:PSOInst} we relate PSO framework to other existing methods showing them to be its instances,
and in Section \ref{sec:Bregman_PSO} we outline its relation towards various statistical divergences.
Furthermore, in Section \ref{sec:ConvA} numerous estimation properties are proved, including consistency and asymptotic normality, and the model kernel's impact on the optimization equilibrium is investigated. Application of PSO for (conditional) density estimation is described in Sections \ref{sec:DensEst}-\ref{sec:CondDeepPDFMain}, various additional PSO applications and relations - in Section \ref{sec:PSOApp}, NN design and its optimization influence - in Section \ref{sec:NNArch}, and an illustration of PSO \overfit - in Section \ref{sec:DeepLogPDFOF}. Finally,
experiments appear in Section \ref{sec:Exper}, and discussion of conclusions - in Section \ref{sec:Concl}.

\section{Related work}
\label{sec:RelW}

In this section we consider very different problems all of which involve reasoning about statistical properties and probability density of a given data, which can be also solved by various instances of PSO as is demonstrated in later sections.
We describe studies done to solve these problems, including both DL and not-DL based methods, and relate their key properties to attributes of PSO.

\subsection{Unsupervised Probabilistic Inference}
\label{sec:RelWUnsupr}

Statistical estimation consists of learning various probabilistic modalities from acquired sample realizations. For example, given a dataset we may want to infer the corresponding probability density function (pdf). Similarly, given two datasets we may want to approximate the density ratio between sample distributions.
The usage of NNs for statistical estimation was studied for several decades \citep{Smolensky86rep, Bishop94tr, Bengio00nips, Hinton06, Uria2013nips}. 
Furthermore, there is a huge amount of work that treats statistical learning in a similar way to PSO, based on sample frequencies, the optimization \emph{energies} and their forces. Arguably, the first methods were Boltzman machines (BMs) and Restricted Boltzman machines (RBMs) \citep{Ackley85cs,Osborn90innc,Hinton02nc}. Similarly to PSO, RBMs can learn a distribution over data samples using a physical equilibrium, and were proved to be very useful for various ML tasks such as dimensionality reduction and feature learning. Yet, they were based on a very basic NN architecture, containing only hidden and visible units, arguably because of over-simplified formulation of the original BM. Moreover, the training procedure  of these methods, the contrastive divergence (CD) \citep{Hinton02nc}, applies computationally expensive Monte Carlo (MC) sampling. In Section \ref{sec:SimpL} we describe CD in detail and outline its exact relation to PSO procedure, showing that the latter replaces the expensive MC by sampling auxiliary distribution which is computationally cheap.

In \citep{Ngiam11icml} authors extended RBMs to Deep Energy Models (DEMs) that contained multiple fully-connected layers, where during  training each layer was trained separately via CD. Further, in \citep{Zhai16icml} Deep Structured Energy Based Models were proposed that used fully-connected, convolutional and recurrent NN architectures for an anomaly detection of vector data, image data and time-series data respectively. Moreover, in the latter work authors proposed to train \emph{energy} based models via a score matching method \citep{Hyvarinen05mlr}, which does not require MC sampling. A similar training method was also recently applied in \citep{Saremi18arxiv} for learning an \emph{energy} function of data - an \emph{unnormalized} model that is proportional to the real density function. However, the produced by score matching \emph{energy} function is typically over-smoothed and entirely unnormalized, with its total integral being arbitrarily far from 1 (see Section \ref{sec:ColumnsEstBslns}). In contrast, PSO based density estimators (e.g. PSO-LDE) yield a model that is \emph{almost} normalized, with its integral being very close to 1 (see Section \ref{sec:ColumnsEstME}).

In \citep{Lecun06tutorial} authors examined many statistical loss functions under the perspective of \emph{energy} model learning. Their overview of existing learning rules describes a typical optimization procedure as a physical system of model pushes at various data samples in \up and \down directions, producing the intuition very similar to the one promoted in this work. Although \citep{Lecun06tutorial} and our works were done in an independent manner with the former preceding the latter, both acknowledged that many objective functions have two types of terms corresponding to two force directions, that are responsible to enforce model to output desired energy levels for various neighborhoods of the input space. Yet, unlike \citep{Lecun06tutorial} we take one step further and derive the precise way to control the involved forces, producing a formal framework for the generation of infinitely many learning rules to infer an infinitely large number of target functions.
The proposed PSO approach is conceptually very intuitive, and permits unification of many various methods under a single algorithm umbrella via a formal yet simple mathematical exposition. 
This in its turn allows to address the investigation of different statistical techniques and their properties as one mutual analysis study.

In context of pdf estimation, one of the most relevant works to presented in this paper PSO-LDE approach is NCE \citep{Smith05acl,Gutmann10aistats}, which formulates the inference problem via a binary classification between original data samples and auxiliary noise samples. The derived loss allows for an efficient (conditional) pdf inference and is widely adapted nowadays in the language modeling domain \citep{Mnih12arxiv,Mnih13nips,Labeau18iccl}. Further, the proposed PSO-LDE can be viewed as a generalization of NCE, where the latter is a specific member of the former for a hyper-parameter $\alpha = 1$. Yet importantly, both algorithms were derived based on different mathematical principles, and their formulations do not exactly coincide.

Furthermore, the presented herein PSO family is not the first endeavor for unifying different statistical techniques under a general algorithm umbrella. In \citep{Pihlaja12arxiv} authors proposed a family of \emph{unnormalized} models to infer log-density, which is based on Maximum Likelihood Monte Carlo estimation \citep{Geyer92jrssb}. Their method infers both the \emph{energy} function of the data and the appropriate normalizing constant. Thus, the produced (log-)pdf estimation is \emph{approximately} normalized.  Further, this work was extended in \citep{Gutmann12arxiv} where it was related to the separable Bregman divergence and where various other statistical methods, including NCE,
were shown to be instances of this inference framework.
In Section \ref{sec:Bregman_PSO} we prove Bregman-based estimators to be contained inside PSO estimation family, and thus both of the above frameworks are strict subsets of PSO.

Further, in \citep{Nguyen10tit} and \citep{Nowozin16nips} new techniques were proposed to infer various $f$-divergences between two densities, based on $M$-estimation procedure and Fenchel conjugate \citep{Hiriart12book}. Likewise, the f-GAN framework in \citep{Nowozin16nips} was shown to include many of the already existing GAN methods. In Section \ref{sec:Bregman_PSO} we prove that estimation methods from \citep{Nguyen10tit} and critic objective functions from \citep{Nowozin16nips} are also strict subsets of PSO.

The above listed methods, as also the PSO instances in Section \ref{sec:PSOInst}, are all derived using various math fields, yet they also could be easily derived via PSO \bp as is described in this paper. Further, the simplest way to show that PSO is a generalization and not just another perspective that is identical to previous works is as follows. In most of the above approaches optimization objective functions are required to have an analytically known closed form,
whereas in our framework knowledge of these functions is not even required. Instead, we formulate the learning procedure via \mfs, the derivatives of various loss terms, knowing which is enough to solve the corresponding minimization problem.
Furthermore, the \ms of PSO-LDE sub-family in Eq.~(\ref{eq:PSOLDELossMU})-(\ref{eq:PSOLDELossMD}) do not have a known antiderivative for the general case of any $\alpha$, with the corresponding PSO-LDE loss being unknown. Thus, PSO-LDE (and therefore PSO) cannot be viewed as an instance of any previous statistical framework. Additionally, the intuition and simplicity in viewing the optimization as merely point-wise pushes over some virtual surface are very important for the investigation of PSO stability and for its applicability in numerous different areas.

\subsection{Parametric vs Non-parametric Approaches}
\label{sec:RelWParamAppr}

The most traditional probabilistic problem, which is also one of the main focuses of this paper, is density approximation for an arbitrary data.
Approaches for statistical density estimation may be divided into two different branches - parametric and non-parametric. Parametric methods assume data to come from a probability distribution of a specific family, and infer parameters of that family, for example via minimizing the negative log-probability of data samples. Non-parametric approaches are distribution-free in the sense that they do not take any assumption over the data population a priori. Instead they infer the distribution density totally from data.

The main advantage of the parametric approaches is their statistical efficiency. Given the assumption of a specific distribution family is correct, parametric methods will produce more accurate density estimation for the same number of samples compared to non-parametric techniques. However, in case the assumption is not entirely valid for a given population, the estimation accuracy will be poor, making parametric methods not statistically robust. For example, one of the most expressive distribution families is a Gaussian Mixture Model (GMM) \citep{McLachlan88}. One of its structure parameters is the number of mixtures. Using a high number of mixtures, it can represent multi-modal populations with a high accuracy. Yet, in case the real unknown distribution has even higher number of modes, or sometimes even an infinite number, the performance of a GMM will be low.

To handle the problem of unknown number of mixture components in parametric techniques, Bayesian statistics can be applied to model a prior over parameters of the chosen family. Models such as Dirichlet process mixture (DPM) and specifically Dirichlet process Gaussian mixture model (DPGMM) \citep{Antoniak74aos, Sethuraman82sdtrt, Gorur10jcst} can represent an uncertainty about the learned distribution parameters and as such can be viewed as infinite mixture models. Although these hierarchical models are more statistically robust (expressive), they still require to manually select a base distribution for DPM, limiting their robustness. Likewise, Bayesian inference applied in these techniques is more theoretically intricate and computationally expensive \citep{MacEachern98}.

On the other hand, non-parametric approaches can infer distributions of an (almost) arbitrary form. Methods such as data histogram and kernel density estimation (KDE) \citep{Scott15book, Silverman18} use frequencies of different points within data samples in order to conclude how a population pdf looks like. In general, these methods require more samples and prone to the \emph{curse of  dimensionality}, but also provide a more robust estimation by not taking any prior assumptions. Observe that "non-parametric" terminology \textbf{does not} imply lack of parametrization.
Both histogram and KDE require selection of (hyper) parameters - bin width for histogram and kernel type/bandwidth for KDE.

In many cases a selection of optimal parameters requires the manual parameter search \citep{Silverman18}. 
Although an automatic parameter deduction was proposed for KDE in several studies \citep{Duong05sjs, Heidenreich13asta, OBrien16csda}, it is typically computationally expensive and its performance is not always optimal. Furthermore, one of the major weaknesses of KDE is its time complexity during the query stage. Even the most efficient KDE methods \citep[e.g. fastKDE,][]{OBrien16csda} require an above linear complexity ($O(m \log m)$) in the number of query points $m$. In contrast, PSO yields robust non-parametric algorithms that optimize the NN model, which in its turn can be queried at any input point by a single forward pass. Since this pass is independent of $m$, the query runtime of PSO is linear in $m$.
When the complexity of NN forward pass is lower than $\log m$,
PSO methods become a much faster alternative.
Moreover, existing KDE implementations do not scale well for data with a high dimension, unlike PSO methods.

\subsection{Additional Density Estimation Techniques}
\label{sec:RelWPIother}

A unique work combining DL and non-parametric inference was done by Baird et al.~\citep{Baird05ijcnn}. The authors represent a target pdf via Jacobian determinant of a bijective NN that has an implicit property of non-negativity with the total integral being 1. Additionally, their \emph{pdf learning algorithm} has similarity to our \emph{pdf loss} described in \citep{Kopitkov18arxiv} and which is also shortly presented in Section \ref{sec:DeepPDF}. Although the authors did not connect their approach to virtual physical forces that are pushing a model surface, their algorithm can be seen as a simple instance of the more general DeepPDF method that we contributed in our previous work.

Furthermore, the usage of Jacobian determinant and bijective NNs in \citep{Baird05ijcnn} is just one instance of DL algorithm family based on a nonlinear independent components analysis. Given the transformation $G_{\phi}(\cdot)$ typically implemented as a NN parametrized by $\phi$, methods of this family \citep{Deco95nips, Rippel13arxiv, Dinh14arxiv, Dinh2016arxiv} exploit the integration by substitution theorem that provides a mathematical connection between random $X$'s pdf $\PP[X]$ and random $G_{\phi}(X)$'s pdf $\PP[G_{\phi}(X)]$ through Jacobian determinant of $G_{\phi}$. In case we know $\PP[X]$ of NN's input $X$, we can calculate in closed form the density $\PP[G_{\phi}(X)]$ of NN's output and vice versa, which may be required in different applications. However, for the substitution theorem to work the transformation $G_{\phi}(\cdot)$ should be invertible, requiring to restrict NN architecture of $G_{\phi}(\cdot)$ which significantly limits NN expressiveness. In contrast, the presented PSO-LDE approach does not require any restriction over its NN architecture.

Further, another body of research in DL-based density estimation was explored in \citep{Larochelle11aistats, Uria2013nips, Germain15icml}, where the autoregressive property of density functions was exploited. The described methods NADE, RNADE and MADE decompose the joint distribution of a  multivariate data into a product of simple conditional densities where a specific variable ordering needs to be selected for better performance. Although these approaches provide high statistical robustness, their performance is still limited since every
simple conditional density is approximated by a specific distribution family thus introducing a bias into the estimation. Moreover, the provided solutions are algorithmically complicated. In contrast, in this paper we develop a novel statistically robust and yet conceptually very simple algorithm for density estimation, PSO-LDE.

\subsection{Relation to GANs}
\label{sec:RelWGans}

Recently, Generative Adversarial Networks (GANs) \citep{Goodfellow14nips, Radford15arxiv, Ledig16arxiv} became popular methods to generate new data samples (e.g.~photo-realistic images). 
GAN learns a generative model of data samples, thus implicitly learning also the data distribution. The main idea behind these methods is to have two NNs, a generator and a critic, competing with each other. The goal of the generator NN is to create new samples statistically similar as much as possible to the given dataset of examples; this is done by transformation 
of samples from a predefined prior distribution $z \sim \probs{Z}$ which is typically a multivariate Gaussian. The responsibility of the critic NN is then to decide which of the samples given to it is the real data example and which is the fake. This is typically done by estimating the ratio between real and fake densities. The latter is performed by minimizing a critic loss, where most popular critic losses \citep{Mohamed16arxiv,Zhao16arxiv, Mao17iccv, Mroueh17nips, Gulrajani17nips, Arjovsky17arxiv} can be shown to be instances of PSO (see Section \ref{sec:PSOInst}). Further, both critic and generator NNs are trained in adversarial manner, forcing generator eventually to create very realistic data samples.

Another extension of GAN is Conditional GAN methods (cGANs). These methods use additional labels provided for each example in the training dataset (e.g. ground-truth digit of image from MNIST dataset), to generate new data samples conditioned on these labels. As an outcome, in cGAN methods we can control to some degree the distribution of the generated data, for example by conditioning the generation process on a specific data label (e.g. generate an image of digit "5"). Similarly, we can use such a conditional generative procedure in robotics where we would like to generate future measurements conditioned on old observations/current state belief. Moreover, cGAN critics are also members of PSO framework as is demonstrated in Section \ref{sec:CondDeepPDF}.

Further, it is a known fact that optimizing GANs is very unstable and fragile, though during the years different studies analyzed various instability issues and proposed techniques to handle them \citep{Arjovsky17_2arxiv}. In \citep{Radford15arxiv}, the authors proposed the DCGAN approach that combines several stabilization techniques such as the batch normalization and Relu non-linearity usage for a better GAN convergence. Further improvement was done in \citep{Salimans16nips} by using a parameter historical average and statistical feature matching. Additionally, in \citep{Arjovsky17_2arxiv} it was demonstrated that the main reason for instability in training GANs is the density support difference between the original and generated data. While this insight was supported by very intricate mathematical proofs, we came to the same conclusion in Section \ref{sec:ConvA} by simply applying equilibrium concepts of PSO. As we show, if there are areas where only one of the densities is positive, the critic's surface is pushed by virtual forces to infinity, causing the optimization instability (see also Figure \ref{fig:SprtMsmtchFig}).

Moreover, in our analysis we detected another significant cause for estimation inaccuracy - the strong implicit bias produced by the model kernel. In our experiments in Section \ref{sec:Exper} the bandwidth of this kernel is shown to be one of the biggest factors for a high approximation error in PSO. Moreover, in this paper we show that there is a strong analogy between the model kernel and the kernel applied in kernel density estimation (KDE) algorithms \citep{Scott15book, Silverman18}. Considering KDE, low/high values of the kernel \emph{bandwidth} can lead to both \underfit and \overfit, depending on the number of training samples. We show the same to be correct also for PSO. See more details in Sections \ref{sec:DeepLogPDFOF} and \ref{sec:Exper}.

\subsection{Classification Domain}
\label{sec:RelWClass}

Considering supervised learning and the image classification domain, convolutional neural networks (CNNs) produce discrete class conditional probabilities \citep{Krizhevsky12nips} for each image. The typical optimization loss used by classification tasks is a categorical cross-entropy of data and label pair, which can also be viewed as an instance of our discovered PSO family, as shown in Section \ref{sec:CrosEntr}. In particular, the classification cross-entropy loss can be seen as a variant of the PSO optimization, pushing in parallel multiple virtual surfaces connected by a softmax transformation, that concurrently estimates multiple Bernoulli distributions. These distributions, in their turn, represent one categorical distribution $\prob{C|I}$ that models probability of each object class $C$ given a specific image $I$.

Beyond cross-entropy loss, many possible objective functions for Bayes optimal decision rule inference exist \citep{Shen05_dissrt,Masnadi09nips,Reid10jmlr}. 
These objectives have various forms and a different level of statistical robustness, yet all of them enforce the optimized model to approximate $\prob{C|I}$. PSO framework promotes a similar relationship among its instances, by allowing the construction of infinitely many estimators with the same equilibrium yet with some being more robust to outliers than the others. Further, the recently proposed framework \citep{Blondel20jmlr} of classification losses extensively relies on notions of Fenchel duality, which is also employed in this paper to derive the sufficient conditions over PSO \ms.

\section{Probabilistic Surface Optimization}
\label{sec:PromOptim}

\begin{figure}
	\centering
	
	\begin{tabular}{c}
		
		\subfloat{\includegraphics[width=0.95\textwidth]{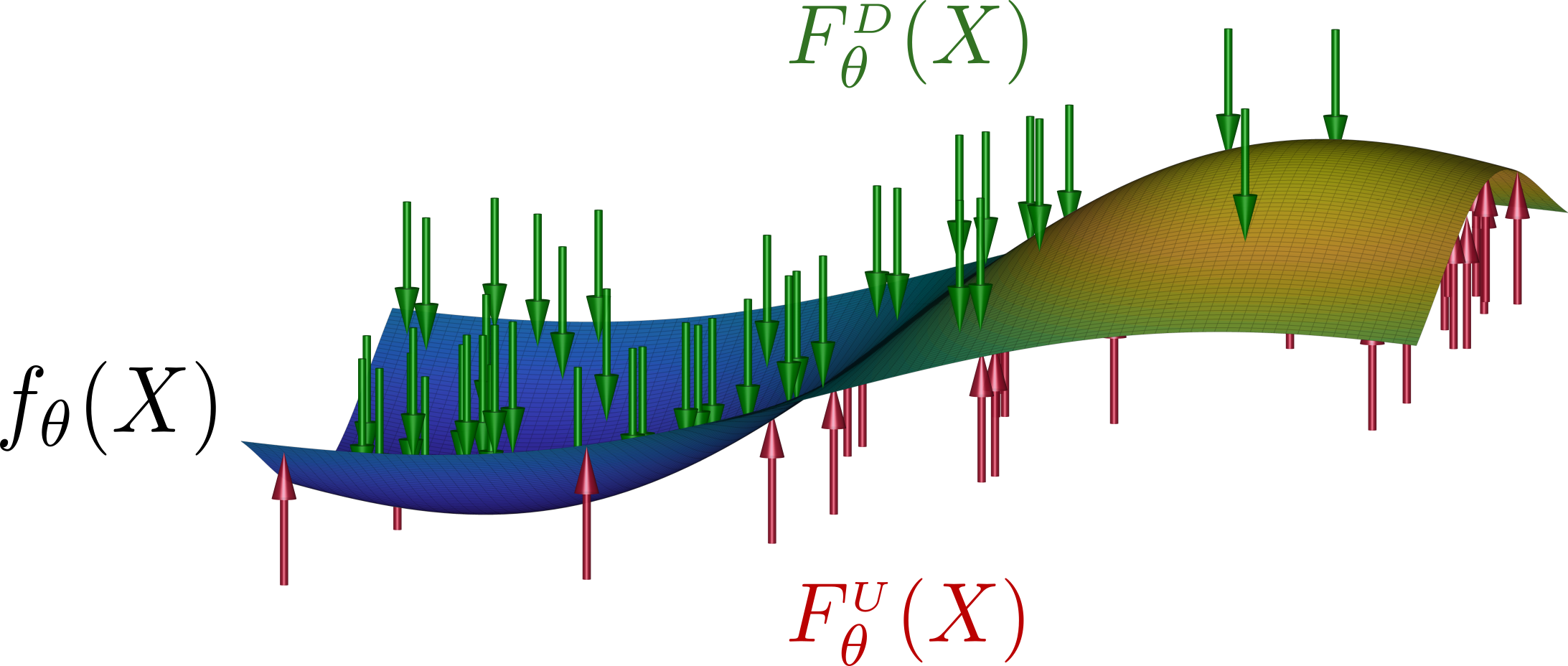}}
		
	\end{tabular}
	
	\protect
	\caption[Illustration of PSO principles.]{Illustration of PSO principles. Model $f_{\theta}(X): \RR^{n} \rightarrow \RR$ (in this paper NN parametrized by $\theta$) represents a virtual surface that is pushed in opposite directions - \up at points $X^{\usuff}$ sampled from $\probi{\usuff}{X}$ and \down at points $X^{\dsuff}$ sampled from $\probi{\dsuff}{X}$. Magnitude of each push is amplified by analytical function - $M^{\usuff}\left[X,f_{\theta}(X)\right]$ when pushing at $X^{\usuff}$ and $M^{\dsuff}\left[X,f_{\theta}(X)\right]$ when pushing at $X^{\dsuff}$, where both functions may have an almost arbitrary form, with only minor restrictions. During optimization via loss gradient in Eq.~(\ref{eq:GeneralPSOLossFrml}) the \up and \down forces $F_{\theta}^{\usuff}(X) = 
		\probi{\usuff}{X} \cdot 
		M^{\usuff}\left[X,f_{\theta}(X)\right]$ and $F_{\theta}^{\dsuff}(X) = 
		\probi{\dsuff}{X} \cdot 
		M^{\dsuff}\left[X,f_{\theta}(X)\right]$, containing both frequency and analytical components, adapt to each other till point-wise \bp $F_{\theta}^{\usuff}(X) = F_{\theta}^{\dsuff}(X)$ is achieved. Such convergence causes final $f_{\theta}(X)$ to be a particular function of $\probi{\usuff}{X}$ and $\probi{\dsuff}{X}$, and can be used for inferring numerous statistical functions of arbitrary data (see Section \ref{sec:PSOInst}).
	}
	\label{fig:SurfForces}
\end{figure}

\begin{table}[tb]
		\small
	\centering
	\begin{tabular}{ll}
		\toprule
		\textbf{Notation}     & \textbf{Description}      \\
		\midrule
		$\RRpos$, $\RRnonneg$ and $\RRneg$
		&
		sets of real numbers; positive  $\{ x \in \RR | x > 0 \}$, non-negative
		\\
		&
		$\{ x \in \RR | x \geq 0 \}$ and negative $\{ x \in \RR | x < 0 \}$ respectively
		\\
		$\FF$
		&
		function space over which PSO is inferred
		\\
		$f_{\theta}(X): \RR^{n} \rightarrow \RR$
		& 
		model $f_{\theta} \in \FF$, parametrized by $\theta$ (e.g. a neural network),  \\
		& can be viewed as a surface with support in $\RR^{n}$ whose  \\
		&
		height is the output of $f_{\theta}(X)$
		\\
		$\theta \in \RR^{\left|\theta\right|}$
		& model parameters (e.g. neural network weights vector)
		\\
		$X \in \RR^{n}$
		&
		input space of $f_{\theta}(X)$, can be viewed as support of 
		\\
		&
		model surface in space $\RR^{n + 1}$
		\\
		$X^{\usuff} \sim \probs{\usuff}$
		&
		$n$-dimensional random variable with pdf $\probs{\usuff}$, samples of which
		\\
		& are the locations where we push the model surface \up
		\\
		$X^{\dsuff} \sim \probs{\dsuff}$
		&
		$n$-dimensional random variable with pdf $\probs{\dsuff}$, samples of which
		\\
		&
		 are the locations where we push the model surface \down
		\\
		$\spp^{\usuff} \subset \RR^n$
		&
		support of $\probs{\usuff}$
		\\
		$\spp^{\dsuff} \subset \RR^n$
		&
		support of $\probs{\dsuff}$
		\\
		$\spp^{\udcupsuff} \subset \RR^n$
		&
		support union of $\probs{\usuff}$ and $\probs{\dsuff}$
		\\
		$\spp^{\udcapsuff} \subset \RR^n$
		&
		support intersection of $\probs{\usuff}$ and $\probs{\dsuff}$
		\\
		$\spp^{\udbacksuff} \subset \RR^n$ 
		&
		support of $\probs{\usuff}$ where $\probs{\dsuff}(X) = 0$
		\\
		$\spp^{\dubacksuff} \subset \RR^n$
		&
		support of $\probs{\dsuff}$ where $\probs{\usuff}(X) = 0$
		\\
		$M^{\usuff}\left[X,f_{\theta}(X)\right]: \RR^{n}\times\RR \rightarrow \RR$
		&
		\emph{force-magnitude} function that amplifies an \up push force\\
		&
		 which we apply at $X^{\usuff}$
		\\
		$M^{\dsuff}\left[X,f_{\theta}(X)\right]: \RR^{n}\times\RR \rightarrow \RR$
		&
		\emph{force-magnitude} function that amplifies a \down push force 
		\\
		&
		which we apply at $X^{\dsuff}$
		\\
		$R\left[X,f_{\theta}(X)\right]: \RR^{n}\times\RR \rightarrow \RR$
		&
		ratio function $\frac{M^{\dsuff}}{M^{\usuff}}$
		\\
		$T\left[ X, \frac{\probi{\usuff}{X}}{\probi{\dsuff}{X}} \right]: \RR^n \times \RR \rightarrow \RR$
		&
		convergence function satisfying $T \equiv R^{-1}$, describes
		\\
		&
		the modality that PSO optima $f^*(X)$ approximates
		\\
		$\KKK = (s_{min}, s_{max}) \subseteq \RR$
		&
		convergence interval defined as the range of $T\left[ X, z \right]$ w.r.t.
		\\
		&
		$z \in \RRpos$, represents a set of values $f^*$ can have, $f^*(X) \in \KKK$
		\\
		$\widetilde{M}^{\usuff}$ and $\widetilde{M}^{\dsuff}$
		&
		antiderivatives of $M^{\usuff}$ and $M^{\dsuff}$ (PSO \ps)
		\\
		$F_{\theta}^{\usuff}(X)$ and $F_{\theta}^{\dsuff}(X)$
		&
		point-wise \up and \down forces, that are applied (on average)
		\\
		&
		at any point $X \in \RR^{n}$
		\\
		$N^{\usuff}$ and $N^{\dsuff}$
		&
		batch sizes of samples from $\probs{\usuff}$ and from $\probs{\dsuff}$, that are used in\\
		&
		 a single optimization iteration
		\\
		$L_{PSO}: \FF \rightarrow \RR$
		&
		population \psofunc
		\\
		$\hat{L}_{PSO}^{N^{\usuff},N^{\dsuff}}: \FF \rightarrow \RR$
		&
		empirical \psofunc, approximates $L_{PSO}$ via training
		\\ 
		& 
		points $\{ X^{\usuff}_{i} \}_{i = 1}^{N^{\usuff}}$ and $\{ X^{\dsuff}_{i} \}_{i = 1}^{N^{\dsuff}}$
		\\
		$g_{\theta}(X, X')$
		&
		model kernel that is responsible for generalization and
		\\ &
		 interpolation during the GD optimization
		\\
		$r_{\theta}(X, X')$
		&
		relative model kernel, a scaled version 
		\\ &
		$r_{\theta}(X, X') = g_{\theta}(X, X') / g_{\theta}(X, X)$ of $g_{\theta}(X, X')$ whose properties
		\\ &
		can be used to analyze the bias-variance trade-off of PSO
		\\
		\bottomrule
	\end{tabular}
	
	\caption{Paper Main Notations}
	\label{tbl:Nottns}

\end{table}

%
%

In this section we formulate the definition of \emph{Probabilistic Surface Optimization} (PSO) algorithm framework. While in previous work \citep{Kopitkov18arxiv} we already explored a particular instance of PSO specifically for the problem of density estimation, in Section \ref{sec:PSOInst} we will see that PSO is actually a very general family of probabilistic inference algorithms that can solve various statistical tasks and that include a great number of existing methods.

\subsection{Formulation}
\label{sec:Formm}

Consider an input space $X \in \RR^n$ and two densities $\probs{\usuff}$ and $\probs{\dsuff}$ defined on it, with appropriate pdfs $\probi{\usuff}{X}$ and $\probi{\dsuff}{X}$ and with supports $\spp^{\usuff} \subset \RR^n$ and $\spp^{\dsuff} \subset \RR^n$; 
$U$ and $D$ denote the \up and \down directions of forces under a physical perspective of the optimization (see below). 
Denote by $\spp^{\udcapsuff}$, $\spp^{\udbacksuff}$ and $\spp^{\dubacksuff}$ sets $\{X: X \in \spp^{\usuff} \vee X \in \spp^{\dsuff} \}$, $\{X: X \in \spp^{\usuff} \vee X \notin \spp^{\dsuff} \}$ and $\{X: X \notin \spp^{\usuff} \vee X \in \spp^{\dsuff} \}$ respectively.
Further, denote a model $f_{\theta}(X): \RR^n \rightarrow \RR$ parametrized by the vector $\theta$ (e.g. NN or a function in Reproducing Kernel Hilbert Space, RKHS). Likewise, define two arbitrary functions $M^{\usuff}(X, s): \RR^n \times \RR \rightarrow \RR$ and $M^{\dsuff}(X, s): \RR^n \times \RR \rightarrow \RR$, which we name \mfs; both functions must be continuous \emph{almost everywhere} w.r.t. argument $s$ (see also Table \ref{tbl:Nottns} for list of main notations). We propose a novel PSO framework for  probabilistic inference, that performs a gradient-based iterative optimization $\theta_{t+1} = \theta_{t} - \delta \cdot d\theta$ with a learning rate $\delta$ where:
\begin{equation}
d\theta
=
-
\frac{1}{N^{\usuff}}
\sum_{i = 1}^{N^{\usuff}}
M^{\usuff}
\left[
X^{\usuff}_{i},
f_{\theta}(X^{\usuff}_{i})
\right]
\cdot
\nabla_{\theta} f_{\theta}(X^{\usuff}_{i})
+
\frac{1}{N^{\dsuff}}
\sum_{i = 1}^{N^{\dsuff}}
M^{\dsuff}
\left[
X^{\dsuff}_{i},
f_{\theta}(X^{\dsuff}_{i})
\right]
\cdot
\nabla_{\theta} f_{\theta}(X^{\dsuff}_{i})
.
\label{eq:GeneralPSOLossFrml}
\end{equation}
$\{ X^{\usuff}_{i} \}_{i = 1}^{N^{\usuff}}$ and $\{ X^{\dsuff}_{i} \}_{i = 1}^{N^{\dsuff}}$ are sample batches from $\probs{\usuff}$ and $\probs{\dsuff}$ respectively.
Each PSO instance is defined by a particular choice of \psocomp that produces a different convergence of $f_{\theta}(X)$ by approximately satisfying PSO \bp (within a mutual support $\spp^{\udcapsuff}$):
\begin{equation}
\frac{\probi{\usuff}{X}}{\probi{\dsuff}{X}}
=
\frac{M^{\dsuff}\left[X, f_{\theta^*}(X)\right]}{M^{\usuff}\left[X, f_{\theta^*}(X)\right]}
.
\label{eq:BalPoint}
\end{equation}
Such optimization, outlined in Algorithm \ref{alg:PSOAlgo}, will allow us to infer various statistical modalities from available data, making PSO a very useful and general algorithm family.

\begin{algorithm}[t]
	\caption{PSO estimation algorithm. Sample batches can be either identical or different for all iterations, which corresponds to GD and stochastic GD respectively.
	} 
	\label{alg:PSOAlgo} 
	\SetKwInput{Initialize}{Initialize}
	\SetKwBlock{AlgoBody}{begin:}{end}
	\SetKwInput{inputs}{Inputs}{}
	\SetKwInput{outputs}{Outputs}{}
	\inputs{
		
		$\probs{\usuff}$ and $\probs{\dsuff}$ : \up and \down densities
		
		$M^{\usuff}$ and $M^{\dsuff}$ : \mfs
		
		$\theta$ : initial parameters of model  $f_{\theta} \in \FF$
		
		$\delta$ : learning rate
	}
	\outputs{
		$f_{\theta^*}$ : PSO solution that satisfies \bp in Eq.~(\ref{eq:BalPoint}) }
	\BlankLine
	
	\AlgoBody{
		
		\While{Not converged}{
			Obtain samples $\{ X^{\usuff}_{i} \}_{i = 1}^{N^{\usuff}}$ from $\probs{\usuff}$
			
			Obtain samples $\{ X^{\dsuff}_{i} \}_{i = 1}^{N^{\dsuff}}$ from $\probs{\dsuff}$
			
			Calculate $d\theta$ via Eq.~(\ref{eq:GeneralPSOLossFrml})
			
			$\theta = \theta - \delta \cdot d\theta$
		}
		
		$\theta^* = \theta$
		
	}

	\BlankLine
\end{algorithm}

\subsection{Derivation}
\label{sec:Drv}

Consider \psofunc over a function $f: \RR^n \rightarrow \RR$ as:
\begin{equation}
L_{PSO}(f)
=
-
\E_{X \sim \probs{\usuff}}
\widetilde{M}^{\usuff}
\left[
X,
f(X)
\right]
+
\E_{X \sim \probs{\dsuff}}
\widetilde{M}^{\dsuff}
\left[
X,
f(X)
\right]
\label{eq:GeneralPSOLossFrmlPr_Limit_f}
\end{equation}
where we define $\widetilde{M}^{\usuff}
\left[
X,
s
\right]
\triangleq
\int_{s_0}^{s}
M^{\usuff}(X, t)
dt$ and 
$\widetilde{M}^{\dsuff}
\left[
X,
s
\right]
\triangleq
\int_{s_0}^{s}
M^{\dsuff}(X, t)
dt$ to be antiderivatives of $M^{\usuff}(\cdot)$ and $M^{\dsuff}(\cdot)$ respectively; these functions, referred below as \pfs of PSO, are not necessarily known analytically. The above integral can be separated into several terms related to $\spp^{\udcapsuff}$, $\spp^{\udbacksuff}$ and $\spp^{\dubacksuff}$. A minima $f^*$ of $L_{PSO}(f)$ is described below, characterizing $f^*$ within each of these areas.

\begin{theorem}[Variational Characterization]
\label{thrm:PSO} 
Consider densities $\probs{\usuff}$ and $\probs{\dsuff}$ and magnitudes $M^{\usuff}$ and $M^{\dsuff}$.
Define an arbitrary function $h: \spp^{\udbacksuff} \rightarrow \RR$.
Then for $f^*$ to be minima of $L_{PSO}$, it must fulfill the following properties:

	\begin{enumerate}
		\item \underline{Mutual support}: under "sufficient" conditions over \pair the $f^*$ must satisfy PSO balance state $\forall X \in \spp^{\udcapsuff}$: $\probi{\usuff}{X} \cdot M^{\usuff}\left[X, f^*(X)\right] = \probi{\dsuff}{X} \cdot M^{\dsuff}\left[X, f^*(X)\right]$.

		\item \underline{Disjoint support}: depending on properties of a function $M^{\usuff}$, it is necessary to satisfy $\forall X \in \spp^{\udbacksuff}$:
	\begin{enumerate}
		\item If $\forall s \in \RR: M^{\usuff}(X, s) > 0$, then $f^*(X) = \infty$.
		\item If $\forall s \in \RR: M^{\usuff}(X, s) < 0$, then $f^*(X) = - \infty$.
		\item If $\forall s \in \RR: M^{\usuff}(X, s) \equiv 0$, then $f^*(X)$ can be arbitrary.
		\item If $\forall s \in \RR:$ 
		\begin{equation}
		M^{\usuff}(X, s) 
		\rightarrow
		\begin{cases}
		= 0,& s = h(X)\\
		> 0,              & s < h(X)\\
		< 0,              & s > h(X)
		\end{cases}
		\end{equation}
		then $f^*(X) = h(X)$.
		\item Otherwise, additional analysis is required.
	\end{enumerate}
	\end{enumerate}
\end{theorem}

The theorem's proof, showing PSO \bp to be Euler-Lagrange equation of $L_{PSO}(f)$, is presented in Section \ref{sec:PSOFuncObj}. Sufficient conditions over \pair are likewise derived there. Part 2 helps to understand dynamics in areas outside of the mutual support, and its analogue for $\spp^{\dubacksuff}$ is stated in Section \ref{sec:FuncDisjointSupp_du}.
Yet, below we will mostly rely on part 1, considering the optimization in area $\spp^{\udcapsuff}$.
Following from the above, finding a minima of $L_{PSO}$ will produce a function $f$ that satisfies Eq.~(\ref{eq:BalPoint}).

To infer the above $f^*$, we consider a function space $\FF$, whose each element $f_{\theta}$ can be parametrized by $\theta$, and solve the problem $\min_{f_{\theta} \in \FF} L_{PSO}(f_{\theta})$. Assuming that $\FF$ contains $f^*$, it will be obtained as a minima of the above minimization. Further, in practice $L_{PSO}$ is optimized via gradient-based optimization where gradient w.r.t. $\theta$ is:
\begin{equation}
\nabla_{\theta}
L_{PSO}(f_{\theta})
=
-
\E_{X \sim \probs{\usuff}}
M^{\usuff}
\left[
X,
f_{\theta}(X)
\right]
\cdot
\nabla_{\theta} f_{\theta}(X)
+
\E_{X \sim \probs{\dsuff}}
M^{\dsuff}
\left[
X,
f_{\theta}(X)
\right]
\cdot
\nabla_{\theta} f_{\theta}(X)
\label{eq:GeneralPSOLossFrml_Limit}
\end{equation}
with Eq.~(\ref{eq:GeneralPSOLossFrml}) being its sampled approximation. Considering a hypothesis class represented by NN and the universal approximation theory \citep{Hornik91nn}, we assume that $\FF$ is rich enough to learn the optimal $f^*$ with high accuracy. 
Furthermore, in our experiments we  show that in practice NN-based PSO estimators approximate the PSO \bp in a very accurate manner.

\begin{remark} 
	PSO can be generalized into a functional gradient flow via the corresponding functional derivative of $L_{PSO}(f)$ in Eq.~(\ref{eq:GeneralPSOLossFrmlPr_Limit_f}). Yet, in this paper we will focus on GD formulation w.r.t. $\theta$ parametrization, outlined in Algorithm \ref{alg:PSOAlgo}, leaving more theoretically sophisticated form for future work.
	Algorithm \ref{alg:PSOAlgo} is easy to implement in practice, if for example $f_{\theta}$ is represented as NN.
	Furthermore, in case $\FF$ is RKHS, this algorithm can be performed by only evaluating RKHS's kernel at training points, by applying the kernel trick. The corresponding optimization algorithm is known as the kernel gradient descent. Further, in this paper we consider loss functions without an additional regularization term such as RKHS norm or weight decay, since a typical GD optimization is known to implicitly produce a regularization effect \citep{Ma19fcm}. Analysis of PSO combined with the explicit regularization term is likewise left for future work.
\end{remark}

\subsection{PSO Balance State}
\label{sec:PSOBal}

Given that $\probs{\usuff}$ and $\probs{\dsuff}$ have the same support, PSO will converge to PSO \bp in Eq.~(\ref{eq:BalPoint}). 
By ignoring possible singularities (due to an assumed identical support) we can see that the converged surface $f_{\theta^*}$ will be such that the ratio of frequency components will be opposite-proportional to the ratio of \emph{magnitude} components. 
To derive a value of the converged $f_{\theta^*}$ for a specific PSO instance, \pair of that PSO instance, which typically involve $f_{\theta}$ inside them,
must be substituted into Eq.~(\ref{eq:BalPoint}) and then it needs to be solved for $f_{\theta}$.
This is equivalent to finding inverse $T(X, z)$ of the ratio $R(X, s) \triangleq \frac{M^{\dsuff}(X, s)}{M^{\usuff}(X, s)}$ w.r.t. the second argument, $T \equiv R^{-1}$, with the convergence described as $f_{\theta^*}(X) = T\left[X, \frac{\probi{\usuff}{X}}{\probi{\dsuff}{X}}\right]$.
Such \bp can be used to mechanically recover many existing methods and to easily derive new ones for inference of numerous statistical functions of data; in Section \ref{sec:PSOInst} we provide full detail on this point. 
Furthermore, the above general formulation of PSO is surprisingly simple, considering that
it provides a strong intuition about its optimization dynamics as a physical system, as explained below.

\subsection{Virtual Surface Perspective}
\label{sec:Inttt}

The main advantage of PSO is in its conceptual simplicity, revealed when viewed via a physical perspective. Specifically, $f_{\theta}(X)$ can be considered as a virtual surface in $\RR^{n + 1}$ space, with its support being $\RR^{n}$, see Figure \ref{fig:SurfForces}. Further, according to the update rule of a gradient-descent (GD) optimization and Eq.~(\ref{eq:GeneralPSOLossFrml}), any particular training point $X$ updates $\theta$ during a single GD update by $\nabla_{\theta} f_{\theta}(X)$ magnified by the output of a \mf ($M^{\usuff}$ or $M^{\dsuff}$). Furthermore, considering the update of a simple form $\theta_{t + 1} = \theta_{t} + \nabla_{\theta} f_{\theta_{t}}(X)$, it is easy to show that the height of the surface at any other point $X'$ changes according to a first-order Taylor expansion as:
\begin{equation}
f_{\theta_{t + 1}}(X') - f_{\theta_{t}}(X')
\approx
\nabla_{\theta} f_{\theta_{t}}(X') ^T \cdot \nabla_{\theta} f_{\theta_{t}}(X)
.
\label{eq:DfffHeight}
\end{equation}
Hence, by pushing (optimizing) a specific training point $X$, the surface at other points changes according to the elasticity properties of the model expressed via a \gs kernel $g_{\theta}(X, X') \triangleq \nabla_{\theta} 
f_{\theta}(X)^T \cdot \nabla_{\theta} 
f_{\theta}(X')$. It helps also to think that during the above update we push at $X$ with a virtual rod, appeared inside Figure \ref{fig:SurfForces} in form of green and red arrows, whose head shape is described by $g_{\theta}(X, X')$.

When optimizing over RKHS, the above expression turns to be identity and $g_{\theta}(X, X')$ collapses into the reproducing kernel\footnote{
In RKHS defined via a feature map $\phi(X)$ and a reproducing kernel $k(X,X') = \phi(X)^T \cdot \phi(X')$, every function has a form $f_{\theta}(X) = \phi(X)^T \cdot \theta$. Since $\nabla_{\theta} f_{\theta}(X) = \phi(X)$, we obtain $g_{\theta}(X, X') \equiv k(X,X')$.}. 
For NNs, this model kernel is known as Neural Tangent Kernel (NTK) \citep{Jacot18nips}.
As was empirically observed, NTK has a strong local behavior with its outputs mostly being large only when points $X$ and $X'$ are close by. More insights about NTK can be found in \citep{Jacot18nips,Dou19arxiv,Kopitkov19arxiv_spectrum}. 
Further assuming for simplicity that $g_{\theta}$ has zero-bandwidth $\forall X \neq X': g_{\theta}(X, X') \equiv 0$ and that \pair are non-negative functions, it follows then that each $X^{\usuff}_{i}$ in Eq.~(\ref{eq:GeneralPSOLossFrml}) pushes the surface at this point \up by $g_{\theta}(X^{\usuff}_{i}, X^{\usuff}_{i})$ magnified by $M^{\usuff}
\left[
X^{\usuff}_{i},
f_{\theta}(X^{\usuff}_{i})
\right]$, whereas each $X^{\dsuff}_{i}$ is pushing it \down in a similar manner.

Considering a macro picture of such optimization, $f_{\theta}(X)$ is pushed \up at samples from $\probs{\usuff}$ and \down at samples from $\probs{\dsuff}$, with the \up and \down averaged point-wise forces being $F_{\theta}^{\usuff}(X) \triangleq \probi{\usuff}{X} \cdot 
M^{\usuff}
\left[
X,
f_{\theta}(X)
\right]$ and $F_{\theta}^{\dsuff}(X) \triangleq \probi{\dsuff}{X} \cdot 
M^{\dsuff}
\left[
X,
f_{\theta}(X)
\right]$ ($g_{\theta}(X, X)$ term is ignored since it is canceled out).
Intuitively, such a physical process converges to a stable state when point-wise forces become equal, $F_{\theta}^{\usuff}(X) = F_{\theta}^{\dsuff}(X)$. 
This is supported mathematically by the part 1 of Theorem \ref{thrm:PSO}, with such equilibrium being named as PSO \bp. Yet, it is important to note that this is only the variational equilibrium, which is obtained when training datasets are infinitely large and when the bandwidth of kernel $g_{\theta}$ is infinitely small. In practice the outcome of GD optimization \textbf{strongly} depends on the actual amount of available sample points and on various properties of $g_{\theta}$, with the model kernel serving as a metric over the function space $\FF$. Thus, the actual equilibrium somewhat deviates from PSO \bp, which we investigate in Section \ref{sec:ConvA}.

Additionally, the actual force direction at samples $X^{\usuff}_{i}$ and $X^{\dsuff}_{i}$ depends on signs of $M^{\usuff}
\left[
X^{\usuff}_{i},
f_{\theta}(X^{\usuff}_{i})
\right]$ and $M^{\dsuff}
\left[
X^{\usuff}_{i},
f_{\theta}(X^{\dsuff}_{i})
\right]$, and hence may be different for each instance of PSO. Nonetheless, in most cases the considered \mfs are non-negative, and thus support the above exposition exactly. Moreover, for negative \ms the above picture of physical forces will still stay correct, after swapping between \up and \down terms.

The physical equilibrium can also explain dynamics outside of the mutual support $\spp^{\udcapsuff}$. Considering the area $\spp^{\udbacksuff} \subset \RR^n$ where only samples $X^{\usuff}_{i}$ are located, when $M^{\usuff}$ has positive outputs, the model surface is pushed indefinitely \up, since there is no opposite force required for the equilibrium. Likewise, when $M^{\usuff}$'s outputs are negative - it is pushed indefinitely \down. Further, when \mf is zero, there are no pushes at all and hence the surface can be anything. Finally, if the output of $M^{\usuff}\left[
X,
f_{\theta}(X)
\right]
$ is changing signs depending on whether $f_{\theta}(X)$ is higher than $h(X)$, then $f_{\theta}(X)$ must be pushed towards $h(X)$ to balance the forces. Such intuition is supported mathematically by part 2 of Theorem \ref{thrm:PSO}.
Observe that convergence at infinity actually implies that PSO will not converge to the steady state for any number of GD iterations. Yet, it can be easily handled by for example limiting range of functions within the considered $\FF$ to some set $[a, b] \subset \RR$.

\begin{remark}
	In this paper we discuss PSO and its applications in context of only continuous multi-dimensional data, while in theory the same principles can work also for  discrete data. The sampled points $X^{\usuff}_{i}$ and $X^{\dsuff}_{i}$ will be located only at discrete locations of the surface $f_{\theta}(X)$ since the points are in $\ZZ^n \subset \RR^n$. Yet, the balance at each such point will still be governed by the same up and down forces. Thus, we can apply similar PSO methods to also infer statistical properties of discrete data.
\end{remark}

\section{PSO Functional}
\label{sec:PSOFuncObj}

Here we provide a detailed analysis of \psofunc $L_{PSO}(f)$, proving Theorem \ref{thrm:PSO} and deriving sufficient optimality conditions to ensure its optimization stability for any considered pair of $M^{\usuff}$ and $M^{\dsuff}$.
Examine $L_{PSO}$'s decomposition:
\begin{equation}
L_{PSO}(f)
=
L^{\udcapsuff}(f) +
L^{\udbacksuff}(f) +
L^{\dubacksuff}(f)
,
\label{eq:GeneralPSOLossFrmlPr_decomp}
\end{equation}
\begin{equation}
\label{eq:PSO_loss_f_separateee}
\begin{split}
L^{\udcapsuff}(f) 
& \triangleq
\int_{\spp^{\udcapsuff}}
-
\probi{\usuff}{X}
\cdot
\widetilde{M}^{\usuff}
\left[
X,
f(X)
\right]
+
\probi{\dsuff}{X}
\cdot
\widetilde{M}^{\dsuff}
\left[
X,
f(X)
\right]
dX
\\
L^{\udbacksuff}(f)
& \triangleq
-
\int_{\spp^{\udbacksuff}}
\probi{\usuff}{X}
\cdot
\widetilde{M}^{\usuff}
\left[
X,
f(X)
\right]
dX
\\
L^{\dubacksuff}(f)
& \triangleq
\int_{\spp^{\dubacksuff}}
\probi{\dsuff}{X}
\cdot
\widetilde{M}^{\dsuff}
\left[
X,
f(X)
\right]
dX
.
\end{split}
\end{equation}
In Section \ref{sec:FuncMutSupp} we analyze $L^{\udcapsuff}$, specifically addressing non-differentiable functionals (Section \ref{sec:FuncMutSupp_non_diff}), differentiable functionals (Section \ref{sec:FuncMutSupp_diff}) and deriving extra conditions for unlimited range of $\FF$ (Section \ref{sec:FuncMutSupp_unlimit_range}). Further, in Section \ref{sec:FuncDisjointSupp} we analyze $L^{\udbacksuff}$ and $L^{\dubacksuff}$, proving part 2 of Theorem \ref{thrm:PSO}.

\subsection{Mutual Support Optima}
\label{sec:FuncMutSupp}

Consider the loss term $L^{\udcapsuff}(f)$ corresponding to the area $\spp^{\udcapsuff}$,
where $\probi{\usuff}{X} > 0$ and $\probi{\dsuff}{X} > 0$. The Euler-Lagrange equation of this loss is $- \probi{\usuff}{X} \cdot M^{\usuff}\left[X, f(X)\right] + \probi{\dsuff}{X} \cdot M^{\dsuff}\left[X, f(X)\right] = 0$, thus yielding the conclusion that $L^{\udcapsuff}$'s minimization must lead to the convergence in Eq.~(\ref{eq:BalPoint}). Yet, calculus of variations does not easily produce the sufficient conditions that \pair must satisfy for such steady state. Instead, below we will apply notions of Legendre-Fenchel (LF) transform from the convex optimization theory.

\subsubsection{PSO Non-Differentiable Case}
\label{sec:FuncMutSupp_non_diff}

The core concepts required for the below derivation are properties of convex functions and their derivatives, and an inversion relation between (sub-)derivatives of two convex functions that are also convex-conjugate of each other. Each convex function $\widetilde{\varphi}(z)$ on an interval $(a,b)$ of real-line can be represented by its derivative $\varphi(z)$, with latter being increasing on $(a,b)$ with finitely many discontinuities (jumps). Each non-differentiable point $z_0$ of $\widetilde{\varphi}$ is expressed within $\varphi$ by a jump at $z_0$, and at each point where $\widetilde{\varphi}$ is not strictly convex the $\varphi$ is locally constant. Left-hand and right-hand derivatives $\widetilde{\varphi}D_{-}(z)$ and $\widetilde{\varphi}D_{+}(z)$ of $\widetilde{\varphi}$ 
\begin{equation}
\widetilde{\varphi}D_{-}(z_0)
=
\lim_{z \rightarrow z_0^{-}}
\frac{\widetilde{\varphi}(z) - \widetilde{\varphi}(z_0)}{z - z_0}
,
\quad
\widetilde{\varphi}D_{+}(z_0)
=
\lim_{z \rightarrow z_0^{+}}
\frac{\widetilde{\varphi}(z) - \widetilde{\varphi}(z_0)}{z - z_0}
\label{eq:sub_der_def}
\end{equation}
can be constructed from $\varphi$ by treating its finitely many discontinuities as left-continuities and right-continuities respectively. Further, $\widetilde{\varphi}$'s subderivative at $z$ is defined as $\partial \widetilde{\varphi}(z) = [\widetilde{\varphi}D_{-}(z), \widetilde{\varphi}D_{+}(z)]$. Note that $\widetilde{\varphi}$ can be recovered from any one-sided derivative by integration \citep[see Appendix C]{Pollard02book}, up to an additive constant which will not matter for our goals. Hence, each one of $\widetilde{\varphi}$, $\varphi$, $\widetilde{\varphi}D_{-}$, $\widetilde{\varphi}D_{+}$ and $\partial \widetilde{\varphi}$ is just a different representation of the same information.

The LF transform $\widetilde{\psi}$ of $\widetilde{\varphi}$ (also known as convex-conjugate of $\widetilde{\varphi}$) is a convex function defined as $\widetilde{\psi}(s) \triangleq \sup_{z \in \RR} \{ sz - \widetilde{\varphi}(z) \}$. Subderivatives $\partial \widetilde{\varphi}$ and $\partial \widetilde{\psi}$ have the following useful inverse relation \citep[see Proposition 11.3]{Rockafellar09book}: $\partial \widetilde{\psi}(s) = \{ z: s \in \partial \widetilde{\varphi}(z) \}$. Moreover, in case $\widetilde{\varphi}$ and $\widetilde{\psi}$ are strictly convex and differentiable, their derivatives $\varphi$ and $\psi$ are strictly increasing and continuous, and actually are inverse functions between $(a,b)$ and $(c,d)$, with $c \triangleq \inf_{z \in (a,b)} \varphi(z)$ and $d \triangleq \sup_{z \in (a,b)} \varphi(z)$. Further, since $\partial \widetilde{\psi}$ can be recovered from $\partial \widetilde{\varphi}$, it contains the same information and is just one additional representation form. See illustration in Figure \ref{fig:ConvexFuncsDemo} and refer to \citep{Zia09ajp} for a more intuitive exposition.

\begin{figure}
	\centering
	
	\begin{tabular}{cccc}
		
		\subfloat[\label{fig:ConvexFuncsDemo-a}]{\includegraphics[width=0.22\textwidth]{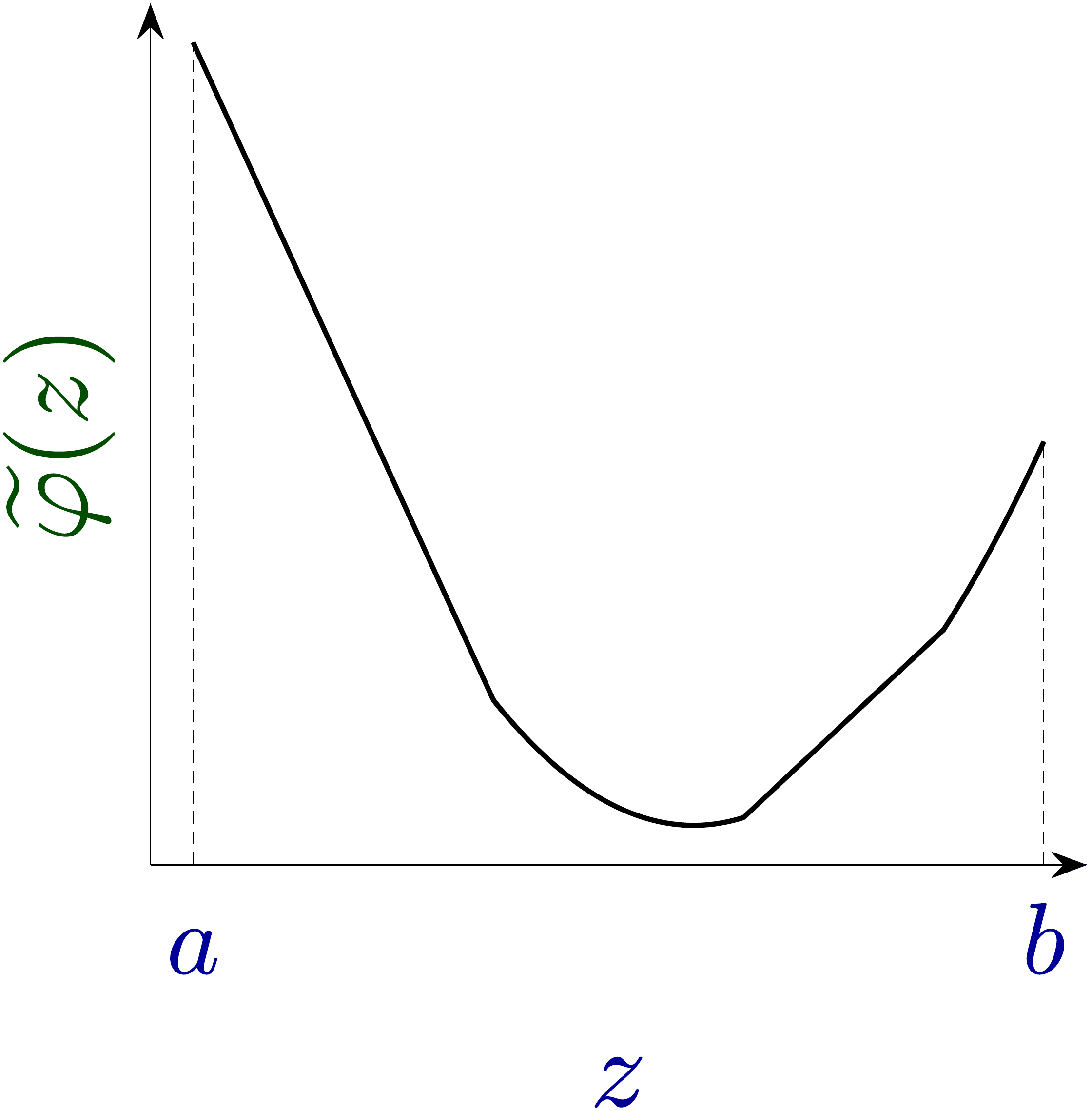}}
		&
		
		\subfloat[\label{fig:ConvexFuncsDemo-b}]{\includegraphics[width=0.22\textwidth]{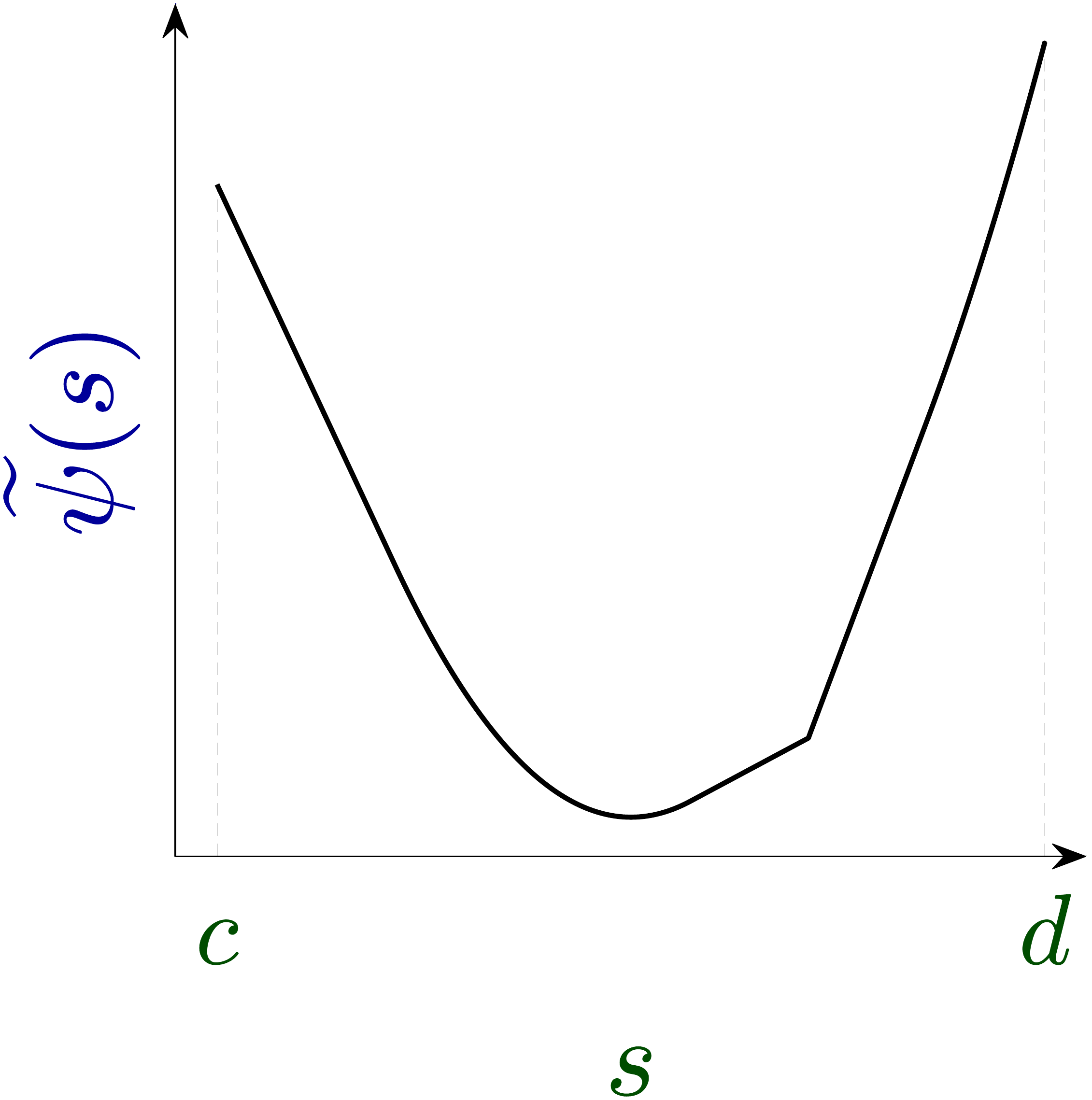}}
		&
		\subfloat[\label{fig:ConvexFuncsDemo-c}]{\includegraphics[width=0.22\textwidth]{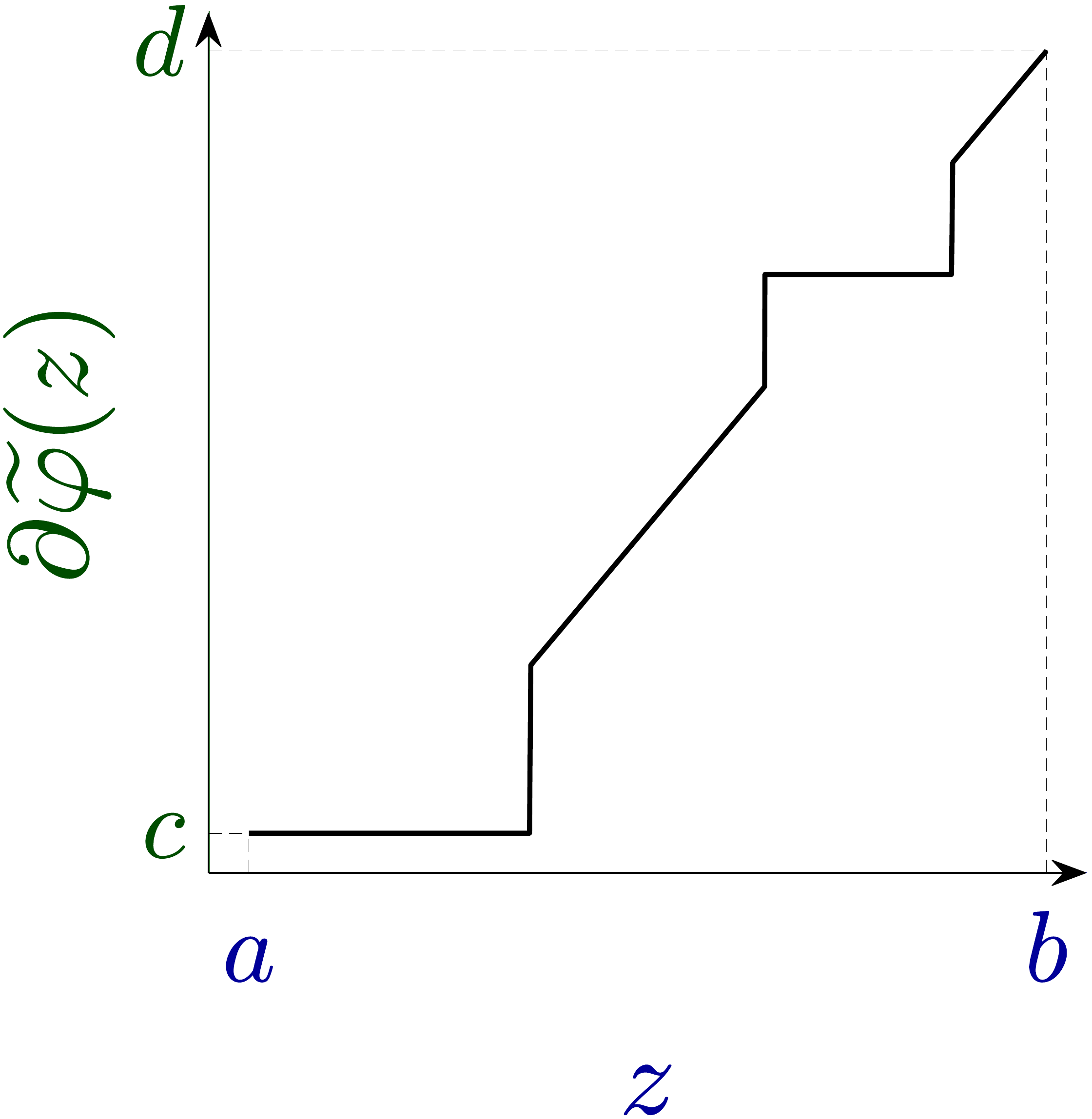}}
		&
		
		\subfloat[\label{fig:ConvexFuncsDemo-d}]{\includegraphics[width=0.22\textwidth]{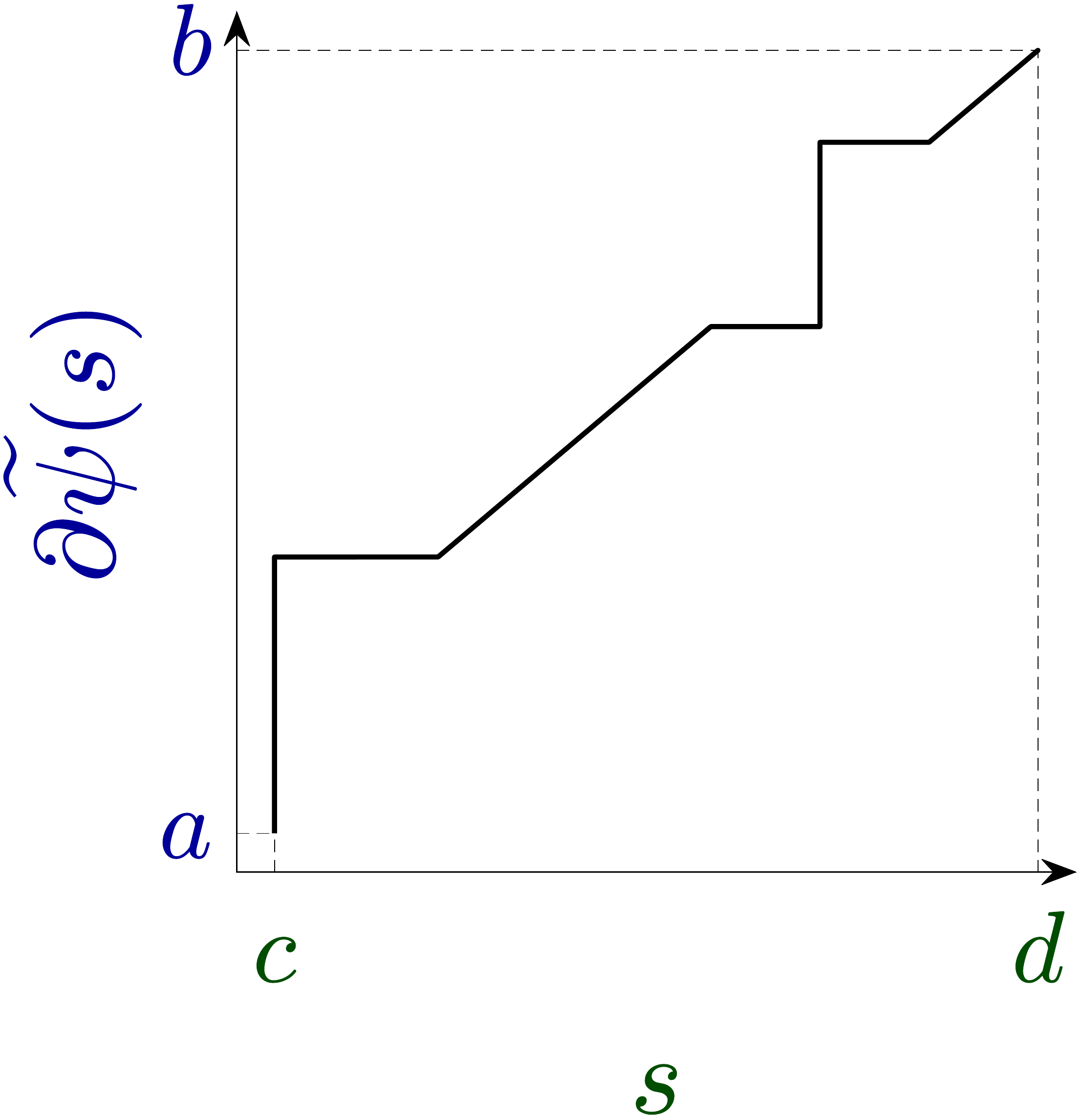}}
		
	\end{tabular}
	
	\protect
	\caption[Example of a convex function, its convex-conjugate and their subderivatives.]{(a) Example of a convex function $\widetilde{\varphi}(z)$ defined on an interval $(a, b)$, (b) its convex-conjugate $\widetilde{\psi}(s)$, (c) subderivative $\partial \widetilde{\varphi}(z)$ and (d) subderivative $\partial \widetilde{\psi}(s)$. The $\partial \widetilde{\varphi}(z)$ has jumps at points in $(a, b)$ where $\widetilde{\varphi}(z)$ is non-differentiable, and is constant at points where $\widetilde{\varphi}(z)$ is not strictly convex. Further, $\partial \widetilde{\varphi}(z)$ and $\partial \widetilde{\psi}(s)$ are inverse mappings between $(a, b)$ and $(c, d)$, with constant regions in $\partial \widetilde{\varphi}(z)$ corresponding to jumps in $\partial \widetilde{\psi}(s)$, and vice versa. Additionally, any of the above four functions can be recovered from any other, thus each of them is a different representation of the same information. Observe that a convex function $\widetilde{\psi}(s)$ can be recovered from its subderivative $\partial \widetilde{\psi}(s)$ only up to an additive constant. Yet, this constant will not affect optima $s^*$ of the considered below optimization $\min_{s \in (c,d)} \widetilde{\psi}(s) - z \cdot s$ and hence can be ignored. Furthermore, codomains of $\widetilde{\varphi}(z)$ and of $\widetilde{\psi}(s)$ do not play any role in our derivation.
	}
	\label{fig:ConvexFuncsDemo}
\end{figure}

The above inverse relation can be used in optimization by applying the Fenchel-Young inequality: for any $z \in (a,b)$ and $s \in (c,d)$ we have $\widetilde{\varphi}(z) + \widetilde{\psi}(s) - z \cdot s \geq 0$, with equality obtained iff $s \in \partial \widetilde{\varphi}(z)$. Thus, for any given $z \in (a,b)$ the optima $s^* = \argmin_{s \in (c,d)} \widetilde{\psi}(s) - z \cdot s$ must be within $\partial \widetilde{\varphi}(z)$. It is helpful to identify a role of each term within the above optimization problem. The $\widetilde{\psi}$ serves as part of the cost, $\partial \widetilde{\varphi}$ defines the optima $s^*$, $\partial \widetilde{\psi}$ is required to solve this optimization in practice (i.e. via subgradient descent), and $\widetilde{\psi}$ is not actually used. Thus, in order to define properties we want $s^*$ to have and to perform the actual optimization, we only need to know $\partial \widetilde{\varphi}$ and $\partial \widetilde{\psi}$, with the latter being easily recovered from the former. For this reason, given an increasing function $\varphi(z)$ with domain $(a,b)$ and codomain (c,d), in practice it is sufficient to know its inverse $\psi(s)$ (or the corresponding subderivative $\partial \widetilde{\psi}$) to solve the optimization and to obtain the optima $s^* \in (c,d)$ s.t. $s^* = \varphi(z)$ (or $s^* \in \partial \widetilde{\varphi}(z)$). Convex functions $\widetilde{\varphi}$ and $\widetilde{\psi}$ can be used symbolically for a math proof, yet their actual form is not required, which was also noted in \citep{Reid10jmlr}. This idea may seem pointless since to find $s^*$ we can compute $\varphi(z)$ in the first place, yet it will help us in construction of a general optimization framework for probabilistic inference.

In following statements we define several functions with two arguments, where the first argument $X$ can be considered as a "spectator" and where all declared functional properties are w.r.t. the second argument for any value of the first one.
Define the required estimation convergence by a transformation $T(X, z): \RR^n \times \RRpos \rightarrow \RR$. Given that $T$ is increasing and right-continuous (w.r.t. $z \in \RRpos$ at any $X \in \RR^n$), below we will derive a new objective functional whose minima is $f^*(X) = T\left[X, \frac{\probi{\usuff}{X}}{\probi{\dsuff}{X}}\right]$ and which will have a form of $L^{\udcapsuff}$. Furthermore, this derivation will yield the sufficient conditions over PSO \ms.

Consider any fixed value of $X \in \RR^n$. Denote by $\KKK = (s_{min}, s_{max})$ the effective convergence interval, with $s_{min} = \inf_{z \in \RRpos} T(X, z)$ and $s_{max} = \sup_{z \in \RRpos} T(X, z)$; at the convergence we will have $f^*(X) \in \KKK$. Further, below we will assume that the effective interval $\KKK$ is identical for any $X$, which is satisfied by all convergence transformations $T$ considered in our work. It can be viewed as an assumption that knowing value of $X$ without knowing value of $\frac{\probi{\usuff}{X}}{\probi{\dsuff}{X}}$ does not yield any information about the convergence $T\left[X, \frac{\probi{\usuff}{X}}{\probi{\dsuff}{X}}\right]$.

Due to its properties, $T(X, z): \RR^n \times \RRpos \rightarrow \KKK$ can be acknowledged as a right-hand derivative of some convex function $\widetilde{T}(X, z)$ (w.r.t. $z$). Denote by $\widetilde{T}D_{-}(X, z)$, $\widetilde{T}D_{+}(X, z) \equiv T(X, z)$ and $\partial \widetilde{T}(X, z)$ left-hand derivative, right-hand derivative and subderivative of $\widetilde{T}$.

Next, define a mapping $G
(
X,
s
)$ to be a strictly increasing and right-continuous function on $s \in \KKK$, with $\bar{\KKK} = \{ G
\left[
X,
s
\right] : s \in \KKK \}$ being 
an image of $\KKK$ under $G$. 
Set $\bar{\KKK}$ may depend on value of $X$, although it will not affect below conclusions.
We denote the left-inverse of $G(X, s): \RR^n \times \KKK \rightarrow \bar{\KKK}$ as $G^{-1}(X, t): \RR^n \times \bar{\KKK} \rightarrow \KKK$ s.t. $\forall s \in \KKK: G^{-1}(X, G(X, s)) = s$.

Define a mapping
$\varPhi(X, z) \triangleq G(X, T(X, z)): \RR^n \times \RRpos \rightarrow \bar{\KKK}$ and note that it is increasing (composition of two increasing functions is increasing) and right-continuous (right-continuity of $G$ preserves all limits required for right-continuity of $\varPhi$). Similarly to $T$, define $\widetilde{\varPhi}(X, z)$, $\widetilde{\varPhi}D_{-}(X, z) \equiv G(X, \widetilde{T}D_{-}(X, z))$, $\widetilde{\varPhi}D_{+}(X, z) \equiv G(X, \widetilde{T}D_{+}(X, z))$ and $\partial \widetilde{\varPhi}(X, z)$ to be the corresponding convex function, its left-hand derivative, right-hand derivative and subderivative respectively.

Denote by $\widetilde{\Psi}(X, t)$ the LF transform of $\widetilde{\varPhi}(X, z)$ w.r.t. $z$. Its subderivative at any $t \in \bar{\KKK}$ is $\partial \widetilde{\Psi}(X, t) = \{ z: G^{-1}(t) \in \partial \widetilde{T}(X, z) \} \subset \RRpos$.
According to the Fenchel-Young inequality, for any given $z \in \RRpos$ the optima $t^*$ of $\min_{t \in \bar{\KKK}} \widetilde{\Psi}(X, t)  - t \cdot z$ must be within $\partial \widetilde{\varPhi}(X, z)$. Further, this optimization can be rewritten as $\min_{G(X, s): s \in \KKK} \widetilde{\Psi}(X, G(X, s))  - G(X, s) \cdot z$ with its solution satisfying $s^*: G(X, s^*) \in \partial \widetilde{\varPhi}(X, z)$, or:
\begin{equation}
s^* = \argmin_{s \in \KKK} 
- z \cdot G(X, s)
+
\widetilde{\Psi}(X, G(X, s))  
,
\label{eq:fenchel_1}
\end{equation}
\begin{equation}
G(X, s^*) \in [\widetilde{\varPhi}D_{-}(X, z), \widetilde{\varPhi}D_{+}(X, z)]
\Rightarrow
s^* \in [\widetilde{T}D_{-}(X, z), \widetilde{T}D_{+}(X, z)]
\Rightarrow
s^* \in \partial \widetilde{T}(X, z)
\subset
\KKK
,
\label{eq:fenchel_2}
\end{equation}
where we applied the left-inverse $G^{-1}$. The above statements are true for any considered $X \in \RR^n$, with the convergence interval $\KKK$ being independent of $X$'s value.
Additionally, while the above problem is not necessarily convex in $s$ (actually it is easily 
proved to be quasiconvex), it still has a well-defined minima $s^*$. Further, various methods can be applied if needed to solve this nonconvex nonsmooth optimization using $\partial \widetilde{\Psi}$ and various notions of $G$'s subdifferential \citep{Bagirov13jota}.

Substituting $z \equiv \frac{\probi{\usuff}{X}}{\probi{\dsuff}{X}}$, $G(X, s) \equiv \widetilde{M}^{\usuff}(X, s)$ and $\widetilde{\Psi}(X, G(X, s)) \equiv \widetilde{M}^{\dsuff}(X, s)$ we get:
\begin{multline}
s^* = \argmin_{s \in \KKK} 
- \frac{\probi{\usuff}{X}}{\probi{\dsuff}{X}} \cdot \widetilde{M}^{\usuff}(X, s)
+
\widetilde{M}^{\dsuff}(X, s)
=\\
=
\arginf_{s \in \KKK} 
- 
\probi{\usuff}{X}
\cdot \widetilde{M}^{\usuff}(X, s)
+
\probi{\dsuff}{X}
\cdot
\widetilde{M}^{\dsuff}(X, s)
,
\label{eq:fenchel_3}
\end{multline}
where we replaced minimum with infimum for the latter use.
Optima $s^*$ is equal to $T\left[X, \frac{\probi{\usuff}{X}}{\probi{\dsuff}{X}}\right]$ if $T$ is continuous at $z = \frac{\probi{\usuff}{X}}{\probi{\dsuff}{X}}$, or must be within $[\widetilde{T}D_{-}(X, \frac{\probi{\usuff}{X}}{\probi{\dsuff}{X}}), \widetilde{T}D_{+}(X, \frac{\probi{\usuff}{X}}{\probi{\dsuff}{X}})]$ otherwise.

Next, denote $\FF$ to be a function space with measurable functions $f: \spp^{\udcapsuff} \rightarrow \KKK$ w.r.t. a base measure $dX$. Then, the optimization problem $\inf_{f \in \FF} L^{\udcapsuff}(f)$ solves the problem in Eq.~(\ref{eq:fenchel_3}) for $f(X)$ at each $X \in \spp^{\udcapsuff}$: 
\begin{equation}
\inf_{f \in \FF}
L^{\udcapsuff}(f) 
=
\int_{\spp^{\udcapsuff}}
\inf_{f(X)}
\left[
-
\probi{\usuff}{X}
\cdot
\widetilde{M}^{\usuff}
\left[
X,
f(X)
\right]
+
\probi{\dsuff}{X}
\cdot
\widetilde{M}^{\dsuff}
\left[
X,
f(X)
\right]
\right]
dX
,
\label{eq:fenchel_4}
\end{equation}
where we can move infimum into the integral since $f$ is measurable, the argument used also in works \citep{Nguyen10tit, Nowozin16nips}. Thus, the solution $f^* = \arginf_{f \in \FF} L^{\udcapsuff}(f)$ must satisfy $\forall X \in \spp^{\udcapsuff}: f^*(X) \in \partial \widetilde{T}\left[X, \frac{\probi{\usuff}{X}}{\probi{\dsuff}{X}}\right]$, given that $f^* \in \FF$.

Finally, we summarize all conditions that are sufficient for the above conclusion: function $T(X, z)$ is increasing and right-continuous w.r.t. $z \in \RRpos$, the convergence interval $\KKK$ is $X$-invariant, $G(X, s)$ (also aliased as $\widetilde{M}^{\usuff}
(
X,
s
)$) is right-continuous and strictly increasing w.r.t. $s \in \KKK$, and $\FF$'s range is $\KKK$. Given $T$, $\KKK$, $G$ and $\FF$ have these properties, the entire above derivation follows.

\subsubsection{PSO Differentiable Case}
\label{sec:FuncMutSupp_diff}

More "nice" results can be obtained if we assume additionally $T(X, z): \RR^n \times \RRpos \rightarrow \KKK$ to be strictly increasing and continuous w.r.t. $z \in \RRpos$, and $G(X, s): \RR^n \times \KKK \rightarrow \bar{\KKK}$ to be differentiable at $s \in \KKK$. This is the main setting on which our work is focused.

In such case $T$ is invertible. Denote its inverse as $R(X, s): \RR^n \times \KKK \rightarrow \RRpos$. The $R$ is strictly increasing and continuous w.r.t. $s \in \KKK$, and satisfies $\forall z \in \RRpos: R\left[X, T\left[X, z\right]\right] = z$ and $\forall s \in \KKK: T\left[X, R\left[X, s\right]\right] = s$.

Further, the derivative of $\widetilde{\varPhi}(X, z)$ is $\varPhi(X, z) = G(X, T(X, z))$. The $\varPhi$ is strictly increasing and continuous w.r.t. $z \in \RRpos$, and thus $\widetilde{\varPhi}$ is strictly convex and differentiable on $\RRpos$. 

By LF transform's rules the derivative $\Psi(X, z)$ of $\widetilde{\Psi}$ is an inverse of $\widetilde{\varPhi}$'s derivative $\varPhi(X, z)$, and thus it can be expressed as $\Psi(X, z) = R(X, G^{-1}(X, z))$. This leads to $\frac{\partial \widetilde{\Psi}(X, G(X, s))}{\partial s} = R(X, s) \cdot G'(X, s)$ where $G'(X, s) \triangleq \frac{\partial G(X, s)}{\partial s}$ is the derivative of $G$.

From above we can conclude that $\widetilde{M}^{\usuff}(X, s) \equiv G(X, s)$ and $\widetilde{M}^{\dsuff}(X, s) \equiv \widetilde{\Psi}(X, G(X, s))$ are both differentiable at $s \in \KKK$, with derivatives $M^{\usuff}(X, s) = G'(X, s)$ and $M^{\dsuff}(X, s) = R(X, s) \cdot G'(X, s)$ satisfying $\frac{M^{\dsuff}(X, s)}{M^{\usuff}(X, s)} = R(X, s)$. Functions $M^{\usuff}$ and $M^{\dsuff}$ can be considered as \ms of physical forces, as explained in Section \ref{sec:Inttt}. Also, $M^{\usuff}(X, s) > 0$ for any $s \in \KKK$ due to properties of $G$. Observe that we likewise have $R(X, s) > 0$ for any $s \in \KKK$ since $R(X, s) \in \RRpos$. This leads to $M^{\dsuff}(X, s) > 0$ at any $s \in \KKK$, which implies $\widetilde{M}^{\dsuff}(X, s)$ to be strictly increasing at $s \in \KKK$ (similarly to $\widetilde{M}^{\usuff}$). Moreover, additionally taken assumptions will enforce the solution $f^* = \arginf_{f \in \FF} L^{\udcapsuff}(f)$ to satisfy $\forall X \in \spp^{\udcapsuff}: f^*(X) = T\left[X, \frac{\probi{\usuff}{X}}{\probi{\dsuff}{X}}\right]$.

Above we derived a new objective function $L^{\udcapsuff}(f)$. Given that terms $\{ T, \KKK, G, \FF \}$ satisfy the declared above conditions, minima $f^*$ of $L^{\udcapsuff}$ will be $T\left[X, \frac{\probi{\usuff}{X}}{\probi{\dsuff}{X}}\right]$. Properties of $\{ \widetilde{M}^{\usuff}, \widetilde{M}^{\dsuff}, M^{\usuff}, M^{\dsuff}\}$ follow from aforementioned conditions. This is summarized by below theorem, where instead of $G$ we enforce the corresponding requirements over \pair. See also Figure \ref{fig:TR_examples} for an illustrative example.

\begin{figure}
	\centering
	
	\begin{tabular}{c}
		
		\subfloat{ \includegraphics[width=0.7\textwidth]{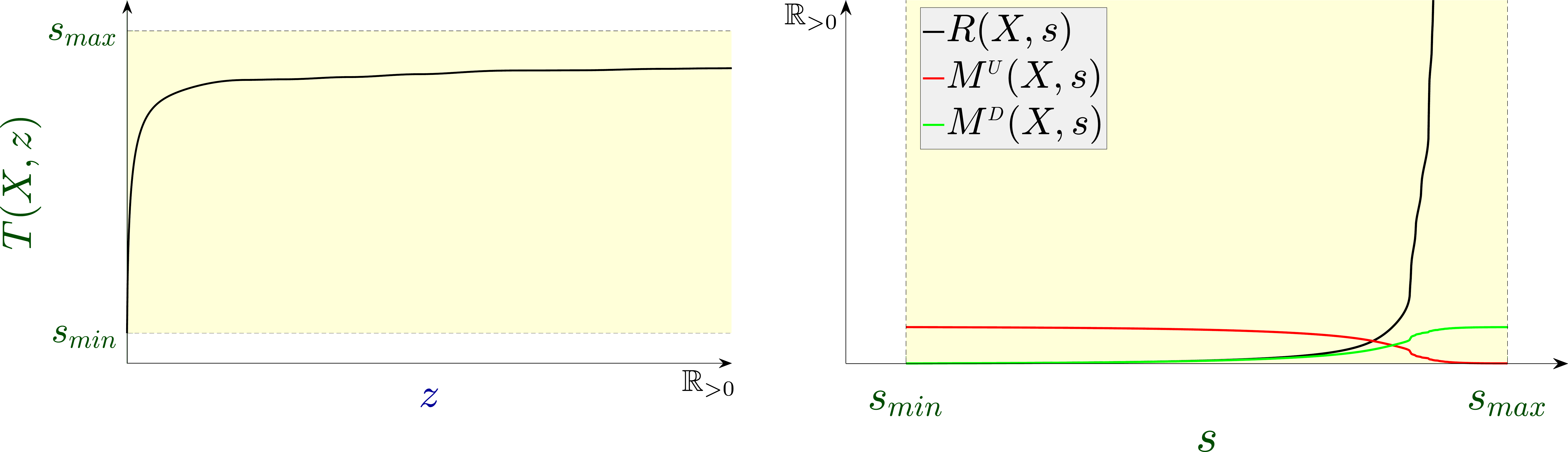}}

	\end{tabular}
	
	\protect
	\caption[Summary of requirements over PSO mappings.]{Summary of requirements over mappings $T(X, z)$, $R(X, s)$, $M^{\usuff}(X, s)$ and $M^{\dsuff}(X, s)$. For any $X \in \RR^{n}$, $T$ must be strictly increasing and continuous on $z \in \RRpos$, with its image $\KKK = (s_{min}, s_{max})$ marked by yellow. The inverse $R$ of $T$ is likewise strictly increasing and continuous function from $\KKK$ to $\RRpos$. Further, $M^{\usuff}$ and $M^{\dsuff}$ are any two functions that are positive and continuous on $\KKK$ with $\frac{M^{\dsuff}}{M^{\usuff}} = R$.
	}
	\label{fig:TR_examples}
\end{figure}

\begin{theorem}[Convergence-Focused]
\label{thrm:PSO_func_conv} 
Consider mappings $T(X, z)$, $M^{\usuff}(X, s)$ and $M^{\dsuff}(X, s)$. Assume:
	\begin{enumerate}
		\item $T(X, z)$ is strictly increasing and continuous w.r.t. $z \in \RRpos$, with its inverse denoted by $R(X, s)$.
		\item The convergence interval $\KKK \triangleq \{ T(X, z) | z \in \RRpos \}$ (image of $\RRpos$ under $T$) is $X$-invariant.
		\item $M^{\usuff}(X, s)$ and $M^{\dsuff}(X, s)$ are continuous and positive at $s \in \KKK$, satisfying $\frac{M^{\dsuff}(X, s)}{M^{\usuff}(X, s)} = R(X, s)$.
		\item Range of $\FF$ is $\KKK$.
	\end{enumerate}
Denote $\widetilde{M}^{\usuff}(X, s)$ and $\widetilde{M}^{\dsuff}(X, s)$ to be antiderivatives of $M^{\usuff}(X, s)$ and $M^{\dsuff}(X, s)$ at $s \in \KKK$, and construct the corresponding functional $L^{\udcapsuff}$.
Then, the minima $f^* = \arginf_{f \in \FF} L^{\udcapsuff}(f)$ will satisfy $\forall X \in \spp^{\udcapsuff}: f^*(X) = T\left[X, \frac{\probi{\usuff}{X}}{\probi{\dsuff}{X}}\right]$.
\end{theorem}

Continuity of \pair in condition 3 is sufficient for existence of antiderivatives \widepair. It is a little too strong criteria for integrability, yet it is more convenient to verify in practice. This leads to differentiability of $\widetilde{M}^{\usuff}$, which in turn implies continuity and differentiability of $G$; positiveness of $M^{\usuff}$ in condition 3 implies $G$ to be strictly increasing on $\KKK$. Conditions 2 and 4 restate assumptions over $\KKK$ and $\FF$. Therefore, the sufficient conditions over $\{ T, \KKK, G, \FF \}$ follow from the above list, which leads to the required $f^*$.

Given any required convergence $T$, the above theorem can be applied to propose valid \ms \pair. This basically comes to requiring \mfs to be continuous and positive on $\KKK$, with their ratio being inverse of $T$. Once such pair of functions is obtained, the loss $L^{\udcapsuff}$ with corresponding optima, and more importantly its gradient, can be easily constructed. Observe that knowledge of \widepair is not necessary neither for condition verification nor for the optimization of $L^{\udcapsuff}$.

Further, given any $L^{\udcapsuff}$ with corresponding \pair, its convergence and sufficient conditions can be verified via Theorem \ref{thrm:PSO_func_magn}.

\begin{theorem}[Magnitudes-Focused]
\label{thrm:PSO_func_magn} 
Consider a functional $L^{\udcapsuff}$ with $\widetilde{M}^{\usuff}(X, s)$ and $\widetilde{M}^{\dsuff}(X, s)$ whose derivatives are $M^{\usuff}(X, s)$ and $M^{\dsuff}(X, s)$. Denote $R(X, s)$ to be the ratio $\frac{M^{\dsuff}(X, s)}{M^{\usuff}(X, s)}$, and define a convergence interval as $\KKK \triangleq (s_{min}, s_{max})$. Assume:
	\begin{enumerate}
		\item $R(X, s): \RR^n \times \KKK \rightarrow \RRpos$ is continuous, strictly increasing and bijective w.r.t. domain $s \in \KKK$ and codomain $\RRpos$, for any $X \in \RR^n$.
		\item $M^{\usuff}(X, s)$ and $M^{\dsuff}(X, s)$ are continuous and positive at $s \in \KKK$.
		\item Range of $\FF$ is $\KKK$.
	\end{enumerate}
	Then, the minima $f^* = \arginf_{f \in \FF} L^{\udcapsuff}(f)$ will satisfy $\forall X \in \spp^{\udcapsuff}: f^*(X) = T\left[X, \frac{\probi{\usuff}{X}}{\probi{\dsuff}{X}}\right]$, where $T \triangleq R^{-1}$.
\end{theorem}

Condition sets in Theorem \ref{thrm:PSO_func_conv} and Theorem \ref{thrm:PSO_func_magn} are identical. Condition 1 of Theorem \ref{thrm:PSO_func_magn} is required for $T(X, z)$ to be strictly increasing, continuous and well-defined for each $z \in \RRpos$.
$\KKK$ can be any interval as long as conditions of Theorem \ref{thrm:PSO_func_magn} are satisfied, yet typically it is a preimage $\{ s \in \RR | R(X, s) \in \RRpos \}$ of $\RRpos$ under $R$. See examples in Section \ref{sec:PSOInstDeriveConv}. Likewise, observe again that knowledge of PSO \ps \widepair is not required.

Below we will use Theorem \ref{thrm:PSO_func_conv} to derive valid \pair for any considered $T$, and Theorem \ref{thrm:PSO_func_magn} to derive $T$ for any considered \pair.
Further, part 1 of Theorem \ref{thrm:PSO} follows trivially from the above statements.
Moreover, due to symmetry between \up and \down terms we can also have $T$ and $R$ to be strictly decreasing functions given $M^{\usuff}(X, s)$ and $M^{\dsuff}(X, s)$ are negative at $s \in \KKK$. Furthermore, many objective functions satisfy the above theorems and thus can be recovered via PSO framework. Estimation methods for which the sufficient conditions do not hold include a hinge loss from the binary classification domain as also other threshold losses \citep{Nguyen09aos}. Yet, these losses can be shown to be included within PSO non-differentiable case in Section \ref{sec:FuncMutSupp_non_diff}, whose analysis we leave for a future work.

\subsubsection{Unlimited Range Conditions}
\label{sec:FuncMutSupp_unlimit_range}

Criteria 4 of Theorem \ref{thrm:PSO_func_conv} and 3 of Theorem \ref{thrm:PSO_func_magn} can be replaced by additional conditions over \pair. These derived below conditions are 
very often satisfied, which allows us to not restrict $\FF$'s range in practice.

Recalling that $\KKK$ is an open interval $(s_{min}, s_{max})$, consider following sets $\KKK^{-}  = \{ s | s \leq s_{min} \}$ and $\KKK^{+}  = \{ s | s \geq s_{max} \}$. Observe that if $s_{min} = - \infty$ then $\KKK^{-}$ is an empty set, and if $s_{max} = \infty$ - $\KKK^{+}$ is empty. Further, $\KKK$, $\KKK^{-}$ and $\KKK^{+}$ are disjoint sets.

To reduce limitation over $\FF$'s range, it is enough to demand the inner optimization problem in Eq.~(\ref{eq:fenchel_3}) to be strictly decreasing on $\KKK^{-}$ and strictly increasing on $\KKK^{+}$. To this purpose, first we require \pair to be well-defined on the entire real line $s \in \RR$.
This can be achieved by restricting $\widetilde{M}^{\usuff}(X, s)$ and $\widetilde{M}^{\dsuff}(X, s)$ to be differentiable on $s \in \RR$ - it is sufficient for $M^{\usuff}(X, s)$ and $M^{\dsuff}(X, s)$ to be well-defined on $\RR$. Alternatively, we may and will require $M^{\usuff}(X, s)$ and $M^{\dsuff}(X, s)$ to be continuous at any $s \in \RR$. This slightly stronger condition will ensure that \pair are well-defined and that the antiderivatives \widepair exist on $\RR$. Moreover, such condition is imposed over \pair, allowing to neglect properties of \widepair.

Further, in case $\KKK^{+}$ is not empty, we require 
$- z \cdot \widetilde{M}^{\usuff}(X, s)
+
\widetilde{M}^{\dsuff}(X, s)$ to be strictly increasing for any $s \in \KKK^{+}$ and any $z \in \RRpos$. Given that $\widetilde{M}^{\usuff}(X, s)$ and $\widetilde{M}^{\dsuff}(X, s)$ are differentiable at $s \in \KKK \cup \KKK^{+}$ (which also implies their continuity at $s_{max}$), this requirement holds iff $\forall s \in \KKK^{+}, z \in \RRpos: M^{\dsuff}(X, s) > z \cdot M^{\usuff}(X, s)$. Verifying all possible cases,
the above criteria is satisfied iff: $\Big[ \forall s \in \KKK^{+}: [ M^{\usuff}(X, s) = 0 \vee M^{\dsuff}(X, s) > 0 ] \wedge [ M^{\usuff}(X, s) < 0 \vee M^{\dsuff}(X, s) \geq 0 ] \Big]$. This can be compactly written as $\forall s \in \KKK^{+}: M^{\usuff}(X, s) < M^{\dsuff}(X, s) \vee M^{\usuff}(X, s) \cdot M^{\dsuff}(X, s) \leq 0$, where the second condition implies that \ms $M^{\usuff}$ and $M^{\dsuff}$ can not have the same sign within $\KKK^{+}$.

Similar derivation for $\KKK^{-}$ will lead to demand the inner problem to be strictly decreasing for any $s \in \KKK^{-}$ and any $z \in \RRpos$. In turn, this leads to criteria $\forall s \in \KKK^{-}: M^{\usuff}(X, s) > M^{\dsuff}(X, s) \vee M^{\usuff}(X, s) \cdot M^{\dsuff}(X, s) \leq 0$. Below we summarize conditions under which no restrictions are required over $\FF$'s range.

\begin{theorem}[Unconstrained Function Range]
\label{thrm:PSO_func_magn_unconstr}
Consider the convergence interval $\KKK = (s_{min}, s_{max})$.
Assume:
	\begin{enumerate}
		\item $M^{\usuff}(X, s)$ and $M^{\dsuff}(X, s)$ are continuous on the entire real line $s \in \RR$, and do not have the same sign outside of $\KKK$.
		\item For any $s \leq s_{min}$: $M^{\usuff}(X, s) > M^{\dsuff}(X, s)$.
		\item For any $s \geq s_{max}$: $M^{\usuff}(X, s) < M^{\dsuff}(X, s)$.
	\end{enumerate}
Then, the condition 4 of Theorem \ref{thrm:PSO_func_conv} and the condition 3 of Theorem \ref{thrm:PSO_func_magn} can be removed.
\end{theorem}

Intuitively, conditions for $\KKK^{+}$ (and similarly for $\KKK^{-}$) can be interpreted as requiring \up force $F_{\theta}^{\usuff}(X)$ to be weaker than \down force $F_{\theta}^{\dsuff}(X)$ for any ratio $\frac{\probi{\usuff}{X}}{\probi{\dsuff}{X}} > 0$, once the surface height $f_{\theta}(X)$ got too high and exceeded the convergence interval $\KKK$. In Section \ref{sec:PSOInst} we will see that almost all PSO estimators satisfy Theorem \ref{thrm:PSO_func_magn_unconstr}.

\subsection{Disjoint Support Optima}
\label{sec:FuncDisjointSupp}

\subsubsection{Area $\spp^{\udbacksuff}$}
\label{sec:FuncDisjointSupp_ud}

Consider the loss term $L^{\udbacksuff}(f)$ corresponding to the area $\spp^{\udbacksuff}$,
where $\probi{\usuff}{X} > 0$ and $\probi{\dsuff}{X} = 0$.
We are going to prove the below theorem (identical to part 2 of Theorem \ref{thrm:PSO}).
The motivation behind this theorem is that in many PSO instances $M^{\usuff}$ satisfies one of its conditions. In such case the theorem can be applied to understand the PSO convergence behavior in the region $\spp^{\udbacksuff}$. 
Moreover, this theorem further supports the PSO framework's perspective, where virtual forces are pushing the model surface towards the physical equilibrium.

\begin{theorem}
\label{thrm:PSO_UD_conv} 
Define an arbitrary space $\FF$ of functions from $\spp^{\udbacksuff}$ to $\RR$, with $h \in \FF$ being its element.
Then, depending on properties of a function $M^{\usuff}$, $f^* = \arginf_{f \in \FF} L^{\udbacksuff}(f)$ must satisfy:
	\begin{enumerate}
		\item If $\forall s \in \RR: M^{\usuff}(X, s) > 0$, then $f^*(X) = \infty$.
		\item If $\forall s \in \RR: M^{\usuff}(X, s) < 0$, then $f^*(X) = - \infty$.
		\item If $\forall s \in \RR: M^{\usuff}(X, s) \equiv 0$, then $f^*(X)$ can be arbitrary.
		\item If $\forall s \in \RR:$ 
		\begin{equation}
		M^{\usuff}(X, s) 
		\rightarrow
		\begin{cases}
		= 0,& s = h(X)\\
		> 0,              & s < h(X)\\
		< 0,              & s > h(X)
		\end{cases}
		\end{equation}
		then $f^*(X) = h(X)$.
		\item Otherwise, additional analysis is required.
	\end{enumerate}
\end{theorem}

\begin{proof}
The inner problem solved by $\inf_{f \in \FF} L^{\udbacksuff}(f)$ for each $X \in \spp^{\udbacksuff}$ is:
\begin{equation}
s^* = \arginf_{s \in \RR} 
-
z
\cdot
\widetilde{M}^{\usuff}
(
X,
s
)
,
\label{eq:disjoint_1}
\end{equation}
where $z \in \RRpos$. Given $\widetilde{M}^{\usuff}
(
X,
s
)
$ is differentiable, a derivative of the inner cost is $-
z
\cdot
M^{\usuff}
(
X,
s
)
$. If the inner cost is a strictly decreasing function of $s$, $\forall s \in \RR: M^{\usuff}(X, s) > 0$, then the infimum is $\inf_{s \in \RR} 
-
z
\cdot
\widetilde{M}^{\usuff}
(
X,
s
) = - \infty$ and $s^* = \infty$ - the inner cost will be lower for the bigger value of $s$. This leads to the entry 1 of the theorem. Similarly, if the cost is a strictly increasing function, $\forall s \in \RR: M^{\usuff}(X, s) < 0$, then the infimum is achieved at $s^* = - \infty$, yielding the entry 2.

If $\forall s \in \RR: M^{\usuff}(X, s) \equiv 0$, then the inner cost is constant. In such case the infimum is obtained at any $s \in \RR$, hence the corresponding $f^*(X)$ can be arbitrary (the entry 3).

Further, denote $s' \equiv h(X)$. Conditions of the entry 4 imply that the inner cost in Eq.~(\ref{eq:disjoint_1}) is strictly decreasing at $s < s'$ and strictly increasing at $s > s'$. Since it is also continuous (consequence of being differentiable), its infimum must be at $s^* = s'$. Thus, we have the entry 4: $f^*(X) = s' = h(X)$.

Otherwise, if $M^{\usuff}(X, s)$ does not satisfy any of the theorem's properties $(1)-(4)$, a further analysis of this particular \mf in the context of $L^{\udbacksuff}$ needs to be done.

\end{proof}

\subsubsection{Area $\spp^{\dubacksuff}$}
\label{sec:FuncDisjointSupp_du}

Consider the loss term $L^{\dubacksuff}(f)$ corresponding to the area $\spp^{\dubacksuff}$,
where $\probi{\usuff}{X} = 0$ and $\probi{\dsuff}{X} > 0$. The below theorem explains PSO convergence in this area.

\begin{theorem}
	\label{thrm:PSO_DU_conv} Define an arbitrary space $\FF$ of functions from $\spp^{\dubacksuff}$ to $\RR$, with $h \in \FF$ being its element.
	Then, depending on properties of a function $M^{\dsuff}$, $f^* = \arginf_{f \in \FF} L^{\dubacksuff}(f)$ must satisfy:
	\begin{enumerate}
		\item If $\forall s \in \RR: M^{\dsuff}(X, s) > 0$, then $f^*(X) = - \infty$.
		\item If $\forall s \in \RR: M^{\dsuff}(X, s) < 0$, then $f^*(X) = \infty$.
		\item If $\forall s \in \RR: M^{\dsuff}(X, s) \equiv 0$, then $f^*(X)$ can be arbitrary.
		\item If $\forall s \in \RR:$ 
		\begin{equation}
		M^{\dsuff}(X, s) 
		\rightarrow
		\begin{cases}
		= 0,& s = h(X)\\
		> 0,              & s > h(X)\\
		< 0,              & s < h(X)
		\end{cases}
		\end{equation}
		then $f^*(X) = h(X)$.
		\item Otherwise, additional analysis is required.
	\end{enumerate}
\end{theorem}

The proof of Theorem \ref{thrm:PSO_DU_conv} is symmetric to the proof of Theorem \ref{thrm:PSO_UD_conv} and hence omitted.

\section{Instances of PSO}
\label{sec:PSOInst}


\begin{table}
	\centering
	\begin{tabular}{lllll}
		\toprule
		\textbf{Method}     & \textbf{
			\emph{\underline{F}inal} $f_{\theta}(X)$ and $\KKK$ / \emph{\underline{R}eferences} /
			\emph{\underline{L}oss} / \emph{$M^{\usuff}(\cdot)$ and $M^{\dsuff}(\cdot)$}
		}      \\
		
		\midrule
		DeepPDF 
		& \underline{F}: $\probi{\usuff}{X}$, $\KKK = \RRpos$
		\\
		& \underline{R}: \citet{Baird05ijcnn,Kopitkov18arxiv}
		\\
		&  
		\underline{L}: 
		$-
		\E_{X \sim \probs{\usuff}}
		f_{\theta}(X) \cdot \probi{\dsuff}{X}		+
		\E_{X \sim \probs{\dsuff}}
		\half
		\Big[
		f_{\theta}(X)
		\Big]^{2}		$
		\\[5pt]
		&
		$M^{\usuff}$, $M^{\dsuff}$: $\probi{\dsuff}{X}$, $f_{\theta}(X)$  
		\\
		\midrule
		PSO-LDE
		& \underline{F}: $\log \probi{\usuff}{X}$, $\KKK = \RR$
		\\
		(Log Density
		& \underline{R}: Introduced and thoroughly analyzed in this paper,
		\\
		Estimators)
		&  \quad \quad see Section \ref{sec:DeepLogPDF}
		\\
		&  
		\underline{L}: unknown  \\[5pt]
		&
		$M^{\usuff}$, $M^{\dsuff}$: $
		\frac{\probi{\dsuff}{X}}
		{\left[
			\left[\exp f_{\theta}(X)\right]^{\alpha} + 
			\left[\probi{\dsuff}{X}\right]^{\alpha}
			\right]^{\frac{1}{\alpha}}}
		$,
		$
		\frac{\exp f_{\theta}(X)}
		{\left[
			\left[\exp f_{\theta}(X)\right]^{\alpha} + 
			\left[\probi{\dsuff}{X}\right]^{\alpha}
			\right]^{\frac{1}{\alpha}}}
		$
		\\[15pt]
		& where $\alpha$ is a hyper-parameter
		\\
		\midrule
		PSO-MAX
		& \underline{F}: $\log \probi{\usuff}{X}$, $\KKK = \RR$
		\\
		
		& \underline{R}: This paper, see Section \ref{sec:ColumnsEstME}
		\\
		&  
		\underline{L}: unknown  \\
		&
		$M^{\usuff}$, $M^{\dsuff}$: $\exp
		\left[
		-
		\max
		\left[
		f_{\theta}(X)
		-
		\log \probi{\dsuff}{X}
		, 0
		\right]
		\right]
		$,
		\\[5pt]
		&
		$
		\quad
		\quad
		\quad
		\quad
		\quad
		\exp
		\left[
		\min
		\left[
		f_{\theta}(X)
		-
		\log \probi{\dsuff}{X}
		, 0
		\right]
		\right]
		$  
		\\
		
		\midrule
		NCE
		& \underline{F}: $\log \probi{\usuff}{X}$, $\KKK = \RR$
		\\
		(Noise
		& \underline{R}: \citet{Smith05acl,Gutmann10aistats};
		\\
		Contrastive
		&
		\quad
		\quad
		\citet{Pihlaja12arxiv,Mnih12arxiv};
		\\
		Estimation)
		&
		\quad
		\quad
		\citet{Mnih13nips}
		\\
		&
		\quad
		\quad
		\citet{Gutmann12jmlr}
		\\[5pt]
		& 		
		\underline{L}: 
		$
		\E_{X \sim \probs{\usuff}}
		\log \frac{\exp[ f_{\theta}(X)] + \probi{\dsuff}{X}}{\exp[ f_{\theta}(X)]}
		+
		\E_{X \sim \probs{\dsuff}}
		\log \frac{\exp[ f_{\theta}(X)] + \probi{\dsuff}{X}}{\probi{\dsuff}{X}}	$
		\\[10pt]
		&
		$M^{\usuff}$, $M^{\dsuff}$: $
		\frac{\probi{\dsuff}{X}}{\exp[ f_{\theta}(X)] + \probi{\dsuff}{X}}
		$,
		$
		\frac{\exp[ f_{\theta}(X)]}{\exp[ f_{\theta}(X)] + \probi{\dsuff}{X}}
		$   \\
		
		\midrule
		IS
		& \underline{F}: $\log \probi{\usuff}{X}$, $\KKK = \RR$
		\\
		(Importance
		& \underline{R}: \citet{Pihlaja12arxiv}
		\\
		Sampling)
		& 		
		\underline{L}: 
		$-
		\E_{X \sim \probs{\usuff}}
		f_{\theta}(X)		+
		\E_{X \sim \probs{\dsuff}}
		\frac{\exp [f_{\theta}(X)]}{\probi{\dsuff}{X}}$
		\\
		&
		$M^{\usuff}$, $M^{\dsuff}$: $1$, $
		\frac{\exp [f_{\theta}(X)]}{\probi{\dsuff}{X}}
		$   \\
		
		\bottomrule
	\end{tabular}
	
	\caption{PSO Instances For Density Estimation, Part 1}
	\label{tbl:PSOInstances1}
	
\end{table}

\begin{table}
	\centering
	\begin{tabular}{ll}
		\toprule
		\textbf{Method}     & \textbf{
			\emph{\underline{F}inal} $f_{\theta}(X)$ and $\KKK$ / \emph{\underline{R}eferences} /
			\emph{\underline{L}oss} / \emph{$M^{\usuff}(\cdot)$ and $M^{\dsuff}(\cdot)$}
		}      \\

		\midrule
		Polynomial 
		& \underline{F}: $\log \probi{\usuff}{X}$, $\KKK = \RR$
		\\
		& \underline{R}: \cite{Pihlaja12arxiv}
		\\
		
		& 		
		\underline{L}: 
		$-
		\E_{X \sim \probs{\usuff}}
		\frac{\exp[ f_{\theta}(X)]}{\probi{\dsuff}{X}}		+
		\E_{X \sim \probs{\dsuff}}
		\half
		\frac{\exp[ 2 \cdot f_{\theta}(X)]}{[\probi{\dsuff}{X}]^2}$
		\\[5pt]
		&
		$M^{\usuff}$, $M^{\dsuff}$: $
		\frac{\exp[ f_{\theta}(X)]}{\probi{\dsuff}{X}}
		$,
		$
		\frac{\exp[ 2 \cdot f_{\theta}(X)]}{[\probi{\dsuff}{X}]^2}
		$   \\
		\midrule
		Inverse
		& \underline{F}: $\log \probi{\usuff}{X}$, $\KKK = \RR$
		\\
		Polynomial 
		& \underline{R}: \citet{Pihlaja12arxiv}
		\\
		& 		
		\underline{L}: 
		$
		\E_{X \sim \probs{\usuff}}
		\half
		\frac{[\probi{\dsuff}{X}]^2}{\exp[2 \cdot f_{\theta}(X)]}
		-
		\E_{X \sim \probs{\dsuff}}
		\frac{\probi{\dsuff}{X}}{\exp[ f_{\theta}(X)]}$
		\\[5pt]
		&
		$M^{\usuff}$, $M^{\dsuff}$: $
		\frac{[\probi{\dsuff}{X}]^2}{\exp[2 \cdot f_{\theta}(X)]}
		$,
		$
		\frac{\probi{\dsuff}{X}}{\exp[f_{\theta}(X)]}
		$   \\
		\midrule
		
		Inverse
		& \underline{F}: $\log \probi{\usuff}{X}$, $\KKK = \RR$
		\\
		Importance
		& \underline{R}: \citet{Pihlaja12arxiv}
		\\
		Sampling
		& 		
		\underline{L}: 
		$
		\E_{X \sim \probs{\usuff}}
		\frac{\probi{\dsuff}{X}}{\exp[f_{\theta}(X)]}		+
		\E_{X \sim \probs{\dsuff}}
		f_{\theta}(X)$
		\\[5pt]
		&
		$M^{\usuff}$, $M^{\dsuff}$: $
		\frac{\probi{\dsuff}{X^{\usuff}}}{\exp[f_{\theta}(X^{\usuff})]}
		$,
		$
		1$   \\
		
		\midrule
		Root  
		& \underline{F}: $\sqrt[\leftroot{-2}\uproot{2}d]{\probi{\usuff}{X}}$, $\KKK = \RRpos$
		\\
		Density & \underline{R}: This paper
		\\
		Estimation &
		\underline{L}: 
		$
		-
		\E_{X \sim \probs{\usuff}}
		f_{\theta}(X) \cdot \probi{\dsuff}{X}		+
		\E_{X \sim \probs{\dsuff}}
		\frac{1}{d + 1}
		\cdot
		\abs{f_{\theta}(X)}^{d + 1}
		$
		\\[5pt]
		&
		$M^{\usuff}$, $M^{\dsuff}$: $\probi{\dsuff}{X}$
		,
		$
		\abs{f_{\theta}(X)}^{d}
		\cdot
		sign[f_{\theta}(X)]
		$ \\
		
		\midrule
		PDF
		& \underline{F}: $\log
		\big[
		[\probs{\usuff} * \probs{\upsilon}](X)
		\big]$, $\KKK = \RR$
		\\
		Convolution
		& \underline{R}: This paper
		\\
		Estimation
		& 		
		\underline{L}: 
		$-
		\E_{X \sim \bar{\PP}^{\usuff}}
		f_{\theta}(X)		+
		\E_{X \sim \probs{\dsuff}}
		\frac{\exp [f_{\theta}(X)]}{\probi{\dsuff}{X}}$
		\\[5pt]
		&
		$M^{\usuff}$, $M^{\dsuff}$: $1$,
		$
		\frac{\exp [f_{\theta}(X)]}{\probi{\dsuff}{X}}
		$   \\
		&
		where $*$ is a convolution operator and
		\\
		&
		$\bar{\PP}^{\usuff}(X) = \probi{\usuff}{X} * \probi{\upsilon}{X}$ serves as \up density, 
		\\ &
		whose sample $X \sim \bar{\PP}^{\usuff}$ 
		can be obtained via $X = X^{\usuff} + \upsilon$  
		\\ &
		with $X^{\usuff} \sim \probi{\usuff}{X}$ and
		$\upsilon \sim \probi{\upsilon}{X}$, see Section \ref{sec:DeepLogPDFOF}
		\\
		
		\bottomrule
	\end{tabular}
	
	\caption{PSO Instances For Density Estimation, Part 2}
	\label{tbl:PSOInstances2}
	
\end{table}

\begin{table}
	\centering
	\begin{tabular}{lllll}
		\toprule
		
		\textbf{Method}     & \textbf{
			\emph{\underline{F}inal} $f_{\theta}(X)$ and $\KKK$ / \emph{\underline{R}eferences} /
			\emph{\underline{L}oss} / \emph{$M^{\usuff}(\cdot)$ and $M^{\dsuff}(\cdot)$}
		}      \\
		
		\midrule
		"Unit" Loss 
		& \underline{F}: Kantorovich potential \citep{Villani08book}, 
		\\ &
		\quad \quad only if the smoothness of $f_{\theta}(X)$ is heavily restricted
		\\
		& \underline{R}: see Section \ref{sec:SimpL}
		\\
		&  
		\underline{L}: 
		$-
		\E_{X \sim \probs{\usuff}}
		f_{\theta}(X)		+
		\E_{X \sim \probs{\dsuff}}
		f_{\theta}(X)$
		\\ &
		$M^{\usuff}$, $M^{\dsuff}$: $1$, $1$
		\\
		\midrule
		EBGAN
		& \underline{F}: $f_{\theta}(X) = m$ at $\{X : \probi{\usuff}{X} < \probi{\dsuff}{X} \}$, and $f_{\theta}(X) = 0$ otherwise
		\\
		Critic & \underline{R}: \citet{Zhao16arxiv}, see Section \ref{sec:PSOStability}
		\\
		&  
		\underline{L}: 
		$
		\E_{X \sim \probs{\usuff}}
		f_{\theta}(X)		+
		\E_{X \sim \probs{\dsuff}}
		\max [
		m - f_{\theta}(X), 0]
		$
		
		\\
		&  
		$M^{\usuff}$, $M^{\dsuff}$: $-1$,
		$
		-
		cut\_at
		\left[
		X, f_{\theta}(X), m
		\right]
		$  
		\\
		&
		where the considered model $f_{\theta}(X)$ is constrained to have non-negative outputs
		\\
		\midrule
		uLSIF   
		& \underline{F}: $\frac{\probi{\usuff}{X}}{\probi{\dsuff}{X}}$, $\KKK = \RRpos$
		\\
		& \underline{R}: \citet{Kanamori09mlr,Yamada11nips,Sugiyama12book};
		\\
		&
		\quad
		\quad
		\citet{Nam15toias,Uehara16}
		\\
		& 		
		\underline{L}: 
		$-
		\E_{X \sim \probs{\usuff}}
		f_{\theta}(X)		+
		\E_{X \sim \probs{\dsuff}}
		\half
		\big[
		f_{\theta}(X)
		\big]^{2}$
		\\
		&
		$M^{\usuff}$, $M^{\dsuff}$: $1$,
		$
		f_{\theta}(X)
		$   \\
		\midrule
		KLIEP    
		& \underline{F}: $\frac{\probi{\usuff}{X}}{\probi{\dsuff}{X}}$, $\KKK = \RRpos$
		\\
		& \underline{R}: \citet{Sugiyama08aism,Sugiyama12aism,Uehara16}
		\\
		& 		
		\underline{L}: 
		$-
		\E_{X \sim \probs{\usuff}}
		\log f_{\theta}(X)		+
		\E_{X \sim \probs{\dsuff}}
		[f_{\theta}(X) - 1]$
		\\
		&
		$M^{\usuff}$, $M^{\dsuff}$: $
		\frac{1}{f_{\theta}(X)}
		$,
		$1$   \\
		
		\midrule
		Classical     
		& \underline{F}: $\frac{\probi{\usuff}{X}}{\probi{\usuff}{X} + \probi{\dsuff}{X}}$, $\KKK = (0, 1)$
		\\
		GAN Critic & \underline{R}: \citet{Goodfellow14nips}
		\\
		&
		\underline{L}: 
		$-
		\E_{X \sim \probs{\usuff}}
		\log f_{\theta}(X)		-
		\E_{X \sim \probs{\dsuff}}
		\log
		\Big[
		1 - f_{\theta}(X)
		\Big]$
		\\
		&
		$M^{\usuff}$, $M^{\dsuff}$: $
		\frac{1}{f_{\theta}(X)}
		$,
		$
		\frac{1}{1 - f_{\theta}(X)}
		$   \\
		&
		* for $f_{\theta}(X) = sigmoid(h_{\theta}(X))$
		this loss is identical to Logistic Loss in Table \ref{tbl:PSOInstances4}
		\\
		
		\midrule
		NDMR
		& \underline{F}: $\frac{\probi{\usuff}{X}}{\probi{\usuff}{X} + \probi{\dsuff}{X}}$, $\KKK = (0, 1)$
		\\
		(Noise-Data & \underline{R}: This paper
		\\
		Mixture &
		\underline{L}: 
		$-
		\E_{X \sim \probs{\usuff}}
		f_{\theta}(X)		+
		\E_{X \sim \probs{M}}
		f_{\theta}(X)^2$
		\\
		Ratio)
		&
		$M^{\usuff}$, $M^{\dsuff}$: $1$, 
		$
		2
		f_{\theta}(X)
		$  
		\\ &
		where $\probi{M}{X} = \half \probi{\usuff}{X} + \half \probi{\dsuff}{X}$ serves as \down density, 
		\\ & instead of density $\probi{\dsuff}{X}$
		\\
		\midrule
		NDMLR     
		& \underline{F}: $\log \frac{\probi{\usuff}{X}}{\probi{\usuff}{X} + \probi{\dsuff}{X}}$, $\KKK = \RRneg$
		\\
		(Noise-Data & \underline{R}: This paper
		\\
		Mixture &
		\underline{L}: 
		$-
		\E_{X \sim \probs{\usuff}}
		f_{\theta}(X)		+
		\E_{X \sim \probs{M}}
		2 \exp [ f_{\theta}(X) ]$
		\\
		Log-Ratio)
		&
		$M^{\usuff}$, $M^{\dsuff}$: $1$,
		$
		2
		\exp [ f_{\theta}(X) ]
		$   
		\\ &
		where $\probi{M}{X} = \half \probi{\usuff}{X} + \half \probi{\dsuff}{X}$ serves as \down density, 
		\\ & instead of density $\probi{\dsuff}{X}$
		\\
		
		\bottomrule
	\end{tabular}
	
	\caption{PSO Instances For Density Ratio Estimation, Part 1}
	\label{tbl:PSOInstances3}
	
\end{table}

\begin{table}
	
	\centering
	\begin{tabular}{lllll}
		\toprule
		
		\textbf{Method}     & \textbf{
			\emph{\underline{F}inal} $f_{\theta}(X)$ and $\KKK$ / \emph{\underline{R}eferences} /
			\emph{\underline{L}oss} / \emph{$M^{\usuff}(\cdot)$ and $M^{\dsuff}(\cdot)$}
		}      \\

		\midrule
		Classical     
		& \underline{F}: $\log \frac{\probi{\usuff}{X}}{\probi{\usuff}{X} + \probi{\dsuff}{X}}$, $\KKK = \RRneg$
		\\
		GAN Critic & \underline{R}: This paper
		\\
		on log-scale
		&
		\underline{L}: 
		$\E_{X \sim \probs{\usuff}}
		\frac{1}{\exp [f_{\theta}(X)]}		-
		\E_{X \sim \probs{\dsuff}}
		\log
		\frac{\exp [f_{\theta}(X)]}{1 - \exp [f_{\theta}(X)]}
		$
		\\
		&
		$M^{\usuff}$, $M^{\dsuff}$: $
		\frac{1}{\exp [f_{\theta}(X)]}
		$,
		$
		\frac{1}{1 - \exp [f_{\theta}(X)]}
		$   \\
		\midrule
		Power    
		& \underline{F}: $\frac{\probi{\usuff}{X}}{\probi{\dsuff}{X}}$, $\KKK = \RRpos$
		\\
		Divergence & \underline{R}: \citet{Sugiyama12aism,Menon16icml}
		\\
		
		Ratio & 		
		\underline{L}: 
		$- \E_{X \sim \probs{\usuff}}
		\frac{f_{\theta}(X)^{\alpha}}{\alpha}
		+
		\E_{X \sim \probs{\dsuff}}
		\frac{f_{\theta}(X)^{\alpha + 1}}{\alpha + 1}
		$
		\\
		Estimation &
		$M^{\usuff}$, $M^{\dsuff}$: $
		f_{\theta}(X)^{\alpha - 1}
		$,
		$
		f_{\theta}(X)^{\alpha}
		$   \\
		\midrule
		Reversed   
		& \underline{F}: $\frac{\probi{\usuff}{X}}{\probi{\dsuff}{X}}$, $\KKK = \RRpos$
		\\
		KL & \underline{R}: \citet{Uehara16}
		\\
		& 		
		\underline{L}: 
		$\E_{X \sim \probs{\usuff}}
		\frac{1}{f_{\theta}(X)}
		+
		\E_{X \sim \probs{\dsuff}}
		\log
		f_{\theta}(X)
		$
		\\
		&
		$M^{\usuff}$, $M^{\dsuff}$: $
		\frac{1}{\left[
			f_{\theta}(X)
			\right]^2}
		$,
		$
		\frac{1}{f_{\theta}(X)}
		$   \\
		\midrule
		Balanced   
		& \underline{F}: $\frac{\probi{\usuff}{X}}{\probi{\dsuff}{X}}$, $\KKK = \RRpos$
		\\
		Density & \underline{R}: This paper
		\\
		Ratio
		& 		
		\underline{L}: 
		$-
		\E_{X \sim \probs{\usuff}}
		\log \left[ f_{\theta}(X) + 1\right]
		+
		\E_{X \sim \probs{\dsuff}}
		f_{\theta}(X)
		-
		\log
		\left[ f_{\theta}(X) + 1\right]
		$
		\\
		&
		$M^{\usuff}$, $M^{\dsuff}$: $
		\frac{1}{f_{\theta}(X) + 1}
		$,
		$
		\frac{f_{\theta}(X)}{f_{\theta}(X) + 1}
		$   \\
		\midrule
		Log-density    		
		& \underline{F}: $\log \frac{\probi{\usuff}{X}}{\probi{\dsuff}{X}}$, $\KKK = \RR$
		\\
		Ratio & \underline{R}: This paper
		\\
		& 		
		\underline{L}: 
		$-
		\E_{X \sim \probs{\usuff}}
		f_{\theta}(X)
		+
		\E_{X \sim \probs{\dsuff}}
		\exp [f_{\theta}(X)]
		$
		\\
		&
		$M^{\usuff}$, $M^{\dsuff}$: $1$,
		$
		\exp [f_{\theta}(X)]
		$   \\
		\midrule
		Square
		& \underline{F}: $\frac{\probi{\usuff}{X} - \probi{\dsuff}{X}}{\probi{\usuff}{X} + \probi{\dsuff}{X}}$, $\KKK = (-1, 1)$
		\\
		Loss & \underline{R}: \citet{Menon16icml}
		\\
		&
		\underline{L}: 
		$
		\E_{X \sim \probs{\usuff}}
		\half [1 - f_{\theta}(X)]^2
		+
		\E_{X \sim \probs{\dsuff}}
		\half [1 + f_{\theta}(X)]^2
		$
		\\
		&
		$M^{\usuff}$, $M^{\dsuff}$: $
		1 - f_{\theta}(X)
		$,
		$
		1 + f_{\theta}(X)
		$ \\
		\midrule
		Logistic
		& \underline{F}: $\log \frac{\probi{\usuff}{X}}{\probi{\dsuff}{X}}$, $\KKK = \RR$
		\\
		Loss & \underline{R}: \citet{Menon16icml}
		\\
		&
		\underline{L}: 
		$
		\E_{X \sim \probs{\usuff}}
		\log \big[
		1 + \exp [- f_{\theta}(X)]
		\big]
		+
		\E_{X \sim \probs{\dsuff}}
		\log \big[
		1 + \exp [f_{\theta}(X)]
		\big]
		$\\
		&
		$M^{\usuff}$, $M^{\dsuff}$: $
		\frac{1}{\exp [f_{\theta}(X)] + 1}
		$,
		$
		\frac{1}{\exp [- f_{\theta}(X)] + 1}
		$ \\

		\bottomrule
	\end{tabular}
	
	\caption{PSO Instances For Density Ratio Estimation, Part 2}
	\label{tbl:PSOInstances4}
	
\end{table}

\begin{table}
	\centering
	\begin{tabular}{lllll}
		\toprule
		
		\textbf{Method}     & \textbf{
			\emph{\underline{F}inal} $f_{\theta}(X)$ and $\KKK$ / \emph{\underline{R}eferences} /
			\emph{\underline{L}oss} / \emph{$M^{\usuff}(\cdot)$ and $M^{\dsuff}(\cdot)$}
		}      \\
		
		\midrule
		Exponential
		& \underline{F}: $\half \log \frac{\probi{\usuff}{X}}{\probi{\dsuff}{X}}$, $\KKK = \RR$
		\\
		Loss & \underline{R}: \citet{Menon16icml}
		\\
		&
		\underline{L}: 
		$
		\E_{X \sim \probs{\usuff}}
		\exp [- f_{\theta}(X)]		+
		\E_{X \sim \probs{\dsuff}}
		\exp [f_{\theta}(X)]$
		\\
		&
		$M^{\usuff}$, $M^{\dsuff}$: $
		\exp [- f_{\theta}(X)]
		$,
		$
		\exp [f_{\theta}(X)]
		$ \\
		
		\midrule
		LSGAN  
		& \underline{F}: $\frac{b \cdot \probi{\usuff}{X} + a \cdot \probi{\dsuff}{X}}{\probi{\usuff}{X} + \probi{\dsuff}{X}}$, $\KKK = (\min \left[ a, b \right], \max \left[ a, b \right])$
		\\
		Critic & \underline{R}: \citet{Mao17iccv}
		\\
		&
		\underline{L}: 
		$
		\E_{X \sim \probs{\usuff}}
		\half
		[f_{\theta}(X) - b]^2
		+
		\E_{X \sim \probs{\dsuff}}
		\half
		[f_{\theta}(X) - a]^2
		$
		\\[4pt]
		&
		$M^{\usuff}$, $M^{\dsuff}$: $
		b - f_{\theta}(X)
		$,
		$
		f_{\theta}(X) - a
		$ \\
		
		\midrule
		Kullback-Leibler
		& \underline{F}: $1 + \log \frac{\probi{\usuff}{X}}{\probi{\dsuff}{X}}$, $\KKK = \RR$
		\\
		Divergence & \underline{R}: \citet{Nowozin16nips}
		\\
		&
		\underline{L}: 
		$-
		\E_{X \sim \probs{\usuff}}
		f_{\theta}(X)
		+
		\E_{X \sim \probs{\dsuff}}
		\exp [f_{\theta}(X) - 1]
		$
		\\
		&
		$M^{\usuff}$, $M^{\dsuff}$: $1$,
		$
		\exp [f_{\theta}(X) - 1]
		$ \\
		\midrule
		Reverse KL
		& \underline{F}: $- \frac{\probi{\dsuff}{X}}{\probi{\usuff}{X}}$, $\KKK = \RRneg$
		\\
		Divergence & \underline{R}: \citet{Nowozin16nips}
		\\
		&
		\underline{L}: 
		$-
		\E_{X \sim \probs{\usuff}}
		f_{\theta}(X)
		+
		\E_{X \sim \probs{\dsuff}}
		[-1 - \log [- f_{\theta}(X)]]
		$
		\\
		&
		$M^{\usuff}$, $M^{\dsuff}$: $1$,
		$
		\frac{1}{- f_{\theta}(X)}
		$ \\
		\midrule
		Lipschitz
		& \underline{F}: $\half \cdot \frac{\probi{\usuff}{X} - \probi{\dsuff}{X}}{\sqrt{\probi{\usuff}{X} \cdot \probi{\dsuff}{X}}}$, $\KKK = \RR$
		\\
		Continuity & \underline{R}: \citet{Zhou18arxiv}
		\\
		Objective
		&
		\underline{L}: 
		$-
		\E_{X \sim \probs{\usuff}}
		\left[
		f_{\theta}(X) -
		\sqrt{f_{\theta}(X)^2 + 1}
		\right]
		+
		\E_{X \sim \probs{\dsuff}}
		\left[
		f_{\theta}(X)
		+
		\sqrt{f_{\theta}(X)^2 + 1}
		\right]
		$
		\\[5pt]
		&
		$M^{\usuff}$, $M^{\dsuff}$: $
		1 - 
		\frac{f_{\theta}(X)}{\sqrt{f_{\theta}(X)^2 + 1}}
		$,
		$
		1 +
		\frac{f_{\theta}(X)}{\sqrt{f_{\theta}(X)^2 + 1}}
		$ \\
		
		\midrule
		LDAR
		& \underline{F}: $\arctan \log \frac{\probi{\usuff}{X}}{\probi{\dsuff}{X}}$, $\KKK = (-\frac{\pi}{2}, \frac{\pi}{2})$
		\\
		(Log-density & \underline{R}: This paper
		\\
		Atan-Ratio)
		&
		\underline{L}: 
		unknown
		\\
		&
		$M^{\usuff}$, $M^{\dsuff}$: $
		\frac{1}
		{\exp \left[ \tan f_{\theta}(X) \right]
			+
			1}
		$,
		$
		\frac{1}
		{\exp \left[ - \tan f_{\theta}(X) \right]
			+
			1}
		$ 
		\\
		\midrule
		LDTR
		& \underline{F}: $\tanh \log \frac{\probi{\usuff}{X}}{\probi{\dsuff}{X}}$, $\KKK = (-1, 1)$
		\\
		(Log-density & \underline{R}: This paper
		\\
		Tanh-Ratio)
		&
		\underline{L}: 
		$
		\E_{X \sim \probs{\usuff}}
		\frac{2}{3}
		\cdot
		\left[
		1 - f_{\theta}(X)
		\right]^{\frac{3}{2}}
		+
		\E_{X \sim \probs{\dsuff}}
		\frac{2}{3}
		\cdot
		\left[
		1 + f_{\theta}(X)
		\right]^{\frac{3}{2}}
		$
		\\[3pt]
		&
		$M^{\usuff}$, $M^{\dsuff}$: $
		\sqrt{1 - f_{\theta}(X)}
		$,
		$
		\sqrt{1 + f_{\theta}(X)}
		$ 
		\\
		\bottomrule
	\end{tabular}
	
	\caption{PSO Instances For Density Ratio Estimation, Part 3}
	\label{tbl:PSOInstances5}
	
\end{table}

Many statistical methods exist whose loss and gradient have PSO forms depicted in Eq.~(\ref{eq:GeneralPSOLossFrmlPr_Limit_f}) and Eq.~(\ref{eq:GeneralPSOLossFrml_Limit}) for some choice of densities $\probs{\usuff}$ and $\probs{\dsuff}$, and of functions $\widetilde{M}^{\usuff}$, $\widetilde{M}^{\dsuff}$, $M^{\usuff}$ and $M^{\dsuff}$,
and therefore being instances of the PSO algorithm family. 
Typically, these methods defined via their loss which involves the pair of \ps $\{ \widetilde{M}^{\usuff}[X,s], \widetilde{M}^{\dsuff}[X,s] \}$. Yet, in practice it is enough to know their derivatives $\{ M^{\usuff}[X,s] = \frac{\partial \widetilde{M}^{\usuff}(X, s)}{\partial s}, M^{\dsuff}[X,s] = \frac{\partial \widetilde{M}^{\dsuff}(X, s)}{\partial s} \}$ for the gradient-based optimization (see Algorithm \ref{alg:PSOAlgo}). Therefore, PSO formulation focuses directly on $\{ M^{\usuff}, M^{\dsuff} \}$, with each PSO instance being defined by a particular choice of this pair.

Moreover, most of the existing PSO instances and subgroups actually require $\widetilde{M}^{\usuff}$ and $\widetilde{M}^{\dsuff}$ to be analytically known, while PSO composition  eliminates such demand. In fact, many pairs $\{ M^{\usuff}, M^{\dsuff} \}$ explored in this paper do not have closed-form known antiderivatives $\{ \widetilde{M}^{\usuff}, \widetilde{M}^{\dsuff} \}$.
Thus, PSO enriches the arsenal of available probabilistic methods.

In Tables \ref{tbl:PSOInstances1}-\ref{tbl:PSOInstances5} we show multiple PSO instances. We categorize all losses into two main categories - density estimation losses in Tables \ref{tbl:PSOInstances1}-\ref{tbl:PSOInstances2} and ratio density estimation losses in Tables \ref{tbl:PSOInstances3}-\ref{tbl:PSOInstances5}. In the former class of losses we are interested to infer density $\probs{\usuff}$ from its available data samples, while $\probs{\dsuff}$ represents some auxiliary distribution with analytically known pdf function $\probi{\dsuff}{X}$ whose samples are used to create the opposite force $F_{\theta}^{\dsuff}(X)$; this force will balance the force $F_{\theta}^{\usuff}(X)$ from $\probs{\usuff}$'s samples. Further, in the latter class we concerned to learn a density ratio, or some function of it, between two unknown densities $\probs{\usuff}$ and $\probs{\dsuff}$ by using the available samples from both distributions.

In the tables we present the PSO loss of each method, if analytically known, and the corresponding pair $\{ M^{\usuff}, M^{\dsuff} \}$.
We also 
indicate to what the surface $f(X)$ will converge assuming that PSO \bp in Eq.~(\ref{eq:BalPoint}) was obtained. Derivation of this convergence appears below. Importantly, it describes the optimal PSO solution only within the area $\spp^{\udcapsuff} \subset \RR^{n}$. For $X$ in $\spp^{\udbacksuff}$ or $\spp^{\dubacksuff}$, the convergence can be explained via theorems \ref{thrm:PSO_UD_conv} and \ref{thrm:PSO_DU_conv} respectively. Yet, in most of the paper we will limit our discussion to the convergence within the mutual support, implicitly assuming $\spp^{\usuff} \equiv \spp^{\dsuff}$.

\subsection{Deriving Convergence of PSO Instance}
\label{sec:PSOInstDeriveConv}

Given a PSO instance with a particular \psocomp, the convergence within $\spp^{\udcapsuff}$ can be derived by solving PSO \bp $\probi{\usuff}{X} \cdot M^{\usuff}\left[X, f^*(X)\right] = \probi{\dsuff}{X} \cdot M^{\dsuff}\left[X, f^*(X)\right]$.

\paragraph{Example 1:}
From Table \ref{tbl:PSOInstances1} we can see that IS method has $M^{\usuff}
\left[
X,
f(X)
\right] = 1$ and $M^{\dsuff}
\left[
X,
f(X)
\right] = \frac{\exp [f(X)]}{\probi{\dsuff}{X}}$. Given that samples within the loss have densities $X^{\usuff} \sim \probi{\usuff}{X}$ and $X^{\dsuff} \sim \probi{\dsuff}{X}$, we can substitute the sample densities and the \emph{magnitude} functions \pair into Eq.~(\ref{eq:BalPoint}) to get:
\begin{equation}
\frac{\probi{\usuff}{X}}{\probi{\dsuff}{X}}
=
\frac{\exp [f^*(X)] / \probi{\dsuff}{X}}{1}
\quad
\Rightarrow
\quad
f^*(X)
=
\log \probi{\usuff}{X}
,
\label{eq:BalPointExample}
\end{equation}
where we use an equality between density ratio of the samples and ratio of \emph{magnitude} functions to derive the final $f^*(X)$.
Thus, in case of IS approach, the surface will converge to the log-density $\log \probs{\usuff}(X)$.
\\

More formally, we can derive PSO convergence according to definitions of Theorem \ref{thrm:PSO_func_magn}, using \mgn ratio $R$ and its inverse $T$. The theorem allows us additionally to verify sufficient conditions required by PSO framework. Furthermore, we can decide whether the restriction of $\FF$'s range is necessarily by testing criteria of Theorem \ref{thrm:PSO_func_magn_unconstr}.

\paragraph{Example 2:}
Consider the "Polynomial" method in Table \ref{tbl:PSOInstances2}, with $M^{\usuff}
\left[
X,
f(X)
\right] = \frac{\exp[ f(X)]}{\probi{\dsuff}{X}}$ and $M^{\dsuff}
\left[
X,
f(X)
\right] = \frac{\exp[ 2 f(X)]}{[\probi{\dsuff}{X}]^2}$. Then, we have $\frac{M^{\dsuff}
\left[
X, f(X)
\right]}{M^{\usuff}
\left[
X, f(X)
\right]} = \frac{\exp[ f(X)]}{\probi{\dsuff}{X}}$ and hence $R(X, s) = \frac{\exp s}{\probi{\dsuff}{X}}$.
Further, consider the convergence interval $\KKK$ to be entire $\RR$. Both conditions 1 and 2 of Theorem \ref{thrm:PSO_func_magn} are satisfied - $R$ is continuous, strictly increasing and bijective w.r.t. domain $\RR$ and codomain $\RRpos$, and both \ms are continuous and positive on the entire $\RR$. Additionally, $\FF$'s range need not to be restricted since $\KKK \equiv \RR$.
Further, $R$ has a simple form and its invert is merely $T(X, z) = \log \probi{\dsuff}{X} + \log z$. The above $T$ and $R$ are inverse of each other w.r.t. the second argument, which can be easily verified. Next, we can calculate PSO convergence as $f^*(X) = T\left[X, \frac{\probi{\usuff}{X}}{\probi{\dsuff}{X}}\right] = \log \probi{\dsuff}{X} + \log \frac{\probi{\usuff}{X}}{\probi{\dsuff}{X}} = \log \probi{\usuff}{X}$. Hence, "Polynomial" method converges to $\log \probi{\usuff}{X}$.

\paragraph{Example 3:}
\label{exmp:atan_ratio} 
Consider the LDAR method in Table \ref{tbl:PSOInstances5}, with $M^{\usuff}
\left[
X,
f(X)
\right] = 		\frac{1}
{\exp \left[ \tan f(X) \right]
	+
	1}$ and $M^{\dsuff}
\left[
X,
f(X)
\right] = \frac{1}
{\exp \left[ - \tan f(X) \right]
	+
	1}$. Then, $\frac{M^{\dsuff}(X, f(X))}{M^{\usuff}(X, f(X))} = \frac{\exp \left[ \tan f(X) \right]
	+
	1}
{\exp \left[ - \tan f(X) \right]
	+
	1}$ and hence $R(X, s) = \frac{\exp \left[ \tan s \right]
	+
	1}
{\exp \left[ - \tan s \right]
	+
	1}$.
Function $R(X, s)$ is not bijective w.r.t. $s \in \RR$ - it has multiple positive increasing copies on each interval $( \pi k - \frac{\pi}{2}, \pi k + \frac{\pi}{2})$. Hence, it does not satisfy the necessary conditions. Yet, if we restrict it to a domain $( \pi k - \frac{\pi}{2}, \pi k + \frac{\pi}{2})$ for any $k \in \ZZ$, then Theorem \ref{thrm:PSO_func_magn} will hold. Particularly, if we choose $\KKK = ( - \frac{\pi}{2}, \frac{\pi}{2})$ then all theorem's conditions are satisfied. Moreover, Theorem \ref{thrm:PSO_func_magn_unconstr} is not applicable here since \ms are not defined at points $s \in \{ \frac{\pi}{2} k \}$. Therefore, we are required to limit range of $\FF$ to be $\KKK$.
The inverse of $R$ for the considered $\KKK$ is $T(X, z) = \arctan \log z$. Hence, LDAR converges to $\arctan \log \frac{\probi{\usuff}{X}}{\probi{\dsuff}{X}}$.

\subsection{Deriving New PSO Instance}
\label{sec:PSOInstDeriveNew}

In order to apply PSO for learning any function of $X$ and  $\frac{\probi{\usuff}{X}}{\probi{\dsuff}{X}}$, the appropriate PSO instance can be derived via Theorem \ref{thrm:PSO_func_conv}. Denote the required PSO convergence by a transformation $T(X, z): \RR^n \times \RR \rightarrow \RR$ s.t. $f^*(X) = T\left[X, \frac{\probi{\usuff}{X}}{\probi{\dsuff}{X}}\right]$ is the function we want to learn. Then according to the theorem, any pair \pair whose ratio $R \equiv \frac{M^{\dsuff}}{M^{\usuff}}$ satisfies $R \equiv T^{-1}$ (i.e. inverses between $\KKK$ and $\RRpos$), will produce the required convergence, given that theorem's conditions 
hold. 
Further, if conditions of Theorem \ref{thrm:PSO_func_magn_unconstr} likewise hold, then no range restriction over $f$ is needed.

Therefore, to learn any function $T\left[X, \frac{\probi{\usuff}{X}}{\probi{\dsuff}{X}}\right]$, first we obtain $R(X, s)$ by finding an inverse of $T(X, z)$ w.r.t. $z$.
Any valid pair of \emph{magnitudes} satisfying $\frac{M^{\dsuff}(X, s)}{M^{\usuff}(X, s)} = R\left[X, s\right]$ will produce the desired convergence.
For example, 
we can use a straightforward choice $M^{\dsuff}\left[
X,
s
\right] = R
\left[
X,
s
\right]$
and
$M^{\usuff}\left[
X,
s
\right] = 1$ in order to converge to the aimed target. Such choice corresponds to minimizing $f$-divergence \citep{Nguyen10tit,Nowozin16nips}, see Section \ref{sec:Bregman_PSO} for details.
Yet, typically these \mfs will be sub-optimal if for example $M^{\dsuff}$ is an unbounded function. When this is the case, we can derive a new pair of bounded \mfs by multiplying the old pair by the same factor $q\left[X, s\right]$ (see also Section \ref{sec:BoundUnboundMFs}).

\paragraph{Example 4:}
Consider a scenario where we would like to infer $f^*(X) = \frac{\probi{\usuff}{X} - \probi{\dsuff}{X}}{\probi{\usuff}{X} + \probi{\dsuff}{X}}$, similarly to "Square Loss" method in
Table \ref{tbl:PSOInstances4}. Treating only points $X \in \spp^{\dsuff}$, we can rewrite our objective as $f^*(X) = \frac{\frac{\probi{\usuff}{X}}{\probi{\dsuff}{X}} - 1}{\frac{\probi{\usuff}{X}}{\probi{\dsuff}{X}} + 1}$ and hence the required PSO convergence is given by $T(X, z) = \frac{z - 1}{z + 1}$. Further, its inverse function is given by $R(X, s) = \frac{1 + s}{1 - s}$. 
Therefore, \mfs must satisfy $\frac{M^{\dsuff}(X, s)}{M^{\usuff}(X, s)} = \frac{1 + s}{1 - s}$. One choice for such \ms is $M^{\usuff}
\left[
X,
s
\right] = 1 - s$ and $M^{\dsuff}
\left[
X,
s
\right] = 1 + s$, just like in the "Square Loss" method \citep{Menon16icml}.
Note that the convergence interval of this PSO instance is $\KKK = (-1, 1)$ which is $X$-invariant. Further, $T$ is strictly increasing and continuous at $z \in \RRpos$ and \pair are continuous and positive at $s \in \KKK$, hence satisfying the conditions of Theorem \ref{thrm:PSO_func_conv}. Moreover, \pair are actually continuous on entire $s \in \RR$, with $\forall s \leq -1$: $M^{\usuff}(X, s) > M^{\dsuff}(X, s)$ and $\forall s \geq 1$: $M^{\usuff}(X, s) < M^{\dsuff}(X, s)$. Since $M^{\usuff}$ and $M^{\dsuff}$ do not have the same sign outside of $\KKK$, conditions of Theorem \ref{thrm:PSO_func_magn_unconstr} are likewise satisfied and the $\FF$'s range can be the entire $\RR$.
Furthermore, other variants with the same PSO \bp can be easily constructed. For instance, we can use $M^{\usuff}
\left[
X,
f(X)
\right] = \frac{1 - f(X)}{D(X,
	f(X))}$ and $M^{\dsuff}
\left[
X,
f(X)
\right] = \frac{1 + f(X)}{D(X,
	f(X))}$ with $D(X,
f(X)) \triangleq |1 - f(X)| + |1 + f(X)|$ instead. Such normalization by function $D(\cdot)$ does not change the convergence, yet it produces bounded \mfs that are typically more stable during the optimization. All the required conditions are satisfied also by these normalized \ms. Additionally, recall that we considered only points within support of $\probi{\dsuff}{X}$. Outside of this support, any $X \in \spp^{\udbacksuff}$ will push the model surface $f(X)$ according to the rules implied by $M^{\usuff}$; note also that $M^{\usuff}$ changes signs at $f(X) = 1$, with force always directed toward the height $h(X) = 1$. Therefore, at points $\{ X \in \spp^{\udbacksuff} \}$ the convergence will be $f^*(X) = 1$, which is also a corollary of the condition 4 in Theorem \ref{thrm:PSO_UD_conv}. Finally at points outside of both supports there is no optimization performed, and hence theoretically nothing moves there - no constraints are applied on the surface $f$ in these areas. Yet, in practice $f(X)$ at $X \notin \spp^{\udcupsuff}$ will be affected by pushes at the training points, according to the elasticity properties of the model expressed via kernel $g_{\theta}(X, X')$ (see Section \ref{sec:ExprrKernelBnd} for details).
\\

\begin{table}[tb]
	\small
	\centering
	\begin{tabular}{lllll}
		\toprule
		\textbf{Description} & \textbf{Target Function}     & \textbf{$T(X, z)$} & \textbf{$R(X, s)$}   & $\KKK$   \\
		\midrule
		Density-Ratio Estimation
		&
		$\frac{\probi{\usuff}{X}}{\probi{\dsuff}{X}}$
		&
		$z$
		&
		$s$
		&
		$\RRpos$
		\\
		Log-Density-Ratio Estimation
		&
		$\log \frac{\probi{\usuff}{X}}{\probi{\dsuff}{X}}$
		&
		$\log z$
		&
		$\exp s$
		&
		$\RR$
		\\
		Density Estimation
		&
		$\probi{\usuff}{X}$
		&
		$\probi{\dsuff}{X} \cdot z$
		&
		$\frac{s}{\probi{\dsuff}{X}}$
		&
		$\RRpos$
		\\
		Log-Density Estimation
		&
		$\log \probi{\usuff}{X}$
		&
		$\log \probi{\dsuff}{X} + \log z$
		&
		$\frac{\exp s}{\probi{\dsuff}{X}}$
		&
		$\RR$
		\\
		\bottomrule
	\end{tabular}
	
	\caption[Common target functions, their corresponding $T$ and $R$ mappings, and the convergence interval $\KKK$.]{Common target functions, their corresponding $T$ and $R$ mappings, and the convergence interval $\KKK$. Note that for density estimation methods (2 last cases) the auxiliary pdf $\probi{\dsuff}{X}$ is known analytically.}
	\label{tbl:ClassicTransforms}
	
\end{table}

In Table \ref{tbl:ClassicTransforms} we present transformations $T$ and $R$ for inference of several common target functions. As shown, if $\probi{\dsuff}{X}$ is analytically known, Theorem \ref{thrm:PSO_func_conv} can be used to also infer any function of density $\probi{\usuff}{X}$, by multiplying $z$ argument by $\probi{\dsuff}{X}$ inside $T$. Thus, we can apply the theorem to derive a new PSO instances for pdf estimation, and to mechanically recover many already existing such techniques. In Section \ref{sec:DeepLogPDF} we will investigate new methods provided by the theorem for the estimation of log-density $\log \probi{\usuff}{X}$.

\begin{remark}
The inverse relation $R \equiv T^{-1}$ and properties of $R$ and $T$ described in theorems \ref{thrm:PSO_func_conv} and \ref{thrm:PSO_func_magn} imply that antiderivatives $\widetilde{R}
\left[
X,
s
\right]
\triangleq
\int_{s_0}^{s}
R(X, t)
dt$ and
$\widetilde{T}
\left[
X,
z
\right]
\triangleq
\int_{z_0}^{z}
R(X, t)
dt$ are Legendre-Fenchel transforms of each other. Such a connection reminds the relation between Langrangian and Hamiltonian mechanics, and opens a bridge between control and learning theories. A detailed exploration of this connection is an additional interesting direction for future research. 
\end{remark}

Further, density of \up and \down sample points within PSO loss can be changed from the described above choice $\probs{\usuff}$ and $\probs{\dsuff}$, to infer other target functions. For example, in NDMR method from Table \ref{tbl:PSOInstances3} instead of $\probi{\dsuff}{X}$ we sample $\half \probi{\usuff}{X} + \half \probi{\dsuff}{X}$ to construct training dataset of \down points (denoted by $\{X^{\dsuff}_{i}\}$ in Eq.~(\ref{eq:GeneralPSOLossFrml})). That is, the updated \down density is mixture of two original densities with equal weights. Then, by substituting sample densities and appropriate \emph{magnitude} functions $\{ M^{\usuff}
\left[
X,
f(X)
\right] = 1, M^{\dsuff}
\left[
X,
f(X)
\right] = 2
f(X) \}$ into the \bp equilibrium in Eq.~(\ref{eq:BalPoint}) we will get:
\begin{equation}
\frac{\probi{\usuff}{X}}{\half \probi{\usuff}{X} + \half \probi{\dsuff}{X}}
=
\frac{2
	f(X)}{1}
\quad
\Rightarrow
\quad
f(X)
=
\frac{\probi{\usuff}{X}}{\probi{\usuff}{X} + \probi{\dsuff}{X}}
.
\label{eq:BalPointExampleSD}
\end{equation}
The NDMR infers the same target function as the Classical GAN Critic loss from Table \ref{tbl:PSOInstances3}, and can be used as its alternative.
Therefore, an additional degree of freedom is acquired by considering different sampling strategies in PSO framework.
Similar ideas will allow us to also infer conditional density functions, as shown in Section \ref{sec:CondDeepPDFMain}.

\subsection{PSO Feasibility Verification and Polar Parametrization}
\label{sec:ConVerComp}

Sometimes it may be cumbersome to test if given \pair satisfy all required conditions over sets $\KKK^{-}$, $\KKK$ and $\KKK^{+}$. Below we propose representing \ms within a complex plane, and use the corresponding polar parametrization. The produced representation yields a graphical visualization of PSO instance which permits for easier feasibility analysis.

For this purpose, define PSO complex-valued function as $c \left[ X, s \right] \triangleq M^{\usuff}
\left[
X,
s
\right] + M^{\dsuff}
\left[
X,
s
\right] \cdot i$ whose real part is \up \mgn, and imaginary part - \down \mgn. Further, denote by $c_{\angle}$ and $c_{r}$ the angle and the radius of $c$ defined as $c_{\angle} \left[ X, s \right] = \atantwo(M^{\dsuff}
\left[
X,
s
\right], M^{\usuff}
\left[
X,
s
\right])$ and $c_{r} \left[ X, s \right] = \sqrt{M^{\usuff}
	\left[
	X,
	s
	\right]^2 + M^{\dsuff}
	\left[
	X,
	s
	\right]^2}$.
Conditions from theorems \ref{thrm:PSO_func_magn} and \ref{thrm:PSO_func_magn_unconstr} can be translated into conditions over $c \left[ X, s \right]$ as following.

\begin{lemma}[Complex Plane Feasibility]
	\label{lmm:FsbltTest} 
	Consider PSO instance that is described by a complex-valued function $c \left[ X, s \right]: \RR^n \times \RR \rightarrow \CC$ and some convergence interval $\KKK \triangleq (s_{min}, s_{max})$. Assume:
	\begin{enumerate}
		\item $c \left[ X, s \right]$ is continuous at any $s \in \KKK$, with $c_{r} \left[ X, s \right] > 0$ and $0 < c_{\angle} \left[ X, s \right] < \pihalf$.
		\item $c_{\angle} \left[ X, s \right]$ is strictly increasing and bijective w.r.t. domain $s \in \KKK$ and codomain $(0, \pihalf)$.
	\end{enumerate}
	Then, given that the range of $\FF$ is $\KKK$, the minima $f^* = \arginf_{f \in \FF} L_{PSO}(f)$ will satisfy $\forall X \in \spp^{\udcapsuff}: f^*(X) = T\left[X, \frac{\probi{\usuff}{X}}{\probi{\dsuff}{X}}\right]$, where $T \left[X, z\right] \triangleq c_{\angle}^{-1} \left[ X, \atan(z) \right]$.
	Further, $\FF$'s range can be entire $\RR$ if the following conditions hold:
	\begin{enumerate}
		\item $c \left[ X, s \right]$ is continuous on $s \in \RR$, with $c_{r} \left[ X, s \right] > 0$.
		\item $\forall s \in \KKK^{-}: \frac{3}{2} \pi \leq c_{\angle} \left[ X, s \right] \leq 2 \pi$.
		\item $\forall s \in \KKK^{+}: \pihalf \leq c_{\angle} \left[ X, s \right] \leq \pi$.
	\end{enumerate}
\end{lemma}

\begin{figure}
	\centering
	
	\begin{tabular}{ccc}
		
		\subfloat[\label{fig:Magns_signs-a}]{ \includegraphics[width=0.65\textwidth]{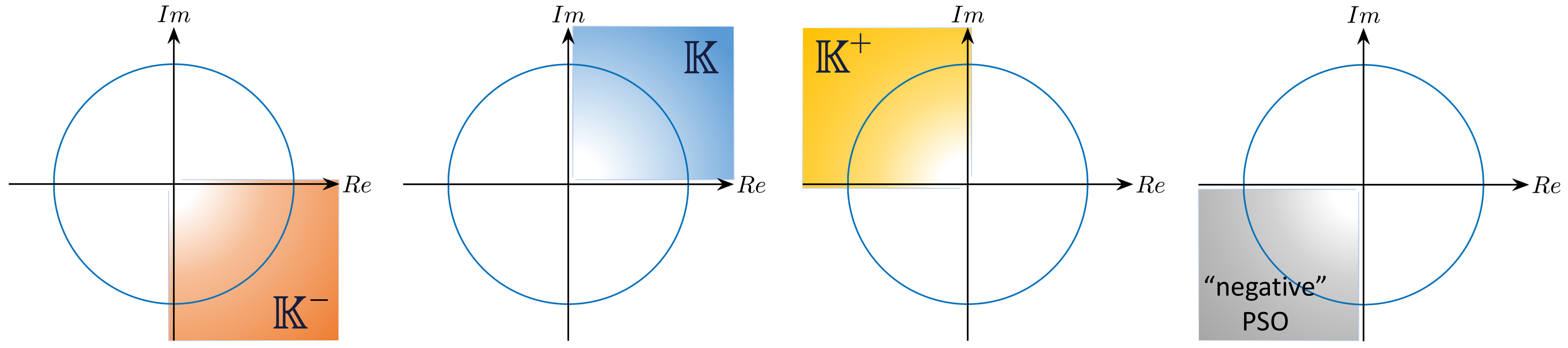}}
		
		\subfloat[\label{fig:Magns_signs-b}]{ \includegraphics[width=0.145\textwidth]{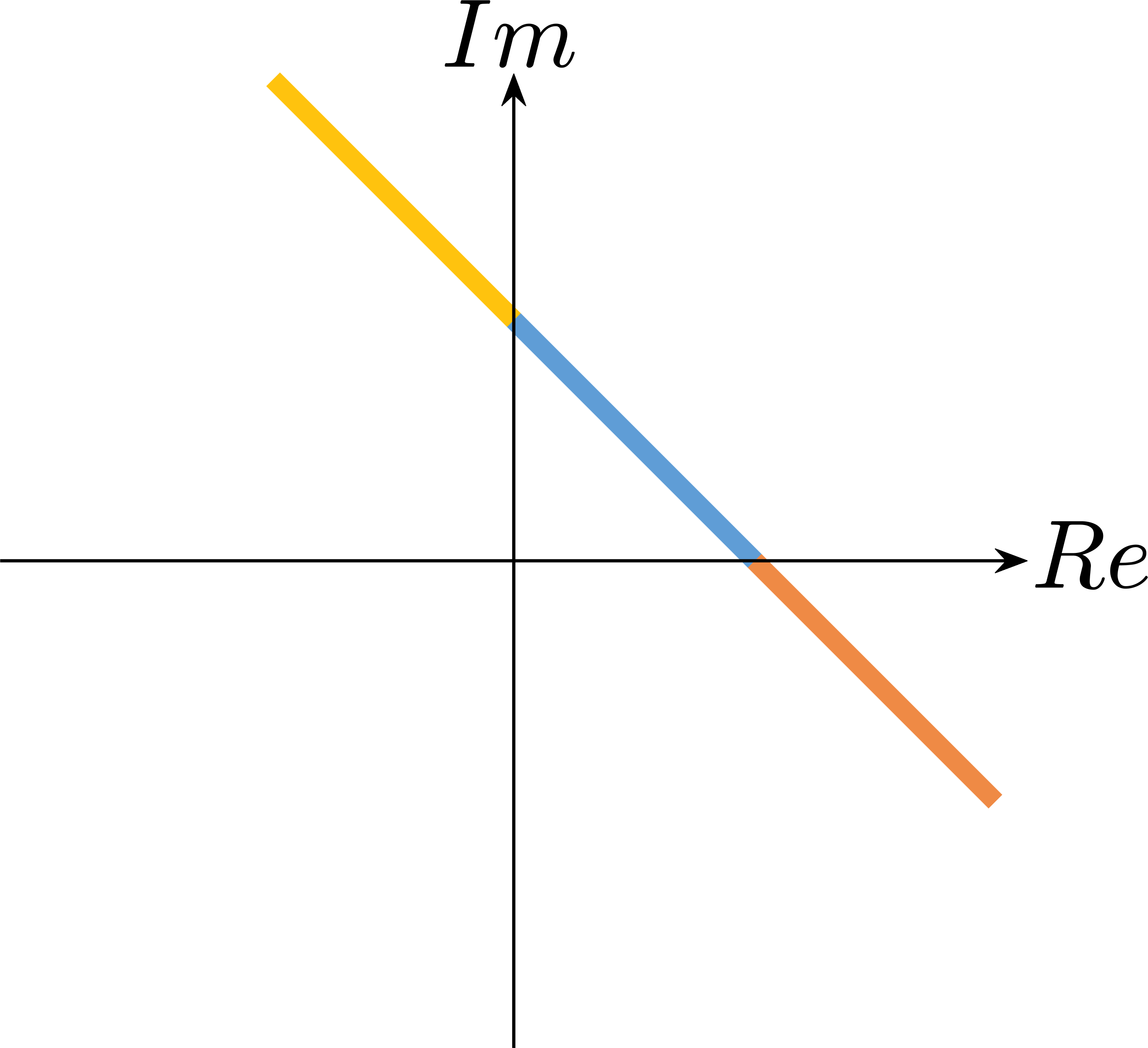}}
		
		\subfloat[\label{fig:Magns_signs-c}]{ \includegraphics[width=0.145\textwidth]{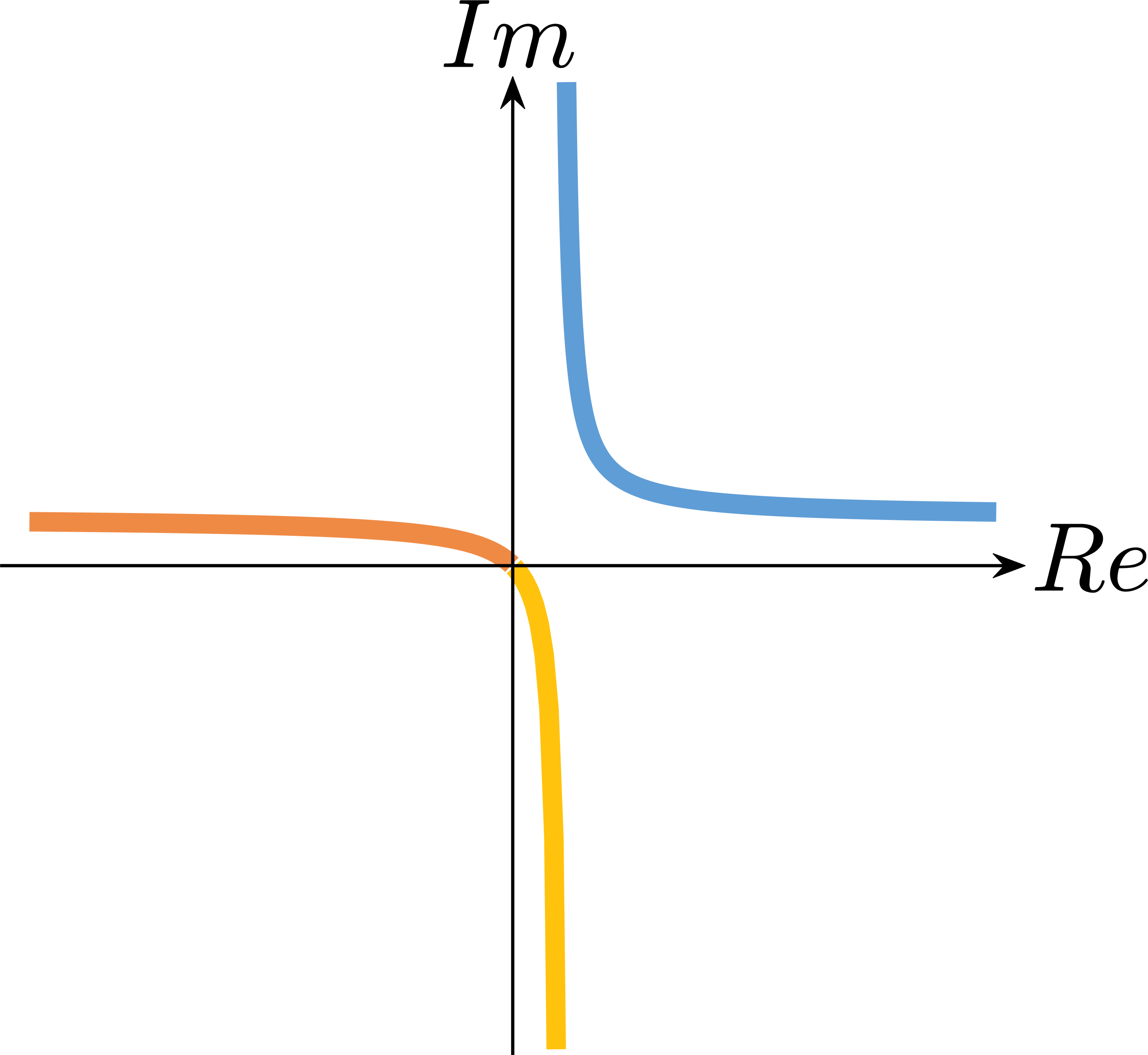}}
	\end{tabular}
	
	\protect
	\caption[Correspondence between different quadrants of a complex plane and different parts of PSO curve.]{(a) Correspondence between different quadrants of a complex plane and different parts of $c \left[ X, s \right]$. The quadrant \rom{1} is where $c$ must be for any $s \in \KKK$. Additionally, if $c$ is continuous w.r.t. $s \in \RR$ and if it is located in the quadrant \rom{4} at $s \in \KKK^{-}$ and in the quadrant \rom{2} at $s \in \KKK^{+}$,
		then we are allowed to reduce restrictions over the function range. Further, the quadrant \rom{3} is associated with "negative" PSO whose \ms have negative outputs and that can learn any decreasing target function $T\left[X, \frac{\probi{\usuff}{X}}{\probi{\dsuff}{X}}\right]$. (b)-(c) $c \left[ X, s \right]$ for two PSO instances is depicted, (b) "Square Loss" from Table \ref{tbl:PSOInstances4} and (c) "Classical GAN Critic Loss" from Table \ref{tbl:PSOInstances3}. Colors red, blue and yellow represent parts of $c$ at $s \in \KKK^{-}$, $s \in \KKK$ and $s \in \KKK^{+}$ respectively. As seen, in case of (b) the requirements are satisfied and hence there is no need to restrict the range of functions within $\FF$. In contrast, in (c) the curve $c \left[ X, s \right]$ parametrized by $s$ is not even continuous, hence $\FF$'s range must be $\KKK$.
	}
	\label{fig:Magns_signs}
\end{figure}

\begin{proof}
The necessary positivity of \ms over $s \in \KKK$ from Theorem \ref{thrm:PSO_func_magn} implies $\forall s \in \KKK: \left[ 0 < c_{\angle} \left[ X, s \right] < \pihalf \right] \vee \left[ c_{r} \left[ X, s \right] > 0 \right]$. Further, conditions of Theorem \ref{thrm:PSO_func_magn_unconstr} are equivalent to require $\forall s \in \KKK^{-}: \left[ \frac{3}{2} \pi \leq c_{\angle} \left[ X, s \right] \leq 2 \pi \right] \vee \left[ c_{r} \left[ X, s \right] > 0 \right]$ and $\forall s \in \KKK^{+}: \left[ \pihalf \leq c_{\angle} \left[ X, s \right] \leq \pi \right] \vee \left[ c_{r} \left[ X, s \right] > 0 \right]$. Likewise, observe that the range of angles $(\pi, \frac{3}{2} \pi)$ is allocated by "negative" PSO family mentioned in Section \ref{sec:FuncMutSupp_diff}, which can be formulated by switching between \up and \down terms of PSO family and which allows to learn any decreasing function $T\left[X, \frac{\probi{\usuff}{X}}{\probi{\dsuff}{X}}\right]$. See also the schematic illustration in Figure \ref{fig:Magns_signs-a}.

Further, continuity of $c \left[ X, s \right]$ (over $\KKK$ or over entire $\RR$) is equivalent to continuity of \ms enforced by theorems \ref{thrm:PSO_func_magn} and \ref{thrm:PSO_func_magn_unconstr}. Likewise, it leads to continuity of $c_{\angle} \left[ X, s \right]$.

Moreover, given that $c \left[ X, s \right]$ at $s \in \KKK$ is located within the quadrant \rom{1} of a complex plane, its angle can be rewritten as $c_{\angle} \left[ X, s \right] =  \atan \left( \frac{M^{\dsuff}
	\left[
	X, s
	\right]}{M^{\usuff}
	\left[
	X, s
	\right]}\right) = \atan \left( R \left[
X, s
\right] \right)$, with $R \left[
X, s
\right] = \tan c_{\angle} \left[ X, s \right]$ and $T \left[X, z\right] = R^{-1} \left[X, z\right] = c_{\angle}^{-1} \left[ X, \atan(z) \right]$ where $c_{\angle}^{-1}$ is an inverse of $c_{\angle}$ w.r.t. second argument. Hence, continuity of $c_{\angle}$ (implied by continuity of $c$) yields continuity of $R$ at $s \in \KKK$, which is required by Theorem \ref{thrm:PSO_func_magn}. 

Finally, strictly increasing property of $c_{\angle}$ is equivalent to the same property of $R$ since they are related via strictly increasing $\tan(\cdot)$ and $\atan(\cdot)$. Similarly, bijectivity is also preserved, with codomain changed from $\RRpos$ to $(0, \pihalf)$ due to limited range of $\atan(\cdot)$.

\end{proof}

The above lemma summarizes conditions required by PSO framework. As noted in Section \ref{sec:FuncMutSupp_diff}, this condition set is overly restrictive and some of its parts may be relaxed. Particularly, we speculate that continuity may be replaced by continuity \emph{almost everywhere}, and "increasing" (without "strictly") may be sufficient. We shall address such condition relaxation in future work.

To verify feasibility of any PSO instance, $c \left[ X, s \right]$ can be drawn as a curve within a convex plane where conditions of the above lemma can be checked. In Figures \ref{fig:Magns_signs-b} and \ref{fig:Magns_signs-c} we show example of this curve for "Square Loss" from Table \ref{tbl:PSOInstances4} and "Classical GAN Critic Loss" from Table \ref{tbl:PSOInstances3} respectively. From the first diagram it is visible that the curve satisfies lemma's conditions, which allows us to not restrict $\FF$'s range when optimizing via "Square Loss". In the second diagram conditions do not hold, leading to the conclusion that "Classical GAN Critic Loss" may be optimized only over $\FF$ whose range is exactly $\KKK$.

\subsection{PSO Subsets}
\label{sec:PSOInstAllSets}

\begin{figure}[tb]
	
	\centering

		\subfloat{\includegraphics[width=0.95\textwidth]{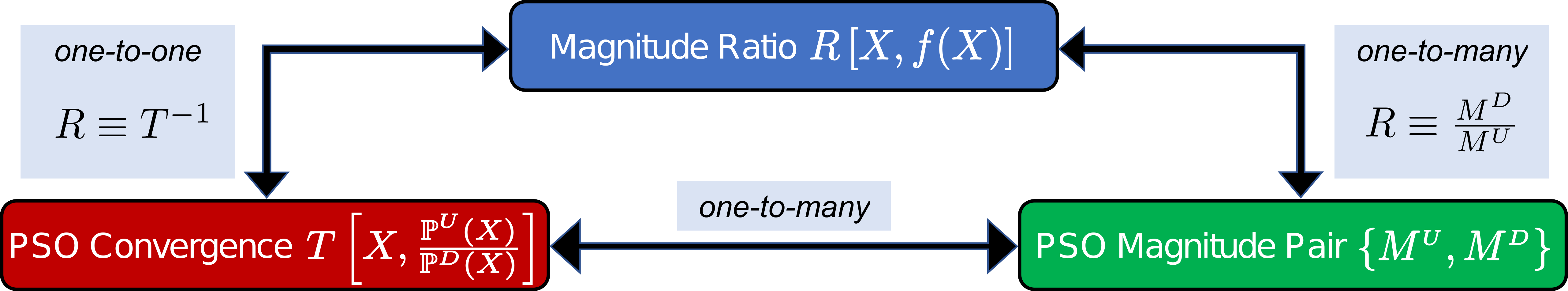}}

	\protect
	
	\caption[Schematic relationship between PSO mappings.]{Schematic relationship between PSO mappings of Theorem \ref{thrm:PSO_func_conv} and Theorem \ref{thrm:PSO_func_magn}. $T$ and $R$ are inverse functions, with one-to-one relation between them. $R$ is ratio of \pair, with infinitely many choices of the latter producing the same ratio. For this reason, many pairs of \mfs will yield the same PSO convergence $T$.
	}
	\label{grh:PSOTransformations}
\end{figure}

Given any two densities $\probs{\usuff}$ and $\probs{\dsuff}$, all PSO instances can be represented as a set of all feasible \mgn pairs $\psoset \triangleq \{ [M^{\usuff}, M^{\dsuff}] : feasible(M^{\usuff}, M^{\dsuff}) \}$, where $feasible(\cdot)$ is a logical indicator that specifies whether the arguments satisfy conditions of Theorem \ref{thrm:PSO_func_magn} (for any set $\KKK$) or not.
Below we systematically modulate the set $\psoset$ into subgroups, providing an useful terminology for the later analysis.

The relation of PSO mappings is presented in Figure \ref{grh:PSOTransformations}.
As observed, many different PSO instances have the same approximated target function. We will use this property to divide the set $\psoset$ of all feasible PSO instances into disjoint subsets, according to the estimation convergence. Considering any target function with the corresponding mapping $T$, we denote by $\psoset[T] \triangleq \{ [M^{\usuff}, M^{\dsuff}] : [M^{\usuff}, M^{\dsuff}] \in \psoset \vee \frac{M^{\dsuff}}{M^{\usuff}} = T^{-1}\}$ all feasible PSO instances that converge to $T$. The subset $\psoset[T]$ is referred below as $T$'s PSO \emph{consistent} \mgn set - PSO-CM set of $T$ for shortness. Further, in this paper we focus on two specific PSO subsets with $T(X, z) = \probi{\dsuff}{X} \cdot z$ and $T(X, z) = \log \probi{\dsuff}{X} + \log z$ that infer $\probi{\usuff}{X}$ and $\log \probi{\usuff}{X}$, respectively. For compactness, we denote the former as $\pdfestset$ and the latter as $\logpdfestset$.
Likewise, below we will term two pairs of \mfs as PSO \emph{consistent} if they belong to the same subset $\psoset[T]$ (i.e. if their \mgn ratio is the same).

Given a specific PSO task at hand, represented by the required convergence $T$, it is necessary to choose the most optimal member of $\psoset[T]$ for the sequential optimization process.
In Sections \ref{sec:ConvA} and \ref{sec:DensEst} we briefly discuss how to choose the most optimal PSO instance from PSO-CM set of any given $T$, based on the properties of \mfs.

\subsection{PSO Methods Summary}
\label{sec:PSOInstSumm}

The entire exposition of this section was based on a relation in Eq.~(\ref{eq:BalPoint}) that sums up the main principle of PSO - \up and \down point-wise forces must be equal at the optimization equilibrium.
In Tables \ref{tbl:PSOInstances1}-\ref{tbl:PSOInstances5} we refer to relevant works in case the specific PSO losses were already discovered in previous scientific studies. The previously discovered ones were all based on various sophisticated mathematical laws and theories, yet they all could be also derived in a simple unified way using PSO concept and Theorem \ref{thrm:PSO_func_conv}. Additionally, besides the previously discovered methods, in Tables \ref{tbl:PSOInstances1}-\ref{tbl:PSOInstances5} we introduce several new losses for inference of different stochastic modalities of the data, as the demonstration of usage and usefulness of the general PSO formulation. In Section \ref{sec:Bregman_PSO} we reveal that PSO framework has a tight relation with Bregman and "$f$" divergencies.
Furthermore, in Section \ref{sec:CrosEntr} we prove that the cross-entropy losses are also instances of PSO. Likewise, in Section \ref{sec:MLE_PSO} we derive Maximum Likelihood Estimation (MLE) from \psofunc.

\section{PSO, Bregman and "$f$" Divergencies}
\label{sec:Bregman_PSO}

Below we define \psodiv and show its connection to Bregman divergence \citep{Bregman67cmmp,Gneiting07jasa} and $f$-divergence \citep{Ali66rsss}, that are associated with many existing statistical methods.

\subsection{PSO Divergence}
\label{sec:PSO_div}

Minimization of $L_{PSO}(f)$ corresponds also to minimization of:
\begin{equation}
D_{PSO}(f^*, f)
\triangleq
L_{PSO}(f)
-
L_{PSO}(f^*)
=
-
\E_{X \sim \probs{\usuff}}
\int_{f^*(X)}^{f(X)}
M^{\usuff}(X, t)
dt
+
\E_{X \sim \probs{\dsuff}}
\int_{f^*(X)}^{f(X)}
M^{\dsuff}(X, t)
dt
\label{eq:PSODiv}
\end{equation}
where $f^*(X) = T\left[X, \frac{\probi{\usuff}{X}}{\probi{\dsuff}{X}}\right]$ is the optimal solution characterized by Theorem \ref{thrm:PSO_func_magn}. 
Since $f^*$ is unique minima (given that theorem's "sufficient" conditions hold) and since $L_{PSO}(f)
\geq
L_{PSO}(f^*)
$, we have the following properties: $D_{PSO}(f^*, f) \geq 0$ and $D_{PSO}(f^*, f) = 0 \iff f^* = f$. Thus, $D_{PSO}$ can be used as a "distance" between $f$ and $f^*$ and we name it \psodiv. Yet, note that $D_{PSO}$ does not measure a distance between any two functions; instead it evaluates distance between any $f$ and the optimal solution of the specific PSO instance, derived from the corresponding \psocomp via PSO \bp.

\subsection{Bregman Divergence}

Define $\FF$ to be a convex set of non-negative functions from $\RR^{n}$ to $\RRnonneg$. Bregman divergence for every $q, p \in \FF$ is defined as $D_{\psi}(q, p)
=
\psi(q)
- \psi(p)
- 
\int
\nabla_{q}
\psi(q(X))
\cdot 
\left[
q(X) - p(X)
\right]
dX
$,
where $\psi(\mu)$ is a continuously-differentiable, strictly convex functional over $\mu \in \FF$, and $\nabla_{q}$ is the differentiation w.r.t. $q$. When $\psi(\mu)$ has a form $\psi(\mu) = \int \varphi(\mu(X)) dX$ with a strictly convex function $\varphi: \RR \rightarrow \RR$, the divergence is referred to as U-divergence \citep{Eguchi09book} or sometimes as separable Bregman divergence \citep{Grunwald04aos}. In such case $D_{\psi}(q, p)$ has the form:
\begin{equation}
D_{\varphi}(q, p)
=
\int
\varphi
\left[
q(X)
\right]
- 
\varphi
\left[
p(X)
\right]
- 
\varphi'
\left[
p(X)
\right]
\cdot 
\left[
q(X) - p(X)
\right]
dX
.
\label{eq:BregSepDiv}
\end{equation}
The above modality is non-negative for all $q, p \in \FF$ and is zero if and only if $q = p$, hence it specifies a "distance" between $q$ and $p$. For this reason, $D_{\varphi}(q, p)$ is widely used in the optimization $\min_{p} D_{\varphi}(q, p)$ where $q$ usually represents an unknown density of available i.i.d. samples and $p$ is a model which is optimized to approximate $q$. For example, when $\varphi(s) = s \cdot \log s$, we obtain the generalized KL divergence $D_{\varphi}(q, p)
=
\int
q(X)
\cdot
\left[
\log q(X)
- 1
\right]
- 
q(X) \cdot \log p(X)
+
p(X)
dX
$ \citep{Uehara18arxiv}. It can be further reduced to MLE objective $-
\E_{X \sim q(X)}
\log
p(X)$ in case $p$ is normalized, as is shown in Section \ref{sec:MLE_PSO}.

Further, the term $\psi(\mu) \triangleq \int
\varphi
\left[
\mu(X)
\right]
dX
$ is called an entropy of $\mu$, and 
$
d_{\varphi}(q, p)
\triangleq
-
\psi(p)
-
\int
\varphi'
\left[
p(X)
\right]
\cdot 
\left[
q(X) - p(X)
\right]
dX
$ is referred to as a cross-entropy between $q$ and $p$. Since $D_{\varphi}(q, p) = d_{\varphi}(q, p) - d_{\varphi}(q, q)$, usually the actual optimization being solved is $\min_{p} d_{\varphi}(q, p)$ where the cross-entropy objective is approximated via Monte-Carlo integration.

\begin{claim}
\label{thrm:PSO_Bregman_rel} 
Consider a strictly convex and twice differentiable function $\varphi$ and two densities $\probs{\usuff}$ and $\probs{\dsuff}$ that satisfy $\spp^{\usuff} \subseteq \spp^{\dsuff}$. Likewise, define magnitudes $\{ M^{\usuff}
	\left[
	X,
	s
	\right]
	=
	\varphi''(s)$,
	$
	M^{\dsuff}
	\left[
	X,
	s
	\right]
	=
	\frac{s \cdot \varphi''(s)}{\probi{\dsuff}{X}}
	\}
	$ and a space of non-negative functions $\FF$. Then PSO functional and PSO divergence are equal to U-cross-entropy (up to an additive constant) and U-divergence defined by $\varphi$, respectively:
	\begin{enumerate}
		\item $\forall f \in \FF: L_{PSO}(f) = d_{\varphi}(\probs{\usuff}, f)$
		
		\item $\forall f \in \FF: D_{PSO}(\probs{\usuff}, f) = D_{\varphi}(\probs{\usuff}, f)$
		
	\end{enumerate}
\end{claim}

\begin{proof}

	First, \ps corresponding to the above \mfs are
	$\{ \widetilde{M}^{\usuff}
	\left[
	X,
	s
	\right]
	=
	\varphi'(s) + c_{\usuff},
	\widetilde{M}^{\dsuff}
	\left[
	X,
	s
	\right]
	=
	\frac{\left[
		s \cdot \varphi'(s)
		-
		\varphi(s)
		\right]
	}{\probi{\dsuff}{X}}
	+ c_{\dsuff}
	\}
	$, where $c_{\usuff}$ and $c_{\dsuff}$ are additive constants. Denote
$c_{\ttsuff} \triangleq c_{\dsuff} - c_{\usuff}$. Introducing above expressions into $L_{PSO}$ defined in Eq.~(\ref{eq:GeneralPSOLossFrmlPr_Limit_f}) will lead to:
	\begin{equation}
	\forall
	f \in \FF:
	\quad
	L_{PSO}(f)
	=
	\int
	- 
	\varphi
	\left[
	f(X)
	\right]
	+ 
	\varphi'
	\left[
	f(X)
	\right]
	\cdot 
	\left[
	f(X) - \probi{\usuff}{X}
	\right]
	dX
	+
	c_{\ttsuff}
	=
	d_{\varphi}(\probs{\usuff}, f)
	+
	c_{\ttsuff}
	.
	\label{eq:Breg_PSO_Func}
	\end{equation}
	Further, according to Table \ref{tbl:ClassicTransforms} the minimizer $f^*$ of $L_{PSO}(f)$ for a given \pair is $\probi{\usuff}{X}$. Therefore, Eq.~(\ref{eq:PSODiv}) leads to $D_{PSO}(f^*, f)
	\triangleq
	L_{PSO}(f)
	-
	L_{PSO}(f^*) = d_{\varphi}(\probs{\usuff}, f) - d_{\varphi}(\probs{\usuff}, f^*) = d_{\varphi}(\probs{\usuff}, f) - d_{\varphi}(\probs{\usuff}, \probs{\usuff}) = D_{\varphi}(\probs{\usuff}, f)$.
	
	Function $\varphi$ has to be twice differentiable in order to solve $\min_{p} d_{\varphi}(q, p)$ via gradient-based optimization - derivative of $d_{\varphi}(q, p)$ w.r.t. $p$ involves $\varphi''$. Hence, this property is typically satisfied by all methods that minimize Bregman divergence.
	Furthermore, if $\varphi''(s)$ is a continuous function then the specified \ms are feasible w.r.t. Theorem \ref{thrm:PSO_func_magn}. Otherwise, we can verify a more general set of conditions at the end of Section \ref{sec:FuncMutSupp_non_diff} which is satisfied for the claim's setting.

\end{proof}

From the above we can conclude that definitions of Bregman and PSO divergencies coincide when the former is limited to U-divergence and the latter is limited to \ms that satisfy $\frac{M^{\dsuff}\left[
	X,
	s
	\right]}{M^{\usuff}\left[
	X,
	s
	\right]} = \frac{s}{\probi{\dsuff}{X}}$. Such PSO subset has \bp at $f^*(X) = \probi{\usuff}{X}$ and is denoted in Section \ref{sec:PSOInstAllSets} as $\pdfestset$. It contains all PSO instances for the density estimation, not including log-density estimation methods in subset $\logpdfestset$. Further, since $L_{PSO}(f) = d_{\varphi}(\probs{\usuff}, f)$, a family of algorithms that minimize U-divergence and PSO subset $\pdfestset$ of density estimators are equivalent.
However, in general \psodiv is different from Bregman as it can also be used to measure a "distance" between any PSO solution $f^*$, including even negative functions such as log-density $\log \probi{\usuff}{X}$ and log-density-ratio $\log \frac{\probi{\usuff}{X}}{\probi{\dsuff}{X}}$ (see Section \ref{sec:PSOInst}),  and any $f \in \FF$ without the non-negative constraint over $\FF$.

\subsection{$f$-Divergence}
\label{sec:fDivv}

Define $q$ and $p$ to be probability densities so that $q$ is absolutely continuous w.r.t. $p$.
Then, $f$-divergence between $q$ and $p$ is defined as $D_{\phi}(q, p)
=
\int
p(X)
\cdot
\phi(\frac{q(X)}{p(X)})
dX
$,
where $\phi: \RR \rightarrow \RR$ is a convex and lower semicontinuous function s.t. $\phi(1) = 0$. This divergence family contains many important special cases, such as KL divergence (for $\phi(z) = z \cdot \log z$) and Jensen-Shannon divergence (for $\phi(z) = - (z + 1) \cdot \log \frac{1 + z}{2} + z \cdot \log z$).

In \citep{Nguyen10tit} authors proved that it is lower-bounded as:
\begin{equation}
D_{\phi}(q, p)
\geq
\sup_{f \in \FF}
J_{\phi}(f, q, p)
,
\quad
J_{\phi}(f, q, p)
\triangleq
\E_{X \sim q(X)}
f(X)
-
\E_{X \sim p(X)}
\phi^{c}
\left[
f(X)
\right]
,
\label{eq:FDiverg}
\end{equation}
where $\phi^{c}$ is the convex-conjugate function of $\phi$: $\phi^{c}(s) \triangleq \sup_{z \in \RR} \{ z \cdot s - \phi(z) \}$. The above expression becomes an equality when $\phi'(\frac{q(X)}{p(X)})$ belongs to a considered function space $\FF$. In such case, the supremum is obtained via $f(X) = \phi'(\frac{q(X)}{p(X)})$. Therefore, the above lower bound is widely used in the optimization $\max_{f} J_{\phi}(f, q, p)$ to approximate $D_{\phi}(q, p)$ \citep{Nguyen10tit}, and to learn any function of $\frac{q(X)}{p(X)}$ \citep{Nowozin16nips}.

\begin{claim}
\label{thrm:PSO_FDiv_rel} 
Consider a strictly convex and differentiable function $\phi$ s.t. $\phi(1) = 0$, its convex conjugate $\phi^{c}$, and two densities $\probs{\usuff}$ and $\probs{\dsuff}$ that satisfy $\spp^{\usuff} \subseteq \spp^{\dsuff}$. Likewise, define PSO primitives $\{ \widetilde{M}^{\usuff}
	\left[
	X,
	s
	\right]
	=
	s,
	\widetilde{M}^{\dsuff}
	\left[
	X,
	s
	\right]
	=
	\phi^{c}(s)
	\}
	$ and assume that $\phi'(\frac{\probi{\usuff}{X}}{\probi{\dsuff}{X}})$ is contained in a function space $\FF$. Then PSO functional and $f$-divergence have the following connection:
	\begin{enumerate}
		\item $\forall f \in \FF: L_{PSO}(f) \equiv -J_{\phi}(f,\probs{\usuff}, \probs{\dsuff})$ and\\ $f^* = \argmin_{f \in \FF} L_{PSO}(f) = \argmax_{f \in \FF} J_{\phi}(f,\probs{\usuff}, \probs{\dsuff}) = \phi'(\frac{\probi{\usuff}{X}}{\probi{\dsuff}{X}})$
		
		\item $\forall f \in \FF: D_{\phi}(\probs{\usuff}, \probs{\dsuff}) = - L_{PSO}(f^*) \geq - L_{PSO}(f)$
		
	\end{enumerate}
\end{claim}

\begin{proof}

First, introducing the given $\{ \widetilde{M}^{\usuff}
	,
	\widetilde{M}^{\dsuff}
	\}
	$ into Eq.~(\ref{eq:GeneralPSOLossFrmlPr_Limit_f}) leads to $L_{PSO}(f) \equiv -J_{\phi}(f,\probs{\usuff}, \probs{\dsuff})$. Further, since $\{ M^{\usuff}
	\left[
	X,
	s
	\right]
	=
	1,
	M^{\dsuff}
	\left[
	X,
	s
	\right]
	=
	\phi^{c}{'}(s)
	\}
	$ we have $\frac{M^{\dsuff}
		\left[
		X,
		s
		\right]
	}{M^{\usuff}
	\left[
	X,
	s
	\right]
} = \phi^{c}{'}(s)$. Denote $R	\left[
X,
s
\right]
=
\frac{M^{\dsuff}
	\left[
	X,
	s
	\right]
}{M^{\usuff}
\left[
X,
s
\right]
}
= \phi^{c}{'}(s)$ and observe that its inverse is $T
\left[
X,
z
\right]
= \phi'(z)$ due properties of LF transform. Applying Theorem \ref{thrm:PSO_func_magn}, we conclude $f^{*}(X) = \phi'(\frac{\probi{\usuff}{X}}{\probi{\dsuff}{X}})$.
Finally, since we assumed $f^{*} \in \FF$, we also have $D_{\phi}(\probs{\usuff}, \probs{\dsuff})
=
J_{\phi}(f^*, \probs{\usuff}, \probs{\dsuff}) = -L_{PSO}(f^*)$.

The required by claim strict convexity and differentiability of function $\phi$ are necessary in order to recover the estimation convergence, as was also noted in Section 5 of \citep{Nguyen10tit}. Likewise, these properties ensure that derivatives $\phi{'}$ and $\phi^{c}{'}$ are well-defined. Further, above \ms satisfy the requirements of PSO framework since the PSO derivation in Section \ref{sec:FuncMutSupp_non_diff} agrees with derivation in \citep{Nguyen10tit} when the mapping $G$ considered in our proof is equal to $G(X, s) = s$.

Finally, observe that given the above pair of \pair, the primitives $\widetilde{M}^{\usuff}$ and $\widetilde{M}^{\dsuff}$ can be recovered up to an additive constant. This constant does not affect the outcome of PSO inference since the optima $f^*$ of \psofunc stays unchanged. Yet, it changes the output of $L_{PSO}(f)$ and thus for arbitrary antiderivatives \widepair the identity $D_{\phi}(\probs{\usuff}, \probs{\dsuff}) = - L_{PSO}(f^*)$ will not be satisfied. To recover the correct additive constant, we can use information produced by $\phi(1) = 0$. Specifically, for the considered above setting the "proper" antiderivative $\widetilde{M}^{\dsuff}$ of $M^{\dsuff}$ is the one that satisfies $\widetilde{M}^{\dsuff} \left[ X, s' \right] = s'$ where $s' \triangleq \phi'(1)$.

\end{proof}

\begin{figure}
	\centering
	
	\begin{tabular}{cccc}
		
		\subfloat{\includegraphics[width=0.3\textwidth]{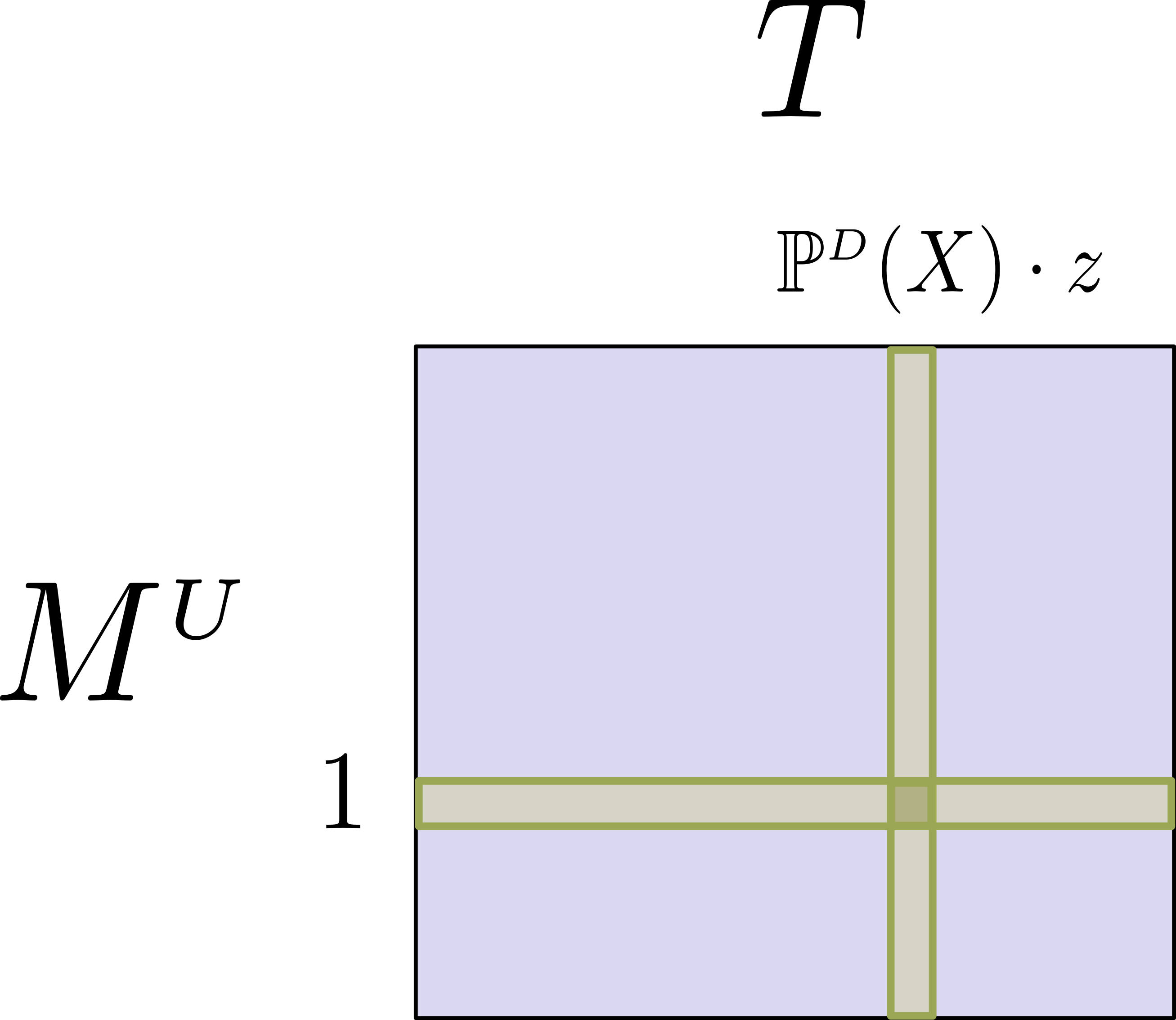}}
		
	\end{tabular}
	
	\protect
	\caption[Schematic relation between PSO and its two subgroups related in Claim \ref{thrm:PSO_Bregman_rel} and Claim \ref{thrm:PSO_FDiv_rel}.]{Schematic relation between PSO and its two subgroups related in Claim \ref{thrm:PSO_Bregman_rel} and Claim \ref{thrm:PSO_FDiv_rel}. All PSO instances within $\psoset$ can be indexed by the convergence $T$, where each $\psoset[T]$ can be further indexed by function $M^{\usuff}$ (given $T$ and $M^{\usuff}$, $M^{\dsuff}$ is derived as $M^{\dsuff} \left[ X, s \right] = T^{-1} \left[ X, s \right] \cdot M^{\usuff} \left[ X, s \right]$). Thus, all elements of $\psoset$ can be viewed as a table indexed by $T$ and $M^{\usuff}$. The column corresponding to $T \left[ X, z \right] = \probi{\dsuff}{X} \cdot z$ encapsulates PSO methods referred by Claim \ref{thrm:PSO_Bregman_rel} that minimize a separable Bregman divergence, and the row corresponding to $M^{\usuff} \left[ X, s \right] = 1$ - methods from Claim \ref{thrm:PSO_FDiv_rel} that approximate $f$-divergence. The column-row intersection is obtained at $M^{\usuff} \left[ X, s \right] = 1$, $M^{\dsuff} \left[ X, s \right] = \frac{s}{\probi{\dsuff}{X}}$ and $T \left[ X, z \right] = \probi{\dsuff}{X} \cdot z$.
	}
	\label{fig:DivRelations}
\end{figure}

Hence, in light of the above connection PSO framework can be seen as an estimation of a negative $f$-divergence between $\probs{\usuff}$ and $\probs{\dsuff}$, $-D_{\phi}(\probs{\usuff}, \probs{\dsuff})$.
Likewise, PSO is identical to an estimation framework of \citep{Nowozin16nips} when the former is limited to have \ms $\{ M^{\usuff}
\left[
X,
s
\right]
=
1,
M^{\dsuff}
\left[
X,
s
\right]
=
\phi^{c}{'}(s)
\}
$. Considering the terminology of Section \ref{sec:PSOInstAllSets}, such set of \mfs can be obtained by retrieving from each PSO subset $\psoset[T]$ an particular \mgn pair with $M^{\usuff}
\equiv 1$
and
$
M^{\dsuff}
\equiv
T^{-1}
$, with all collected pairs being equivalent to the estimation family of \citep{Nowozin16nips}, as shown in Figure \ref{fig:DivRelations}.
Due to such limitation this algorithm family is a strict subgroup of PSO.

Furthermore, the formulation of this framework in Eq.~(\ref{eq:FDiverg}) can be generalized as:
\begin{equation}
\min_{f \in \FF}
-
\E_{X \sim \probs{\usuff}}
G
\left[
X,
f(X)
\right]
+
\E_{X \sim \probs{\dsuff}}
\phi^{c}
\left[
X,
G
\left[
X,
f(X)
\right]
\right]
,
\label{eq:FDivergGeneral}
\end{equation}
where $f$ is replaced by
a transformation $G(X, s): \RR^n \times \RR \rightarrow \RR$ over the inner model $f$, and where $\phi^{c}$ is the convex-conjugate of $\phi ( X, z ): \RR^n \times \RR \rightarrow \RR$ that was extended to accept two arguments. Consequently, $G$ can be selected to obtain any required convergence $f^* = G^{-1} \odot \phi' \odot \frac{\probs{\usuff}}{\probs{\dsuff}}$
, where $\odot$ defines the composition via a second argument $g \odot h \Longleftrightarrow g(X, h(\cdot))$. 
Such extension parametrized by a pair $\{ \phi^c, G \}$ becomes identical to PSO parametrized by a pair \widepair, as is evident from our derivation in Section \ref{sec:FuncMutSupp_non_diff}. Particularly, the parametrization \widepair can be recovered from $\{ \phi^c, G \}$ via $\widetilde{M}^{\usuff} \left[ X, s \right] = G \left[ X, s \right]$ and $\widetilde{M}^{\dsuff} \left[ X, s \right] = \phi^c \left[ X, G \left[ X, s \right] \right]$.

In the backward direction, $\{ \phi^c, G \}$ can be recovered from \widepair as following. Denote the corresponding derivatives by \pair, and their ratio by $R \left[ X, s \right]$. Further, denote $s^{\bullet} \left[ X \right] \triangleq T \left[ X, 1 \right]$ where $T$ is an inverse of $R$ w.r.t. second argument.
Then, $\{ \phi^c, G \}$ can be obtained via $G \left[ X, s \right] = \widetilde{M}^{\usuff} \left[ X, s \right]$ and $\phi^c \left[ X, s \right] = \widetilde{M}^{\dsuff} \left[ X, G^{-1} \left[ X, s \right] \right] + \alpha \left[ X \right]$, with $\alpha \left[ X \right] \triangleq \widetilde{M}^{\usuff} \left[ X, s^{\bullet} \left[ X \right] \right] - \widetilde{M}^{\dsuff} \left[ X, s^{\bullet} \left[ X \right] \right]$. Note that the additive $X$-dependent component $\alpha \left[ X \right]$ ensures $\phi^c \left[ X, s' \right] = s'$ at $s' \triangleq \phi' \left[ X, 1 \right]$, which produces the condition $\phi \left[ X, 1 \right] = 0$ required by the definition of $f$-divergence. Thus, for the above "properly" normalized $\phi^c$  the expression in Eq.~(\ref{eq:FDivergGeneral}) is equal to $- D_{\phi}(\probs{\usuff}, \probs{\dsuff})$, where we extend $f$-divergence to have a form\footnote{
This extended modality can be shown to satisfy the non-negativity $D_{\phi}(q, p) \geq 0$ and the identification $D_{\phi}(q, p) = 0 \Leftrightarrow q \equiv p$, required by the statistical divergence definition. We leave the proof of that for future work since it is not the main focus of this paper.} $D_{\phi}(q, p)
=
\int
p(X)
\cdot
\phi \left[ X, \frac{q(X)}{p(X)} \right]
dX
$.

Moreover, the relation between parametrizations $\{ \phi^c, G \}$ and \widepair allows us to compute $D_{\phi}(\probs{\usuff}, \probs{\dsuff})$ from PSO loss $L_{PSO}(f)$ even in case of arbitrary antiderivatives \widepair that were not necessarily "properly" normalized. Specifically, $D_{\phi}(\probs{\usuff}, \probs{\dsuff})$ can be recovered from $L_{PSO}(f^*) \equiv \min_{f \in \FF} L_{PSO}(f)$ via:
\begin{equation}
D_{\phi}(\probs{\usuff}, \probs{\dsuff}) = - L_{PSO}(f^*) - \beta
,
\quad
\beta = \E_{X \sim \probs{\dsuff}}
\alpha \left[ X \right]
.
\label{eq:FDivergFromPSOLoss}
\end{equation}
Furthermore, if \widepair depend only on their second argument, which is often the case, then $T \left[ X, z \right]$ is also only a function of $z$ and the additive constant $\beta = \alpha = \widetilde{M}^{\usuff} \left[ T \left( 1 \right) \right] - \widetilde{M}^{\dsuff} \left[ T \left( 1 \right) \right]$ does not require the expected value computation.

The extension in Eq.~(\ref{eq:FDivergGeneral}) requires a remark about novelty and usefulness of PSO compared to works in \citep{Nguyen10tit,Nowozin16nips}. First, note that the generalization in Eq.~(\ref{eq:FDivergGeneral}) was not proposed by any study, to the best of our knowledge, although the original paper \citep{Nowozin16nips} was considering a transformation $G
\left[
f(X)
\right]$ to control the domain of particular $\phi^{c}$. Likewise, we note that the implied by $f$-divergence framework restriction $M^{\usuff}
\left[
X,
s
\right]
=
1$ over \up \mgn inevitably causes $M^{\dsuff}
\left[
X,
s
\right]$ to be an unbounded function. Such unboundedness property causes instability during optimization as discussed and empirically shown in sections \ref{sec:BoundUnboundMFs} and \ref{sec:Exper} respectively. In contrast, PSO formulation allows to easily construct bounded \ms for any desired target function by removing the aforementioned restriction (see Section \ref{sec:BoundUnboundMFs} for details).

Further, there are three main points of difference between two estimation procedures. 
PSO is defined via \pair, permitting $\{ \widetilde{M}^{\usuff}
,
\widetilde{M}^{\dsuff}
\}
$ to not have an analytical form which in turn increases the number of feasible estimators.
This allows us to propose novel techniques not considered before, such as PSO-LDE in Section \ref{sec:DeepLogPDF} which is shown in Section \ref{sec:Exper} to be superior over other state-of-the-art baselines.
Further, using $G
\left[
X,
f(X)
\right]$ instead of $G
\left[
f(X)
\right]$ is important to allow more freedom in selecting the required convergence $f^*$. Particularly, employing PSO to learn the density $\probi{\usuff}{X}$ is possible only under this two-argument setting.
Finally, PSO allows for a better intuition which is not available for the $f$-divergence formulation in \citep{Nowozin16nips}, and can further be used to analyze convergence outside of the mutual support $\spp^{\udcapsuff}$.

\subsection{Divergence Relation Summary}
\label{sec:DivSumm}

Due to above links PSO loss can be rewritten as $L_{PSO}(f) = D_{PSO}(f^*, f)
-
D_{\phi}(\probs{\usuff}, \probs{\dsuff}) - \beta$, where \psodiv $D_{PSO}(f^*, f)$ can be considered as an extension of separable Bregman divergence and where $D_{\phi}(\probs{\usuff}, \probs{\dsuff})$ is $f$-divergence defined via some $\phi$. Further, minimization of $L_{PSO}(f)$ is same as the minimization of \psodiv between $f^*$ and $f$. Finally, at the convergence $f = f^*$, $D_{PSO}(f^*, f^*)$ becomes zero and $L_{PSO}(f^*)$ is equal to $-
D_{\phi}(\probs{\usuff}, \probs{\dsuff}) - \beta$.

\section{Properties of PSO Estimators}
\label{sec:ConvA}

In above sections we saw that PSO principles are omnipresent within many statistical techniques. In this section we will investigate how these principles work in practice, by extending our analysis beyond the described above variational equilibrium. Particularly, we will study the consistency and the convergence of PSO, as also the actual equilibrium obtained via GD optimization. Furthermore, we will describe the model kernel impact on a learning task and emphasize the extreme similarity between PSO actual dynamics and the physical illustration in Figure \ref{fig:SurfForces}. Likewise, we will propose several techniques to improve the practical stability of PSO algorithms.

\subsection{Consistency and Asymptotic Normality}
\label{sec:Consist}

PSO solution $f_{\theta^*}$ is obtained by solving $\min_{f_{\theta} \in \FF} L_{PSO}(f_{\theta})$, and hence PSO belongs to the family of extremum estimators \citep{Amemiya85ae}. Members of this family are known to be consistent and asymptotically normal estimators under appropriate regularity conditions. Below we restate the corresponding theorems.

Herein, we will reduce the scope to PSO instances whose \up and \down densities have an identical support. This is required to avoid the special PSO cases for which $f_{\theta^*}$ has convergence at infinity outside of the mutual support (see theorems \ref{thrm:PSO_UD_conv} and \ref{thrm:PSO_DU_conv}). Although it is possible to evade such convergence by considering $\FF$ whose functions' output is lower and upper bounded (see Section \ref{sec:PSOStability}), we sidestep this by assuming $\spp^{\usuff} \equiv \spp^{\dsuff}$, for simplicity.

The empirical PSO loss over $N^{\usuff}$ samples from $\probs{\usuff}$ and $N^{\dsuff}$ samples from $\probs{\dsuff}$ has a form:
\begin{equation}
\hat{L}_{PSO}^{N^{\usuff},N^{\dsuff}}(f_{\theta})
=
-
\frac{1}{N^{\usuff}}
\sum_{i = 1}^{N^{\usuff}}
\widetilde{M}^{\usuff}
\left[
X^{\usuff}_{i},
f_{\theta}(X^{\usuff}_{i})
\right]
+
\frac{1}{N^{\dsuff}}
\sum_{i = 1}^{N^{\dsuff}}
\widetilde{M}^{\dsuff}
\left[
X^{\dsuff}_{i},
f_{\theta}(X^{\dsuff}_{i})
\right]
,
\label{eq:GeneralPSOLossEmpir}
\end{equation}
with its gradient $\nabla_{\theta}
\hat{L}_{PSO}^{N^{\usuff},N^{\dsuff}}(f_{\theta})
$ being defined in Eq.~(\ref{eq:GeneralPSOLossFrml}).

\begin{theorem}[Consistency]
\label{thrm:PSO_consist} 
Assume:
\begin{enumerate}
\item Parameter space $\Theta$ is a compact set.
\item $L_{PSO}(f_{\theta})$ (defined in Eq.~(\ref{eq:GeneralPSOLossFrmlPr_Limit_f})) is continuous on $\theta \in \Theta$.
\item $\exists! \theta^* \in \Theta, \forall X \in \spp^{\udcapsuff}: \probi{\usuff}{X} \cdot M^{\usuff}\left[X, f_{\theta^*}(X)\right] = \probi{\dsuff}{X} \cdot M^{\dsuff}\left[X, f_{\theta^*}(X)\right]$.
\item $\sup_{\theta \in \Theta} \left| \hat{L}_{PSO}^{N^{\usuff},N^{\dsuff}}(f_{\theta}) -  L_{PSO}(f_{\theta})\right| \pconv 0$ along with $\min(N^{\usuff},N^{\dsuff}) \rightarrow \infty$.
\end{enumerate}
Define $\hat{\theta}_{N^{\usuff},N^{\dsuff}} = \argmin_{\theta \in \Theta} \hat{L}_{PSO}^{N^{\usuff},N^{\dsuff}}(f_{\theta})$. 
Then $\hat{\theta}_{N^{\usuff},N^{\dsuff}}$ converges in probability to $\theta^*$ along with $\min(N^{\usuff},N^{\dsuff}) \rightarrow \infty$, $\hat{\theta}_{N^{\usuff},N^{\dsuff}} \pconv \theta^*$.
\end{theorem}

The above assumptions are typically taken by many studies to claim the estimation consistency.
Assumption 3 ensures that there is only single vector $\theta^*$ for which PSO \bp is satisfied. Hence it ensures that $L_{PSO}(f_{\theta})$ is uniquely minimized at $\theta^*$, according to Theorem \ref{thrm:PSO}. 
Assumption 4 is the uniform convergence of empirical loss towards its population form along with $\min(N^{\usuff},N^{\dsuff}) \rightarrow \infty$.
The above consistency statement and its proof appear as theorem 2.1 in \citep{Newey94book}.

Theorem \ref{thrm:PSO_consist} ensures the estimation consistency of PSO under technical conditions over space $\Theta$. However, some of these conditions (e.g. compactness of $\Theta$) are not satisfied by NNs. Another way, taken by \citep{Nguyen10tit,Kanamori12ml}, is to use a complexity metric defined directly over the space $\FF$, such as the bracketing entropy. We shall leave this more sophisticated consistency derivation for future work.

For asymptotic normality we will apply the following statements where we use notations $\II_{\theta}(X, X') \triangleq \nabla_{\theta} 
f_{\theta}(X) \cdot \nabla_{\theta} 
f_{\theta}(X')^T$, $N \triangleq N^{\usuff} + N^{\dsuff}$, $\tau \triangleq \frac{N^{\usuff}}{N^{\dsuff}}$, $M^{\usuff}{'}(X, s) \triangleq \frac{\partial M^{\usuff}(X, s)}{\partial s}$ and
$M^{\dsuff}{'}(X, s) \triangleq \frac{\partial M^{\dsuff}(X, s)}{\partial s}$.
\begin{lemma}[Hessian]
	\label{lmm:PSO_second_der} 
Assume $\exists \theta^* \in \Theta$ s.t. $\forall X \in \spp^{\udcapsuff}: \probi{\usuff}{X} \cdot M^{\usuff}\left[X, f_{\theta^*}(X)\right] = \probi{\dsuff}{X} \cdot M^{\dsuff}\left[X, f_{\theta^*}(X)\right]$. Denote PSO convergence as $f_{\theta^*}(X) = T\left[X, \frac{\probi{\usuff}{X}}{\probi{\dsuff}{X}}\right] \equiv f^{*}(X)$, according to Theorem \ref{thrm:PSO_func_magn}.
Then the Hessian $\hess$ of population PSO loss at $\theta^*$ has a form:
\begin{equation}
\hess \equiv 
\nabla_{\theta\theta}
L_{PSO}(f_{\theta^*})
=
-
\E_{X \sim \probs{\usuff}}
M^{\usuff}{'}
\left[
X,
f^{*}(X)
\right]
\cdot
\II_{\theta^{*}}(X, X)
+
\E_{X \sim \probs{\dsuff}}
M^{\dsuff}{'}
\left[
X,
f^{*}(X)
\right]
\cdot
\II_{\theta^{*}}(X, X)
.
\label{eq:GeneralPSOLossHessian}
\end{equation}
\end{lemma}

\begin{lemma}[Gradient Variance]
	\label{lmm:PSO_infor} 
	Under the same setting,
	the variance of\\ $\nabla_{\theta}
	\hat{L}_{PSO}^{N^{\usuff},N^{\dsuff}}(f_{\theta})$ at $\theta^*$ has a form $	\variance
	\left[
	\nabla_{\theta}
	\hat{L}_{PSO}^{N^{\usuff},N^{\dsuff}}(f_{\theta^*})
	\right]
	= \frac{1}{N} \JJ
	$ where:
	\begin{multline}
\JJ
=
\frac{\tau + 1}{\tau}
\E_{X \sim \probs{\usuff}}
\left[
M^{\usuff}
\left[
X,
f^*(X)
\right]^2
\cdot
\II_{\theta^*}(X, X)
\right]
+\\
+
(\tau + 1)
\E_{X \sim \probs{\dsuff}}
\left[
M^{\dsuff}
\left[
X,
f^*(X)
\right]^2
\cdot
\II_{\theta^*}(X, X)
\right]
-\\
-
\frac{(\tau + 1)^2}{\tau}
\E_{\substack{X \sim \probs{\usuff}\\ X' \sim \probs{\dsuff}}}
M^{\usuff}
\left[
X,
f^{*}(X)
\right]
\cdot
M^{\dsuff}
\left[
X',
f^{*}(X')
\right]
\cdot
\II_{\theta^{*}}(X, X')
.
	\label{eq:GeneralPSOLossGradientOuterPr}
	\end{multline}
\end{lemma}
$\hess$ and $\JJ$ have several additional forms that appear in Appendix \ref{sec:PSO_derivatives_proof} along with the lemmas' proofs.

\begin{theorem}[Asymptotic Normality]
	\label{thrm:PSO_asymp_norm} 
Given assumptions and definitions of Theorem \ref{thrm:PSO_consist}, assume additionally:
	\begin{enumerate}
		\item $\theta^* \in \text{int}(\Theta)$.
		\item $\hat{L}_{PSO}^{N^{\usuff},N^{\dsuff}}(f_{\theta})$ is twice continuously differentiable in a neighborhood of $\theta^*$.
		\item $\nabla_{\theta\theta}
		L_{PSO}(f_{\theta})$ is continuous in $\theta$.
		\item Both $\JJ$ and $\hess$ are non-singular matrices.
		\item Each entry of $\nabla_{\theta\theta}
		L_{PSO}(f_{\theta})$ and of $\JJ$ is uniformly bounded by an integrable function.
		\item $\tau > 0$ is a finite and fixed scalar.
	\end{enumerate}
Then $\sqrt{N}
\cdot
\left[
\hat{\theta}_{N^{\usuff},N^{\dsuff}}
-
\theta^*
\right]
$ converges in distribution to $\NN(0, 
\hess^{-1} \JJ \hess^{-1})$
along with $N \rightarrow \infty$.

\end{theorem}

Proof of the above theorem is in Appendix \ref{sec:AsympNormalProof}. Thus, we can see that for large sample sizes the parametric estimation error $\sqrt{N^{\usuff} + N^{\dsuff}}
\cdot
\left[
\hat{\theta}_{N^{\usuff},N^{\dsuff}}
-
\theta^*
\right]
$ is normal with zero mean and covariance matrix $\Sigma = \hess^{-1} \JJ \hess^{-1}$. Further, it is possible to simplify an expression for $\Sigma$ when considering \pair of a specific PSO instance. Observe also that none of the above theorems and lemmas require the analytical knowledge of $\{ \widetilde{M}^{\usuff}
,
\widetilde{M}^{\dsuff}
\}
$.

\subsection{Bounded vs Unbounded Magnitude Functions}
\label{sec:BoundUnboundMFs}

Any considered functions \pair will produce PSO \bp at the convergence, given that they satisfy conditions of Theorem \ref{thrm:PSO_func_magn}. Further, as was shown in Section \ref{sec:PSOInst}, there is infinite number of possible \ms within $\psoset[T]$ that will lead to the same convergence $T$. Thus, we need analytical tools to establish the superiority of one \mf pair over another. While this is an advanced estimation topic from robust statistics and is beyond this paper's scope, herein we describe one desired property for these functions to have - boundedness.

In Tables \ref{tbl:PSOInstances1}-\ref{tbl:PSOInstances5} we can see many choices of \pair, with both bounded (e.g. NCE in Table \ref{tbl:PSOInstances1}) and unbounded (e.g. IS in Table \ref{tbl:PSOInstances1}) outputs.
Further, note that IS method has a \emph{magnitude} function $M^{\dsuff}\left[X,f_{\theta}(X)\right] = \frac{\exp (f_{\theta}(X))}{\probi{\dsuff}{X}}$ with outputs that can be extremely high or low, depending on the difference $[f_{\theta}(X) - \log \probi{\dsuff}{X}]$. From Eq.~(\ref{eq:GeneralPSOLossFrml}) we can likewise see that the gradient contribution of the \down term in PSO loss is $M^{\dsuff}\left[X,f_{\theta}(X)\right] \cdot \nabla_{\theta} f_{\theta}(X)$. Thus, in case $f_{\theta}(X) \gg \log \probi{\dsuff}{X}$,  high values from $M^{\dsuff}(\cdot)$ will produce gradients with large norm. Such large norm causes instability during the optimization, known as exploding gradients problem in DL community. Intuitively, when we make a large step inside $\theta$-space, the consequences can be unpredictable, especially for highly non-linear models such as modern NNs.

Therefore, in practice the loss with bounded gradient is preferred. Such conclusion was also empirically supported in context of unnormalized density estimation \citep{Pihlaja12arxiv}. Further, while it is possible to solve this issue by for example gradient clipping and by decreasing the learning rate \citep{Pascanu12corr}, such solutions also slow down the entire learning process.
PSO framework allows to achieve the desired gradient boundedness by bounding \mfs' outputs via the following lemma.

\begin{lemma}[PSO Consistent Modification]
	\label{lmm:PSO_consist_modific} 
	Consider any PSO instance $[M^{\usuff}, M^{\dsuff}] \in \psoset$ with the corresponding convergence interval $\KKK$ on which criteria of Theorem \ref{thrm:PSO_func_magn} hold. Define $D(X, s): \RR^n \times \KKK \rightarrow \RR$ to be a continuous and positive function on $s \in \KKK$ for any $X \in \spp^{\udcapsuff}$. Then, $[M^{\usuff}, M^{\dsuff}]$ and $[\frac{M^{\usuff}}{D}, \frac{M^{\dsuff}}{D} ]$ are PSO consistent (i.e. have identical convergence).
\end{lemma}

The proof is trivial, by noting that the ratio of both pairs is preserved which places them into the same PSO subset $\psoset[T]$. The feasibility of the second pair is a result of the first pair's feasibility and of the fact that a devision by $D$ does not change the sign and continuity of \mfs.

To produce bounded \ms, consider any pair of functions $\{ M^{\usuff}\left[X, s \right], M^{\dsuff}\left[X, s \right] \}$ with some desired PSO convergence $T$. In case these are unbounded functions, a new pair of bounded functions can be constructed as 
\begin{equation}
M^{\usuff}_{bounded}\left[X, s \right]
=
\frac{M^{\usuff}\left[X, s \right]}{|M^{\usuff}\left[X, s \right]| + |M^{\dsuff} \left[X, s \right]|}
,
\quad
M^{\dsuff}_{bounded}\left[X, s \right]
=
\frac{M^{\dsuff}\left[X, s \right]}{|M^{\usuff}\left[X, s \right]| + |M^{\dsuff} \left[X, s \right]|}
.
\label{eq:BoundedFuncTransform}
\end{equation}
It is clear from their structure that the new functions' outputs are in $[-1, 1]$. Furthermore, their PSO convergence will be identical to the one of the original pair, due to Lemma \ref{lmm:PSO_consist_modific}. Similarly, we can replace the absolute value also by other norms.

\emph{Magnitudes} of many popular estimation methods (e.g. NCE, logistic loss, cross-entropy) are already bounded, making them more stable compared to unbounded variants. This may explain their wide adaptation in Machine Learning community.
In Section \ref{sec:DeepLogPDF} we will use the above transformation to develop a new family of robust log-pdf estimators.

\subsection{Statistics of Surface Change}
\label{sec:DiffStats}

Herein we will analyze statistical properties of $f_{\theta}$'s evolution during the optimization.
To this end, define the \df $df_{\theta_t}(X) \triangleq f_{\theta_{t + 1}}(X) - f_{\theta_{t}}(X)$ as a change of $f_{\theta}(X)$ after $t$-th GD iteration. Its first-order Taylor approximation is:
\begin{multline}
df_{\theta}(X)
\approx
- \delta \cdot 
\nabla_{\theta} f_{\theta}(X)^T
\cdot
\nabla_{\theta}
\hat{L}_{PSO}^{N^{\usuff},N^{\dsuff}}(f_{\theta})
=\\
=
\delta \cdot 
\Bigg[
\frac{1}{N^{\usuff}}
\sum_{i = 1}^{N^{\usuff}}
M^{\usuff}\left[X^{\usuff}_{i},f_{\theta}(X^{\usuff}_{i})\right]
\cdot
g_{\theta}(X, X^{\usuff}_{i})
-
\frac{1}{N^{\dsuff}}
\sum_{i = 1}^{N^{\dsuff}}
M^{\dsuff}\left[X^{\dsuff}_{i},f_{\theta}(X^{\dsuff}_{i})\right]
\cdot
g_{\theta}(X, X^{\dsuff}_{i})
\Bigg]
,
\label{eq:PSODffrntl}
\end{multline}
where $\delta$ is the learning rate, value of $\theta$ is the one before GD iteration, and $\nabla_{\theta}
\hat{L}_{PSO}^{N^{\usuff},N^{\dsuff}}(f_{\theta})$ is the loss gradient. When the considered model is NN, the above first-order dynamics are typically a very good approximation of the real $df_{\theta}(X)$ \citep{Jacot18nips,Kopitkov19arxiv_spectrum}. Further, when $f_{\theta}$ belongs to RKHS, the above approximation becomes an identity. Therefore, below we will treat Eq.~(\ref{eq:PSODffrntl}) as an equality, neglecting the fact that this is only the approximation.

\begin{theorem}[Differential Statistics]
	\label{thrm:PSO_diff_stats_props} 
Denote $F_{\theta}^{\tsuff}(X)
\triangleq
\probi{\usuff}{X} \cdot 
M^{\usuff}\left[X,f_{\theta}(X)\right]
-
\probi{\dsuff}{X} \cdot 
M^{\dsuff}\left[X,f_{\theta}(X)\right]
$ to be a difference of two point-wise forces $F_{\theta}^{\usuff}$ and $F_{\theta}^{\dsuff}$ defined in Section \ref{sec:PromOptim}, $F_{\theta}^{\tsuff}(X) = F_{\theta}^{\usuff}(X) - F_{\theta}^{\dsuff}(X)$. Additionally, define $G_{\theta}$ to be the integral operator $[G_{\theta} u](\cdot) = \int
g_{\theta}(\cdot, X)
u(X)
dX
$.
Then, considering training points as random i.i.d. realizations from the corresponding densities $\probs{\usuff}$ and $\probs{\dsuff}$, the expected value and the covariance of $df_{\theta}(X)$ at any fixed $\theta$ are:
\begin{multline}
\expectv
\left[
df_{\theta}(X)
\right]
=
\delta \cdot 
\bigg[
\E_{X' \sim \probs{\usuff}}
\Big[
M^{\usuff}\left[X',f_{\theta}(X')\right]
\cdot
g_{\theta}(X, X')
\Big]
-\\
-
\E_{X' \sim \probs{\dsuff}}
\Big[
M^{\dsuff}\left[X',f_{\theta}(X')\right]
\cdot
g_{\theta}(X, X')
\Big]
\bigg]
=
\delta \cdot
\int
g_{\theta}(X', X)
\cdot
F_{\theta}^{\tsuff}(X')
dX'
=
\delta \cdot
[G_{\theta} F_{\theta}^{\tsuff}](X)
,
\label{eq:PointChangePSOLossExpV}
\end{multline}
\begin{equation}
\cov
\left[
df_{\theta}(X),
df_{\theta}(X')
\right]
=
\delta^2 \cdot 
\nabla_{\theta} 
f_{\theta}(X)^T
\cdot
\variance
\left[
\nabla_{\theta}
\hat{L}_{PSO}^{N^{\usuff},N^{\dsuff}}(f_{\theta})
\right]
\cdot
\nabla_{\theta} 
f_{\theta}(X')
,
\label{eq:PointChangePSOLossCovar}
\end{equation}
with $\variance
\left[
\nabla_{\theta}
\hat{L}_{PSO}^{N^{\usuff},N^{\dsuff}}(f_{\theta})
\right]
$ being proportional to $\frac{1}{N^{\usuff} + N^{\dsuff}}$.

\end{theorem}

Proof of the above theorem together with the explicit form of $\variance
\left[
\nabla_{\theta}
\hat{L}_{PSO}^{N^{\usuff},N^{\dsuff}}(f_{\theta})
\right]$ appears in Appendix \ref{sec:PSO_diff_stats_props_proof}.
The theorem provides insights about stochastic dynamics of the surface $f_{\theta}$ caused by point-wise forces $F_{\theta}^{\usuff}(X)$ and $F_{\theta}^{\dsuff}(X)$.
Eq.~(\ref{eq:PointChangePSOLossExpV}) is an \emph{integral transform} of the total force $F_{\theta}^{\tsuff}(X)$, which can also be seen as a point-wise error. That is, on the average $df_{\theta}(X)$ changes proportionally to the convolution of $\big[
F_{\theta}^{\usuff}(X')
-
F_{\theta}^{\dsuff}(X')
\big]$ w.r.t. the kernel $g_{\theta}(X', X)$. When $F_{\theta}^{\usuff}$ is larger than $F_{\theta}^{\dsuff}$, after both being convolved via $g_{\theta}$, then on the average $f_{\theta}(X)$ is pushed \up, and vice versa. When convolutions of $F_{\theta}^{\usuff}$ and $F_{\theta}^{\dsuff}$ are equal around the point $X$, $f_{\theta}(X)$ stays constant, again on the average. 
Further, the variance of $df_{\theta}(X)$ depends on the alignment between $\nabla_{\theta} 
f_{\theta}(X)$ and the eigenvectors of $\variance
\left[
\nabla_{\theta}
\hat{L}_{PSO}^{N^{\usuff},N^{\dsuff}}(f_{\theta})
\right]$ that correspond to the largest eigenvalues (the directions in $\theta$-space to which PSO loss mostly propagates), and in addition can be reduced by increasing the size of training datasets.

Neglecting higher moments of random variable $df_{\theta}(X)$, this \df can be expressed as the sum $df_{\theta}(X)
=
\delta \cdot
[G_{\theta} F_{\theta}^{\tsuff}](X)
+
\omega_{\theta}(X)
$
where $\omega_{\theta}$ is a zero-mean indexed-by-$X$ stochastic process with the covariance function defined in Eq.~(\ref{eq:PointChangePSOLossCovar}). Such evolution implies that
$f_{\theta}(X)$ will change on average towards the height where the \emph{convoluted} PSO equilibrium $[G_{\theta} F_{\theta}^{\tsuff}](\cdot) = 0 \Leftrightarrow G_{\theta} F_{\theta}^{\usuff} = G_{\theta} F_{\theta}^{\dsuff}$ is satisfied:
\begin{equation}
\int
g_{\theta}(X, X')
\cdot
F_{\theta}^{\usuff}(X')
dX'
=
\int
g_{\theta}(X, X')
\cdot
F_{\theta}^{\dsuff}(X')
dX'
,
\label{eq:ConvPSOBalStatePopul}
\end{equation}
while $\omega_{\theta}$'s variance will cause it to vibrate around such target height.
Likewise, informally $\delta$ in Eq.~(\ref{eq:PointChangePSOLossCovar}) has a role of configuration parameter that controls the diapason around the target height where the current estimation $f_{\theta}(X)$ is vibrating.
Further, sequential tuning/decaying of the learning rate $\delta$ will decrease this vibration amplitude (distance between the function that satisfies Eq.~(\ref{eq:ConvPSOBalStatePopul}) and the current model).

Furthermore, considering a large training dataset regime where $\omega_{\theta}$'s effect is insignificant, and replacing $\theta$ by the iteration time $t$, dynamics of the model can be written as $df_{t}(X) = f_{t + 1}(X) - f_{t}(X) = \delta \cdot
[G_{t} F_{t}^{\tsuff}](X)
$, or as $f_{t + 1} = f_{t} + \delta \cdot
G_{t} F_{t}^{\tsuff}$. Here, $F_{t}^{\tsuff}$ represents a negative functional derivative - the steepest descent direction of loss $L_{PSO}$ in the function space. Further, $G_{t}$ is GD operator that stretches and shrinks $F_{t}^{\tsuff}$ according to the alignment of the latter with eigenfunctions of the model kernel. Hence, it serves as a metric over the function space, defining what directions are "fast" to move in and in which directions it is "slow".

\subsection{Convoluted PSO Balance State}
\label{sec:ConvPSOBalState}

Above we observed that in fact GD optimization will lead to the \emph{convoluted} PSO equilibrium. This can also be derived from the first-order-conditions argument as follows. Assume that at the convergence of PSO algorithm $\nabla_{\theta}
\hat{L}_{PSO}^{N^{\usuff},N^{\dsuff}}(f_{\theta}) = 0$ is satisfied, with $\theta$ containing final parameter values. Multiplying it by $\nabla_{\theta} f_{\theta}(X)^T$ will lead to:
\begin{equation}
\frac{1}{N^{\usuff}}
\sum_{i = 1}^{N^{\usuff}}
M^{\usuff}\left[X^{\usuff}_{i},f_{\theta}(X^{\usuff}_{i})\right]
\cdot
g_{\theta}(X, X^{\usuff}_{i})
=
\frac{1}{N^{\dsuff}}
\sum_{i = 1}^{N^{\dsuff}}
M^{\dsuff}\left[X^{\dsuff}_{i},f_{\theta}(X^{\dsuff}_{i})\right]
\cdot
g_{\theta}(X, X^{\dsuff}_{i})
.
\label{eq:ConvPSOBalStateEmpir}
\end{equation}
According to the weak law of large numbers the above equality converges in probability (under appropriate regularity conditions) to Eq.~(\ref{eq:ConvPSOBalStatePopul}) as batch sizes $N^{\usuff}$ and $N^{\dsuff}$ increase.

Further, considering the asymptotic equilibrium $[G_{\theta} F_{\theta}^{\tsuff}](\cdot) = 0$, if the $\theta$-dependent operator $G_{\theta}$ is injective then we also have $F_{\theta}^{\tsuff}(\cdot) = 0$ which leads to $F_{\theta}^{\usuff} \equiv F_{\theta}^{\dsuff}$ and to the variational PSO \bp in Eq.~(\ref{eq:BalPoint}). Yet, in general case during the optimization $F_{\theta}^{\tsuff}$ will project into the null-space of $G_{\theta}$, with $F_{\theta}^{\tsuff} \not\equiv 0$. Moreover, even for injective $G_{\theta}$ it may take prohibitively many GD iterations in order to obtain the \emph{convoluted} PSO equilibrium in Eq.~(\ref{eq:ConvPSOBalStatePopul}), depending on the conditional number of $G_{\theta}$. Hence, at the end of a typical optimization $F_{\theta}^{\tsuff}$
is expected to be outside of (orthogonal to) $G_{\theta}$'s high-spectrum space (i.e. a space spanned by $G_{\theta}$'s eigenfunctions related to its highest eigenvalues), and to be within $G_{\theta}$'s null-space or its low-spectrum space (associated with the lowest eigenvalues).

Since the low-spectrum is affiliated with high-frequency functions \citep{Ronen19nips}, $F_{\theta}^{\tsuff} \not\equiv 0$ will resemble some sort of a noise function. That is, during GD optimization the Fourier transform $\hat{F}_{\theta}^{\tsuff}(\xi)$ of $F_{\theta}^{\tsuff}(X)$ is losing its energy around the origin, yet it almost does not change at frequencies $\xi$ whose norm is large. See \citep{Kopitkov19arxiv_spectrum} for the empirical evidence of the above conclusions and for additional analysis of $G_{\theta}$'s role in a least-squares optimization.

The above described behavior implicitly introduces a bias into PSO solution. Namely, when a function space $\FF$ with particular $g_{\theta}$  and $G_{\theta}$ is chosen, this decision will affect our learning task exactly via the above relation to $G_{\theta}$'s null-space. The deeper analysis is required to answer the following questions: How closely are two equilibriums in Eq.~(\ref{eq:ConvPSOBalStatePopul}) and Eq.~(\ref{eq:ConvPSOBalStateEmpir})? What is the impact of batch sizes $N^{\usuff}$ and $N^{\dsuff}$, and what can we say about PSO solution when these batches are finite/small?
How $g_{\theta}$'s properties, specifically its eigenvalues and eigenfunctions, will effect the PSO solution? What is the rate of converge towards $G_{\theta}$'s null-space? And how these aspects behave in a setting of stochastic mini-batch optimization? These advanced questions share their key concepts with the topics of RKHS estimation and Deep Learning theory, and deserve their own avenue. Therefore, we leave most of them out of this paper's scope and shall address them as part of future research. Further, below we analyze a specific property of $g_{\theta}$ - its bandwidth.

\subsection{Model Expressiveness and Smoothness vs Kernel Bandwidth}
\label{sec:ExprrKernelBnd}

The model kernel $g_{\theta}(X, X')$ expresses the impact over $f_{\theta}(X)$ when we optimize at a data point $X'$.
Intuitively, under the physical perspective (see Section \ref{sec:PromOptim})
PSO algorithm can be viewed as pushing (\up and \down) at the training points with some wand whose end's shape is described by the above kernel.
Here we will show that the flexibility of the surface $f_{\theta}$ strongly depends on $g_{\theta}(X, X')$'s bandwidth (i.e. on flatness of the pushing wand's end).

For this purpose we will define a notion of the model relative kernel:
\begin{equation}
r_{\theta}(X, X')
\triangleq
\frac{g_{\theta}(X, X')}{g_{\theta}(X, X)}
,
\label{eq:RelKernel}
\end{equation}
which can be interpreted as a \emph{relative} side-influence over $f_{\theta}(X)$ from $X'$, scaled w.r.t. the self-influence $g_{\theta}(X, X)$ of $X$. Further, assume that the model relative kernel is bounded as:
\begin{equation}
0
<
\exp
\left[
- \frac{d(X, X')}{h_{min}}
\right]
\leq
r_{\theta}(X, X')
\leq
\exp
\left[
- \frac{d(X, X')}{h_{max}}
\right]
\leq
1
,
\label{eq:RelKernelBounds}
\end{equation}
where $d(X, X')$ is any function that satisfies the triangle inequality $d(X, X') \leq d(X, X'') + d(X', X'')$ (e.g. metric or pseudometric over $\RR^n$), and where $h_{min}$ and $h_{max}$ can be considered as lower and upper bounds on $r_{\theta}(X, X')$'s bandwidth. Below, we will explore how $h_{min}$ and $h_{max}$ effect the smoothness of $f_{\theta}$. Note, that in case $g_{\theta}(X, X)$ is identical for any $X$, the below analysis can be performed w.r.t. properties of $g_{\theta}$ instead of $r_{\theta}$. However, in case $f_{\theta}$ is NN, the NTK $g_{\theta}$ can not be bounded as in Eq.~(\ref{eq:RelKernelBounds}), yet its scaled version $r_{\theta}$ clearly manifests such bounded-bandwidth properties. That is, $r_{\theta}(X, X')$ of NN typically decreases when the distance between $X$ and $X'$ increases (see Section \ref{sec:NNArch}), exhibiting some implicit bandwidth induced by a NN architecture. Moreover, while the magnitude of $g_{\theta}$ can be quiet different for various NN models and architectures, the normalized $r_{\theta}$ is on the same scale, allowing to compare the smoothness properties of particular models. Likewise, Eq.~(\ref{eq:RelKernelBounds}) is satisfied by many popular kernels used for RKHS construction, such as Gaussian and Laplacian kernels, making the below analysis relevant also for kernel models.

\begin{theorem}[Spike Convergence]
	\label{thrm:PSO_peak_convergence} 
	Assume:
	\begin{enumerate}
		\item PSO algorithm converged, with $\nabla_{\theta}
		\hat{L}_{PSO}^{N^{\usuff},N^{\dsuff}}(f_{\theta}) = 0$.
		\item \pair are non-negative functions.
		\item $M^{\usuff}\left[ X', s\right]$ is continuous and strictly decreasing w.r.t. $s$, at $\forall X' \in \spp^{\usuff}$.
	\end{enumerate}
	Denote by $(M^{\usuff})^{-1}
	\left[
	X',
	z
	\right]$ the inverse function of $M^{\usuff}$ w.r.t. second argument. Further, consider any training sample from $\probs{\usuff}$ and denote it by $X$. 
	Then the following is satisfied:
	\begin{enumerate}
		\item $f_{\theta}(X)
		\geq
		(M^{\usuff})^{-1}
		\left[
		X,
		\alpha
		\right]$ where $\alpha = \frac{N^{\usuff}}{N^{\dsuff}}
		\sum_{i = 1}^{N^{\dsuff}}
		M^{\dsuff}
		\left[
		X^{\dsuff}_{i},
		f_{\theta}(X^{\dsuff}_{i})
		\right]
		\cdot
		\exp
		\left[
		- \frac{d(X, X^{\dsuff}_{i})}{h_{max}}
		\right]$.
		
		\item $(M^{\usuff})^{-1}
		\left[
		X,
		z
		\right]$ is strictly decreasing w.r.t. $z$, with $(M^{\usuff})^{-1}
		\left[
		X,
		\alpha
		\right]
		\rightarrow 
		\infty
		$ when $\alpha \rightarrow 0$.
	\end{enumerate}
	
\end{theorem}

Proof of the above theorem is in Appendix \ref{sec:PSO_peak_convergence_proof}. From it we can see that for smaller $\alpha$ the surface at any \up training point $X$ converges to some very high height, where $f_{\theta}(X)$ can be arbitrary big. This can happen when $X$ is faraway from all \down samples $\{ X^{\dsuff}_{i} \}_{i = 1}^{N^{\dsuff}}$ (i.e. causing $d(X, X^{\dsuff}_{i})$ to have a large output), or when $h_{max}$ is very small (i.e. $r_{\theta}$ has a very narrow bandwidth). In these cases we will have \up spikes at locations $\{ X^{\usuff}_{i} \}_{i = 1}^{N^{\usuff}}$. Similarly, when $h_{max}$ is very small, there will be \down spikes at locations $\{ X^{\dsuff}_{i} \}_{i = 1}^{N^{\dsuff}}$ - the corresponding theorem is symmetric to Theorem \ref{thrm:PSO_peak_convergence} and thus is omitted. Hence, the above theorem states that a very narrow bandwidth of $r_{\theta}$ will cause at the convergence \up and \down spikes within the surface $f_{\theta}$, which can be interpreted as an overfitting behavior of PSO algorithm (see Section \ref{sec:DeepLogPDFOF} for the empirical demonstration). Note that the theorem's assumptions are not very restrictive, with many PSO instances satisfying them such as PSO-LDE, NCE and logistic loss (see Section \ref{sec:PSOInst}). Providing a more general theorem with less assumptions (specifically reducing the assumption 3) we shall leave for future work.

\begin{figure}[tb]
	\centering
	
	\begin{tabular}{cc}

		\subfloat[\label{fig:FlexBnd-a}]{\includegraphics[width=0.4\textwidth]{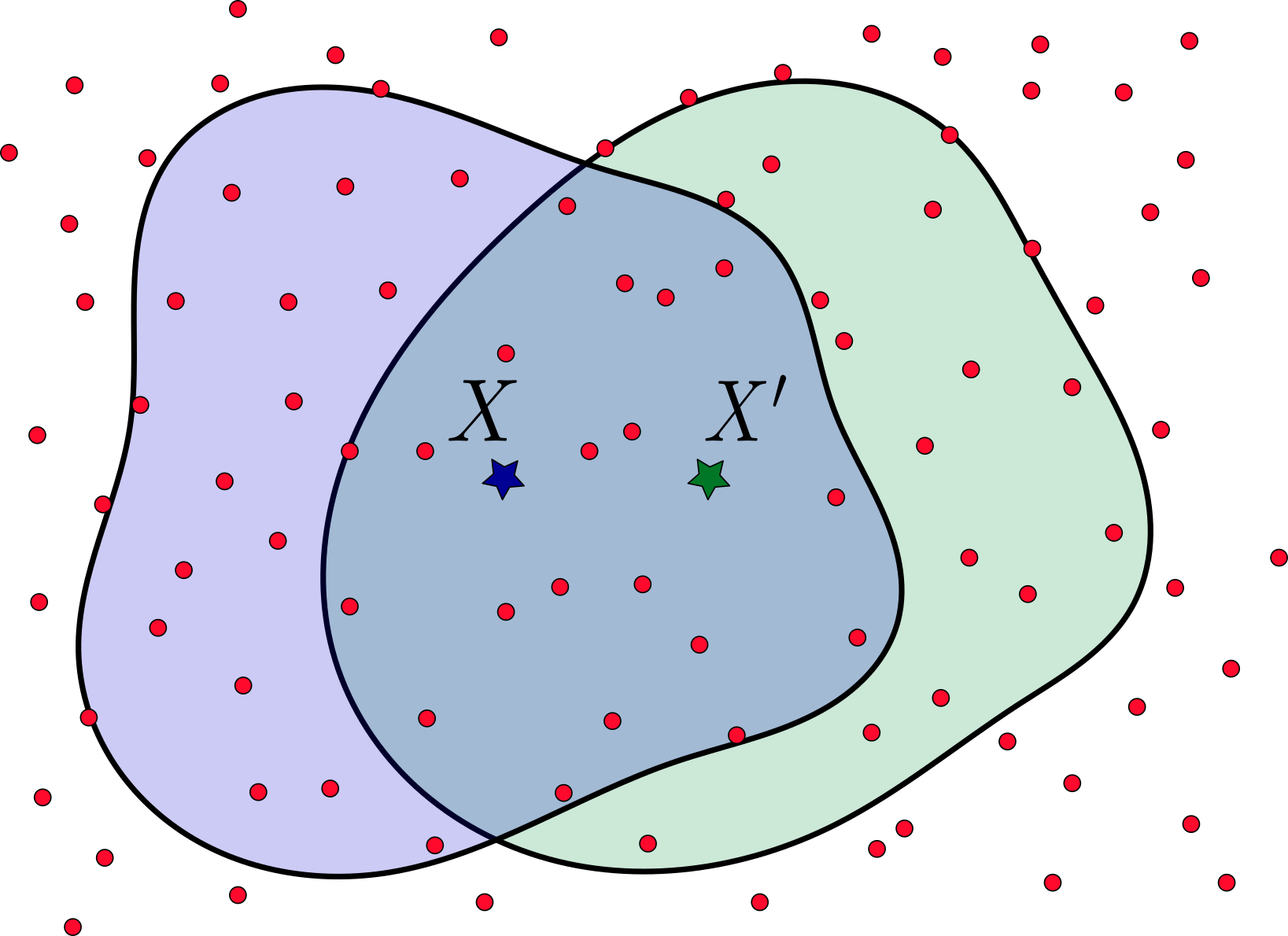}}
		&
		
		\subfloat[\label{fig:FlexBnd-b}]{\includegraphics[width=0.4\textwidth]{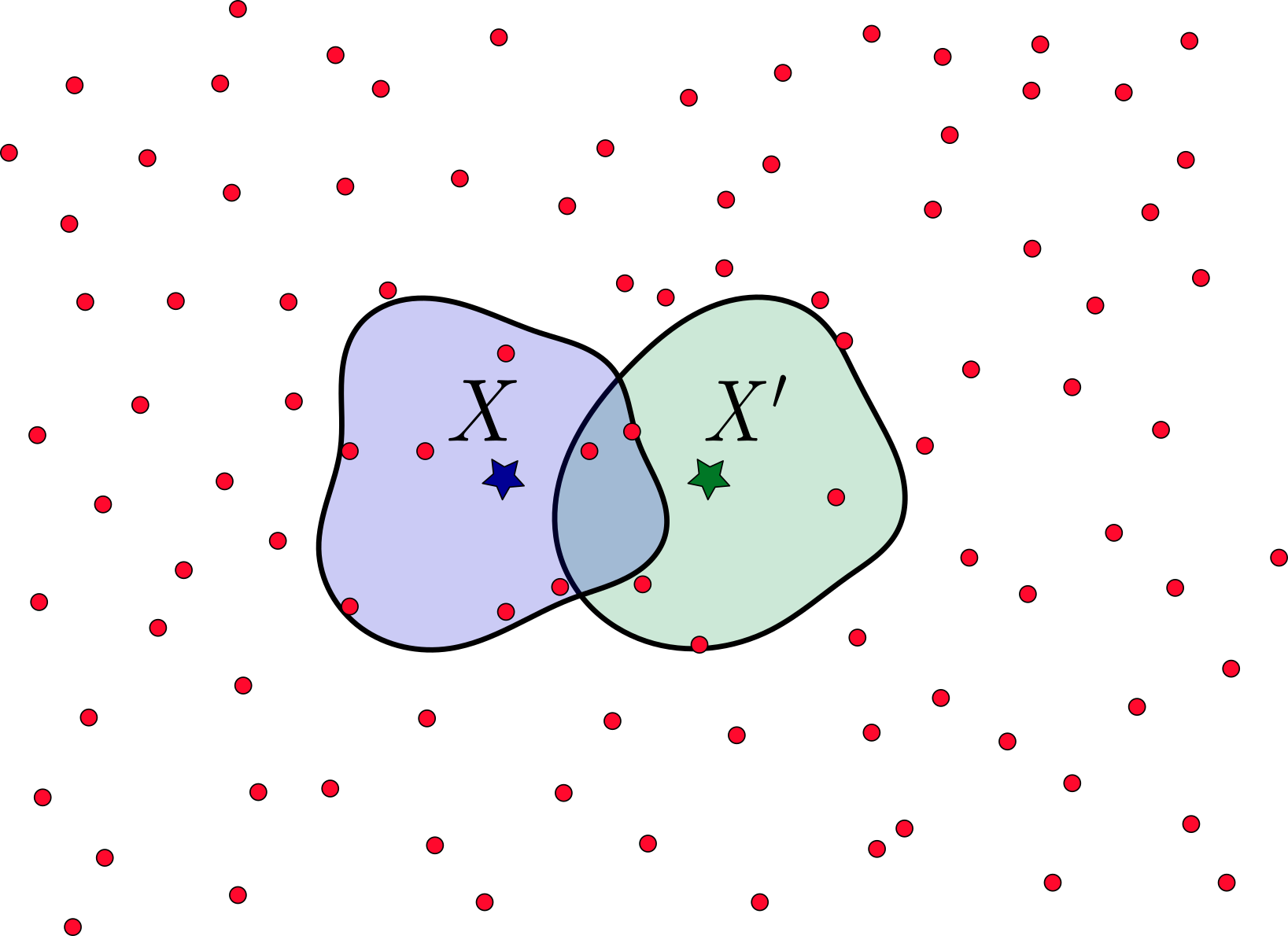}}
	\end{tabular}
	
	\protect
	\caption[Model flexibility vs \emph{influence} decay rate (bandwidth) of the model kernel.]{Model flexibility vs \emph{influence} decay rate (bandwidth) of \gs $g_{\theta}$. Assume for simplicity $\forall X: g_{\theta}(X, X) \equiv \gamma$ for some constant $\gamma$. Considering Eq.~(\ref{eq:PSODffrntl}), the \df at $X$ is a weighted average of terms belonging to the training points around $X$, where $g_{\theta}(X, \cdot)$ serves as a weighting coefficient for each term. Further, consider \gs to have a local-support. Its \emph{influence} decay can be seen to express an \emph{influence} area around each point $X$ outside of which the \gs $g_{\theta}(X, \cdot)$ becomes very small and negligible, on average. Above we illustrate two possible scenarios where the \emph{influence} area is (a) large and (b) small. Red points represent training points, from both $\probs{\usuff}$ and $\probs{\dsuff}$. Further, blue and green regions around points $X$ and $X' = X + \Delta$ express the neighborhoods around the points where \gs has large values. Note that in context of NNs these regions in general are not centered at $X$ (or $X'$) and are not symmetric, yet exhibit a particular \emph{influence} decay rate (see Section \ref{sec:NNArch} and Appendix \ref{sec:UncorRes} for the empirical evaluation of $g_{\theta}$). The training points in each such region around some point $X$ can be considered as \emph{support} training points of $X$ that will influence its surface height. As observed from plots, when the \emph{influence} decay rate is low (i.e. large \emph{influence} area, see plot (a)) the \dfs of $X$ and $X'$ will be very similar since most of the \emph{support} training points stay the same for both $X$ and $X'$. In contrast, when the \emph{influence} decay rate is high (see plot (b)), the \dfs at $X$ and $X'$ will be very different since most of the \emph{support} training points for both $X$ and $X'$ are not the same. Hence, the \df as a function of $X$ changes significantly for a step $\Delta$ within the input space when the decay rate is high, and vice versa. Furthermore, when the \df $df_{\theta}(X)$ changes only slightly for different points, the overall update of the surface height at each point is similar/identical to other points. Such surface is pushed \up/\down as one physical rigid body, making $f_{\theta}(X)$ "inelastic." Moreover, when the \emph{influence} decay rate is significantly high, the point $X$ may have only a single \emph{support} (\up or \down) training point and $f_{\theta}(X)$ will be pushed only in a single direction (\up or \down) yielding the spike near $X$. Finally, in case the \emph{influence} area of $X$ will not contain any training point, $f_{\theta}(X)$ will stay constant along the entire optimization.
	}
	\label{fig:FlexBnd}
\end{figure}

\begin{theorem}[Change Difference]
	\label{thrm:PSO_diff_diff} 
Consider the differential $df_{\theta}$ defined in Eq.~(\ref{eq:PSODffrntl}). For any two points $X_1$ and $X_2$, their change difference is bounded as:
\begin{multline}
|df_{\theta}(X_1) - df_{\theta}(X_2)|
\leq
\delta \cdot 
g_{\theta}(X_1, X_1)
\cdot
\Bigg[
\frac{1}{N^{\usuff}}
\sum_{i = 1}^{N^{\usuff}}
\left|
M^{\usuff}\left[X^{\usuff}_{i},f_{\theta}(X^{\usuff}_{i})\right]
\right|
\cdot
\nu_{\theta}
\left[
X_1, X_2, X^{\usuff}_{i}
\right]
+\\
+
\frac{1}{N^{\dsuff}}
\sum_{i = 1}^{N^{\dsuff}}
\left|
M^{\dsuff}\left[X^{\dsuff}_{i},f_{\theta}(X^{\dsuff}_{i})\right]
\right|
\cdot
\nu_{\theta}
\left[
X_1, X_2, X^{\dsuff}_{i}
\right]
\Bigg]
,
\label{eq:PSODffrntlChangeBound}
\end{multline}
where
\begin{equation}
\nu_{\theta}
\left[
X_1, X_2, X
\right]
\triangleq
\epsilon\left[ X_1, X_2, X \right]
+
\frac{\left|
	g_{\theta}(X_1, X_1) - g_{\theta}(X_2, X_2)
	\right|
}{g_{\theta}(X_1, X_1)}
,
\label{eq:PSODffrntlChangeBound_nu_def}
\end{equation}
\begin{equation}
\epsilon\left[ X_1, X_2, X \right]
\triangleq
1 -
\exp \left[
-
\frac{1}{h_{min}}
d(X_1, X_2)
\right]
\cdot
\exp \left[
-
\frac{1}{h_{min}}
\max
\left[
d(X_1, X),
d(X_2, X)
\right]
\right]
.
\label{eq:PSODffrntlChangeBound_eps_def}
\end{equation}

\end{theorem}

Proof of the above theorem is in Appendix \ref{sec:PSO_diff_diff_Proof}. In the above relation we can see that $\epsilon$ is smaller for a smaller distance $d(X_1, X_2)$. Likewise, $\epsilon \rightarrow 0$ along with $h_{min} \rightarrow \infty$. Neglecting the second term in $\nu_{\theta}$'s definition ($g_{\theta}(X_1, X_1)$ and $g_{\theta}(X_2, X_2)$ are typically very similar for any two close-by points $X_1$ and $X_2$), $\nu_{\theta}$ has analogous trends. Therefore, the upper bound in Eq.~(\ref{eq:PSODffrntlChangeBound}) is smaller when two points are nearby or when $h_{min}$ is large. From this we can conclude that $f_{\theta}(X_1)$ and $f_{\theta}(X_2)$ are evolving in a similar manner for the above specified setting. Particularly, for $h_{min} \rightarrow \infty$ (and if the second term of $\nu_{\theta}$ is relatively small) the entire surface $f_{\theta}$ will change almost identically at each point, intuitively resembling a rigid geometric body that moves \up and \down without changing its internal shape.

To conclude, when we decide which function space $\FF$ to use for PSO optimization, this decision is equivalent to choosing the model kernel $g_{\theta}$ with the most desired properties. 
The effect of $g_{\theta}$'s bandwidth on the optimization is described by the above theorems, whose intuitive summary is given in Figure \ref{fig:FlexBnd}. In particular, when a bandwidth of the (relative) kernel is too narrow - there will be spikes at the training points, and when this bandwidth is too wide - the converged surface will be overly smoothed. These two extreme scenarios are also known as \overfit and \underfit, and the kernel bandwidth can be considered as a flexibility parameter of the surface $f_{\theta}$. 
Furthermore, the above exposition agrees with existing works for RKHS models \citep{Principe10information} where the kernel bandwidth is known to affect the estimation bias-variance trade-off.

\subsection{Infinite Height Problem and its Solutions}
\label{sec:PSOStability}

In this section we study main stability issues encountered due to a mismatch between supports of $\probs{\usuff}$ and $\probs{\dsuff}$. From part 2 of Theorem \ref{thrm:PSO} we see that there are settings under which $f_{\theta}(X)$ will go to infinity at $X$ outside of mutual support $\spp^{\udcapsuff}$. In Figure \ref{fig:SprtMsmtchFig} a simple experiment is shown that supports this conclusion empirically, where $f_{\theta}(X)$ at $X \in \spp^{\udbacksuff}$ is increasing during the entire learning process. Observe that while GD optimization obtains the \emph{convoluted} PSO equilibrium in Eq.~(\ref{eq:ConvPSOBalStatePopul}) instead of the variational PSO \bp in Eq.~(\ref{eq:BalPoint}), the conclusions of Theorem \ref{thrm:PSO} still remain valid in practice.

Similarly, the above described \infh problem can happen at $X \in \spp^{\udcapsuff}$ where one of the ratios $\frac{\probi{\usuff}{X}}{\probi{\dsuff}{X}}$ and $\frac{\probi{\dsuff}{X}}{\probi{\usuff}{X}}$ is too small, yet is not entirely zero (i.e. $| \log \probi{\usuff}{X} - $ $\log \probi{\dsuff}{X} |$ is large). Such \emph{relative} support mismatch can cause instability as following. During the sampling process, at areas where $\probi{\usuff}{X}$ and $\probi{\dsuff}{X}$ are very different, we can obtain many samples from one density yet almost no samples from the other. Taking Theorem \ref{thrm:PSO_peak_convergence} into the account, any training point $X$ that is isolated from samples of the opposite force (i.e. when $d(X, \cdot)$ is large) will enforce a spike at $f_{\theta}(X)$. Moreover, when combined with the narrow kernel bandwidth, such spike behavior will be even more extreme with $f_{\theta}(X)$ being pushed to the \infh (see Section \ref{sec:DeepLogPDFOF} for the empirical illustration).

Hence, in both of the above cases $f_{\theta}$ will go to $\pm \infty$ at various locations. Meaning of this is lack of the optimization convergence. Furthermore, too large $f_{\theta}$' outputs may cause arithmetic underflow and overflow instabilities when computing the loss gradient, and hence will lead to a divergence of the learning task.

\begin{figure}
	\centering
	
	\begin{tabular}{cccc}

		\subfloat[\label{fig:SprtMsmtchFig-a}]{\includegraphics[width=0.22\textwidth]{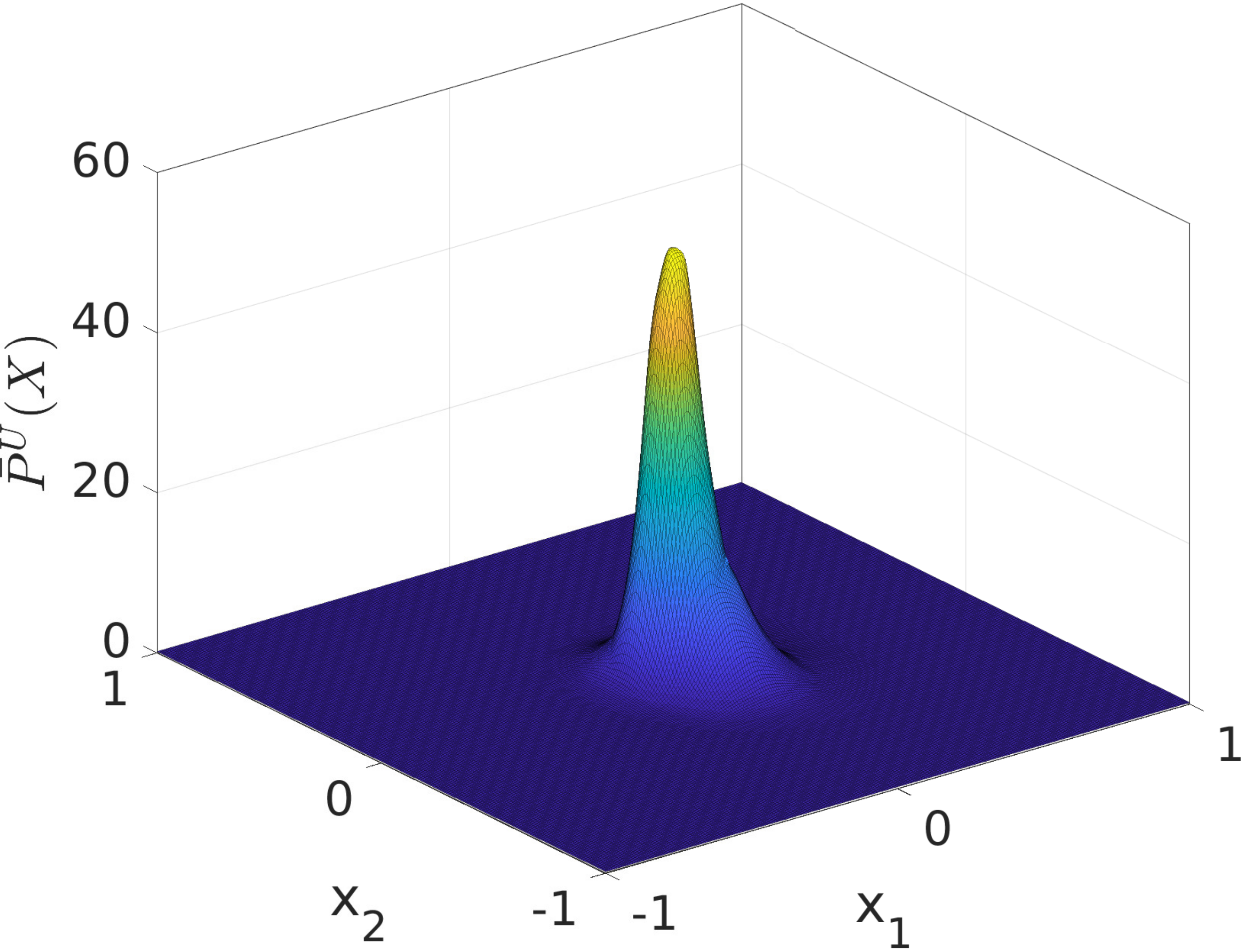}}
		&
		
		\subfloat[\label{fig:SprtMsmtchFig-b}]{\includegraphics[width=0.22\textwidth]{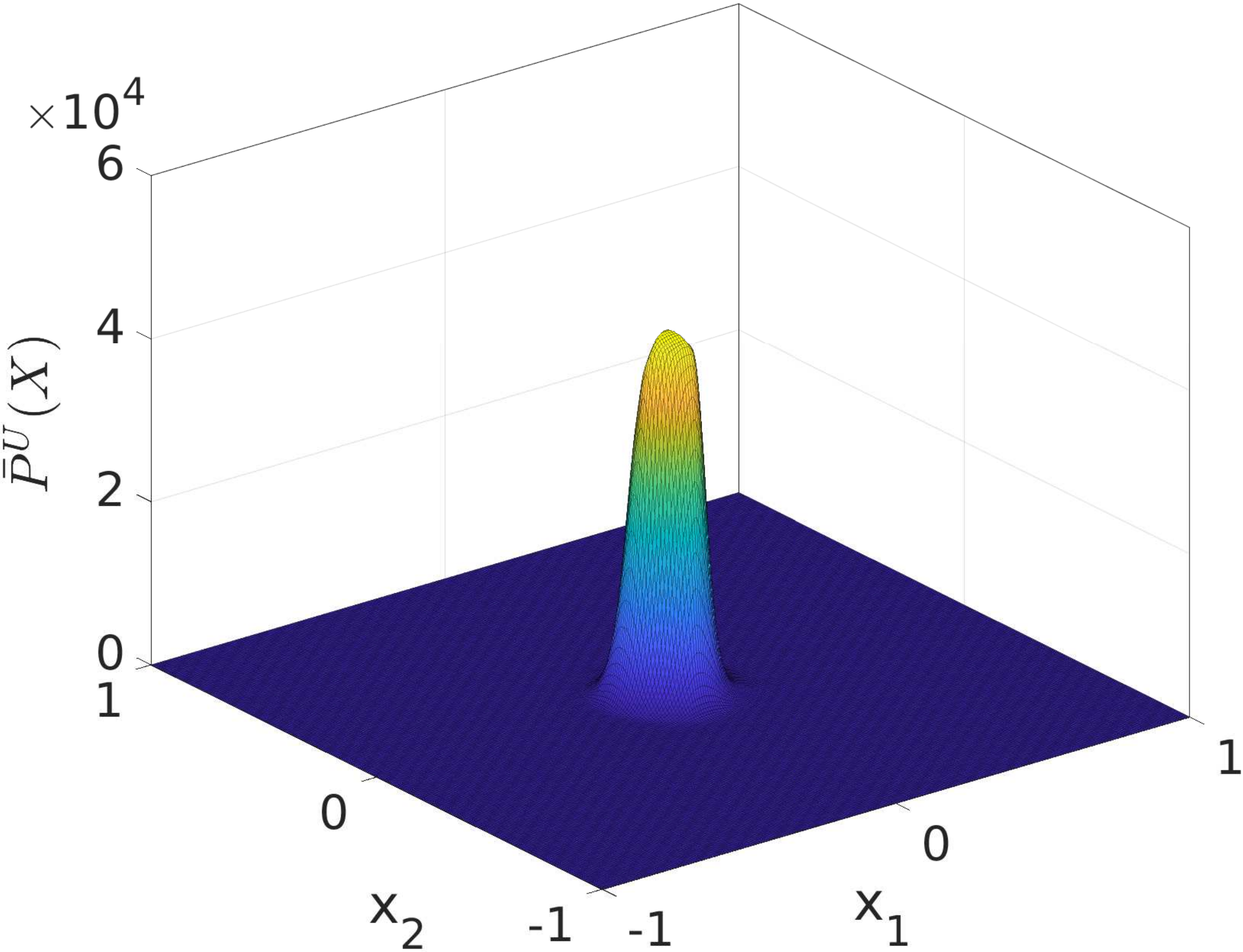}}
		&
		\subfloat[\label{fig:SprtMsmtchFig-c}]{\includegraphics[width=0.22\textwidth]{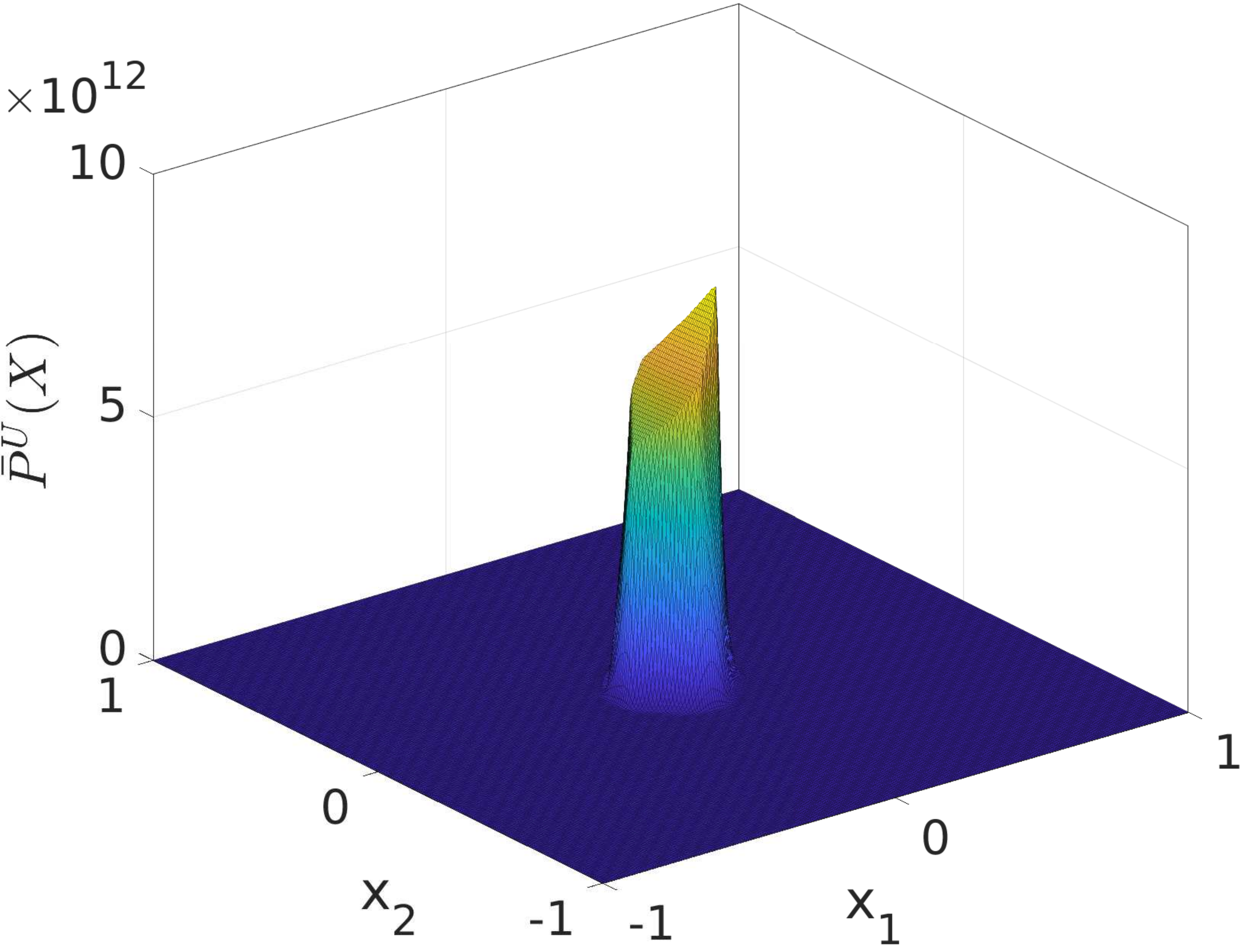}}
		&
		
		\subfloat[\label{fig:SprtMsmtchFig-d}]{\includegraphics[width=0.22\textwidth]{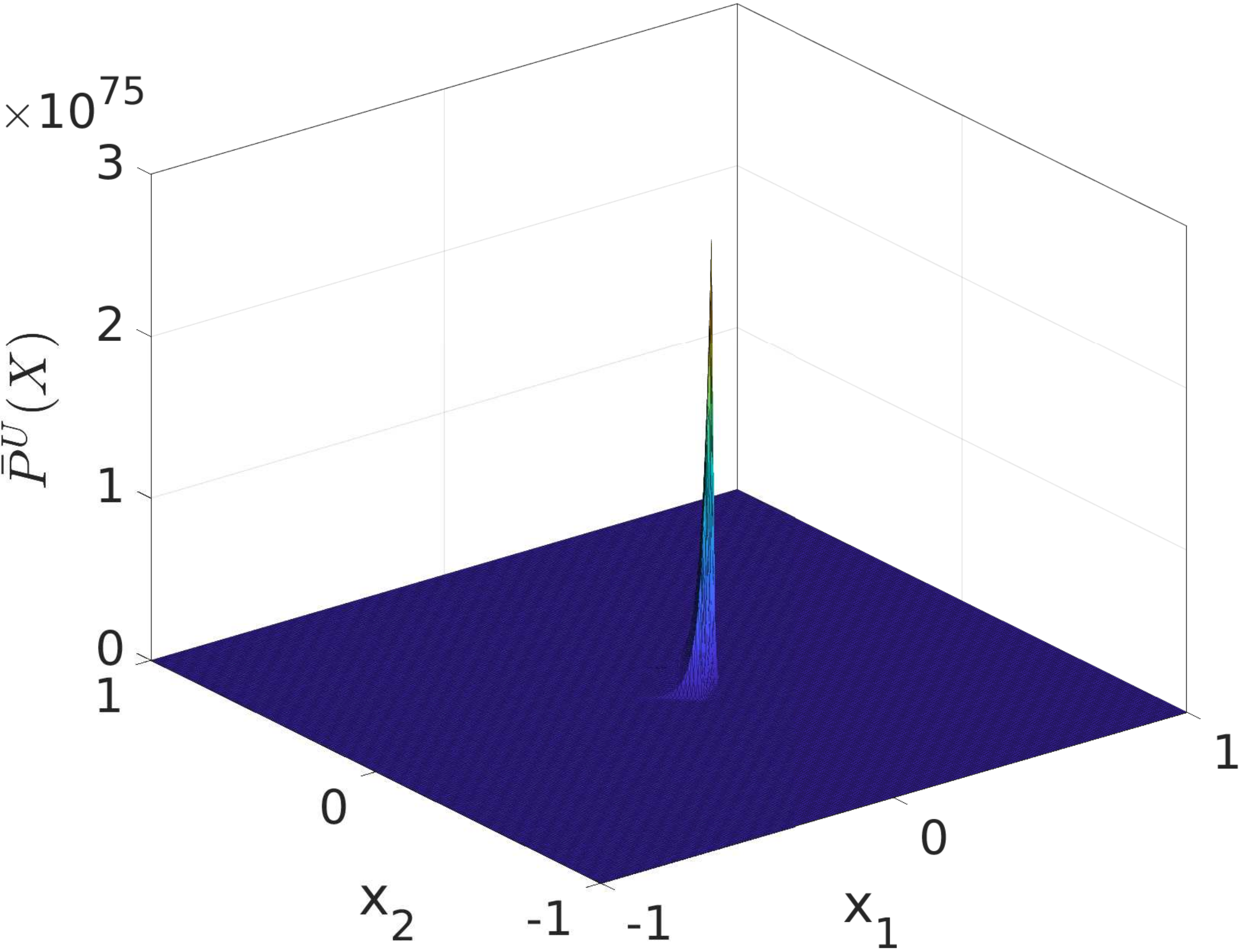}}
		
	\end{tabular}
	
	\protect
	\caption[Illustration of PSO behavior in areas outside of the mutual support.]{Illustration of PSO behavior in areas outside of the mutual support $\spp^{\udcapsuff}$. We inferred 2D \emph{Uniform} distribution $\probs{\usuff}$ via $\bar{\PP}^{\usuff}(X) = \exp f_{\theta}(X)$ by using PSO-LDE with $\alpha = \frac{1}{4}$ (see Table \ref{tbl:PSOInstances1} and Section \ref{sec:DeepLogPDF}). The applied NN architecture is block-diagonal with 6 layers, number of blocks $N_B = 50$ and block size $S_B = 64$ (see Section \ref{sec:BDLayers}). We plot $\bar{\PP}^{\usuff}(X)$ at different optimization times $t$: (a) $t = 100$, (b) $t = 200$, (c) $t = 10000$ and (d) $t = 100000$.
	The support $\spp^{\usuff}$ of $\probs{\usuff}$ is $[-1, 1]$ for both dimensions. The chosen \down density $\probs{\dsuff}$ is defined with $\spp^{\dsuff}$ being identical to $\spp^{\usuff}$, minus the circle of radius $0.3$ around the origin $(0, 0)$. In its entire support $\probs{\dsuff}$ is distributed uniformly. Thus, in this setup the zero-centered circle is outside of $\spp^{\udcapsuff}$ - we have samples $X^{\usuff}_i$ there but no samples $X^{\dsuff}_i$. For this reason, there is only the \up force $F_{\theta}^{\usuff}$ that is present in the circle area, pushing the surface there indefinitely \up. This can be observed from how the centered spire rises along the optimization time.
	}
	\label{fig:SprtMsmtchFig}
\end{figure}

In case $\probs{\usuff}$ and $\probs{\dsuff}$ are relatively similar distributions and when $g_{\theta}$'s bandwidth is wide enough, the above problem of \infh will not occur in practice. For other cases, there are two possible strategies to avoid the optimization divergence.

\paragraph{Indicator Magnitudes}

The above problem can be easily fixed
by multiplying any given \emph{magnitude} $M^{\usuff}$ (or $M^{\dsuff}$) with the following function:
\begin{equation}
reverse\_at
\left[
X,
f_{\theta}(X),
\varphi
\right]
= 
\begin{cases}
-1,
& \text{if } f_{\theta}(X) > \varphi \quad (\text{or } f_{\theta}(X) < \varphi)\\
1
,              & \text{otherwise}
\end{cases}
\label{eq:IHStopper}
\end{equation}
The $reverse\_at(\cdot)$ will change sign of near \mgn term when $f_{\theta}(X)$ at the training point $X$ passes the threshold height $\varphi$. This in turn will change a direction of the force, making it to oscillate the surface $f_{\theta}(X)$ around $\varphi$ (similarly to part 2-d of Theorem \ref{thrm:PSO}). Such behavior will happen only at the "problematic" areas where $f_{\theta}(X)$ got too high/low. In other "safe" areas the function $reverse\_at(\cdot)$ will not have any impact. Hence, applying $reverse\_at(\cdot)$ next to both $M^{\usuff}$ and $M^{\dsuff}$ will enforce the surface height at all $X$ to be between some minimal and maximal thresholds, improving in this way the optimization stability.

\begin{figure}
	\centering
	
	\begin{tabular}{cc}

		\subfloat[\label{fig:StopAtFig-a}]{\includegraphics[width=0.45\textwidth]{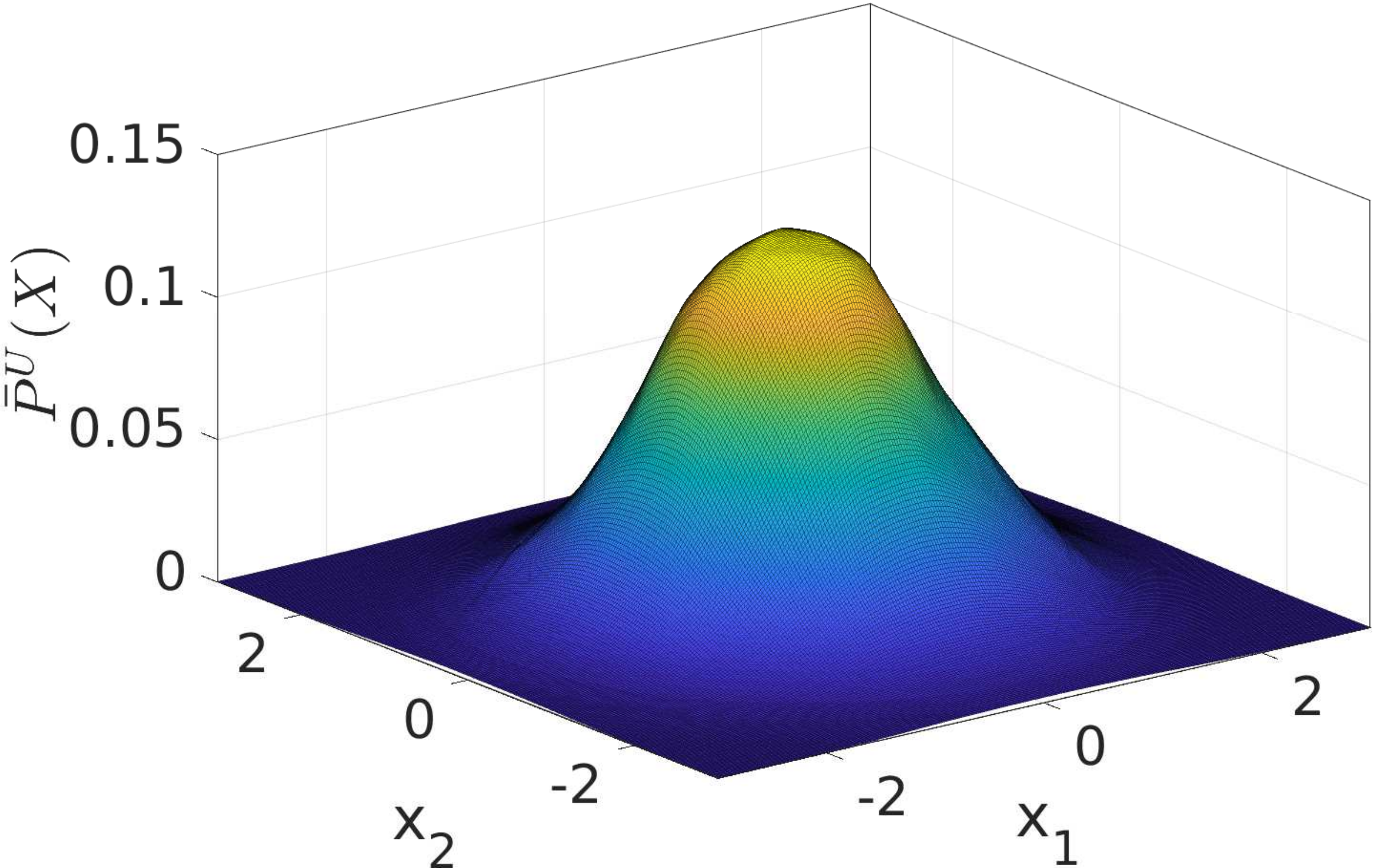}}
		&
		
		\subfloat[\label{fig:StopAtFig-b}]{\includegraphics[width=0.45\textwidth]{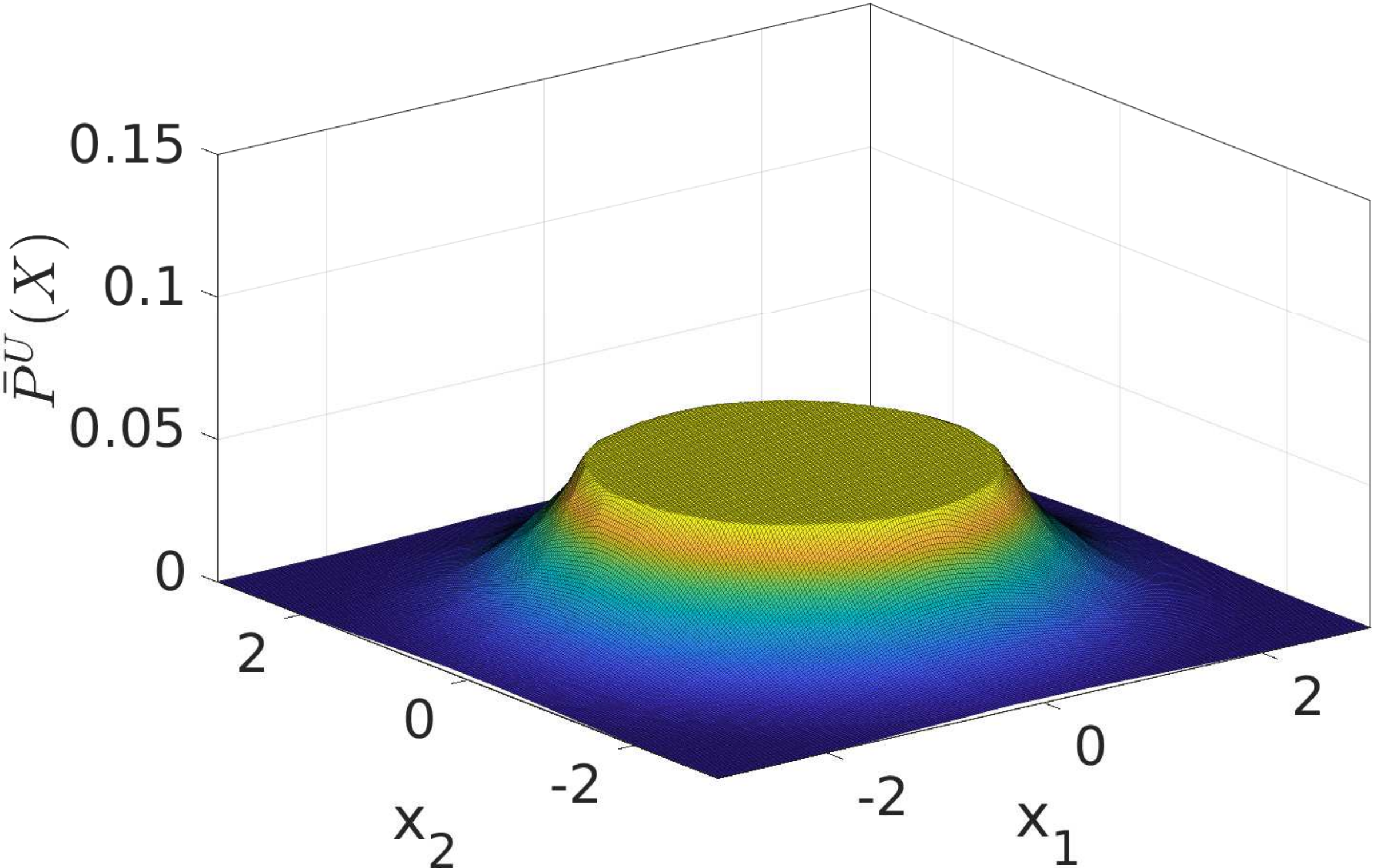}}
	\end{tabular}
	
	\protect
	\caption[Impact illustration of the function $cut\_at$.]{Impact illustration of the function $cut\_at\left[
		X,
		f_{\theta}(X),
		\varphi
		\right]
		$. (a) We inferred 2D \emph{Normal} distribution via $\bar{\PP}^{\usuff}(X) = \exp f_{\theta}(X)$ by using PSO-LDE with $\alpha = \frac{1}{4}$ (see Table \ref{tbl:PSOInstances1} and Section \ref{sec:DeepLogPDF}). The applied NN architecture is block-diagonal with 6 layers, number of blocks $N_B = 50$ and block size $S_B = 64$ (see Section \ref{sec:BDLayers}).
		(b) We performed the same learning algorithm as in (a), but with modified \up \emph{magnitude} function $\bar{M}^{\usuff}\left[X,f_{\theta}(X)\right] = M^{\usuff}\left[X,f_{\theta}(X)\right] \cdot cut\_at\left[
		X,
		f_{\theta}(X),
		-3
		\right]$, where $cut\_at$ is defined in Eq.~(\ref{eq:IHCutter}). This function deactivates \up pushes for points with $f_{\theta}(X) > -3$ (and $\exp f_{\theta}(X) > 0.0498$), thus the surface $f_{\theta}(X)$ at height above threshold -3 is only pushed by \down force and hence is enforced to converge to the threshold. Yet, note that at points where the converged model satisfies $f_{\theta}(X) \leq -3$ the convergence is the same as in (a).
	}
	\label{fig:StopAtFig}
\end{figure}

The alternative to the above function is:
\begin{equation}
cut\_at
\left[
X,
f_{\theta}(X),
\varphi
\right]
= 
\begin{cases}
0,
& \text{if } f_{\theta}(X) > \varphi \quad (\text{or } f_{\theta}(X) < \varphi)\\
1
,              & \text{otherwise}
\end{cases}
\label{eq:IHCutter}
\end{equation}
In contrast to $reverse\_at(\cdot)$, once the surface $f_{\theta}(X)$ at some training point $X$ passed the threshold height $\varphi$, the corresponding gradient term of $X$ (either \up or \down) is withdrawn from the total gradient $\nabla_{\theta}
\hat{L}_{PSO}^{N^{\usuff},N^{\dsuff}}(f_{\theta})
$ in Eq.~(\ref{eq:GeneralPSOLossFrml}), entirely deactivating the influence of $X$ on the learning task. 
Thus, the surface $f_{\theta}$ is not pushed anymore in areas where its height got too high/low; yet, it is still pushed at samples of the opposite force. Such force composition will constrain the surface height at unbalanced areas to converge to the threshold height $\varphi$, similarly to $reverse\_at(\cdot)$. Further, unlike $reverse\_at(\cdot)$, 
once any particular training point $X$ got disabled by $cut\_at(\cdot)$, it also does not have a side-influence (via $g_{\theta}(X, X')$) on the surface at other points. Likewise, it stops affecting the value of $\theta$, leading to a higher movement freedom within $\theta$-space. Empirically we observed $cut\_at(\cdot)$ to result in overall higher approximation accuracy compared to $reverse\_at(\cdot)$. The optimization outcome of $cut\_at(\cdot)$ usage is illustrated in Figure \ref{fig:StopAtFig}.

\paragraph{Restriction over Functions' Range}

Alternatively, we can enforce all functions within the considered function space $\FF$ to have any desired range $\AAA = [H_{min}, H_{max}]$. The constants $H_{min}$ and $H_{max}$ will represents the minimal and maximal surface heights respectively. For example, this can be accomplished by using the model:
\begin{equation}
f_{\theta}(X) =
\half
\cdot
[H_{max} - H_{min}]
\cdot
\tanh
[h_{\theta}(X)]
+
\half
\cdot
[H_{max} + H_{min}]
,
\label{eq:BoundedModel}
\end{equation}
where $h_{\theta}$ represents an inner model that may have unbounded outputs. Since $\tanh(\cdot)$ is bounded to have values in $[-1, 1]$, it is easy to verify that the above model can return only values between $H_{min}$ and $H_{max}$. Thus, using such model will eliminate the divergence of the surface to an infinite height. Likewise, other bounded functions can be used instead of $\tanh(\cdot)$ such as $\erf(\cdot)$, $\text{sigmoid}(\cdot)$, $\arctan(\cdot)$ and many others.

While the impact of both above strategies is intuitive and simple - prevention of $f_{\theta}$ from being pushed beyond pre-defined thresholds, the rigorous math proof is more difficult to obtain. First strategy induces estimators within PSO non-differentiable family presented in Section \ref{sec:FuncMutSupp_non_diff} - the setting that is not analyzed by this work. Second strategy leads to PSO functionals minimized over a function space, whose range may not contain the entire convergence interval $\KKK$.
A detailed analysis of the above special cases is left for future work since they are not the main focus of this paper.

\section{Density Estimation via PSO}
\label{sec:DensEst}

Till now we discussed a general formulation of PSO, where the presented analysis addressed properties of any PSO instance. In this section we focus in more detail on groups of PSO instances that can be applied for the density estimation problem, denoted above as $\pdfestset$ and $\logpdfestset$. In particular, in Section \ref{sec:DeepPDF} we shortly describe our previous work on density estimation, while in Section \ref{sec:DeepLogPDF} we explore new PSO approaches to infer density on a logarithmic scale, with bounded \emph{magnitude} functions that lead to the enhanced optimization performance.

\subsection{DeepPDF}
\label{sec:DeepPDF}

Here we briefly describe the density estimation approach, DeepPDF, introduced in \citep{Kopitkov18arxiv}, as a particular instance  of the PSO paradigm.
The density estimation problem involves learning a pdf function $\probi{\usuff}{X}$ from a dataset of i.i.d. samples $\{ X^{\usuff}_{i} \}$. For this purpose, the proposed \emph{pdf loss} was defined as:
\begin{equation}
L_{pdf}(\theta) = 
-
\E_{X \sim \probs{\usuff}}
f_{\theta}(X) \cdot \probi{\dsuff}{X}		+
\E_{X \sim \probs{\dsuff}}
\half
\Big[
f_{\theta}(X)
\Big]^{2}
,
\label{eq:PDFLoss}
\end{equation}
with corresponding \mfs
$M^{\usuff}\left[X,f_{\theta}(X)\right] = \probi{\dsuff}{X}$ and $M^{\dsuff}\left[X,f_{\theta}(X)\right] = f_{\theta}(X)$.
$\probs{\dsuff}$ is an arbitrary density with a known pdf function which can be easily sampled.

The above loss is a specific instance of PSO with \bp achieved when the surface $f_{\theta}(X)$ converges to $\probi{\usuff}{X}$ point-wise $\forall X \in \spp^{\udcapsuff}$ (see Theorem \ref{thrm:PSO_func_magn}), with the corresponding convergence interval being $\KKK = \RRpos$. Since conditions of Theorem \ref{thrm:PSO_func_magn_unconstr} hold, no restriction over $\FF$'s range is required.
Moreover, according to the condition 3 in Theorem \ref{thrm:PSO_UD_conv} upon the convergence $f_{\theta}(X)$ at $X \in \spp^{\udbacksuff}$ can be arbitrary - $M^{\usuff}$ is zero in the area $\spp^{\udbacksuff}$ and hence the force $F_{\theta}^{\usuff}$ is disabled. Due to similar Theorem \ref{thrm:PSO_DU_conv}, at $X \in \spp^{\dubacksuff}$ the optimal solution must satisfy $f_{\theta}(X) = 0$. Therefore, any candidate for $\probs{\dsuff}$ with $\spp^{\usuff} \subseteq \spp^{\dsuff}$ will lead to the convergence $\forall X \in \spp^{\udcupsuff}: f_{\theta}(X) = \probi{\usuff}{X}$. Outside of the support union $\spp^{\udcupsuff}$ a convergence can be arbitrary, depending of the model kernel, which is true for any PSO method.

Concluding from the above, selected $\probs{\dsuff}$ must satisfy $\spp^{\usuff} \subseteq \spp^{\dsuff}$.
In our experiments we typically use a Uniform distribution for \down density $\probs{\dsuff}$ (yet in practice any density can be applied). The minimum and maximum for each dimension of $\probs{\dsuff}$'s support are assigned to minimum and maximum of the same dimension from the available $\probs{\usuff}$'s data points. Thus, the available samples $\{ X^{\usuff}_{i} \}$ define $n$-dimensional \emph{hyperrectangle} in $\RR^n$ as support of $\probs{\dsuff}$, with $\probs{\usuff}$'s support being its subset. Inside this \emph{hyperrectangle} the surface is pushed by $F_{\theta}^{\usuff}$ and $F_{\theta}^{\dsuff}$. Note that if borders of this support \emph{hyperrectangle} can not be computed a priori (e.g. active learning), the $reverse\_at(\cdot)$ and $cut\_at(\cdot)$ functions can be used to prevent a possible optimization divergence as described in Section \ref{sec:PSOStability}.

After training is finished, the converged $f_{\theta}(X)$ may have slightly negative values at points $\{X \in \spp^{\dubacksuff} \}$ being that during optimization the oscillation around height zero is stochastic in nature. Moreover, surface values outside of the \emph{hyperrectangle} may be anything since the $f_{\theta}(X)$ was not optimized there. In order to deal with these possible inconsistencies, we can use the following \emph{proxy} function as our estimation of target $\probi{\usuff}{X}$:
\begin{equation}
\bar{f}_{\theta}(X) = 
\begin{cases}
0,& \text{if } f_{\theta}(X) < 0 \text{ or } \probi{\dsuff}{X} = 0\\
f_{\theta}(X),              & \text{otherwise}
\end{cases}
\label{eq:PDFProxy}
\end{equation}
which produces the desirable convergence $\forall X \in \RR^{n}: \bar{f}_{\theta}(X) = \probi{\usuff}{X}$.

In \citep{Kopitkov18arxiv} we demonstrated that the above DeepPDF method with $f_{\theta}$ parametrized by NN outperforms the kernel density estimation (KDE) in an inference accuracy, and is significantly faster at the query stage when the number of training points is large.

\subsection{PSO-LDE - Density Estimation on Logarithmic Scale}
\label{sec:DeepLogPDF}

Typically, the output from a multidimensional density $\probi{\usuff}{X}$ will tend to be extra small, where higher data dimension causes smaller pdf values. Representing very small numbers in a computer system may cause underflow and precision-loss problems. To overcome this, in general it is recommended to represent such small numbers at a logarithmic scale. Furthermore, the estimation of log-pdf is highly useful. For example, in context of robotics it can represent log-likelihood of sensor measurement and can be directly applied to infer an unobservable state of robot \citep{Kopitkov18iros}. Likewise, once log-pdf $\log \probi{\usuff}{X}$ is learned its average for data samples approximates the entropy of $\probs{\usuff}$, which can further be used for robot planning \citep{Kopitkov17ijrr}.

Here we derive several estimator families from PSO subgroup $\logpdfestset$ that infers logarithm of a pdf, $\log \probi{\usuff}{X}$, as its target function. Although some members of these families were already reported before (e.g. NCE, \citet{Gutmann10aistats}), the general formulation of these families was not considered previously. 
Further, presented below PSO instances with the convergence $\log \probi{\usuff}{X}$ can be separated into two groups - instances with unbounded and bounded \mfs \pair. As was discussed in Section \ref{sec:BoundUnboundMFs} and as will be shown in Section \ref{sec:Exper}, the latter group yields a better optimization stability and also produces a higher accuracy.

According to Table \ref{tbl:ClassicTransforms}, PSO convergence of the subgroup $\logpdfestset$ is described by $T\left[X, z \right] = \log z + \log \probi{\dsuff}{X}$, with its inverse being $R\left[X, s \right] = \frac{\exp s}{\probi{\dsuff}{X}}$.
Further, according to Theorem \ref{thrm:PSO_func_conv} the convergence interval is $\KKK = \RR$. Then, any pair of continuous positive functions \pair that satisfies:
\begin{equation}
\frac{M^{\dsuff}(X, f_{\theta}(X))}{M^{\usuff}(X, f_{\theta}(X))} = 
\frac{\exp f_{\theta}(X)}{\probi{\dsuff}{X}}
,
\label{eq:PSO_bal_log_PDF}
\end{equation}
will produce $f_{\theta}(X) = \log \probi{\usuff}{X}$ at the convergence.

\paragraph{Unbounded \emph{Magnitudes}}

To produce new PSO instances with the above equilibrium, infinitely many choices over \pair can be taken. In Table \ref{tbl:LogPDFInstances} we show several such alternatives. Note that according to PSO we can merely move any term $q(X, s)$ from $M^{\usuff}(X, s)$ into $M^{\dsuff}(X, s)$ as $\frac{1}{q(X, s)}$, and vice versa. Such modification will not change the PSO \bp and therefore allows for the exploration of various \ms that lead to the same convergence.

\begin{table}[tb]

	\centering
	\begin{tabular}{cllll}
		\toprule
		\textbf{Loss Version}     & \textbf{
			\emph{\underline{L}oss} / \emph{$M^{\usuff}(\cdot)$ and $M^{\dsuff}(\cdot)$}
		}      \\
		\midrule
		1  
		& \underline{L}: 
		$-
		\E_{X \sim \probs{\usuff}}
		f_{\theta}(X) \cdot \probi{\dsuff}{X}		+
		\E_{X \sim \probs{\dsuff}}
		\exp (f_{\theta}(X))
$
		\\
		& $M^{\usuff}$, $M^{\dsuff}$: 
		$\probi{\dsuff}{X}$,
		$
		\exp (f_{\theta}(X))
		$
		\\
		\midrule
		2  
		& \underline{L}: 
		$-
		\E_{X \sim \probs{\usuff}}
		f_{\theta}(X)		+
		\E_{X \sim \probs{\dsuff}}
		\frac{\exp (f_{\theta}(X))}{\probi{\dsuff}{X}}
		$
		
		\\
		& $M^{\usuff}$, $M^{\dsuff}$: 
		$1$,
		$
		\frac{\exp (f_{\theta}(X))}{\probi{\dsuff}{X}}
		$
		\\
		
		\midrule
		3  
		& \underline{L}: 
		$
		\E_{X \sim \probs{\usuff}}
		\frac{\probi{\dsuff}{X}}{\exp (f_{\theta}(X))}		+
		\E_{X \sim \probs{\dsuff}}
		f_{\theta}(X)
		$
		\\
		& $M^{\usuff}$, $M^{\dsuff}$: 
		$
		\frac{\probi{\dsuff}{X}}{\exp (f_{\theta}(X))}
		$,
		$
		1
		$
		\\
		
		\midrule
		4  
		& 
		\underline{L}: 
		$
		\E_{X \sim \probs{\usuff}}
		2 \cdot
		\exp \Big[
		\half \cdot 
		\big(
		\log \probi{\dsuff}{X} - f_{\theta}(X)
		\big)
		\Big]		+$
		\\ &
		$ \quad +
		\E_{X \sim \probs{\dsuff}}
		2 \cdot
		\exp \Big[
		\half \cdot 
		\big(
		f_{\theta}(X) - \log \probi{\dsuff}{X}
		\big)
		\Big]
		$
		\\ &
		$M^{\usuff}$, $M^{\dsuff}$: $
		\exp \Big[
		\half \cdot 
		\big(
		\log \probi{\dsuff}{X} - f_{\theta}(X)
		\big)
		\Big]
		$,
		\\ &
		$
		\quad
		\exp \Big[
		\half \cdot 
		\big(
		f_{\theta}(X) - \log \probi{\dsuff}{X}
		\big)
		\Big]
		$
		\\
		
		\midrule
		5  
		& \underline{L}: 
		$-
		\E_{X \sim \probs{\usuff}}
		\exp (f_{\theta}(X))
		\cdot
		\probi{\dsuff}{X}		+
		\E_{X \sim \probs{\dsuff}}
		\half
		\cdot
		\exp (2 \cdot f_{\theta}(X))
		$

		\\
		& $M^{\usuff}$, $M^{\dsuff}$: 
		$
		\probi{\dsuff}{X}
		\cdot
		\exp (f_{\theta}(X))
		$,
		$
		\exp (2 \cdot f_{\theta}(X))
		$
		\\
		
		\bottomrule
	\end{tabular}
	
	\caption{Several PSO Instances that converge to $f_{\theta}(X) = \log \probi{\usuff}{X}$}
	\label{tbl:LogPDFInstances}
	
\end{table}

\begin{remark}
Although we can see an obvious similarity and a relation between magnitudes in Table \ref{tbl:LogPDFInstances} (they all have the same ratio $M^{\dsuff}(\cdot)/M^{\usuff}(\cdot)$), the corresponding losses have a much smaller resemblance. Without applying PSO rules, it would be hard to deduce that they all approximate the same target function.
\end{remark}

The Table \ref{tbl:LogPDFInstances} with acquired losses serves as a demonstration for simplicity of applying PSO concepts to forge new methods for the log-density estimation. However, produced losses have unbounded \mfs, and are not very stable during the real optimization, as will be shown in our experiments.
The first loss in Table \ref{tbl:LogPDFInstances} can lead to precision problems since its \ms return (very) small outputs from $\probi{\dsuff}{X}$ and $\exp (f_{\theta}(X))$. Further, $M^{\usuff}$ of the third loss has devision by output from the current model $\exp (f_{\theta}(X))$ which is time-varying and can produce values arbitrarily close to zero. Likewise, methods 4 and 5 hold similar problems.

\paragraph{Bounded \emph{Magnitudes}}

Considering the above point, the PSO instances in Table \ref{tbl:LogPDFInstances} are sub-optimal. Instead, 
we want to find PSO losses with bounded $M^{\dsuff}(\cdot)$ and $M^{\usuff}(\cdot)$. Further, the required relation in Eq.~(\ref{eq:PSO_bal_log_PDF}) between two \emph{magnitudes} can be seen as:
\begin{equation}
\frac{M^{\dsuff}\left[X,f_{\theta}(X)\right]}{M^{\usuff}\left[X,f_{\theta}(X)\right]}
=
\exp \bar{d}
\left[
X, f_{\theta}(X)
\right]
,
\label{eq:MangConn}
\end{equation}
\begin{equation}
\bar{d}
\left[
X, f_{\theta}(X)
\right]
\triangleq 
f_{\theta}(X) - \log \probi{\dsuff}{X}
.
\label{eq:DDiffDefinition}
\end{equation}
$\bar{d}\left[
X, f_{\theta}(X)
\right]
$ is a logarithm difference between the model surface and log-pdf of \down density, which will play an essential role in \mfs below.

According to Section \ref{sec:BoundUnboundMFs} and Lemma \ref{lmm:PSO_consist_modific}, from Eq.~(\ref{eq:MangConn}) we can produce the following family of PSO instances:
\begin{equation}
M^{\usuff}\left[X,f_{\theta}(X)\right]
=
\frac{\probi{\dsuff}{X}}{D\left[X,f_{\theta}(X)\right]}
,
\quad
M^{\dsuff}\left[X,f_{\theta}(X)\right]
=
\frac{\exp f_{\theta}(X)}{D\left[X,f_{\theta}(X)\right]}
,
\label{eq:PSOLogEstNormalized}
\end{equation}
where the denominator function $D\left[X,f_{\theta}(X)\right] > 0$ takes the responsibility to normalize output of \emph{magnitude} functions to be in some range $[0,\epsilon]$. Moreover, choice of $D\left[X,f_{\theta}(X)\right]$ does not affect the PSO \bp; it is reduced when the above \emph{magnitudes} are introduced into Eq.~(\ref{eq:PSO_bal_log_PDF}).

\begin{figure}[tb]
	\centering
	
	\begin{tabular}{cccc}
		
		\subfloat[\label{fig:MagnFuncs-a}]{\includegraphics[width=0.4\textwidth]{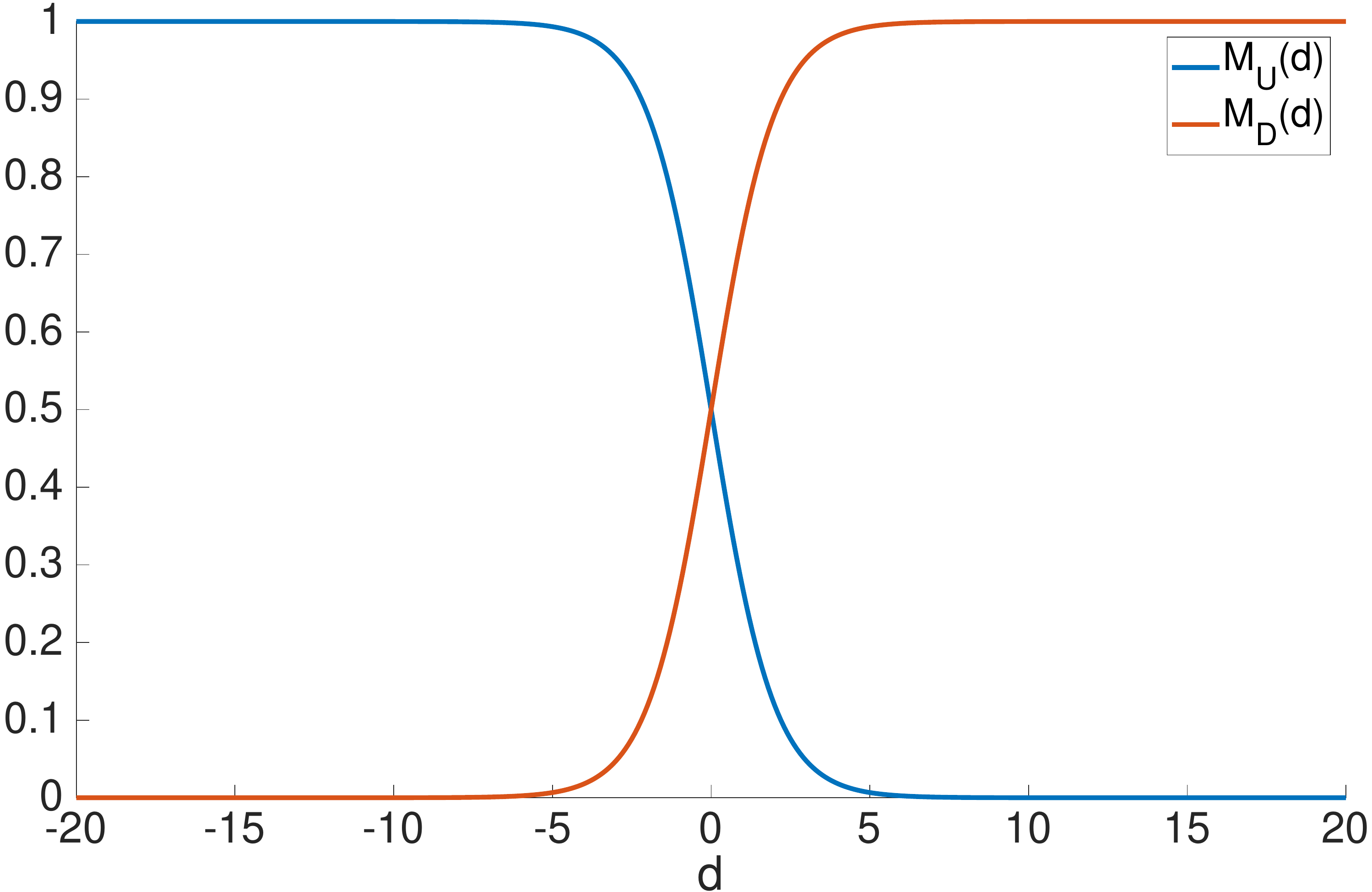}}
&
		\subfloat[\label{fig:MagnFuncs-b}]{\includegraphics[width=0.4\textwidth]{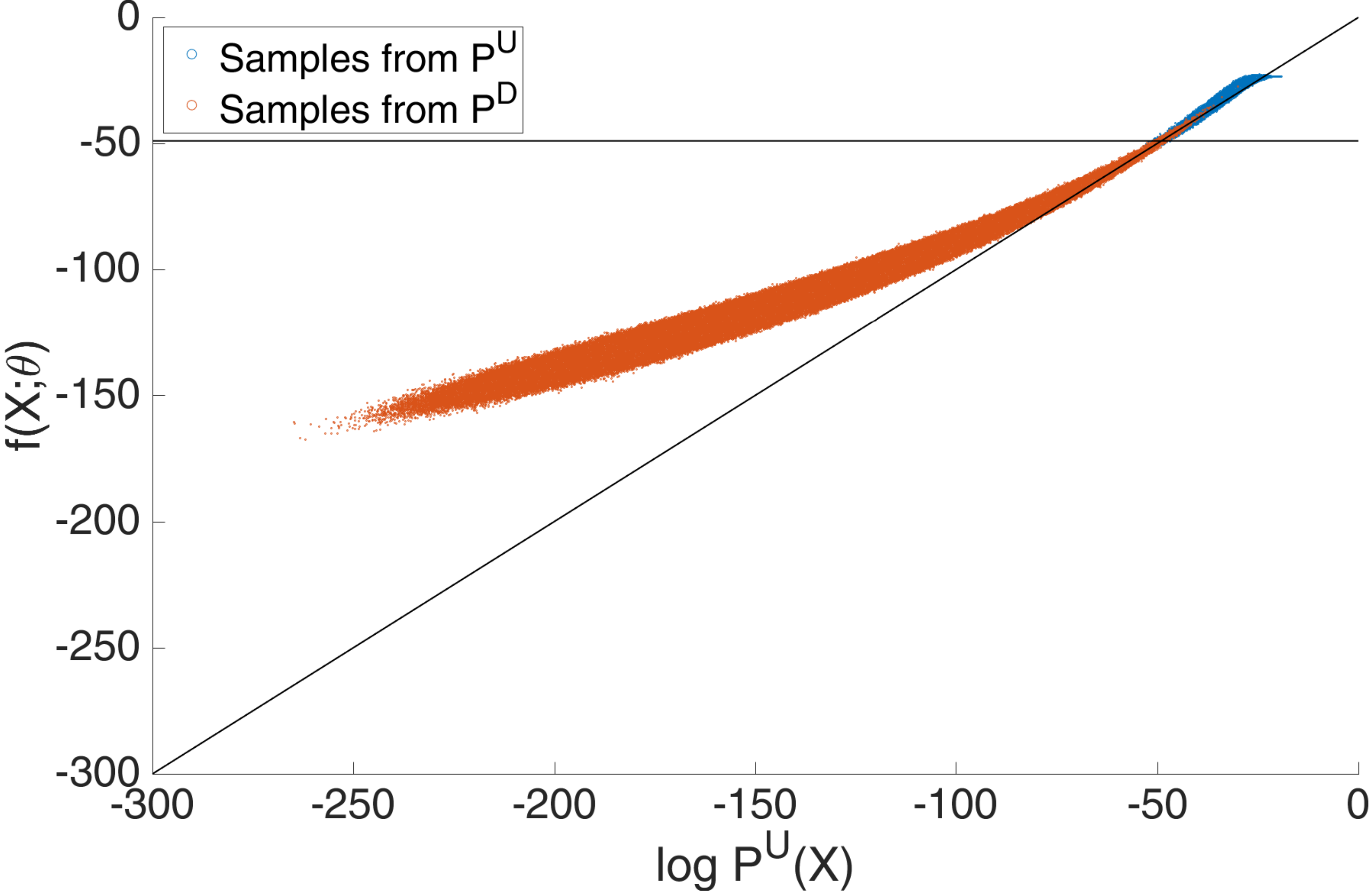}}

	\end{tabular}
	
	\protect
	\caption[NCE \emph{magnitudes} as functions of a log-difference.]{(a) NCE \emph{magnitudes} as functions of a difference $\bar{d}\left[
		X, f_{\theta}(X)
		\right]
		 \triangleq f_{\theta}(X) - \log \probi{\dsuff}{X}$.
	(b) Log density estimation via NCE for 20D data, 
	where $\probs{\usuff}$ is standard Normal 20D distribution and $\probs{\dsuff}$ is minimal Uniform 20D distribution that covers all samples from $\probs{\usuff}$. Blue points are sampled from $\probs{\usuff}$, while red points - from $\probs{\dsuff}$. The $x$ axes represent $\log \probi{\usuff}{X}$ for each sample, $y$ axes - the surface height $f_{\theta}(X)$ after the optimization was finished. Diagonal line represents $f_{\theta}(X) = \log \probi{\usuff}{X}$, where we would see all points in case of \emph{perfect} model inference. The black horizontal line represents $\log \probi{\dsuff}{X} = - 49$ which is constant for the Uniform density. As can be seen, these two densities have a \emph{relative} support mismatch - 
	the sampled points from both densities are obviously located mostly in different space neighborhoods; this can be concluded from values of $\log \probi{\usuff}{X}$ that are very different for both point populations. Further, points with relatively small $|\bar{d}|$ (around the horizontal line) have a small estimation error since the ratio $\frac{\probi{\usuff}{X}}{\probi{\dsuff}{X}}$ there is bounded and both points $X^{\usuff}$ and $X^{\dsuff}$ are sampled from these areas. In contrast, we can see that in areas where $|\bar{d}| > \varepsilon$ for some positive constant $\varepsilon$ we have samples only from one of the densities. Further, points with $\bar{d} > 0$ (\emph{above} the horizontal line) are pushed \up till a some threshold where the surface height is stuck due to \up \emph{magnitude} $M^{\usuff}(\cdot)$ going to zero (see also Figure (a)). Additionally, points with $\bar{d} < 0$ (\emph{below} the horizontal line) are pushed \down till their \emph{magnitude} $M^{\dsuff}(\cdot)$ also becomes zero. Note that \emph{below} points are pushed further from the horizontal line than the \emph{above} points. This is because the \emph{above} points are near the origin (mean of Normal distribution) which is a less flexible area; it is side-influenced from all directions by surrounding \down samples via $g_{\theta}$, which prevents it from getting too high. On opposite, the \emph{below} points are located far from the origin center on the edges of the considered point space. In these areas there are almost no samples from $\probs{\usuff}$ and thus the surface is much more easily pushed \down. Clearly, the PSO estimation task for the above choice of \up and \down densities can not yield a high accuracy, unless $g_{\theta}$ is a priori chosen in data-dependent manner (not considered in this paper). Yet, we can see that NCE does not push the surface to $\pm \infty$ at the unbalanced areas.
	}
	\label{fig:MagnFuncs}
\end{figure}

To bound functions $M^{\dsuff}(\cdot)$ and $M^{\usuff}(\cdot)$ in Eq.~(\ref{eq:PSOLogEstNormalized}),  $D\left[X,f_{\theta}(X)\right]$ can take infinitely many forms. One such form, that was implicitly applied by NCE technique \citep{Smith05acl, Gutmann10aistats}, is $D\left[X,f_{\theta}(X)\right] = \exp f_{\theta}(X) + \probi{\dsuff}{X}$ (see also Table \ref{tbl:PSOInstances1}). Such choice of normalization enforces outputs of both \emph{magnitude} functions in Eq.~(\ref{eq:PSOLogEstNormalized}) to be between 0 and 1. Moreover, NCE \emph{magnitudes} can be seen as functions of a logarithm difference $\bar{d}\left[
X, f_{\theta}(X)
\right]
$ in Eq.~(\ref{eq:DDiffDefinition}),
$M^{\usuff}(\bar{d}) = sigmoid(-\bar{d}\left[
X, f_{\theta}(X)
\right]
)$ and $M^{\dsuff}(\bar{d}) = sigmoid(\bar{d}\left[ X, f_{\theta}(X) \right])$. Thus, an output of \mfs at point $X \in \RR^n$ entirely depends on this logarithm difference at $X$.
Furthermore, the \up \emph{magnitude} reduces to zero for a large positive $\bar{d}$ and the \down \emph{magnitude} reduces to zero for a large negative $\bar{d}$ (see also Figure \ref{fig:MagnFuncs-a}).

Such property, produced by bounding \ms, is highly helpful and intuitively can be viewed as an elastic springy constraint over the surface $f_{\theta}$; it prevents \infh problem described in Section \ref{sec:PSOStability}, even when \up and \down densities are very different and when their support does not match. In neighborhoods where we sample many points from $\probs{\usuff}$ but almost no points from $\probs{\dsuff}$ (ratio $\frac{\probi{\usuff}{X}}{\probi{\dsuff}{X}}$ is large), the surface is pushed indefinitely \up through \up term within PSO loss, as proved by Theorem \ref{thrm:PSO_peak_convergence}. 
Yet, as it pushed higher, $\bar{d}$ for these neighborhoods becomes larger and thus the \up \emph{magnitude} $M^{\usuff}(\bar{d})$ goes quickly to zero. Therefore, when at a specific point $X$ the surface $f_{\theta}(X)$ was pushed \up too far from $\log \probi{\dsuff}{X}$, the \up \emph{magnitude} at this point becomes almost zero hence deactivating the \up force $F_{\theta}^{\usuff}(X)$ at this $X$. The same logic also applies to \down force $F_{\theta}^{\dsuff}(X)$ - in NCE this force is deactivated at points where $f_{\theta}(X)$ was pushed \down too far from $\log \probi{\dsuff}{X}$.

Critically, since $f_{\theta}(X)$ approximates $\log \probi{\usuff}{X}$, $\bar{d}\left[
X, f_{\theta}(X)
\right]
$ can also be viewed as an estimation of $\log \frac{\probi{\usuff}{X}}{\probi{\dsuff}{X}}$.
Therefore, the above exposition of NCE dynamics can be also summarized as follows. At points where logarithm difference $\log \probi{\usuff}{X} - \log \probi{\dsuff}{X}$ is in some dynamical active range $[- \varepsilon, \varepsilon]$ for positive $\varepsilon$, the \up and \down forces will be active and will reach the equilibrium with $f_{\theta}(X) = \log \probi{\usuff}{X}$. At points where $[\log \probi{\usuff}{X} - \log \probi{\dsuff}{X}] > \varepsilon \Leftrightarrow \frac{\probi{\usuff}{X}}{\probi{\dsuff}{X}} > \exp \varepsilon$, the surface will be pushed \up to height $\varepsilon$. And at points where $[\log \probi{\usuff}{X} - \log \probi{\dsuff}{X}] < -\varepsilon \Leftrightarrow \frac{\probi{\usuff}{X}}{\probi{\dsuff}{X}} < \frac{1}{\exp \varepsilon}$, the surface will be pushed \down to height $-\varepsilon$. Once the surface at some point $X$ passes above the height $\varepsilon$ or below the height $-\varepsilon$, the NCE loss stops pushing it due to (near) zero \emph{magnitude} component. Yet, the side-influence induced by model kernel $g_{\theta}(X, X')$ from non-zero \emph{magnitude} areas can still affect the surface height at $X$. The above NCE behavior is illustrated in Figure \ref{fig:MagnFuncs-b} where 20D log-density estimation is performed via NCE for Gaussian distribution $\probs{\usuff}$ and Uniform distribution $\probs{\dsuff}$.

\begin{remark}
Note that the scalar $\varepsilon$ represents a sensitivity threshold, where pushes at points with $\bar{d} > \varepsilon \Leftrightarrow |M^{\usuff}(\cdot)| < sigmoid(- \varepsilon)$ or at points with $\bar{d} < - \varepsilon \Leftrightarrow |M^{\dsuff}(\cdot)| < sigmoid(- \varepsilon)$ have a neglectable effect on the surface due to their small magnitude component. Such sensitivity is different for various functional spaces $f_{\theta} \in \FF$; for some spaces a small change of $\theta$ can only insignificantly affect the surface $f_{\theta}(X)$, while causing huge impact in others.
Hence, the value of $\varepsilon$ depends on specific choice of $\FF$ and of magnitude functions $M^{\usuff}(\cdot)$ and $M^{\dsuff}(\cdot)$.
\end{remark}

The above described relationship between NCE \emph{magnitude} functions and ratio $\frac{\probi{\usuff}{X}}{\probi{\dsuff}{X}}$ is very beneficial in the context of density estimation, since it produces high accuracy for points with bounded density ratio $| \log \probi{\usuff}{X} - $ $\log \probi{\dsuff}{X} | \leq \varepsilon$
and it is not sensitive to instabilities of areas where $| \log \probi{\usuff}{X} - $ $\log \probi{\dsuff}{X} | > \varepsilon$. Thus, even for very different densities $\probs{\usuff}$ and $\probs{\dsuff}$ the optimization process is still very stable. Further,  in our experiments we observed NCE to be much more accurate than unbounded losses in Table \ref{tbl:LogPDFInstances}.

Moreover, such dynamics are not limited only to the loss of NCE, and can actually be enforced through other PSO variants. Herein, we introduce a novel general algorithm family for PSO \emph{log density estimators} (PSO-LDE) that takes a \emph{normalized} form in Eq.~(\ref{eq:PSOLogEstNormalized}). The denominator function is defined as $D_{PSO\!-\!LDE}^{\alpha} \left[X,f_{\theta}(X)\right] \triangleq \big[[\exp f_{\theta}(X)]^{\alpha} + [\probi{\dsuff}{X}]^{\alpha}\big]^{
	\frac{1}{\alpha}}$ with $\alpha$ being family's hyper-parameter. Particularly,
each member of PSO-LDE has bounded \emph{magnitude} functions:
\begin{equation}
M^{\usuff}_{\alpha}
\left[
X, f_{\theta}(X)
\right]
=
\frac{\probi{\dsuff}{X}}
{\left[
	\left[\exp f_{\theta}(X)\right]^{\alpha} + 
	\left[\probi{\dsuff}{X}\right]^{\alpha}
	\right]^{\frac{1}{\alpha}}}
=
\left[
\exp
\left[\alpha
\! \cdot \!
\bar{d}
\left[
X, f_{\theta}(X)
\right]
\right]
+ 1
\right]^{-\frac{1}{\alpha}}
,
\label{eq:PSOLDELossMU}
\end{equation}
\begin{equation}
M^{\dsuff}_{\alpha}
\left[
X, f_{\theta}(X)
\right]
=
\frac{\exp f_{\theta}(X)}
{\left[
	\left[\exp \! f_{\theta}(X)\right]^{\alpha} + 
	\left[\probi{\dsuff}{X}\right]^{\alpha}
	\right]^{\frac{1}{\alpha}}}
=
\left[
\exp
\left[- \alpha
\! \cdot \!
\bar{d}
\left[
X, f_{\theta}(X)
\right]
\right]
+ 1
\right]^{-\frac{1}{\alpha}}
.
\label{eq:PSOLDELossMD}
\end{equation}

In Figure \ref{fig:MagnFuncsPSOLDE} the above \emph{magnitude} functions are plotted w.r.t. logarithm difference $\bar{d}$, for different values of $\alpha$. As can be observed, $\alpha$ controls the smoothness and the rate of a \emph{magnitude} decay to zero. 
Specifically, for smaller $\alpha$ \ms go faster to zero, which implies that the aforementioned active range $[- \varepsilon, \varepsilon]$ is narrower. Thus, small $\alpha$ introduce some elasticity constraints over $f_{\theta}$ that induce smoothness of the converged model.
We argue that these smoother dynamics of smaller $\alpha$ values allow for a more stable optimization and a more accurate convergence, similarly to the robustness of redescending M-estimators \citep{Shevlyakov08jspai}. Yet, we leave the theoretical analysis of this affect for future work. In Section \ref{sec:Exper} we will empirically investigate the impact of $\alpha$ on the performance of density estimation, where we will see that $\alpha = \frac{1}{4}$ typically has a better performance.

Additionally, the formulation of PSO-LDE in Eqs.~(\ref{eq:PSOLDELossMU})-(\ref{eq:PSOLDELossMD}) can be exploited to overcome possible underflow and overflow issues. In a typically used single-precision floating-point format the function $\exp(\cdot)$ can only be computed for values in the range $[-81, 81]$. Hence, there is an upper bound for values of $|\bar{d}\left[
X, f_{\theta}(X)
\right]
|$ above which 
$M^{\usuff}_{\alpha}(X,f_{\theta}(X))$ and $M^{\dsuff}_{\alpha}(X,f_{\theta}(X))$ can not be computed in practice.
Yet, we can set the hyper-parameter $\alpha$ to be small enough to overcome this numerical limitation.

\begin{figure}
	\centering
	
	\begin{tabular}{cccc}
		
		\subfloat{\includegraphics[width=0.7\textwidth]{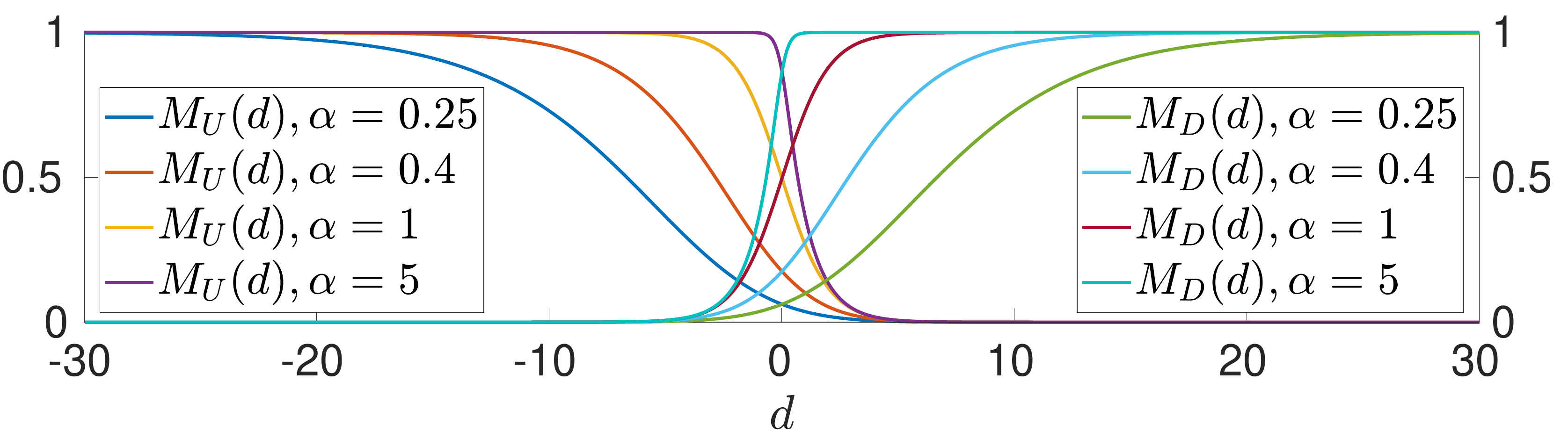}}

	\end{tabular}
	
	\protect
	\caption{PSO-LDE \emph{magnitudes} as functions of a difference $\bar{d} \triangleq f_{\theta}(X) - \log \probi{\dsuff}{X}$ for different values of a hyper-parameter $\alpha$.
	}
	\label{fig:MagnFuncsPSOLDE}
\end{figure}

\begin{remark}
Note that NCE is a member of the above PSO-LDE family for $\alpha = 1$. Further, the analytic loss for magnitudes in Eqs.~(\ref{eq:PSOLDELossMU})-(\ref{eq:PSOLDELossMD}) is unknown for general $\alpha$. Yet, the gradient of this loss can be easily calculated.
\end{remark}

To summarize, by replacing \emph{pdf loss} in Eq.~(\ref{eq:PDFLoss}) with PSO-LDE we succeeded to increase approximation accuracy of density estimation. We show these results in Section \ref{sec:Exper}. Furthermore, unlike typical density estimators, for both DeepPDF and PSO-LDE cases the total integral of the density estimator
is not explicitly constrained to 1, yet was empirically observed to be very close to it. This implies that the proposed herein methods produce an \emph{approximately} normalized density model. For many applications such approximate normalization is suitable. For example, in the estimation of a measurement likelihood model for Bayesian state inference in robotics \citep{Kopitkov18iros} the model is required only to be proportional to the real measurement likelihood.

\section{Conditional Density Estimation}
\label{sec:CondDeepPDFMain}

In this section we show how to utilize PSO \bp to infer conditional (ratio) density functions.

\subsection{Conditional Density Estimation}
\label{sec:CondDeepPDF}

Herein we will focus on problem of conditional density estimation, 
where i.i.d. samples of pairs $\{ X^{\usuff}_{i}, Y^{\usuff}_{i} \}$ are given as:
\begin{equation}
\DD =
\begin{blockarray}{cc}
( \underline{\text{columns of } X^{\usuff}}) &
( \underline{\text{columns of } Y^{\usuff}}) \\
\\
\begin{block}{(c|c)}
X^{\usuff}_{1} & Y^{\usuff}_{1}\\
X^{\usuff}_{2} & Y^{\usuff}_{2}\\
\vdots & \vdots\\
\end{block}
\end{blockarray}
,
\quad
\text{where }
X^{\usuff}_{i} \in \RR^{n_{x}}
,
Y^{\usuff}_{i} \in \RR^{n_{y}}
.
\label{eq:DataStruct}
\end{equation}
Again, we use $U$ to refer to \up force in PSO framework as will be described below.
For any dataset $\DD$, the generation process of its samples is governed by the following unknown data densities: $\PP^{\usuff}_{XY}(X,Y)$, $\PP^{\usuff}_{X}(X)$ and $\PP^{\usuff}_{Y}(Y)$. 
Specifically, in Eq.~(\ref{eq:DataStruct}) rows under $X^{\usuff}$ columns will be distributed by marginal pdf $\PP^{\usuff}_{X}(X)$, rows under $Y^{\usuff}$ columns - by marginal pdf $\PP^{\usuff}_{Y}(Y)$, and entire rows of $\DD$ will have the joint density $\PP^{\usuff}_{XY}(X,Y)$ (see also Table \ref{tbl:CondNottns} for list of main notations). Likewise, these densities induce the conditional likelihoods
$\PP^{\usuff}_{X|Y}(X|Y)$ and $\PP^{\usuff}_{Y|X}(Y|X)$, which can be formulated via Bayes theorem:
\begin{equation}
\PP^{\usuff}_{X|Y}(X|Y)
=
\frac{\PP^{\usuff}_{XY}(X,Y)}{\PP^{\usuff}_{Y}(Y)}
,
\quad
\PP^{\usuff}_{Y|X}(Y|X)
=
\frac{\PP^{\usuff}_{XY}(X,Y)}{\PP^{\usuff}_{X}(X)}
.
\label{eq:BayesThrm}
\end{equation}
Depending on the task at hand, conditional pdf $\PP^{\usuff}_{X|Y}(X|Y)$ can produce valuable information about given data.

\begin{table}[tb]
	\small
	\centering
	\begin{tabular}{ll}
		\toprule
		\textbf{Notation}     & \textbf{Description}      \\
		\midrule
		$X^{\usuff} \sim \PP^{\usuff}_{X}(X)$
		&
		$n_{x}$-dimensional random variable with marginal pdf $\PP^{\usuff}_{X}$
		\\
		$Y^{\usuff} \sim \PP^{\usuff}_{Y}(Y)$
		&
		$n_{y}$-dimensional random variable with marginal pdf $\PP^{\usuff}_{Y}$
		\\
		$[X^{\usuff}, Y^{\usuff}]$
		&
		$n$-dimensional random variable with joint pdf $\PP^{\usuff}_{XY}(X,Y)$, 
		\\
		& at samples of which we push the model surface \up
		\\
		$n = n_{x} + n_{y}$
		&
		joint dimension of random variable $[X^{\usuff}, Y^{\usuff}]$
		\\
		$\PP^{\usuff}_{X|Y}(X|Y)$
		&
		conditional probability density function of $X \equiv X^{\usuff}$ given $Y \equiv Y^{\usuff}$
		\\
		$X^{\dsuff} \sim \probs{\dsuff}$
		&
		$n_{x}$-dimensional random variable with pdf $\probs{\dsuff}$
		\\
		$[X^{\dsuff}, Y^{\dsuff}]$
		&
		$n$-dimensional random variable with joint pdf $\probi{\dsuff}{X} \cdot \PP^{\usuff}_{Y}(Y)$, 
		\\
		& at samples of which we push the model surface \down
		\\
		\bottomrule
	\end{tabular}
	
	\caption{Main Notations for Conditional Density Estimators}
	\label{tbl:CondNottns}
	
\end{table}

The simple way to infer $\PP^{\usuff}_{X|Y}(X|Y)$ is by first approximating separately the $\PP^{\usuff}_{XY}(X,Y)$ and $\PP^{\usuff}_{Y}(Y)$ from data samples (e.g. by using DeepPDF or PSO-LDE), and further applying Bayes theorem in Eq.~(\ref{eq:BayesThrm}).
Yet, such method is not computationally efficient and typically is also not optimal, since approximation errors of both functions can produce even bigger error in the combined function.

A different technique, based on PSO principles, can be performed as follows. Consider a model (PSO surface) $f_{\theta}(X, Y): \RR^{n} \rightarrow \RR$ with $n = n_{x} + n_{y}$, where the concatenated input $[X, Y]$ can be seen as the surface support. Define an arbitrary density $\probs{\dsuff}$ over $\RR^{n_{x}}$ with a known pdf function which can be easily sampled (e.g. Uniform). Density $\probs{\dsuff}$ will serve as a \down force to balance samples from $\PP^{\usuff}_{X}$, and thus is required to cover the support of $\PP^{\usuff}_{X}$. Further, $\PP^{\usuff}_{XY}(X, Y)$ will serve as \up density in PSO framework, and its sample batch $\{ X^{\usuff}_{i}, Y^{\usuff}_{i} \}_{i = 1}^{N^{\usuff}}$ will contain all rows from $\DD$. As well, we will use $\probi{\dsuff}{X} \cdot \PP^{\usuff}_{Y}(Y)$ as \down density. Corresponding samples $\{ X^{\dsuff}_{i}, Y^{\dsuff}_{i} \}_{i = 1}^{N^{\dsuff}}$ will be sampled in two steps. $\{ Y^{\dsuff}_{i} \}_{i = 1}^{N^{\dsuff}}$ are taken from $\DD$ under $Y^{\usuff}$ columns; $\{ X^{\dsuff}_{i} \}_{i = 1}^{N^{\dsuff}}$ are sampled from $\probi{\dsuff}{X}$.

Considering the above setup, we can apply PSO to push $f_{\theta}(X, Y)$ via \up and \down forces. The optimization gradient will be identical to Eq.~(\ref{eq:GeneralPSOLossFrml}) where $X^{\usuff}_{i}$ and $X^{\dsuff}_{i}$ are substituted by $\{ X^{\usuff}_{i}, Y^{\usuff}_{i} \}$ and $\{ X^{\dsuff}_{i}, Y^{\dsuff}_{i} \}$ respectively.
For any particular \pair the associated PSO \bp will be:
\begin{equation}
\frac{M^{\dsuff}\left[
	X, Y,
	f_{\theta}(X, Y)
	\right]}{M^{\usuff}\left[
	X, Y,
	f_{\theta}(X, Y)
	\right]}
=
\frac{\PP^{\usuff}_{XY}(X, Y)}{\probi{\dsuff}{X} \cdot \PP^{\usuff}_{Y}(Y)}
=
\frac{\PP^{\usuff}_{X|Y}(X|Y)}{\probi{\dsuff}{X}}
,
\label{eq:CondPSOBalll}
\end{equation}
where we can observe $\PP^{\usuff}_{X|Y}(X|Y)$, which we aim to learn. Similarly to Section \ref{sec:PSOInst}, below we formulate PSO subgroup for the conditional density estimation (or any function of it).

\begin{theorem}[Conditional Density Estimation]
\label{thrm:PSO_bal_state_new_cond} 

Denote the required PSO convergence by a transformation $T(X, Y, z): \RR^{n} \times \RR \rightarrow \RR$ s.t. $f_{\theta}(X, Y) = T\left[X, Y, \PP^{\usuff}_{X|Y}(X|Y) \right]$ is the function we want to learn. 
Denote its inverse function w.r.t. $z$ as $T^{-1}(X, Y, s)$. Then, any pair \pair satisfying:
	\begin{equation}
	\frac{M^{\dsuff}(X, Y, s)}{M^{\usuff}(X, Y, s)} = 
	\frac{T^{-1}(X, Y, s)}{\probi{\dsuff}{X}}
	,
	\label{eq:PSO_conv_th_new_cond}
	\end{equation}
	will produce the required convergence.
\end{theorem}
The proof is trivial by noting that:
\begin{multline}
T^{-1}(X, Y, f_{\theta}(X, Y))
= 
\probi{\dsuff}{X}
\cdot
\frac{M^{\dsuff}(X, Y, f_{\theta}(X, Y))}{M^{\usuff}(X, Y, f_{\theta}(X, Y))}
=
\probi{\dsuff}{X}
\cdot
\frac{\PP^{\usuff}_{X|Y}(X|Y)}{\probi{\dsuff}{X}}
\quad
\Rightarrow\\
\Rightarrow
\quad
T^{-1}(X, Y, f_{\theta}(X, Y))
= \PP^{\usuff}_{X|Y}(X|Y)
\label{eq:RatioTransform_cond}
\end{multline}
where we used both Eq.~(\ref{eq:CondPSOBalll}) and Eq.~(\ref{eq:PSO_conv_th_new_cond}). From properties of inverse functions it follows:
\begin{equation}
T\left[X, Y, \PP^{\usuff}_{X|Y}(X|Y) \right]=
T\left[X, Y, T^{-1}(X, Y, f_{\theta}(X, Y)) \right]
=
f_{\theta}(X, Y)
.
\label{eq:RatioTransform_cond2}
\end{equation}
Further, the sufficient conditions over mappings $T$, $M^{\usuff}$ and $M^{\dsuff}$ are omitted since they already appear in Theorem \ref{thrm:PSO_func_conv}.

\begin{table}
	\centering
	\begin{tabular}{lllll}
		\toprule
		
		\textbf{Method}     & \textbf{
			\emph{\underline{F}inal} $f_{\theta}(X)$ / \emph{\underline{R}eferences} /
			\emph{\underline{L}oss} / \emph{$M^{\usuff}(\cdot)$ and $M^{\dsuff}(\cdot)$}
		}      \\

		\midrule

		Conditional
		& \underline{F}: $\probi{\usuff}{X|Y}$
		\\
		Density & \underline{R}: This paper
		\\
		Estimation & 		
		\underline{L}: 
		$-
		\E_{[X, Y] \sim \PP^{\usuff}_{XY}(X,Y)}
		f_{\theta}(X, Y)
		\cdot
		\probi{\dsuff}{X}
		+$
		\\ &
		$\quad +
		\E_{[X, Y] \sim \probi{\dsuff}{X} \cdot \PP^{\usuff}_{Y}(Y)}
		\half
		\Big[
		f_{\theta}(X, Y)
		\Big]^{2}
		$
		\\
		&
		$M^{\usuff}$, $M^{\dsuff}$: $
		\probi{\dsuff}{X}
		$
		,
		$
		f_{\theta}(X, Y)
		$   \\
		\midrule
		Conditional     
		& \underline{F}: $\log \probi{\usuff}{X|Y}$
		\\
		Log-density & \underline{R}: This paper
		\\
		Estimation & 		
		\underline{L}: 
		$-
		\E_{[X, Y] \sim \PP^{\usuff}_{XY}(X,Y)}
		f_{\theta}(X, Y)
		+
		\E_{[X, Y] \sim \probi{\dsuff}{X} \cdot \PP^{\usuff}_{Y}(Y)}
		\frac{\exp [f_{\theta}(X, Y)]}{\probi{\dsuff}{X}}$
		\\[5pt]
		&
		$M^{\usuff}$, $M^{\dsuff}$: $1$,
		$
		\frac{\exp [f_{\theta}(X, Y)]}{\probi{\dsuff}{X}}
		$  \\
		\midrule
		NCE   
		& \underline{F}: 
		$\log \probi{\usuff}{X|Y}$
		\\
		Conditional
		& \underline{R}: \citet{Mnih12arxiv,Mnih13nips}
		\\[3pt]
		Form
		& 		
		\underline{L}: 
		$
		\E_{[X, Y] \sim \PP^{\usuff}_{XY}(X,Y)}
		\log \frac{\exp[ f_{\theta}(X, Y)] + \probi{\dsuff}{X}}{\exp[ f_{\theta}(X, Y)]}
		+
		$
		\\ &
		$\quad +
		\E_{[X, Y] \sim \probi{\dsuff}{X} \cdot \PP^{\usuff}_{Y}(Y)}
		\log \frac{\exp[ f_{\theta}(X, Y)] + \probi{\dsuff}{X}}{\probi{\dsuff}{X}}$
		\\[5pt]
		&
		$M^{\usuff}$, $M^{\dsuff}$: $
		\frac{\probi{\dsuff}{X}}{\exp[ f_{\theta}(X, Y)] + \probi{\dsuff}{X}}
		$,
		$
		\frac{\exp[ f_{\theta}(X, Y)]}{\exp[ f_{\theta}(X, Y)] + \probi{\dsuff}{X}}
		$   \\
		\midrule
		PSO-LDE   
		& \underline{F}: $\log \probi{\usuff}{X|Y}$
		\\
		Conditional
		& \underline{R}: This paper
		\\
		Form
		& 		
		\underline{L}: 
		unknown
		\\
		&
		$M^{\usuff}$, $M^{\dsuff}$: 
		$
		\frac{\probi{\dsuff}{X}}
		{\left[
			\left[\exp f_{\theta}(X, Y)\right]^{\alpha} + 
			\left[\probi{\dsuff}{X}\right]^{\alpha}
			\right]^{\frac{1}{\alpha}}}
		$,
		$
		\frac{\exp f_{\theta}(X, Y)}
		{\left[
			\left[\exp f_{\theta}(X, Y)\right]^{\alpha} + 
			\left[\probi{\dsuff}{X}\right]^{\alpha}
			\right]^{\frac{1}{\alpha}}}
		$  
		\\
		\midrule
		Conditional  
		& \underline{F}: 
		$\frac{\PP^{\usuff}_{X|Y}(X|Y)}{\PP^{\usuff}_{X|Y}(X|Y) + \probii{\dsuff}{\phi}{X|Y}}$
		,
		\\ 
		GAN Critic
		&
		\quad\quad
		where
		$\probii{\dsuff}{\phi}{X|Y}$ is density of generator $h_{\phi}$ parametrized by $\phi$
		\\
		& \underline{R}: \citet{Mirza14arxiv}
		\\
		&
		\underline{L}: 
		$-
		\E_{[X, Y] \sim \PP^{\usuff}_{XY}(X,Y)}
		\log f_{\theta}(X, Y)		
		-
		$
		\\ &
		$\quad -
		\E_{[X, Y] \sim \probii{\dsuff}{\phi}{X|Y} \cdot \PP^{\usuff}_{Y}(Y)}
		\log
		\Big[
		1 - f_{\theta}(X, Y)
		\Big]$
		
		\\
		&
		$M^{\usuff}$, $M^{\dsuff}$: $
		\frac{1}{f_{\theta}(X, Y)}
		$,
		$
		\frac{1}{1 - f_{\theta}(X, Y)}
		$ \\
		\midrule
		Likelihood-Ratio
		& \underline{F}: $\log \frac{\PP^{\usuff}_{X|Y}(X|Y)}{\probii{\dsuff}{\phi}{X|Y}}$,
		\\ 
		with
		&
		\quad\quad
		where
		$\probii{\dsuff}{\phi}{X|Y}$ is density of generator $h_{\phi}$ parametrized by $\phi$
		\\
		Logistic Loss
		& \underline{R}: This paper
		\\
		&
		\underline{L}: 
		$
		\E_{[X, Y] \sim \PP^{\usuff}_{XY}(X,Y)}
		\log \big[
		1 + \exp [- f_{\theta}(X, Y)]
		\big]
		+$
		\\ &
		$\quad + 
		\E_{[X, Y] \sim \probii{\dsuff}{\phi}{X|Y} \cdot \PP^{\usuff}_{Y}(Y)}
		\log \big[
		1 + \exp [f_{\theta}(X, Y)]
		\big]
		$
		\\[5pt]
		&
		$M^{\usuff}$, $M^{\dsuff}$: $
		\frac{1}{\exp [f_{\theta}(X, Y)] + 1}
		$
		,
		$
		\frac{1}{\exp [- f_{\theta}(X, Y)] + 1}
		$ \\
		
		\bottomrule
	\end{tabular}
	
	\caption[PSO Instances For Conditional Density (Ratio) Estimation.]{PSO Instances For Conditional Density (Ratio) Estimation, see Sections \ref{sec:CondDeepPDF} and \ref{sec:CondGANs} for a detailed exposition of conditional PSO}
	\label{tbl:PSOInstances6}
	
\end{table}

\paragraph{Example 5:}
Consider a scenario where we would like to infer $f_{\theta}(X, Y) = \PP^{\usuff}_{X|Y}(X|Y)$. Thus, the PSO convergence is described by $T(X, Y, z) = z$. Its inverse is $T^{-1}(X, Y, s) = s$. Hence, \mfs must satisfy $\frac{M^{\dsuff}(X, Y, f_{\theta}(X, Y))}{M^{\usuff}(X, Y, f_{\theta}(X, Y))} = \frac{f_{\theta}(X, Y)}{\probi{\dsuff}{X}}$. One choice for such \ms is $M^{\usuff}
\left[
X,
Y,
f_{\theta}(X, Y)
\right] = \probi{\dsuff}{X}$ and $M^{\dsuff}
\left[
X,
Y,
f_{\theta}(X, Y)
\right] = f_{\theta}(X, Y)$, defined in Table \ref{tbl:PSOInstances6} as "Conditional Density Estimation".
\\

\paragraph{Example 6:}
Consider a scenario where we would like to infer $f_{\theta}(X, Y) = \log \PP^{\usuff}_{X|Y}(X|Y)$, which can be essential for high-dimensional data. The PSO convergence is described by $T(X, Y, z) = \log z$, and its inverse is $T^{-1}(X, Y, s) = \exp s$. Hence, \mfs must satisfy $\frac{M^{\dsuff}(X, Y, f_{\theta}(X, Y))}{M^{\usuff}(X, Y, f_{\theta}(X, Y))} = \frac{\exp f_{\theta}(X, Y)}{\probi{\dsuff}{X}}$. One choice for such \ms is $M^{\usuff}
\left[
X,
Y,
f_{\theta}(X, Y)
\right] = \frac{\probi{\dsuff}{X}}{D(X, Y, f_{\theta}(X, Y))}$ and $M^{\dsuff}
\left[
X,
Y,
f_{\theta}(X, Y)
\right] = \frac{\exp f_{\theta}(X, Y)}{D(X, Y, f_{\theta}(X, Y))}$, defined in Table \ref{tbl:PSOInstances6} as "PSO-LDE Conditional Form". The denominator $D(X, Y, f_{\theta}(X, Y)) = \left[
\left[\exp f_{\theta}(X, Y)\right]^{\alpha} + 
\left[\probi{\dsuff}{X}\right]^{\alpha}
\right]^{\frac{1}{\alpha}}$ serves as a normalization to enforce \ms to be bounded functions, similarly to PSO-LDE method in Section \ref{sec:DeepLogPDF}.
\\

Thus, we can estimate the conditional density, or any function of it, in a one-step algorithm by applying PSO procedure with \up and \down densities defined above. This again emphasizes the simplicity and usability of PSO formulation. Further, note that it is also possible to reuse sample $Y^{\usuff}_{i}$ as $Y^{\dsuff}_{i}$, since within the \down term this sample will still be independent of $X^{\dsuff}_{i}$ and its density is still the marginal $\PP^{\usuff}_{Y}(Y)$. Such reuse is popular for example in NCE methods \citep{Mnih12arxiv,Mnih13nips} in context of language modeling.

The above examples and several other options are listed in Table \ref{tbl:PSOInstances6}. Similarly to a case of the ordinary density estimation, also in the conditional case there are numerous PSO instances with the same target function $\PP^{\usuff}_{X|Y}(X|Y)$ (or $\log \PP^{\usuff}_{X|Y}(X|Y)$). Analyses of these techniques and search for the most "optimal" can be an interesting direction for future research.

\subsection{Relation to Conditional GANs}
\label{sec:CondGANs}

Furthermore, a similar idea was also presented in the context of GANs, where a conditional generation of data (e.g. images given labels) was explored. Below we show its connection to PSO framework.

Denote the dataset $\DD$ as in Eq.~(\ref{eq:DataStruct}), where \emph{real} sample pairs $\{ X^{\usuff}_{i}, Y^{\usuff}_{i} \}$ are distributed according to unknown $\PP^{\usuff}_{XY}(X,Y)$.
In Conditional GAN (cGAN) \citep{Mao17iccv} the generator produces \emph{fake} samples from the generator's conditional density $\probii{\dsuff}{\phi}{X|Y}$, where we again use notations $U$ and $D$ to refer to PSO forces, as is described below. Density $\probii{\dsuff}{\phi}{X|Y}$ is an implicit distribution of \emph{fake} samples that are returned by the generator $h_{\phi}(\upsilon, Y)$ from the latent space $\upsilon \in \RR^{n_{\upsilon}}$, where the label $Y$ was a priori sampled from $\PP^{\usuff}_{Y}(Y)$; $\phi$ is a generator's parametrization. Further, the critic sees pairs $[X, Y]$ coming from $\DD$ and from the generator, and tries to decide where the pair is originated from. This is done by estimating a statistical divergence between $\PP^{\usuff}_{X|Y}(X|Y)$ implicitly defined by $\DD$, and between $\probii{\dsuff}{\phi}{X|Y}$ implicitly defined by $h_{\phi}$. Moreover, the divergence estimation is typically done by first inferring the ratio $\frac{\PP^{\usuff}_{X|Y}(X|Y)}{\probii{\dsuff}{\phi}{X|Y}}$ (or some function of this ratio).

The proposed by \citep{Mao17iccv} algorithm is identical to PSO procedure, when $\PP^{\usuff}_{XY}(X,Y)$ serves as \up density, and $\probii{\dsuff}{\phi}{X|Y} \cdot \PP^{\usuff}_{Y}(Y)$ - as \down density. The \up sample batch $\{ X^{\usuff}_{i}, Y^{\usuff}_{i} \}_{i = 1}^{N^{\usuff}}$ will contain all rows from $\DD$. Further, samples $\{ X^{\dsuff}_{i}, Y^{\dsuff}_{i} \}_{i = 1}^{N^{\dsuff}}$ from \down density will be sampled in three steps. $\{ Y^{\dsuff}_{i} \}_{i = 1}^{N^{\dsuff}}$ are taken from $\DD$ under $Y^{\usuff}$ columns; $\{ {\upsilon}^i \}_{i = 1}^{N^{\dsuff}}$ are sampled from generator's base distribution; $\{ X^{\dsuff}_{i} \}_{i = 1}^{N^{\dsuff}}$ are generator's outputs for inputs $\{ Y^{\dsuff}_{i}, {\upsilon}_i \}_{i = 1}^{N^{\dsuff}}$. Extending the setup of Section \ref{sec:CondDeepPDF} to the above sampling procedure,
any particular \pair will produce PSO \bp:
\begin{equation}
\frac{M^{\dsuff}\left[
	X, Y,
	f_{\theta}(X, Y)
	\right]}{M^{\usuff}\left[
	X, Y,
	f_{\theta}(X, Y)
	\right]}
=
\frac{\PP^{\usuff}_{XY}(X,Y)}{\probii{\dsuff}{\phi}{X|Y} \cdot \PP^{\usuff}_{Y}(Y)}
=
\frac{\PP^{\usuff}_{X|Y}(X|Y)}{\probii{\dsuff}{\phi}{X|Y}}
,
\label{eq:CondPSOBalll_cgan}
\end{equation}
where the conditional ratio shows up. Similarly to the conditional density estimation, below we formulate PSO subgroup for inference of this ratio (or any function of it).

\begin{theorem}[Conditional Ratio Estimation]
\label{thrm:PSO_bal_state_new_cond_ratio} 
	
Denote the required PSO convergence by a transformation $T(X, Y, z): \RR^{n} \times \RR \rightarrow \RR$ s.t. $f_{\theta}(X, Y) = T\left[X, Y, \frac{\PP^{\usuff}_{X|Y}(X|Y)}{\probii{\dsuff}{\phi}{X|Y}} \right]$ is the function we want to learn.
Denote its inverse function w.r.t. $z$ as $T^{-1}(X, Y, s)$. Then, any pair \pair satisfying:
	\begin{equation}
	\frac{M^{\dsuff}(X, Y, s)}{M^{\usuff}(X, Y, s)} = 
	T^{-1}(X, Y, s)
	,
	\label{eq:PSO_conv_th_new_cond_ratio}
	\end{equation}
	will produce the required convergence.
\end{theorem}
The proof is trivial by noting that:
\begin{multline}
T^{-1}(X, Y, f_{\theta}(X, Y))
= 
\frac{M^{\dsuff}(X, Y, f_{\theta}(X, Y))}{M^{\usuff}(X, Y, f_{\theta}(X, Y))}
=
\frac{\PP^{\usuff}_{X|Y}(X|Y)}{\probii{\dsuff}{\phi}{X|Y}}
\quad
\Rightarrow\\
\Rightarrow
\quad
T^{-1}(X, Y, f_{\theta}(X, Y))
= 
\frac{\PP^{\usuff}_{X|Y}(X|Y)}{\probii{\dsuff}{\phi}{X|Y}}
.
\label{eq:RatioTransform_cond_ratio}
\end{multline}
From properties of inverse functions the Theorem follows.

The cGAN method aimed to infer $\frac{\PP^{\usuff}_{X|Y}(X|Y)}{\PP^{\usuff}_{X|Y}(X|Y) + \probii{\dsuff}{\phi}{X|Y}}$ to measure Jensen-Shannon divergence between \emph{real} and \emph{fake} distributions. This is associated with PSO convergence $T(X, Y, z) = \frac{z}{z + 1}$ and the corresponding inverse $T^{-1}(X, Y, s) = \frac{s}{1 - s}$. According to Theorem \ref{thrm:PSO_bal_state_new_cond_ratio} the \ms must satisfy:
\begin{equation}
\frac{M^{\dsuff}(X, Y, f_{\theta}(X, Y))}{M^{\usuff}(X, Y, f_{\theta}(X, Y))}
=
\frac{f_{\theta}(X, Y)}{1 - f_{\theta}(X, Y)}
,
\label{eq:RatioTransform_cond_ratio_cgan}
\end{equation}
with the specific choice $\{
M^{\usuff}(X, Y, f_{\theta}(X, Y)) = \frac{1}{f_{\theta}(X, Y)}
,
M^{\dsuff}(X, Y, f_{\theta}(X, Y)) =
\frac{1}{1 - f_{\theta}(X, Y)}
\}
$ selected by the critic loss in \citep{Mao17iccv}.

Hence, we can see that cGAN critic loss is a particular instance of PSO, when the sampling procedure of \up and \down samples is as described above.
Further,
two main problems of the classical cGAN critic are the not-logarithmic scale of the target function and unboundedness of \mfs (for a general model $f_{\theta}(X, Y)$). In Table \ref{tbl:PSOInstances6} we propose a "Likelihood-Ratio with Logistic Loss" to learn $\log \frac{\PP^{\usuff}_{X|Y}(X|Y)}{\probii{\dsuff}{\phi}{X|Y}}$ whose \ms are bounded. We argue such choice to be more stable during the optimization which will lead to a better accuracy. 
Moreover, for specific case when cGAN critic $f_{\theta}(X, Y)$ is parameterized as $sigmoid(h_{\theta}(X, Y))$ with $h_{\theta}(X, Y)$ being the inner model, cGAN critic loss can be shown to be reduced to the above logistic loss. Thus, with such parametrization the inner NN $h_{\theta}(X, Y)$ within cGAN critic will converge to $\log \frac{\PP^{\usuff}_{X|Y}(X|Y)}{\probii{\dsuff}{\phi}{X|Y}}$.

\section{Additional Applications and Relations of PSO Framework}
\label{sec:PSOApp}

In this section we demonstrate how PSO principles can be exploited beyond the (conditional) pdf inference problem. 
Particularly, we relate PSO and cross-entropy loss, showing the latter to be a specific instance of the former. We also outline relation between PSO and MLE by deriving the latter from \psofunc. Further, we analyze PSO instance with unit \mfs and describe its connection to CD method \citep{Hinton02nc}. Additionally, we show how to use PSO for learning mutual information from available data samples.

\subsection{Cross-Entropy as Instance of PSO}
\label{sec:CrosEntr}

In this section we will show that the binary cross-entropy loss combined with a $sigmoid$ non-linearity, typical in binary classification problems, is instance of PSO. Further, in Appendix \ref{sec:App9} we extend this setup also to a more general case of $softmax$ cross-entropy. Similarly to a binary $sigmoid$ cross-entropy, a multi-class $softmax$ cross-entropy is shown to be a PSO instance, extended to models with multi-dimensional outputs. Thus, the optimization of multi-class $softmax$ cross-entropy can be seen as pushes of dynamical forces over $C$ different surfaces $\{f_{\theta}(X)_i\}_{i = 1}^{C}$ - the outputs of the model $f_{\theta}(X) \in \RR^{C}$ per each class.

To prove the above point, we derive the binary cross-entropy loss using PSO principles. Define training dataset of pairs $\{ X_i, Y_i \}_{i = 1}^{N}$ where $X_i \in \RR^{n}$ is data point of an arbitrary dimension $n$ (e.g. image) and $Y_i$ is its label - the discrete number that takes values from $\{0, 1\}$. Denote by $N_1$ and $N_0$ the number of samples with labels 1 and 0 respectively. Further, assume each sample pair to be i.i.d. sampled from an unknown density $\PP(X, Y) = \PP(X) \cdot \PP(Y | X)$. Our task is to enforce the output of $\sigma(f_{\theta}(X))$, the $sigmoid$ non-linearity over inner model $f_{\theta}(X)$, to converge to unknown conditional $\PP(Y = 1 | X)$. Such convergence is equivalent to
\begin{equation}
f_{\theta}(X) = 
- \log
\left[
\frac{1}{\PP(Y = 1 | X)}
-
1
\right]
=
\log
\frac{\PP(X, Y = 1)}{\PP(X, Y = 0)}
=
\log
\frac{\PP(X | Y = 1) \cdot \PP(Y = 1)}{\PP(X | Y = 0) \cdot \PP(Y = 0)}
.
\label{eq:CEBConvF}
\end{equation}

To apply PSO, we consider $\PP(X | Y = 1)$ as \up density $\probs{\usuff}$ and $\PP(X | Y = 0)$ as \down density $\probs{\dsuff}$. Sample batches $\{X^{\usuff}_{i}\}_{i = 1}^{N_1}$ and $\{X^{\dsuff}_{i}\}_{i = 1}^{N_0}$ from both can be obtained by fetching $X_i$ with appropriate label $Y_i$. Then the required convergence is described by $T(X, z) = \log
\frac{\PP(Y = 1)}{\PP(Y = 0)}
\cdot
z
$, with $f^*(X) = T(X, \frac{\PP(X | Y = 1)}{\PP(X | Y = 0)})$. Further, the $T$'s inverse is $R(X, s) = \frac{\PP(Y = 0)}{\PP(Y = 1)} \cdot \exp s$, and
according to Theorem \ref{thrm:PSO_func_conv}  \ms must satisfy $\frac{M^{\dsuff}
	\left[
	X,
	f_{\theta}(X)
	\right]}{M^{\usuff}
	\left[
	X,
	f_{\theta}(X)
	\right]} = \frac{\PP(Y = 0) \cdot \exp f_{\theta}(X)}{\PP(Y = 1)}$. One possible choice is:
\begin{equation}
M^{\usuff}
\left[
X,
f_{\theta}(X)
\right]
=
\frac{\PP(Y = 1)}{1 + \exp f_{\theta}(X)}
,
\quad
M^{\dsuff}
\left[
X,
f_{\theta}(X)
\right]
=
\frac{\PP(Y = 0) \cdot \exp f_{\theta}(X)}{1 + \exp f_{\theta}(X)}
\label{eq:CEBMagnFuncs}
\end{equation}
where the denominator $1 + \exp f_{\theta}(X)$ serves as a normalization factor that enforces \pair to be between 0 and 1. Further, the above \mfs have known antiderivatives:
\begin{equation}
\widetilde{M}^{\usuff}
\left[
X,
f_{\theta}(X)
\right]
=
\PP(Y = 1)
\cdot
\log
\left[
\sigma(f_{\theta}(X))
\right]
,
\;
\widetilde{M}^{\dsuff}
\left[
X,
f_{\theta}(X)
\right]
=
-
\PP(Y = 0)
\cdot
\log
\left[
1 -
\sigma(f_{\theta}(X))
\right]
\label{eq:CEBMagnTildeFuncs}
\end{equation}
that produce the following \emph{PSO functional}:
\begin{equation}
L_{PSO}(f)
=
-
\E_{X \sim \PP(X | Y = 1)}
\PP(Y = 1)
\cdot
\log
\left[
\sigma(f_{\theta}(X))
\right]
-
\E_{X \sim \PP(X | Y = 0)}
\PP(Y = 0)
\cdot
\log
\left[
1 -
\sigma(f_{\theta}(X))
\right]
.
\label{eq:CEBLoss_PSO}
\end{equation}

Finally, considering $\frac{N_1}{N}$ and $\frac{N_0}{N}$ as estimators of $\PP(Y = 1)$ and $\PP(Y = 0)$ respectively, the empirical version of the above loss is:
\begin{multline}
L_{PSO}(f)
\approx
-
\frac{1}{N_1}
\sum_{i = 1}^{N_1}
\frac{N_1}{N}
\cdot
\log
\left[
\sigma(f_{\theta}(X^{\usuff}_{i}))
\right]
-
\frac{1}{N_0}
\sum_{i = 1}^{N_0}
\frac{N_0}{N}
\cdot
\log
\left[
1 -
\sigma(f_{\theta}(X^{\dsuff}_{i}))
\right]
=\\
=
-
\frac{1}{N}
\sum_{i = 1}^{N}
\Big[
Y_i
\cdot
\log
\left[
\sigma(f_{\theta}(X_i))
\right]
+
\left[
1 - Y_i
\right]
\cdot
\log
\left[
1 -
\sigma(f_{\theta}(X_i))
\right]
\Big]
,
\label{eq:CEBLoss_PSO_emp}
\end{multline}
where we combine two sums of the first row into a single sum after introducing indicators $Y_i$ and $1 - Y_i$.

The second row is known in Machine Learning community as the binary cross-entropy loss.
Therefore, we can conclude that PSO instance with \ms in Eq.~(\ref{eq:CEBMagnFuncs}) corresponds to cross-entropy when $\PP(X | Y = 1)$ and $\PP(X | Y = 0)$ serve as \up and \down densities respectively. See a similar derivation for multi-class cross-entropy in Appendix \ref{sec:App9}. Therefore, convergence and stability properties of PSO are also shared by the \emph{supervised} classification domain which further motivates PSO analysis.

\subsection{Relation to Maximum Likelihood Estimation}
\label{sec:MLE_PSO}

Below we establish the relation between MLE and PSO procedures, by deriving MLE approach from principles of PSO. Consider a batch of i.i.d. samples $\{ X^{\usuff}_{i} \}_{i = 1}^{N^{\usuff}}$ sampled from density $\probs{\usuff}$, whose pdf we aim to estimate.
Define an auxiliary distribution $\probs{\dsuff}$ with analytically known pdf $\probi{\dsuff}{X}$ that satisfies $\spp^{\usuff} \subseteq \spp^{\dsuff}$, and define \psofunc as:
\begin{equation}
L_{PSO}(f)
=
-
\E_{X \sim \probs{\usuff}}
\left[
1
+
\log
f(X)
\right]
+
\E_{X \sim \probs{\dsuff}}
\frac{f(X)}{\probi{\dsuff}{X}}
\label{eq:MLE_PSO_loss}
\end{equation}
that is induced by the following \mfs:
\begin{equation}
M^{\usuff}
\left[
X,
f(X)
\right]
=
\frac{1}{f(X)}
,
\quad
M^{\dsuff}
\left[
X,
f(X)
\right]
=
\frac{1}{\probi{\dsuff}{X}}
.
\label{eq:MLE_PSO_loss_MagnFuncs}
\end{equation}
Define the hypothesis class $\FF$ with functions that are positive on $\spp^{\usuff}$ so that $\log
f(X)
$ is properly defined for all $f \in \FF$ and $X \in \spp^{\usuff}$. Then, the optimal $f^* = \argmin_{f \in \FF} L_{PSO}(f)$ will satisfy PSO \bp in Eq.~(\ref{eq:BalPoint}) which yields $f^*(X) = \probi{\usuff}{X}$.

Furthermore, in case $\FF$ is a space of positive functions whose total integral is equal to 1 (i.e. probability measure space), the above loss can be reduced to:
\begin{equation}
L_{PSO}(f)
=
\int
-
\probi{\usuff}{X}
\cdot
\left[
1
+
\log
f(X)
\right]
+
f(X)
dX
=
-
\E_{X \sim \probs{\usuff}}
\log
f(X)
,
\label{eq:MLE_PSO_loss2}
\end{equation}
where we apply an equality $\int \probi{\usuff}{X} dX = \int f(X) dX$ since $f$ is normalized. Note that limiting $\FF$ to be a probability measure space does not affect the \bp of PSO since the optimal solution $f^*$ is also a probability measure. Further, considering the physical perspective of PSO such choice of $\FF$ has an implicit regularization affect, removing a need for the \down force $F_{\theta}^{\dsuff}$ in order to achieve the force equilibrium over the surface $f(X)$.

The loss in Eq.~(\ref{eq:MLE_PSO_loss2}) and its empirical variant $L_{PSO}(f)
\approx
-
\frac{1}{N^{\usuff}}
\sum_{i = 1}^{N^{\usuff}}
\log f(X^{\usuff}_{i})
$ define the standard MLE procedure.
Therefore, we can conclude that PSO with \ms defined in Eq.~(\ref{eq:MLE_PSO_loss_MagnFuncs}) corresponds to MLE when the data distribution $\probs{\usuff}$ is absolutely continuous w.r.t. $\probs{\dsuff}$ and when each $f \in \FF$ is a normalized function. Likewise, such relation can also be explained by the connection between PSO and Kullback-Leibler (KL) divergencies described in Section \ref{sec:Bregman_PSO}.

\subsection{PSO with Unit \emph{Magnitudes} and Contrastive Divergence}
\label{sec:SimpL}

The PSO instance with unit \emph{magnitudes} $M^{\usuff}
\left[
X,
f_{\theta}(X)
\right] = M^{\dsuff}
\left[
X,
f_{\theta}(X)
\right] = 1$
can be frequently met in Machine Learning (ML) literature. For example, Integral Probability Metrics (IPMs) \citep{Muller97aap}, contrastive divergence (CD) \citep{Hinton02nc}, Maximum
Mean Discrepancy (MMD) \citep{Gretton07nips} and critic of the Wasserstein GAN \citep{Arjovsky17arxiv} all rely on this loss to measure some distance between densities $\probs{\usuff}$ and $\probs{\dsuff}$. In this section we will explore this \emph{unit} loss
\begin{equation}
L_{unit}(f_{\theta}) = 
-
\E_{X \sim \probs{\usuff}}
f_{\theta}(X)
+
\E_{X \sim \probs{\dsuff}}
f_{\theta}(X)
\label{eq:SimpleLoss}
\end{equation}
in a context of the proposed PSO framework.

By following the derivation from Section \ref{sec:FuncMutSupp}, 
the inner minimization problem solved by $\inf_{f \in \FF} L_{unit}(f)$ for each $X$ is:
\begin{equation}
s^* = \arginf_{s \in \RR} 
\left[
-
\probi{\usuff}{X}
+
\probi{\dsuff}{X}
\right]
\cdot
s
.
\label{eq:CD_inner}
\end{equation}
Since it is linear in $s$, the optima $s^*$ will be either $+ \infty$ (if $\probi{\usuff}{X} > \probi{\dsuff}{X}$) or $- \infty$ (if $\probi{\usuff}{X} < \probi{\dsuff}{X}$). Using physical system perspective,
we can say that given a flexible enough surface $f_{\theta}(X)$ (e.g. typical NN) the straight forward optimization via \emph{unit} loss in Eq.~(\ref{eq:SimpleLoss}) will diverge since forces $F_{\theta}^{\usuff}(X) = \probi{\usuff}{X}$ and $F_{\theta}^{\dsuff}(X) = \probi{\dsuff}{X}$ are actually independent of $\theta$ and cannot adapt to each other. That is, \bp $F_{\theta}^{\usuff}(X) = F_{\theta}^{\dsuff}(X)$ can not be achieved by the \emph{unit} loss. Thus, the model is pushed to $\pm \infty$ at various input points,
up to the surface flexibility. During such optimization, training will eventually fail due to numerical instability that involves too large/small numbers. Furthermore, this point can be easily verified in practice by training NN with loss in Eq.~(\ref{eq:SimpleLoss}).

\paragraph{CD}

One way to enforce the convergence of \emph{unit} loss is by adapting/changing density $\probs{\dsuff}$ towards $\probs{\usuff}$ along the optimization. Indeed, this is the main idea behind the CD method presented in \citep{Hinton02nc} and further improved in \citep{Ngiam11icml} and \citep{Liu17arxiv}. In CD, the \down density $\probi{\dsuff}{X}$ in Eq.~(\ref{eq:SimpleLoss}) represents the
current model distribution $\hat{\PP}_{\theta}(X) \triangleq \exp[f_{\theta}(X)]/ \int \exp[f_{\theta}(X')] dX'$, $\probs{\dsuff} \equiv \hat{\PP}_{\theta}$.
At each iteration, $\{ X^{\dsuff}_{i} \}_{i = 1}^{N^{\dsuff}}$ are sampled from $\hat{\PP}_{\theta}(X)$ by Gibbs sampling \citep{Hinton02nc},
Monte Carlo with Langevin dynamics \citep{Hyvarinen07tnn}, Hybrid Monte Carlo sampling \citep{Ngiam11icml}, or Stein Variational Gradient Descent (SVGD) \citep{Liu17arxiv,Liu16nips}. Thus, in CD algorithm forces $F_{\theta}^{\usuff}(X) = \probi{\usuff}{X}$ and $F_{\theta}^{\dsuff}(X) = \hat{\PP}_{\theta}(X)$ are adapted to each other via their frequency components $\probs{\usuff}$ and $\probs{\dsuff}$ instead of their \mgn components $M^{\usuff}
\left[
X,
f_{\theta}(X)
\right]$ and $M^{\dsuff}
\left[
X,
f_{\theta}(X)
\right]$.
The dynamics of such optimization will converge to the equilibrium only when $\probi{\usuff}{X} = \hat{\PP}_{\theta}(X)$ which will also lead to $\exp[f_{\theta}(X)] \propto \probi{\usuff}{X}$.

\paragraph{WGAN}

We additionally consider the relation between PSO concepts and Wasserstein GAN \citep{Arjovsky17arxiv} (WGAN) which has been recently proposed and is considered nowadays to be state-of-the-art. Apparently, the critic's loss in WGAN is exactly  Eq.~(\ref{eq:SimpleLoss}). 
It pushes the surface $f_{\theta}(X)$ \up at points sampled from the real data distribution $\probs{\usuff}$, and pushes \down at points sampled from the generator density $\probs{\dsuff}_\phi$, which is an implicit distribution of fake samples returned by a generator from the latent space, with $\phi$ being a generator parametrization. 

The critic's loss of WGAN was chosen as proxy to force critic's output to approximate Earth Mover (Wasserstein) distance between $\probs{\usuff}$ and $\probs{\dsuff}_\phi$.
Specifically, the \emph{unit} loss is a dual form of Wasserstein distance under the constraint that $f_{\theta}(X)$ is 1-Lipschitz continuous.
Intuitively, the critic network will return high values for samples coming from $\probs{\usuff}$, and low values for samples coming from $\probs{\dsuff}_\phi$, thus it learns to deduce if its input is sampled from $\probs{\usuff}$ or from $\probs{\dsuff}_\phi$. Once critic's optimization stage ends, the generator of WGAN optimizes its weights $\phi$ in order to increase $f_{\theta}(X)$'s output for samples coming from $\probs{\dsuff}_\phi$ via the loss
\begin{equation}
L_{WGAN}^{G}(\phi) = 
-
\E_{X \sim \probs{\dsuff}_\phi}
f_{\theta}(X)
.
\label{eq:WGANGenerLoss}
\end{equation}

The described above "infinity" divergence of \emph{unit} loss and 1-Lipschitz constraint may explain why the authors needed to clip NN weights to stabilize the approach's learning. In \citep{Arjovsky17arxiv} after each iteration the NN weights are constrained to be between $[-c,c]$ for some constant $c$. Likely, such handling reduces the flexibility of a surface $f_{\theta}(X)$, thus preventing it from getting too high/low output values.
Such conclusion about the reduced flexibility is also supported by \citep{Gulrajani17nips}.

Further, in \citep{Gulrajani17nips} authors prove that 1-Lipschitz constraint of WGAN implies that the surface has gradient (w.r.t. $X$) with norm at most 1 everywhere,
$\norm{\frac{\partial f_{\theta}(X)}{\partial X}|_{\theta = \theta^*}} = 1$. Instead of weight clipping, they proposed to combine the \emph{unit} loss with a $X$-gradient penalty term that forces this gradient norm to be close to 1. The effect of such regularization can be explained as follows. Considering the PSO principles, the optimal surface for the \emph{unit} loss has areas of $+ \infty$ and $- \infty$, thus requiring sharp slopes between these areas. The gradient penalty term constrains these slopes to be around 1, hence it prevents the surface from getting too high/low, solving in this way the "infinity" oscillations. Overall, by using weight clipping and other regularization techniques like gradient penalty \citep{Gulrajani17nips}, WGAN is in general highly successful in data generation task. Thus, we can see that basically unstable PSO instance with unit \emph{magnitudes} can be stabilized by a surface flexibility restriction via appropriate regularization terms within the loss.

\paragraph{MMD}

Finally, MMD algorithm \citep{Gretton07nips} exploits the \emph{unit} loss in Eq.~(\ref{eq:SimpleLoss}) to test if two separate datasets are generated from the same distribution. Authors express this loss over RKHS function with a bounded RKHS norm, thus implicitly constraining the model smoothness and eliminating the \infh problem.

\begin{remark}
	As was observed empirically on 20D data, even the prolonged GD optimization via the above unit loss in Eq.~(\ref{eq:SimpleLoss}) leaves the randomly initialized NN surface $f_{\theta}(X)$ almost unchanged for the case when $\probs{\usuff} \equiv \probs{\dsuff}$. This is due to the implicit force balance produced by the identical densities. In contrast, when densities are different the optimization diverges very fast, after only a few thousands of iterations. Also, the optimization gradient during these iterations is typically smaller for the same density scenario than for the different densities. Similarly to MMD method, such behavior can be exploited for example to test if samples from two datasets have the same density or not, by performing the optimization and seeing if it diverges.
\end{remark}

Overall, all of the above PSO instances with unit \ms, except for CD, handle the instabilities of \emph{unit} loss by restricting the flexibility of the model $f_{\theta}(X)$. Thus, a typical strategy is to enforce $K$-Lipschitz constraint. Yet, in context of DL it is still unclear if and how it is possible to enforce a model to be exact $K$-Lipschitz, even though there are several techniques recently proposed for this goal \citep{Gulrajani17nips,Petzka17arxiv,Miyato18arxiv,Zhou18arxiv}.

\subsection{Mutual Information Estimation}
\label{sec:MI_PSO}

Mutual information (MI) between two random multi-variable distributions represents correlation between their samples, and is highly useful in the Machine Learning domain \citep{Belghazi18arxiv}. Here we shortly describe possible techniques to learn MI from data, based on PSO principles.

Consider two random variables $X \in \RR^{n_x}$ and $Y \in \RR^{n_y}$ with marginal densities $\PP_X$ and $\PP_Y$. Additionally, denote by $\PP_{XY}$ their joint distribution. The MI between $X$ and $Y$ is defined as:
\begin{equation}
I(X, Y)
=
\int
\int
\PP_{XY}(X,Y)
\cdot
V(X,Y)
dX dY
,
\quad
V(X,Y)
\triangleq
\log
\frac{\PP_{XY}(X,Y)}{\PP_{X}(X) \cdot \PP_{Y}(Y)}
.
\label{eq:MIexpr}
\end{equation}
If log-ratio $V(X,Y)$ is known/learned in some way, and if we have samples $\{ X^i, Y^i \}_{i = 1}^{N}$ from joint density $\PP_{XY}$, we can approximate MI via a sample approximation:
\begin{equation}
I(X, Y)
\approx
\frac{1}{N}
\sum_{i = 1}^{N}
V(X_i,Y_i)
.
\label{eq:MIapprox}
\end{equation}

Further, $V(X,Y)$ can be easily learned by one of PSO instances in Tables \ref{tbl:PSOInstances1}-\ref{tbl:PSOInstances5} for logarithm density-ratio estimation as follows. Consider a model $f_{\theta}(X, Y): \RR^{n} \rightarrow \RR$, with $n = n_{x} + n_{y}$. Additionally, we will use $\PP_{XY}(X,Y)$ as \up density in PSO framework, and $\PP_{X}(X) \cdot \PP_{Y}(Y)$ - as \down density. To obtain sample from \up density, we can pick random pair from available dataset $\{ X^i, Y^i \}_{i = 1}^{N}$, similarly to conditional density estimation in Section \ref{sec:CondDeepPDF}. Further, samples from \down density can be acquired by picking $X^i$ and $Y^i$ from dataset independently.

Considering the above sampling procedure, we can apply PSO to push $f_{\theta}(X, Y)$ via \up and \down forces, using the corresponding \mfs. 
For any particular \pair the associated PSO \bp will be:
\begin{equation}
\frac{M^{\dsuff}\left[
	X, Y,
	f_{\theta}(X, Y)
	\right]}{M^{\usuff}\left[
	X, Y,
	f_{\theta}(X, Y)
	\right]}
=
\frac{\PP_{XY}(X,Y)}{\PP_{X}(X) \cdot \PP_{Y}(Y)}
.
\label{eq:PSOBalll_Mut_Inf}
\end{equation}
Hence, to produce the convergence $f_{\theta}(X, Y) = V(X, Y)$, the appropriate choice of \ms must satisfy $\frac{M^{\dsuff}\left[
	X, Y,
	s
	\right]}{M^{\usuff}\left[
	X, Y,
	s
	\right]} = \exp s$ - to infer log-ratio between \up and \down densities the \mgn ratio $R$ must be equal to $\exp s$, according to Table \ref{tbl:ClassicTransforms}.

For example, $M^{\usuff}\left[
X, Y,
f_{\theta}(X, Y)
\right] = \frac{1}{\exp [f_{\theta}(X, Y)] + 1}$ and $M^{\dsuff}\left[
X, Y,
f_{\theta}(X, Y)
\right] = \frac{1}{\exp [- f_{\theta}(X, Y)] + 1}$ can be used (e.g. a variant of the logistic loss in Table \ref{tbl:PSOInstances4}), yet many other alternatives can also be considered.
Recently, similar ideas were also presented in \citep{Belghazi18arxiv}.

\section{NN Architecture}
\label{sec:NNArch}

In this section we describe various design choices when constructing NN $f_{\theta}(X)$, and their impact on density estimation task. The discussed below are various connectivity architectures, activation functions and pre-conditioning techniques that helped us to improve overall learning accuracy. Likewise, where possible we relate the design choice to corresponding properties acquired by a model kernel $g_{\theta}(X, X')$.

Algorithms DeepPDF (see Section \ref{sec:DeepPDF}) and log-density estimators in Section \ref{sec:DeepLogPDF} typically produce highly accurate density approximations in low dimensional cases. For example, in \citep{Kopitkov18arxiv} we showed that DeepPDF produces a better accuracy than KDE methods in 2D and 3D scenarios. This likely can be accounted to the flexibility of NN - its universal ability to approximate any function. As empirically observed in \citep{Kopitkov19arxiv_spectrum}, the implicit model kernel $g_{\theta}(X', X)$ adapts to better represent any learned target function. Further, its bandwidth, discussed in Section \ref{sec:ExprrKernelBnd}, is typically different in various areas of the considered input space. This allows to prevent \overfit in areas with small amount of training points, and to reduce \underfit in areas where the amount of training data is huge.
In contrast, KDE methods are typically limited to a specific choice of a kernel and a bandwidth (yet variable-bandwidth KDE methods exist, see \citet{Terrell92aos}) that is applied to estimate the entire pdf surface with its many various details.

Yet, we also observed a considerable \underfit problem of the above PSO instances that grows with larger data dimension and with higher frequency/variability contained within the target function. Particularly, even in case where a high-dimensional training dataset is huge, for a typical fully-connected (FC) NN architecture the produced estimation is far away from the real data density, and often contains mode-collapses and other inference inconsistencies. We argue
that it is caused by
too wide bandwidth of $g_{\theta}(X', X)$ in FC architecture, which leads to a growing estimation bias. Such conclusion is supported by Theorem \ref{thrm:PSO_diff_diff} which stated that the bandwidth of $g_{\theta}(X, X')$ defines the flexibility of $f_{\theta}(X)$.

As was observed, in FC architecture $g_{\theta}$'s bandwidth is growing considerably with the higher data dimension, thus producing more side interference between the different training points and decreasing the overall elasticity of the network. In its turn, this limits the accuracy produced by PSO. Below we propose a new NN architecture that mitigates the bandwidth problem and increases a flexibility of the surface. Further, in Section \ref{sec:ColumnsEstNNAME}
we show that such architecture extremely improves the estimation accuracy.

\begin{remark}
Note that in context of generative adversarial networks (GANs), whose critics are also instances of PSO (see Tables \ref{tbl:PSOInstances1}-\ref{tbl:PSOInstances5}), the convergence problems (e.g. mode-collapse and non-convergence) were also reported. Typically, these problems are blamed on Nash equilibrium between critic and generator networks which is hard to optimize. Yet, in our work we see that such problems exist even without a two-player optimization. That is, even when only a  specific PSO instance (the critic in GAN's context) is trained separately, it is typically underfitting and has mode-collapses within the converged surface $f_{\theta}(X)$ where several separate "hills" from target $T\left[ X, \frac{\probi{\usuff}{X}}{\probi{\dsuff}{X}} \right]$ are represented as one hill inside $f_{\theta}(X)$. Below we present several techniques that allowed us to reduce these convergence problems.
\end{remark}

\subsection{Block-Diagonal Layers}
\label{sec:BDLayers}

\begin{figure}
	\centering
	
	\begin{tabular}{cccc}
		
		\subfloat[\label{fig:NNLayers-a}]{\includegraphics[width=0.44\textwidth]{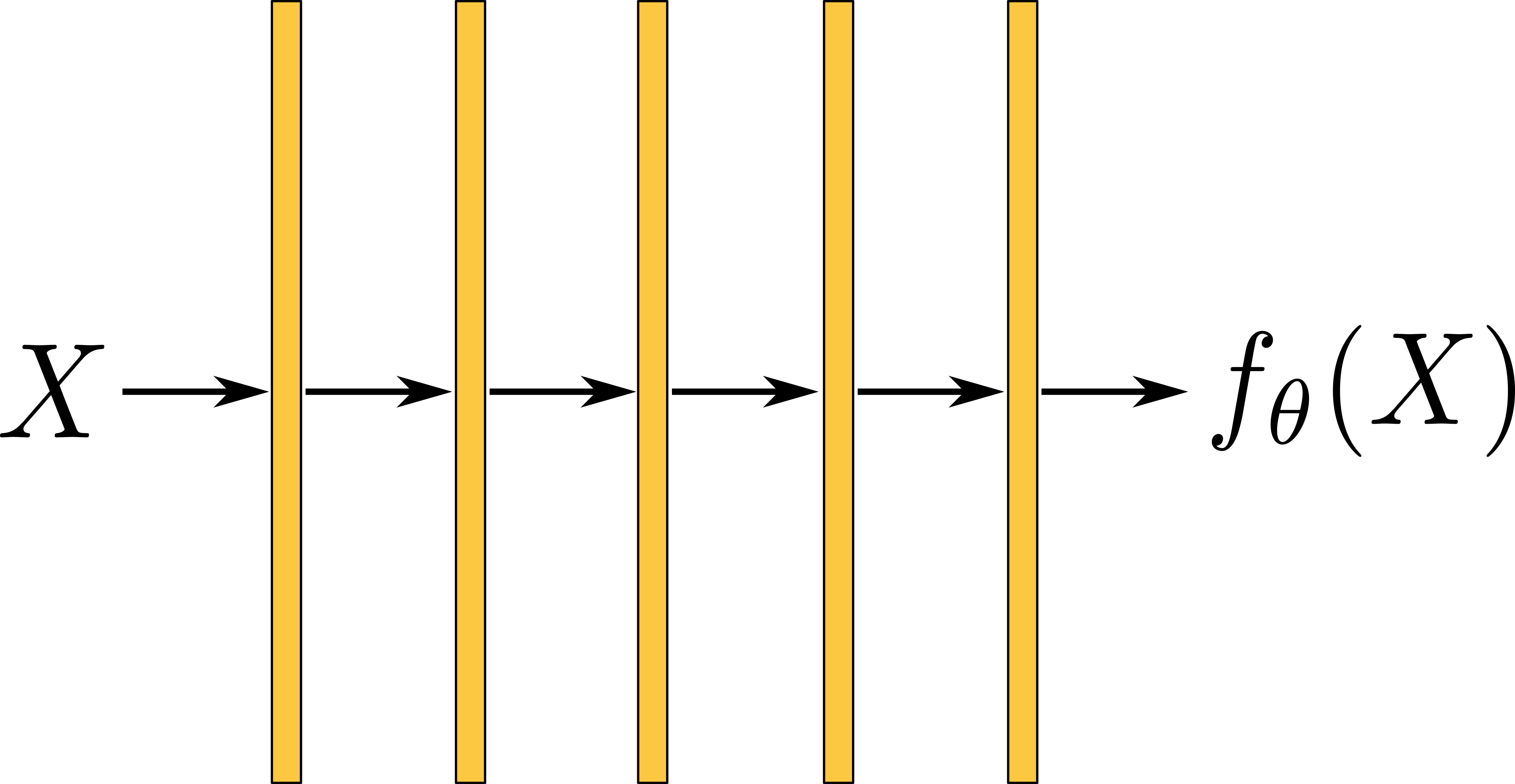}}
		&
		
		\subfloat[\label{fig:NNLayers-b}]{\includegraphics[width=0.45\textwidth]{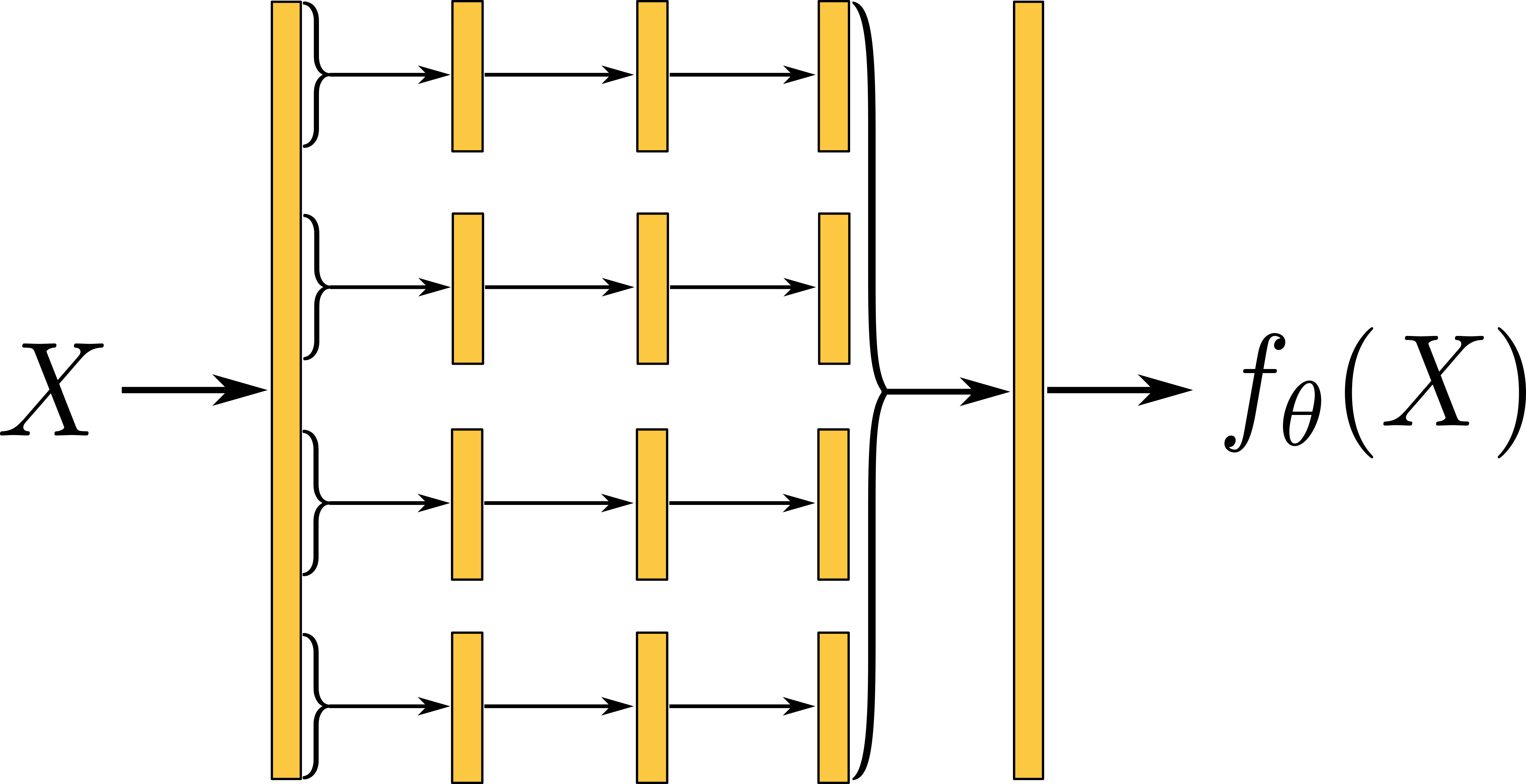}}
		
	\end{tabular}
	
	\protect
	\caption[FC and BD architectures.]{(a) Typical NN architecture used in \citep{Kopitkov18arxiv}. Yellow blocks are FC layers with non-linearity, except for last layer which is FC layer without activation function. Entire network can be seen as single transformation channel.
	(b) Proposed NN architecture used in this paper. Block-diagonal layers can be seen as set of independent transformation channels. This independence improves network's flexibility. Output vector of a first FC layer is sliced into $N_B$ separate vectors. Each small vector is used as input to separate channel - FC sub-network. At the end outputs of all channels are concatenated and sent to the final FC layer. All FC layers in this architecture (the yellow blocks) use non-linearity (typically Relu or Leaky-Relu), except for the final FC layer.
	}
	\label{fig:NNLayers}
\end{figure}

A typical FC network (Figure \ref{fig:NNLayers-a}) can be seen as one channel transformation from point $X$ to its surface height $f_{\theta}(X)$. During the optimization, due to such NN structure almost every parameter inside $\theta$ is updated as a consequence of pushing/optimizing at any training point $X \in \RR^{n}$ within PSO loss. 
Such high sharing of the weights between various regions of $\RR^{n}$ creates a huge side-influence between them - a dot product between $\nabla_{\theta} 
f_{\theta}(X)$ and $\nabla_{\theta} 
f_{\theta}(X')$ for faraway points $X$ and $X'$ is large. This in turn increases bandwidth of $g_{\theta}$ and decreases the flexibility of $f_{\theta}$.

The above line of thought guided us to propose an alternative NN architecture where several separate transformation channels are built into one network (see Figure \ref{fig:NNLayers-b}). As we will see below, such architecture is identical to a simple FC network where each layer's weight matrix $W_i$ is block-diagonal.

Specifically, we propose to pass input $X$ through a typical FC layer with output dimension $S$, and split this output into a set of $N_B$ smaller vectors of size $S_B = S/N_B$. Further, for each $S_B$-sized vector we construct a channel: a subnetwork with $[N_L - 2]$ FC layers of the same size $S_B$. Finally, the outputs of all channels are concatenated into vector of size $S$ and this vector is sent to the final FC layer (see illustration in Figure \ref{fig:NNLayers-b}). All FC layers within this architecture use non-linearity (typically Relu or Leaky-Relu), except for the last layer.

Exactly the same computational flow will be produced if we use the usual FC network from Figure \ref{fig:NNLayers-a} with inner layers having block-diagonal weight matrices. Namely, we can build the same simple network as in Figure \ref{fig:NNLayers-a}, with $N_L$ layers overall, where $[N_L - 2]$ inner FC layers have block-diagonal weight matrices $W_i$ of size $S \times S$. Each $W_i$ in its turn will contain $N_B$ blocks at its diagonal, of $S_B \times S_B$ size each, with rest of the entries being constant zeros.

A straightforward implementation of block-diagonal (BD) layers by setting off-diagonal entries to be constant zeros can be wasteful w.r.t. memory and computation resources. Instead, we can use multi-dimensional tensors for a more efficient implementation as follows. Consider output of the first FC layer as a tensor $\bar{v}$ with dimensions $[B, S]$, where $B$ is a batch dimension and $S$ is an output dimension of the layer. We can reshape $\bar{v}$ to have dimensions $[B, N_B, S_B]$, where the last dimension of $\bar{v}$ will contain small vectors $\bar{u}_j$ of size $S_B$ each, i.e.~inputs for independent channels. Further, each inner BD layer can be parametrized by a weight matrix $W$ with dimensions $[N_B, S_B, S_B]$ and bias vector $b$ with dimensions $[N_B, S_B]$. The multiplication between $\bar{v}$ and $W$, denoted as $V$, has to be done for each $\bar{u}_j$ with an appropriate slice of weight matrix, $W[j,:,:]$. Moreover, it should be done for every instance of the batch. This can be done via the following Einstein summation convention:
\begin{equation}
V[i,j,k]
=
\sum_{m = 1}^{S_B}
W[j,k,m]
\cdot
\bar{v}[i,j,m]
,
\label{eq:BDEinSum}
\end{equation}
which produces tensor $V$ with size $[B, N_B, S_B]$. Further, bias can be added as:
\begin{equation}
U[i,:,:]
=
V[i,:,:]
+
b
,
\label{eq:BDBias}
\end{equation}
where afterwards the tensor $U$ is transformed by point-wise activation function $\sigma(\cdot)$, finally producing the output of BD layer $\hat{U} = \sigma(U)$ of size $[B, N_B, S_B]$.

We construct $[N_L - 2]$ such BD layers that represent $N_B$ independent channels. Further, the output of the last BD layer is reshaped back to have dimensions $[B, S]$, and is sent to the final ordinary FC layer that returns a scalar.

\begin{remark}
The Einstein summation operation is typically offered by modern DL frameworks, thus implementing the above BD layers is convenient and easy. Yet, their runtime is slower relative to FC layers.
We hope that in future versions of DL frameworks such BD layers would be implemented efficiently on GPU level.
Also, our code for this layer can be found in open source library \url{\psorepo}.
\end{remark}

\subsubsection{Flexibility of BD vs FC}
\label{sec:BDLayersCompare}

The above BD architecture allowed us to tremendously improve accuracy of density estimation for high-dimensional data. This was achieved due to the enhanced flexibility of BD architecture vs FC, as we empirically demonstrate below.

\begin{figure}[tb] 
	\centering
	
	\begin{tabular}{ccc}
		
		\subfloat[\label{fig:NNCovariance-a}]{\includegraphics[width=0.31\textwidth]{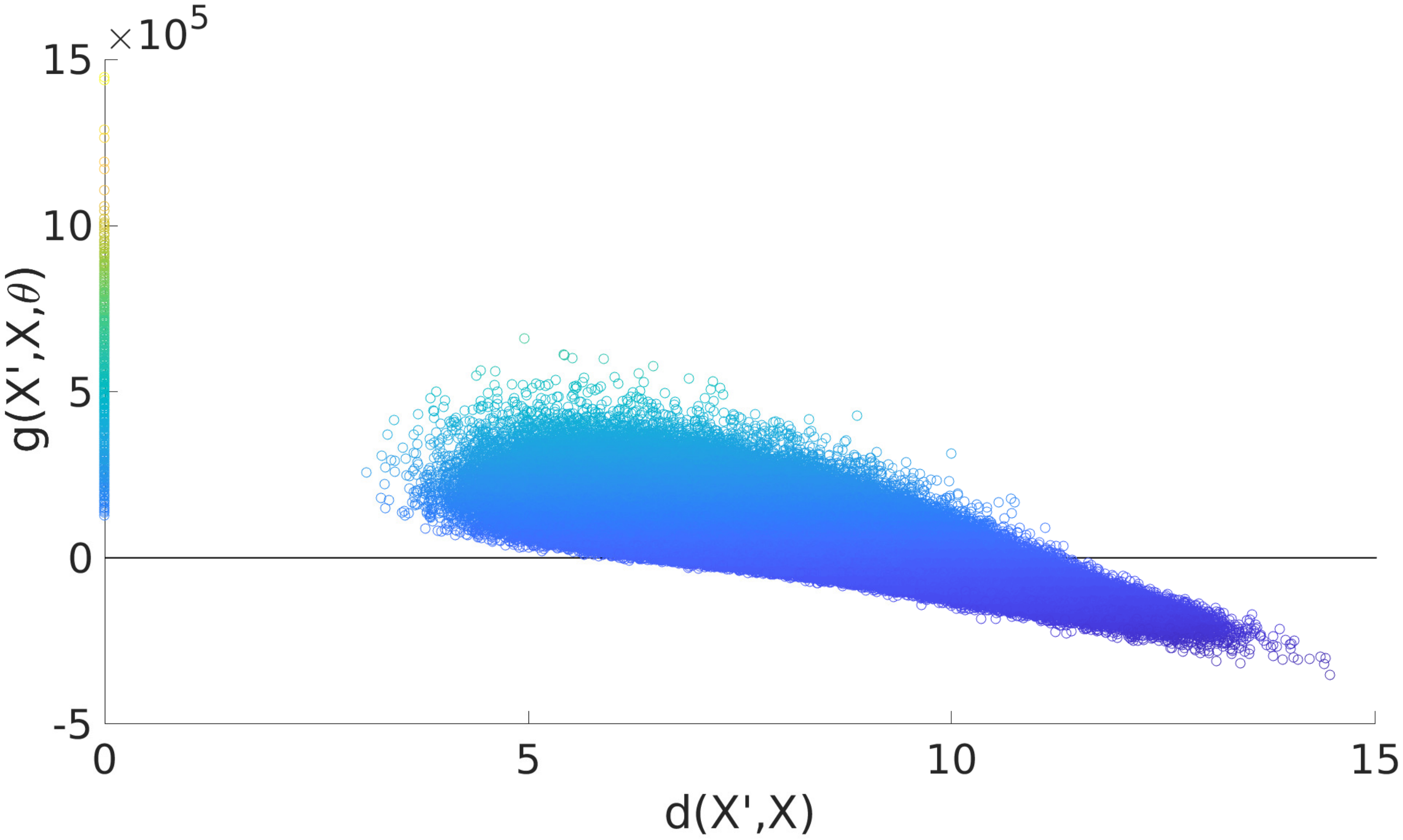}}
		&
		
		\subfloat[\label{fig:NNCovariance-b}]{\includegraphics[width=0.31\textwidth]{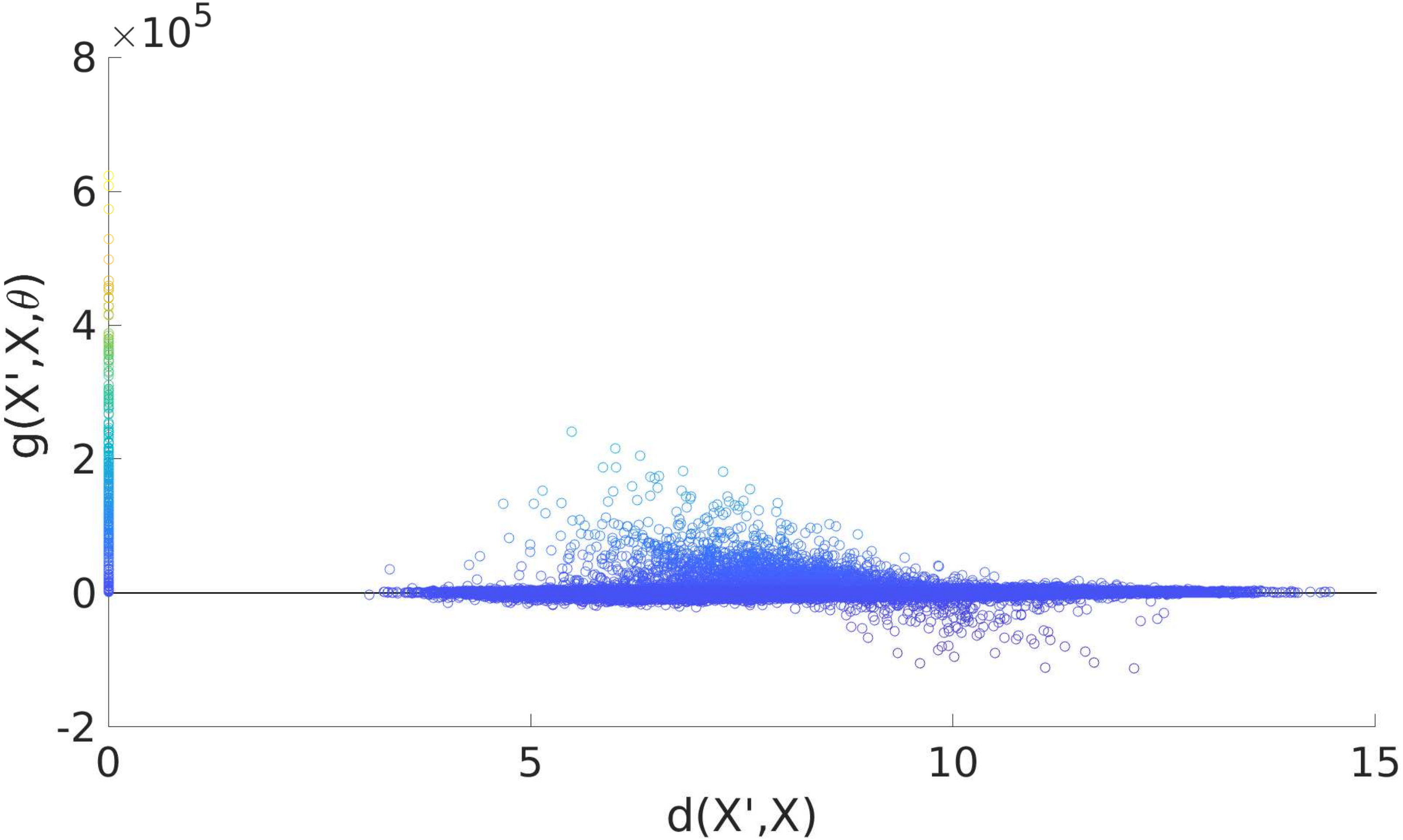}}
		
		&
		
		\subfloat[\label{fig:NNCovariance-c}]{\includegraphics[width=0.31\textwidth]{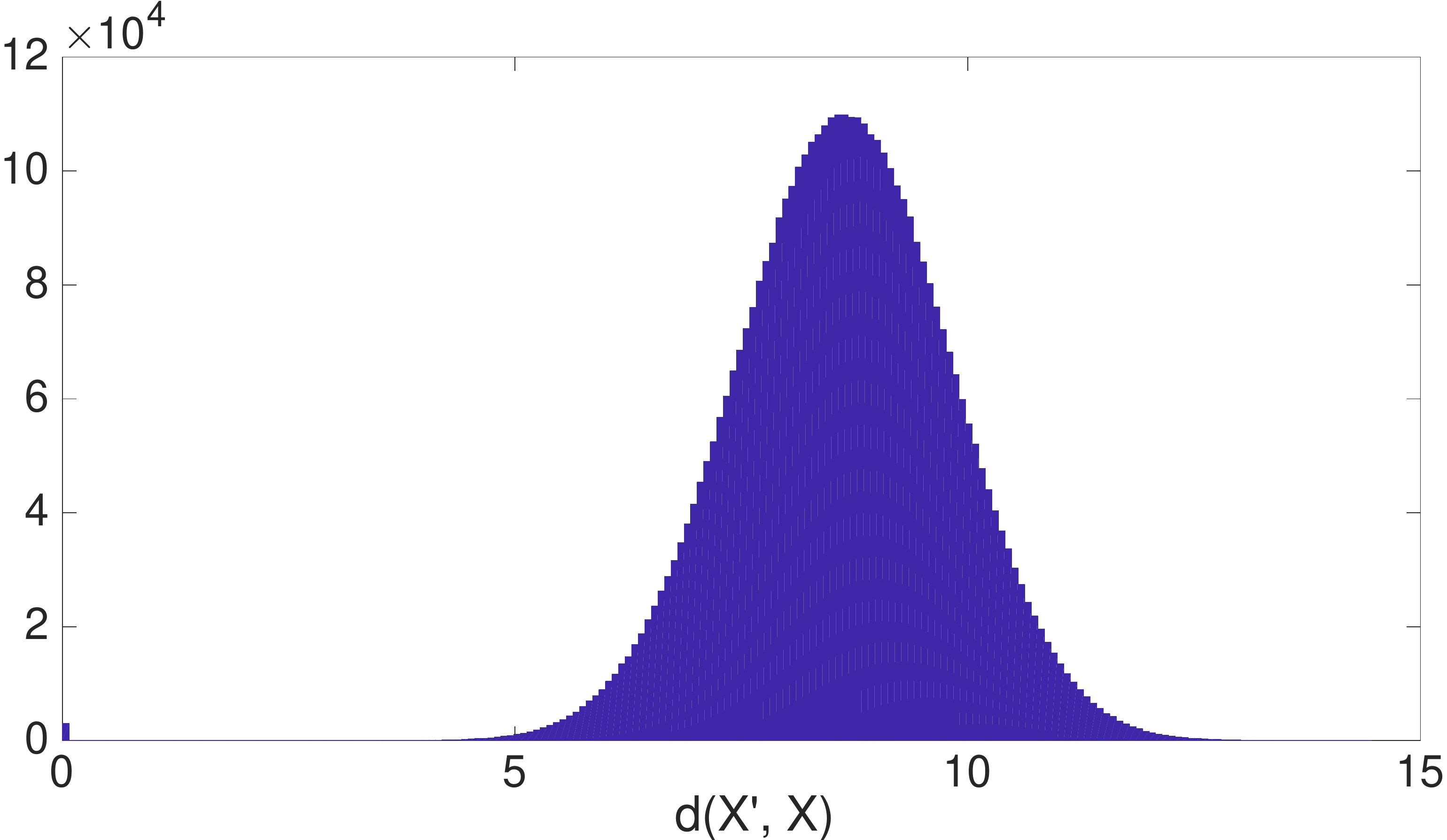}}
		
	\end{tabular}
	
	\newcommand{\width}[0] {0.48}
	\newcommand{\height}[0] {0.18}

	\begin{tabular}{cc}
		
		\subfloat[\label{fig:NNCovariance-d}]{\includegraphics[height=\height\textheight,width=\width\textwidth]{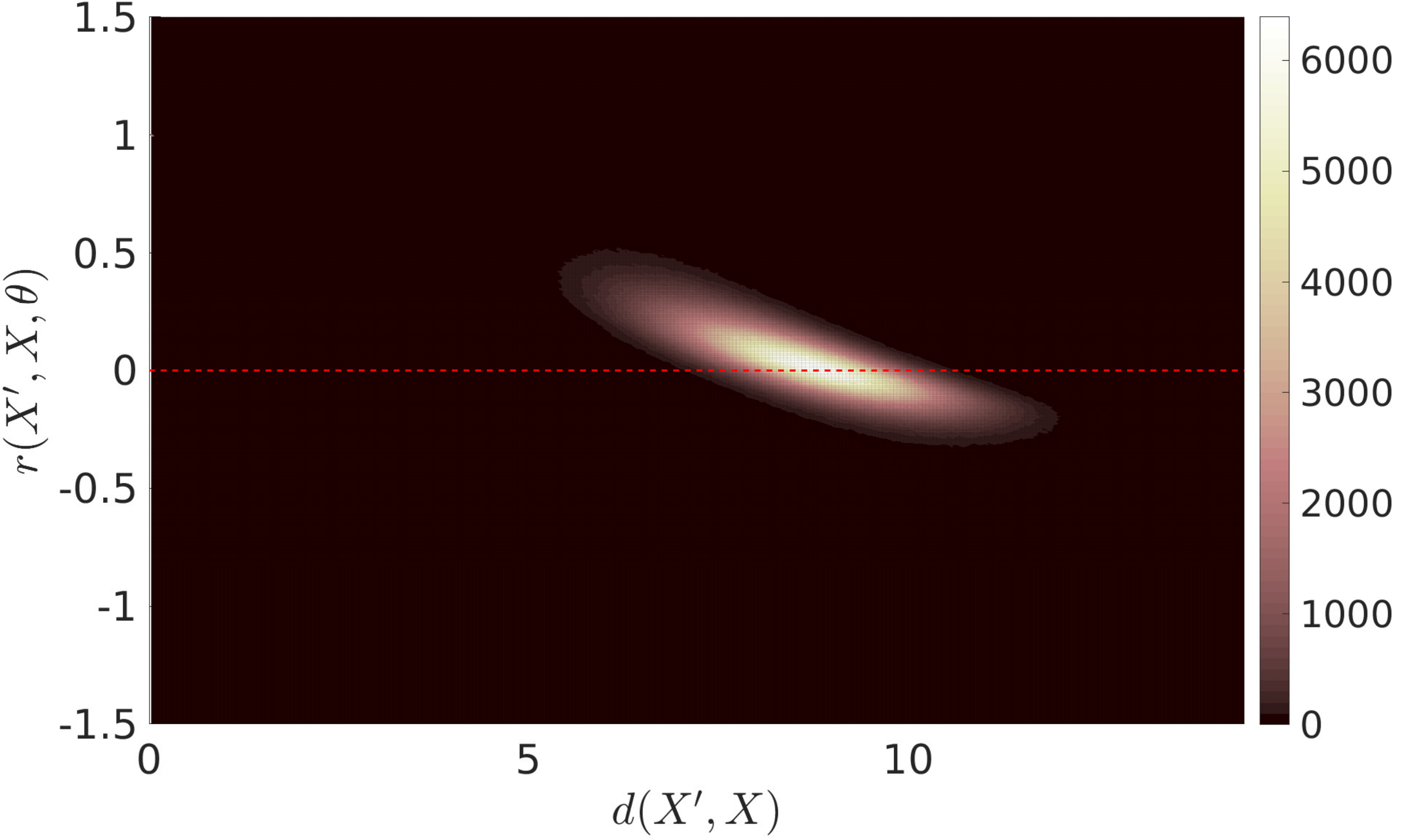}}
		&
		
		\subfloat[\label{fig:NNCovariance-e}]{\includegraphics[height=\height\textheight,width=\width\textwidth]{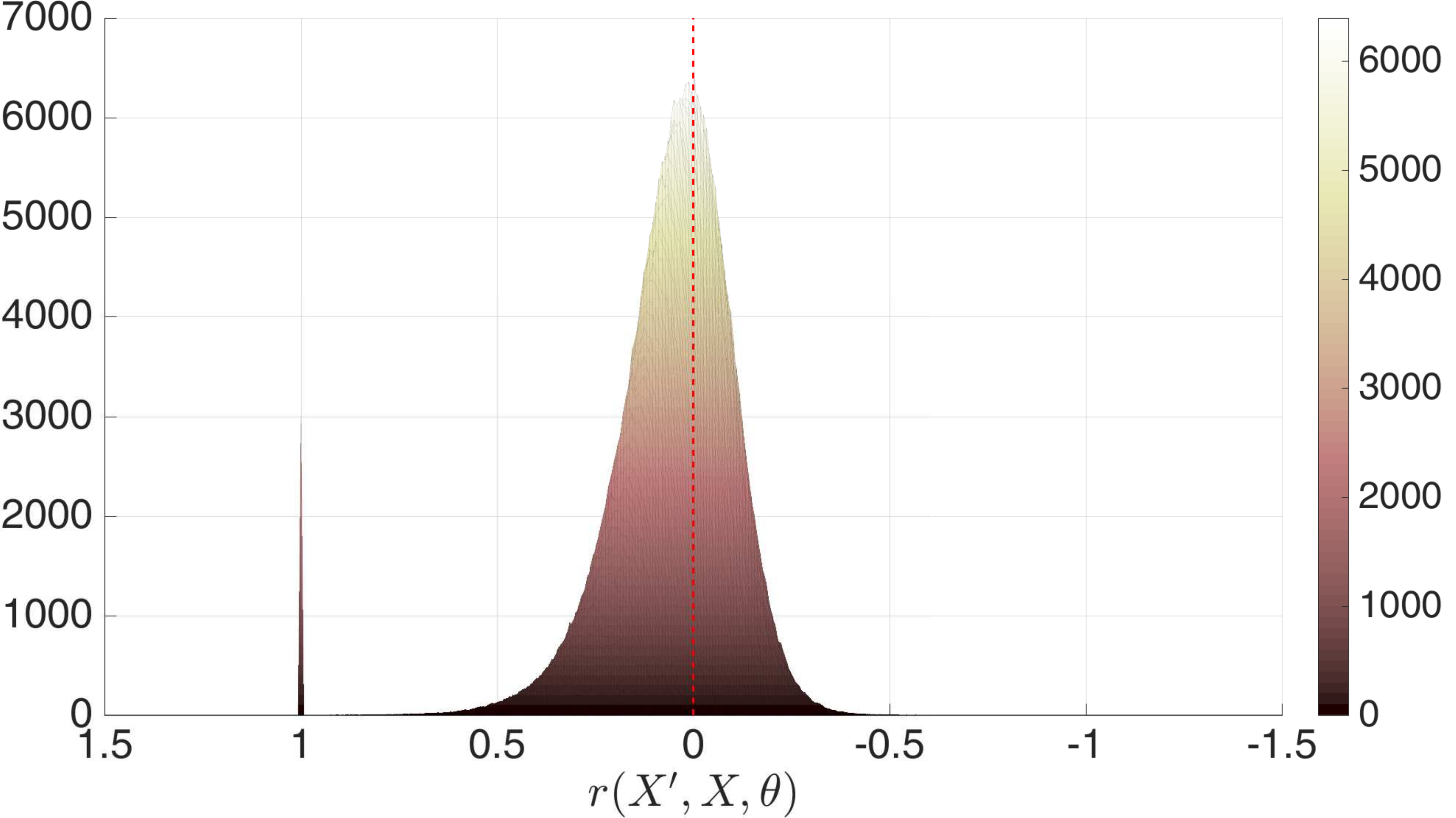}}
		
		\\
		
		\subfloat[\label{fig:NNCovariance-f}]{\includegraphics[height=\height\textheight,width=\width\textwidth]{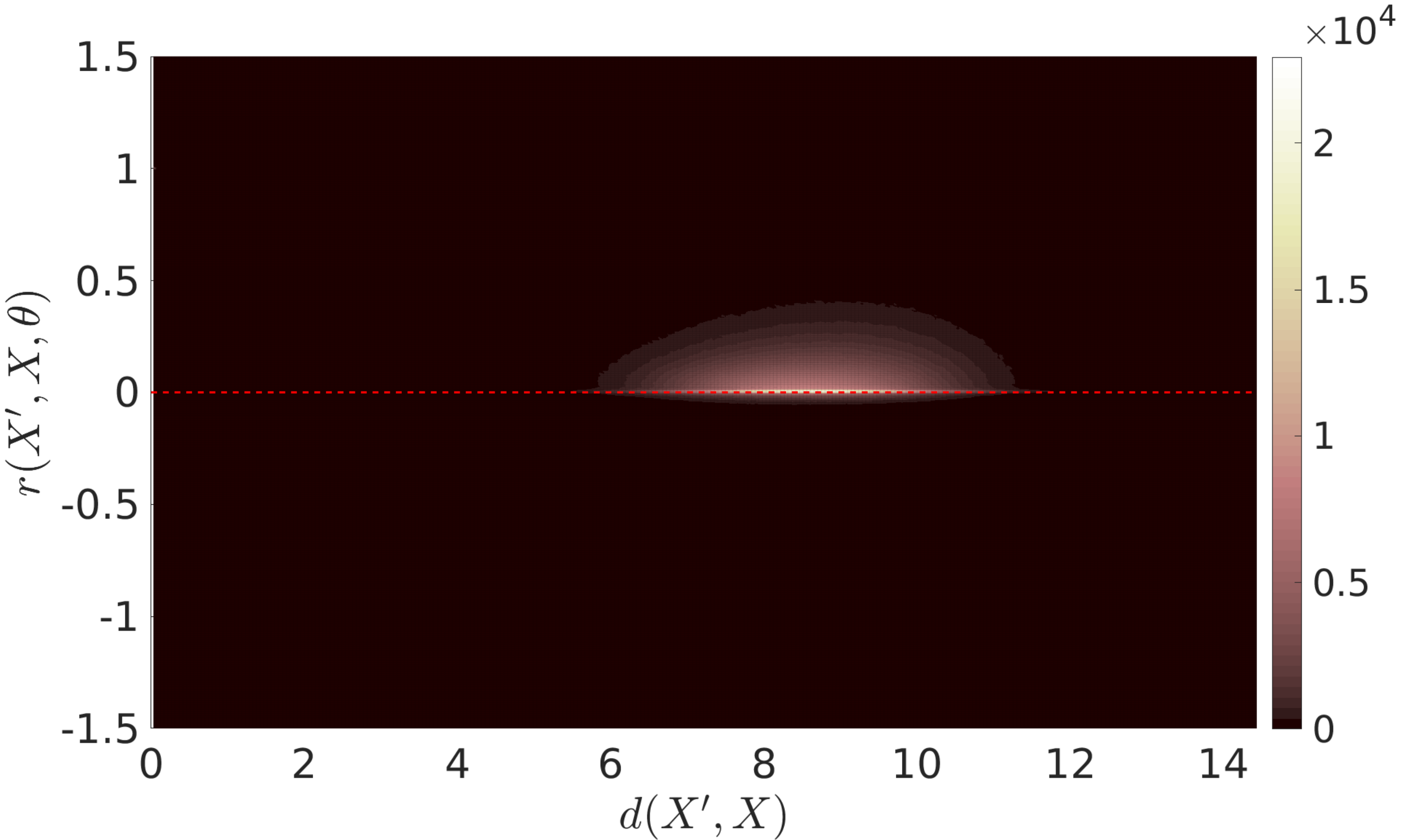}}
		&
		
		\subfloat[\label{fig:NNCovariance-g}]{\includegraphics[height=\height\textheight,width=\width\textwidth]{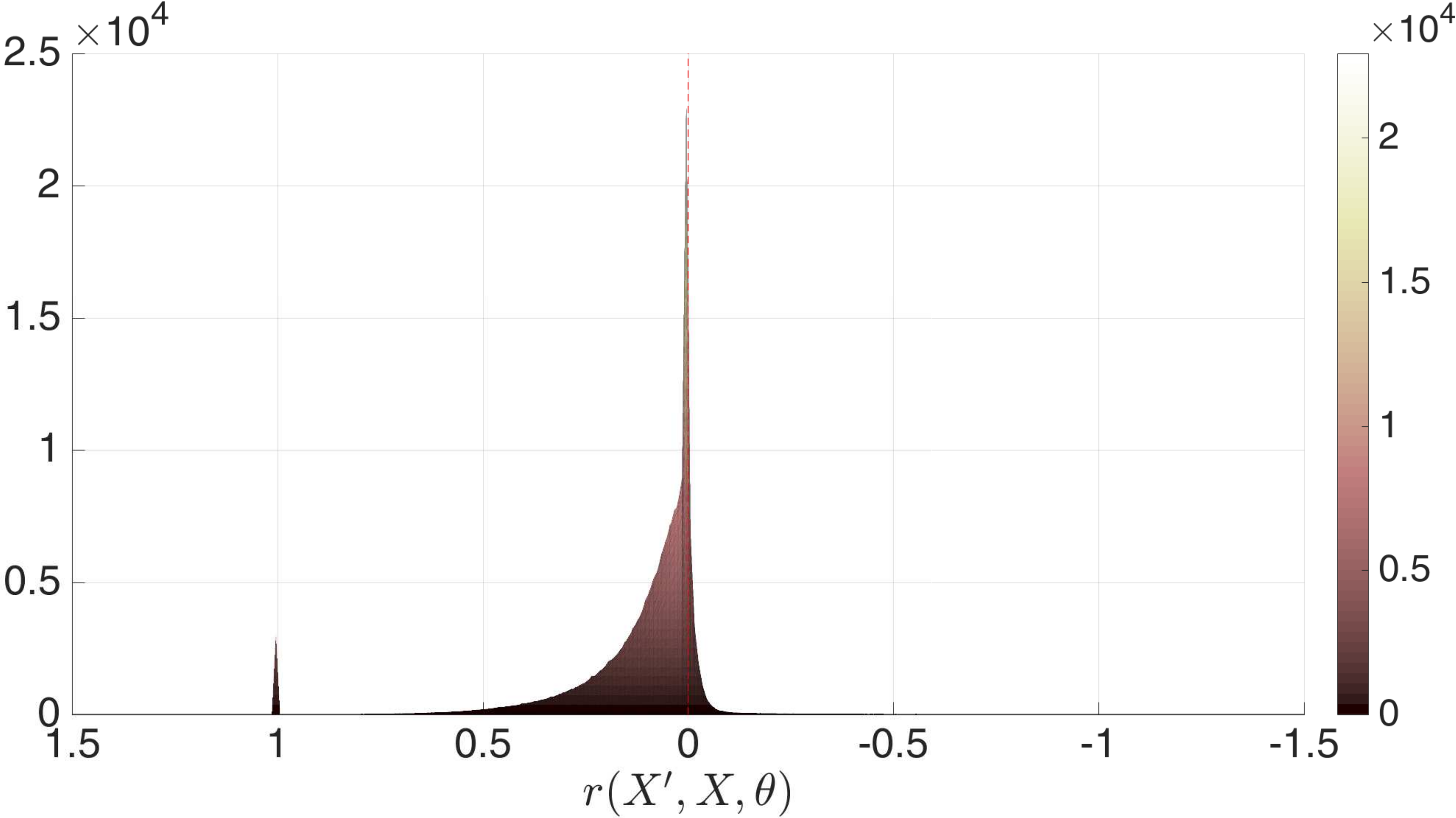}}
		
	\end{tabular}
	
	\protect
	\caption[Bandwidth of $g_{\theta}(X', X)$ and the surface flexibility of FC and BD models.]{Bandwidth of $g_{\theta}(X', X)$ and the surface flexibility of FC and BD models. We infer 20D \emph{Columns} distribution (see Section \ref{sec:ColumnsEst}) by using PSO-LDE with $\alpha = \frac{1}{4}$ (see Table \ref{tbl:PSOInstances1} and Section \ref{sec:DeepLogPDF}). Two networks were trained, FC and BD. The applied FC architecture contains 4 FC layers of size 1024. The applied BD architecture has 6 layers, number of blocks $N_B = 50$ and block size $S_B = 64$. Values of \emph{gradient similarity} $g_{\theta}(X', X)$ and values of Euclidean distance $d(X', X)$ are plotted for (a) FC network and (b) BD network. (c) Histogram of $d(X', X)$ calculated between all sample pairs from dataset $D$.
	Further, a histogram of obtained $\{ r_{\theta}(X_i, X_j) \}$ and $\{ d(X_i, X_j) \}$ is plotted for (d) FC model and (f) BD model. Side views of these histograms are depicted in (e) and (g). See more details in the main text.
	}
	\label{fig:NNCovariance}
\end{figure}

To this end, we analyze the bandwidth of $g_{\theta}(X, X')$ for each of the architectures as follows. We perform a typical density estimation task via PSO-LDE method proposed in Section \ref{sec:DeepLogPDF}. After training a model we sample $D = \{ X_i \}_{i = 1}^{3000}$ testing points from the target density $\probs{\usuff}$ and calculate their gradients $\nabla_{\theta} f_{\theta}(X_i)$. Further, we calculate Euclidean distance $d(X_i, X_j) = \norm{X_i - X_j}$ and \emph{gradient similarity} $g_{\theta}(X_i, X_j)$ between every two points within $D$, producing $\frac{3000 \cdot 3001}{2}$ pairs of distance and similarity values (we consider only unique pairs here). These values are plotted in Figures \ref{fig:NNCovariance-a}-\ref{fig:NNCovariance-b}. As can be seen, $g_{\theta}(X', X)$ values ($y$ axis) of FC network are much higher than these values in BD network. Likewise, there is strong correlation between values of \emph{gradient similarity} and Euclidean distance. In FC case, for $d(X', X) > 0$ values of $g_{\theta}(X', X)$ are far away from being zeros, thus implying strong side-influence of optimization pushes on surface $f_{\theta}(X)$ even between far away points. In contrast, for BD case we can see that $g_{\theta}(X', X)$ is centered around zero for $d(X', X) > 0$, hence side-influence here is less significant.
Furthermore, we stress that similar trends were achieved in all our experiments, for various densities $\probs{\usuff}$ and $\probs{\dsuff}$.

\begin{remark}
The gap in Figures \ref{fig:NNCovariance-a}-\ref{fig:NNCovariance-b} between points with $d(X', X) = 0$ and rest of the samples is explained as follows. At $d(X', X) = 0$ all point pairs are of a form $(X_i, X_i)$, with overall $3000$ such pairs. Rest of the samples are $\{(X_i, X_j) | i \neq j \}$. Furthermore, the histogram of $d(X', X)$ between points in $D$ is illustrated in Figure \ref{fig:NNCovariance-c}. As can be observed, $d(X', X)$ is distributed with Gaussian-like density centered around $8.6$. Hence, the gap between the points in Figures \ref{fig:NNCovariance-a}-\ref{fig:NNCovariance-b} can be explained by a very low probability of two sampled points to be close to each other when the considered space volume (here the subset of $\RR^{20}$) is huge.
\end{remark}

Further, for each sample pair in $D$ we also calculate the \emph{relative} model kernel $r_{\theta}(X_i, X_j)$ defined in Eq.~(\ref{eq:RelKernel}).
When $r_{\theta}$ is greater than 1, it implies that point $X_j$ has stronger impact over $f_{\theta}(X_i)$ than the point $X_i$ itself, and vice versa. Hence, for each point $X_i$ the $r_{\theta}(X_i, \cdot)$ can be interpreted as a \emph{relative} side-influence from other areas over $f_{\theta}(X_i)$, scaled w.r.t. the self-influence of $X_i$. Such normalization allows us to see the actual side-influence impact between two different points, since the value of $g_{\theta}(X, X')$ by itself is meaningless and only achieves significance when compared to the self-similarity $g_{\theta}(X, X)$. Moreover, unlike $g_{\theta}(X, X)$, $r_{\theta}(X, X')$ of different models and NN architectures is on the same scale, allowing to compare the side-influence level between different models.

For $9 \cdot 10^6$ calculated pairs of a \emph{relative} side-influence $r_{\theta}(X_i, X_j)$ and a Euclidean distance $d(X_i, X_j)$ we constructed a histogram in Figures \ref{fig:NNCovariance-d} and \ref{fig:NNCovariance-f} for FC and BD networks respectively. 
Here, we can see the real difference between side-similarities of two models. Within FC network we have a strong \emph{relative} side-influence even between far away regions. This side-influence interferes with the proper PSO optimization by introducing a bias, as was explained in Section \ref{sec:ExprrKernelBnd}. 
In contrast, within BD model the \emph{relative} side-influence between far away regions stays very close to zero, implying that the surface height $f_{\theta}(X)$ at point $X$ is only pushed by training points that are relatively close to $X$. Furthermore, this increases the flexibility of the surface, since with a less side-influence the surface is less constrained and can be pushed at each specific neighborhood more freely.

Hence, we see empirically that the kernel bandwidth of BD NN is smaller than the bandwidth of FC NN, which implies that BD surface is much more flexible than FC. Such flexibility also improves the overall accuracy performance achieved by BD networks (see Section \ref{sec:ColumnsEstNNAME}). The exact mechanism responsible for such difference in the \gs is currently unknown and we shall investigate it in future work.

\subsubsection{Relation between BD and FC - Additional Aspects}
\label{sec:BDLayersModelRel}

The multi-dimensional tensor implementation of BD layers in Eq.~(\ref{eq:BDEinSum}-\ref{eq:BDBias}) allows to significantly reduce size of $\theta$. For example, BD network applied in Figure \ref{fig:NNCovariance} with 6 layers has less than $10^6$ weights ($\left|\theta\right| = 902401$), while the straightforward implementation would require above $10^7$ weights - 10 times more; the same size that appropriate FC network with 6 layers of size 3200 would take.
Further, the size of FC network used in Figure \ref{fig:NNCovariance} is $\left|\theta\right| = 2121729$.
Yet, surprisingly the more compact BD network produces a narrower model kernel (and higher approximation accuracy as will be shown in Section \ref{sec:ColumnsEstNNAME}) than the more memory consuming FC network.

Interestingly, BD layers are contained in the hypothesis class of FC layers, thus being instance of the latter. Yet, a typical optimization of FC architecture will not impose weight matrices to be block-diagonal, since the local minima of FC network typically has dense weight matrices. However, as already stated, an optimized network with BD structure has a significantly lower error compared to FC structure.
This suggests that local minima of FC networks has a much bigger error compared to the error of the global minima for such architecture, since the global minima should be even smaller than the one achieved by BD network. Hence, this implies that common statement about local and global errors of NN being close is not always correct.

\subsubsection{Similar Proposed Architectures}
\label{sec:BDLayersRelWork}

The BD model can be expressed as a sum of sub-models, each representing separate network channel. Such design has a high resemblance to the products of experts (PoE) \citep{Hinton99} where model is constructed as sum (or product) of smaller models. Yet, in typical PoE each expert is trained separately, while herein we represent our block-diagonal model as single computational graph that is trained as whole by the classical backpropagation method.

In addition, we argue that also other DL domains can benefit from a BD architecture, and such investigation can be an interesting future work. In fact, separating network into several independent channels is not new. The family of convolutional Inception models \citep{Szegedy15cvpr,Szegedy16cvpr,Szegedy17aaai} also applied the \emph{split-transform-merge} paradigm, where each network block was separated into a set of independent transformations (channels). These models succeeded to achieve high accuracy at the image classification problem. Further, ResNeXt convolutional model in \citep{Xie17cvpr} generalized this idea to produce NN computational blocks that contain $C$ independent identical transformations, where $C$ is a \emph{cardinality} parameter of NN. Authors showed that increasing \emph{cardinality} instead of width/number of layers can significantly improve the accuracy produced by NN. In context of BD architecture, we have seen a similar trend where increasing number of channels $N_B$ (which is parallel to $C$) allows to provide a better approximation of the target function. We demonstrate this in our experiments in Section \ref{sec:Exper}.

Further, a similar architecture was proposed also in \citep{Nesky18icann} in the context of the classification, although it was implemented in a different way. The main motivation of that work was to condense a network size to improve the computational complexity of a NN. Authors showed that by forcing weight matrices of FC layers to be block-diagonal a significant speedup in time can be achieved with small loss in accuracy. In contrast, in our work we see that such NN structure not only improves runtime and reduces number of weights, but also produces a higher approximation performance.

\subsection{NN Pre-Conditioning}
\label{sec:PrecondNN}

It is a common practice in Machine Learning to pre-condition a learning algorithm by, for example,  whitening and uncorrelating data or performing any other transformation that improves a condition number of the optimization. We also found that such techniques can be valuable for the density estimation task. Specifically, the main considered by this paper application of PSO framework is to learn $\log \probi{\usuff}{X}$.
In our experiments we combine two pre-conditioning methods within our NN $f_{\theta}(X)$: data normalization and NN height bias.

First, we normalize data to have zero mean and unit standard deviation for each dimension $i$ independently, via:
\begin{equation}
\hat{X}_i
=
\frac{X_i - \mu_i}{\sigma_i}
,
\label{eq:NormData}
\end{equation}
where $\mu$ and $\sigma$ are mean and standard deviation vectors calculated for all available data $\{ X^{\usuff}_{i} \}_{i = 1}^{N^{\usuff}}$ from the target density $\probs{\usuff}$.

Second, we bias an initial surface $f_{\theta}(X)$ to coincide with logarithm of the chosen auxiliary \down density $\probi{\dsuff}{X}$. We assume that the target $\log \probi{\usuff}{X}$ and $\log \probi{\dsuff}{X}$ reside on a similar height on average. Thus, to accelerate the convergence we force the initial height of surface $f_{\theta}(X)$ to be identical to $\log \probi{\dsuff}{X}$ as follows. First, as observed a typical initialization of NN produces the initial surface $f_{\theta}(X) \approx 0$ for all points $X$. Hence, in order to bias it to the initial height $\log \probi{\dsuff}{X}$, we only need to add this log-pdf function to the output of the last layer, $f^{L}_{\theta}(X)$:
\begin{equation}
f_{\theta}(X)
=
f^{L}_{\theta}(X) + \log \probi{\dsuff}{X}
.
\label{eq:BiasNN}
\end{equation}

Moreover, such NN initialization enforces the logarithm difference $\bar{d}\left[
X, f_{\theta}(X)
\right]
 \triangleq f_{\theta}(X) - \log \probi{\dsuff}{X}$ from Eq.~(\ref{eq:DDiffDefinition})
to be approximately zero for all points $X \in \RR^n$ at beginning of the optimization. Further, considering \ms of PSO-LDE in Eqs.~(\ref{eq:PSOLDELossMU})-(\ref{eq:PSOLDELossMD}), for the above initialization each of $\{ M^{\usuff}_{\alpha}, M^{\dsuff}_{\alpha} \}$ will return $2^{- \frac{1}{\alpha}}$ for any point $X$.
Since both $M^{\usuff}_{\alpha}(X, f_{\theta}(X))$ and $M^{\dsuff}_{\alpha}(X, f_{\theta}(X))$ have the same value at every point $X \in \RR^n$, such NN bias produces a more balanced PSO gradient (see Eq.~(\ref{eq:GeneralPSOLossFrml})) at start of the training, which improves the optimization numerical stability. Furthermore, as mentioned above in case the chosen $\probs{\dsuff}$ is indeed close to the target $\probs{\usuff}$, the initial value of the surface $f_{\theta}(X)$ is also close to its final converged form; this in turn increases the convergence rate of PSO-LDE.
Further, recently a similar idea was suggested in \citep{Labeau18iccl} specifically for NCE method (PSO-LDE with $\alpha = 1$) in the context of discrete density estimation for language modeling, where NN initialization according to outputs of the noise density (parallel to $\probs{\dsuff}(X)$ in our work) helped to improve the learned model accuracy.

We perform both techniques inside the computational graph of NN $f_{\theta}(X)$, by adding at the beginning of graph the operation in Eq.~(\ref{eq:NormData}), and at the end of graph - the operation in Eq.~(\ref{eq:BiasNN}).

\subsection{Other NN Architecture Aspects}
\label{sec:NNOther}

In our experiments we also explored two choices of non-linear activation function to use within NN $f_{\theta}(X)$, Relu and Leaky Relu. We found that both have their advantages and disadvantages. Relu reduces training time by $30 \%$ w.r.t. Leaky Relu, allegedly due to its implicit gradient sparsity, and the converged surface looks more \emph{smooth}. Yet, it often contains mode collapse areas where several modes of $\probs{\usuff}$ are represented within $f_{\theta}(X)$ by a single "hill". On the other hand, Leaky-Relu sometimes produces artifacts near sharp edges within $f_{\theta}$, that resembles  Gibbs phenomenon. Yet, it yields significantly less mode collapses.

We argue that these mode collapses are in general caused by the reduced model flexibility, which in case of Relu is induced by more sparse gradients $\nabla_{\theta} 
f_{\theta}(\cdot)$ as follows. 
The implicit gradient sparsity of Relu, i.e. zero-gradient $\frac{\partial f_{\theta}(X)}{\partial \theta_i} = 0$ for the most part of the weights $\theta_i \in \theta$ at all input points $X \in \RR^{n}$, also reduces the effective dimension of subspace spanned by $\nabla_{\theta} 
f_{\theta}(\cdot)$ evaluated at training points.
Thus, the number of possible independent gradient vectors at different points (i.e. the rank of the aforementioned subspace) is also reduced, which increases (on average) the correlation $g_{\theta}(X, X') \triangleq \nabla_{\theta} 
f_{\theta}(X)^T \cdot \nabla_{\theta} 
f_{\theta}(X')$ between various (even faraway) points $X$ and $X'$. This increase can also be interpreted as an increase of the kernel bandwidth, which will lead to
expressiveness reduction of the model, according to Section \ref{sec:ExprrKernelBnd}.

Further, residual (skip) connections between different NN layers became very popular in recent NN architectures \citep{He16cvpr,Szegedy17aaai,Xie17cvpr}. Such connections allowed for using deeper neural networks and for the acceleration of learning convergence. Yet, in our work we did not observe any performance improvement from introducing skip connections into NN $f_{\theta}(X)$ with 8 or less layers. Thus, in most part of our experiments we did not employ these shortcuts. The only part where they were used is the Section \ref{sec:ImageEst} where networks with 14 layers were trained.

Additionally, the Batch Normalization (BN) \citep{Ioffe15arxiv} technique was shown in many DL works to stabilize the training process and improve the overall approximation accuracy. However, in our experiments on the density estimation we saw the opposite trend. That is, when BN is combined with PSO density estimators, the outcome is usually worsen than without it.

Finally, dropout \citep{Srivastava14jmlr} is known to be an useful regularization method to fight the \overfit. In our experiment we indeed observed the $g_{\theta}(X, X')$'s bandwidth increase along with an increase in the dropout probability. Hence, dropout can be considered as a tool to increase side-influence and bias of the estimation, and to reduce its variance. Further, the detailed investigation of dropout impact over $g_{\theta}(X, X')$ is outside of this paper scope.

\section{Overfitting of PSO}
\label{sec:DeepLogPDFOF}

In this section we will illustrate one of the major challenges involved in training PSO in a small dataset setting - the over-flexibility of model $f_{\theta}$ that induces \overfit. Likewise, herein we will also discuss possible solutions to overcome this issue.

\subsection{Problem Illustration}
\label{sec:DeepLogPDFOFIllustr}

\begin{figure}
	\centering
	
	\newcommand{\width}[0] {0.45}
	\newcommand{\height}[0] {0.16}

	\begin{tabular}{cc}

		\subfloat[\label{fig:GaussianOverfitted-a}]{\includegraphics[height=\height\textheight,width=\width\textwidth]{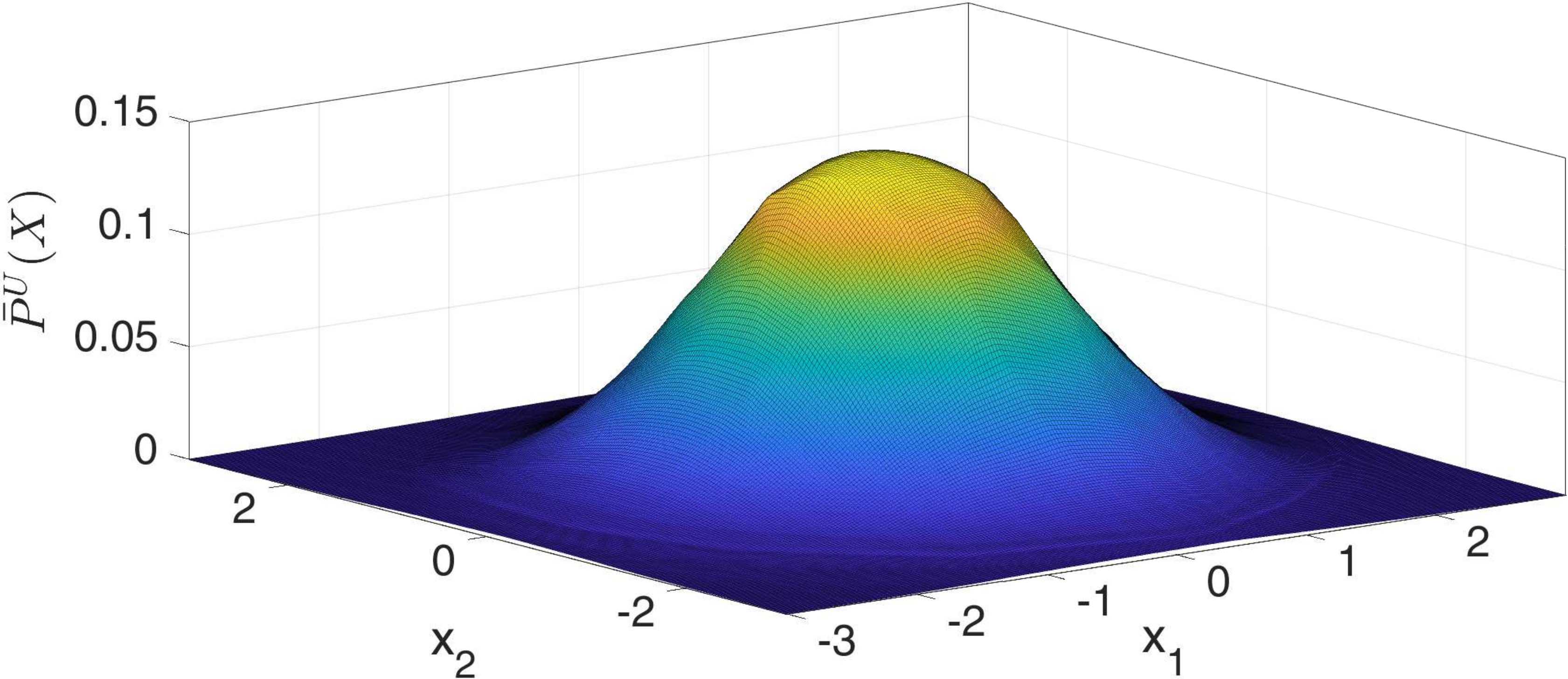}}
		&
		
		\subfloat[\label{fig:GaussianOverfitted-b}]{\includegraphics[height=\height\textheight,width=\width\textwidth]{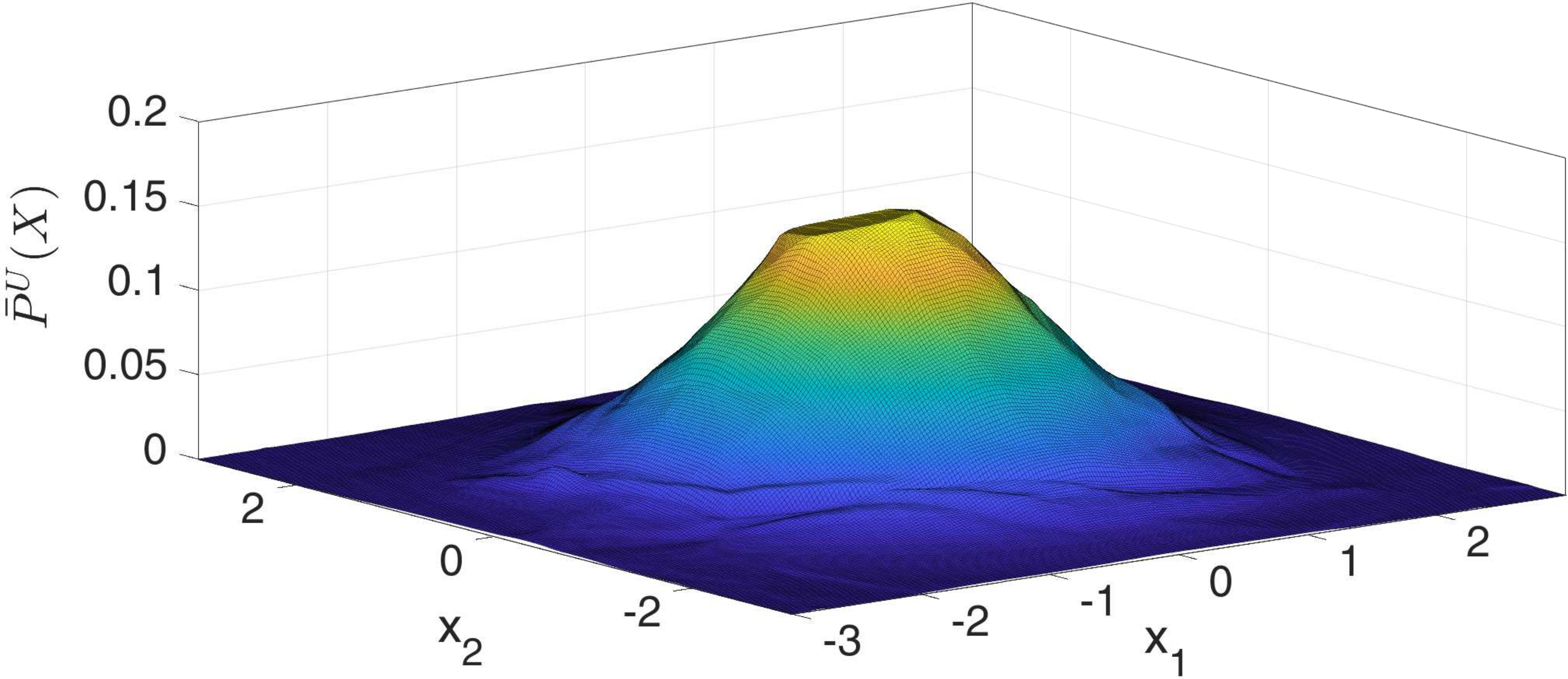}}
		
		\\
		
		\subfloat[\label{fig:GaussianOverfitted-c}]{\includegraphics[height=\height\textheight,width=\width\textwidth]{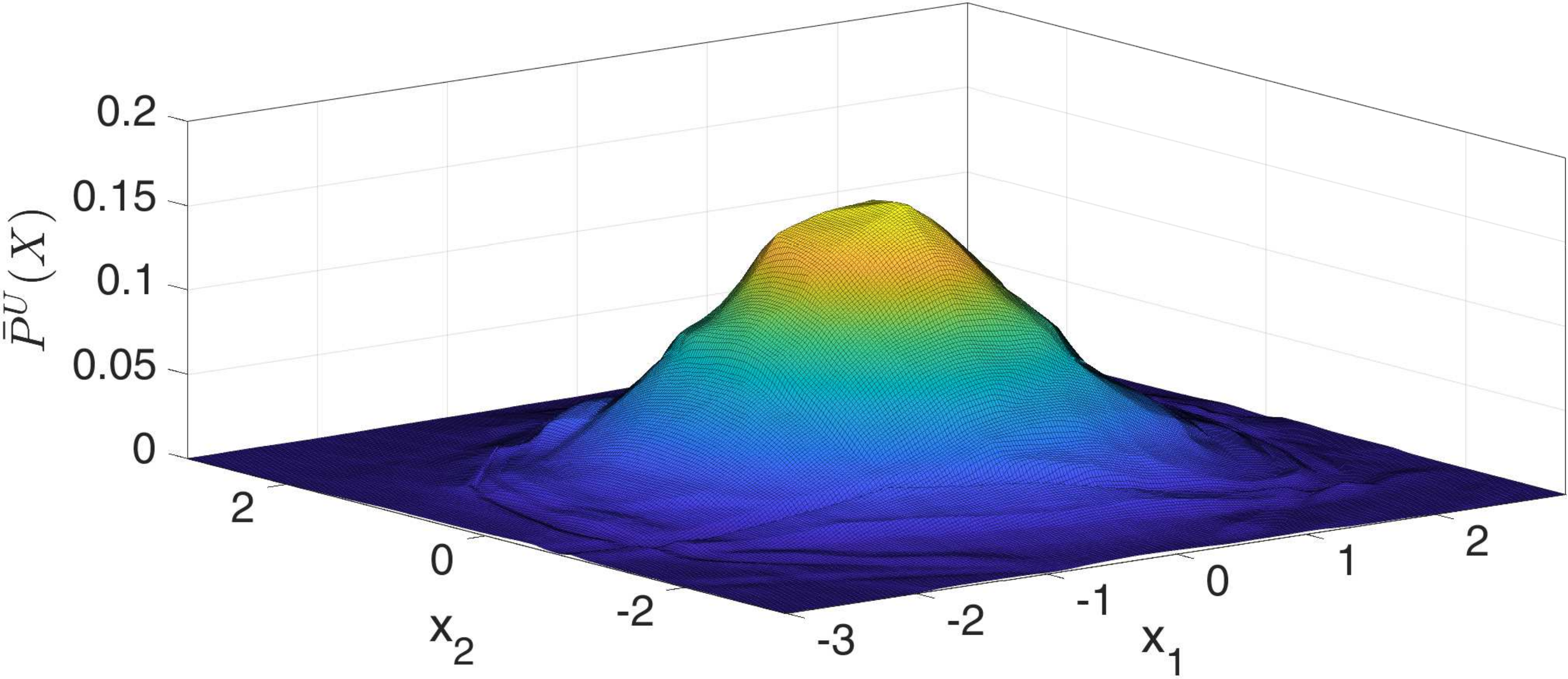}}
		&
		
		\subfloat[\label{fig:GaussianOverfitted-d}]{\includegraphics[height=\height\textheight,width=\width\textwidth]{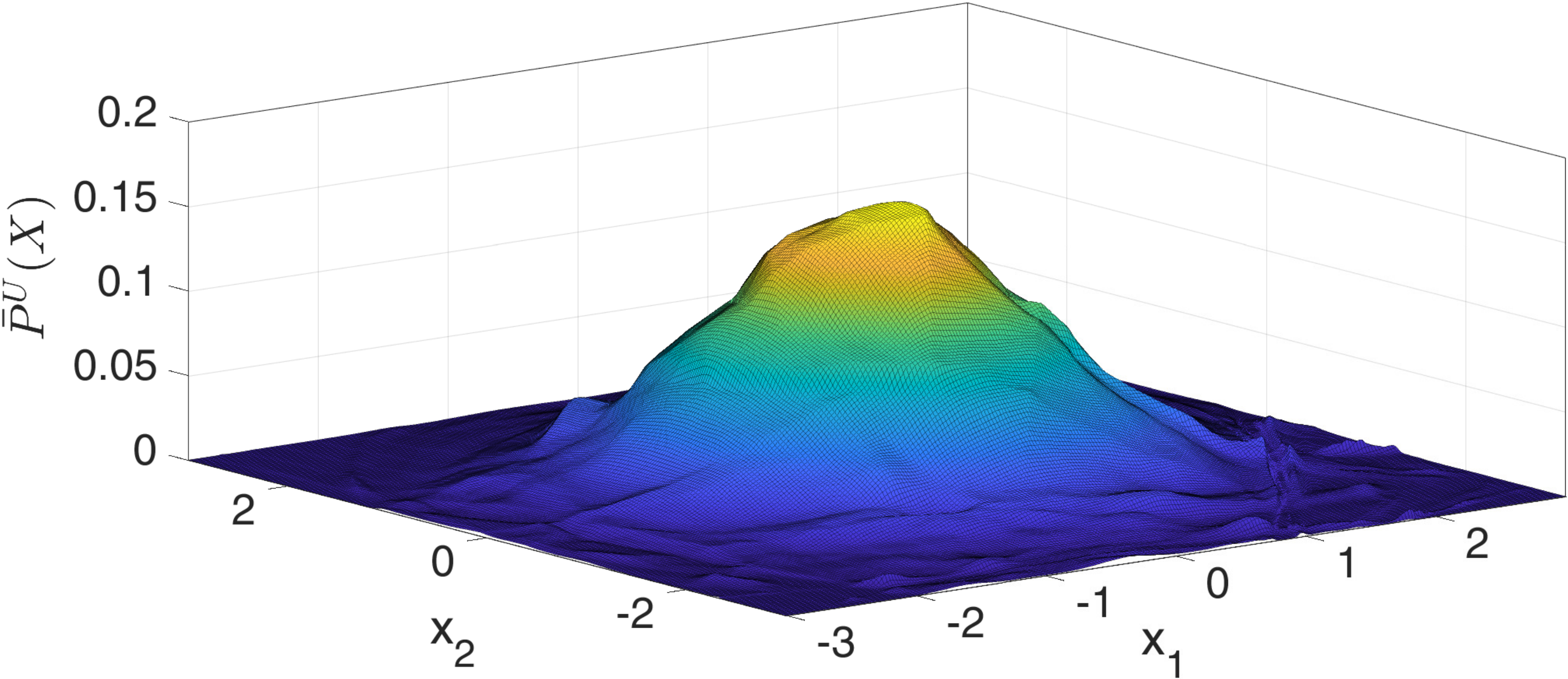}}

		\\
		
		\subfloat[\label{fig:GaussianOverfitted-e}]{\includegraphics[height=\height\textheight,width=\width\textwidth]{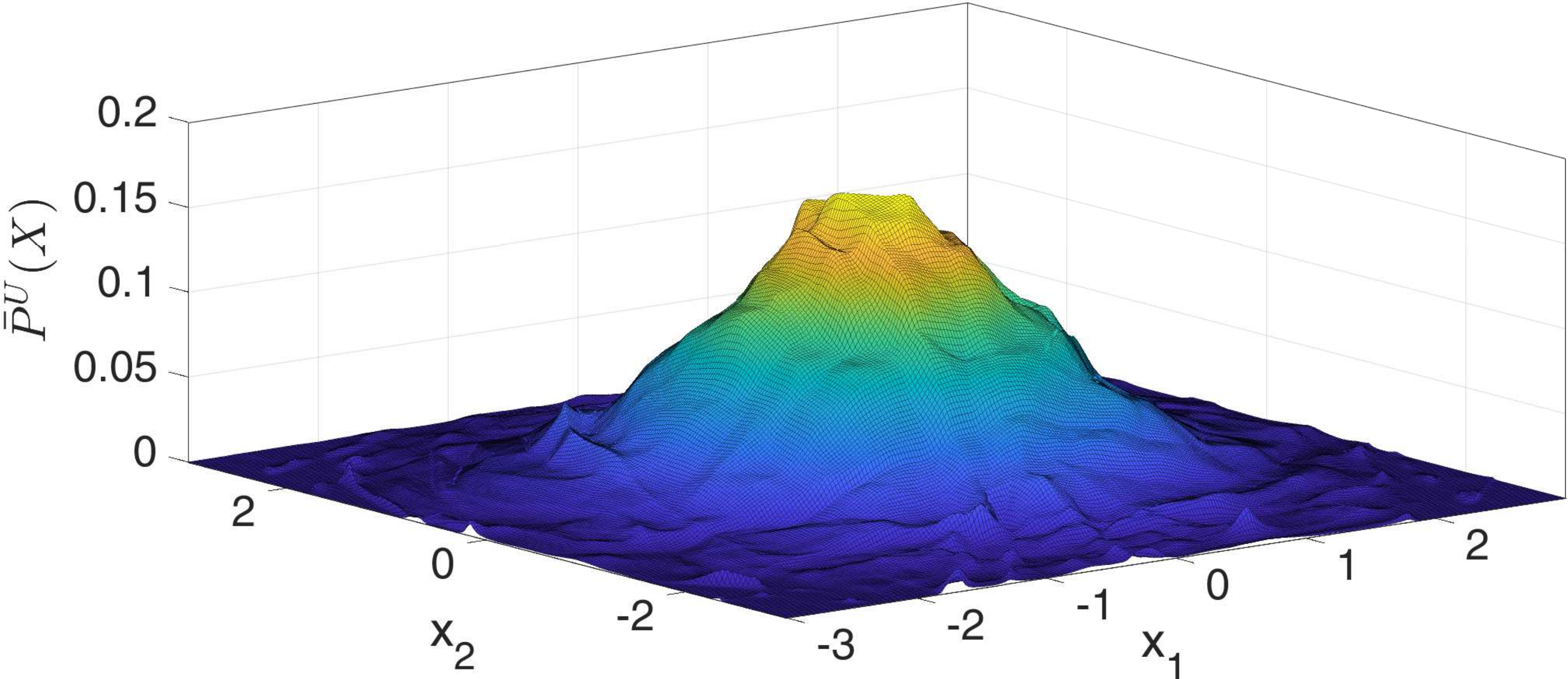}}
		&
		
		\subfloat[\label{fig:GaussianOverfitted-f}]{\includegraphics[height=\height\textheight,width=\width\textwidth]{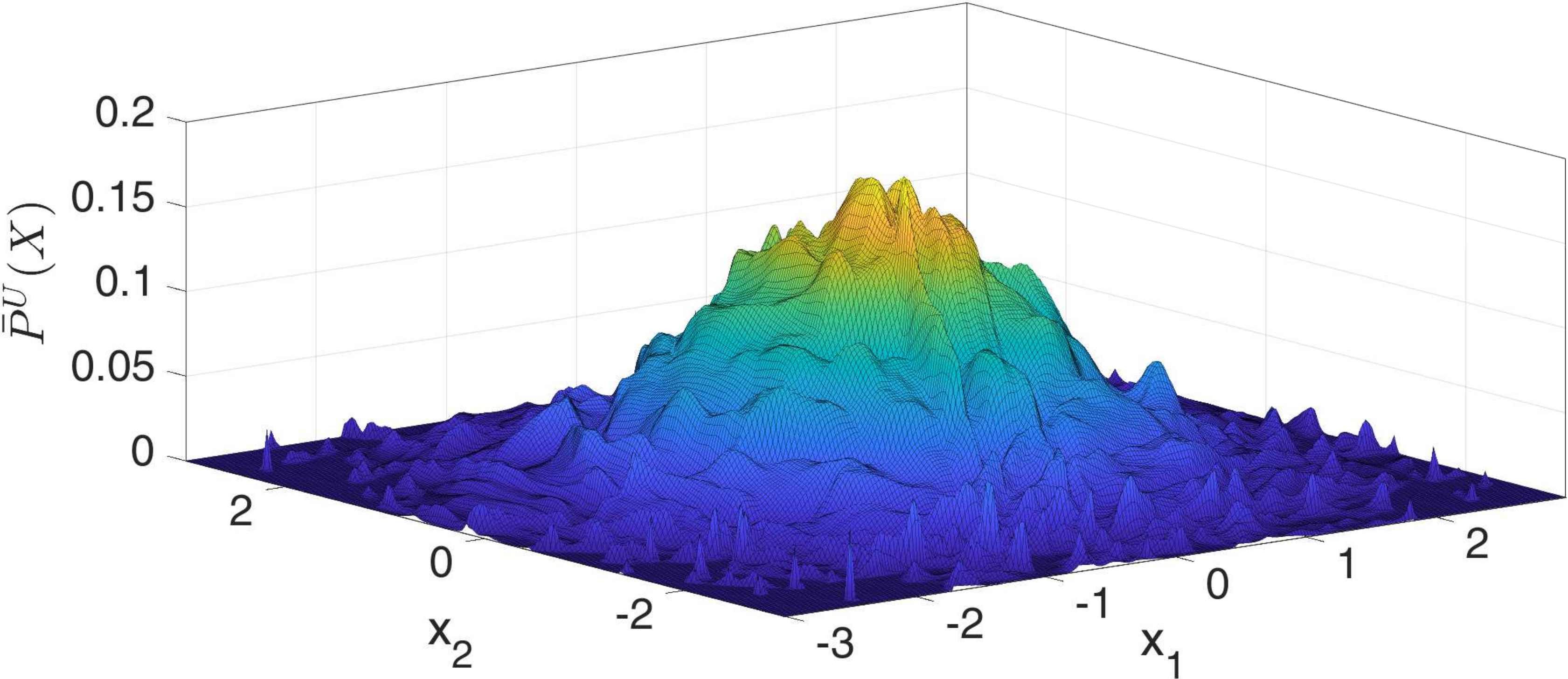}}
		
		\\
		
		\subfloat[\label{fig:GaussianOverfitted-g}]{\includegraphics[height=\height\textheight,width=\width\textwidth]{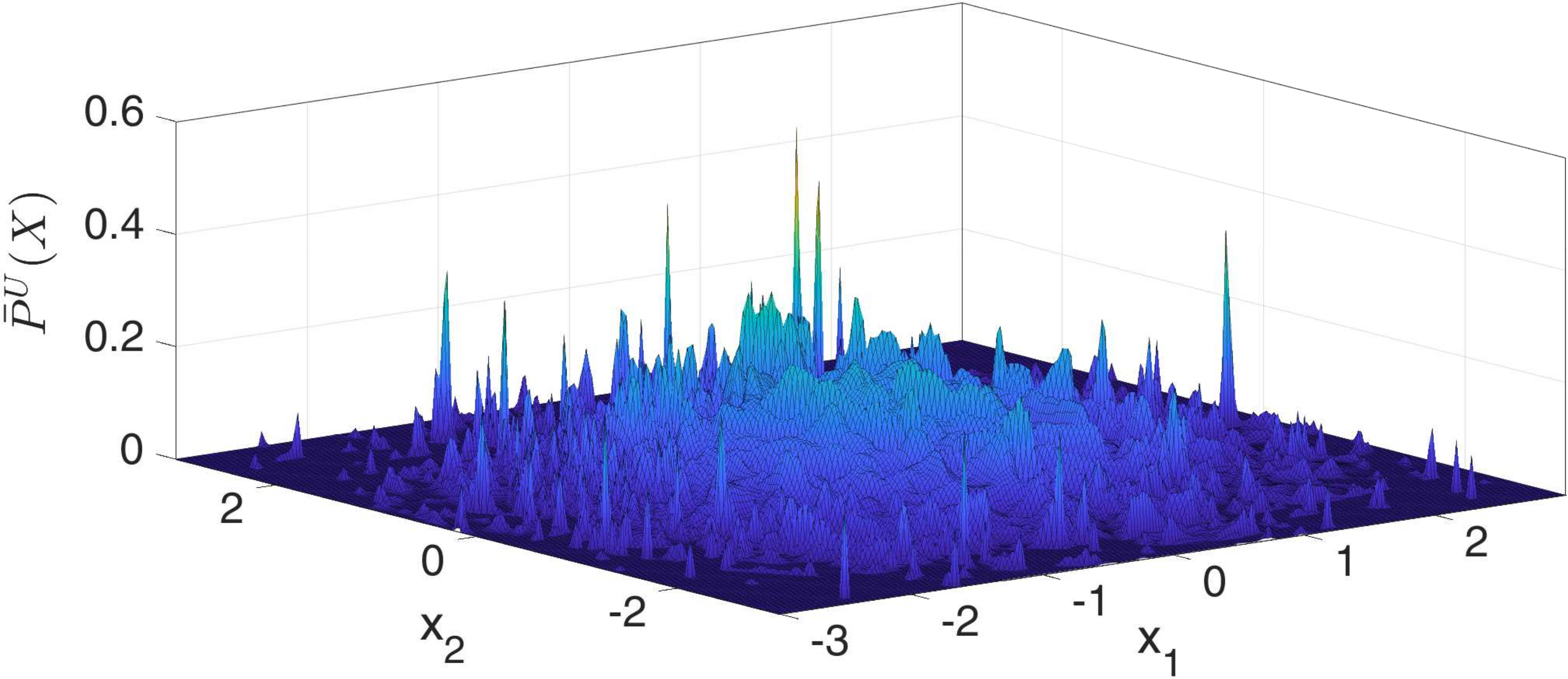}}
		&
		
		\subfloat[\label{fig:GaussianOverfitted-h}]{\includegraphics[height=\height\textheight,width=\width\textwidth]{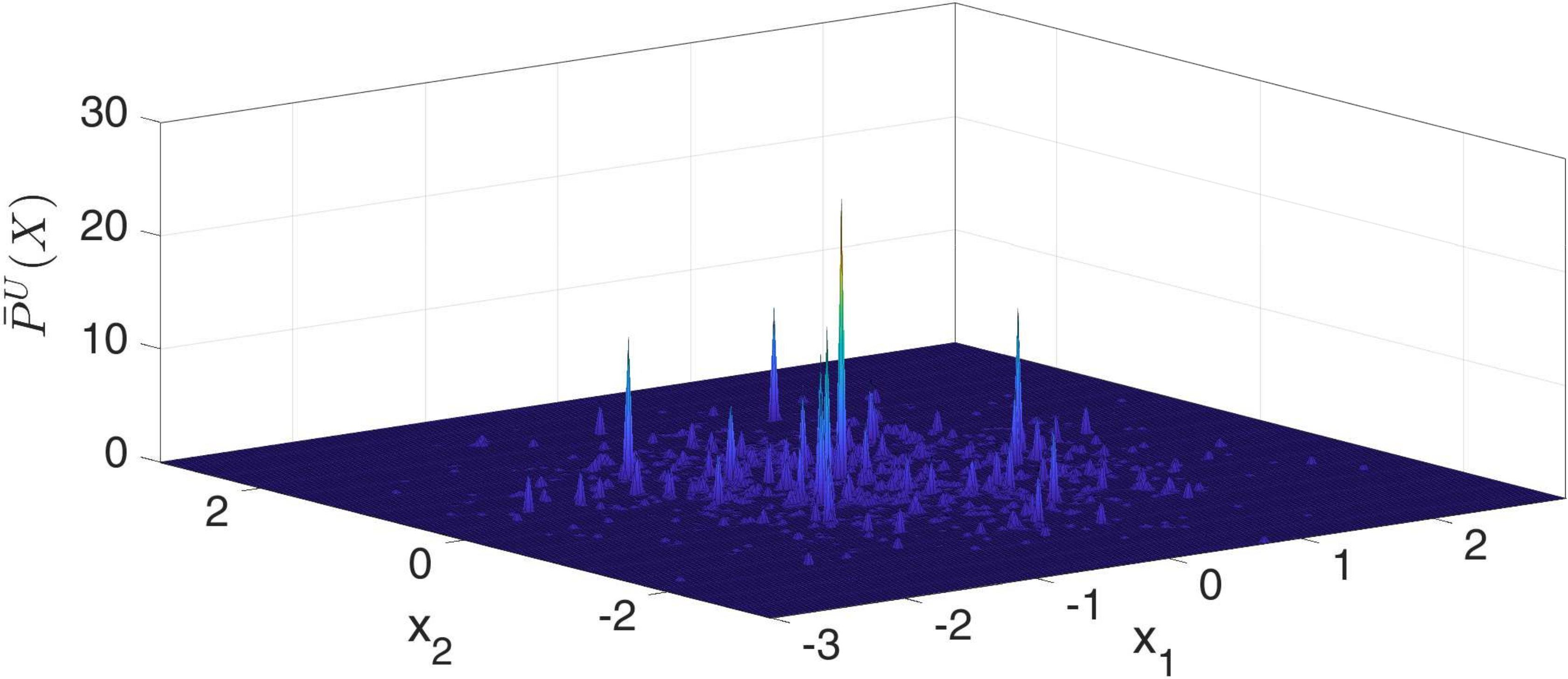}}
		
	\end{tabular}
	
	\protect
	\caption[Illustration of PSO \overfit when the training dataset is small.]{Illustration of PSO \overfit when the training dataset is small. We infer 2D \emph{Normal} distribution via $\bar{\PP}^{\usuff}(X) = \exp f_{\theta}(X)$ by using PSO-LDE with $\alpha = \frac{1}{4}$ (see Table \ref{tbl:PSOInstances1} and Section \ref{sec:DeepLogPDF}). The applied NN architecture is block-diagonal with 6 layers, number of blocks $N_B = 50$ and block size $S_B = 64$ (see Section \ref{sec:BDLayers}). Number of \up training points $\{ X^{\usuff}_{i} \}$ is (a) $10^6$, (b) $10^5$, (c) $80000$, (d) $60000$, (e) $40000$, (f) $20000$, (g) $10000$ and (h) $1000$. As observed, when using the same NN architecture, that is when we do not reduce the flexibility level of the model, the smaller number of training points leads to the spiky approximation. In other words, the converged model will contain a peak around each training sample point.
	}
	\label{fig:GaussianOverfitted}
\end{figure}

\begin{figure}[tb]
	\centering
	
	\newcommand{\width}[0] {0.43}
	\newcommand{\height}[0] {0.13}

	\begin{tabular}{cc}

		\subfloat[\label{fig:GaussianOverfittedR-a}]{\includegraphics[height=\height\textheight,width=\width\textwidth]{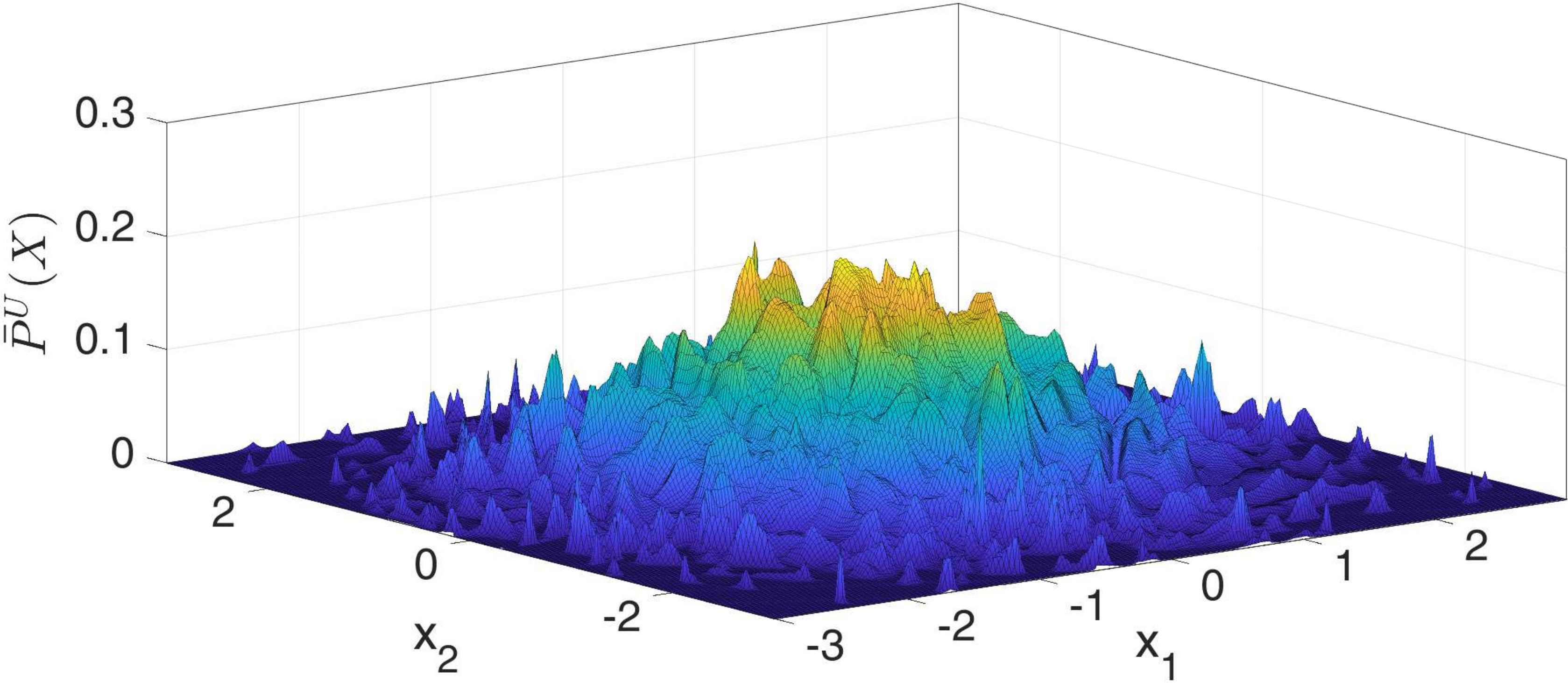}}

&
		
		\subfloat[\label{fig:GaussianOverfittedR-b}]{\includegraphics[height=\height\textheight,width=\width\textwidth]{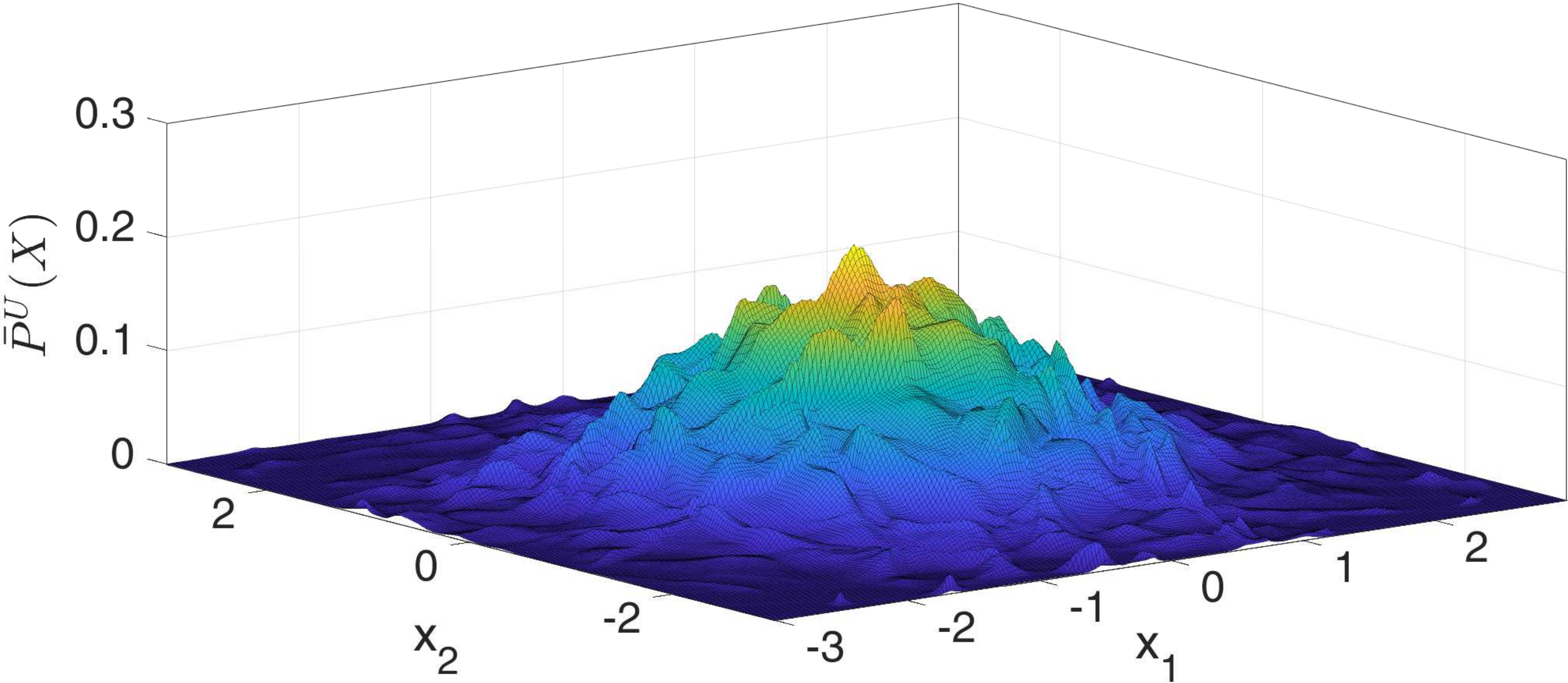}}
		
		\\
		
		\subfloat[\label{fig:GaussianOverfittedR-c}]{\includegraphics[height=\height\textheight,width=\width\textwidth]{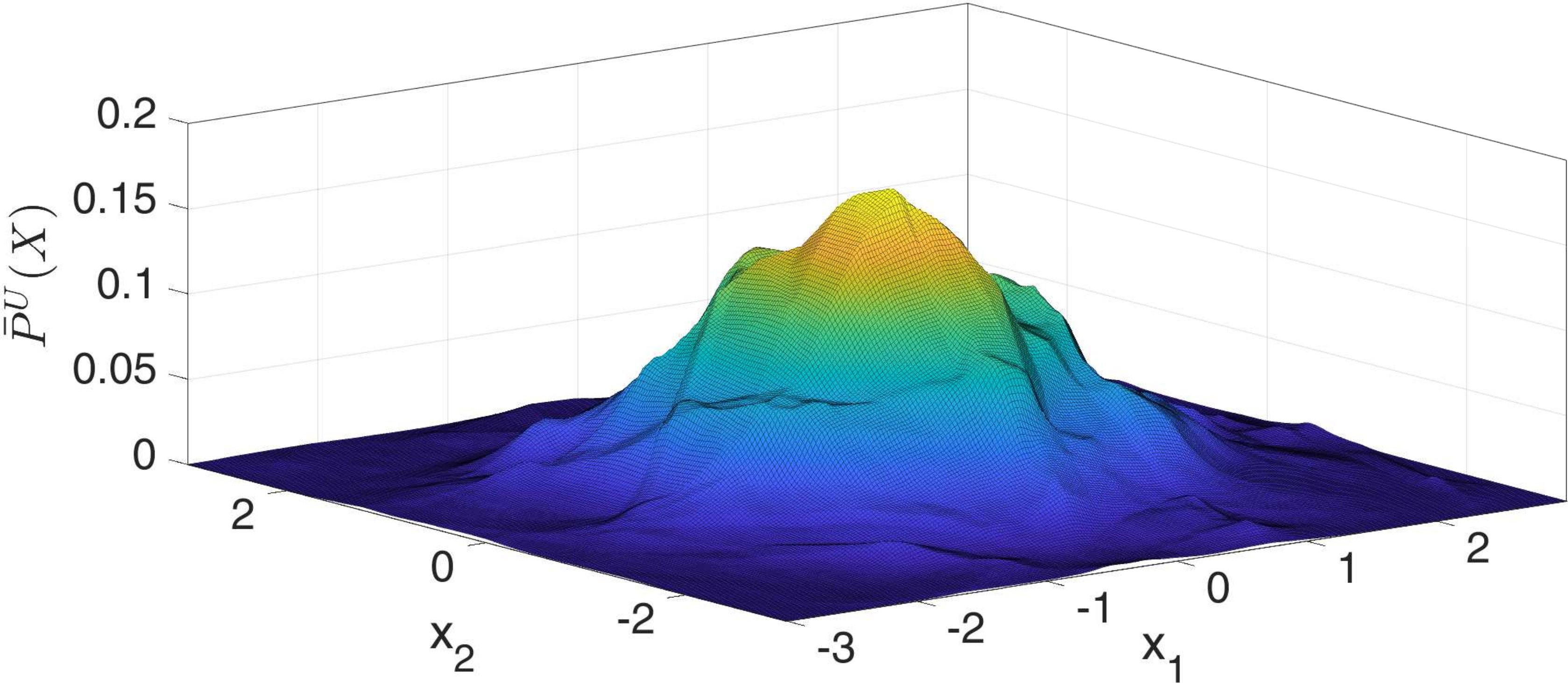}}
		&
		
		\subfloat[\label{fig:GaussianOverfittedR-d}]{\includegraphics[height=\height\textheight,width=\width\textwidth]{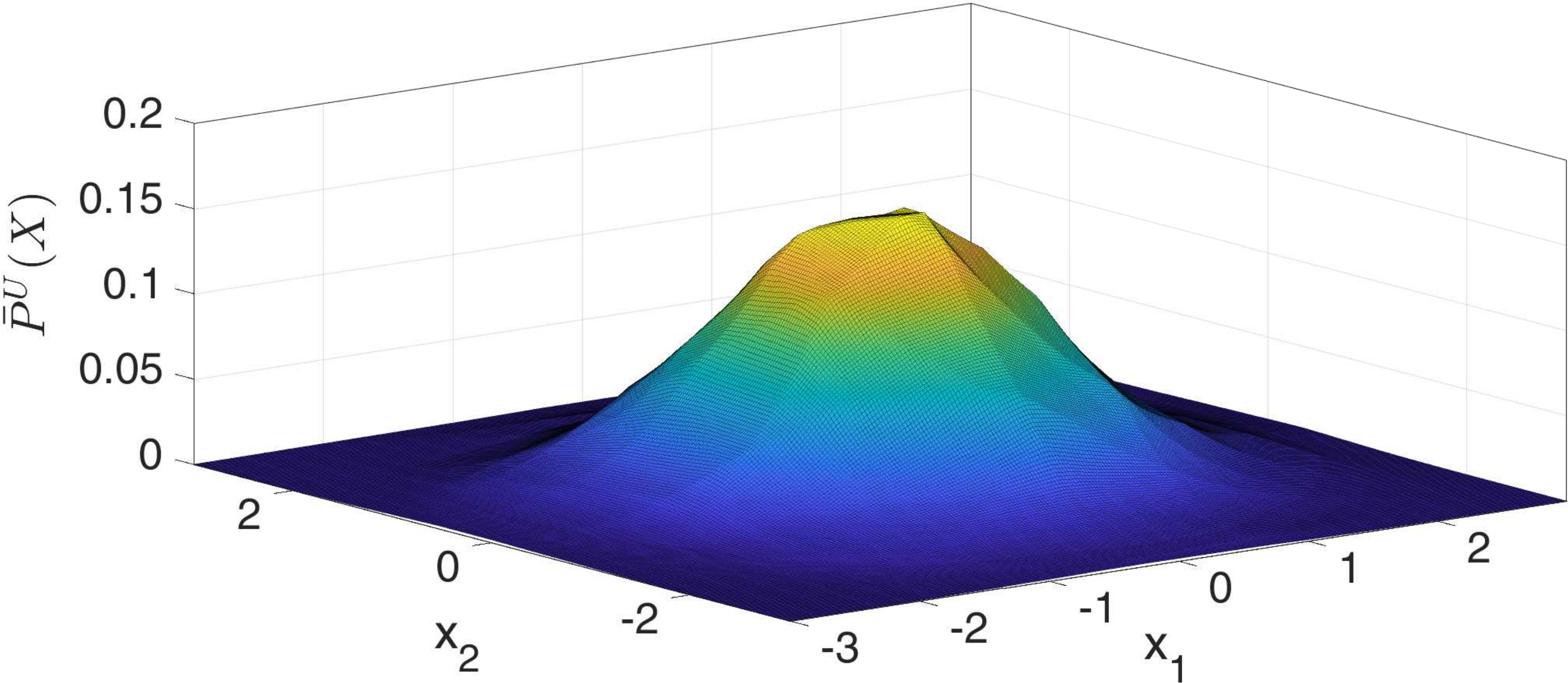}}

	\end{tabular}
	
	\protect
	\caption[Illustration of decrease in PSO \overfit when the NN flexibility is reduced.]{Illustration of decrease in PSO \overfit when the NN flexibility is reduced. We infer 2D \emph{Normal} distribution via $\bar{\PP}^{\usuff}(X) = \exp f_{\theta}(X)$, using only 10000 training samples $\{ X^{\usuff}_{i} \}$. The applied loss is PSO-LDE with $\alpha = \frac{1}{4}$ (see Table \ref{tbl:PSOInstances1} and Section \ref{sec:DeepLogPDF}). The applied NN architecture is block-diagonal with number of blocks $N_B = 20$ and block size $S_B = 64$ (see Section \ref{sec:BDLayers}). Number of layers within NN is (a) $5$, (b) $4$, (c) $3$ and (d) $2$. As observed, when the number of layers is decreasing, the converged model is more smooth, with less peaks around the training points.
	}
	\label{fig:GaussianOverfittedR}
\end{figure}

\begin{figure}[!tb]
	\centering

	\newcommand{\width}[0] {0.43}
	\newcommand{\height}[0] {0.123}

	\begin{tabular}{cc}

		\subfloat[\label{fig:GaussianOverfittedR_KDE-a}]{\includegraphics[height=\height\textheight,width=\width\textwidth]{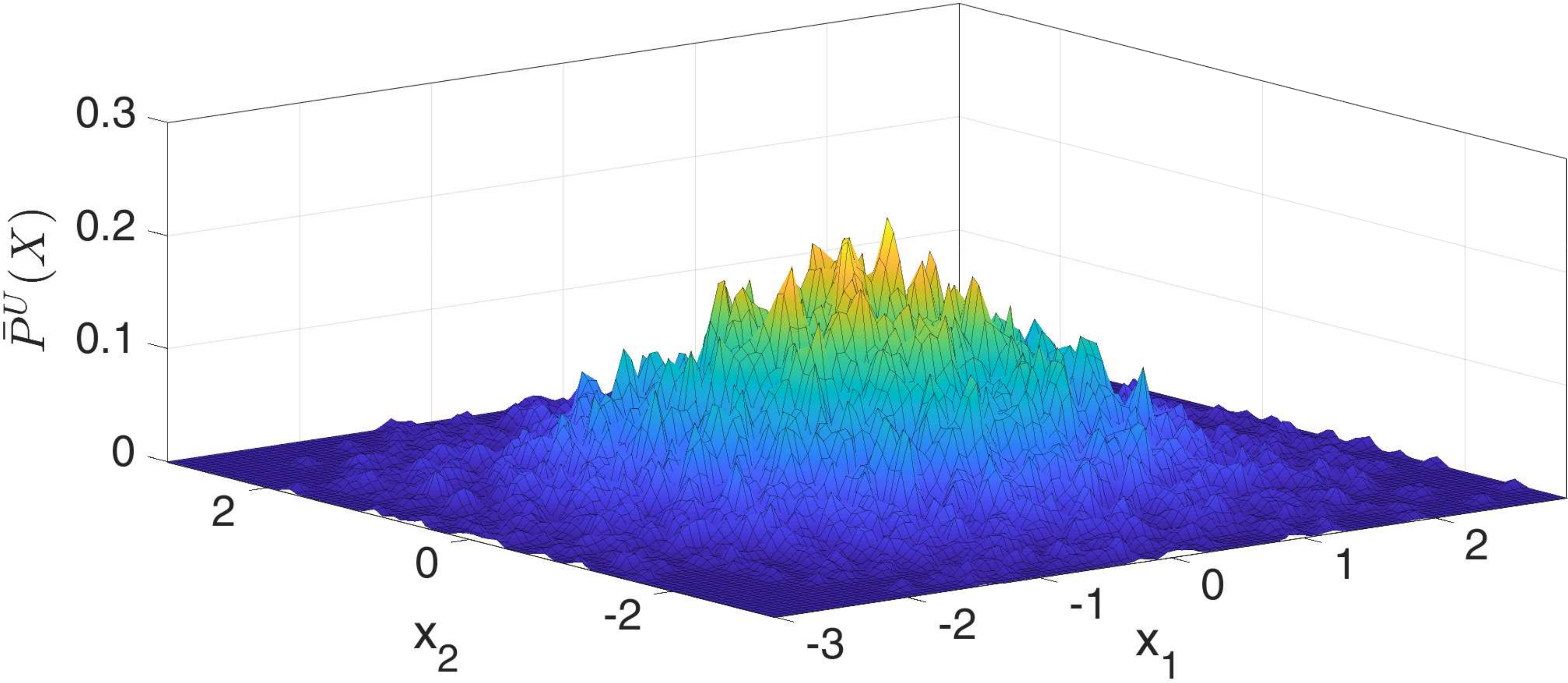}}
		
		&
		
		\subfloat[\label{fig:GaussianOverfittedR_KDE-b}]{\includegraphics[height=\height\textheight,width=\width\textwidth]{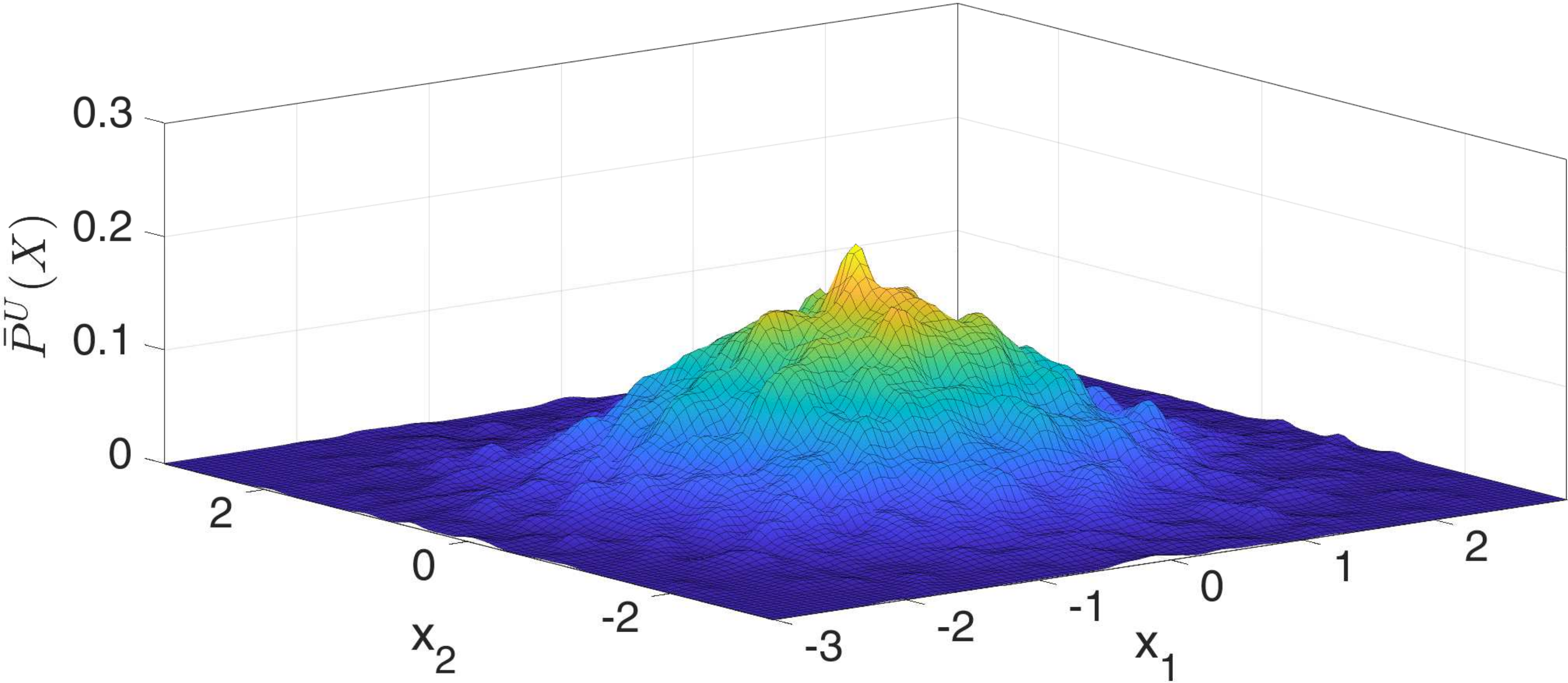}}
		
		\\
		
		\subfloat[\label{fig:GaussianOverfittedR_KDE-c}]{\includegraphics[height=\height\textheight,width=\width\textwidth]{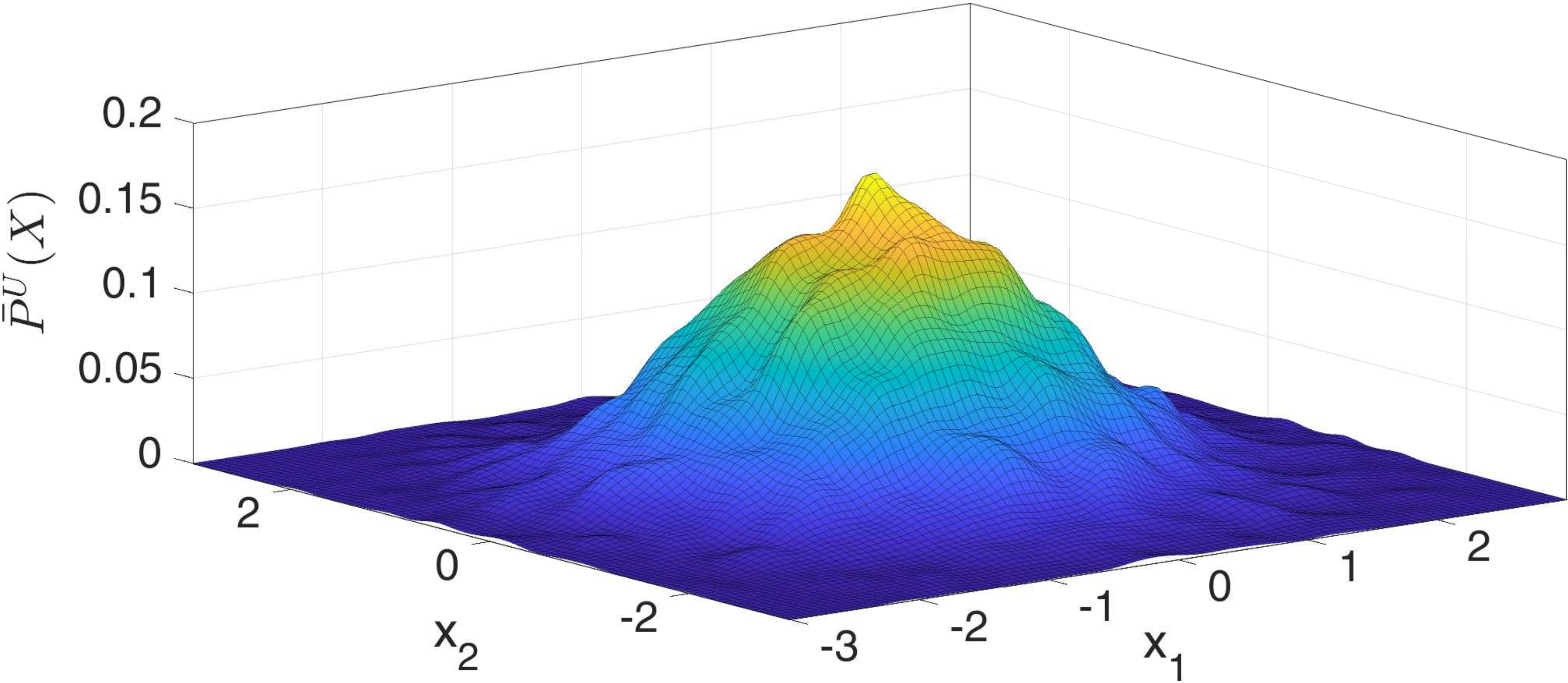}}
		&
		
		\subfloat[\label{fig:GaussianOverfittedR_KDE-d}]{\includegraphics[height=\height\textheight,width=\width\textwidth]{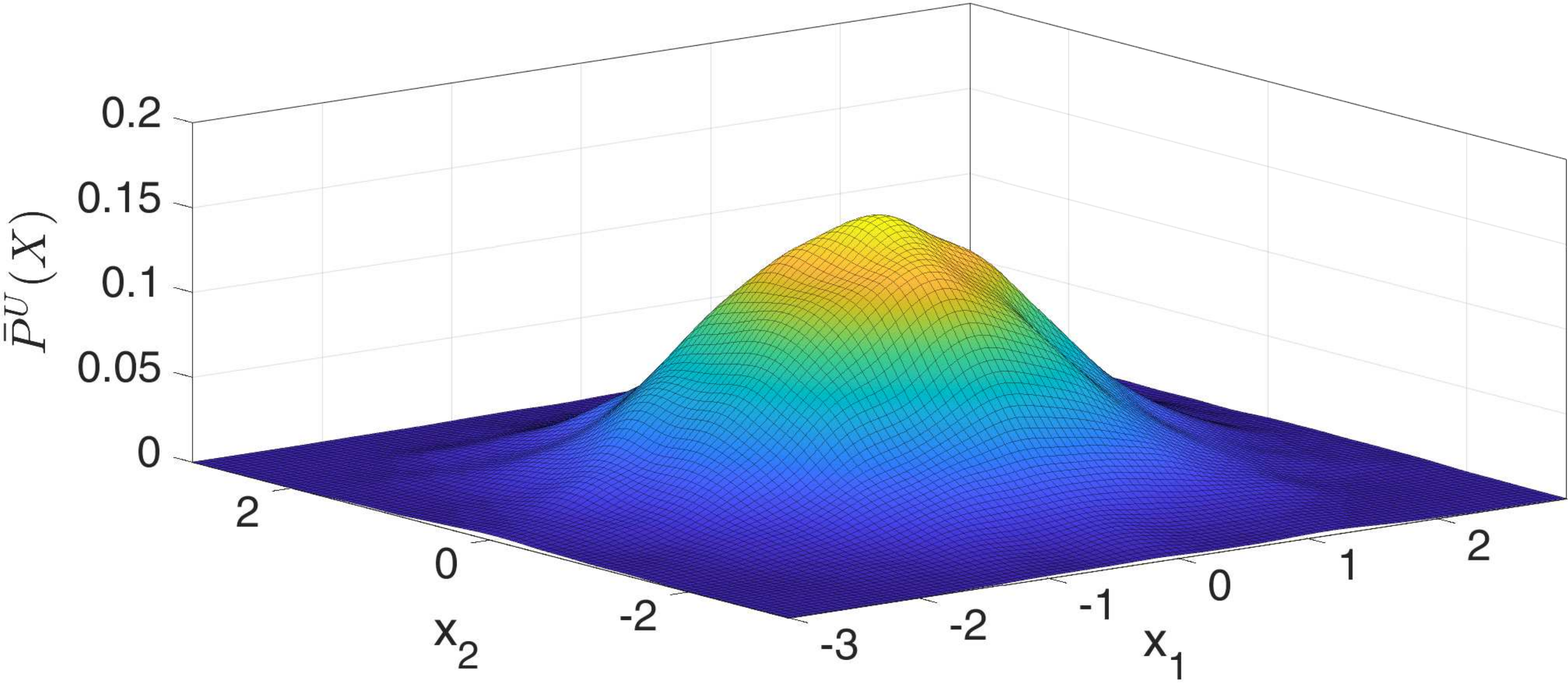}}

	\end{tabular}
	
	\protect
	\caption[Illustration of decrease in KDE (kernel density estimation) \overfit when the \emph{bandwidth} $h$ of applied Gaussian kernel is increased.]{Illustration of decrease in KDE (kernel density estimation) \overfit when the \emph{bandwidth} $h$ of applied Gaussian kernel is increased. We infer 2D \emph{Normal} distribution via KDE, using only 10000 training samples $\{ X^{\usuff}_{i} \}$. Used kernel has $h$ equal to (a) $0.04$, (b) $0.08$, (c) $0.12$ and (d) $0.2$. As observed, when the \emph{bandwidth} $h$ is increasing, the converged model is more smooth, with less peaks around the training points. Similar trend is observed for PSO in Figure \ref{fig:GaussianOverfittedR}.
	}
	\label{fig:GaussianOverfittedR_KDE}
\end{figure}

\begin{figure}[tbp]
	\centering
	
	\newcommand{\width}[0] {0.45}
	\newcommand{\height}[0] {0.18}
	
	\begin{tabular}{cc}

		\subfloat[\label{fig:GaussianOverfittedR_SI-a}]{\includegraphics[height=\height\textheight,width=\width\textwidth]{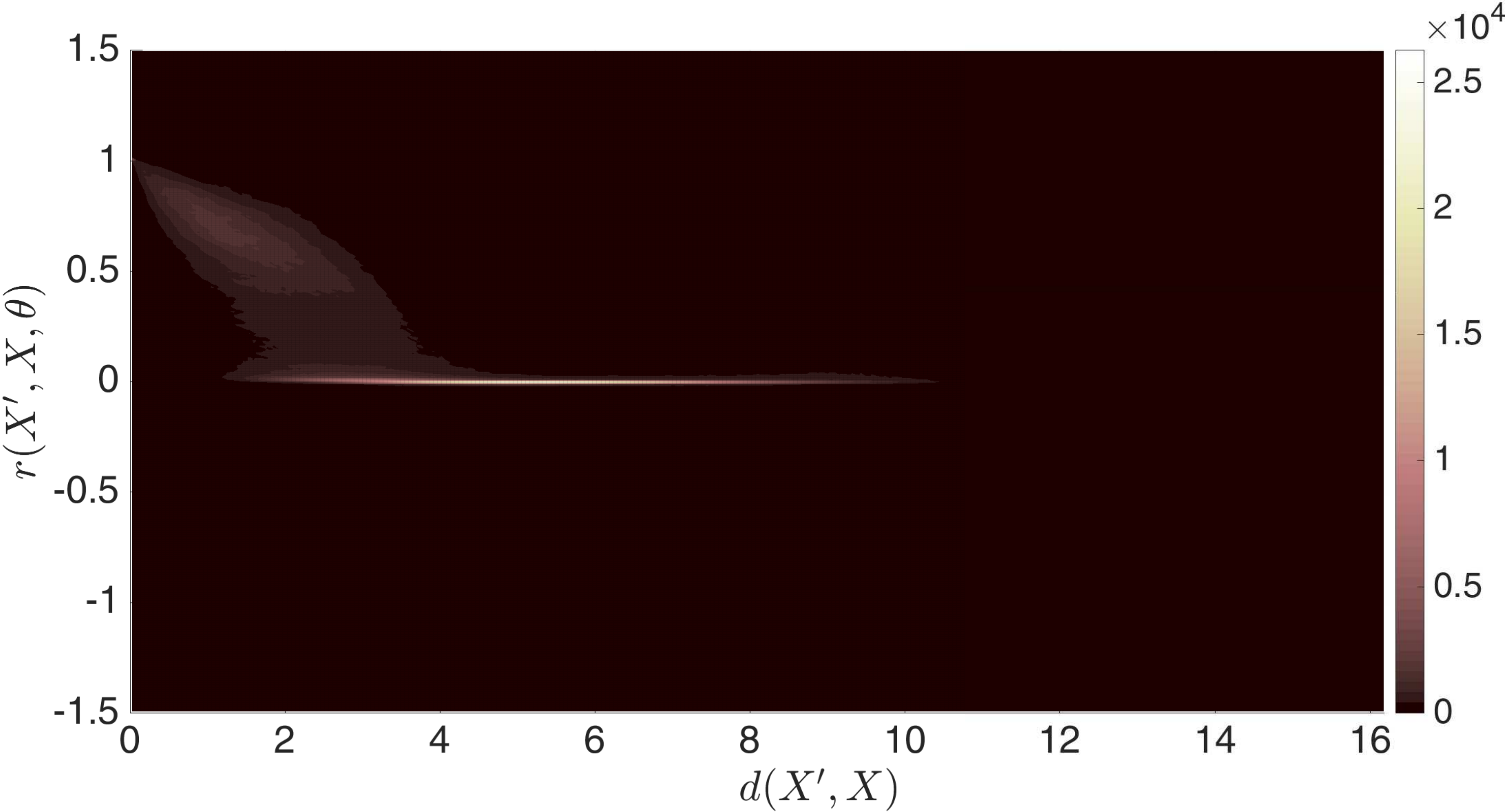}}
		
		&
		
		\subfloat[\label{fig:GaussianOverfittedR_SI-b}]{\includegraphics[height=\height\textheight,width=\width\textwidth]{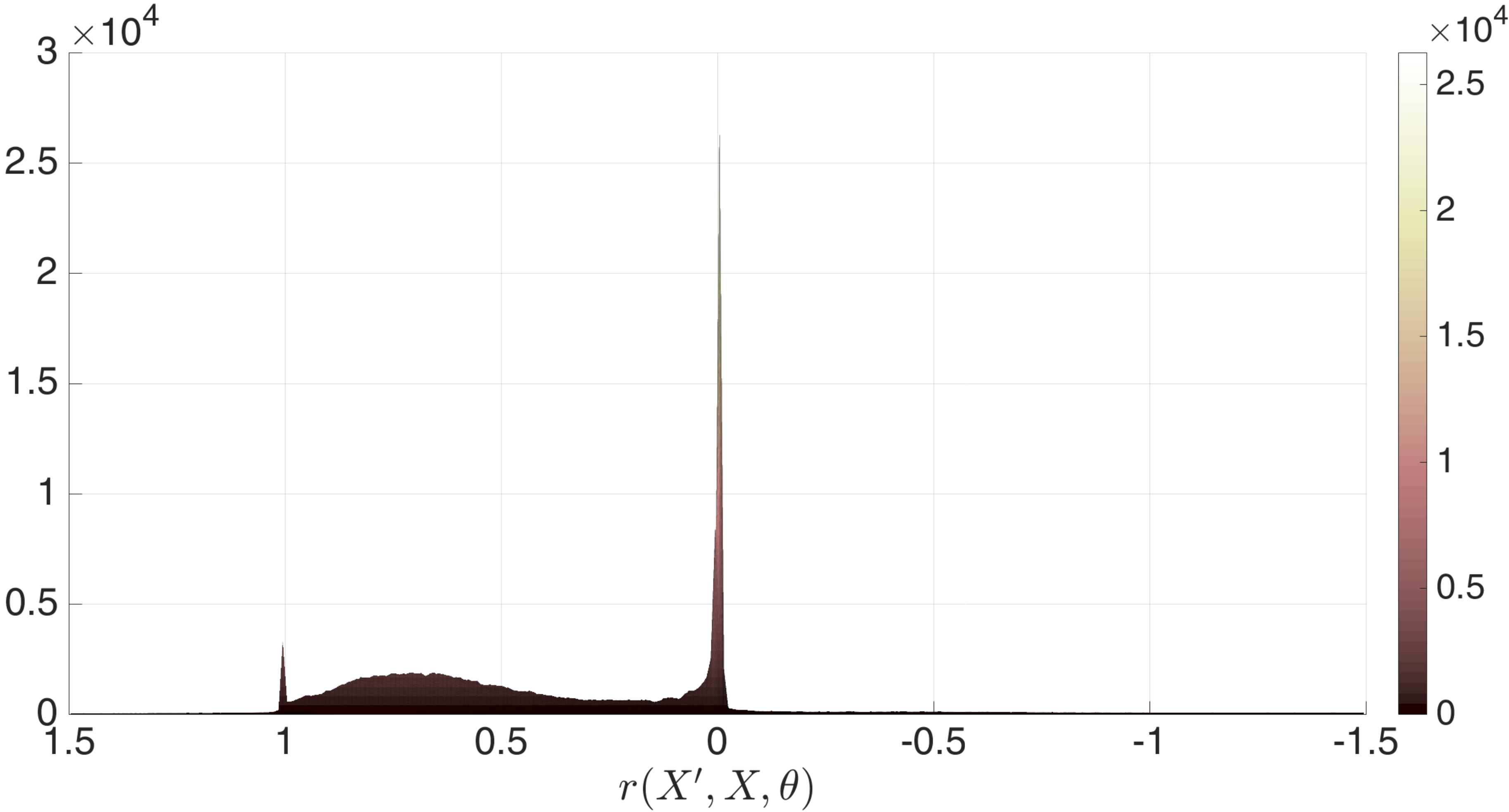}}
		
		\\
		
		\subfloat[\label{fig:GaussianOverfittedR_SI-c}]{\includegraphics[height=\height\textheight,width=\width\textwidth]{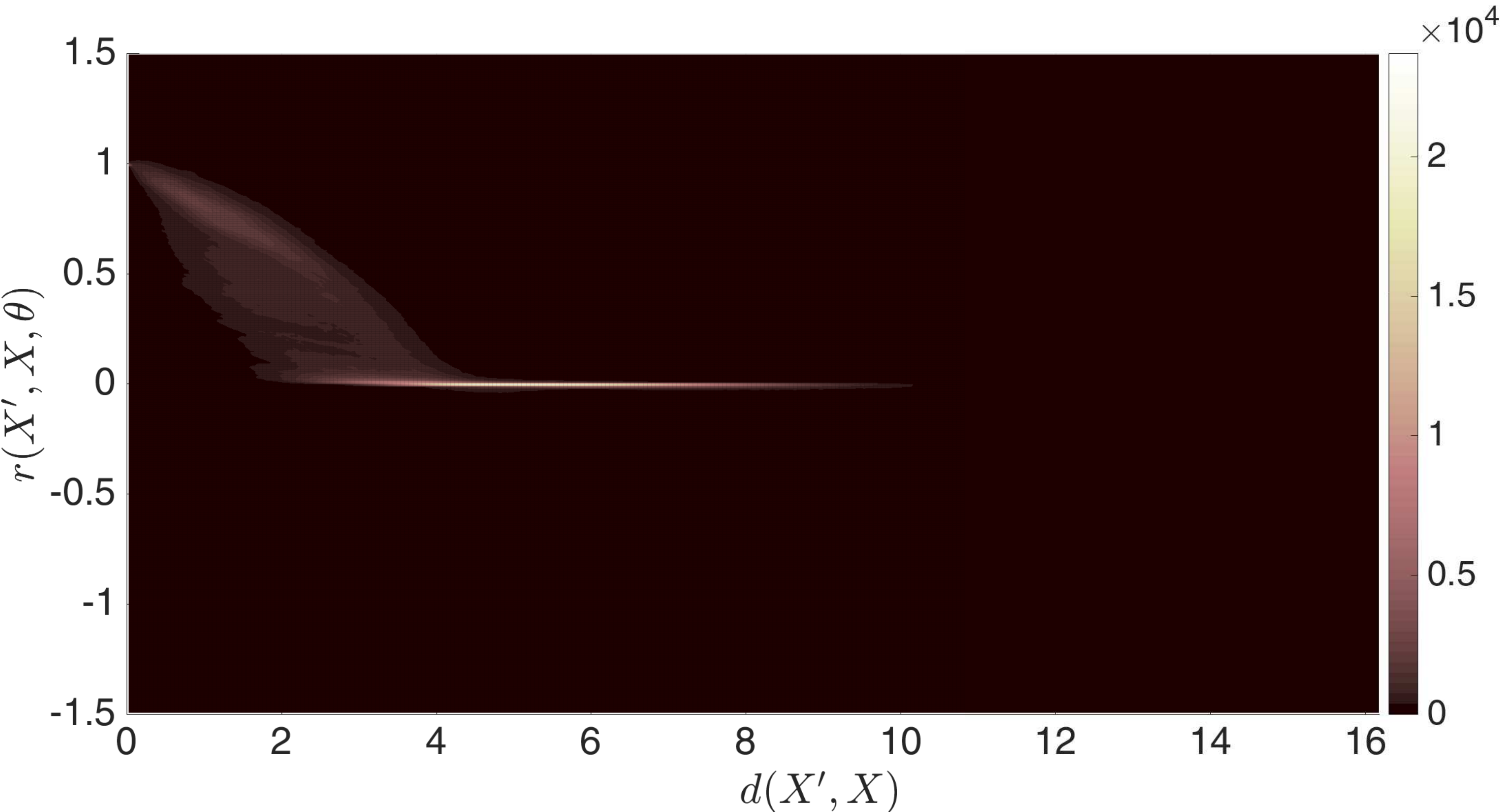}}
		
		&
		
		\subfloat[\label{fig:GaussianOverfittedR_SI-d}]{\includegraphics[height=\height\textheight,width=\width\textwidth]{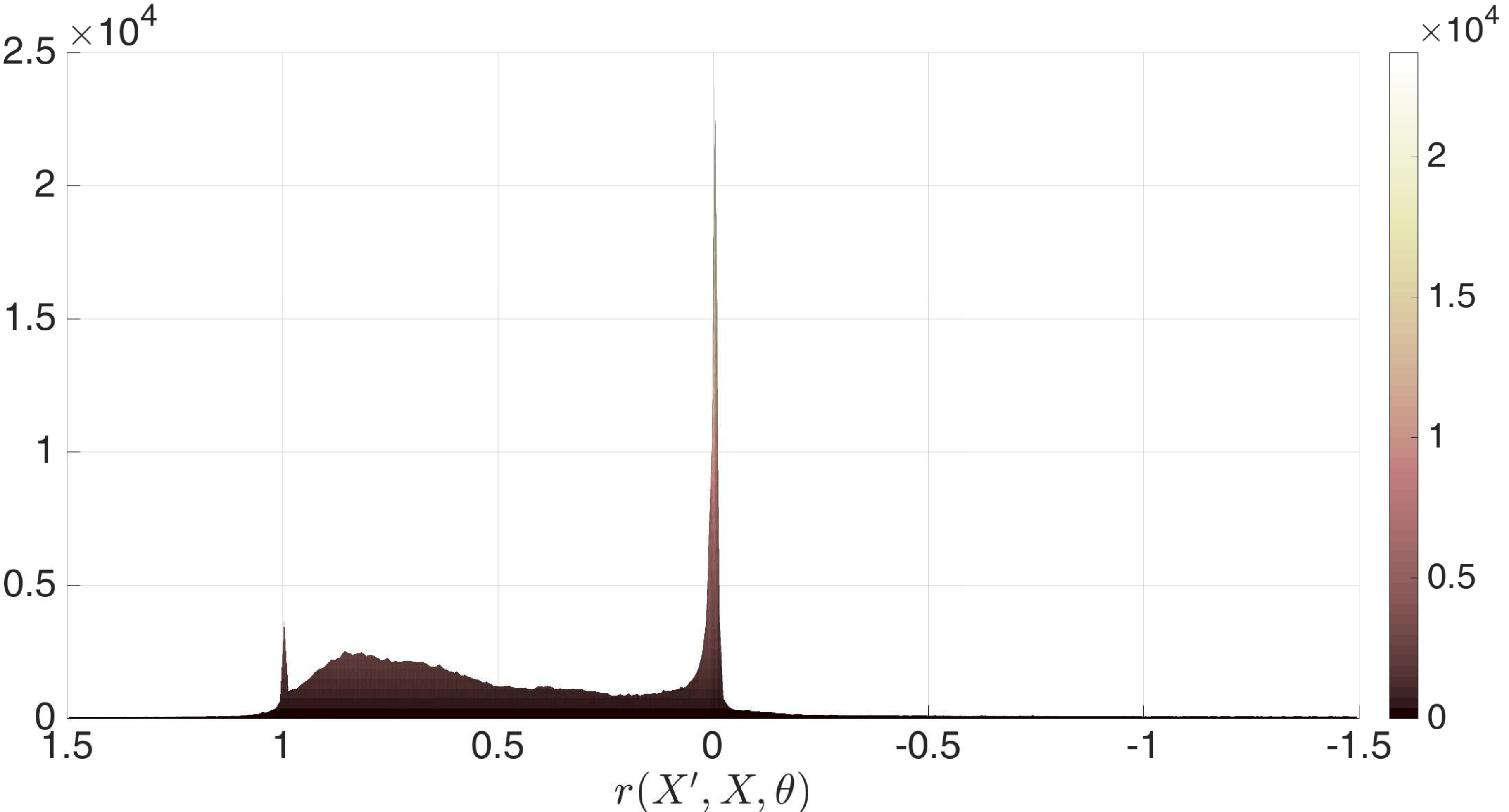}}
		
		\\

		\subfloat[\label{fig:GaussianOverfittedR_SI-e}]{\includegraphics[height=\height\textheight,width=\width\textwidth]{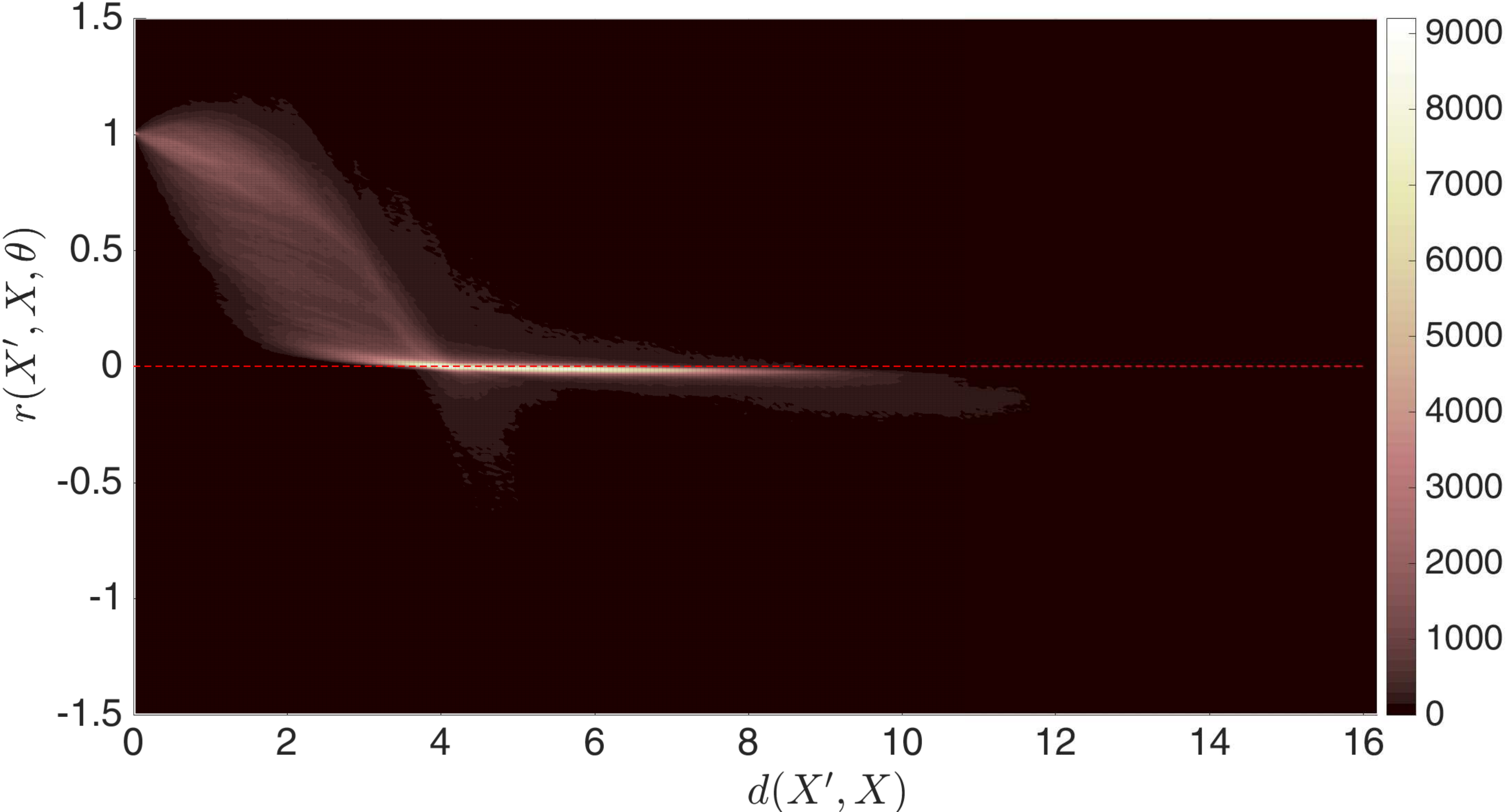}}
		
		&
		
		\subfloat[\label{fig:GaussianOverfittedR_SI-f}]{\includegraphics[height=\height\textheight,width=\width\textwidth]{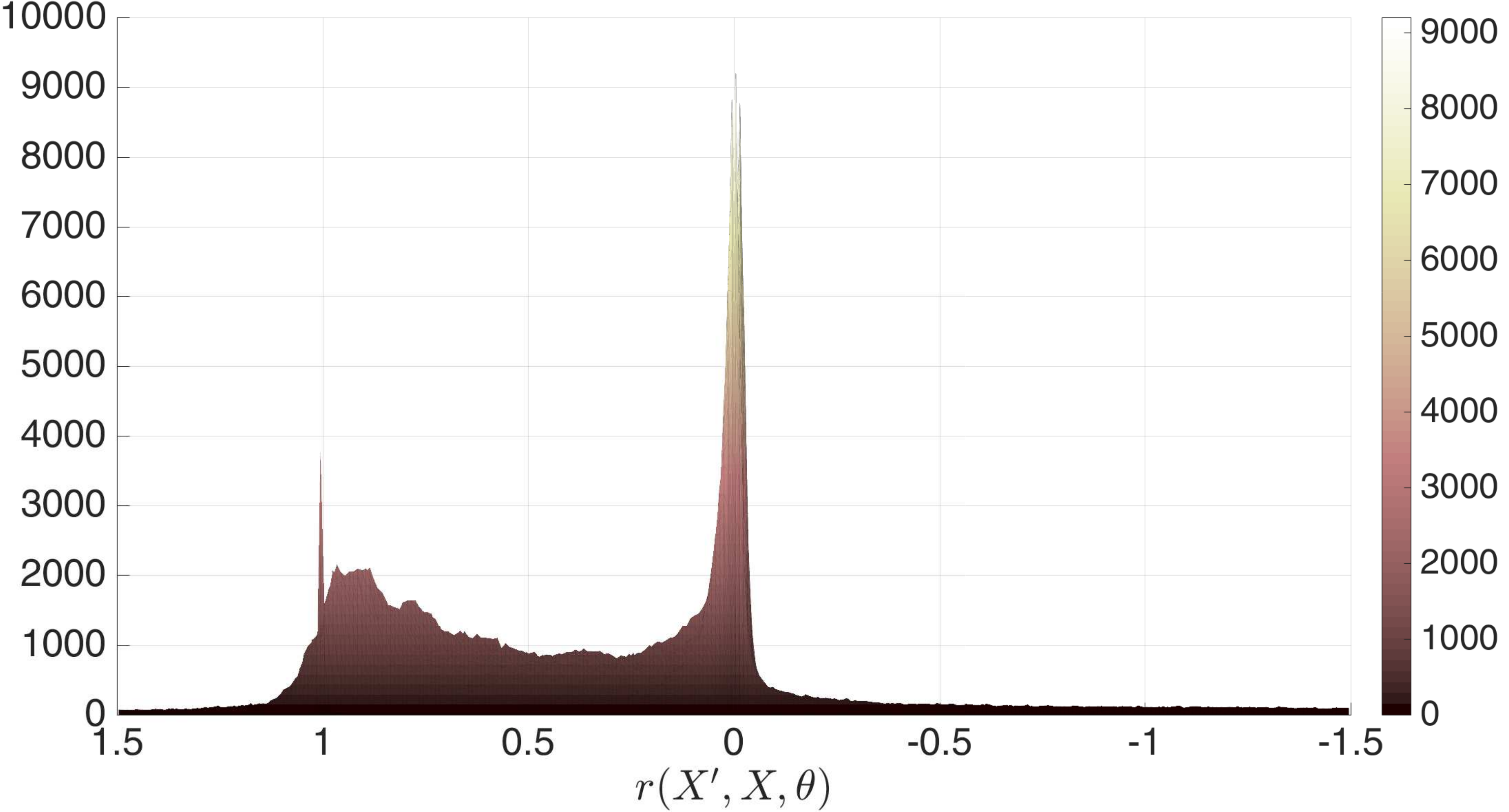}}
		
		\\
		
		\subfloat[\label{fig:GaussianOverfittedR_SI-g}]{\includegraphics[height=\height\textheight,width=\width\textwidth]{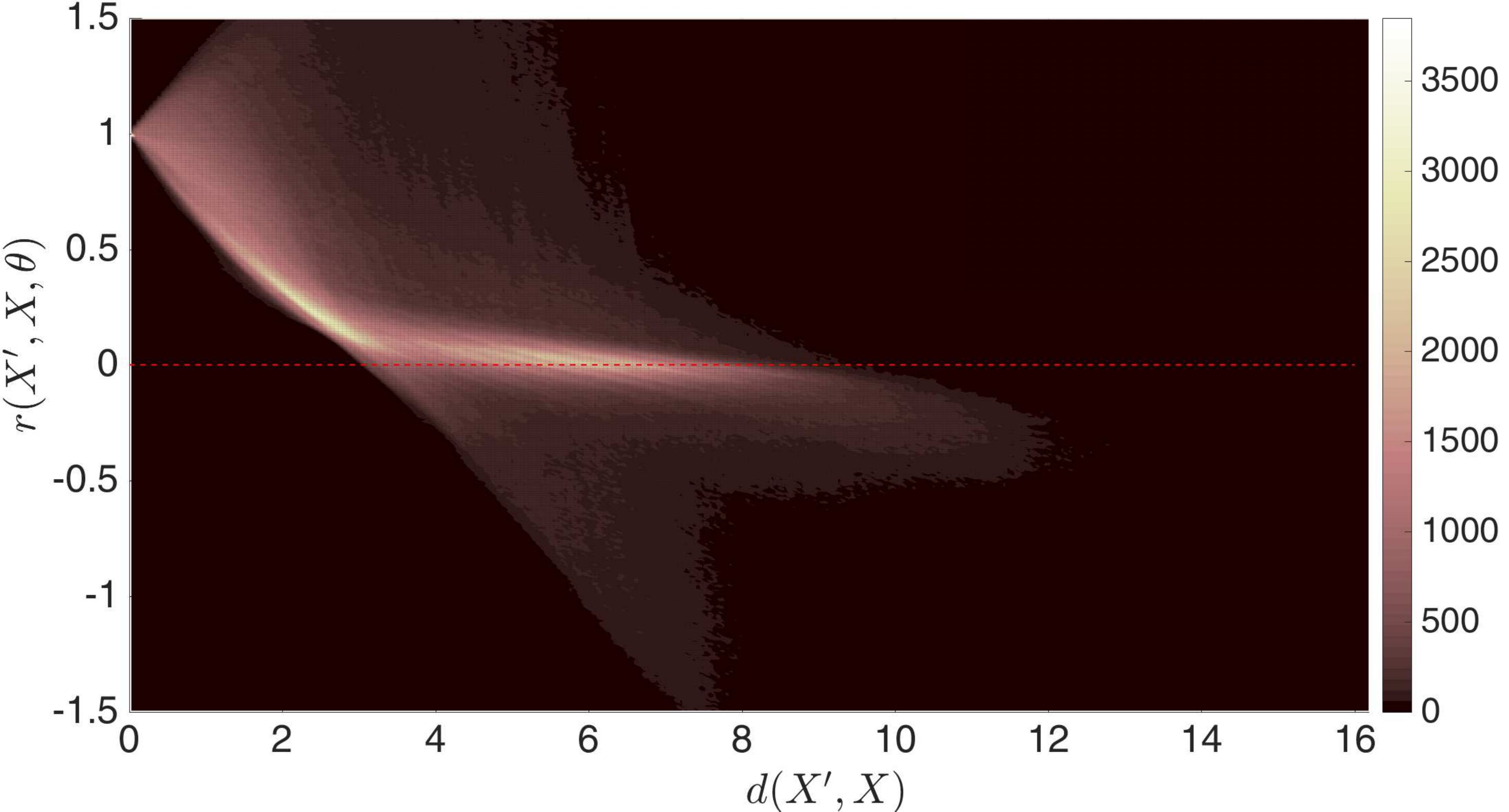}}
		
		&
		
		\subfloat[\label{fig:GaussianOverfittedR_SI-h}]{\includegraphics[height=\height\textheight,width=\width\textwidth]{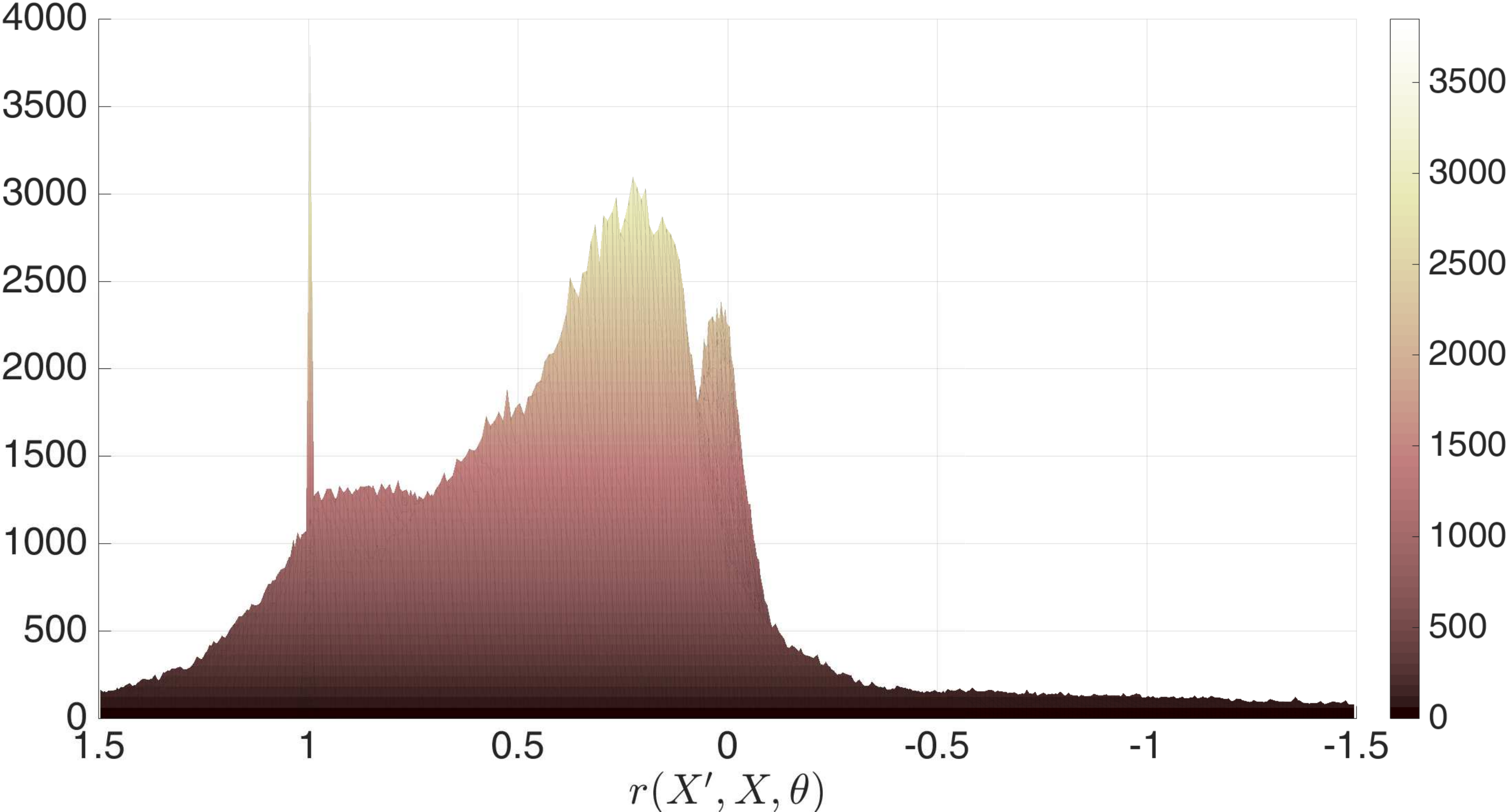}}

	\end{tabular}
	
	\protect
	\caption[Illustration of NN flexibility and the corresponding bandwidth of $g_{\theta}(X, X')$, for each model in Figure \ref{fig:GaussianOverfittedR}.]{Illustration of NN flexibility and the corresponding bandwidth of $g_{\theta}(X, X')$, for each model in Figure \ref{fig:GaussianOverfittedR}. We calculate a \emph{relative} model kernel $r_{\theta}(X_i, X_j)$ and a Euclidean distance $d(X_i, X_j)$ for $9 \cdot 10^6$ point pairs and depict a histogram of obtained $\{ r_{\theta}(X_i, X_j) \}$ and $\{ d(X_i, X_j) \}$ in left column. Likewise, a side view of this histogram is depicted in right column. Number of layers within NN is (a)-(b) $5$, (c)-(d) $4$, (e)-(f) $3$ and (g)-(h) $2$. See more details in the main text.
	}
	\label{fig:GaussianOverfittedR_SI}
\end{figure}

As was described in \citep{Kopitkov18arxiv}, the accuracy of the estimated density can be very low for a small dataset setting, since the converged surface can be flat with several spikes at locations of available training data points. This is due the fact that apparently we estimate the empirical density of data which in case of sparse datasets can be represented as a flat surface with several peaks. If the used model $f_{\theta}(X)$ is overly flexible and is not properly regularized, it can be indeed pushed to such spiky form, as was proved by Theorem \ref{thrm:PSO_peak_convergence} and as we empirically demonstrate below.

According to Section \ref{sec:ExprrKernelBnd}, the flexibility of the surface $f_{\theta}$ can be expressed via properties of the model kernel $g_{\theta}$, such as its bandwidth. The kernel acts as a connector of various input space areas, creating the side influence/force between these areas; it is balancing the overall physical force at each input point, with its equilibrium described by the convoluted PSO balance state in Eq.~(\ref{eq:ConvPSOBalStatePopul}).
Further, when the bandwidth of the model kernel is too narrow w.r.t. distance between training points, the \emph{influence} area of any training point $X$ will be some small neighborhood around $X$. Any input point $X'$ outside of all \emph{influence} areas is basically not optimized during the learning - $f_{\theta}(X')$ will mostly not change during such optimization.
Furthermore, the size of available dataset also has its impact, since with more data the distance between training samples decreases (on average) and the volume of overall \emph{influence} area increases.

In Figure \ref{fig:GaussianOverfitted}
the \overfit nature of NN surface is illustrated in an experiment where 2D Gaussian distribution is inferred. When the same network is used and the number of samples is decreasing, the outcome is the spiky surface at the end of the optimization.

Onwards, in Figure \ref{fig:GaussianOverfittedR} we can see the experiment where a small dataset of a size 10000 is used for the same pdf inference, and where the number of used layers is decreased. As observed, with less layers the spiky nature of the surface is decreasing due to the reduced NN flexibility/capacity. Similar behavior is also observed in KDE method in Figure \ref{fig:GaussianOverfittedR_KDE} where the bandwidth of Gaussian kernel is increasing; we can see that the surface estimated via KDE becomes more and more flexible for a smaller kernel \bnd, similarly to what we observed in Figure \ref{fig:GaussianOverfittedR}.
Thus, both KDE and PSO exhibit a similar flexibility behavior when the bandwidth of former is increased and when the layers depth of latter is reduced. Moreover, the reduction of layers in case of PSO produces a similar increase of $g_{\theta}(X, X')$'s bandwidth as we further show.

Particularly, in Figure \ref{fig:GaussianOverfittedR_SI} we present the bandwidth histogram of $g_{\theta}(X, X')$ for each trained model in Figure \ref{fig:GaussianOverfittedR}. We sample $\{ X_i \}_{i = 1}^{3000}$ testing points from 2D Gaussian and calculate \emph{relative} side-influence $r_{\theta}(X_i, X_j)$ defined in Eq.~(\ref{eq:RelKernel}) for each pair of points. Further, for each pair we also compute the Euclidean distance $d(X_i, X_j)$.

For $9 \cdot 10^6$ pairs of a \emph{relative} side-influence $r_{\theta}(X_i, X_j)$ and a Euclidean distance $d(X_i, X_j)$ we construct a histogram in Figure \ref{fig:GaussianOverfittedR_SI}. As observed, the \emph{relative} side-influence is reduced with $d(X_i, X_j)$ - faraway points affect each other on much lower level.
Further, we can see in left column of Figure \ref{fig:GaussianOverfittedR_SI} a sleeve right from a vertical line $d(X', X) = 0$ that implies an existence of overall local-support structure of $r_{\theta}(X_i, X_j)$, and a presence of some implicit kernel \bnd.
Likewise, we can also see a clear trend between $r_{\theta}(X_i, X_j)$ and the number of NN layers. For shallow networks (see Figures \ref{fig:GaussianOverfittedR_SI-g}-\ref{fig:GaussianOverfittedR_SI-h}) the \emph{relative} side-influence is strong even for faraway regions.  In contrast, in deeper networks (see Figures \ref{fig:GaussianOverfittedR_SI-a}-\ref{fig:GaussianOverfittedR_SI-b}) $r_{\theta}(X_i, X_j)$ is centered around zero for a pair of faraway points, with some close by points having non-zero side-influence. Hence, we can see the obvious relation between the network depth, the bandwidth of model kernel and the model flexibility, which supports conclusions made in Section \ref{sec:ExprrKernelBnd}.

Furthermore, since the impact of a kernel bandwidth is similar for PSO and KDE, we can compare our conclusions with well-known properties of KDE methods.
For instance, optimality of the KDE bandwidth was already investigated in many works \citep{Duong05sjs, Heidenreich13asta, OBrien16csda, Silverman18} and is known to strongly depend on the number of training data samples. Hence, this implies that the optimal/"desired" bandwidth of $g_{\theta}(X, X')$ also depends on the size of training dataset, which agrees with statements of Section \ref{sec:ExprrKernelBnd}.

\begin{remark}
In Figure \ref{fig:GaussianOverfittedR_SI-a} we can see that the side-influence between most points is zero, implying that gradients $\nabla_{\theta} f_{\theta}(X_i)$ at different points tend to be orthogonal for a highly flexible model. This gradient orthogonality does not present at NN initialization, and there is some mechanism that enforces it during NN training as shown in Appendix \ref{sec:UncorRes}. The nature of this mechanism is currently unknown.
\end{remark}

\subsection{Possible Solutions}
\label{sec:DeepLogPDFOFISol}

How big the training dataset should be and how to control NN flexibility to achieve the best performance are still open research questions. Some insights can be taken from KDE domain, yet we shall leave such analysis for future investigation.
Further, the over-flexibility issue yields a significant challenge for application of PSO density estimators on small datasets, as well as also for other PSO instances. However, there are relatively simple regularization methods to reduce such \overfit and to eliminate peaks from the converged surface $f_{\theta}$.

The first method is to introduce a weight regularization term into the loss, such as L2 norm of $\theta$. This will enforce the weight vector to be inside a ball in the parameter space, thus limiting the flexibility of the NN. Yet, it is unclear what is the exact impact of any specific weight regularization method on the final surface and on properties of $g_{\theta}(X, X')$, due to highly non-linear nature of modern deep models. Typically, this regularization technique is used in try-and-fail regime, where different norms of $\theta$ and regularization coefficients are applied till a good performance is achieved.

Another arguably more consistent method is data augmentation, which is highly popular in Machine Learning. In context of PSO and its gradient in Eq.~(\ref{eq:GeneralPSOLossFrml}), we can consider to introduce an additive noise into each sample $X^{\usuff}_{i}$ as:
\begin{equation}
\bar{X}^{\usuff}_{i} = X^{\usuff}_{i} + \upsilon
\label{eq:DataAugm}
\end{equation}
where $X^{\usuff}_{i}$ is the original sample from the data density $\probi{\usuff}{X}$ and $\upsilon$ is a random noise sampled from some density $\probi{\upsilon}{X}$ (e.g. Gaussian distribution). When using $\bar{X}^{\usuff}_{i}$ instead of $X^{\usuff}_{i}$, we will actually estimate the density of the random variable $\bar{X}^{\usuff}_{i}$ which is the convolution between two densities. Thus, for PSO-LDE the converged surface $f_{\theta}(X)$ will be:
\begin{equation}
f_{\theta^*}(X) = \log
\left(
[\probs{\usuff} * \probs{\upsilon}](X)
\right),
\label{eq:DataAugmConv}
\end{equation}
where $*$ defines the convolution operator.

Considering $\probi{\upsilon}{X}$ to be Gaussian and recalling that it is a solution of the \emph{heat equation} \citep{Beck92book}, the above expression elucidates the effect of such data augmentation as a simple \emph{diffusion} of the surface that would be estimated for the original $X^{\usuff}_{i}$. That is, assuming that  $f_{\theta}(X)$ would get a spiky form when approximating $\log \probi{\usuff}{X}$, the updated target density function (and thus also its approximation $f_{\theta}$) undergoes \emph{diffusion} in order to yield a \emph{smoother} final surface. In case of Gaussian noise, the smoothness depends on its covariance matrix. Yet, the other distributions can be used to perform  appropriate convolution and to achieve different \emph{diffusion} effects. We employ the above technique to improve an inference accuracy under a small training dataset setting in Section \ref{sec:ColumnsEstSmallDT}.

\begin{remark}
Additionally, in context of image processing, a typical data augmentation involves image flipping, resizing and introducing various photographic effects \citep{Wong16arxiv,Perez17arxiv}. Such methods produce new samples $\bar{X}^{\usuff}_{i}$ that are still assumed to have the original density $\probi{\usuff}{X}$, which can be justified by our prior knowledge about the space of all possible images. Given this knowledge is correct, the final estimation is still of $\probi{\usuff}{X}$ and not of its convolution (or any other operator) with the noise.
\end{remark}

\section{Experimental Evaluation}
\label{sec:Exper}

Below we report several experimental scenarios that demonstrate the efficiency of the proposed PSO algorithm family.
Concretely, in Section \ref{sec:ColumnsEst} we apply PSO to infer a pdf of 20D \emph{Columns} distribution, where in sub-section \ref{sec:ColumnsEstME} we compare between various PSO instances; in \ref{sec:ColumnsEstBslns} we experiment with state-of-the-art baselines and compare their accuracy with PSO-LDE; in \ref{sec:ColumnsEstNNAME} we evaluate the pdf inference performance for different NN architectures; 
in \ref{sec:BatchSizeI} we investigate the impact of a batch size on PSO performance;
and in \ref{sec:ColumnsEstSmallDT} we show how different sizes of training dataset affect  inference accuracy and explore different techniques to overcome difficulties of a small dataset setting. Furthermore, in Section \ref{sec:TrColumnsEst} we perform pdf inference over a more challenging distribution \emph{Transformed Columns} and in Section \ref{sec:ImageEst} we apply our pdf estimation approach over 3D densities generated from pixel landscape of RGB images. 
Additionally, in Appendix \ref{sec:DiffApr} we show that the first-order Taylor approximation of the surface \df in Eq.~(\ref{eq:PSODffrntl}) is actually very accurate in practice, and in Appendix \ref{sec:UncorRes} we empirically explore dynamics of $g_{\theta}(X, X')$ and of its bandwidth during a learning process.

Importantly, our main focus in this paper is to introduce a novel paradigm for inferring various statistics of an arbitrary data in a highly accurate and consistent manner. To this end and concretely in context of density estimation, we are required to demonstrate \textbf{quantitatively} that the converged approximation $\bar{\PP}_{\theta}(X)$ of the pdf function $\PP(X)$ is indeed very close to its target. Therefore, we avoid experiments on  real datasets (e.g. MNIST, \citet{LeCun98ieee}) since they lack  information about the true pdf values of the samples. Instead, we generate datasets for our experiments from analytically known pdf functions, which allows us to evaluate the \emph{ground truth} error between $\bar{\PP}_{\theta}(X)$ and $\PP(X)$. However, all selected pdf functions are highly multi-modal and therefore are very challenging to infer.

Likewise, in this work we purposely consider vector datasets instead of image data, to decouple our main approach from complexities coming with images and CNN models. Our main goal is to solve general unsupervised learning, and we do not want it to be biased towards spatial data. Moreover, vector data is mostly neglected in modern research and our method together with the new BD architecture addresses this gap.

\subsection{Learning Setup}
\label{sec:ColumnsEstLS}

All the pdf inference experiments were done using Adam optimizer \citep{Kingma14arxiv}, since Adam showed better convergence rate compared to stochastic GD. Note that replacing GD with Adam does not change the target function approximated by PSO estimation, although it changes the implicit model kernel. That is, the variational PSO \bp in Eq.~(\ref{eq:BalPoint}) stays the same while the convoluted equilibrium described in Section \ref{sec:ConvPSOBalState} will change according to the model kernel associated with Adam update rule.

The used Adam hyper-parameters are $\beta_1 = 0.75$, $\beta_2 = 0.999$ and $\epsilon = 10^{-10}$. Each experiment optimization is performed for 300000 iterations, which typically takes about one hour to run on a GeForce GTX 1080 Ti GPU card. The batch size is $N^{\usuff}= N^{\dsuff} = 1000$. During each iteration next batch of \up points $\{X^{\usuff}_{i}\}_{i = 1}^{1000}$ is retrieved from the training dataset of size $N_{DT}$, and next batch of \down points $\{X^{\dsuff}_{i}\}_{i = 1}^{1000}$ is sampled from \down density $\probs{\dsuff}$. For \emph{Columns} distribution in Section \ref{sec:ColumnsEst} we use a Uniform distribution as $\probs{\dsuff}$. Next, the optimizer updates the weights vector $\theta$ according to the loss gradient in Eq.~(\ref{eq:GeneralPSOLossFrml}), where the \emph{magnitude} functions are specified by a particular PSO instance. The applied learning rate is 0.0035. We keep it constant for first 40000 iterations and then exponentially decay it down to a minimum learning rate of $3\cdot10^{-9}$. Further, in all our models we use Leaky-Relu as a non-linearity activation function. Additionally, weights are initialized via popular Xavier initialization \citep{Glorot10aistats}. Each model is learned 5 times; we report its mean accuracy and the standard deviation. Further, PSO implementation based on \citet{Tensorflow_url} framework can be accessed via open source library \url{\psorepo}.

To evaluate performance and consistency of each learned model, we calculate three different errors over testing dataset $\{X^{\usuff}_{i}\}_{i = 1}^{N}$, where each point was sampled from $\probs{\usuff}$ and $N$ is $10^5$. First one is pdf squared error $PSQR = \frac{1}{N} \sum_{i = 1}^{N} \left[ \probi{\usuff}{X^{\usuff}_{i}} - \bar{\PP}_{\theta}(X^{\usuff}_{i}) \right]^2$, with $\bar{\PP}_{\theta}(\cdot)$ being the pdf estimator produced by a specific model after an optimization convergence. Further, since we deal with high-dimensional data, $PSQR$ involves operations with very small numbers. To prevent inaccuracies caused by the computer precision limit, the second used error is log-pdf squared error $LSQR = \frac{1}{N} \sum_{i = 1}^{N} \left[ \log \probi{\usuff}{X^{\usuff}_{i}} - \log \bar{\PP}_{\theta}(X^{\usuff}_{i}) \right]^2$. Since in this paper we target $\log \prob{\cdot}$ in the first place, the $LSQR$ error expresses a distance between the data log-pdf and the learned NN surface $f_{\theta}(\cdot)$. Moreover, $LSQR$ is related to statistical divergence between $\probs{\usuff}$ and $\bar{\PP}_{\theta}$ (see more details in Appendix \ref{sec:LSQRDivSec}).

Further, the above two errors require to know ground truth $\probs{\usuff}$ for their evaluation. Yet, in real applications such ground truth is not available. As an alternative, we can approximate \psofunc in Eq.~(\ref{eq:GeneralPSOLossFrmlPr_Limit_f}) over testing dataset. As explained in Section \ref{sec:PSO_div}, this loss is equal to \psodiv up to an additive constant, and thus can be used to measure a discrepancy between the PSO-optimized model and the target function. However, for most of the below applied PSO instances $L_{PSO}(f)$ is not analytically known and hence can not be computed. To overcome this problem and to measure model performance in real applications, we propose to use the loss of IS method from Table \ref{tbl:PSOInstances1}, $IS = - \frac{1}{N} \sum_{i = 1}^{N} f_{\theta}(X^{\usuff}_{i})
+ 
\frac{1}{N} \sum_{i = 1}^{N}
\frac{\exp [f_{\theta}(X^{\dsuff}_{i})]}{\probi{\dsuff}{X^{\dsuff}_{i}}}$, where $\{X^{\dsuff}_{i}\}_{i = 1}^{N}$ are i.i.d. samples from $\probs{\dsuff}$. As we will see, while $IS$ is less accurate than the ground truth errors, it still is a reliable indicator for choosing the best member from a set of learned models. Additionally, during the optimization $IS$ is correlated with the real error and if required can be used to monitor current convergence and to allow an early stop evaluation.

\begin{remark}
	Note that unlike typical density estimator evaluation, herein we do not use performance metrics such as perplexity \citep{Jelinek77} and various kinds of $f$-divergences, or negative-likelihood scores. This is because the PSO-learned models are only approximately normalized, while the aforementioned metrics typically require strictly normalized models for their metric consistency. Still, both $PSQR$ and $LSQR$ are mean squared errors between target and approximation functions, and are similar to other performance metrics that are widely applied in regression problems of Machine Learning domain. 
\end{remark}

\subsection{PDF Estimation via PSO - \emph{Columns} Distribution}
\label{sec:ColumnsEst}

\begin{figure}
	\centering
	
	\begin{tabular}{cccc}
		
		\subfloat[\label{fig:ColsRes1-a}]{\includegraphics[width=0.45\textwidth]{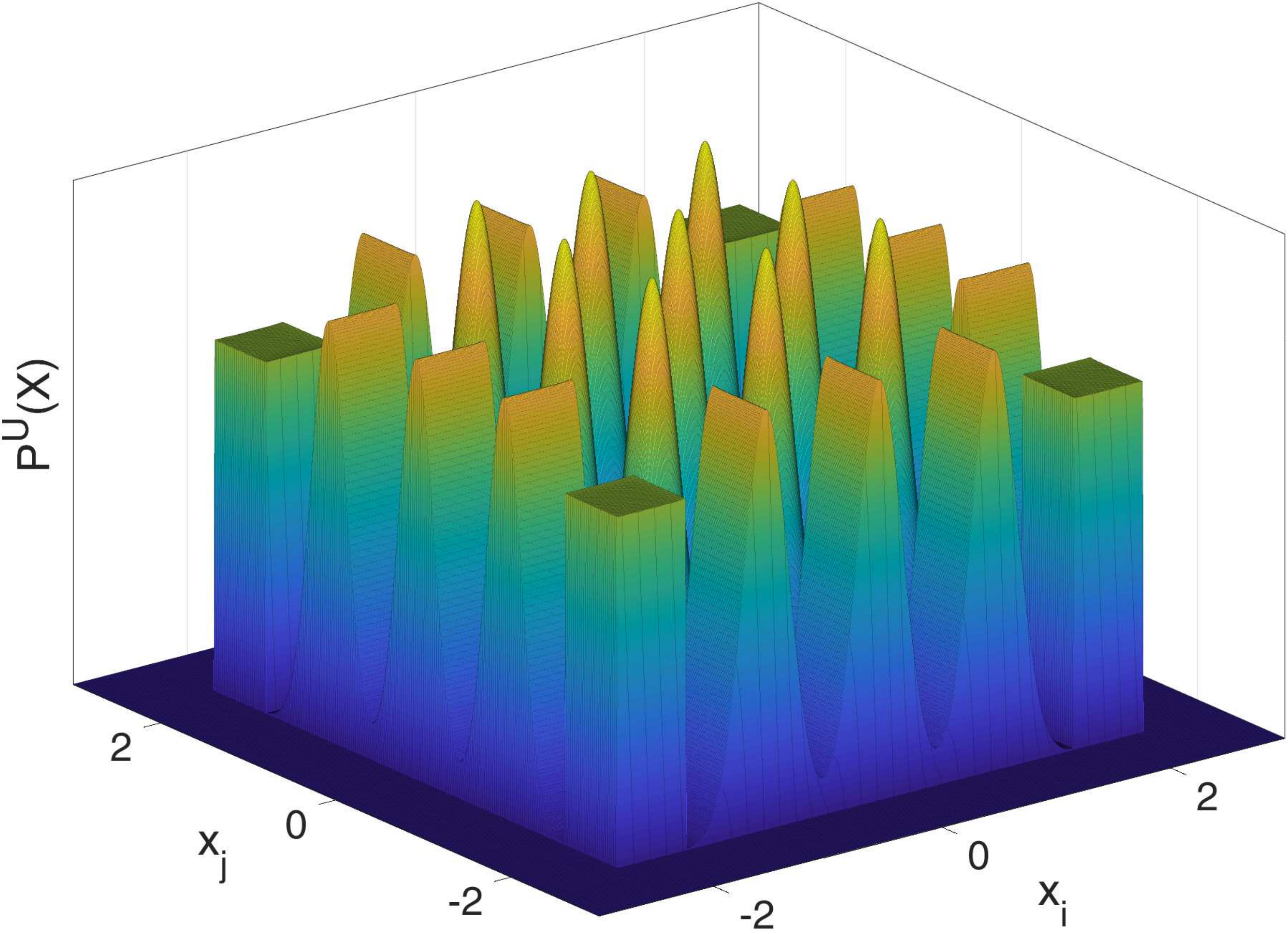}}
		&
		
		\subfloat[\label{fig:ColsRes1-b}]{\includegraphics[width=0.45\textwidth]{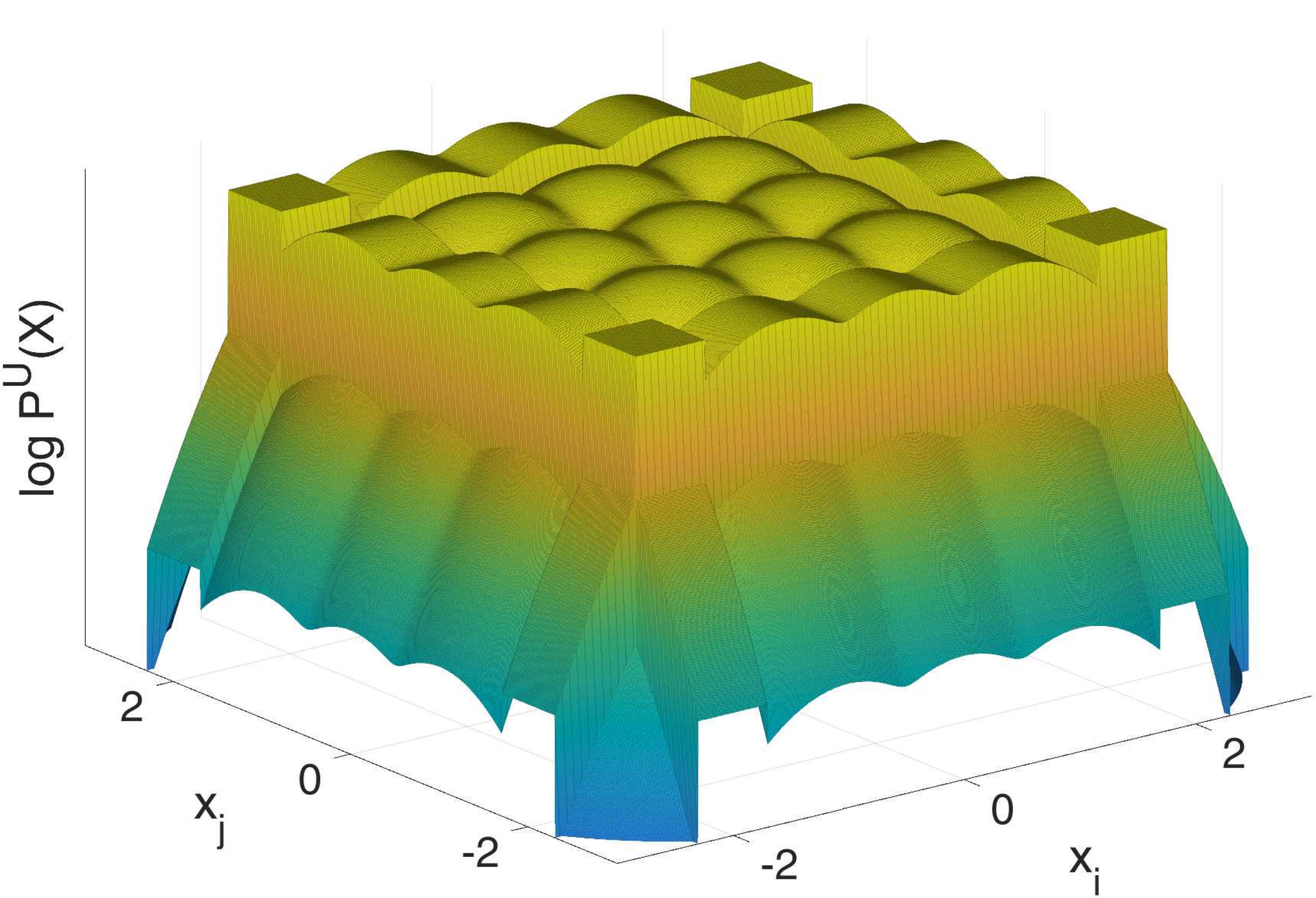}}

	\end{tabular}
	
	\protect
	\caption[Illustration of \emph{Columns} distribution.]{(a) Illustration of \emph{Columns} distribution. Every slice of its pdf function $\PP(x_i, x_j) = \probi{\usuff}{0, \ldots , 0, x_{i}, 0, \ldots , 0, x_{j}, 0, \ldots, 0}$ in Eq.~(\ref{eq:ColumnsDef}) contains 25 modes of different shape. Overall, this distribution has $5^{20}$ modes.
	(b) Logarithm of pdf slice in (a) that will be learned via NN surface.
	}
	\label{fig:ColsRes1}
\end{figure}

In this Section we will infer a 20D \emph{Columns} distribution from its sampled points, using various PSO instances and network architectures. The target pdf here is $\probi{\usuff}{X} = \probi{Clmns}{X}$ and is defined as:
\begin{equation}
\probi{Clmns}{x_1, \ldots , x_{20}}
=
\prod_{i = 1}^{20}
p(x_i)
,
\label{eq:ColumnsDef}
\end{equation}
where $p(\cdot)$ is a 1D mixture distribution with 5 components
$\{Uniform(-2.3, -1.7), \mathcal{N}(-1.0,$ $std = 0.2), \mathcal{N}(0.0, std = 0.2), \mathcal{N}(1.0, std = 0.2), Uniform(1.7, 2.3) \}$; each component has weight 0.2. This distribution has overall $5^{20} \approx 9.5 \cdot 10^{13}$ modes, making the structure of its entire pdf surface very challenging to learn. For the illustration see Figure \ref{fig:ColsRes1-a}.

First, we evaluate the proposed density estimation methods under the setting of infinite training dataset, with number of overall training points being $N_{DT} = 10^8$. Later, in Section \ref{sec:ColumnsEstSmallDT} we will investigate how a smaller dataset size affects the estimation accuracy, and propose various techniques to overcome issues of sparse data scenario.

\subsubsection{PSO Instances Evaluation}
\label{sec:ColumnsEstME}

\begin{figure}[tb]
	\centering
	
	\newcommand{\width}[0] {0.4}
	\newcommand{\height}[0] {0.17}
	
	\begin{tabular}{cc}

		\subfloat[\label{fig:ColsRes2-a}]{\includegraphics[height=\height\textheight,width=\width\textwidth]{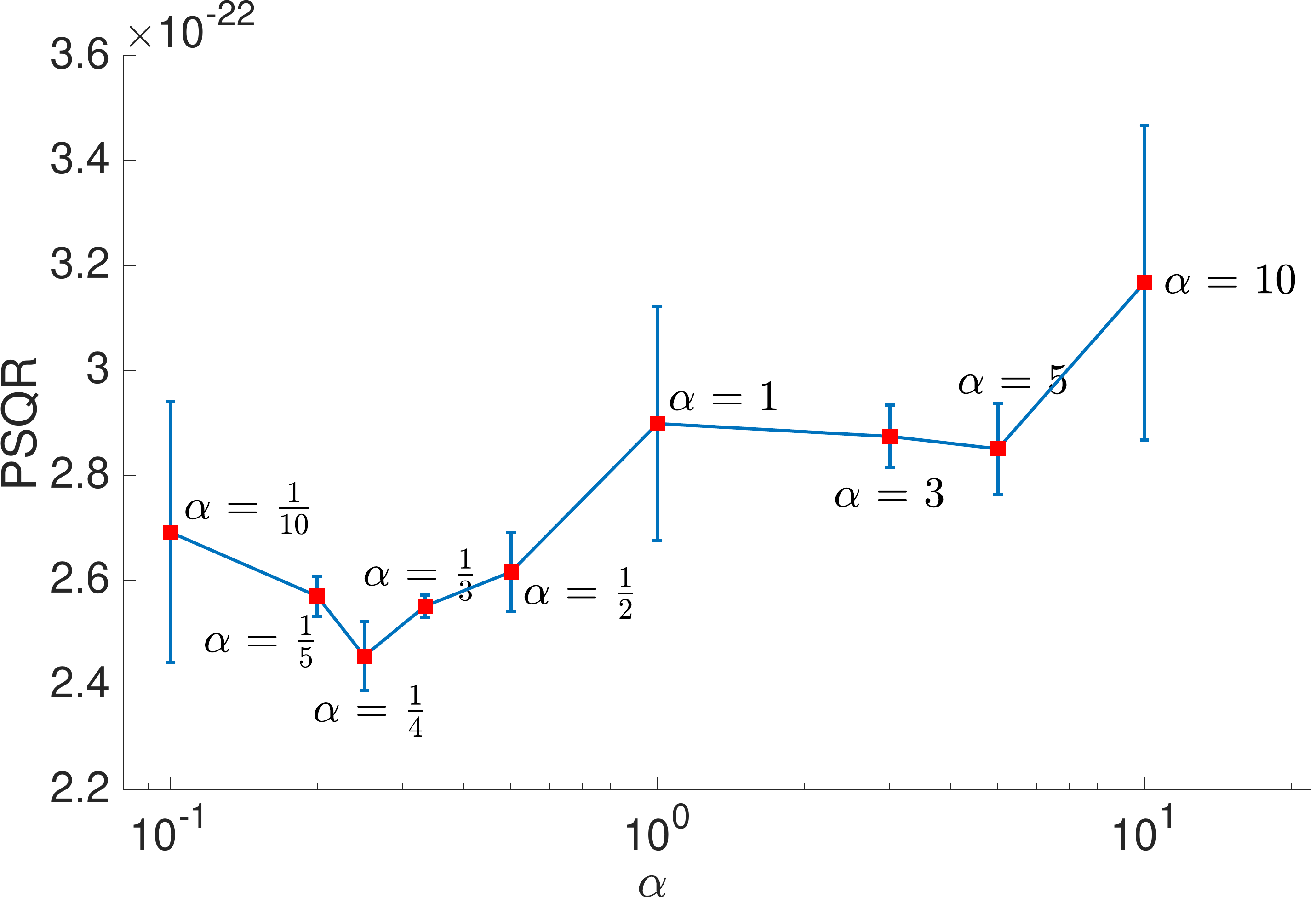}}
		&
		
		\subfloat[\label{fig:ColsRes2-b}]{\includegraphics[height=\height\textheight,width=\width\textwidth]{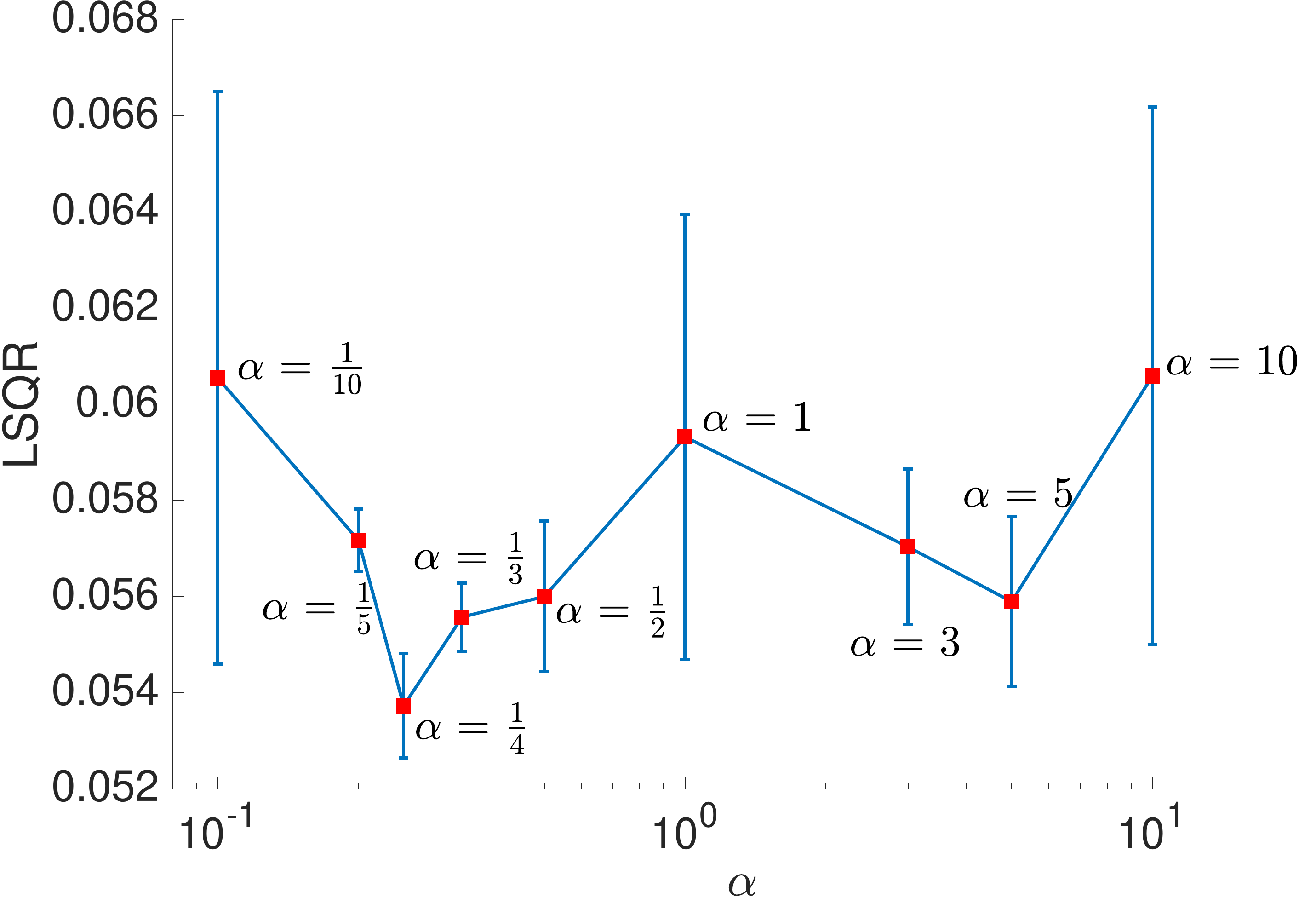}}
		\\
		\subfloat[\label{fig:ColsRes2-c}]{\includegraphics[height=\height\textheight,width=\width\textwidth]{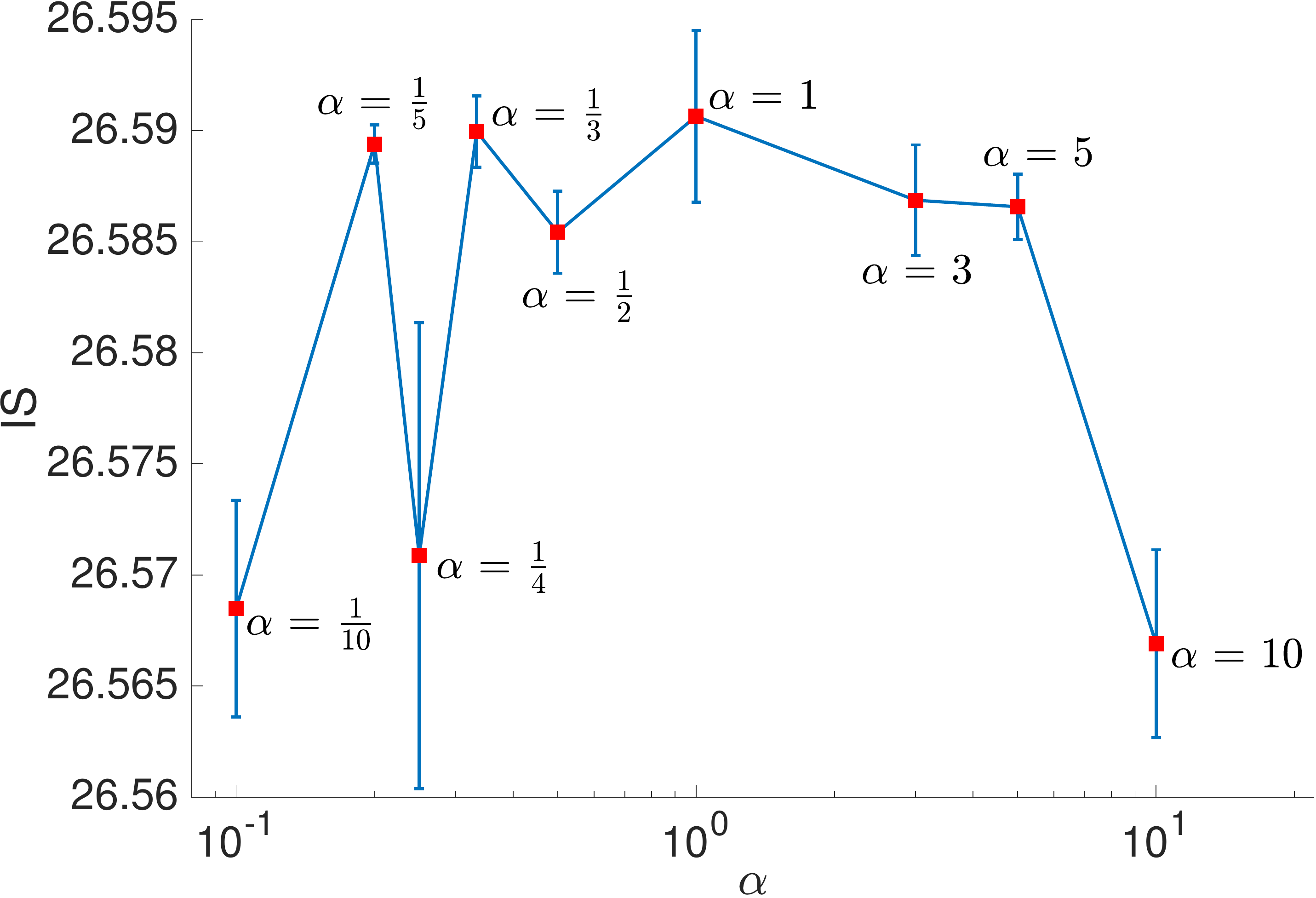}}
		&
		\subfloat[\label{fig:ColsRes2-d}]{\includegraphics[height=\height\textheight,width=\width\textwidth]{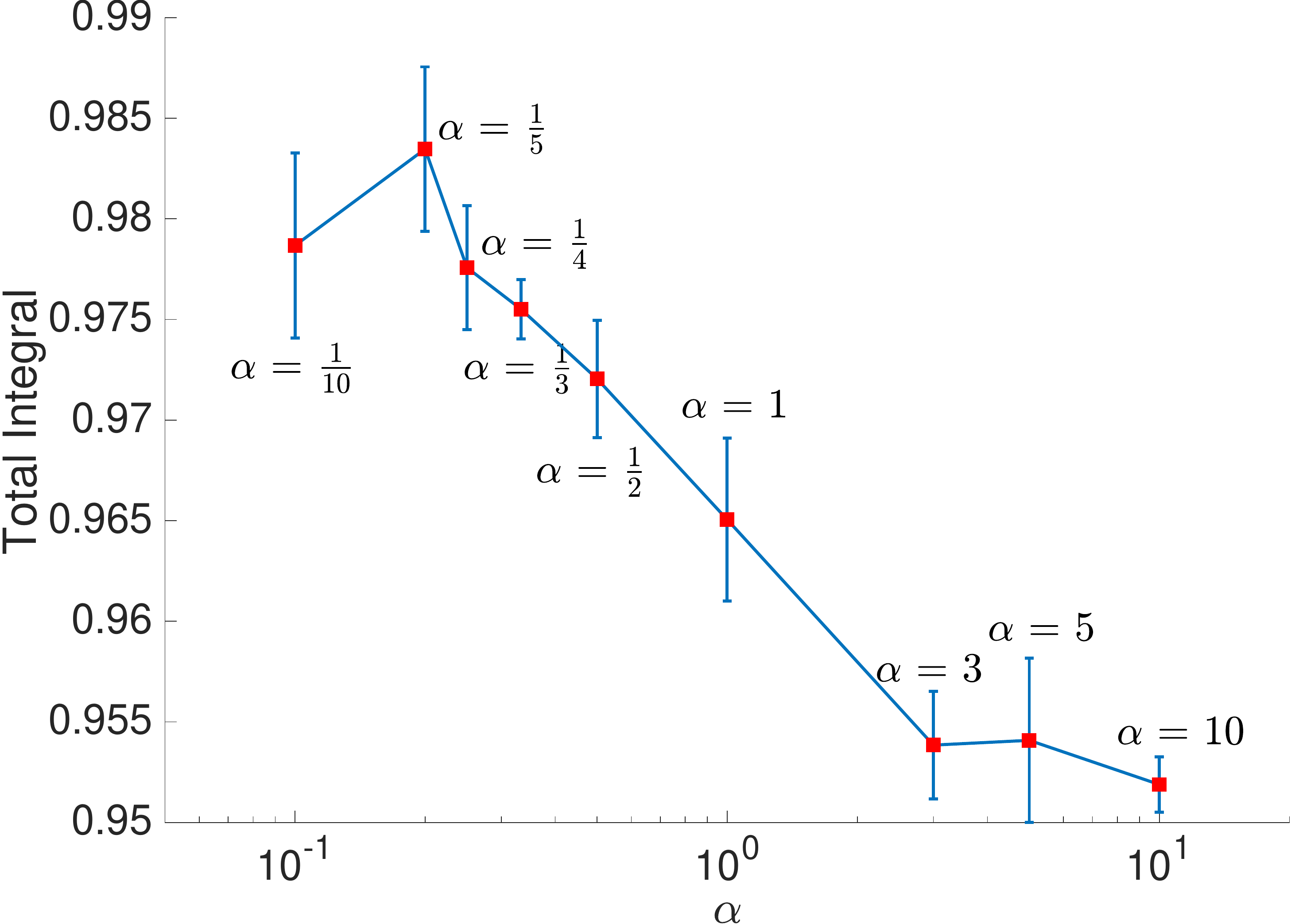}}

	\end{tabular}
	
	\protect
	\caption[Evaluation of PSO-LDE for estimation of \emph{Columns} distribution, for different values of a hyper-parameter $\alpha$.]{Evaluation of PSO-LDE for estimation of \emph{Columns} distribution, where NN architecture is block-diagonal with 6 layers, number of blocks $N_B = 50$ and block size $S_B = 64$ (see Section \ref{sec:BDLayers}). For different values of a hyper-parameter $\alpha$, (a) $PSQR$, (b) $LSQR$ and (c) $IS$  are reported, along with their empirical standard deviation. 
	(d) Estimators of the total integral $TI = \int \bar{\PP}_{\theta}(X) dX$ of learned models for each value of $\alpha$. For a specific learned model $\bar{\PP}_{\theta}(X)$ this integral is estimated through importance sampling as $\overline{TI} = \sum_{i = 1}^{N} \frac{\bar{\PP}_{\theta}(X^{\dsuff}_{i})}{\probi{\dsuff}{X^{\dsuff}_{i}}}$ over $N = 10^8$ samples from \down density $\probs{\dsuff}$. Note that such estimator is consistent, with $TI = \overline{TI}$ for $N \rightarrow \infty$.
	}
	\label{fig:ColsRes2}
\end{figure}

Here we perform pdf learning using different PSO instances, and compare their performance.
The applied NN architecture is block-diagonal from Section \ref{sec:BDLayers}, with 6 layers, number of blocks $N_B = 50$ and block size $S_B = 64$.

\paragraph{PSO-LDE and \boldmath$\alpha$}

First, we apply the PSO-LDE instances from Section \ref{sec:DeepLogPDF}, where we try various values for the hyper-parameter $\alpha$. In Figure \ref{fig:ColsRes2} we can see all three errors for different $\alpha$. All models produce highly accurate pdf estimation, with average $LSQR$ being around 0.057. That is, the learned NN surface $f_{\theta}(X)$ is highly close to the target $\log \probi{\usuff}{X}$. 
Further, we can see that some $\alpha$ values (e.g. $\alpha = \frac{1}{4}$) produce slightly better accuracy than others. This can be explained by smoother \emph{magnitude} dynamics with respect to logarithm difference $\bar{d}$ from Eq.~(\ref{eq:DDiffDefinition}), that small values of $\alpha$ yield (see also Section \ref{sec:DeepLogPDF}). Note that here $IS$ error is not very correlative with ground truth errors $PSQR$ and $LSQR$ since the accuracy of all models is very similar and $IS$ is not sensitive enough to capture the difference.

Further, we also estimate the total integral $TI = \int \bar{\PP}_{\theta}(X) dX$ for each learned model via importance sampling. In Figure \ref{fig:ColsRes2-d} we can see that the learned models are indeed very close to be \emph{normalized}, with the estimated total integral being on average 0.97 - very close to the proper value 1.
Note that in our experiments the model normalization was not enforced in any explicit way, and PSO-LDE achieved it via an implicit force equilibrium.

Furthermore, in Figure \ref{fig:ColsRes2-d} it is also shown that smaller values of $\alpha$ are more properly \emph{normalized}, which also correlates with the approximation error. Namely, in Figure \ref{fig:ColsRes2-a} same models with smaller $\alpha$ are shown to have a lower error. We argue that further approximation improvement (e.g. via better NN architecture) will also increase the \emph{normalization} quality of produced models.

\begin{figure}[tb]
	\centering
	
	\begin{tabular}{cccc}

		\subfloat[\label{fig:ColsRes3-a}]{\includegraphics[width=0.4\textwidth]{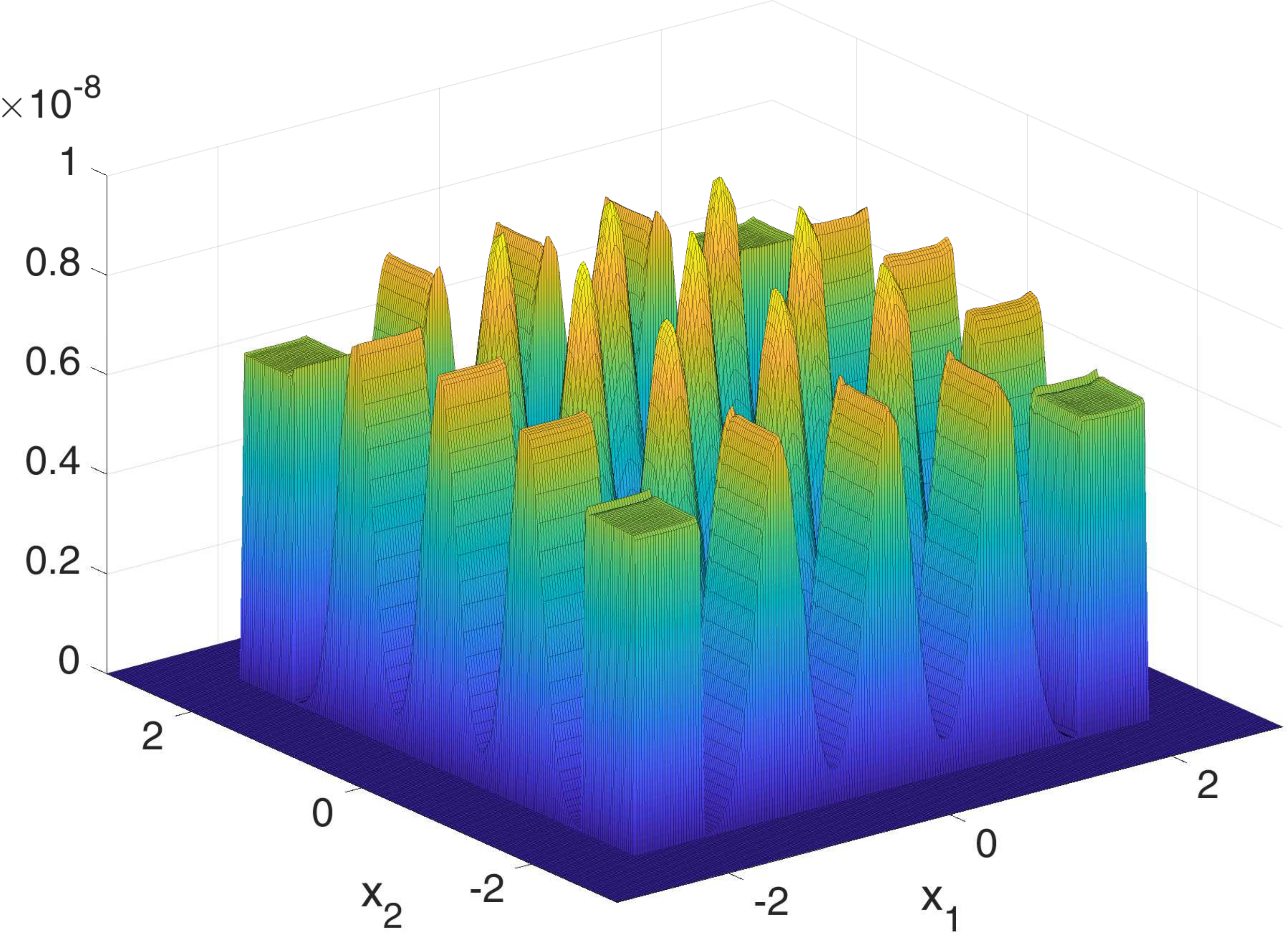}}
		&
		
		\subfloat[\label{fig:ColsRes3-b}]{\includegraphics[width=0.4\textwidth]{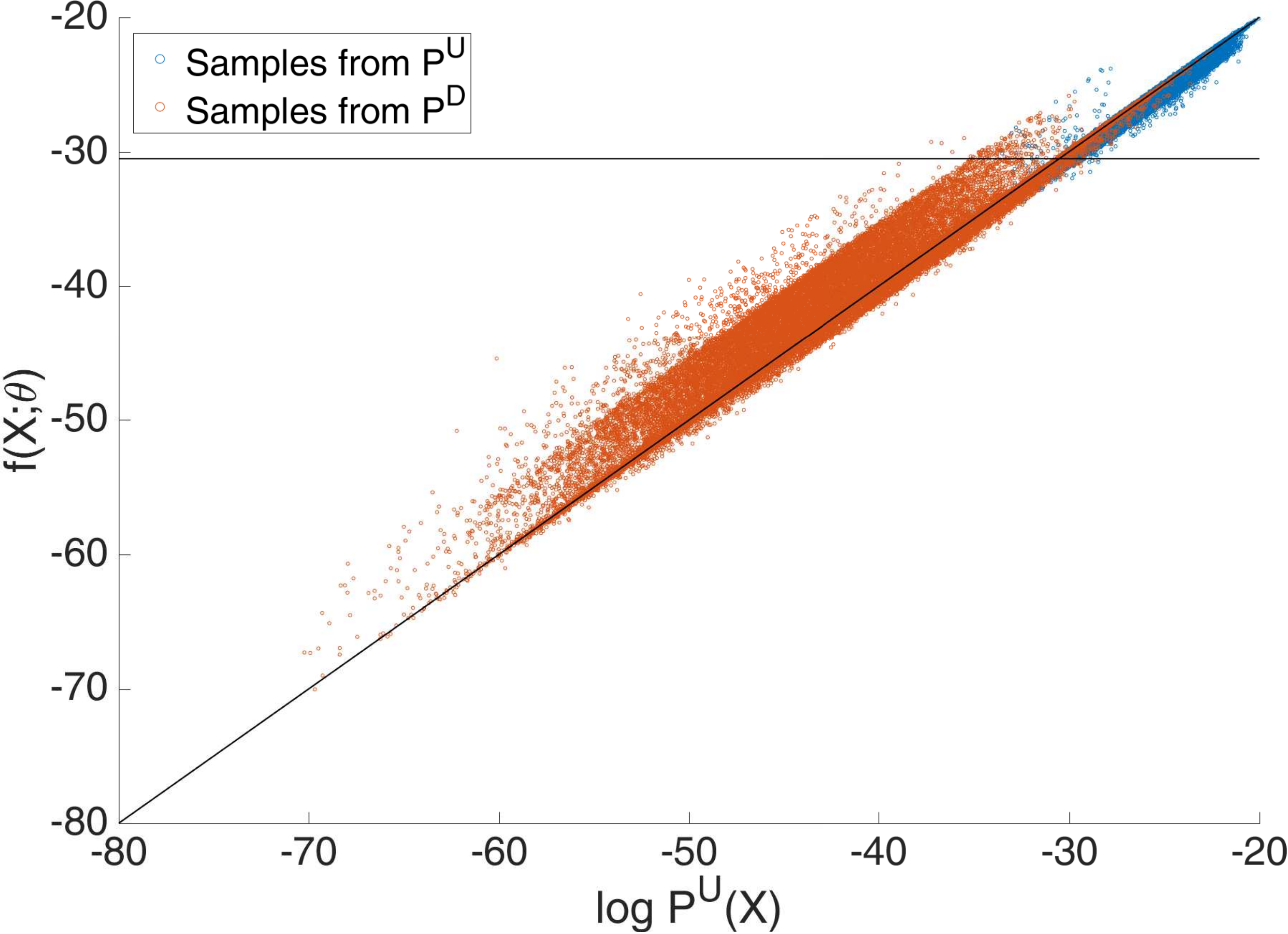}}
	\end{tabular}
	
		\subfloat[\label{fig:ColsRes3-c}]{\includegraphics[width=0.9\textwidth]{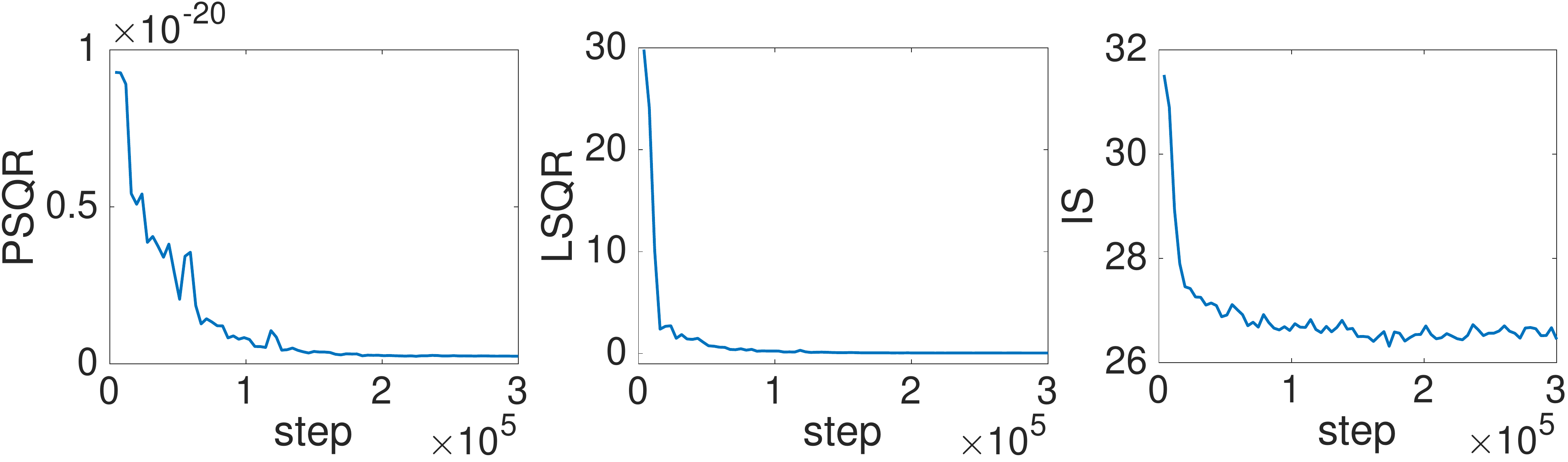}}
		
	\protect
	\caption[Learned pdf function of \emph{Columns} distribution by PSO-LDE with $\alpha = \frac{1}{4}$, where the applied NN is BD.]{Learned pdf function of \emph{Columns} distribution by PSO-LDE with $\alpha = \frac{1}{4}$, where NN architecture is block-diagonal with 6 layers, number of blocks $N_B = 50$ and block size $S_B = 64$ (see Section \ref{sec:BDLayers}).
		(a) Illustration of learned pdf function. The depicted slice is $\PP(x_1, x_2) = \bar{\PP}^{\usuff}(x_1, x_2, 0, \ldots , 0)$, with $x_1$ and $x_2$ forming a grid of points in first two dimensions of the density's support. As can be seen, all modes (within first two dimensions) and their appropriate shapes are recovered.
	(b) Illustration of the learned surface $f_{\theta}(X)$. Blue points are sampled from $\probs{\usuff}$, while red points - from $\probs{\dsuff}$, minimal 20D Uniform distribution that covers all samples from $\probs{\usuff}$. The $x$ axis represents $\log \probi{\usuff}{X}$ for each sample, $y$ axis represents the surface height $f_{\theta}(X)$ after the optimization was finished. The diagonal line represents $f_{\theta}(X) = \log \probi{\usuff}{X}$, where we would see all points in case of \emph{perfect} model inference. The black horizontal line represents $\log \probi{\dsuff}{X} = - 30.5$ which is a constant for the Uniform density. As can be seen, these two densities have a \emph{relative} support mismatch - although their pdf values are not zero within the considered point space, the sampled points from both densities are obviously located mostly in different space neighborhoods. This can be concluded from values of $\log \probi{\usuff}{X}$ that are very different for both point populations. Further, we can see that there are errors at both $X^{\usuff}$ and $X^{\dsuff}$ locations, possibly due to high bias of the surface estimator $f_{\theta}(X)$ imposed by the model kernel $g_{\theta}$ (see also Figure \ref{fig:ColsRes4}).
	(c) Testing errors as functions of the optimization iteration. All three errors can be used to monitor the learning convergence. Further, $IS$ error can be calculated without knowing the ground truth.
	}
	\label{fig:ColsRes3}
\end{figure}

Moreover, in Figure \ref{fig:ColsRes3-a} we can see a slice of the learned $\exp f_{\theta}(X)$ for the first two dimensions, where the applied PSO instance was PSO-LDE with $\alpha = \frac{1}{4}$. As observed, it is highly close to the real pdf slice from Figure \ref{fig:ColsRes1-a}. In particular, all modes and their shapes (within this slice) were recovered during learning. Further, in Figure \ref{fig:ColsRes3-b} we can observe the learned surface height $f_{\theta}(X)$ and the ground truth height $\log \probi{\usuff}{X}$ for test sample points from $\probs{\usuff}$ and $\probs{\dsuff}$. As shown, there are approximation errors at both $X^{\usuff}$ and $X^{\dsuff}$, with \down points having bigger error than \up points. As we will see below, these errors are correlated with the norm of $\theta$ gradient at each point.

Additionally, in Figure \ref{fig:ColsRes3-b} we can see an asymmetry of error w.r.t. horizontal line $\log \probi{\dsuff}{X} = - 30.5$, where points above this line (mostly blue points) have a NN height $f_{\theta}(X)$ slightly lower than a target $\log \probi{\usuff}{X}$, and points below this line (mostly red points) have a NN height $f_{\theta}(X)$ slightly higher than target $\log \probi{\usuff}{X}$. This trend was observed in all our experiments. Importantly, this error must be accounted to an estimation bias (in contrast to an estimation variance), since the considered herein setting is of infinite dataset setting where theoretically the variance is insignificant.

Further, we speculate that the reason for such bias can be explained as follows. The points above this horizontal line have a positive logarithm difference $\bar{d} = f_{\theta}(X) + 30.5$ defined in Eq.~(\ref{eq:DDiffDefinition}), $\bar{d} \geq 0$, whereas points below this line have a negative $\bar{d} \leq 0$. From the relation between $\bar{d}$ and \mfs discussed in Section \ref{sec:DeepLogPDF}, we know that for "above" points the \up \emph{magnitude} $M^{\usuff}(\cdot)$ is on average smaller than \down \emph{magnitude} $M^{\dsuff}(\cdot)$ (see Figure \ref{fig:MagnFuncsPSOLDE}). The opposite trend can be observed within "below" points. There $M^{\dsuff}(\cdot)$ has smaller values relatively to $M^{\usuff}(\cdot)$. Thus, the surface parts above this horizontal line have large \down \emph{magnitudes}, while parts below the line have large \up \emph{magnitudes}, which in its turn creates a global side-influence imposed via the model kernel. Finally, these global side influences generate this asymmetric error with $\log \probi{\dsuff}{X} = - 30.5$ being the center of the pressure. 
Likewise, we argue that this asymmetric tendency can be reduced by selecting $\probs{\usuff}$ and $\probs{\dsuff}$ densities that are closer to each other, as also enhancing NN architecture to be more flexible, with the side-influence between far away regions being reduced to zero. In fact, we empirically observed that NN architectures with bigger side-influence (e.g. FC networks) have a greater error asymmetry; the angle identified in Figure \ref{fig:ColsRes3-b} between a point cloud and line $f_{\theta}(X) = \log \probi{\usuff}{X}$ is bigger for a bigger overall side-influence within the applied model (see also Figure \ref{fig:ColsRes5-b} below). We leave a more thorough investigation of this asymmetry nature for future research.

Further, the above asymmetry also clears out why all learned models in Figure \ref{fig:ColsRes2-d} had the total integral less than 1. Since this integral is calculated by taking exponential over the learned $f_{\theta}(X)$, red points in Figure \ref{fig:ColsRes3-b} almost do not have any impact on it, compared with the blue points (red points' exponential is much lower than exponential of blue points). Yet, blue points have smaller $\exp f_{\theta}(X)$ than their real pdf values $\probi{\usuff}{X}$. Therefore, the total integral comes out to be slightly smaller than 1.

Also, in Figure \ref{fig:ColsRes3-c} we can see all three errors along the optimization time; the $IS$ is shown to monotonically decrease, similarly to ground truth errors. Hence, in theory it can be used in real applications where no ground truth is available, to monitor the optimization convergence.

\paragraph{Point-wise Error}

\begin{figure}[t!]
	\centering
	
	\newcommand{\width}[0] {0.44}
	\newcommand{\height}[0] {0.12}
	
	\begin{tabular}{cc}

		\subfloat[\label{fig:ColsRes4-a}]{\includegraphics[height=\height\textheight,width=\width\textwidth]{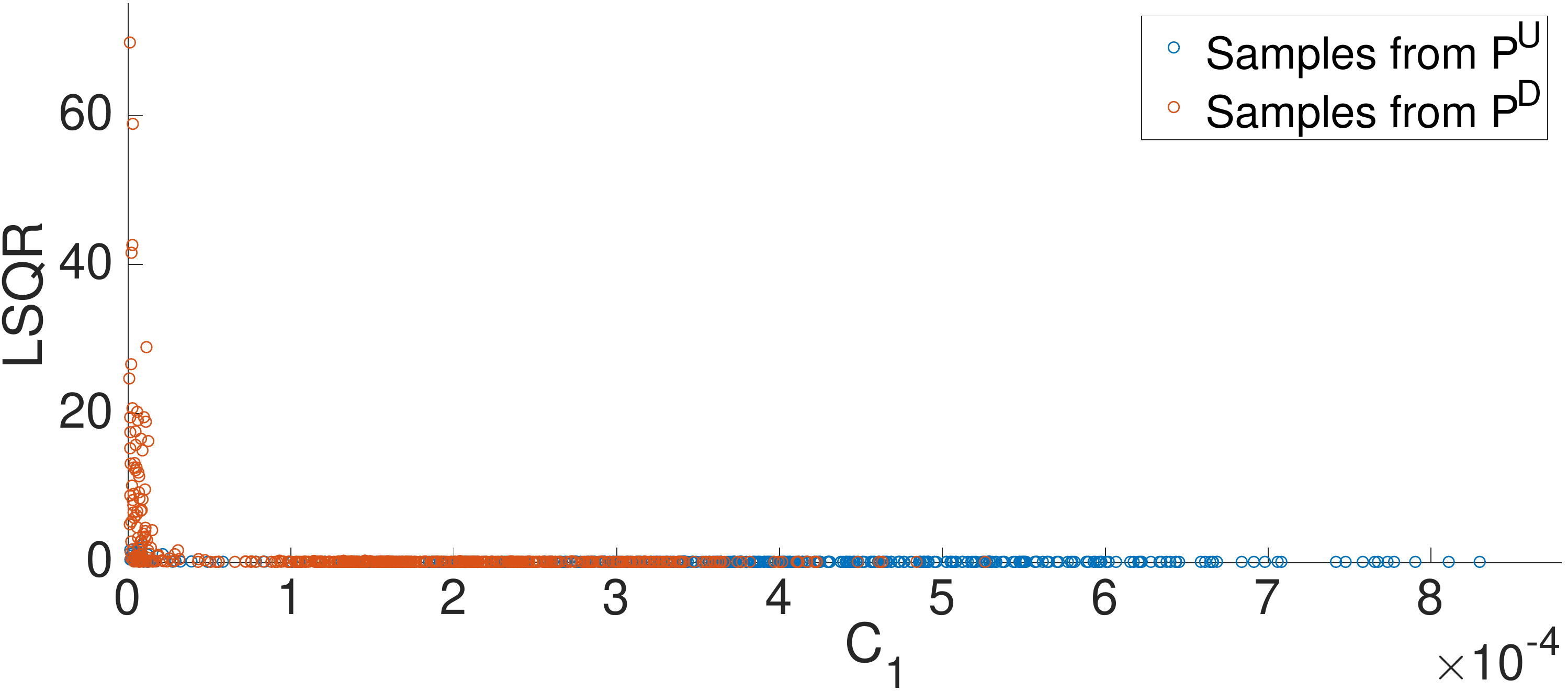}}
		&
		
		\subfloat[\label{fig:ColsRes4-b}]{\includegraphics[height=\height\textheight,width=\width\textwidth]{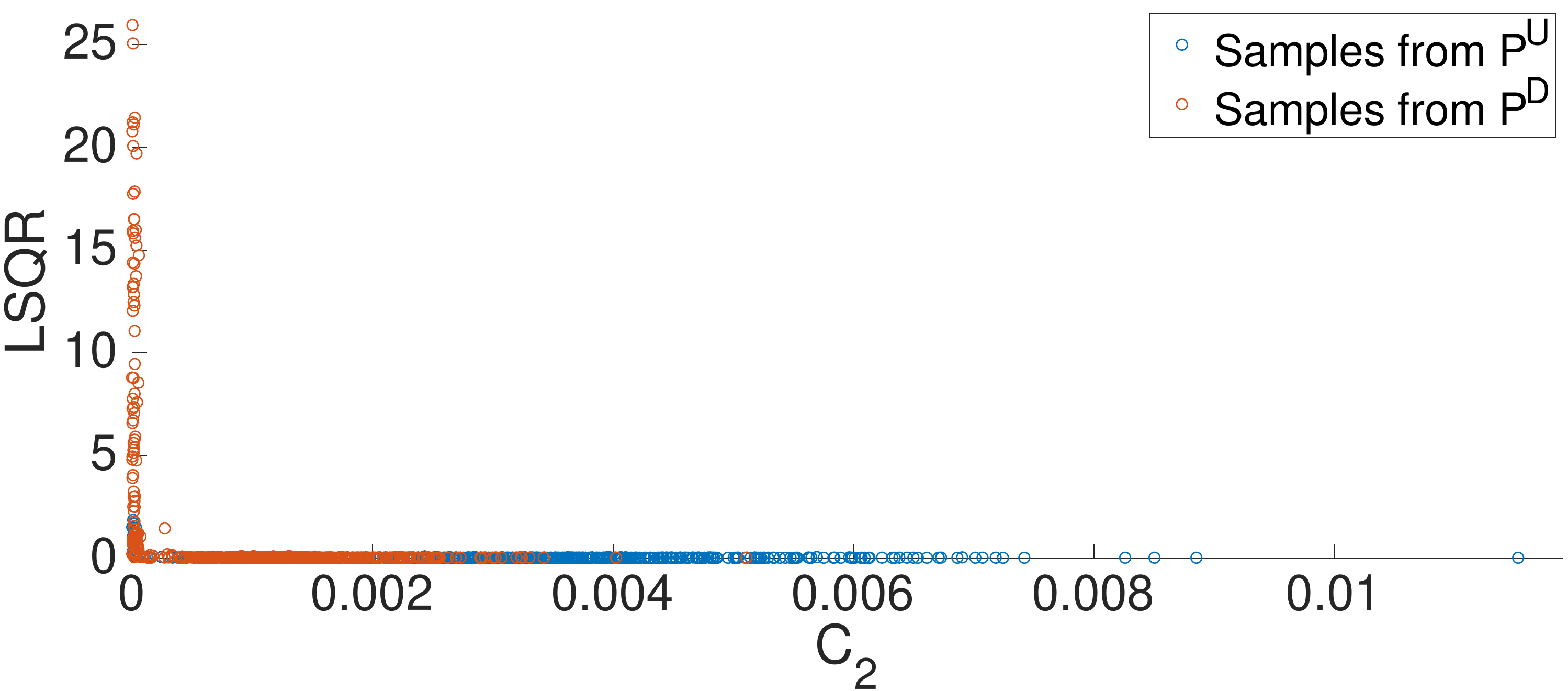}}
		
		\\
		\subfloat[\label{fig:ColsRes4-c}]{\includegraphics[height=\height\textheight,width=\width\textwidth]{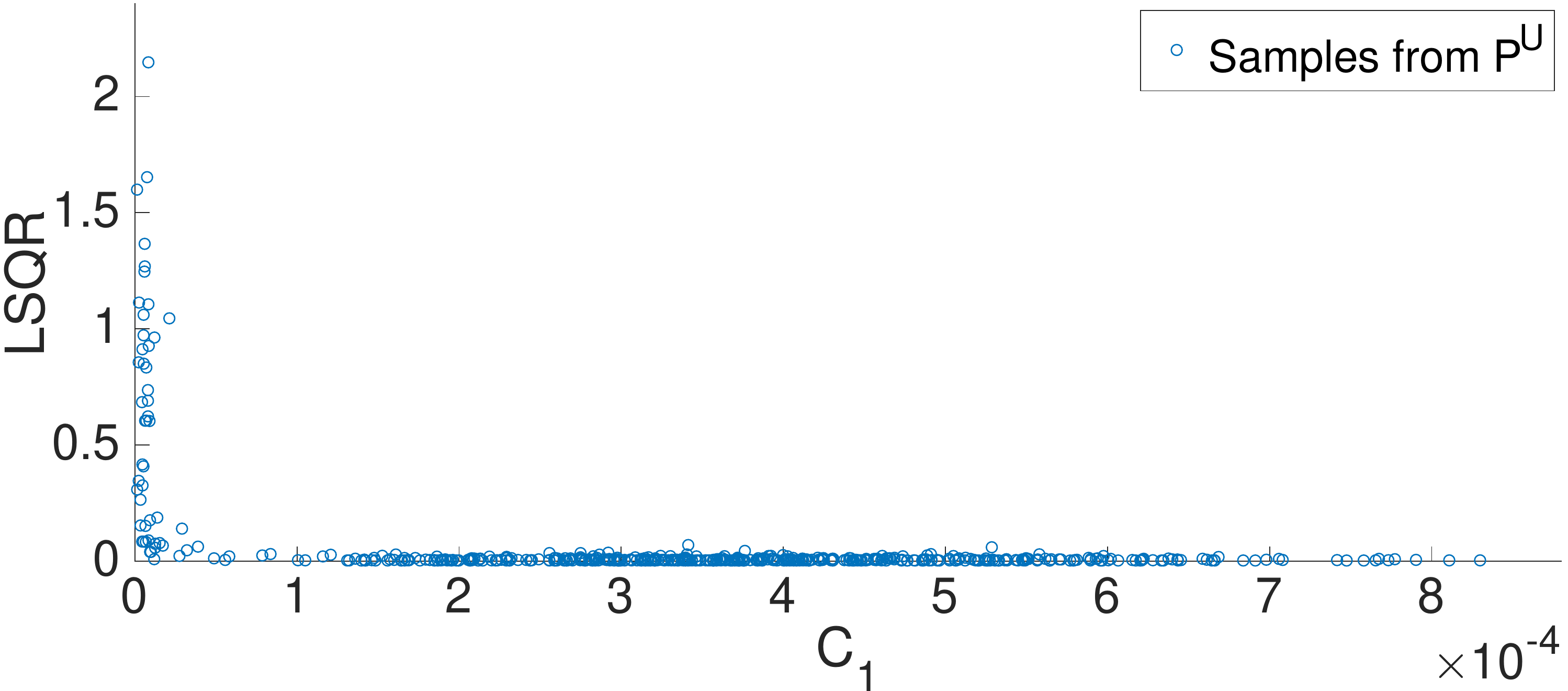}}
		&
		
		\subfloat[\label{fig:ColsRes4-d}]{\includegraphics[height=\height\textheight,width=\width\textwidth]{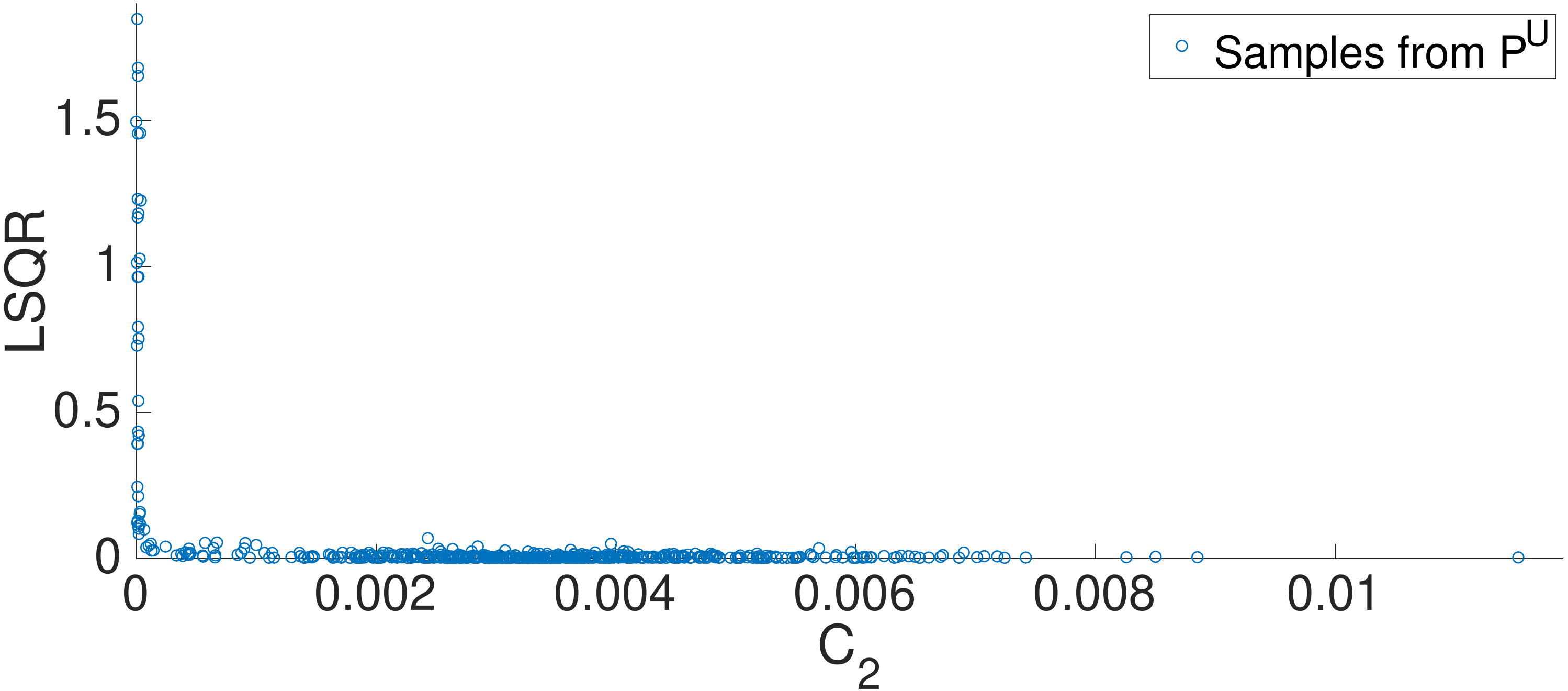}}	
		
		\\
		\subfloat[\label{fig:ColsRes4-e}]{\includegraphics[height=\height\textheight,width=\width\textwidth]{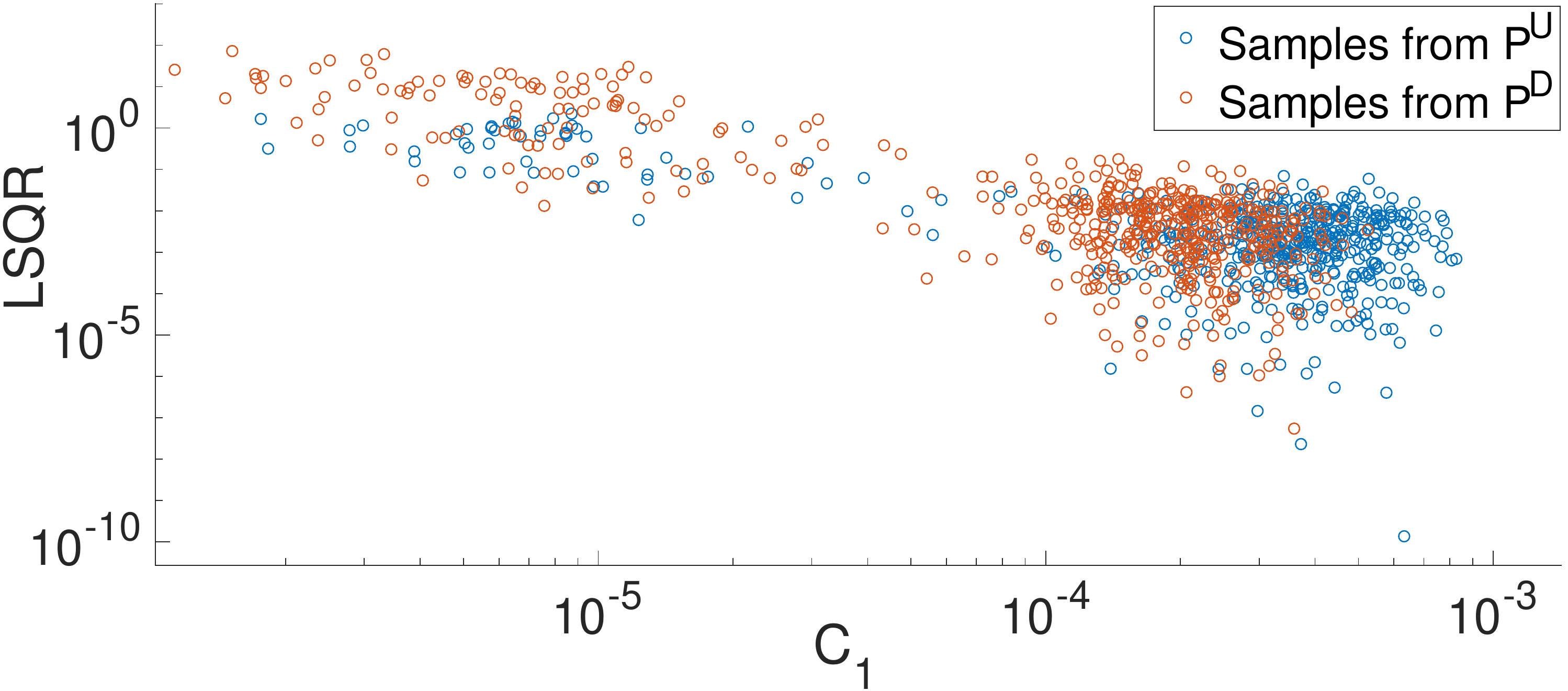}}
		&
		
		\subfloat[\label{fig:ColsRes4-f}]{\includegraphics[height=\height\textheight,width=\width\textwidth]{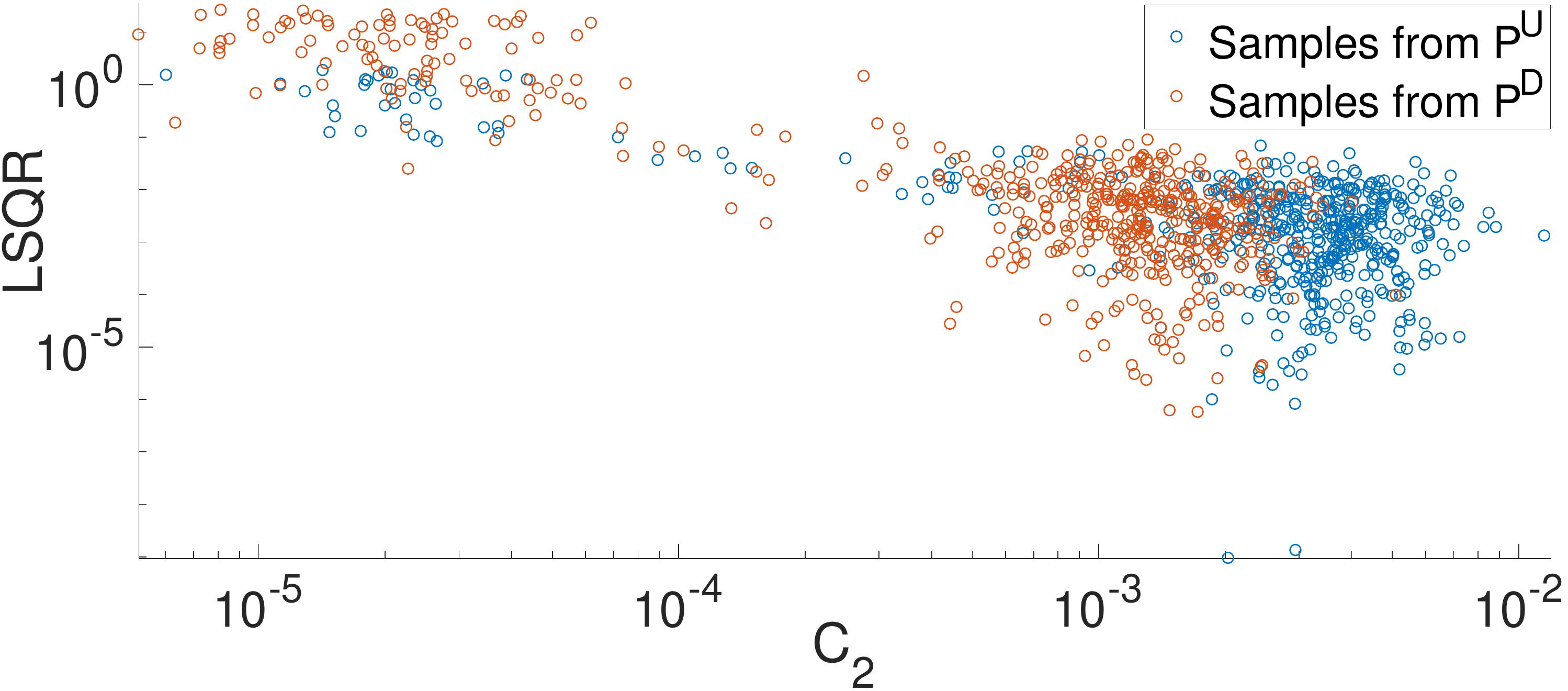}}

	\end{tabular}
	
	\subfloat[\label{fig:ColsRes4-g}]{\includegraphics[width=0.9\textwidth]{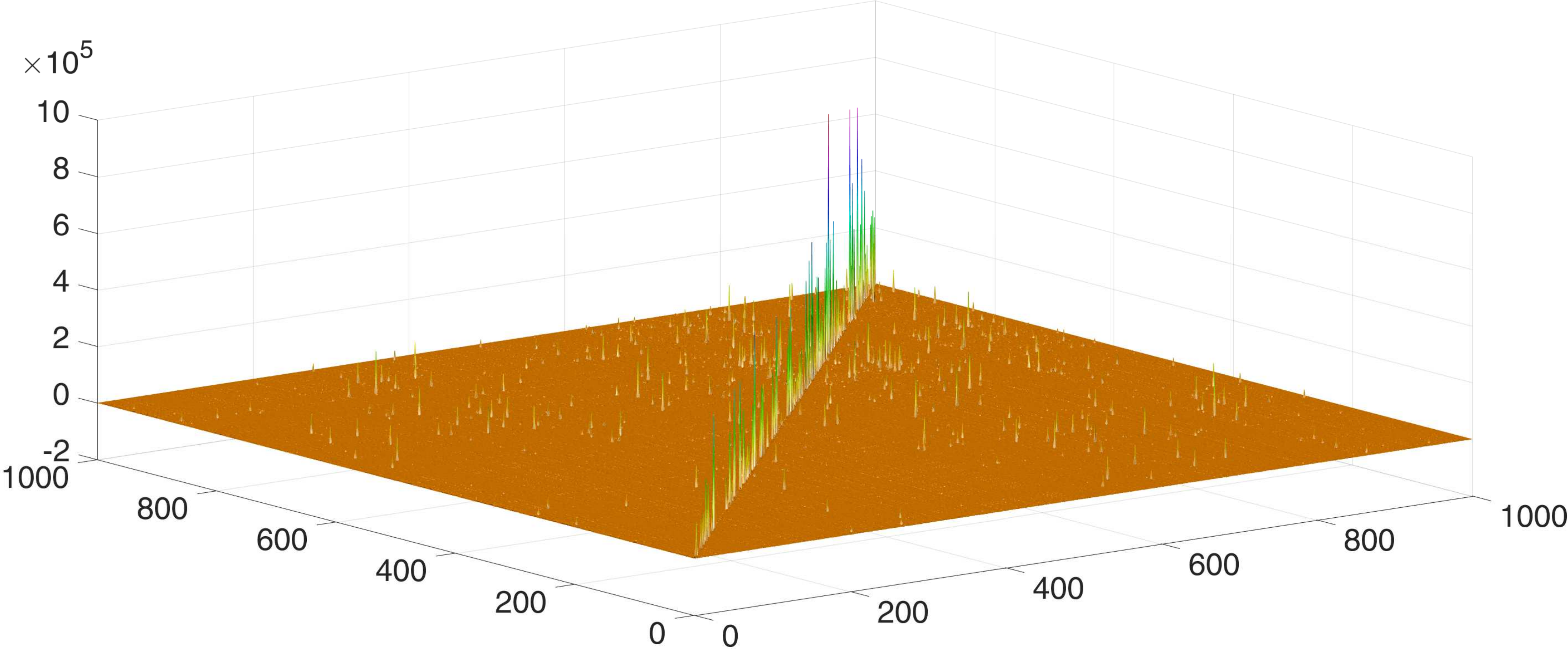}}
	
	\protect
	\caption[Relation between a point-wise error and a gradient norm.]{Relation between a point-wise error and a gradient norm. The pdf function of \emph{Columns} distribution is learned by PSO-LDE with $\alpha = \frac{1}{4}$, where NN architecture is block-diagonal with 6 layers, number of blocks $N_B = 50$ and block size $S_B = 64$ (see Section \ref{sec:BDLayers}).
		(a)-(b) Relation between inverse-gradient-norm metrics $C_1$ and $C_2$ and a point-wise error $LSQR$. As can be seen, points with the smaller inverse-gradient-norm (that is, with a bigger norm of $\theta$ gradient) have a greater approximation error. See details in the main text.
		(c)-(d) Plots of (a)-(b) with only samples from $\probs{\usuff}$ density.
		(e)-(f) Plots of (a)-(b) with both $x$ and $y$ axes scaled logarithmically.
		(g) Matrix $G$, with $G_{ij} = g_{\theta}(X_i, X_j)$.
	}
	\label{fig:ColsRes4}
\end{figure}

Furthermore, we empirically observe a direct connection between point-wise ground truth error and self \emph{gradient similarity} $g_{\theta}(X, X)$ (squared norm of gradient $\nabla_{\theta} 
f_{\theta}(X)$ at the point). To demonstrate this, we define two inverse-gradient-norm empirical metrics as follows. First, after training was finished we sample 1000 points $D = \{ X_i \}$, where 500 are sampled from $\probs{\usuff}$ and 500 - from $\probs{\dsuff}$, and calculate their gradients $\nabla_{\theta} 
f_{\theta}(X_i)$. Next, we compute the \emph{Gramian} matrix $G$ that contains all \emph{gradient similarities} among the samples, with $G_{ij} = g_{\theta}(X_i, X_j)$. Then, the first empirical metric $C_1$ for sample $X_i$ is calculated as
\begin{equation}
C_1(X_i)
=
\frac{1}{G_{ii}}
=
\frac{1}{g_{\theta}(X_i, X_i)}
.
\label{eq:IGNMetric1}
\end{equation}
The above $C_1(X_i)$ is bigger if $g_{\theta}(X, X)$ is smaller, and vice versa. The second metric $C_2$ is defined as
\begin{equation}
C_2(X_i)
=
\left[ G^{-1} \right]_{ii}
.
\label{eq:IGNMetric2}
\end{equation}
Since matrix $G$ is almost diagonal (see Figure \ref{fig:ColsRes4-g}), both $C_1$ and $C_2$ usually have a similar trend.

In Figure \ref{fig:ColsRes4} we can see that the above metrics $C_1(X_i)$ and $C_2(X_i)$ are highly correlated with point-wise $LSQR(X_i) = \left[ \log \probi{\usuff}{X_i} - \log \bar{\PP}_{\theta}(X_i) \right]^2$. That is, points with a bigger norm of the gradient $\nabla_{\theta} 
f_{\theta}(X)$ (bigger $g_{\theta}(X, X)$) have a bigger approximation error. 
One possible explanation for this trend is that there exists an estimation bias, which is amplified by a bigger gradient norm at the point. 
Further investigation is required to clarify this aspect.
Concluding, we empirically demonstrate that in the infinite data setting we can measure model uncertainty (error) at query point $X$ via a norm of its gradient. For a smaller dataset size the connection between the gradient norm and the approximation error is less obvious, probably because there we have another/additional factors that increase the approximation error (e.g. an estimation variance). Also, note that herein we use metrics $C_1$ and $C_2$ that are opposite-proportional to the gradient norm instead of using the gradient norm directly since the \emph{inverse} relation is visually much more substantial.

Additionally, in Figure \ref{fig:ColsRes4} it is visible that on average samples from $\probs{\dsuff}$ have a bigger gradient norm than samples from $\probs{\usuff}$. This can explain why in Figure \ref{fig:ColsRes3-b} we have higher error at samples from \down density.

\paragraph{Other PSO Instances}

Further, several other PSO instances were executed to compare with PSO-LDE. First is the IS method from Table \ref{tbl:PSOInstances1}. As was discussed in Section \ref{sec:DeepLogPDF}, its \mfs are unbounded which may cause instability during the optimization. In Table \ref{tbl:Perf1} we can see that indeed its performance is much inferior to PSO-LDE with bounded \emph{magnitudes}.

\begin{table}
	\centering
	\begin{tabular}{llll}
		\toprule
		Method     & $PSQR$ & $LSQR$     & $IS$ \\
		\midrule
		PSO-LDE, & $2.7 \cdot 10^{-22} \pm 2.58 \cdot 10^{-23}$ &
		$0.057 \pm 0.004$  &
		$26.58 \pm 0.01$  \\
		averaged over all $\alpha$ \\
		\midrule
		IS & $1.79 \cdot 10^{-21} \pm 5 \cdot 10^{-22}$ &
		$0.46 \pm 0.14$  &
		$26.84 \pm 0.07$  \\
		\midrule
		PSO-MAX & $3.04 \cdot 10^{-22} \pm 1.55 \cdot 10^{-23}$ &
		$0.058 \pm 0.002$  &
		$26.57 \pm 0.001$  \\
		
		\bottomrule
	\end{tabular}
	
	\caption{Performance comparison between various PSO instances}
	\label{tbl:Perf1}
	
\end{table}

Additionally, we used an instance of a normalized family defined in Eq.~(\ref{eq:PSOLogEstNormalized}), which we name PSO-MAX, with the following \mfs:
\begin{equation}
M^{\usuff}\left[X,f_{\theta}(X)\right]
=
\frac{\probi{\dsuff}{X}}{\max \left[\probi{\dsuff}{X}, \exp f_{\theta}(X) \right]}
=
\exp
\left[
- \max \left[\bar{d}\left[
X, f_{\theta}(X)
\right]
, 0\right]
\right]
,
\label{eq:PSOMaxLogEstNormalizedMU}
\end{equation}
\begin{equation}
M^{\dsuff}\left[X,f_{\theta}(X)\right]
=
\frac{\exp f_{\theta}(X)}{\max \left[\probi{\dsuff}{X}, \exp f_{\theta}(X) \right]}
=
\exp
\left[
\min \left[\bar{d}\left[
X, f_{\theta}(X)
\right]
, 0\right]
\right]
.
\label{eq:PSOMaxLogEstNormalizedMD}
\end{equation}
In Figure \ref{fig:ColsRes4.1} the above \emph{magnitudes} are depicted as functions of a logarithm difference $\bar{d}$ where we can see them to be also bounded. In fact, PSO-MAX is also an instance of PSO-LDE for a limit $\alpha \rightarrow \infty$.  Similarly to other instances of PSO-LDE, the bounded \emph{magnitudes} of PSO-MAX allow to achieve a high approximation accuracy, which gets very close to the performance of PSO-LDE for finite values of $\alpha$ (see Table \ref{tbl:Perf1}). Yet, PSO-MAX is slightly worse, suggesting that very high values of $\alpha$ are sub-optimal for the task of pdf inference.

\begin{figure}[tb]
	\centering
	
	\begin{tabular}{c}
		
		\subfloat{\includegraphics[width=0.4\textwidth]{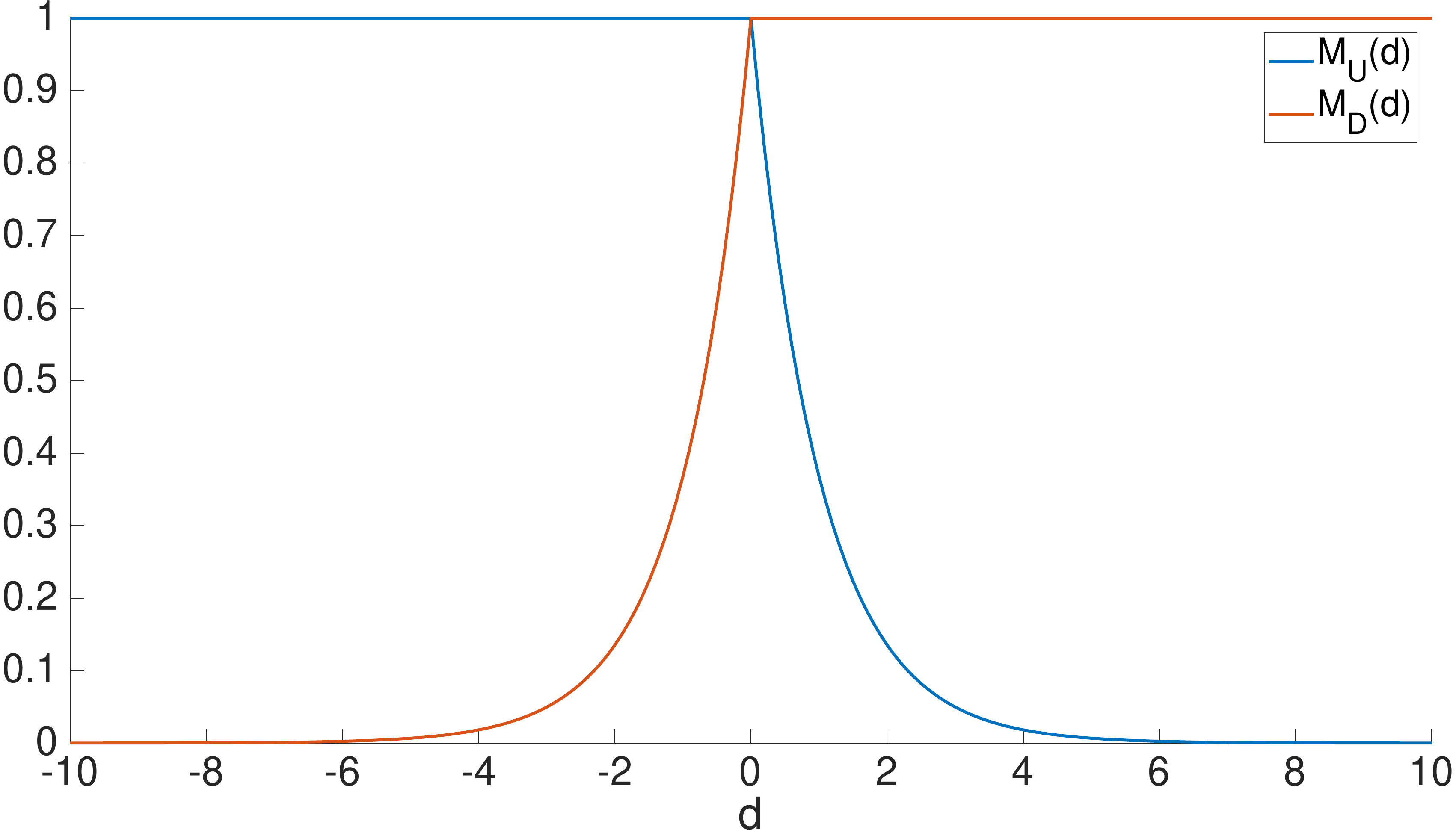}}

	\end{tabular}
	
	\protect
	\caption{PSO-MAX \emph{magnitudes} as functions of a difference $\bar{d}\left[
		X, f_{\theta}(X)
		\right]
		 = f_{\theta}(X) - \log \probi{\dsuff}{X}$.
	}
	\label{fig:ColsRes4.1}
\end{figure}

In overall, our experiments show that PSO instances with bounded \emph{magnitudes} have superior performance at pdf inference task. Further, PSO-LDE with $\alpha = \frac{1}{4}$ has better accuracy w.r.t. other values of $\alpha$. Note that this implies PSO-LDE with $\alpha = \frac{1}{4}$ is being superior to NCE \citep{Smith05acl,Gutmann10aistats}, which is PSO-LDE with $\alpha = 1$. Finally, in an infinite dataset setting and when using BD network architecture, we can measure model uncertainty of a specific query point $X$ via the self \emph{gradient similarity} $g_{\theta}(X, X)$.

\subsubsection{Baselines}
\label{sec:ColumnsEstBslns}

In the above section we showed that particular instances of PSO-LDE perform better than the NCE method (i.e. PSO-LDE with $\alpha = 1$). Likewise, in our previous work \citep{Kopitkov18arxiv} we showed on 2D and 3D data that PSO-based methods are much more accurate than kernel density estimation (KDE) approach. Unfortunately, the KDE method does not scale well with higher dimensions, with very few implementations handling data of arbitrary dimension. Instead, below we evaluate score matching
 \citep{Hyvarinen05mlr,Hyvarinen07csda,Zhai16icml,Saremi18arxiv}, Masked Auto-encoder for Distribution Estimation (MADE) \citep{Germain15icml} and Masked Autoregressive Flow (MAF) \citep{Papamakarios17nips} as state-of-the-art baselines in the context of density estimation.

\paragraph{Score Matching}

The originally introduced score matching approach  \citep{Hyvarinen05mlr} employed the following loss over samples $\{X_i^{\usuff}\}_{i = 1}^{N^{\usuff}}$ from the target density $\probi{\usuff}{X}$:
\begin{equation}
L_{SM}(\theta, \{X_i^{\usuff}\}_{i = 1}^{N^{\usuff}})
=
\frac{1}{N^{\usuff}}
\sum_{i = 1}^{N^{\usuff}}
\sum_{j = 1}^{n}
\left[
- 
\left(
\frac{\partial^2 f_{\theta}(X_i^{\usuff})}{\partial (X_{ij}^{\usuff})^2}
\right)
+
\half
\left(
\frac{\partial f_{\theta}(X_i^{\usuff})}{\partial X_{ij}^{\usuff}}
\right)^2
\right]
,
\label{eq:SMloss1}
\end{equation}
where $\frac{\partial f_{\theta}(X_i^{\usuff})}{\partial X_{ij}^{\usuff}}$ and $\frac{\partial^2 f_{\theta}(X_i^{\usuff})}{\partial (X_{ij}^{\usuff})^2}$ are first and second derivatives of $f_{\theta}(X_i^{\usuff})$ w.r.t. $j$-th entry of the $n$-dimensional sample $X_i^{\usuff}$.
Intuitively, we can see that this loss tries to construct a surface $f_{\theta}(X)$ where each sample point will be a local minima - its first derivative is "softly" enforced to be zero via the minimization of a term $\left(
\frac{\partial f_{\theta}(X_i^{\usuff})}{\partial X_{ij}^{\usuff}}
\right)^2$, whereas the second one is "softly" optimized to be positive via maximization of $\left(
\frac{\partial^2 f_{\theta}(X_i^{\usuff})}{\partial (X_{ij}^{\usuff})^2}
\right)$.
The inferred $f_{\theta}(X)$ of such optimization converges to the data \emph{energy} function,
which is proportional to the real negative log-pdf with some unknown partition constant, $\exp \left[ - f_{\theta}(X) \right] \sim \probi{\usuff}{X}$. Further, note that to optimize a NN model via $L_{SM}(\cdot)$, the typical GD-based back-propagation process will require to compute a third derivative of $f_{\theta}(X)$, which is typically computationally unfeasible for large NN models.

Due to the last point, in \citep{Zhai16icml,Saremi18arxiv} it was proposed to use the following loss as a proxy:
\begin{equation}
L_{SM}(\theta, \{X_i^{\usuff}\}_{i = 1}^{N^{\usuff}})
=
\frac{1}{N^{\usuff}}
\sum_{i = 1}^{N^{\usuff}}
\norm{
- \upsilon_i + 
\sigma^2
\cdot
\frac{\partial f_{\theta}(X)}{\partial X}|_{X = X_i^{\usuff} + \upsilon_i}
}_2^2
,
\label{eq:Deenloss}
\end{equation}
where $\upsilon_i$ is zero-centered i.i.d. noise that is typically sampled from $\sim \mathcal{N}(0, \sigma^2 \cdot I)$. This "denoising" loss was shown in \citep{Vincent11nc} to converge to the same target of the score matching loss in Eq.~(\ref{eq:SMloss1}).

Furthermore, the above loss enforces $f_{\theta}(X)$ to converge to the data \emph{energy} function. However, in this paper we are interested to estimate the data log-pdf, which is proportional to the negative data \emph{energy} function. To infer the latter via score matching, we employ the following sign change of the noise term:
\begin{equation}
L_{SM}(\theta, \{X_i^{\usuff}\}_{i = 1}^{N^{\usuff}})
=
\frac{1}{N^{\usuff}}
\sum_{i = 1}^{N^{\usuff}}
\norm{
	\upsilon_i + 
	\sigma^2
	\cdot
	\frac{\partial f_{\theta}(X)}{\partial X}|_{X = X_i^{\usuff} + \upsilon_i}
}_2^2
,
\label{eq:DeenlossPos}
\end{equation}
which has the same equilibrium as the loss in Eq.~(\ref{eq:Deenloss}), yet with the negative sign. Namely, at a convergence $f_{\theta}(X)$ will satisfy now $\exp \left[ f_{\theta}(X) \right] \sim \probi{\usuff}{X}$. In our experiments we used this version of score matching loss for the density estimation of 20D \emph{Columns} distribution.

\begin{figure}[tb]
	\centering
	
	\begin{tabular}{cc}
		
		\subfloat[\label{fig:ColsRes4.2-a}]{\includegraphics[width=0.4\textwidth]{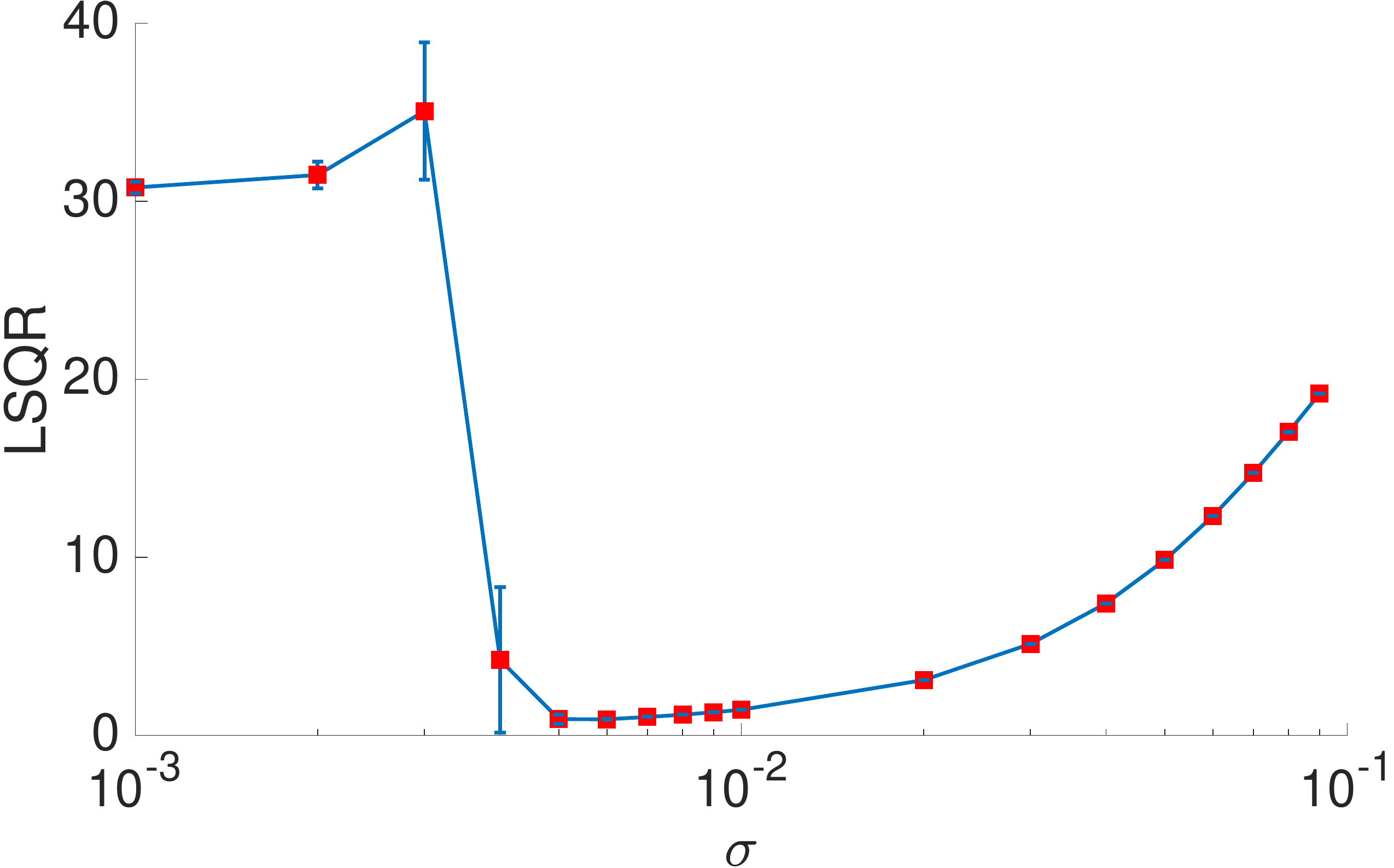}}
		
		&
		
		\subfloat[\label{fig:ColsRes4.2-b}]{\includegraphics[width=0.4\textwidth]{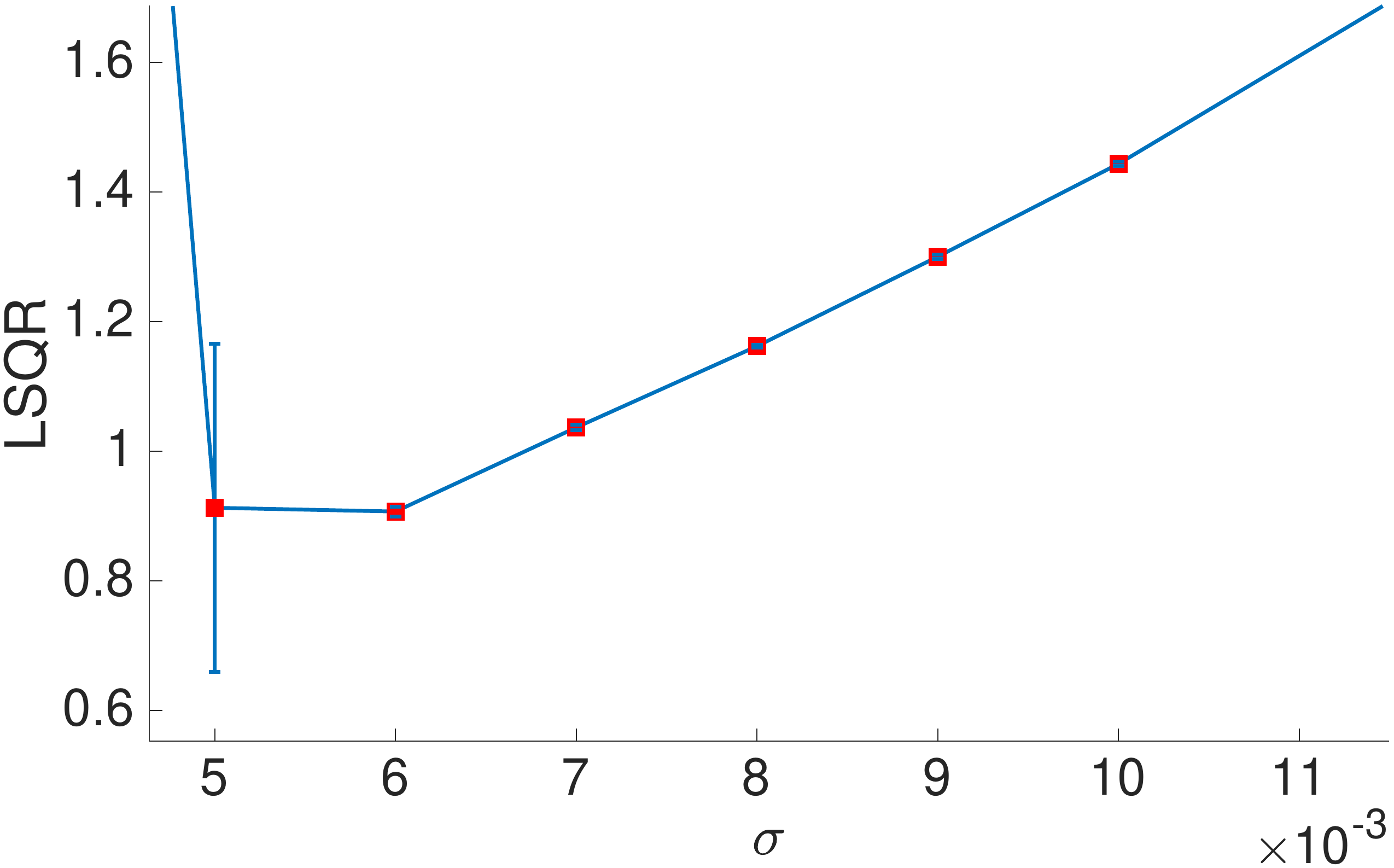}}
		
	\end{tabular}
	
	\begin{tabular}{ccc}
		
		\subfloat[\label{fig:ColsRes4.2-c}]{\includegraphics[width=0.37\textwidth]{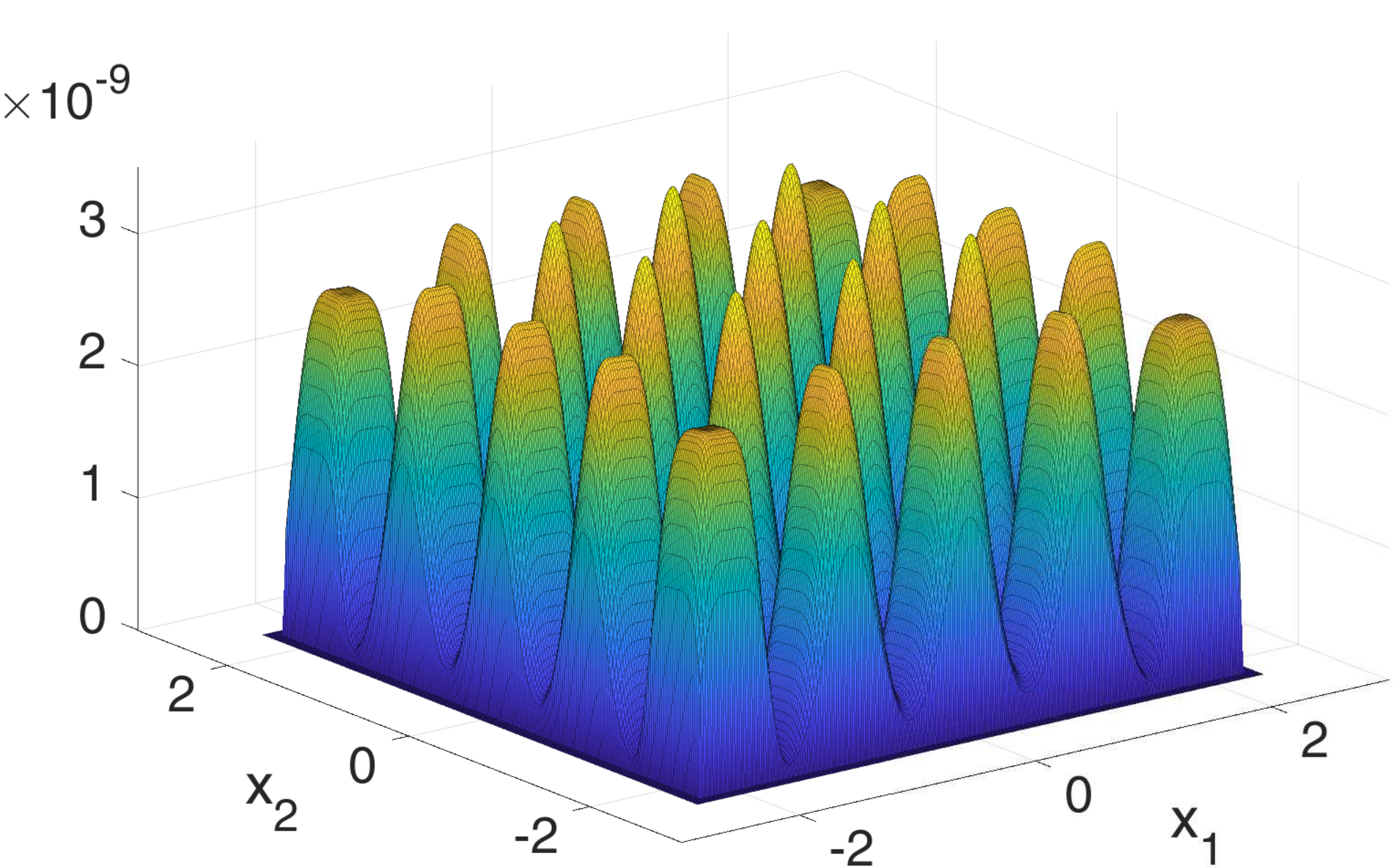}}
		&
		
		\subfloat[\label{fig:ColsRes4.2-d}]{\includegraphics[width=0.27\textwidth]{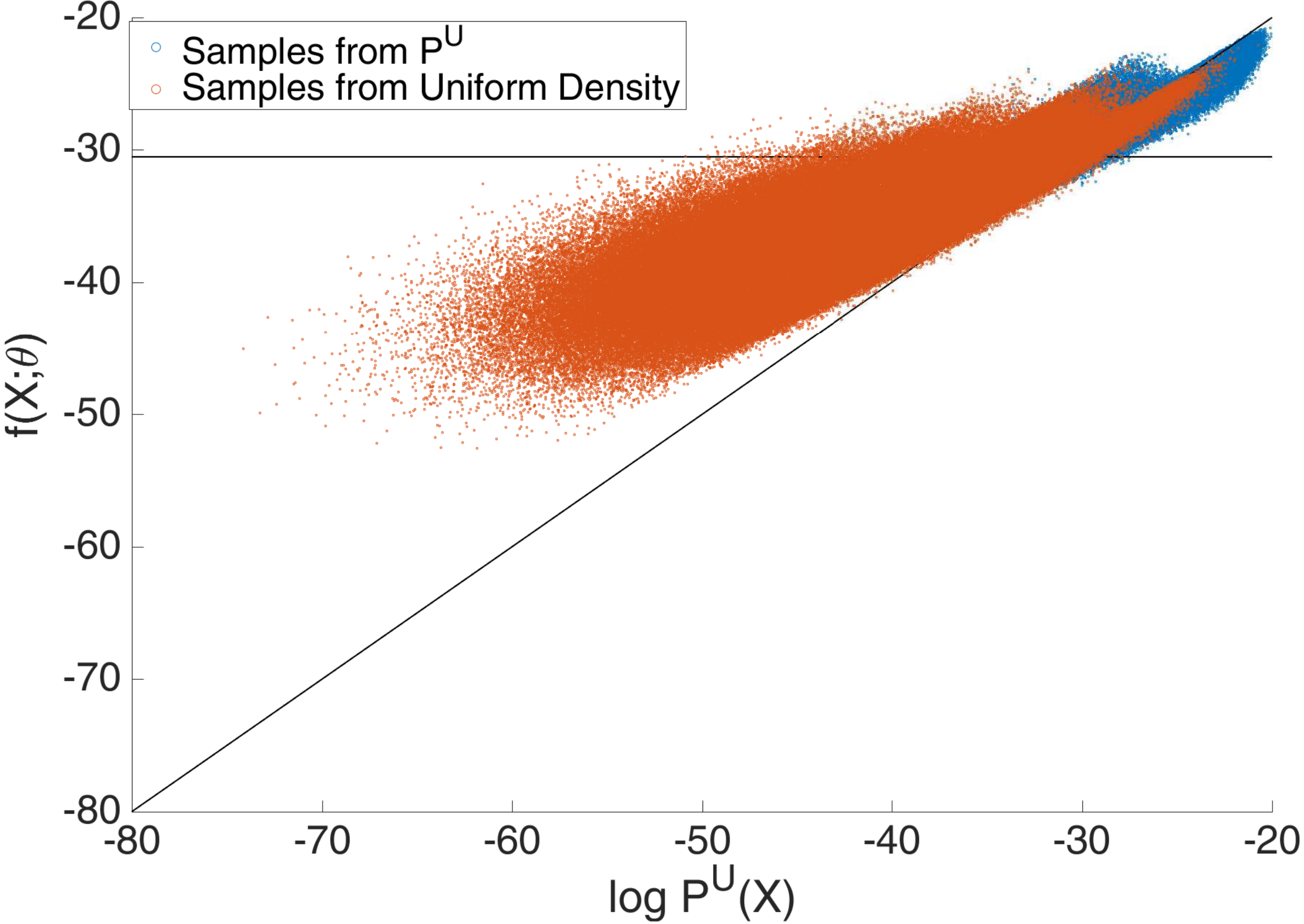}}
		
		&
		
		\subfloat[\label{fig:ColsRes4.2-e}]{\includegraphics[width=0.27\textwidth]{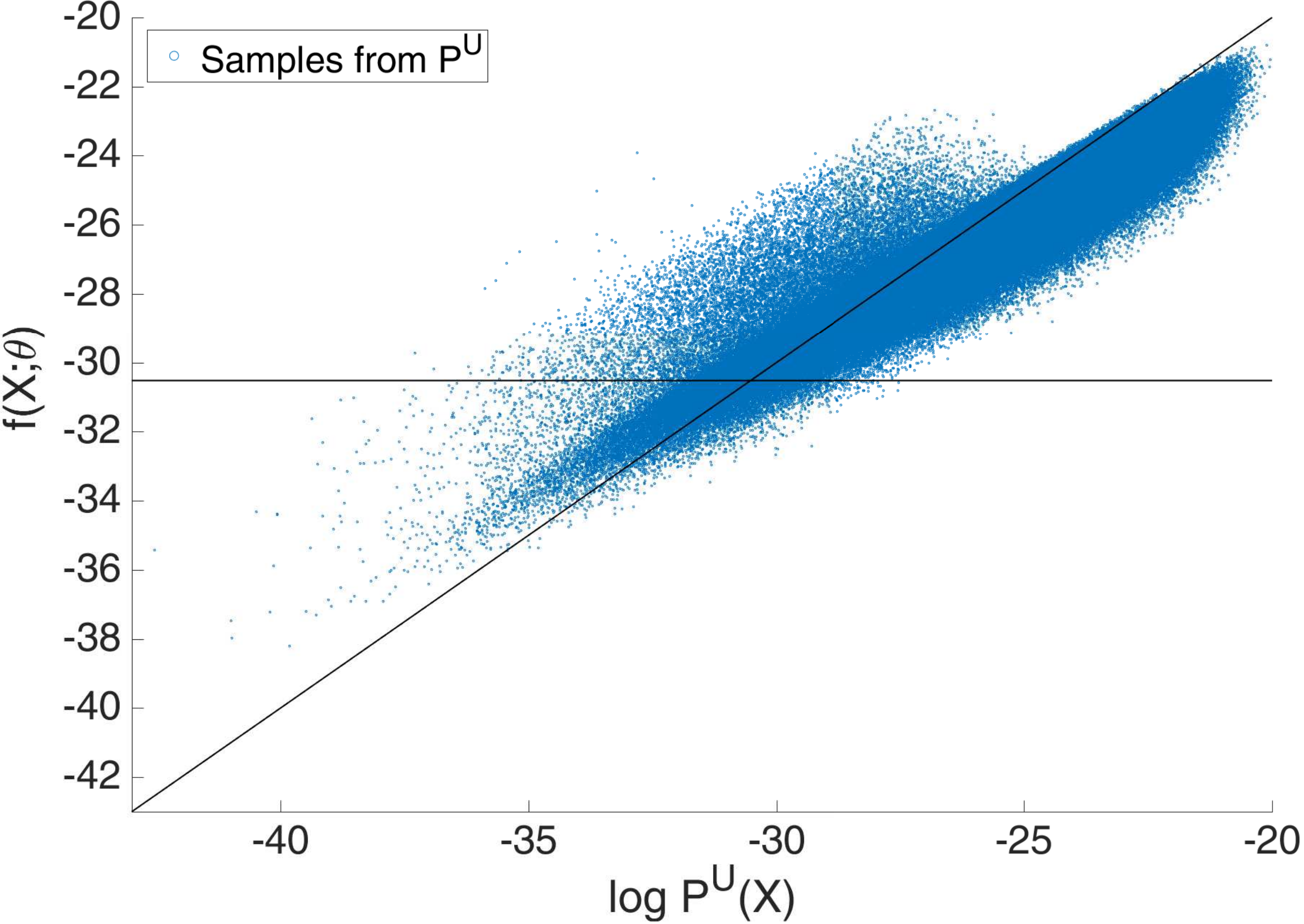}}
		
	\end{tabular}

	\protect
	\caption[Learned pdf function of \emph{Columns} distribution by score matching, where the applied NN is BD.]{Learned pdf function of \emph{Columns} distribution by score matching, where NN architecture is block-diagonal with 6 layers, number of blocks $N_B = 50$ and block size $S_B = 64$ (see Section \ref{sec:BDLayers}). The employed activation function is $tanh()$.
		(a) $LSQR$ error (mean and standard deviation) for various values of a scaling hyper-parameter $\sigma$;
		(b) Zoom of (a);
		(c) Illustration of learned pdf function for best model with $\sigma = 0.006$. The depicted slice is $\PP(x_1, x_2) = \bar{\PP}^{\usuff}(x_1, x_2, 0, \ldots , 0)$, with $x_1$ and $x_2$ forming a grid of points in first two dimensions of the density's support. As can be seen, the estimated pdf is over-smoothed w.r.t. real pdf in Figure \ref{fig:ColsRes1-a}. In contrast, PSO-LDE estimation in Figure \ref{fig:ColsRes3-a} does not have this extra-smoothing nature.
		(d) Illustration of the learned surface $f_{\theta}(X)$. Blue points are sampled from $\probs{\usuff}$, while red points - from $\probs{\dsuff}$, minimal 20D Uniform distribution that covers all samples from $\probs{\usuff}$. The $x$ axis represents $\log \probi{\usuff}{X}$ for each sample, $y$ axis represents the surface "normalized" height $\bar{f}_{\theta}(X) = f_{\theta}(X) - \log(\overline{TI}) = \log \bar{\PP}_{\theta}(X)$ after optimization was finished. The diagonal line represents $\bar{f}_{\theta}(X) = \log \probi{\usuff}{X}$, where we would see all points in case of \emph{perfect} model inference.
		(e) Plot from (d) with only samples from $\probs{\usuff}$.
		We can see that the produced surface is significantly less accurate than the one produced by PSO-LDE in Figures \ref{fig:ColsRes3-a}-\ref{fig:ColsRes3-b}.
	}
	\label{fig:ColsRes4.2}
\end{figure}

The employed learning setup of score matching is identical to PSO-LDE, with the loss in Eq.~(\ref{eq:DeenlossPos}) being applied in a mini-batch mode, where at each optimization iteration  a batch of samples $\{X_i^{\usuff}\}_{i = 1}^{N^{\usuff}}$ was fetched from the training dataset and the new noise batch $\{\upsilon^i\}_{i = 1}^{N}$ was generated. The learning rate of Adam optimizer was 0.003. Note that this method infers $\exp \left[ f_{\theta}(X) \right]$ which is only proportional to the real pdf with some unknown partition constant. Therefore, in order to compute $LSQR$ of such model we also calculated its partition via importance sampling. Specifically, for each learned model $\exp \left[ f_{\theta}(X) \right]$ its integral was estimated through $\overline{TI} = \sum_{i = 1}^{N^{\dsuff}} \frac{\exp \left[ f_{\theta}(X^{\dsuff}_{i}) \right]}{\probi{\dsuff}{X^{\dsuff}_{i}}}$ over $N^{\dsuff} = 10^8$ samples from density $\probs{\dsuff}$, which is the minimal 20D Uniform distribution that covers all samples from $\probs{\usuff}$. Further, we used $\bar{\PP}_{\theta}(X) = \exp \left[ f_{\theta}(X) - \log(\overline{TI}) \right]$ as the final estimation of data pdf.

Furthermore, we trained the score matching model for a range of $\sigma$ values. After the explicit normalization of each trained model, in Figures \ref{fig:ColsRes4.2-a}-\ref{fig:ColsRes4.2-b} we can see the $LSQR$ error for each value of a hyper-parameter $\sigma$. Particularly, for $\sigma = 0.006$ we got the smaller error $LSQR = 0.907 \pm 0.0075$, which is still much inferior to the accuracy obtained by PSO-LDE. Moreover, in Figure \ref{fig:ColsRes4.2-c} we can see that the estimated surface is over-smoothed and does not accurately approximate sharp edges of the target pdf. In contrast, PSO-LDE produces a very close pdf estimation of an arbitrary shape, as was shown in Section \ref{sec:ColumnsEstME}. Likewise, comparing Figures \ref{fig:ColsRes4.2-d} and \ref{fig:ColsRes3-b} we can see again that PSO-LDE yields a much better accuracy.

\paragraph{Masked Auto-encoder for Distribution Estimation}

\begin{figure}[tb]
	\centering
	
	\begin{tabular}{cc}
		
		\subfloat[\label{fig:ColsResMADE-a}]{\includegraphics[width=0.4\textwidth]{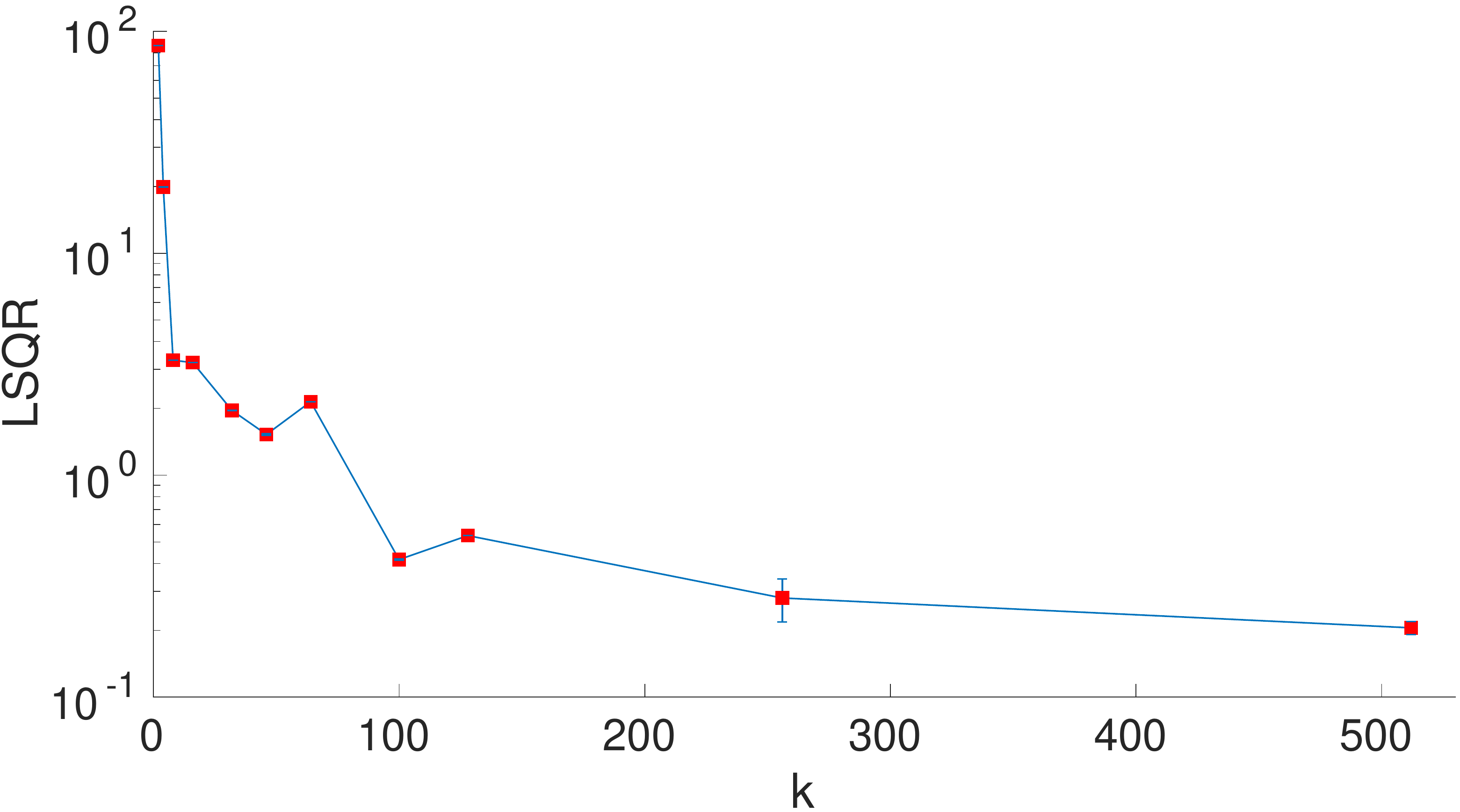}}
		
		&
		
		\subfloat[\label{fig:ColsResMADE-b}]{\includegraphics[width=0.4\textwidth]{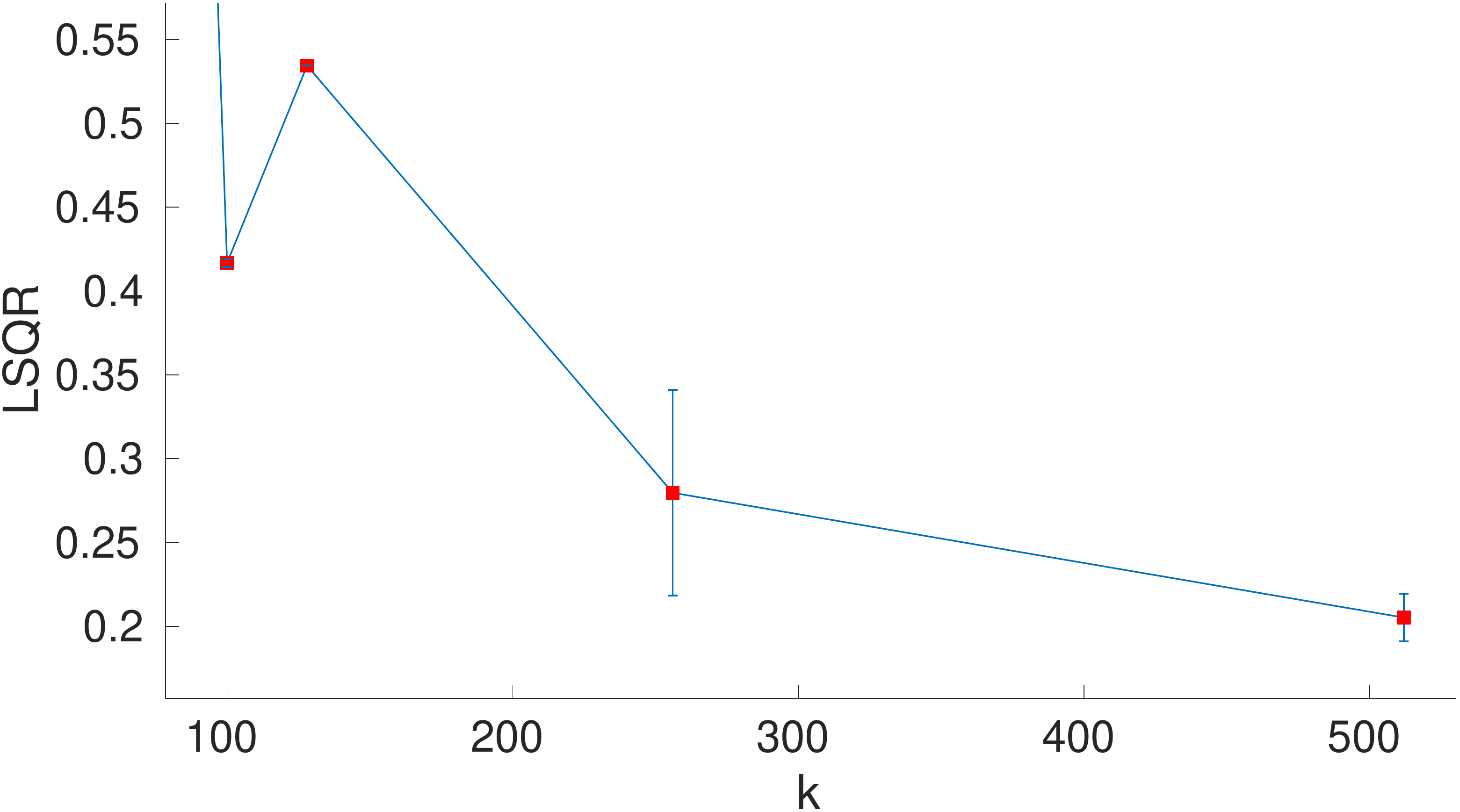}}
		
	\end{tabular}
	
	\begin{tabular}{ccc}
		
		\subfloat[\label{fig:ColsResMADE-c}]{\includegraphics[width=0.37\textwidth]{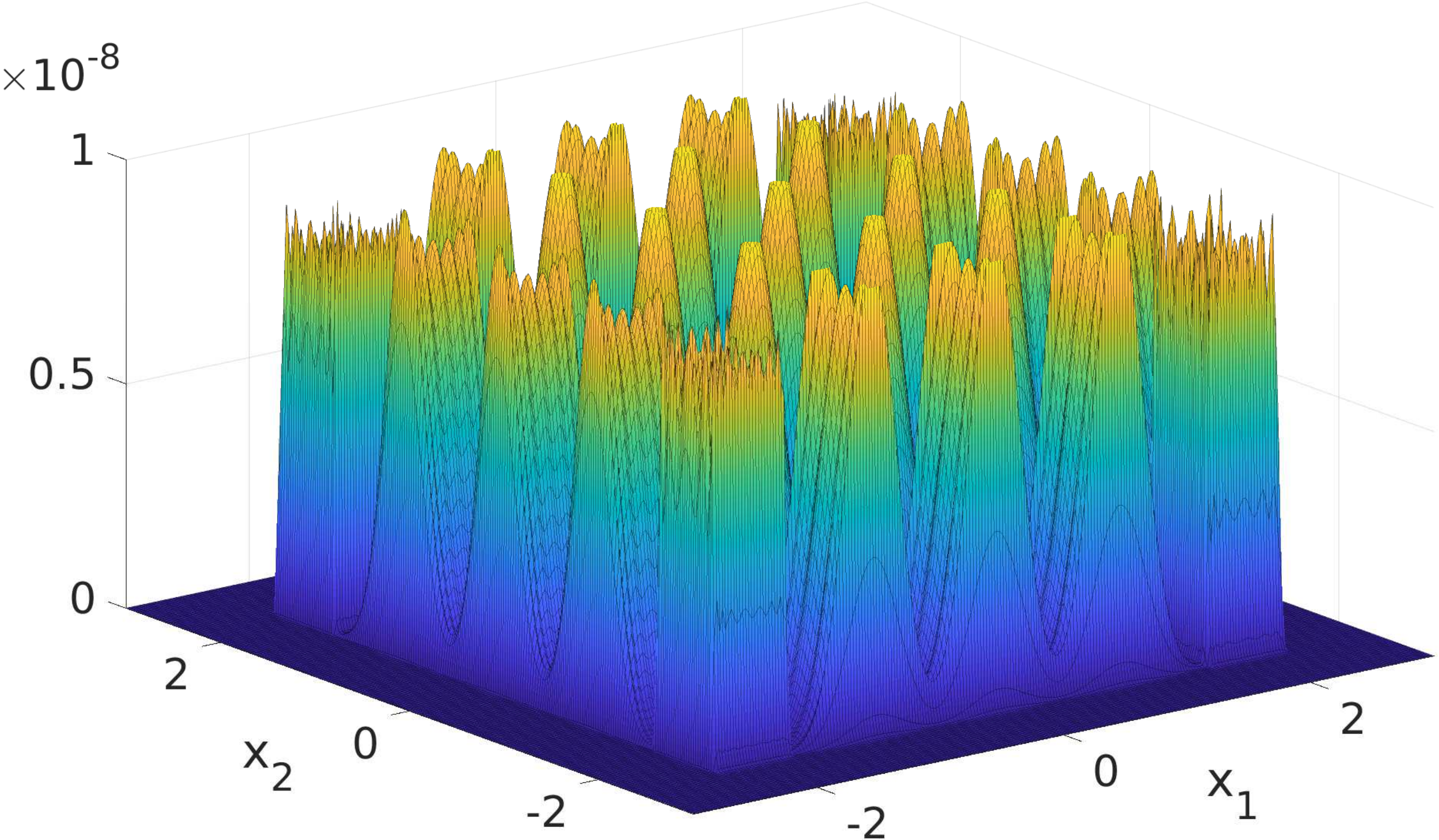}}
		&
		
		\subfloat[\label{fig:ColsResMADE-d}]{\includegraphics[width=0.27\textwidth]{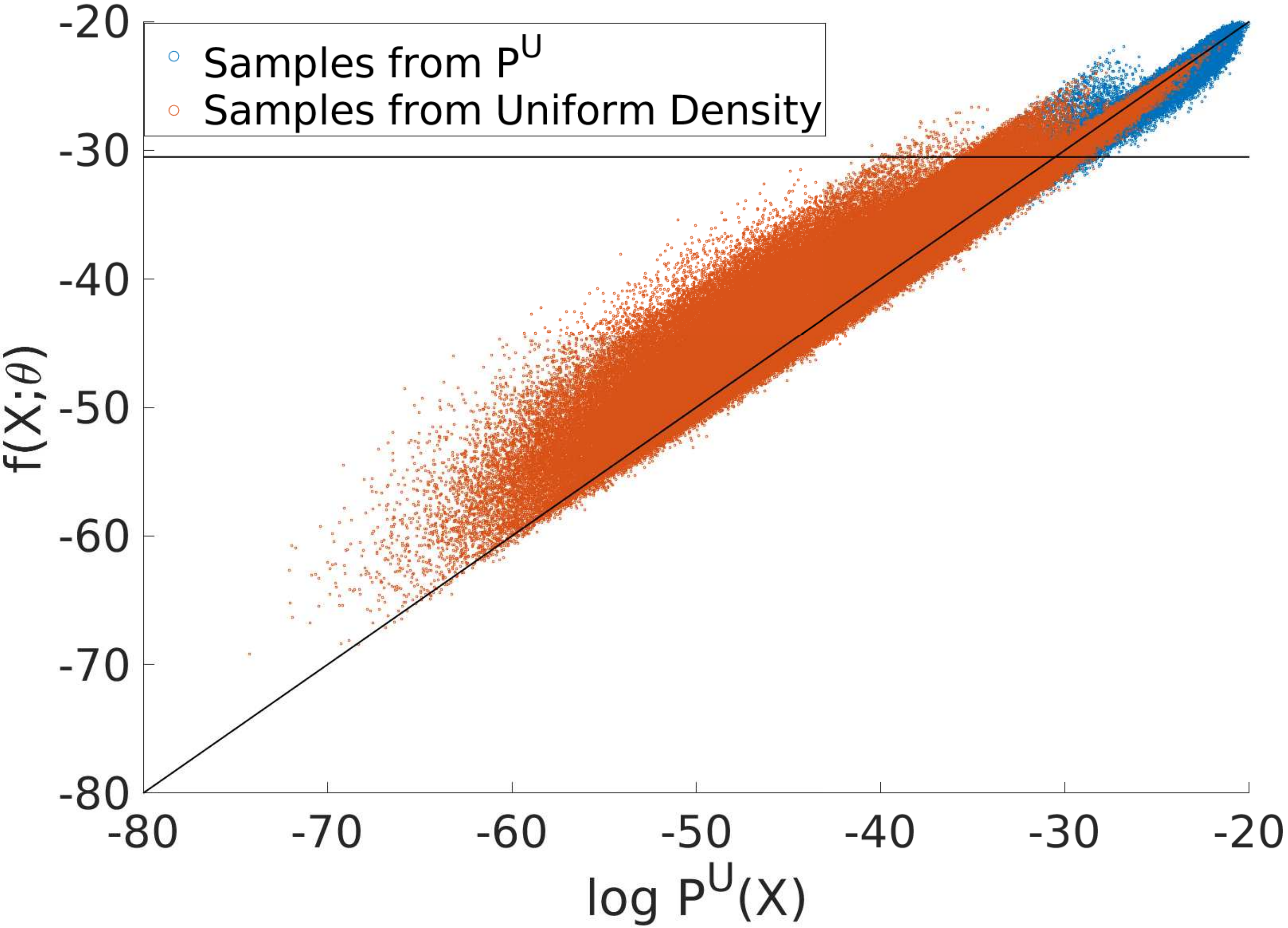}}
		
		&
		
		\subfloat[\label{fig:ColsResMADE-e}]{\includegraphics[width=0.27\textwidth]{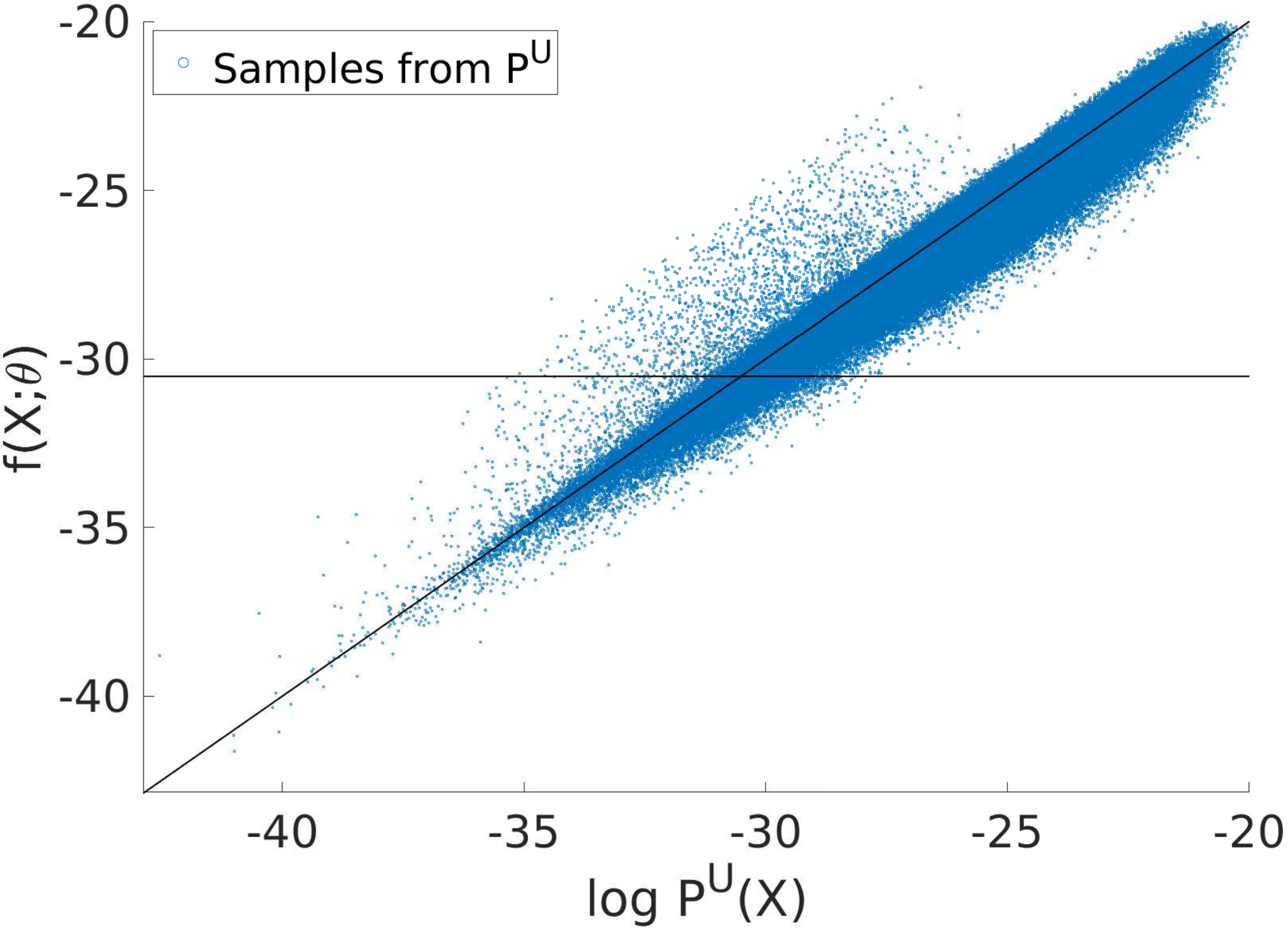}}
		
	\end{tabular}

	\protect
	\caption[Learned pdf function of \emph{Columns} distribution by MADE, where the applied NN is FC.]{Learned pdf function of \emph{Columns} distribution by MADE, where NN architecture is fully-connected with 4 layers of size 1024. The employed activation function is Relu.
		(a) $LSQR$ error (mean and standard deviation) for various values of a $k$ - a number of mixture components;
		(b) Zoom of (a);
		(c) Illustration of learned pdf function for the best model with $k = 512$. The depicted slice is $\PP(x_1, x_2) = \bar{\PP}^{\usuff}(x_1, x_2, 0, \ldots , 0)$, with $x_1$ and $x_2$ forming a grid of points in first two dimensions of the density's support. As can be seen, the estimated pdf is over-spiky in areas where the real pdf in Figure \ref{fig:ColsRes1-a} is flat. This is due to an inability of MoG model to represent flat non-zero surfaces. In contrast, PSO-LDE estimation in Figure \ref{fig:ColsRes3-a} does not have this issue.
		(d) Illustration of the estimated pdf $\bar{\PP}_{\theta}(X)$. Blue points are sampled from $\probs{\usuff}$, while red points - from $\probs{\dsuff}$, minimal 20D Uniform distribution that covers all samples from $\probs{\usuff}$. The $x$ axis represents $\log \probi{\usuff}{X}$ for each sample, $y$ axis represents $\bar{f}_{\theta}(X) \triangleq \log \bar{\PP}_{\theta}(X)$ after optimization was finished. The diagonal line represents $\bar{f}_{\theta}(X) = \log \probi{\usuff}{X}$, where we would see all points in case of \emph{perfect} model inference.
		(e) Plot from (d) with only samples from $\probs{\usuff}$.
	}
	\label{fig:ColsResMADE}
\end{figure}

This technique is based on the autoregressive property of density functions, $\PP(x_1, \ldots, x_n) = \prod_{i = 1}^{n} \PP(x_i | x_1, \ldots, x_{i-1})$, where each conditional $\PP(x_i | x_1, \ldots, x_{i-1})$ is parameterized by NN. MADE constructs a network with sequential FC layers, where the autoregressive property is preserved via masks applied on activations of each layer \citep{Germain15icml}. Likewise, each conditional can be modeled as 1D density of any known distribution family, with a typical choice being Gaussian or Mixture of Gaussians (MoG).

In our experiments we used MoG with $k$ components to model each conditional, due to the highly multi-modal nature of \emph{Columns} distribution. Moreover, we evaluated MADE for a range of various $k$, to see how the components number affects the technique's performance. Furthermore, the learning setup was similar to other experiments, with the only difference that the applied NN architecture was FC, with 4 layers of size 1024 each, and the exploited non-linearity was Relu.

In Figures \ref{fig:ColsResMADE-a}-\ref{fig:ColsResMADE-b} the $LSQR$ error is shown for each value of $k$. We can clearly see that with higher number of components the accuracy improves, where the best performance was achieved by $k = 512$ with $LSQR = 0.2 \pm 0.0141$. Furthermore, in Figure \ref{fig:ColsResMADE-c} we can see an estimated surface for the best learned model.
As observed, most of the MoG components are spent to represent flat peaks of the target density.
Such outcome is natural since for MoG to approximate flat areas the value of $k$ has to go to infinity. Moreover, this demonstrates the difference between parametric and non-parametric techniques. Due to an explicit parametrization of each conditional, MADE can be considered as a member of the former family, while PSO-LDE is definitely a member of the latter. Further, non-parametric approaches are known to be more robust/flexible in general. In overall, we can see that PSO-LDE outperforms MADE even for a large number of mixture components.

\paragraph{Masked Autoregressive Flow}

\begin{figure}[tb]
	\centering
	
	\begin{tabular}{cc}
		
		\subfloat[\label{fig:ColsResMAF-a}]{\includegraphics[width=0.4\textwidth]{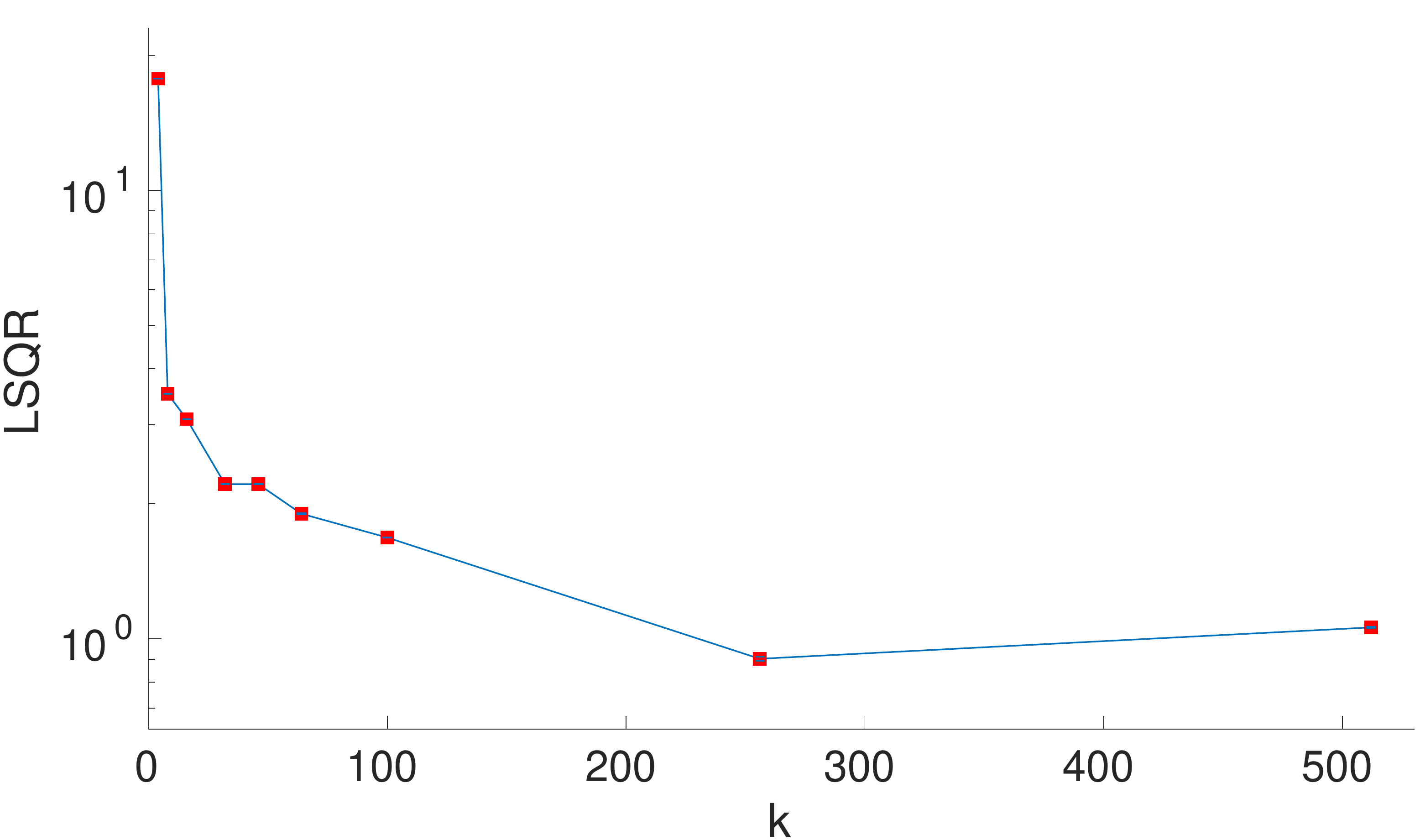}}
		
		&
		
		\subfloat[\label{fig:ColsResMAF-b}]{\includegraphics[width=0.4\textwidth]{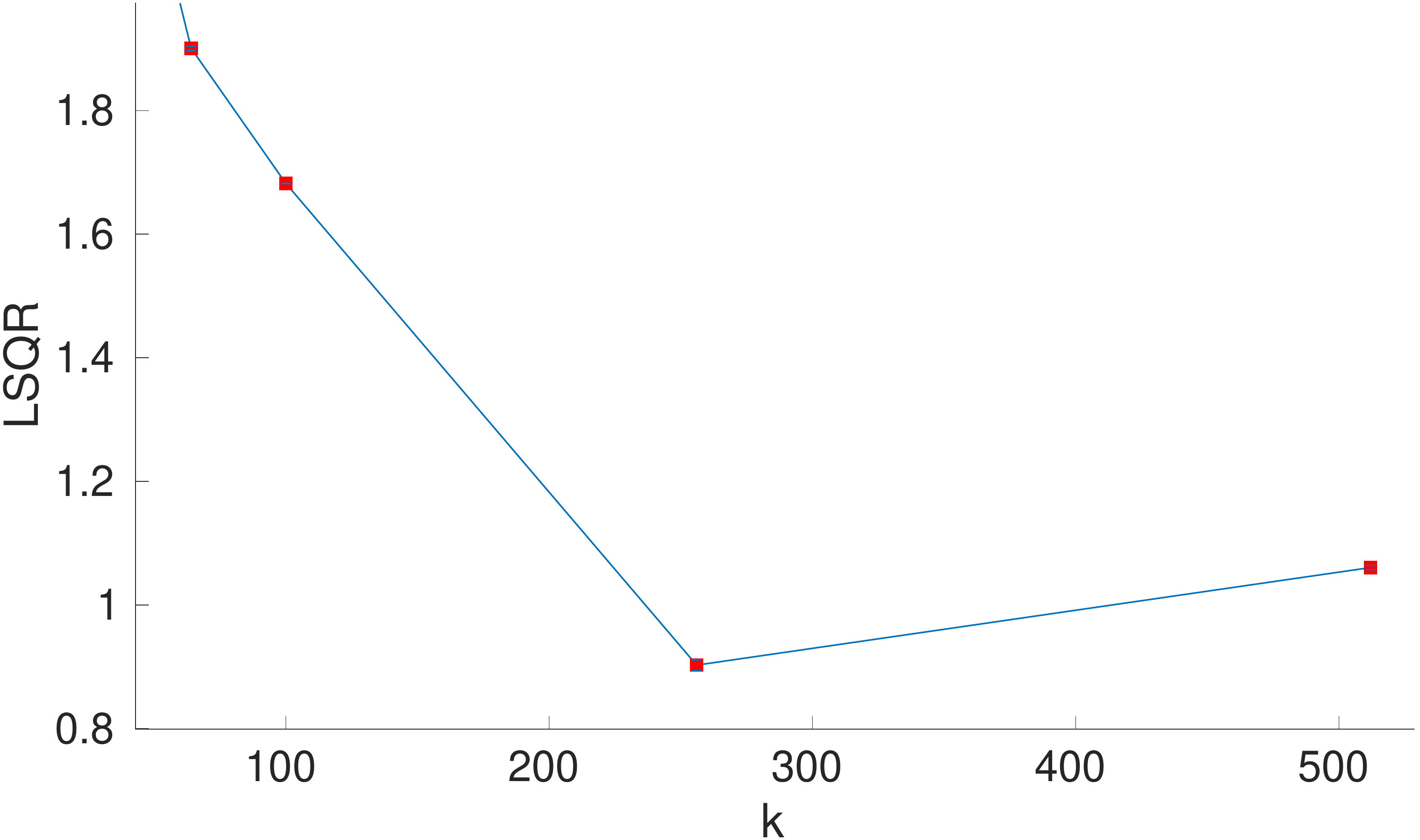}}
		
	\end{tabular}
	
	\begin{tabular}{ccc}
		
		\subfloat[\label{fig:ColsResMAF-c}]{\includegraphics[width=0.37\textwidth]{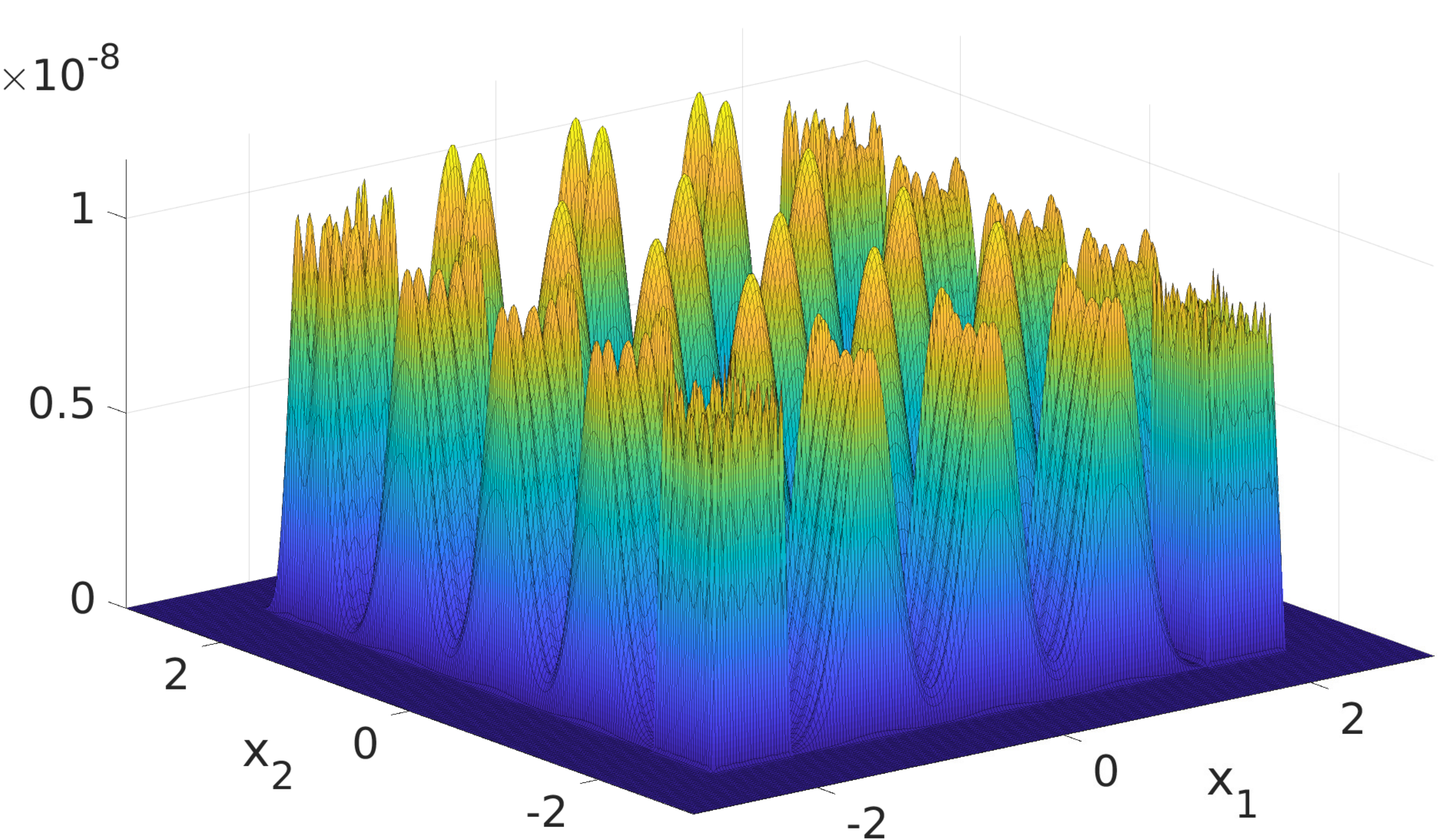}}
		&
		
		\subfloat[\label{fig:ColsResMAF-d}]{\includegraphics[width=0.27\textwidth]{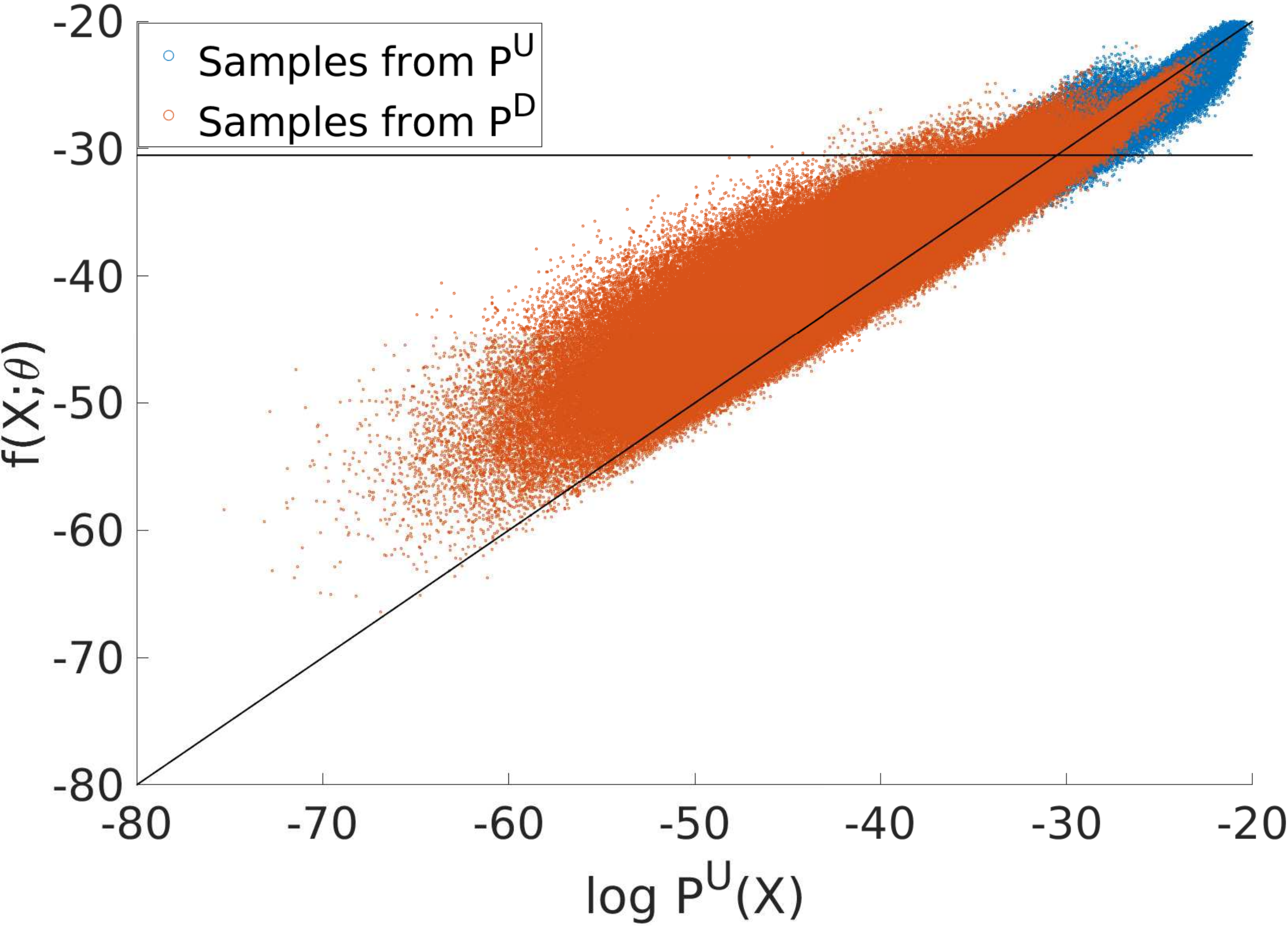}}
		
		&
		
		\subfloat[\label{fig:ColsResMAF-e}]{\includegraphics[width=0.27\textwidth]{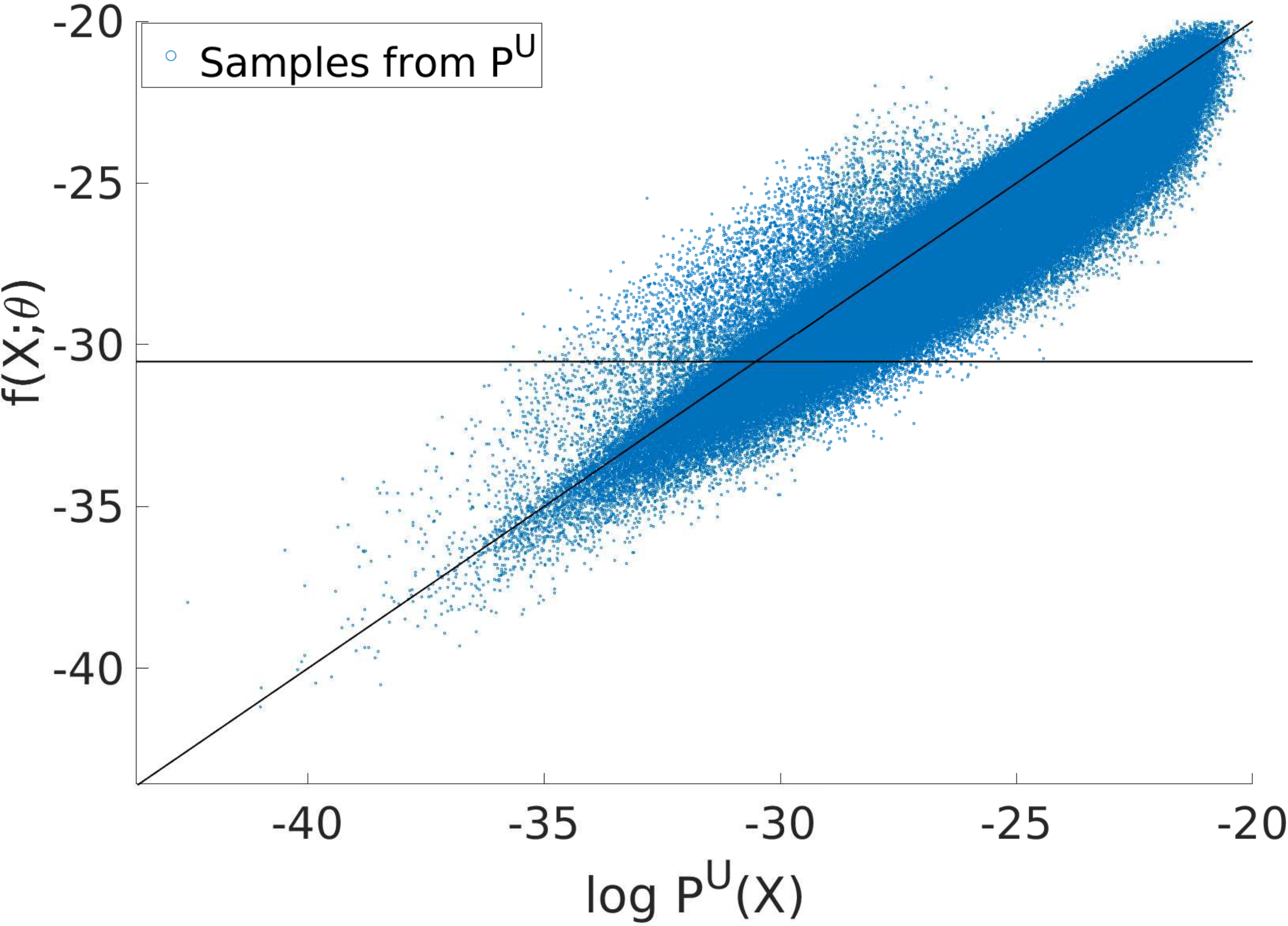}}
		
	\end{tabular}

	\protect
	\caption[Learned pdf function of \emph{Columns} distribution by MAF, with 5 inner MADE bijections and MADE MoG as a base density.]{Learned pdf function of \emph{Columns} distribution by MAF, with 5 inner MADE bijections and MADE MoG as a base density.
		(a) $LSQR$ error (mean and standard deviation) for various values of a $k$ - a number of mixture components;
		(b) Zoom of (a);
		(c) Illustration of learned pdf function for the best model with $k = 256$. The depicted slice is $\PP(x_1, x_2) = \bar{\PP}^{\usuff}(x_1, x_2, 0, \ldots , 0)$, with $x_1$ and $x_2$ forming a grid of points in first two dimensions of the density's support. The same over-spiky behavior can be observed as in Figure \ref{fig:ColsResMADE-c}.
		(d) Illustration of the estimated pdf $\bar{\PP}_{\theta}(X)$ for the best model, constructed similarly to Figure \ref{fig:ColsResMADE-d}.
		(e) Plot from (d) with only samples from $\probs{\usuff}$.
	}
	\label{fig:ColsResMAF}
\end{figure}

Shortly MAF, this technique combines an NN architecture of the previous MADE method with the idea of a normalizing flow \citep{Rezende15arxiv} where a bijective transformation $h(\cdot)$ is applied to transform a priori chosen base density into the target density. Such bijective transformation allows to re-express the density of target data via an inverse of $h(\cdot)$ and via the known pdf of a base density, and further to infer the target pdf via a standard MLE loss. Moreover, the architecture of MADE can be seen as such bijective transformation, which is specifically exploited by MAF method \citep{Papamakarios17nips}. Particularly, several MADE transformations are stuck together into one large bijective transformation, which allows for richer representation of the inferred pdf. In our experiments we evaluated MAF method with 5 inner MADE bijections.

Furthermore, the original paper proposed two MAF types. First one, referred as MAF($\cdot$) in the paper, uses multivariate normal distribution as a base density. During the evaluation this type did not succeed to infer 20D \emph{Columns} distribution at all, probably because of its inability to handle distribution with $5^{20}$ modes.

The second MAF type, referred as MAF MoG($\cdot$) in the paper, uses MADE MoG as a base density in addition to the MADE-based bijective transformation. This type showed better performance w.r.t. first type, and we used it as an another baseline. Likewise, note that also in the original paper \citep{Papamakarios17nips} this type was shown on average to be superior between the two.

Like in MADE experiments, also here we tested MAF MoG for different values of $k$ - mixture components number of the base density, parametrized by a separate MADE MoG model.
In Figures \ref{fig:ColsResMAF-a}-\ref{fig:ColsResMAF-b} the $LSQR$ error is shown for each value of $k$. As in MADE case, also here accuracy improves with a higher number of components. The top accuracy was achieved by $k = 256$ with $LSQR = 0.9 \pm 0.009$. On average, MAF MoG showed the same trends as MADE method, yet with some higher $LSQR$ error. Moreover, during the experiments it was observed as a highly unstable technique, with thorough hyper-parameter tuning needed to overcome numerical issues of this approach.

In overall, we observed that non-parametric PSO-LDE is superior to other state-of-the-art baselines when dealing with highly multi-modal \emph{Columns} distribution.

\subsubsection{NN Architectures Evaluation}
\label{sec:ColumnsEstNNAME}

Here we compare performance of various NN architectures for the pdf estimation task. 

\paragraph{FC Architecture}

We start with applying PSO-LDE with different values of $\alpha$ where the used NN architecture is now fully-connected (FC), with 4 layers of size 1024. In Figure \ref{fig:ColsRes5-a} we show $LSQR$ for different $\alpha$, where again we can see that $\alpha = \frac{1}{4}$ (and now also $\alpha = \frac{1}{3}$) performs better than other values of $\alpha$. On average, $LSQR$ error is around 2.5 which is significantly higher than 0.057 for BD architecture. Note also that BD network, used in Section \ref{sec:ColumnsEstME}, is twice smaller than FC network, containing only 902401 weights in BD vs 2121729 in FC, yet it produced a significantly better performance.

Further, in Figure \ref{fig:ColsRes5-b} we illustrated the learned surface $f_{\theta}(X)$ for a single FC model with $\alpha = \frac{1}{4}$. Compared with Figure \ref{fig:ColsRes3-b}, we can see that FC architecture produces a much less accurate NN surface. We address it to the fact that in BD network the \emph{gradient similarity} $g_{\theta}(X, X')$ has much smaller overall side-influence (bandwidth) and the induced bias compared to the FC network, as was demonstrated in Section \ref{sec:BDLayers}. Hence, BD models are more flexible than FC and can be pushed closer to the target function $\log \probi{\usuff}{X}$, producing more accurate estimations.

Additionally, note the error asymmetry in Figure \ref{fig:ColsRes5-b} which was already observed in Figure \ref{fig:ColsRes3-b}. Also here we can see that the entire cloud of points is rotated from zero error line $f_{\theta}(X) = \log \probi{\usuff}{X}$ by some angle where the rotation axis is also around horizontal line $\log \probi{\dsuff}{X} = - 30.5$. As explained in Section \ref{sec:ColumnsEstME}, according to our current hypotheses there are global \up and \down side-influence forces that are responsible for this angle.

\begin{figure}[tb]
	\centering
	
	\begin{tabular}{cccc}

		\subfloat[\label{fig:ColsRes5-a}]{\includegraphics[width=0.4\textwidth]{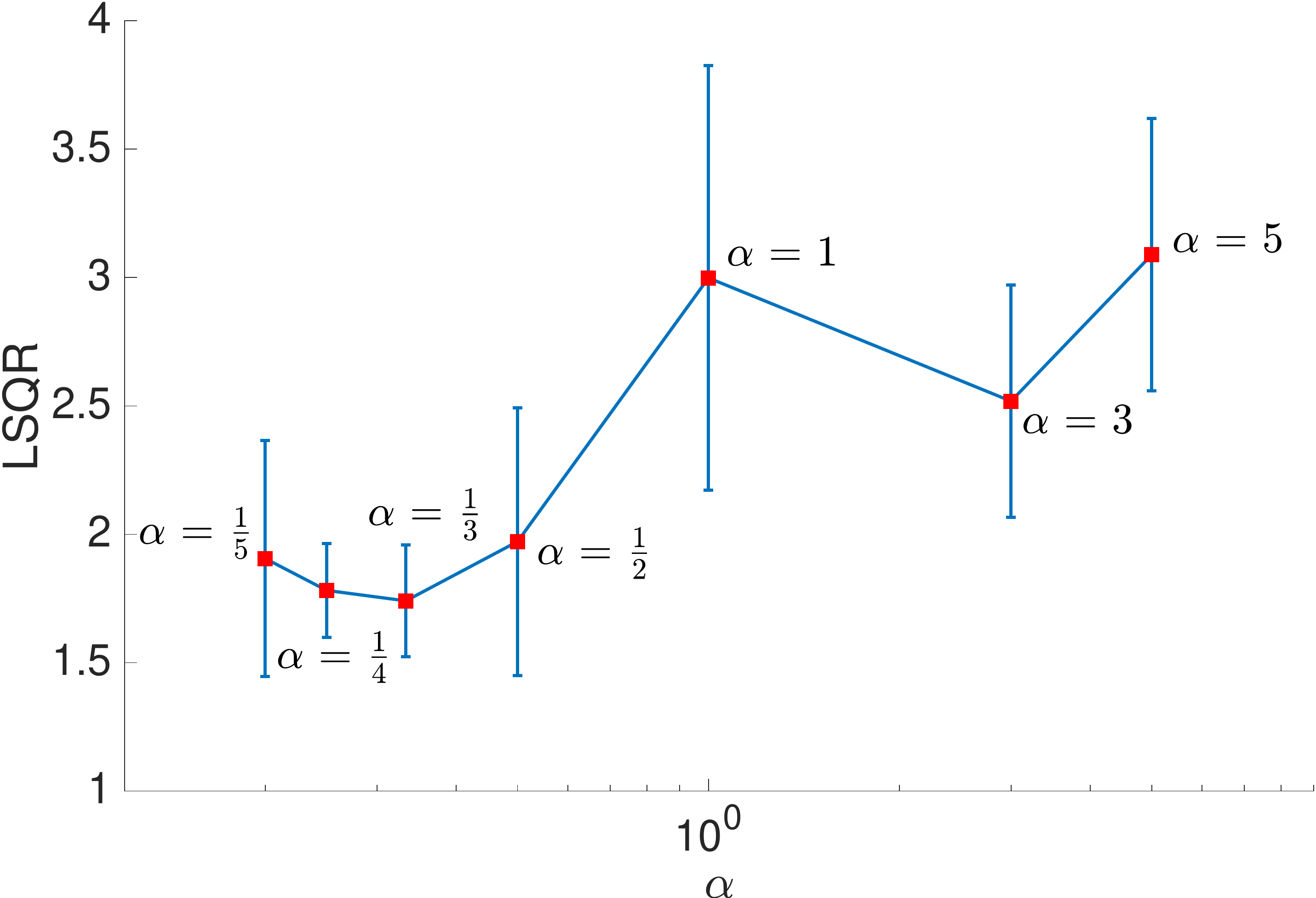}}
		&
		
		\subfloat[\label{fig:ColsRes5-b}]{\includegraphics[width=0.4\textwidth]{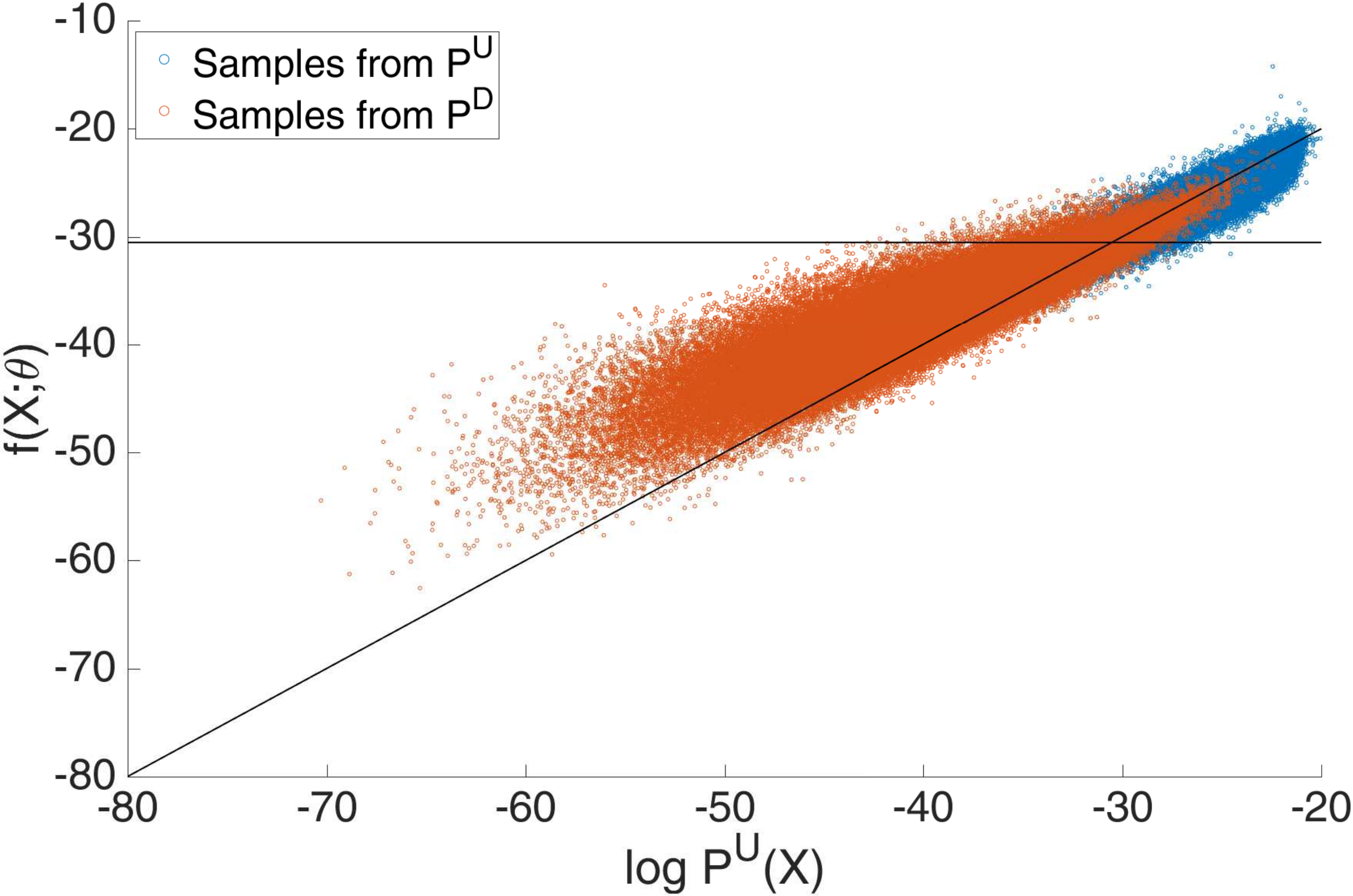}}
		
	\end{tabular}
	
	\protect
	\caption[Evaluation of PSO-LDE for estimation of \emph{Columns} distribution for different values of a hyper-parameter $\alpha$, where NN architecture is fully-connected with 4 layers of size 1024.]{Evaluation of PSO-LDE for estimation of \emph{Columns} distribution, where NN architecture is fully-connected with 4 layers of size 1024 (see Section \ref{sec:BDLayers}). (a) For different values of a hyper-parameter $\alpha$, $LSQR$ error is reported along with its empirical standard deviation.
		(b) Illustration of the learned surface $f_{\theta}(X)$. Blue points are sampled from $\probs{\usuff}$, while red points from $\probs{\dsuff}$. The $x$ axes represent $\log \probi{\usuff}{X}$ for each sample, $y$ axes - the surface height $f_{\theta}(X)$ after optimization was finished. Diagonal line represents $f_{\theta}(X) = \log \probi{\usuff}{X}$, where we would see all points in case of \emph{perfect} model inference. The black horizontal line represents $\log \probi{\dsuff}{X} = - 30.5$ which is constant for the Uniform density. As can be seen, there are high approximation errors at both $X^{\usuff}$ and $X^{\dsuff}$ locations. Compared with BD architecture in Figure \ref{fig:ColsRes3-b}, on average the error is much higher for FC network.
	}
	\label{fig:ColsRes5}
\end{figure}

\begin{figure}[!tbp]
	\centering
	
	\begin{tabular}{cc}

		\subfloat[\label{fig:ColsRes5.2-a}]{\includegraphics[width=0.45\textwidth]{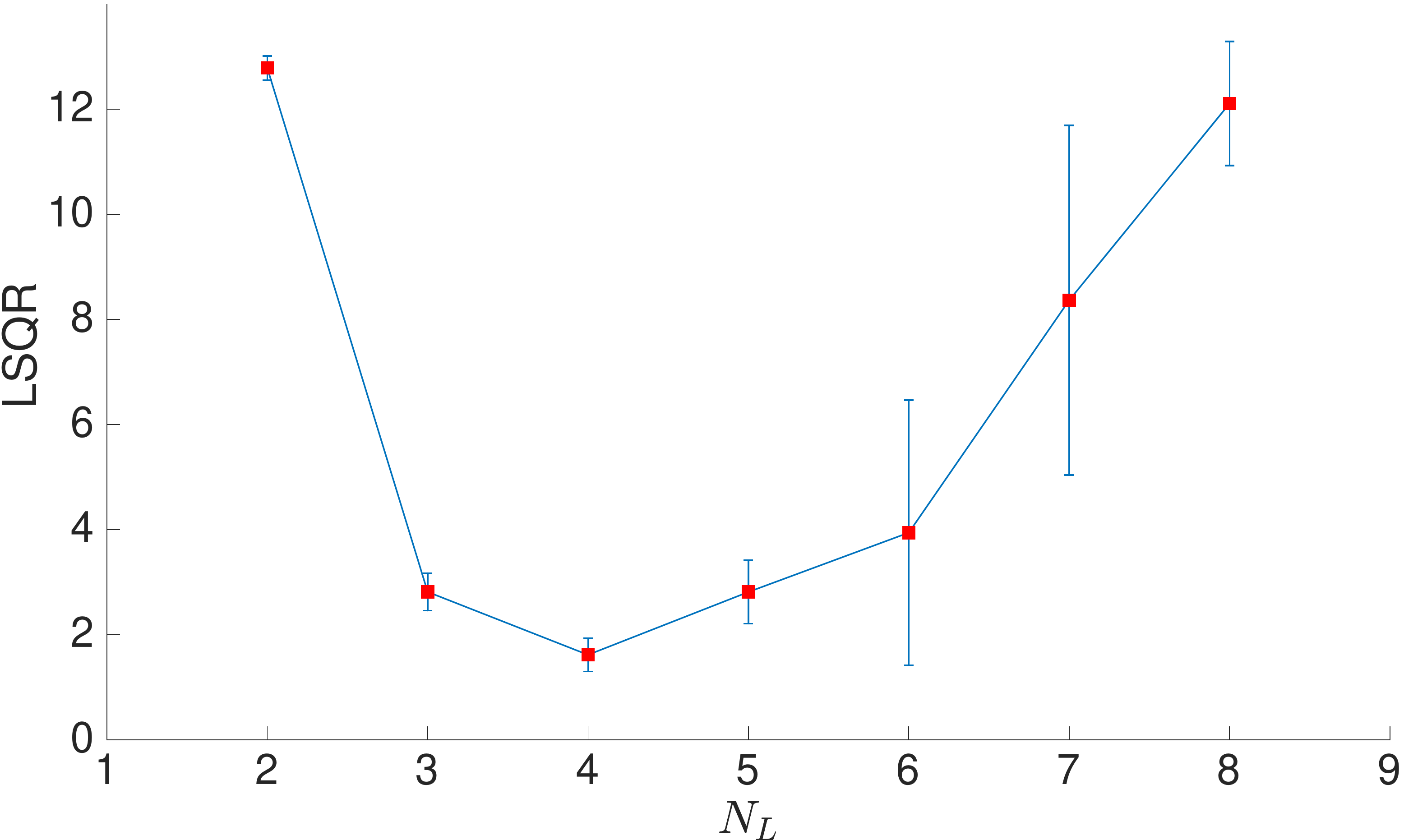}}

		&
		
		\subfloat[\label{fig:ColsRes5.2-b}]{\includegraphics[width=0.45\textwidth]{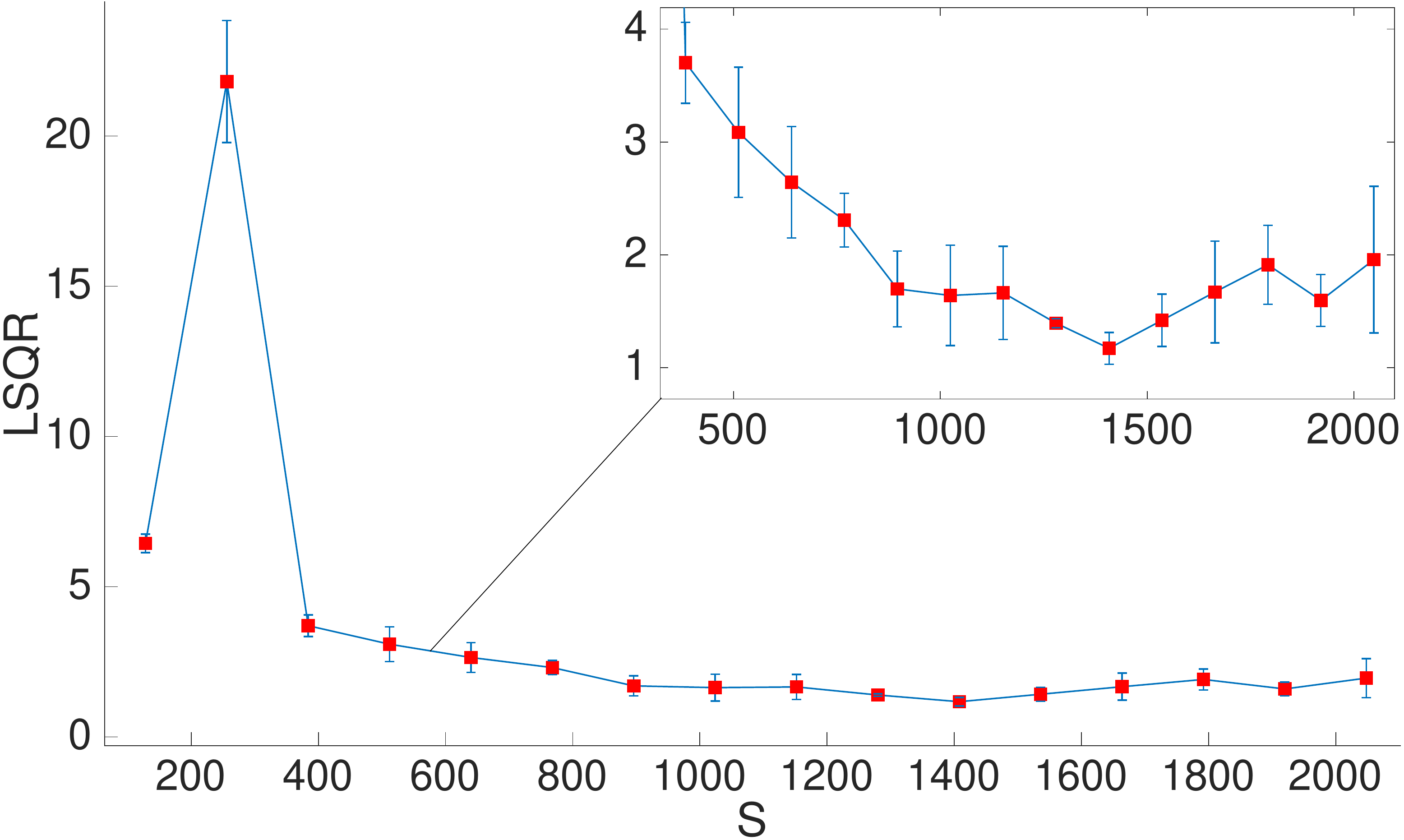}}

	\end{tabular}
	
	\protect
	\caption[Evaluation of PSO-LDE for estimation of \emph{Columns} distribution, where NN architecture is fully-connected with various sizes.]{Evaluation of PSO-LDE for estimation of \emph{Columns} distribution, where NN architecture is fully-connected (see Section \ref{sec:BDLayers}). The applied loss is PSO-LDE with $\alpha = \frac{1}{4}$. (a) For different number of layers $N_L$, $LSQR$ error is reported, where a size of each layer is $S = 1024$.
		(b) For different values of layer size $S$, $LSQR$ error is reported, where a number of layers is $N_L = 4$.
	}
	\label{fig:ColsRes5.2}
\end{figure}

Further, to ensure that FC architecture can not produce any better results for the given inference task, we also evaluate it for different values of $N_L$ and $S$ - number of layers and size of each layer respectively. In Figure \ref{fig:ColsRes5.2} we see that $N_L = 4$ and $S = 1408$ achieve best results for FC NN. Yet, the achieved performance is only $LSQR = 1.17$, which is still nowhere near the accuracy of BD architecture.

\begin{figure}[!tbp]
	\centering
	
	\begin{tabular}{c}

		\subfloat[\label{fig:ColsRes6-a}]{\includegraphics[width=0.95\textwidth]{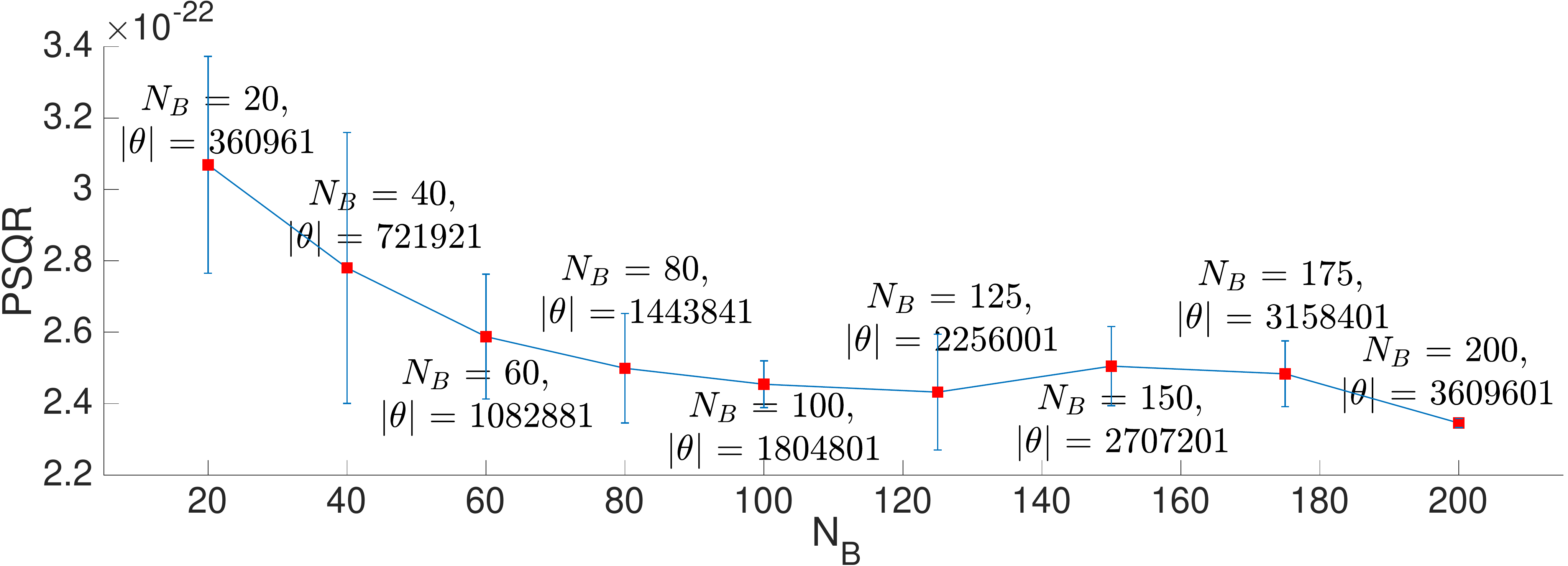}}
		\\
		
		\subfloat[\label{fig:ColsRes6-b}]{\includegraphics[width=0.95\textwidth]{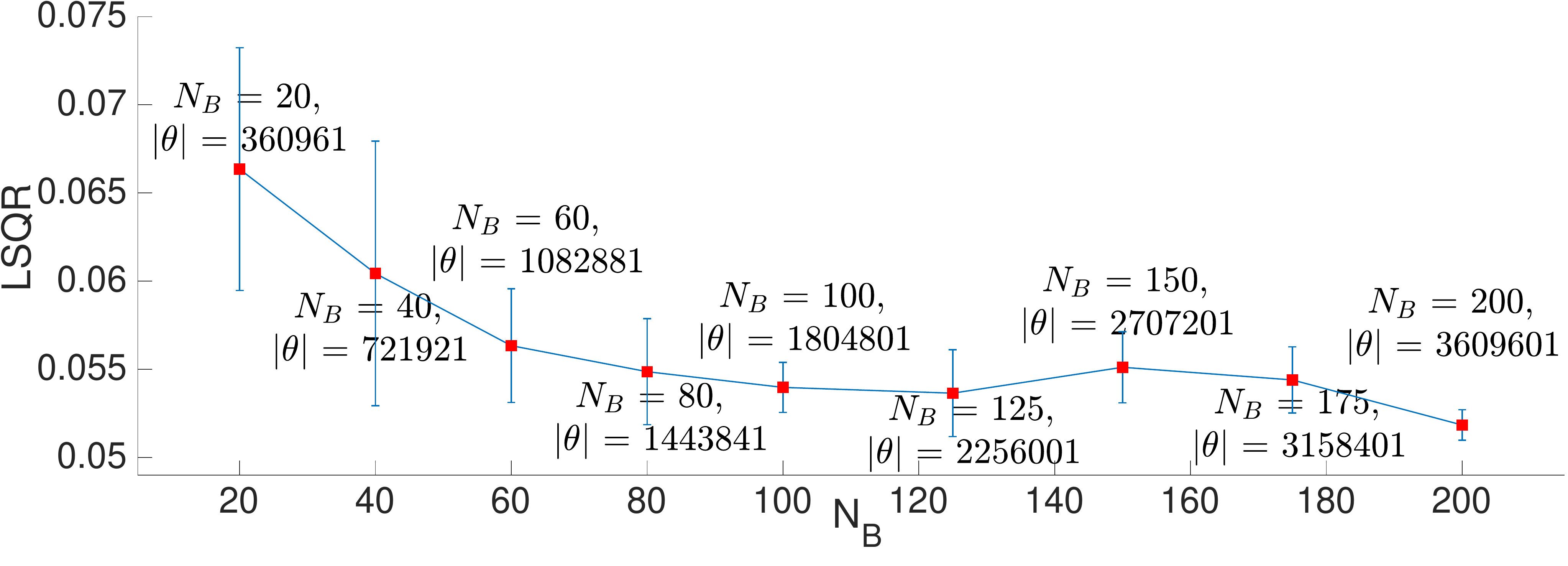}}
		
		\\

\subfloat[\label{fig:ColsRes6-c}]{\includegraphics[width=0.95\textwidth]{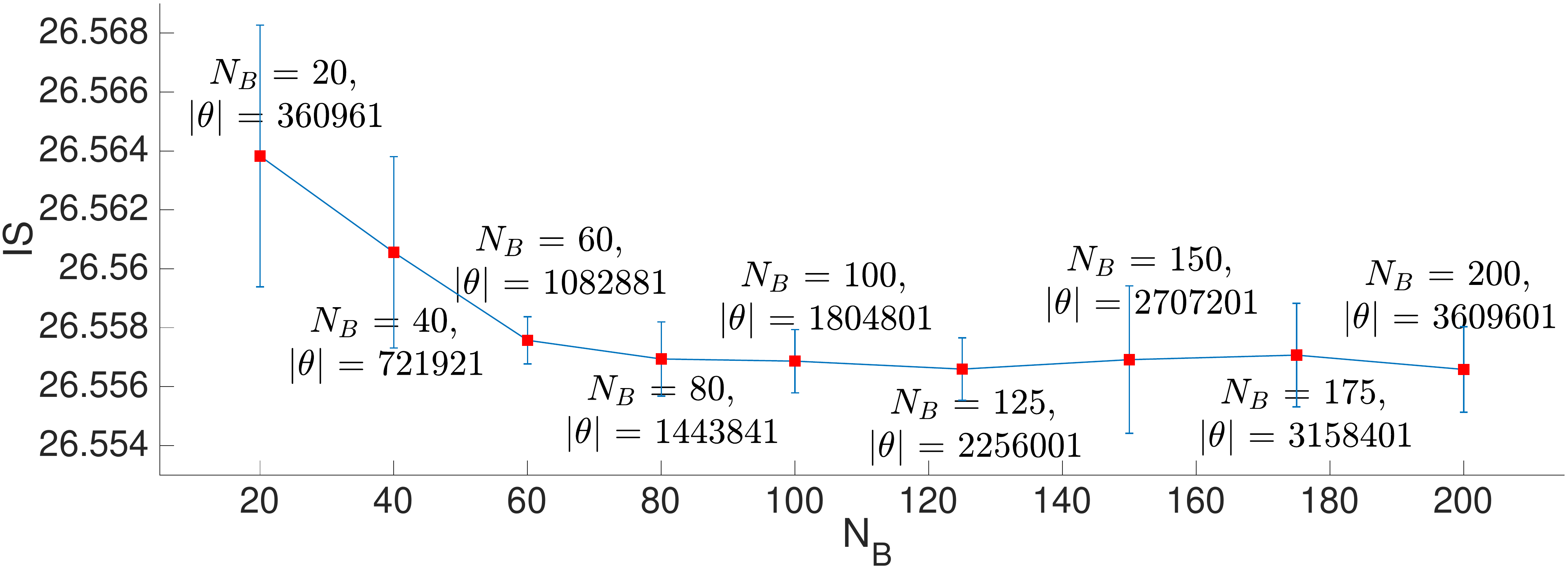}}

	\end{tabular}
	
	\protect
	\caption[Evaluation of PSO-LDE for estimation of \emph{Columns} distribution, where NN architecture is block-diagonal with various block numbers.]{Evaluation of PSO-LDE for estimation of \emph{Columns} distribution, where NN architecture is block-diagonal with 6 layers and block size $S_B = 64$ (see Section \ref{sec:BDLayers}). The number of blocks $N_B$ is changing. The applied loss is PSO-LDE with $\alpha = \frac{1}{4}$.
		For different values of $N_B$, (a) $PSQR$ (b) $LSQR$ and (c) $IS$ are reported. As observed, the bigger number of blocks (e.g. independent channels) $N_B$ improves the pdf inference.
	}
	\label{fig:ColsRes6}
\end{figure}

\paragraph{BD Architecture}
Further, we performed learning with a BD architecture, but with increasing number of blocks $N_B$. For $N_B$ taking values between 20 and 200, in Figure \ref{fig:ColsRes6} we can see that with bigger $N_B$ there is improvement in approximation accuracy. This can be explained by the fact that bigger $N_B$ produces bigger number of independent transformation channels inside NN; with more such channels there is less parameter sharing and side-influence between far away input regions - different regions on average rely on different transformation channels. As a result, the NN becomes highly flexible.
Further, in the setting of infinite dataset such high NN flexibility is desirable, and leads to a higher approximation accuracy. In contrast, in Section \ref{sec:ColumnsEstSmallDT} below we will see that for a smaller dataset size the relation between NN flexibility and the accuracy is very different.

\begin{figure}[tb]
	\centering
	
	\begin{tabular}{c}

		\subfloat[\label{fig:ColsRes6.0-a}]{\includegraphics[width=0.75\textwidth]{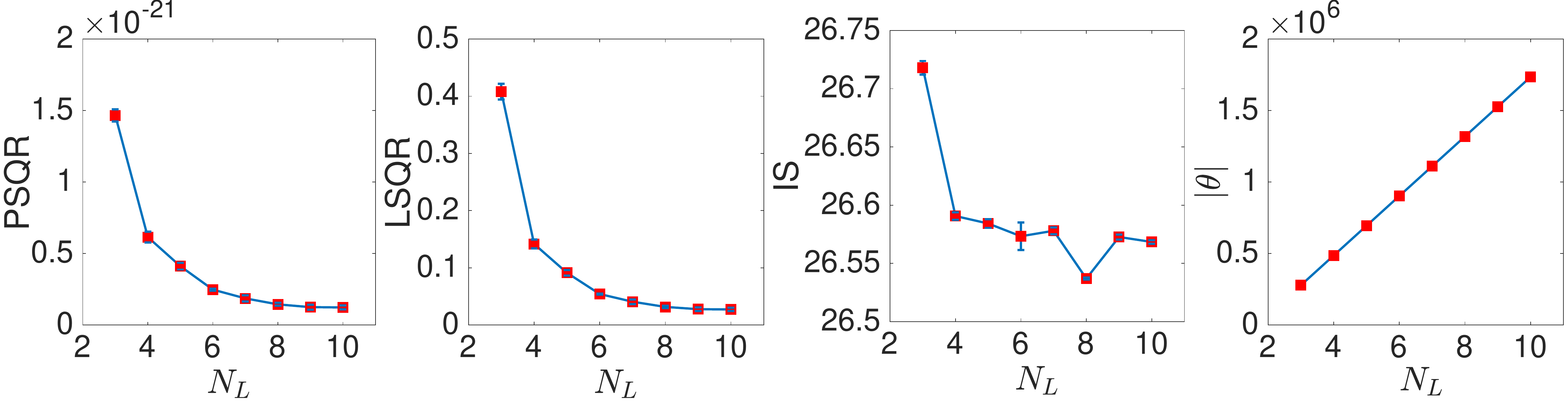}}
		\\
		
		\subfloat[\label{fig:ColsRes6.0-b}]{\includegraphics[width=0.408\textwidth]{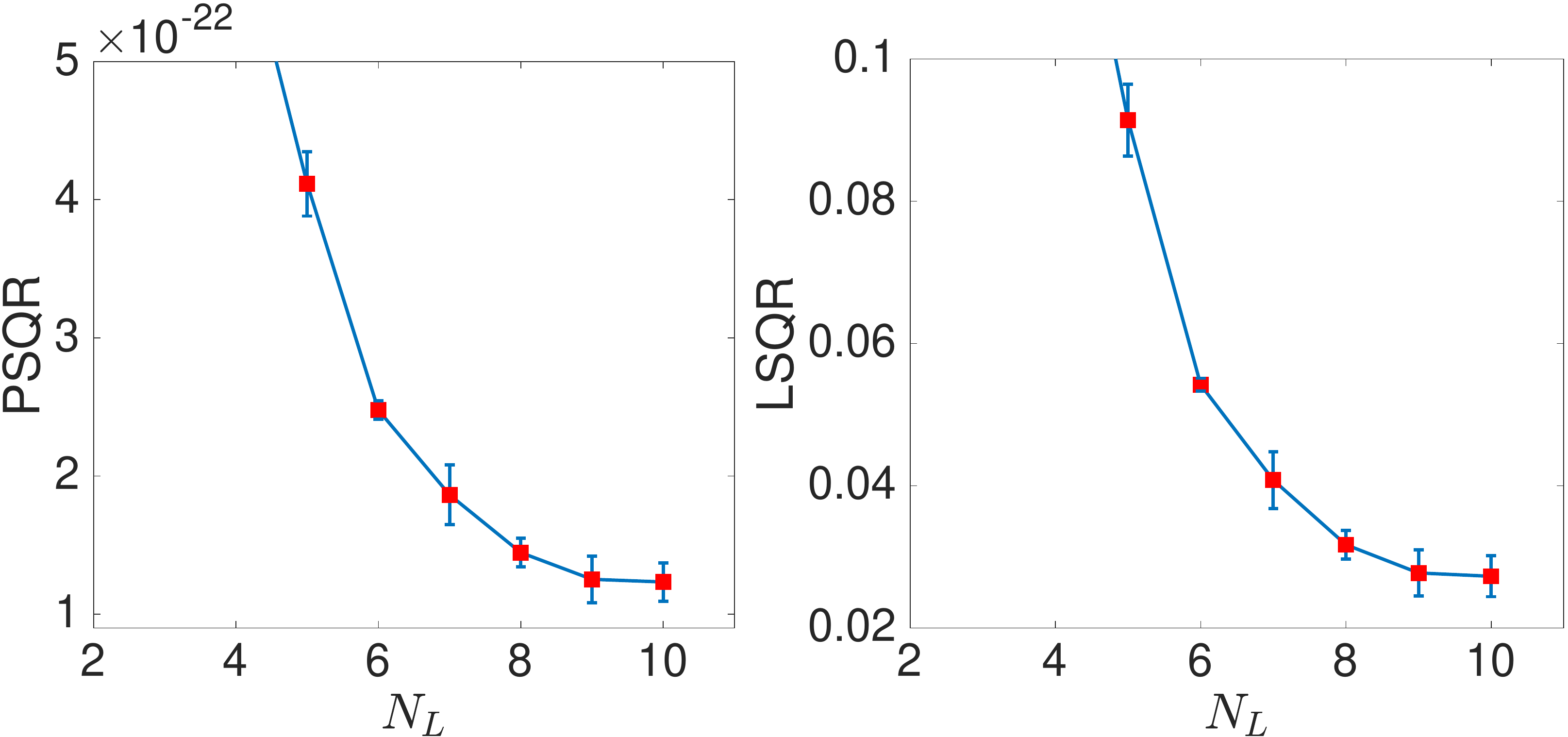}}
		
	\end{tabular}
	
	\protect
	\caption[Evaluation of PSO-LDE for estimation of \emph{Columns} distribution, where NN architecture is block-diagonal with various numbers of layers.]{Evaluation of PSO-LDE for estimation of \emph{Columns} distribution, where NN architecture is block-diagonal with number of blocks $N_B = 50$ and block size $S_B = 64$ (see Section \ref{sec:BDLayers}). The number of layers $N_L$ is changing from 3 to 10. The applied loss is PSO-LDE with $\alpha = \frac{1}{4}$.
		(a) For different values of $N_L$ we report $PSQR$, $LSQR$ and $IS$, and their empirical standard deviation. Additionally, in the last column we depict the size of $\theta$ for each value of $N_L$.
		(b) Zoom of (a).
	}
	\label{fig:ColsRes6.0}
\end{figure}

\begin{figure}[tb]
	\centering

	\begin{tabular}{c}

		\subfloat[\label{fig:ColsRes6.01-a}]{\includegraphics[width=0.75\textwidth]{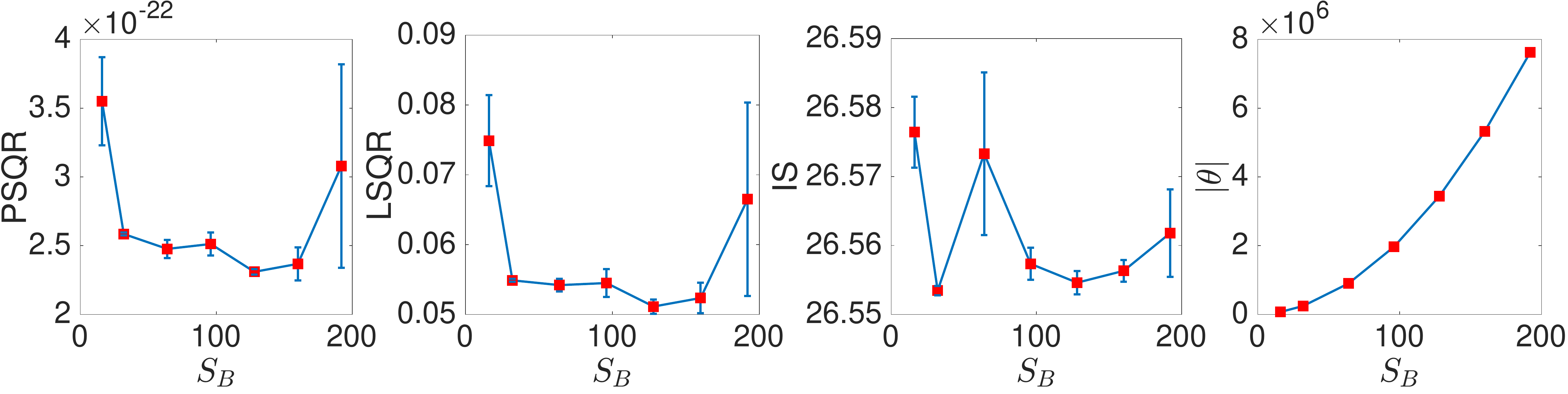}}
		\\
		
		\subfloat[\label{fig:ColsRes6.01-b}]{\includegraphics[width=0.4099\textwidth]{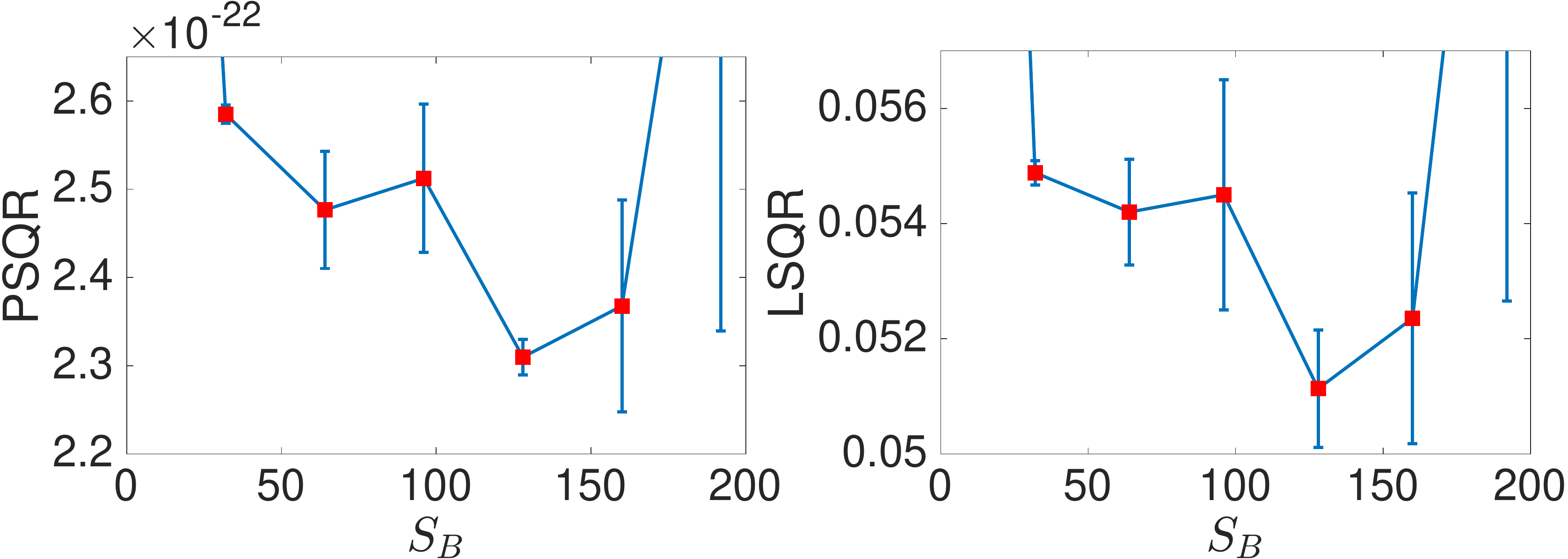}}
		
	\end{tabular}
	
	\protect
	\caption[Evaluation of PSO-LDE for estimation of \emph{Columns} distribution, where NN architecture is block-diagonal with various block sizes.]{Evaluation of PSO-LDE for estimation of \emph{Columns} distribution, where NN architecture is block-diagonal with 6 layers and number of blocks $N_B = 50$ (see Section \ref{sec:BDLayers}). The block size $S_B$ is taking values $\{16, 32, 64, 96, 128, 160, 192 \}$. The applied loss is PSO-LDE with $\alpha = \frac{1}{4}$.
		(a) For different values of $S_B$ we report $PSQR$, $LSQR$ and $IS$, and their empirical standard deviation. Additionally, in the last column we depict the size of $\theta$ for each value of $S_B$.
		(b) Zoom of (a).
	}
	\label{fig:ColsRes6.01}
\end{figure}

Likewise, we experiment with number of layers $N_L$ to see how the  network depth of a BD architecture affects the accuracy of pdf inference. In Figure \ref{fig:ColsRes6.0} we see that deeper networks allow us to further decrease $LSQR$ error to around $0.03$. Also, we can see that at some point increasing $N_L$ causes only a slight error improvement. Thus, increasing $N_L$ beyond that point is not beneficial, since for a very small error reduction we will pay with higher computational cost due to the increased size of $\theta$.

Furthermore, in Figure \ref{fig:ColsRes6.01} we evaluate BD performance for different sizes of blocks $S_B$. Here we don't see anymore a monotonic error decrease that we observed above for $N_B$ and $N_L$. The error is big for $S_B$ below 32 or above 160. This probably can be explained as follows. For a small block size $S_B$ each independent channel has too narrow width that is not enough to properly transfer the required signal from NN input to output. Yet, surprisingly it still can achieve a very good approximation, yielding $LSQR$ error of 0.075 for $S_B = 16$ with only 72001 weights, which is very impressive for such small network. Further, for a large block size $S_B$ each independent channel becomes too wide, with information from too many various regions in $\RR^n$ passing through it. This in  turn causes interference (side-influence) between different regions and reduces overall NN flexibility, similarly to what is going on inside a regular FC network.

\begin{table}

	\centering
	\begin{tabular}{lllll}
		\toprule
		Method     & $PSQR$ & $LSQR$     & $IS$ \\
		\midrule
		BD, \cmark DT, \cmark HB & 
		$2.48 \cdot 10^{-22} \pm 6.63 \cdot 10^{-24}$ &
		$0.054 \pm 0.0009$  &
		$26.57 \pm 0.01$  \\
		BD, \xmark DT, \cmark HB  & 
		$2.49 \cdot 10^{-22} \pm 9.8 \cdot 10^{-24}$ &
		$0.055 \pm 0.0021$  &
		$26.57 \pm 0.002$  \\
		BD, \cmark DT, \xmark HB & 
		$2.51 \cdot 10^{-22} \pm 1.1 \cdot 10^{-23}$ &
		$0.056 \pm 0.0026$  &
		$26.57 \pm 0.002$  \\
		BD, \xmark DT, \xmark HB & 
$3 \cdot 10^{-22} \pm 4.79 \cdot 10^{-23}$ &
$0.066 \pm 0.011$  &
$26.57 \pm 0.005$  \\

		\midrule
		FC, \cmark DT, \cmark HB & 
		$5.56 \cdot 10^{-18} \pm 1.02 \cdot 10^{-17}$ &
		$1.78 \pm 0.18$  &
		$27.29 \pm 0.05$  \\
		FC, \xmark DT, \cmark HB  & 
		$1.2 \cdot 10^{-16} \pm 2.67 \cdot 10^{-16}$ &
		$1.35 \pm 0.058$  &
		$27.157 \pm 0.029$  \\
		FC, \cmark DT, \xmark HB & 
		$7.9 \cdot 10^{-21} \pm 2.37 \cdot 10^{-21}$ &
		$2.38 \pm 0.3$  &
		$27.52 \pm 0.08$  \\
		FC, \xmark DT, \xmark HB & 
		$1.36 \cdot 10^{-12} \pm 3 \cdot 10^{-12}$ &
		$2.5 \pm 0.23$  &
		$27.54 \pm 0.08$  \\

		\bottomrule
	\end{tabular}
	
	\caption[Performance comparison between various NN pre-conditioning ways.]{Performance comparison between various NN pre-conditioning ways. The pdf function of \emph{Columns} distribution is learned by PSO-LDE with $\alpha = \frac{1}{4}$. The applied models are fully-connected (FC) with 4 layers of size 1024, and block-diagonal (BD) with 6 layers, number of blocks $N_B = 50$ and block size $S_B = 64$ (see Section \ref{sec:BDLayers}). Evaluated pre-conditioning techniques are the data normalization in Eq.~(\ref{eq:NormData}) (DT), and the height bias in Eq.~(\ref{eq:BiasNN}) (HB).}
	\label{tbl:Perf2}
	
\end{table}

\paragraph{NN Pre-conditioning}

Finally, we verified efficiency of pre-conditioning techniques proposed in Section \ref{sec:PrecondNN}, namely the data normalization in Eq.~(\ref{eq:NormData}) and the height bias in Eq.~(\ref{eq:BiasNN}). In Table \ref{tbl:Perf2} we see that both methods improve the estimation accuracy. Further, in case the used model is FC, the $LSQR$ error improvement produced by the height bias is much more significant. Yet, for FC architecture it is unclear if the data normalization indeed helps.

Overall, our experiments \textbf{combined} with empirical observations from Section \ref{sec:BDLayersCompare} show that BD architecture has a smaller side-influence (small values of $g_{\theta}(X, X')$ for $X \neq X'$) and a higher flexibility than FC architecture. This in turn yields superior accuracy for BD vs FC networks.
Moreover, in an infinite dataset setting we can see that further increase of NN flexibility by increasing $N_B$ or $N_L$ yields even a better approximation accuracy. The block size $S_B$ around 64 produces a better performance in general. Yet, its small values are very attractive since they yield small networks with appropriate computational benefits, with relatively only a little error increase.

\subsubsection{Batch Size Impact}
\label{sec:BatchSizeI}

\begin{figure}
	\centering
	
	\newcommand{\width}[0] {0.44}
	\newcommand{\height}[0] {0.16}

	\begin{tabular}{cc}

		\subfloat[\label{fig:BatchRes1-a}]{\includegraphics[height=\height\textheight,width=\width\textwidth]{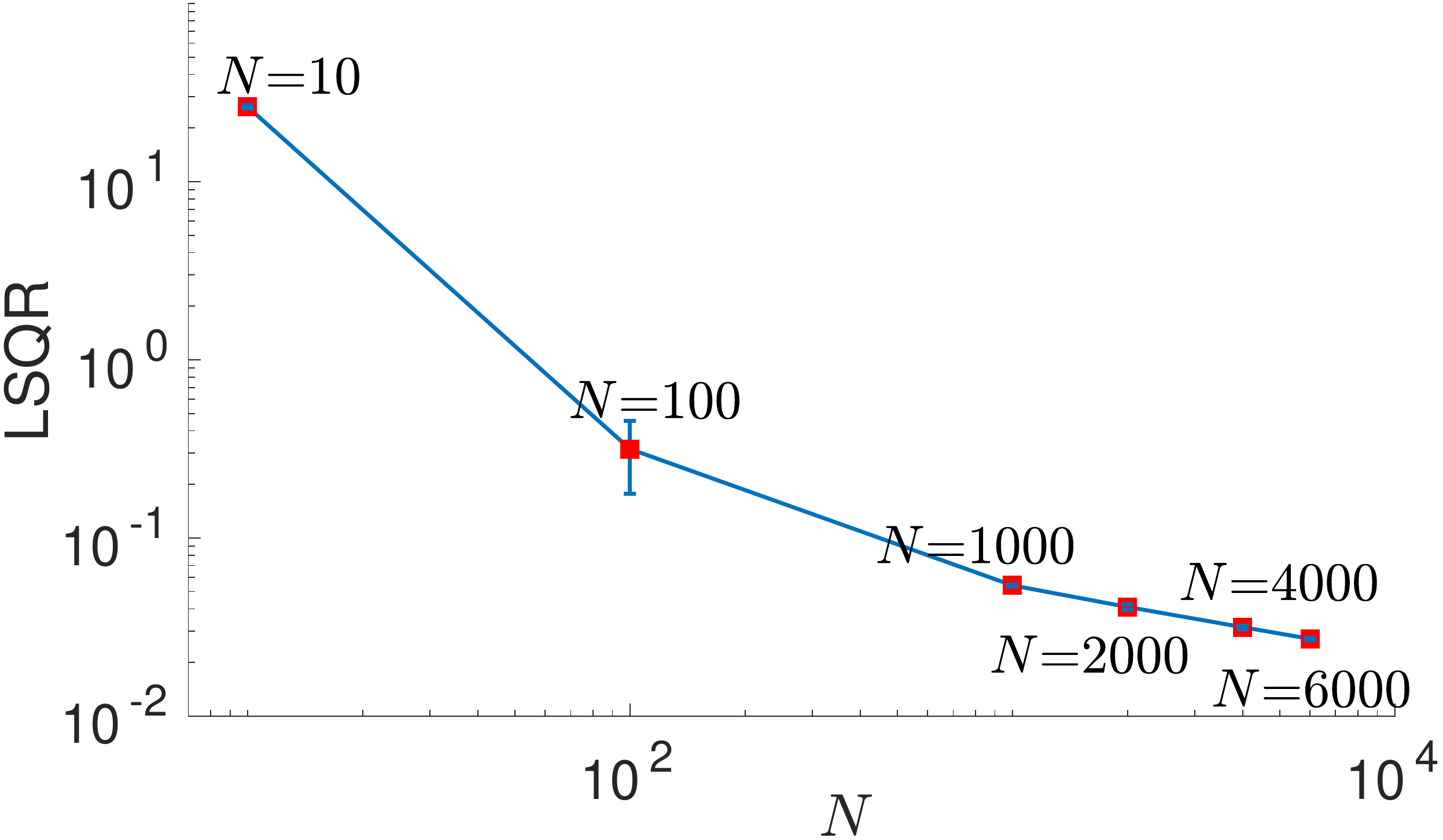}}

&

		\subfloat[\label{fig:BatchRes1-b}]{\includegraphics[height=\height\textheight,width=\width\textwidth]{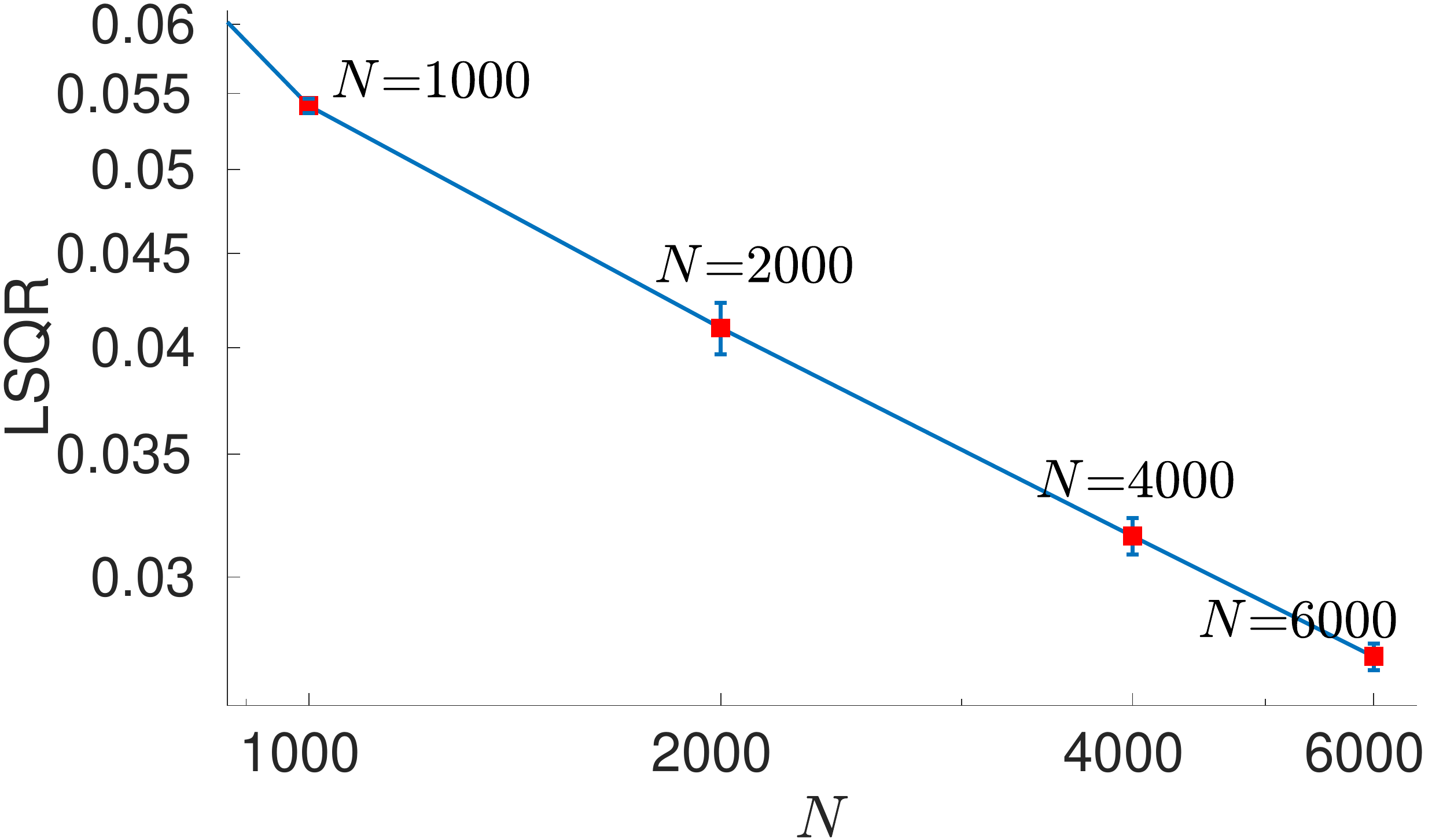}}
		
		\\
		
		\subfloat[\label{fig:BatchRes1-c}]{\includegraphics[height=\height\textheight,width=\width\textwidth]{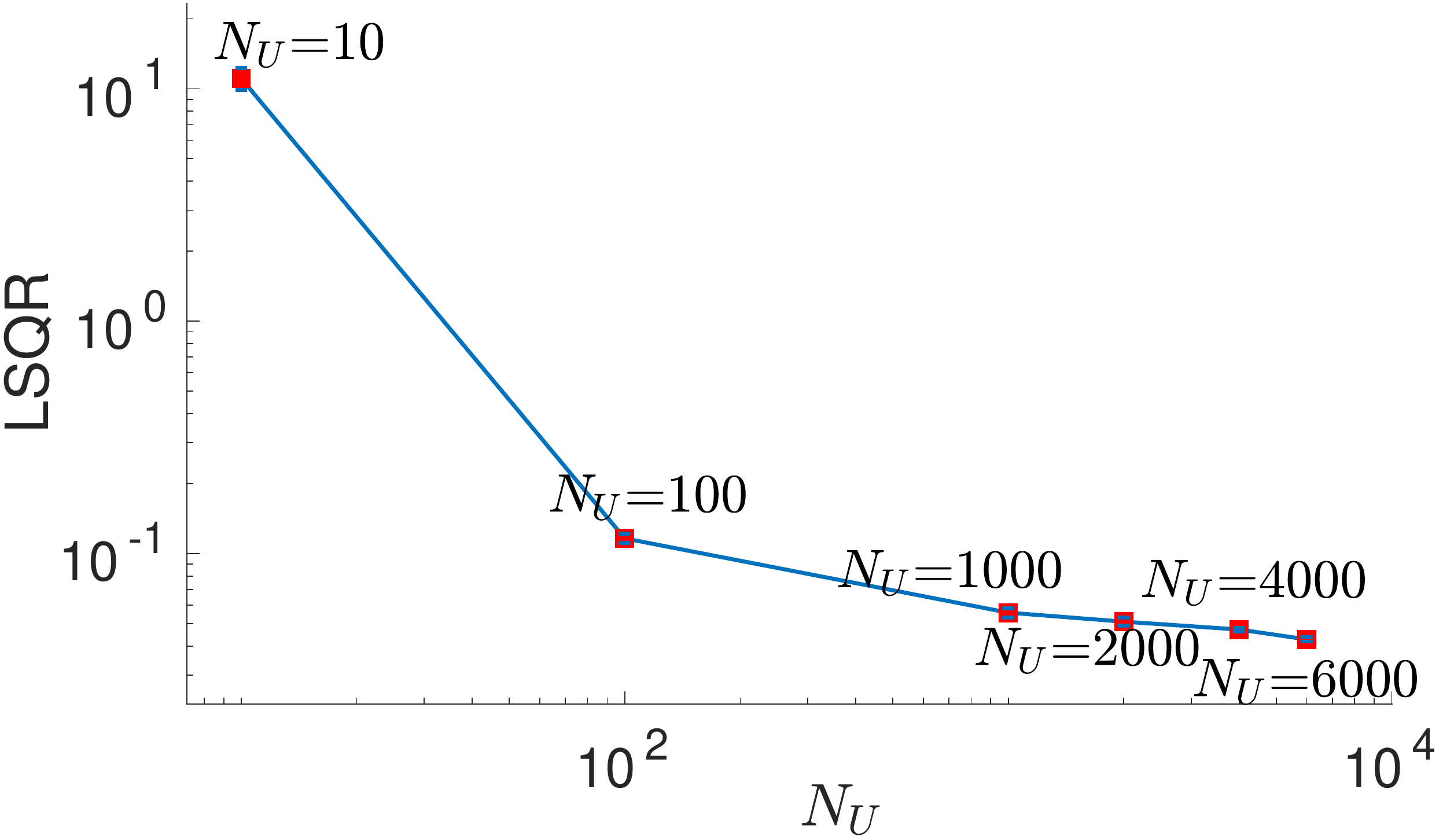}}
		
		&
		
		\subfloat[\label{fig:BatchRes1-d}]{\includegraphics[height=\height\textheight,width=\width\textwidth]{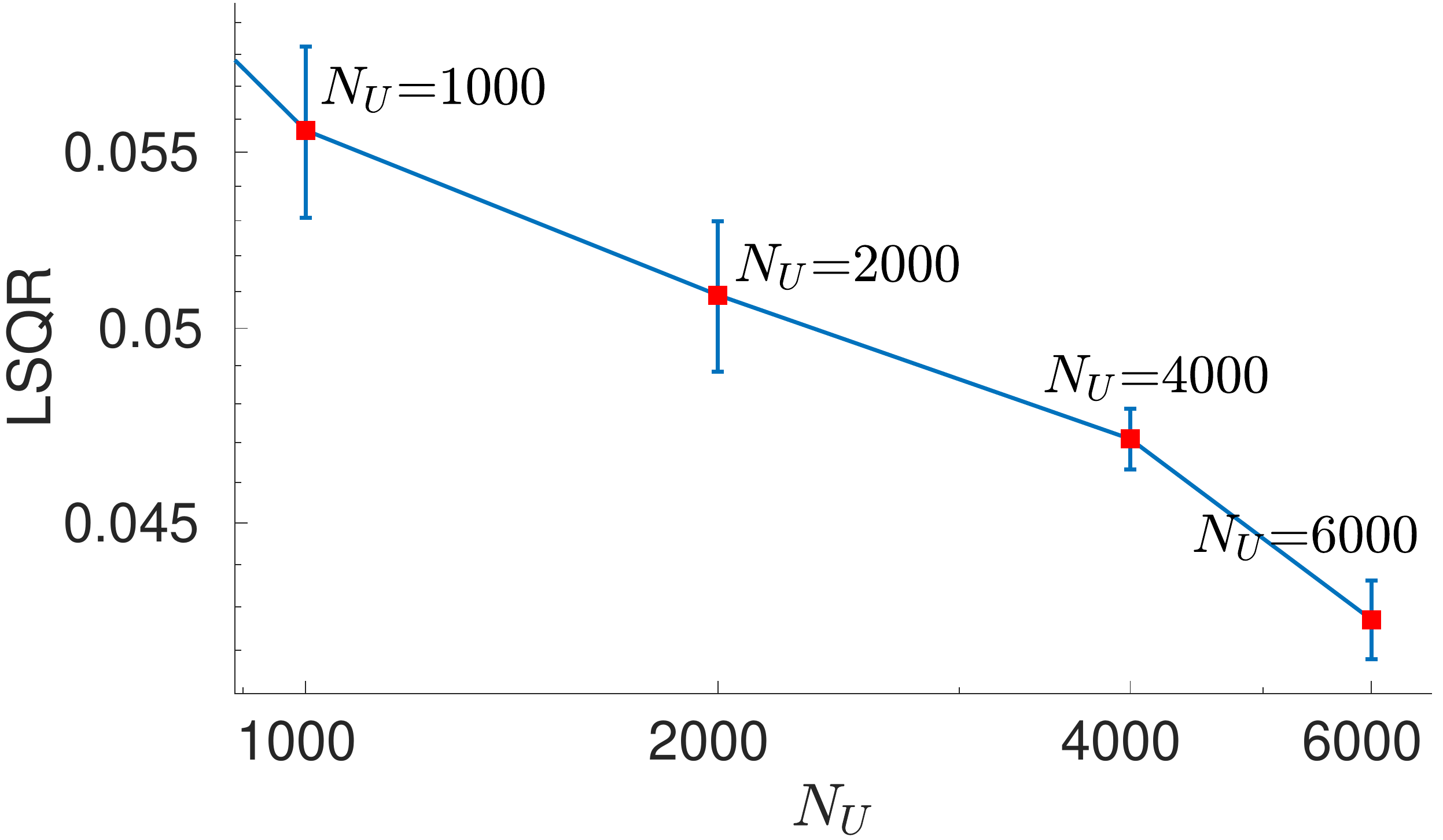}}

		\\
		
		\subfloat[\label{fig:BatchRes1-e}]{\includegraphics[height=\height\textheight,width=\width\textwidth]{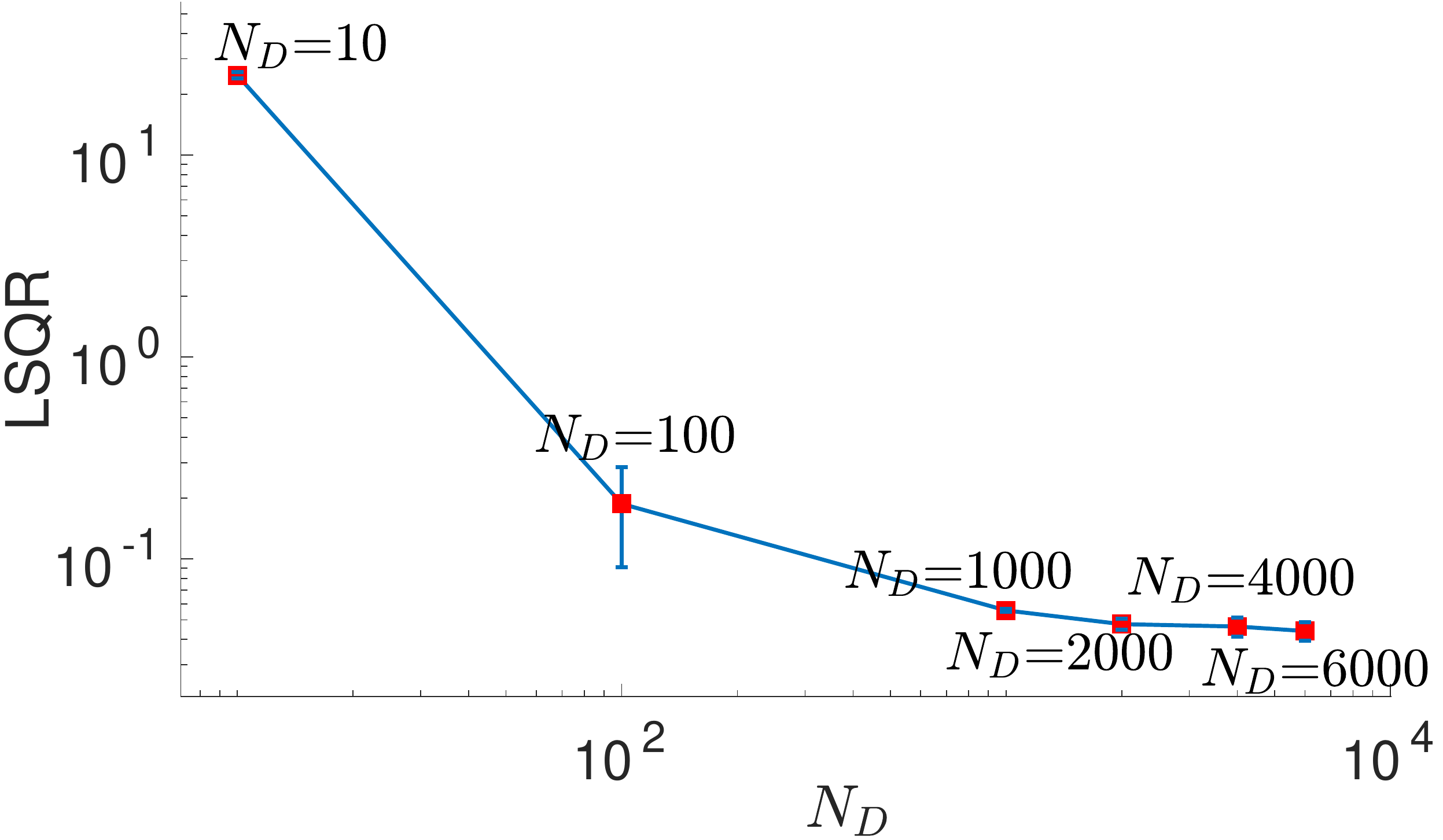}}
		
		&
		
		\subfloat[\label{fig:BatchRes1-f}]{\includegraphics[height=\height\textheight,width=\width\textwidth]{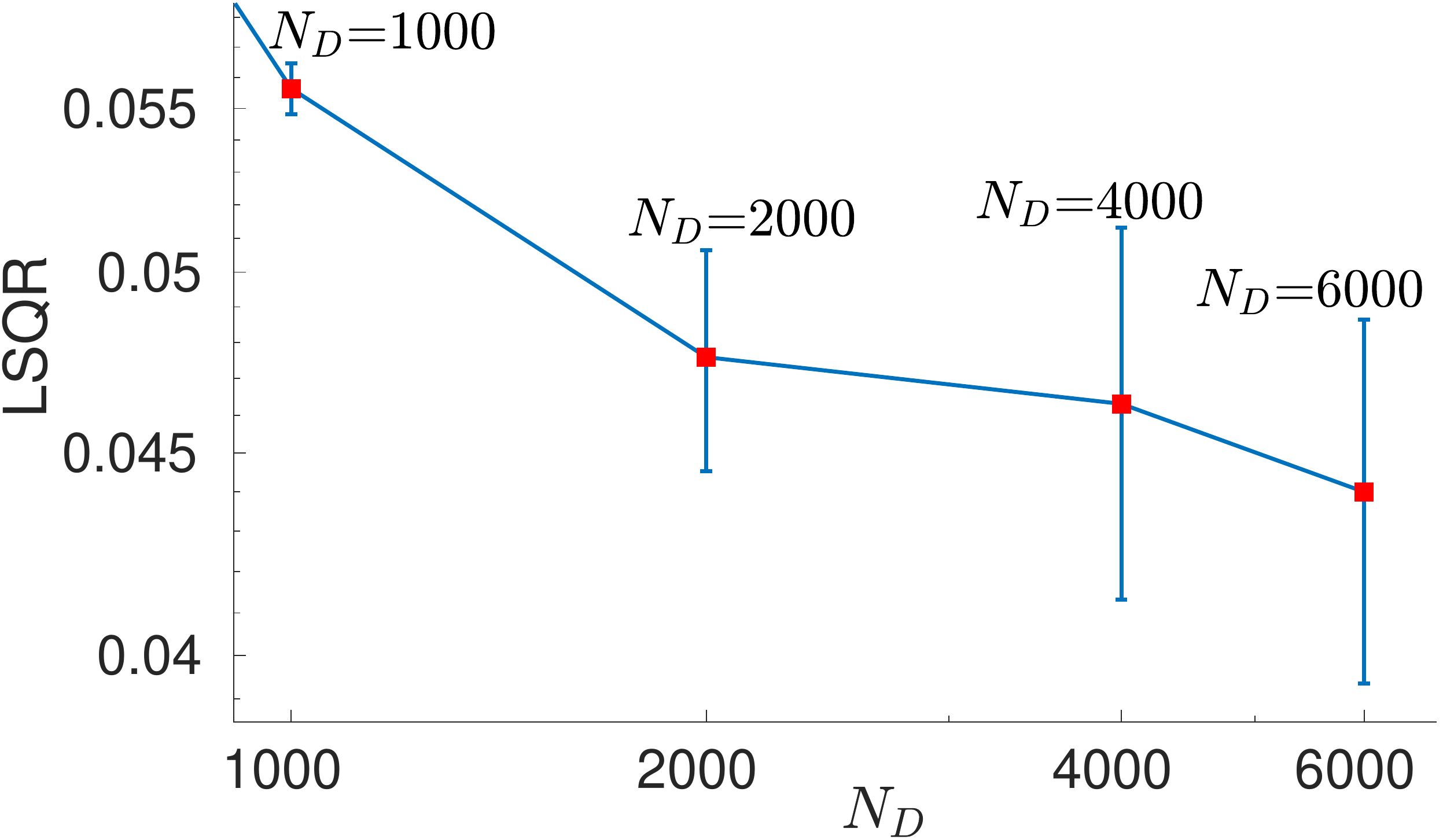}}
		
		\\
		
		\subfloat[\label{fig:BatchRes1-g}]{\includegraphics[height=\height\textheight,width=\width\textwidth]{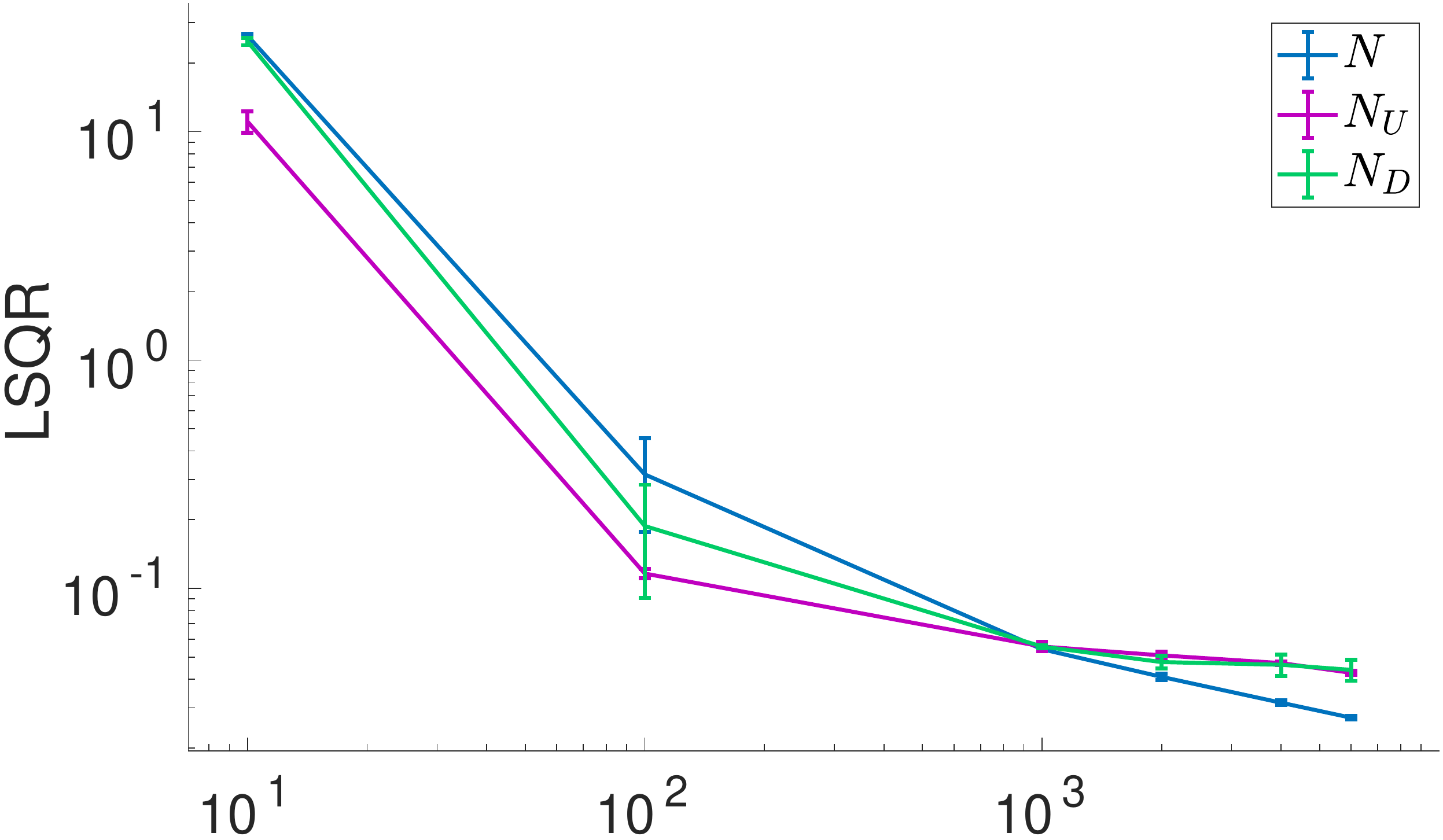}}
		
		&
		
		\subfloat[\label{fig:BatchRes1-h}]{\includegraphics[height=\height\textheight,width=\width\textwidth]{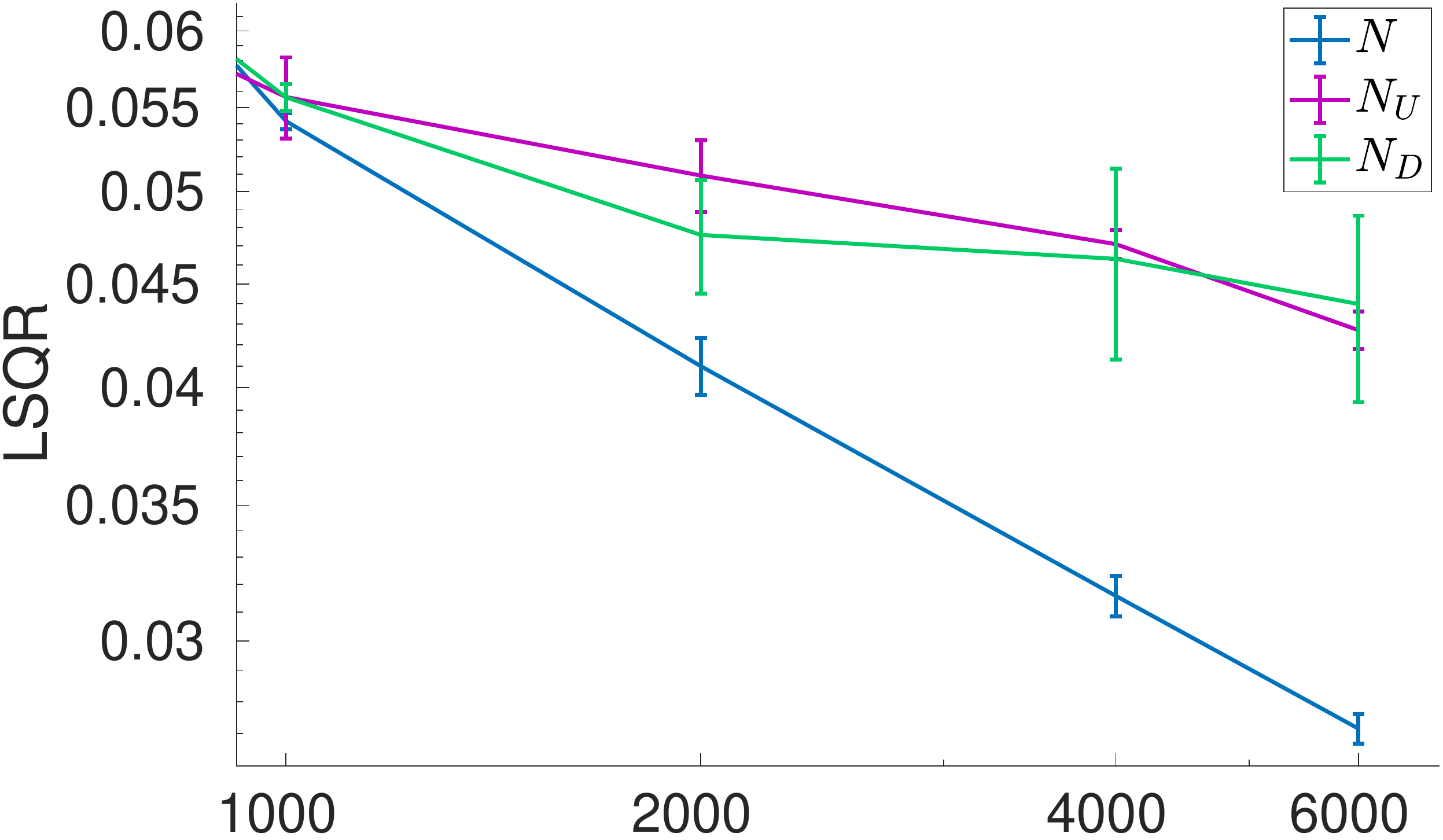}}

	\end{tabular}
	
	\protect
	\caption[Evaluation of PSO-LDE for various batch sizes.]{Batch size evaluation. \emph{Columns} distribution is estimated via PSO-LDE with $\alpha = \frac{1}{4}$, where NN architecture is block-diagonal with 6 layers, number of blocks $N_B = 50$ and block size $S_B = 64$ (see Section \ref{sec:BDLayers}).
	(a)-(b) For different values of a batch size $N = N^{\usuff} = N^{\dsuff}$ we report $LSQR$ and its empirical standard deviation. Both the \up batch size $N^{\usuff}$ and the \down batch size $N^{\dsuff}$ are kept the same.
	(c)-(d) The $N^{\usuff}$ receives different values while the $N^{\dsuff}$ is 1000.
	(e)-(f) The $N^{\dsuff}$ receives different values while the $N^{\usuff}$ is 1000.
	(g)-(h) All scenarios are plotted together in the same graph.
	Note that both $x$ and $y$ axes are log-scaled. Right column is zoom-in of left column.
	}
	\label{fig:BatchRes1}
\end{figure}

Herein we investigate the relation between PSO approximation error and a batch size of training points. In particular, in PSO loss we have two terms, the \up term which is a sum over batch points $\{X^{\usuff}_{i}\}_{i = 1}^{N^{\usuff}}$ and the \down term which is a sum over batch points $\{X^{\dsuff}_{i}\}_{i = 1}^{N^{\dsuff}}$. We will see how values of $N^{\usuff}$ and $N^{\dsuff}$ affect PSO performance.

\paragraph{Increasing both \boldmath$N^{\usuff}$ and $N^{\dsuff}$}
First, we run a scenario where both batch sizes are the same, $N^{\usuff} = N^{\dsuff} = N$. We infer \emph{Columns} distribution with different values of $N$, ranging from 10 to 6000. In Figures \ref{fig:BatchRes1-a}-\ref{fig:BatchRes1-b} we can observe that $LSQR$ error is decreasing for bigger $N$, which is expected since then the stochastic forces at each point $X \in \RR^n$ are getting closer to the averaged forces $F_{\theta}^{\usuff}(X) = 
\probi{\usuff}{X} \cdot 
M^{\usuff}\left[X,f_{\theta}(X)\right]$ and $F_{\theta}^{\dsuff}(X) = 
\probi{\dsuff}{X} \cdot 
M^{\dsuff}\left[X,f_{\theta}(X)\right]$. In other words, the sampled approximation of PSO gradient in Eq.~(\ref{eq:GeneralPSOLossFrml_Limit}) becomes more accurate.

Moreover, we can observe that for the smaller batch size the actual PSO-LDE performance is very poor, with $LSQR$ being around 26.3 for $N = 10$ and decreasing to 0.31 for $N = 100$. The high accuracy, in range 0.03-0.05, is only achieved when we increase the number of batch points to be above 1000. This implies that to reach a higher accuracy, PSO will require a higher demand over the memory/computation resources. Therefore, the higher available resources, expected from future GPU cards, will lead to a higher PSO accuracy.

\paragraph{Increasing only \boldmath$N^{\usuff}$}
Further, we experiment with increasing/decreasing only one of the batch sizes while the other stays constant. In Figures \ref{fig:BatchRes1-c}-\ref{fig:BatchRes1-d}  a scenario is depicted where $N^{\usuff}$ is changing while $N^{\dsuff}$ is 1000. Its error for small values of $N^{\usuff}$ is smaller than in the previous scenario, with $LSQR$ being around 11 for $N^{\usuff} = 10$ and decreasing to 0.11 for $N^{\usuff} = 100$. Comparing with the previous experiment, we can see that even if $N^{\usuff}$ is small, a high value of $N^{\dsuff}$ (1000) improves the optimization performance.

\paragraph{Increasing only \boldmath$N^{\dsuff}$}
Furthermore, in Figures \ref{fig:BatchRes1-e}-\ref{fig:BatchRes1-f} we depict the opposite scenario where $N^{\dsuff}$ is changing while $N^{\usuff}$ is 1000. Unlike the experiment in Figures \ref{fig:BatchRes1-c}-\ref{fig:BatchRes1-d}, here the improvement of error for small values of $N^{\dsuff}$ w.r.t. the first experiment is not that significant. For $N^{\dsuff} = 10$ the error is 24.8 and for $N^{\dsuff} = 100$ it is 0.19. Hence, the bigger number of \up points ($N^{\usuff} = 1000$) does not lead to a much higher accuracy if a number of \down points $N^{\dsuff}$ is stll too small.

Finally, in Figures \ref{fig:BatchRes1-g}-\ref{fig:BatchRes1-h} we plot all three experiments together. Note that all lines cross at the same point, where all experiments were configured to have $N^{\usuff} = N^{\dsuff} = 1000$. We can see that in case our resource budget is low (smaller values of $x$ in Figure \ref{fig:BatchRes1-g}), it is more efficient to spend them to increase $N^{\dsuff}$. Yet, for a high overall resource budget (higher values of $x$ in Figure \ref{fig:BatchRes1-g}) both $N^{\usuff}$ and $N^{\dsuff}$ affect the error similarly, and it is better to keep them equal and increase them as much as possible.

\subsubsection{Small Training Dataset}
\label{sec:ColumnsEstSmallDT}

In this section we will learn 20D \emph{Columns} distribution using only 100000 training sample points. As we will see below, the density inference task via non-parametric PSO becomes much more challenging when the size of the training dataset is limited.

\begin{figure}[!tbp]
	\centering
	
	\begin{tabular}{cc}

		\subfloat[\label{fig:ColsRes7.0-a}]{\includegraphics[width=0.35\textwidth]{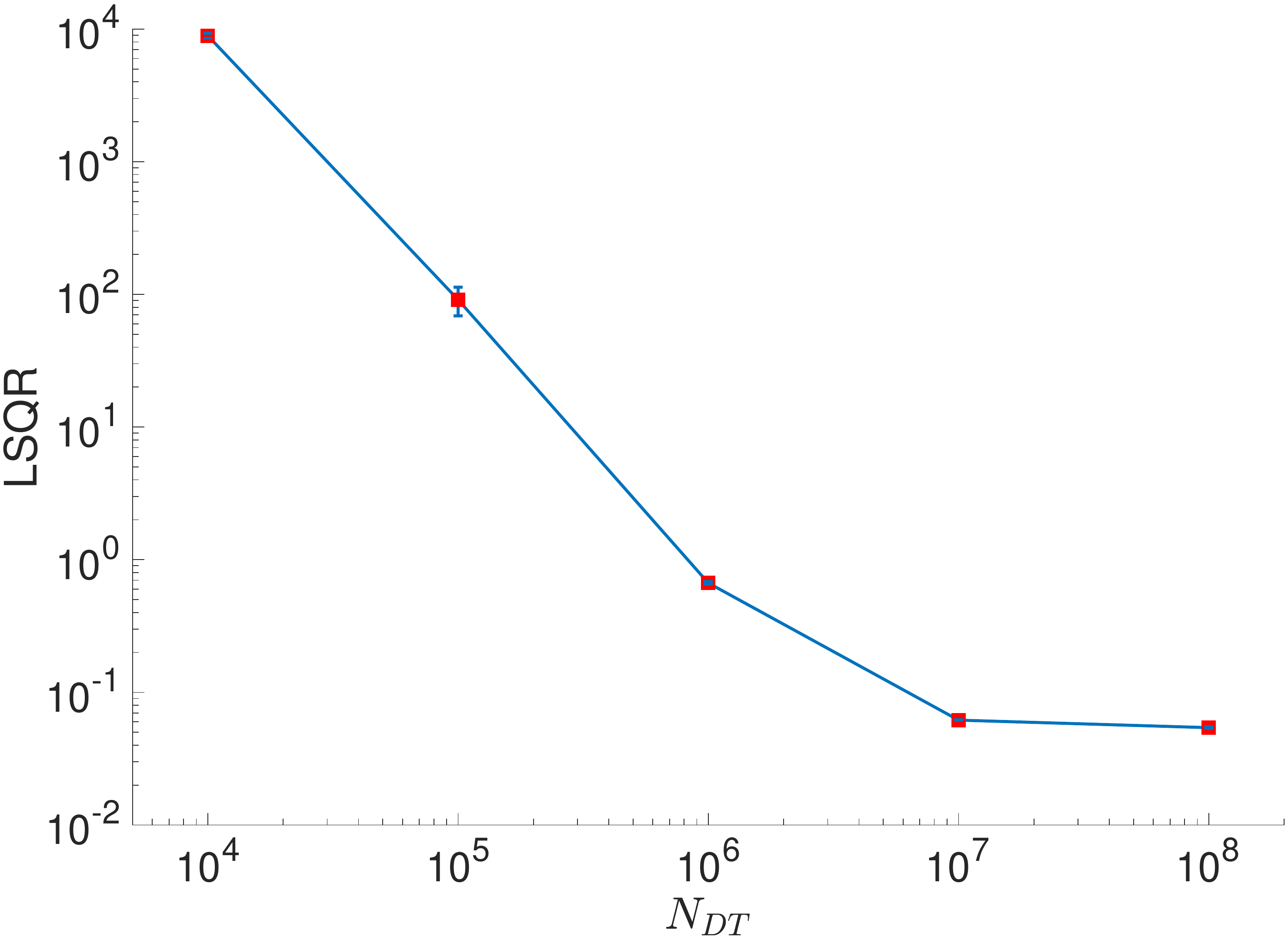}}
		&
		
		\subfloat[\label{fig:ColsRes7.0-b}]{\includegraphics[width=0.4\textwidth]{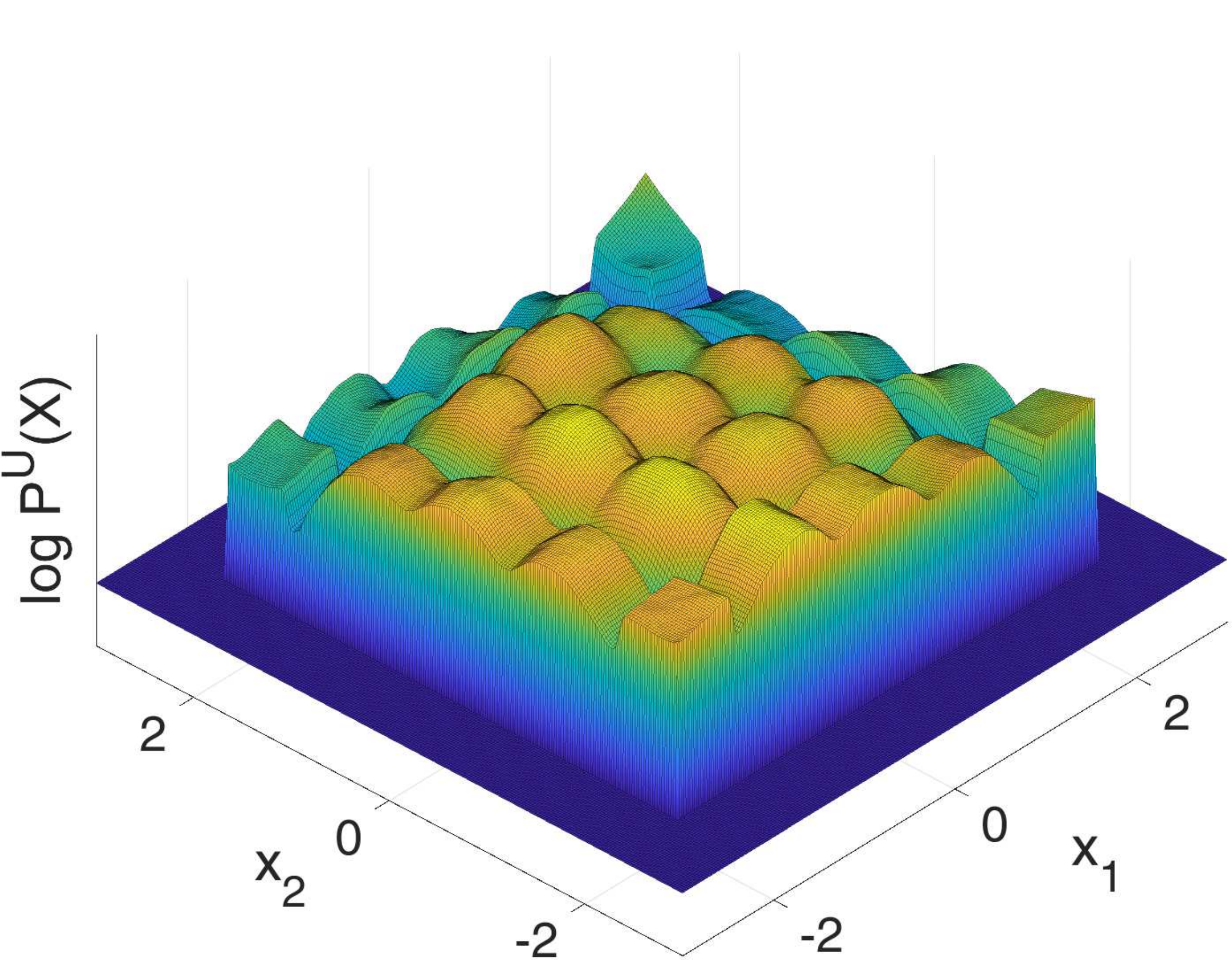}}
		
	\end{tabular}

	\protect
	\caption[Evaluation of PSO-LDE for estimation of \emph{Columns} distribution with different sizes of the training dataset.]{Evaluation of PSO-LDE for estimation of \emph{Columns} distribution with different sizes $N_{DT}$ of the training dataset.
		Used NN architecture is block-diagonal with 6 layers, number of blocks $N_B = 50$ and block size $S_B = 64$ (see Section \ref{sec:BDLayers}). Number of weights is $\left| \theta \right| = 902401$.
		Applied PSO instance is PSO-LDE with $\alpha = \frac{1}{4}$.
		(a) $LSQR$ error as a function of the training dataset size $N_{DT}$, where both $x$ and $y$ axes are log-scaled.
		(b) Illustration of the learned log-pdf surface $f_{\theta}(X)$ for the dataset size $N_{DT} = 10^5$. The depicted slice is $\log \PP(x_1, x_2) \equiv f_{\theta}([x_1, x_2, 0, \ldots , 0])$, with $x_1$ and $x_2$ forming a grid of points in first two dimensions of the density's support. 
		Note the resemblance similarity between this plot and the target surface in Figure \ref{fig:ColsRes1-b}. Even though $LSQR$ error of this model in (a) is very high ($\approx 91.2$), the local structure within the converged surface is close to the local structure of the target function. Unfortunately, its global structure is inconsistent and far away from the target.
		}
	\label{fig:ColsRes7.0}
\end{figure}

\paragraph{Various Sizes of Dataset}

To infer the data pdf, here we applied PSO-LDE with $\alpha = \frac{1}{4}$. Further, we use BD NN architecture since it has superior approximation performance over FC architecture. First, we perform the inference task using the same BD network as in Section \ref{sec:ColumnsEstME}, with 6 layers, $N_B = 50$ and $S_B = 64$, where the size of the entire weights vector is $\left| \theta \right| = 902401$.
The \emph{Columns} pdf is inferred using a various number $N_{DT}$ of \textbf{overall} training points $\{X^{\usuff}_{i}\}_{i = 1}^{N_{DT}}$. As can be observed from Figure \ref{fig:ColsRes7.0-a}, $LSQR$ error increases for a smaller size $N_{DT}$ of the training dataset. From error $0.05$ for $N_{DT} = 10^8$ it gets to $0.062$ for $N_{DT} = 10^7$, $0.67$ for $N_{DT} = 10^6$, $91.2$ for $N_{DT} = 10^5$ and $8907$ for $N_{DT} = 10^4$. As we will see below and as was already discussed in Section \ref{sec:DeepLogPDFOF}, one of the main reasons for such large errors is a too flexible NN model, which in a small dataset setting can significantly damage the performance of PSO-LDE, and PSO in general.

Interestingly, although $LSQR$ error is high for models with $N_{DT} = 10^5$, in Figure \ref{fig:ColsRes7.0-b} we can see that the converged surface $f_{\theta}(X)$ visually highly resembles the real target log-pdf in Figure \ref{fig:ColsRes1-b}. We can see that main lines and forms of the target surface were learned by $f_{\theta}(X)$, yet the overall global shape of NN surface is far away from the target. Furthermore, if we calculate  $\exp f_{\theta}(X)$, we will get a surface that is very far from its target $\probi{\usuff}{X}$ since the exponential will amplify small errors into large.

Further, in Figure \ref{fig:ColsRes7.01} we can observe error curves for models learned in Figure \ref{fig:ColsRes7.0-a}. Both train and test errors are reported for all three error types, per a different dataset size $N_{DT}$. Train and test errors are very similar for big $N_{DT}$ in Figures \ref{fig:ColsRes7.01-a}-\ref{fig:ColsRes7.01-b}. Moreover, in Figure \ref{fig:ColsRes7.01-b} we can see according to the $LSQR$ error (the middle column) that at the beginning error decreases but after $10^5$ steps it starts increasing, which suggests a possible \overfit to the training dataset at $N_{DT} = 10^6$. Further, in Figures \ref{fig:ColsRes7.01-c}-\ref{fig:ColsRes7.01-d} we can see that train and test errors are more distinct from each other. Likewise, here $PSQR$ and $LSQR$, both train and test, are increasing almost from the start of the optimization. In contrast, in Figures \ref{fig:ColsRes7.01-c}-\ref{fig:ColsRes7.01-d} in case of $IS$ the train and test errors have different trends compared with each other. While the test $IS$ is increasing, the train $IS$ is decreasing. This is a typical behavior of the optimization error that indicates strong \overfit of the model, here for $N_{DT} = 10^5$ and $N_{DT} = 10^4$. In  turn, this means that we apply a too rich model family - over-flexible NN which can be pushed to form peaks around the training points, as was demonstrated in Section \ref{sec:DeepLogPDFOF}. Further, we can detect such \overfit by comparing train and test $IS$ errors, which do not depend on ground truth.

Note that train and test errors of $PSQR$/$LSQR$ have the same trend herein, unlike $IS$. The reason for this is that $PSQR$ and $LSQR$ express a real ground truth distance between NN surface and the target function, whereas $IS$ error is only a some rough estimation of it.

\begin{figure}[!tbp]
	\centering
	
	\newcommand{\width}[0] {0.85}
	\newcommand{\height}[0] {0.16}

	\begin{tabular}{c}

		\subfloat[\label{fig:ColsRes7.01-a}]{\includegraphics[height=\height\textheight,width=\width\textwidth]{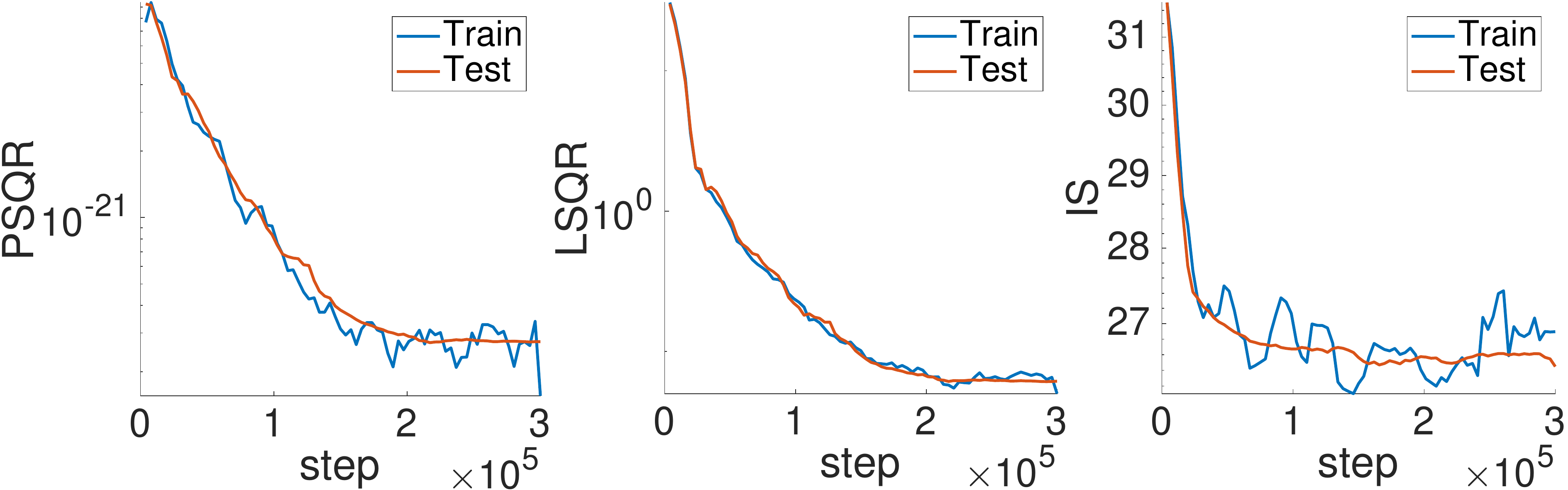}}
		\\
		\subfloat[\label{fig:ColsRes7.01-b}]{\includegraphics[height=\height\textheight,width=\width\textwidth]{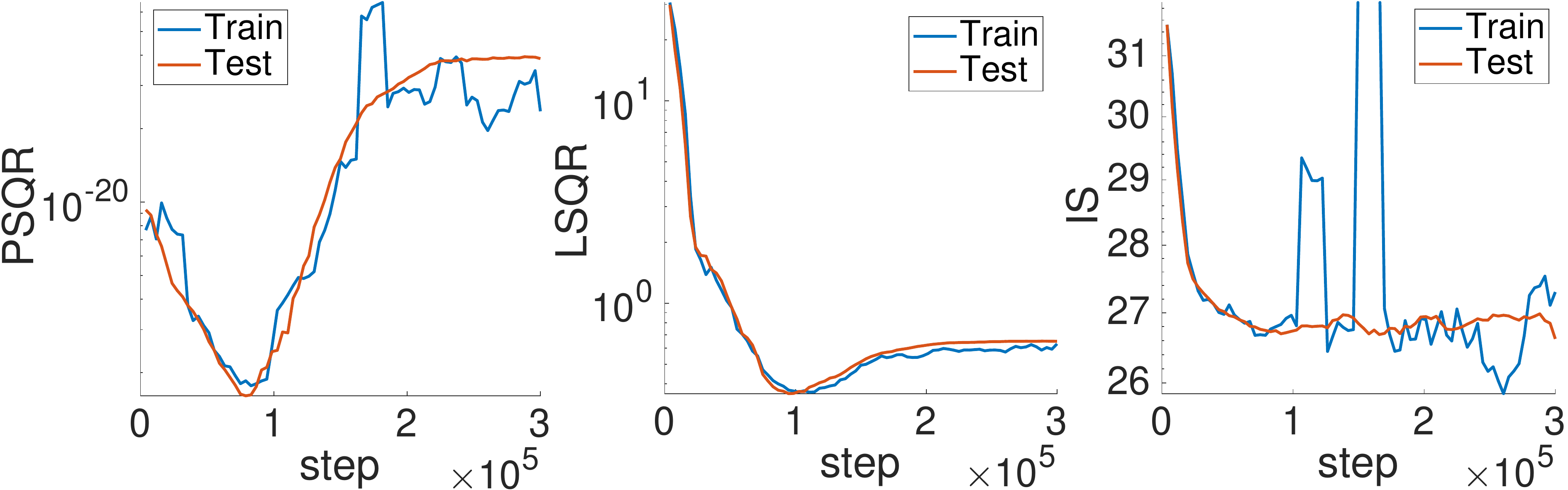}}
		\\
		\subfloat[\label{fig:ColsRes7.01-c}]{\includegraphics[height=\height\textheight,width=\width\textwidth]{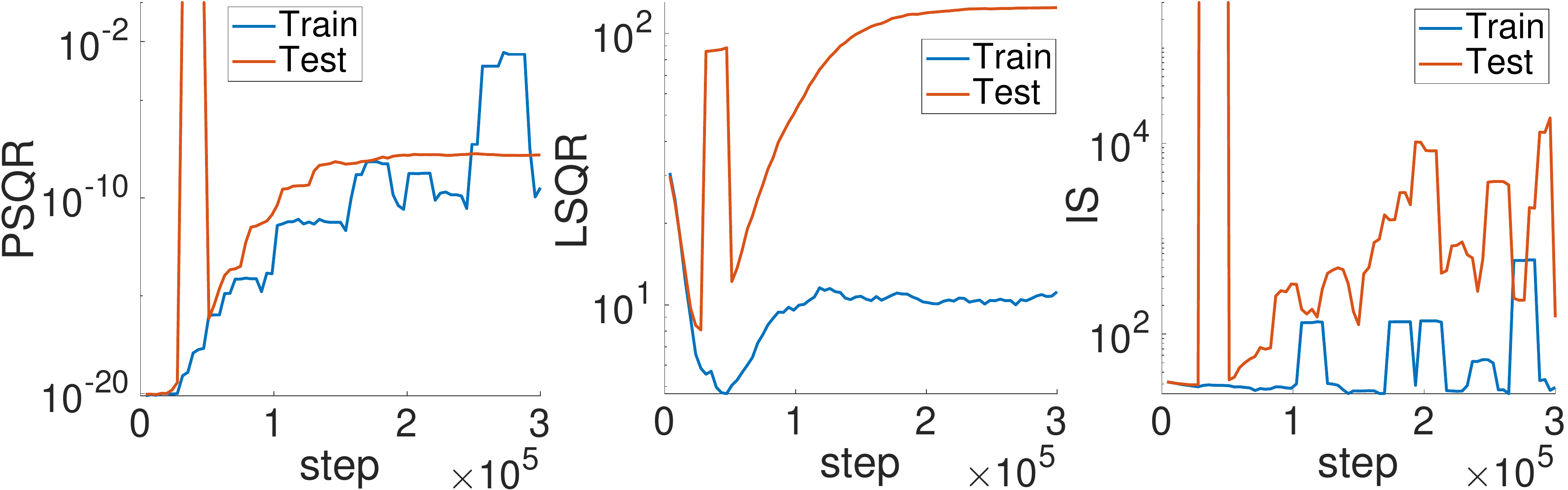}}
		\\
		\subfloat[\label{fig:ColsRes7.01-d}]{\includegraphics[height=\height\textheight,width=\width\textwidth]{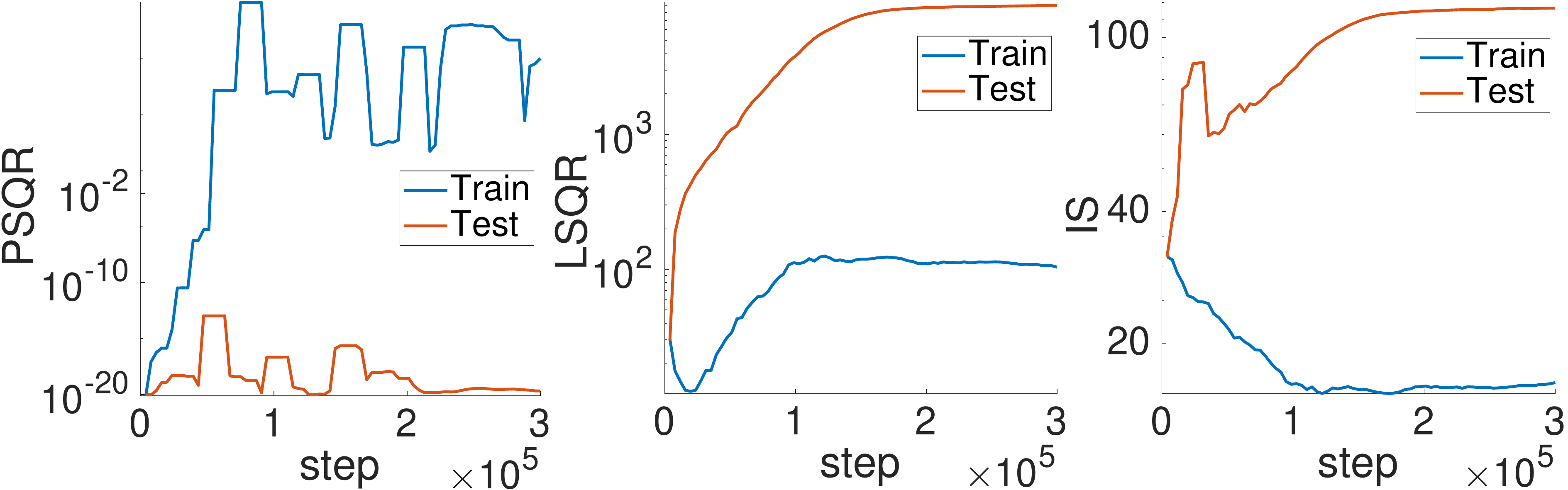}}

	\end{tabular}

	\protect
	\caption[Error curves during the optimization for models learned in Figure \ref{fig:ColsRes7.0-a}.]{Error curves during the optimization for models learned in Figure \ref{fig:ColsRes7.0-a}, where various dataset sizes $N_{DT}$ were evaluated. The error is reported for (a) $N_{DT} = 10^{7}$, (b) $N_{DT} = 10^{6}$, (c) $N_{DT} = 10^{5}$ and (d) $N_{DT} = 10^{4}$. In each row we report: $PSQR$ - first column; $LSQR$ - middle column; $IS$ - last column. Each plot contains both train and test errors; the former is evaluated over $10^3$ \up points from training dataset chosen as a batch for a specific optimization iteration, whereas the latter - over $10^5$ \up points from the testing dataset.
	}
	\label{fig:ColsRes7.01}
\end{figure}

\begin{figure}[!tbp]
	\centering
	
	\begin{tabular}{cc}

		\subfloat[\label{fig:ColsRes7-a}]{\includegraphics[width=0.45\textwidth]{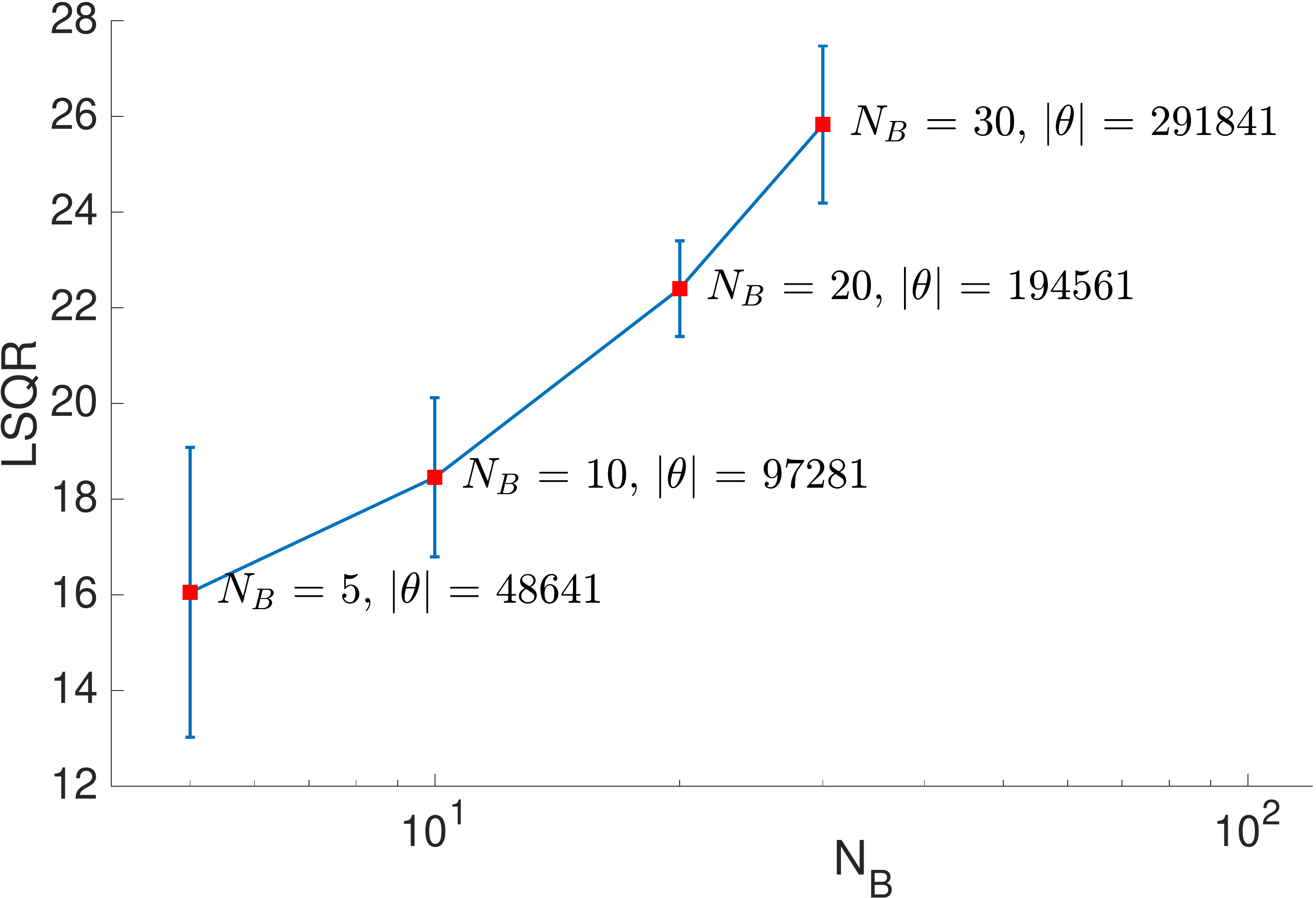}}
		&
		
		\subfloat[\label{fig:ColsRes7-b}]{\includegraphics[width=0.45\textwidth]{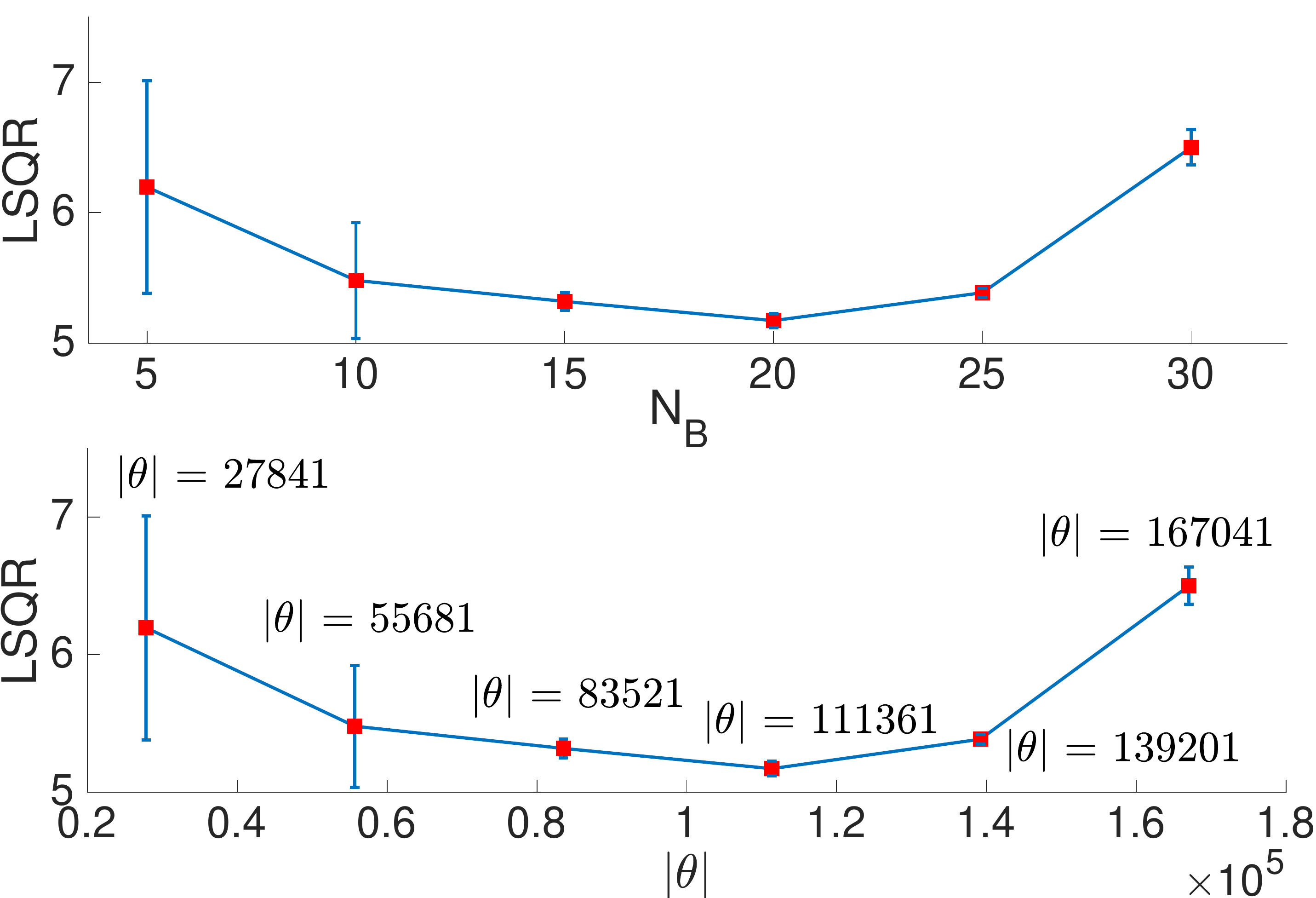}}
		
	\end{tabular}
	
	\protect
	\caption[Evaluation of PSO-LDE for estimation of \emph{Columns} distribution with only $N_{DT} = 10^5$ training samples.]{Evaluation of PSO-LDE for estimation of \emph{Columns} distribution with only $N_{DT} = 10^5$ training samples.
		Used NN architecture is block-diagonal with block size $S_B = 64$ (see Section \ref{sec:BDLayers}). Applied PSO instance is PSO-LDE with $\alpha = \frac{1}{4}$.
		(a) $LSQR$ error for models with 4 layers and various values of the blocks number $N_B$. Size of weights vector $\theta$ is depicted for each model.
		(b) $LSQR$ error for models with 3 layers and various values of blocks number $N_B$. Top - $LSQR$ is shown as a function of $N_B$; bottom - $LSQR$ is shown as a function of the size $\left| \theta \right|$.
		As observed, too big and too small number of parameters, $\left| \theta \right|$, produces a less accurate pdf approximation in the small dataset setting.
	}
	\label{fig:ColsRes7}
\end{figure}

\paragraph{Reduction of NN Flexibility to Tackle Overfitting}

Next, we perform the inference task using BD network with only 4 layers, on the training dataset of size $N_{DT} = 10^5$. We learn models for various numbers of blocks $N_B$ inside our network, being between 5 and 30. In Figure \ref{fig:ColsRes7-a} we can see that $LSQR$ error is still very high compared to the results of an infinite dataset setting in Section \ref{sec:ColumnsEstME}: $\approx \! 20$ vs $\approx \! 0.05$. Yet, it is smaller than in Figure \ref{fig:ColsRes7.0-a}, where we used 6 layers instead of 4 and $N_B = 50$. Moreover, we can see in Figure \ref{fig:ColsRes7-a} that the error is reducing with a smaller number of blocks $N_B$ and a smaller number of NN parameters $\left| \theta \right|$.

Further, we perform the same experiment where BD architecture has only 3 layers (first and last are FC layers and in the middle there is BD layer). In Figure \ref{fig:ColsRes7-b} we can observe that for a too small/large value of $N_B$ the $LSQR$ is higher. That is, for a too big/small number of weights in $\theta$ we have a worse pdf approximation. This can be explained as follows. Small size $\left| \theta \right|$ implies NN with low flexibility which is not enough to closely approximate a target surface $\log \probi{\usuff}{X}$, thus producing \underfit. We observed similar results also in an infinite dataset setting in Section \ref{sec:ColumnsEstNNAME}, where bigger size of $\theta$ yielded an even smaller error. Furthermore, when $\left| \theta \right|$ is too big, the NN surface becomes too flexible and causes \overfit. Such over-flexibility is not appropriate for small dataset setting, since it allows to closely approximate a peak around each training sample $X^{\usuff}$ (or $X^{\dsuff}$), where the produced spiky surface $f_{\theta}(X)$ will obviously have a high approximation error, as was demonstrated in Section \ref{sec:DeepLogPDFOF}. In other words, in contrast to common regression learning, in case of \emph{unsupervised} PSO approaches the size (and the flexibility) of NN should be adjusted according to the number of available training points, otherwise the produced approximation error will be enormous. In contrast, in common DL-based regression methods such over-flexible NN may cause \overfit to training samples and increase the testing error, yet it will not affect the overall approximation performance as destructively as in PSO.

Interestingly, the optimal size $\left| \theta \right|$ in Figure \ref{fig:ColsRes7-b}-bottom is around 100000 - the number of available training points. It would be an important investigatory direction to find the exact mathematical relation between the dataset size and properties of NN (e.g. its size, architecture and \gs) for the optimal inference. We shall leave it for future research.

\begin{figure}[tb]
	\centering
	
	\newcommand{\width}[0] {0.35}
	\newcommand{\height}[0] {0.18}
	
	\begin{tabular}{cc}

		\subfloat[\label{fig:ColsRes8-a}]{\includegraphics[height=\height\textheight,width=\width\textwidth]{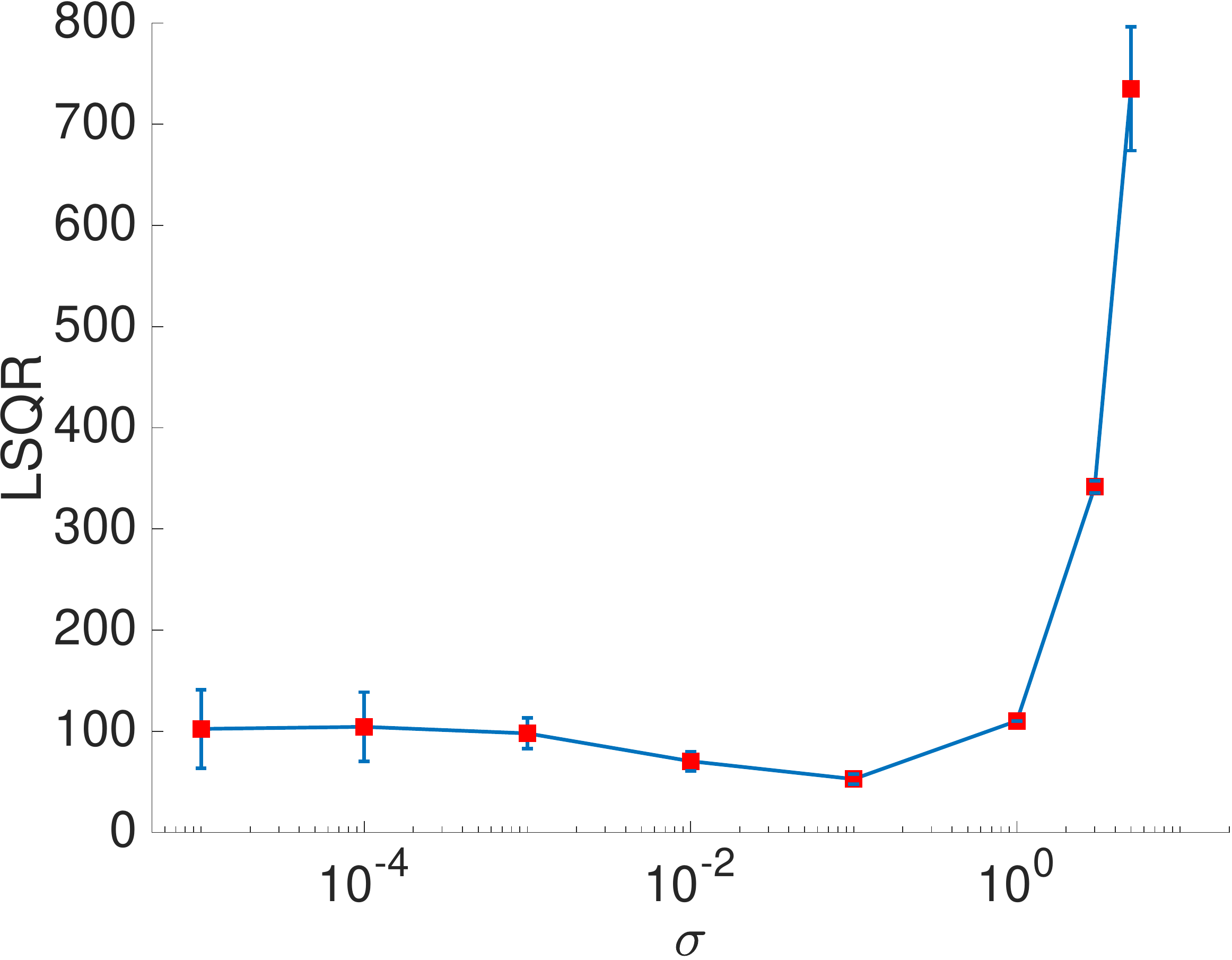}}
		&
		
		\subfloat[\label{fig:ColsRes8-b}]{\includegraphics[height=\height\textheight,width=\width\textwidth]{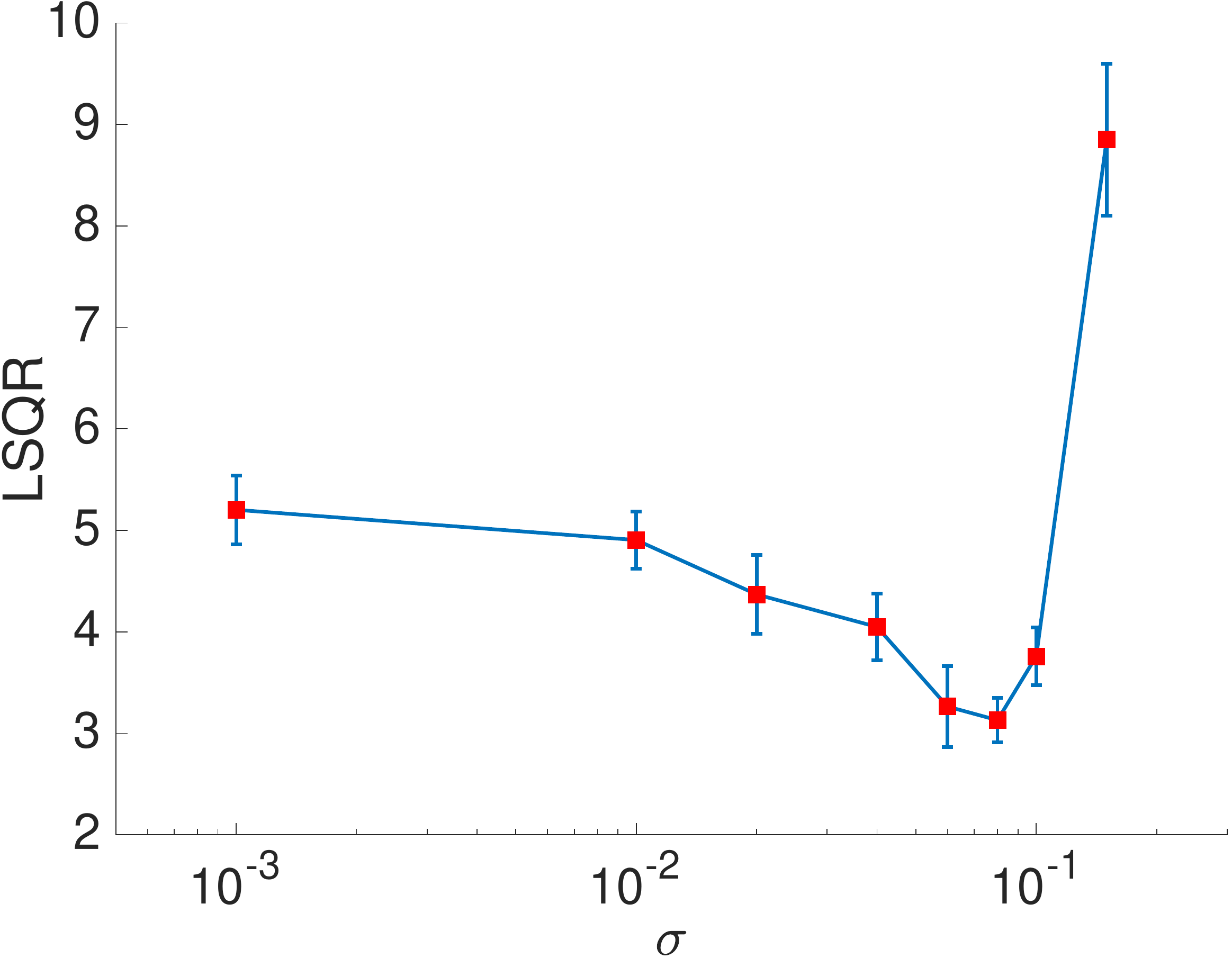}}
		
		\\
		
		\subfloat[\label{fig:ColsRes8-c}]{\includegraphics[height=\height\textheight,width=\width\textwidth]{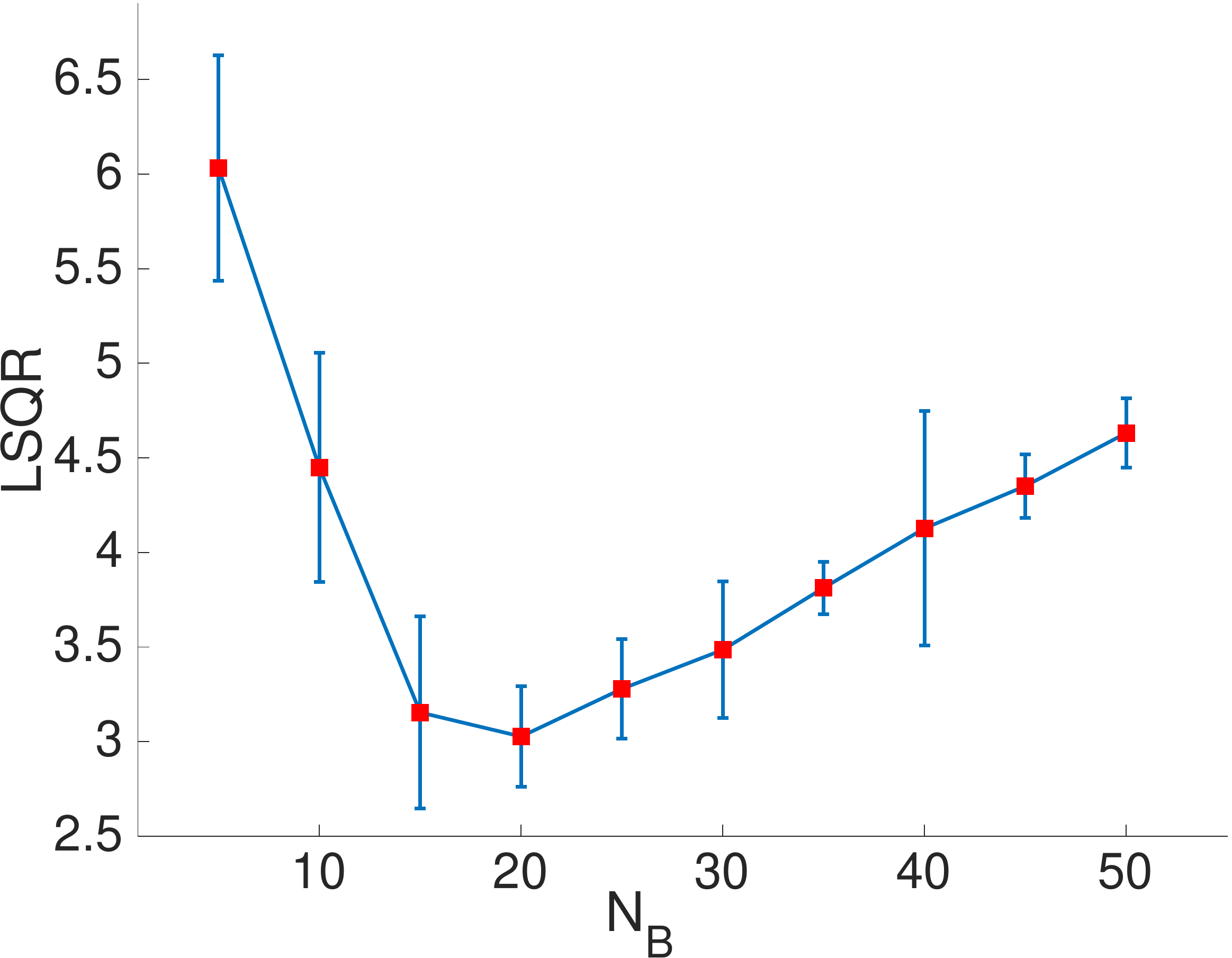}}
&

\subfloat[\label{fig:ColsRes8-d}]{\includegraphics[height=\height\textheight,width=\width\textwidth]{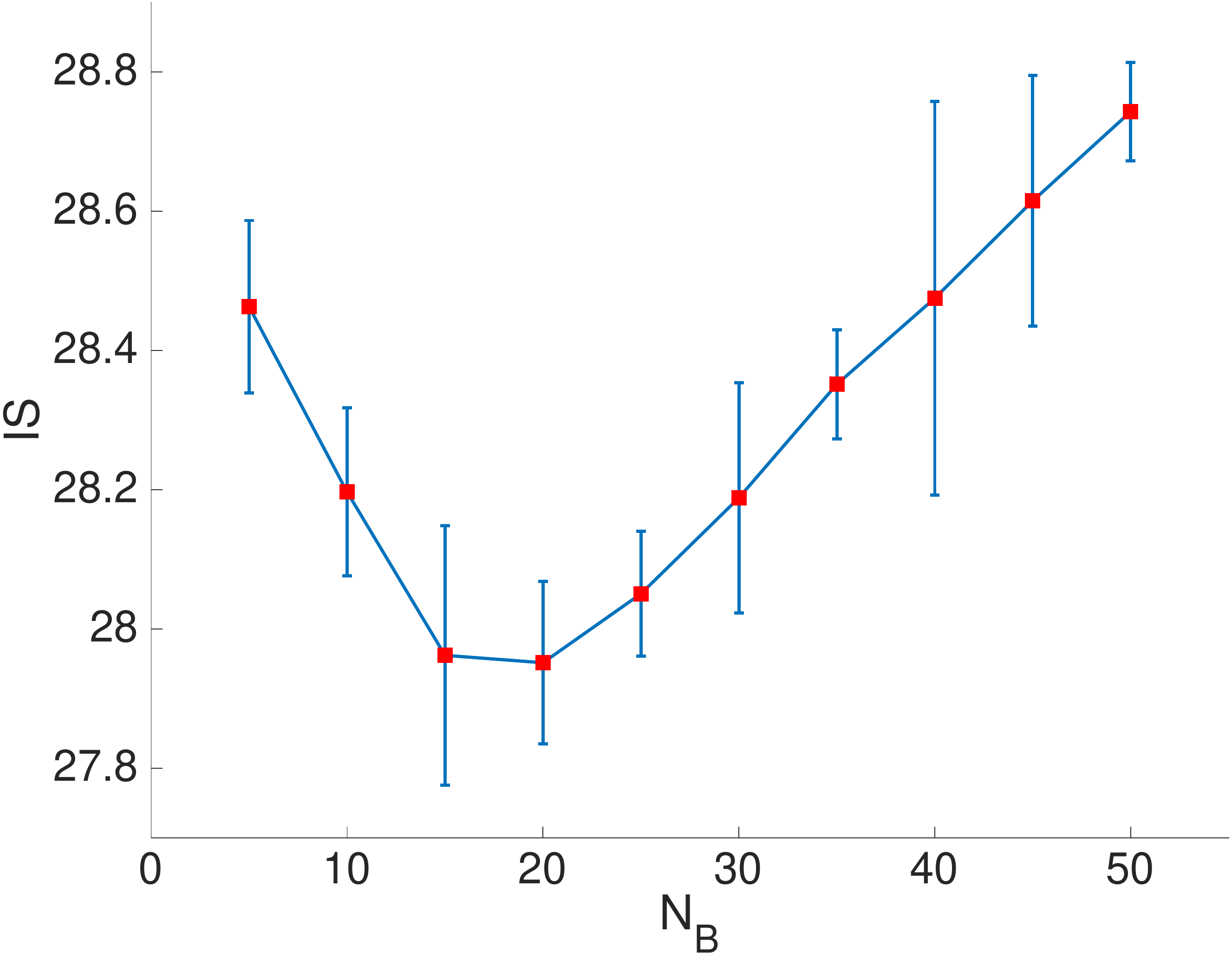}}

	\end{tabular}
	
	\protect
	\caption[Evaluation of PSO-LDE for estimation of \emph{Columns} distribution with only $N_{DT} = 10^5$ training samples, using a data augmentation noise.]{Evaluation of PSO-LDE for estimation of \emph{Columns} distribution with only $N_{DT} = 10^5$ training samples, using a data augmentation noise.
		The used NN architecture is block-diagonal with block size $S_B = 64$ (see Section \ref{sec:BDLayers}). Applied PSO instance is PSO-LDE with $\alpha = \frac{1}{4}$. Each \up training point $X^{\usuff}$ is sampled from the data density $\probs{\usuff}$. Further, an additive 20D noise $\upsilon \sim \mathcal{N}(0, \sigma^2 \cdot I)$ is sampled and added to $X^{\usuff}$. The PSO-LDE loss is applied on $\bar{X}^{\usuff} = X^{\usuff} + \upsilon$ instead of the original $X^{\usuff}$.
		Number of overall $\probs{\usuff}$'s samples is constant $10^5$; new samples of noise $\upsilon$ are sampled at each optimization iteration.
		We learn models using various values of $\sigma$.
		(a) $LSQR$ error as a function of $\sigma$, for models with 6 layers and the blocks number $N_B = 50$.
		(b) $LSQR$ error as a function of $\sigma$, for models with 3 layers and the blocks number $N_B = 20$.
		(c) $LSQR$ error  and (d) $IS$ error as a function of the blocks number $N_B$, for models with 3 layers and $\sigma = 0.08$.
		As can be seen, smaller size (up to some threshold) of NN produces much higher accuracy in a limited training dataset setting.
		Also, the additive noise can yield an accuracy improvement.
	}
	\label{fig:ColsRes8}
\end{figure}

\paragraph{Data Augmentation to Tackle Overfitting}

Additionally, we also apply the augmentation data technique to smooth the converged surface $f_{\theta}(X)$, as was described in Section \ref{sec:DeepLogPDFOF}. Concretely, we use samples of r.v. $\bar{X}^{\usuff} = X^{\usuff} + \upsilon$ as our \up points, where we push the NN surface \up. The $X^{\usuff}$ is sampled from the data density $\probs{\usuff}$, while the additive 20D noise $\upsilon$ is sampled from $\sim \mathcal{N}(0, \sigma^2 \cdot I)$. At each optimization iteration the next batch of $\{ X^{\usuff}_{i} \}_{i = 1}^{N^{\usuff}}$ is fetched from a priori prepared training dataset of size $10^5$, and new noise instances $\{ \upsilon^i \}_{i = 1}^{N^{\usuff}}$ are generated. Further,  $\{ \bar{X}^{\usuff}_{i} \}_{i = 1}^{N^{\usuff}}$ is used as the batch of \up points within PSO-LDE loss, where $\bar{X}^{\usuff}_{i} = X^{\usuff}_{i} + \upsilon^i$. Such a method allows us to push the $f_{\theta}(X)$ \up not only at the limited number of training points $X^{\usuff}_{i}$, but also at other points in some ball neighborhood around each $X^{\usuff}_{i}$, thus implicitly changing the approximated function to be smoother and less spiky. Another perspective to look over it is that we apply Gaussian diffusion over our NN surface, since adding Gaussian noise is identical to replacing the target pdf $\probi{\usuff}{X}$ with the convolution $\probi{\usuff}{X} * \mathcal{N}(0, \sigma^2 \cdot I)$.

In Figure \ref{fig:ColsRes8-a} we can see results of such data augmentation for BD network with 6 layers and 50 blocks, with $S_B = 64$. In such case the used model is over-flexible with too many degrees of freedom, and the data augmentation is not helpful, with an overall error being similar to the one obtained in Figure \ref{fig:ColsRes7.0-a}. Yet, when the model size (and its flexibility) is reduced to only 3 layers and 20 blocks, the performance trend becomes different. In Figure \ref{fig:ColsRes8-b} we can see that for particular values of the noise s.t.d. $\sigma$ (e.g. 0.08) the data augmentation technique reduces $LSQR$ error from 5.17 (see Figure \ref{fig:ColsRes7-b}) to only 3.13.

Further, in Figures \ref{fig:ColsRes8-c}-\ref{fig:ColsRes8-d} we can see again that there is an optimal NN size/flexibility that produces the best performance, where a smaller NN suffers from \underfit and a larger NN suffers from \overfit. Moreover, in Figure \ref{fig:ColsRes8-d} we can see again that the empirical error $IS$ is correlated with the ground truth error $LSQR$, although the former is less accurate than the latter. Hence, $IS$ can be used in practice to select the best learned model.

Overall, in our experiments we observed both \underfit and \overfit cases of PSO optimization. The first typically happens for large dataset size when NN is not flexible enough to represent all the information contained within training samples. In contrast, the second typically happens for small datasets when NN is over-flexible so that it can be pushed to have spikes around the training points. Further, \overfit can be detected via $IS$ test error. Finally, the data augmentation reduces effect of \overfit.

\subsection{PDF Estimation via PSO - \emph{Transformed Columns} Distribution}
\label{sec:TrColumnsEst}

In this section we will show that PSO is not limited only to isotropic densities (e.g. \emph{Columns} distribution from Section \ref{sec:ColumnsEst}) where there is no correlation among different data dimensions, and can be actually applied also to data with a complicated correlation structure between various dimensions.
Specifically, herein we infer a 20D \emph{Transformed Columns} distribution, $\probi{\usuff}{X} = \probi{TrClmns}{X}$, which is produced from isotropic \emph{Columns} by multiplying a random variable $X \sim \probs{Clmns}$ (defined in Eq.~(\ref{eq:ColumnsDef})) by a dense invertible matrix $A$ that enforces correlation between different dimensions. Its pdf can be written as:
\begin{equation}
\probi{TrClmns}{X}
=
\frac{1}{abs
	\left[ \det A \right]}
\probi{Clmns}{A^{-1} \cdot X}
,
\label{eq:TransformedColumnsDef}
\end{equation}
where $A$ appears in Appendix \ref{sec:App8}. As we will see below, the obtained results for this more sophisticated distribution have similar trends to results of \emph{Columns}. Additionally, we will also show how important is the choice of $\probs{\dsuff}$. Unconcerned reader may skip it to the next Section \ref{sec:ImageEst}.

\begin{figure}[tb]
	\centering
	
	\begin{tabular}{ccc}

		\subfloat[\label{fig:TrColsRes1-a}]{\includegraphics[width=0.31\textwidth]{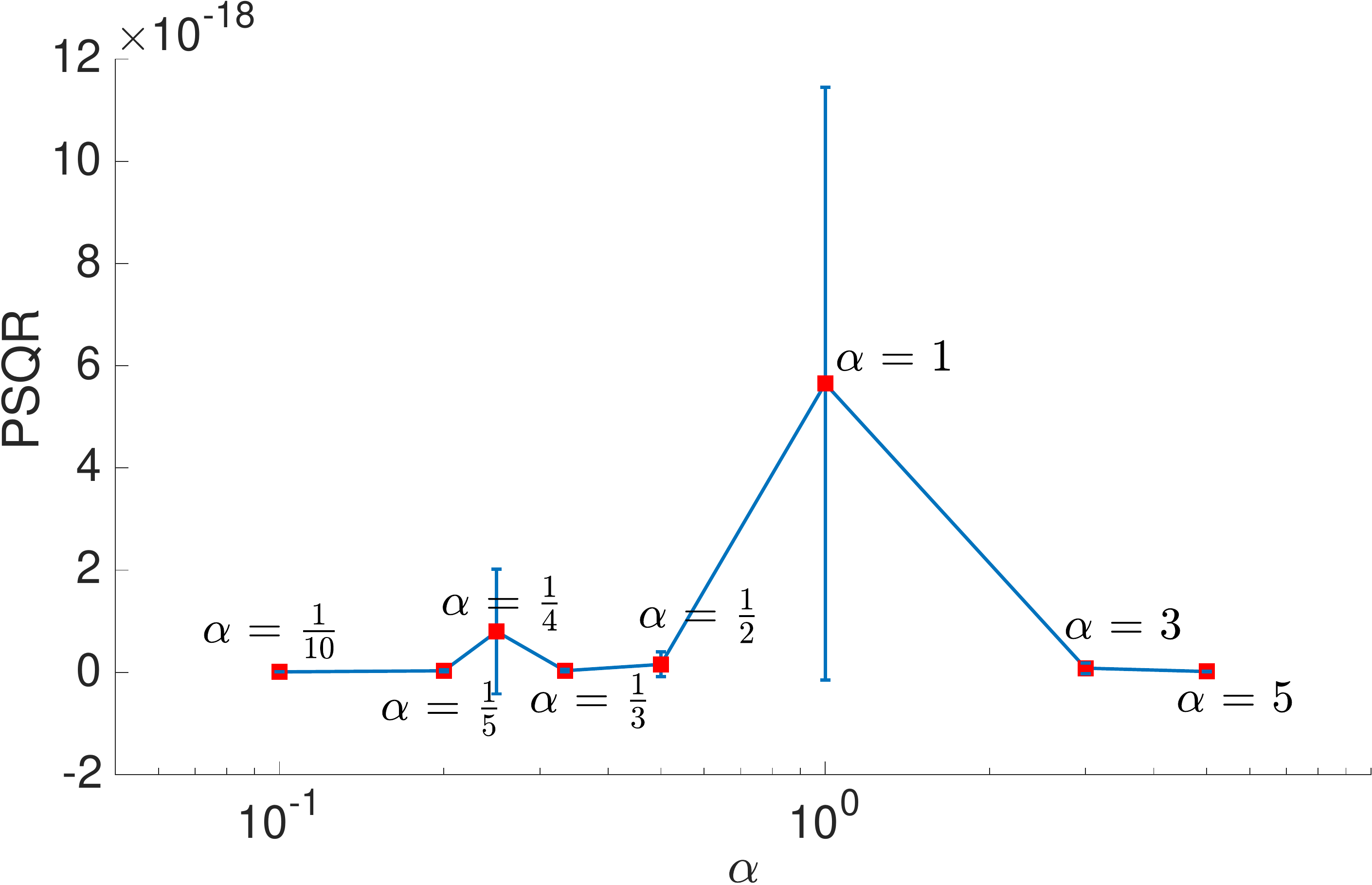}}
		&
		
		\subfloat[\label{fig:TrColsRes1-b}]{\includegraphics[width=0.31\textwidth]{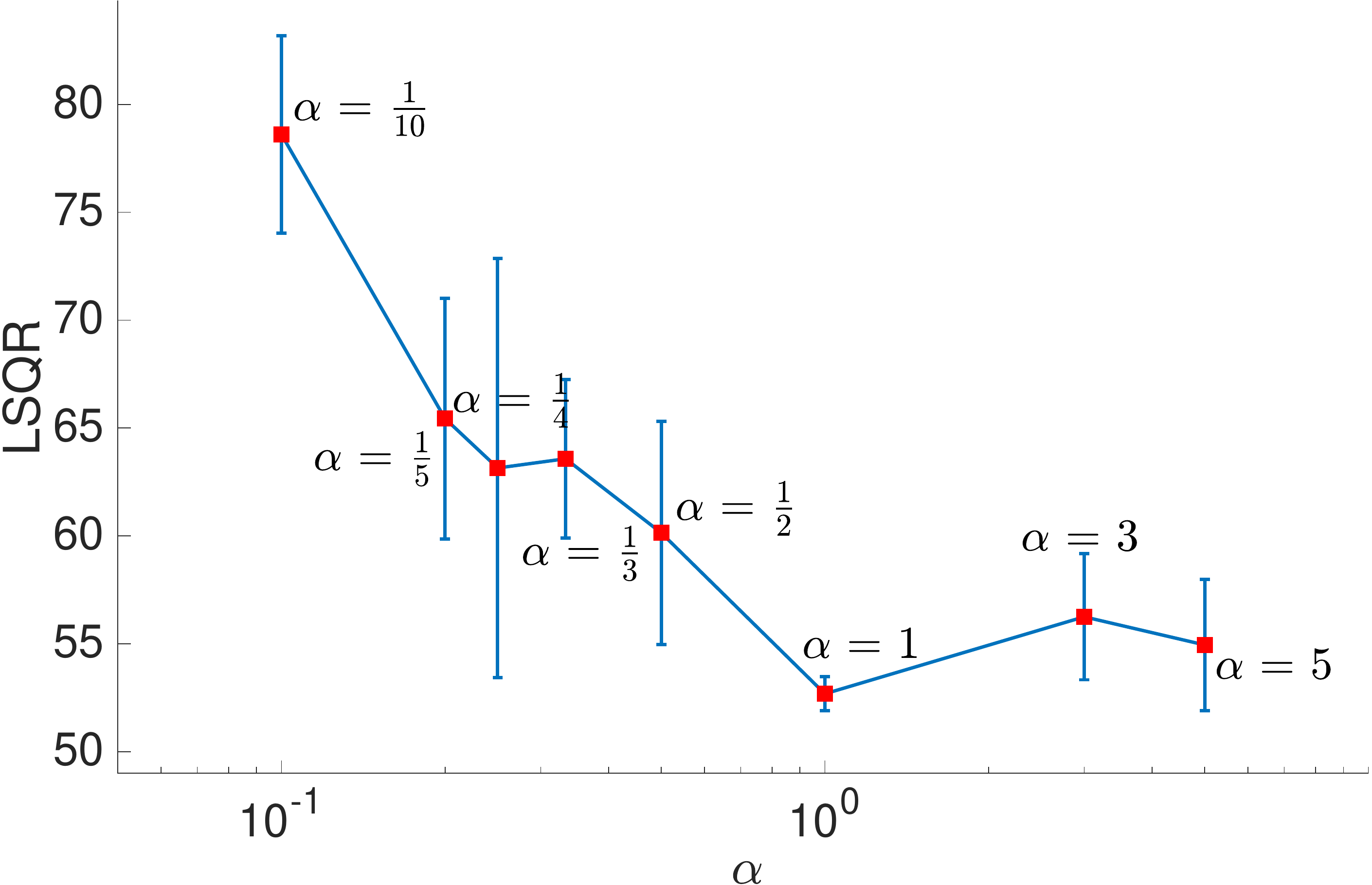}}
		
		&
		\subfloat[\label{fig:TrColsRes1-c}]{\includegraphics[width=0.31\textwidth]{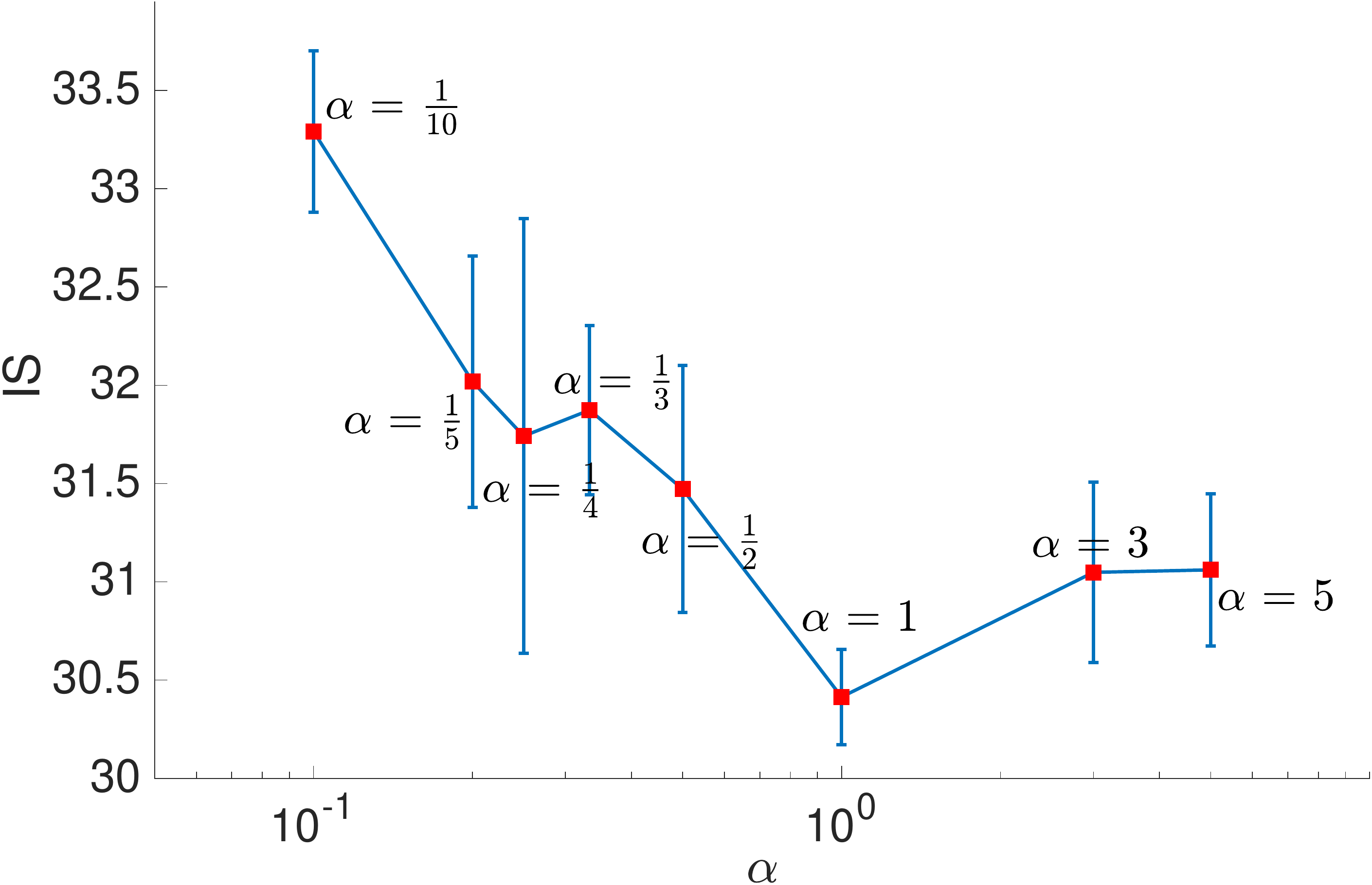}}
		
		\\
		\subfloat[\label{fig:TrColsRes1-d}]{\includegraphics[width=0.31\textwidth]{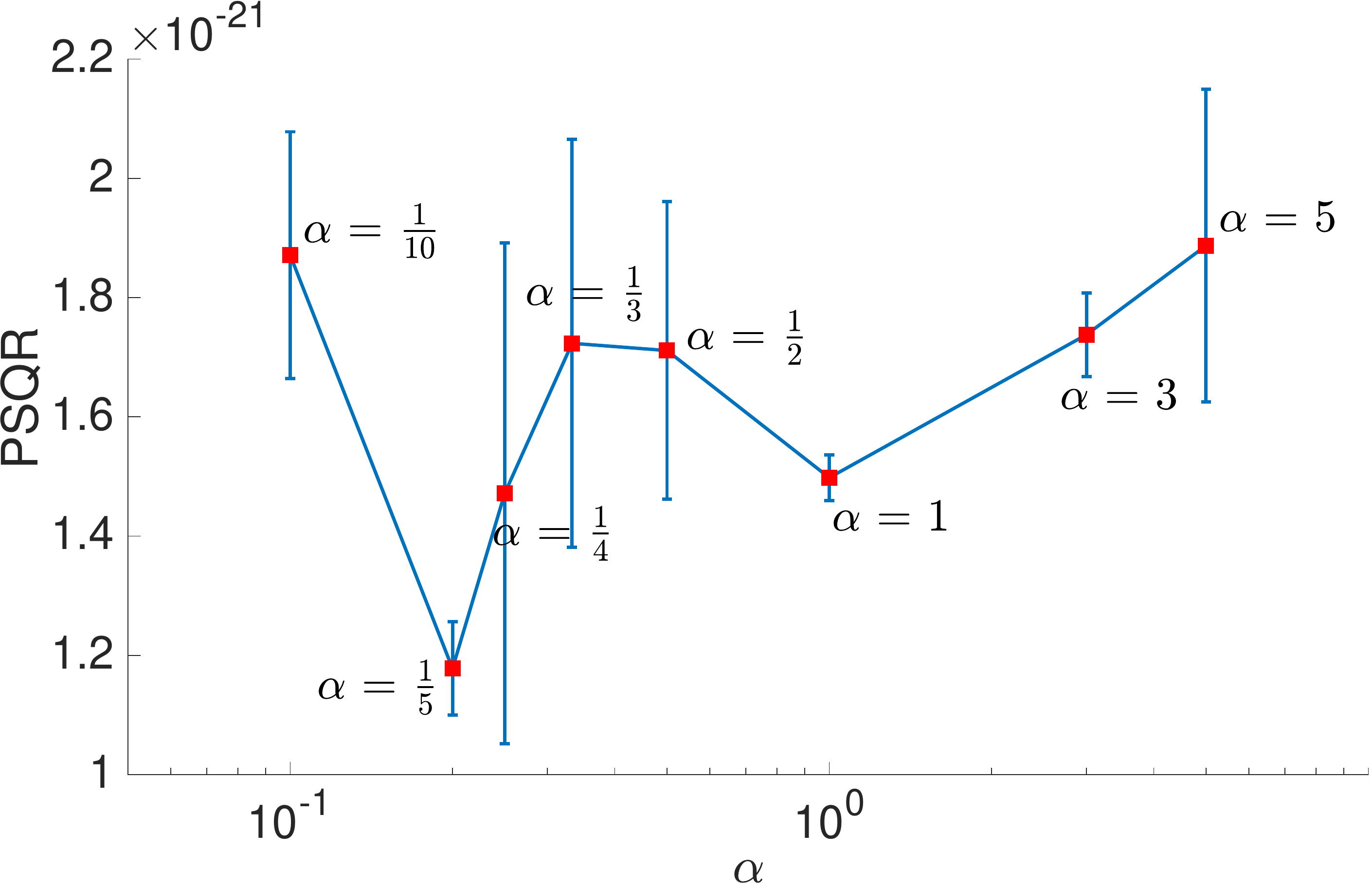}}
		&
		
		\subfloat[\label{fig:TrColsRes1-e}]{\includegraphics[width=0.31\textwidth]{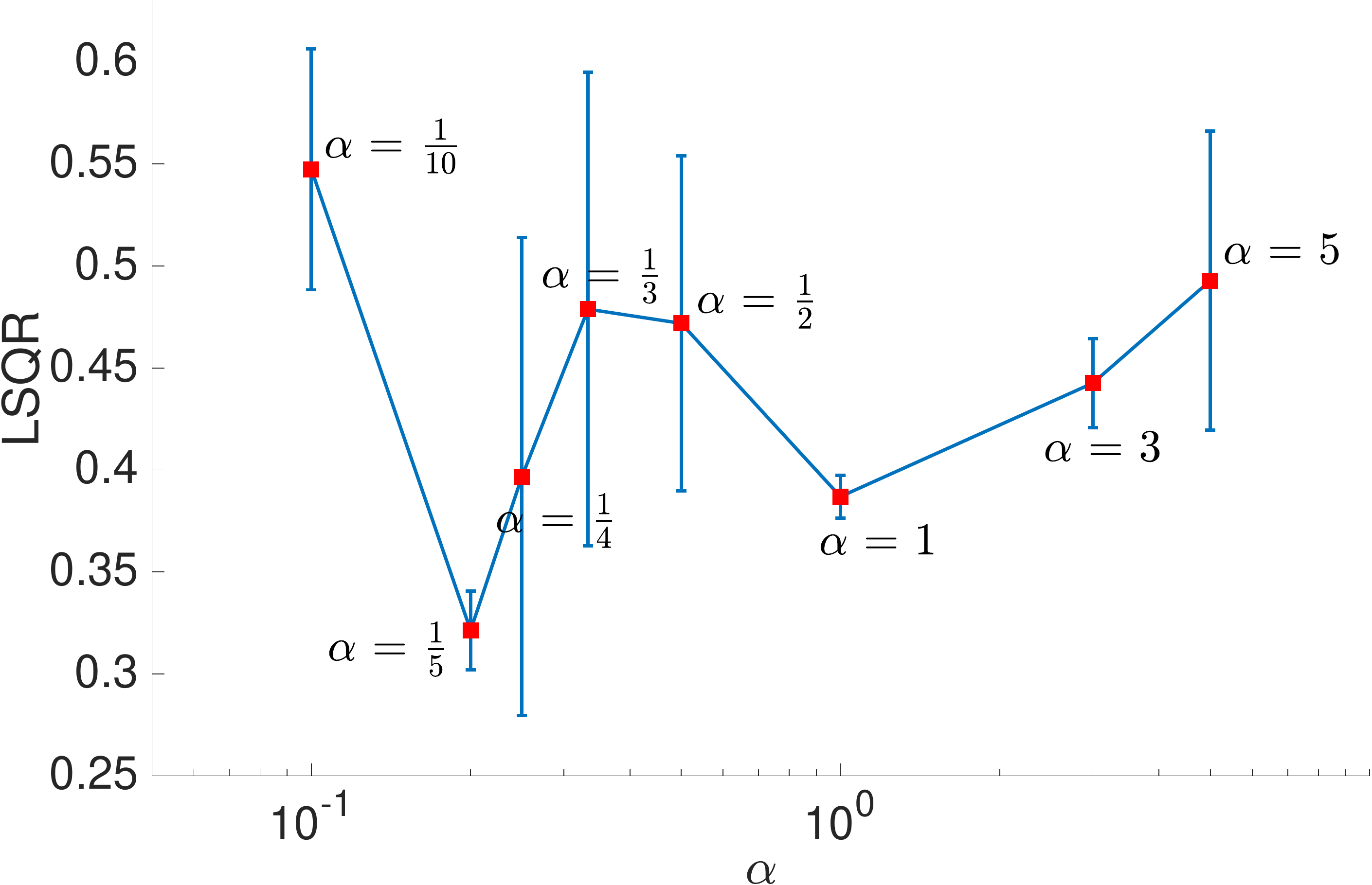}}
		
		&
		\subfloat[\label{fig:TrColsRes1-f}]{\includegraphics[width=0.31\textwidth]{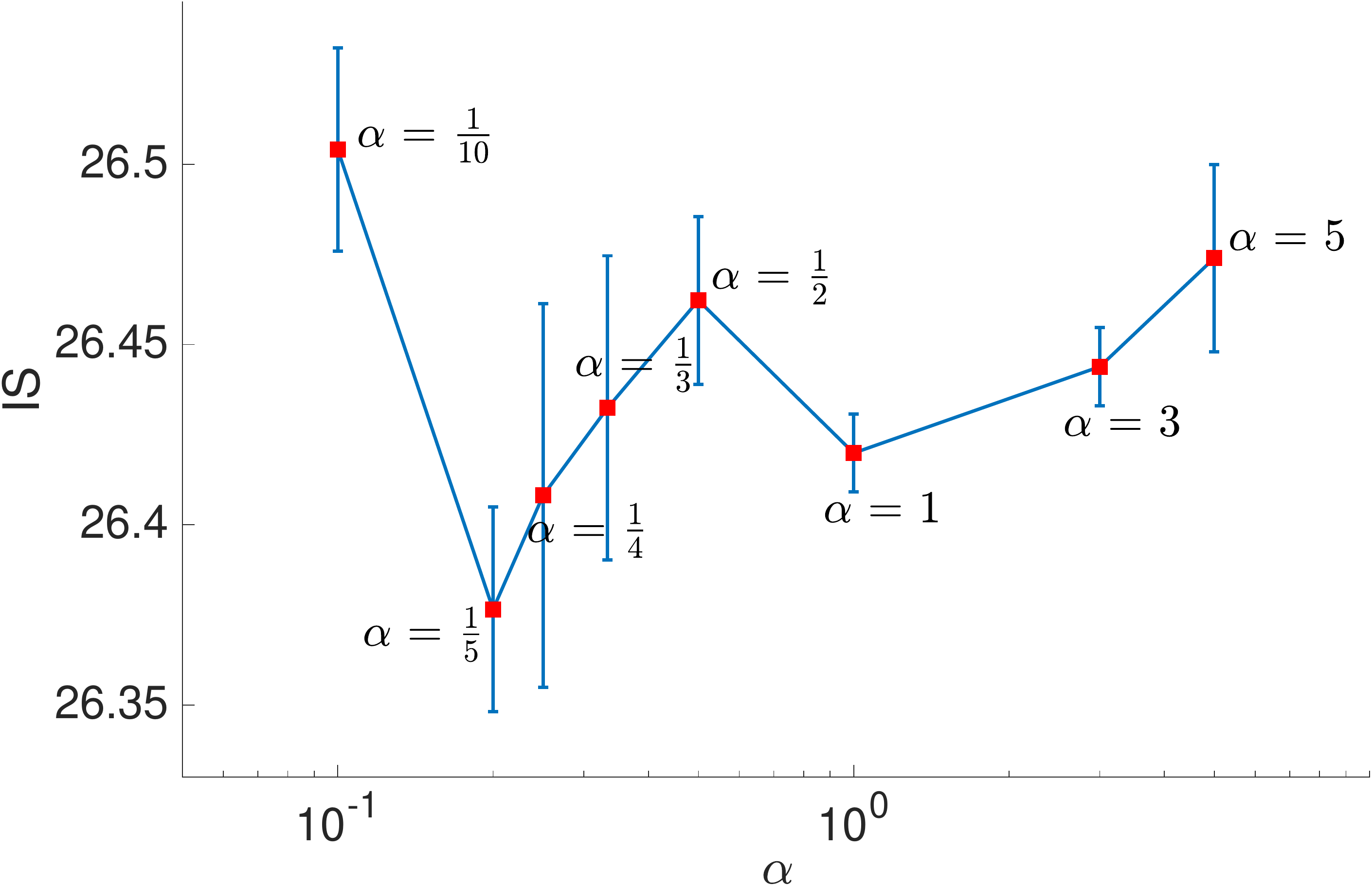}}

	\end{tabular}
	
	\protect
	\caption[Evaluation of PSO-LDE for estimation of \emph{Transformed Columns} distribution, where NN architecture is BD. Different values of a hyper-parameter $\alpha$ are verified.]{Evaluation of PSO-LDE for estimation of \emph{Transformed Columns} distribution, where NN architecture is BD with 6 layers, number of blocks $N_B = 50$ and block size $S_B = 64$ (see Section \ref{sec:BDLayers}). The \down density $\probs{\dsuff}$ is Uniform in the top row (a)-(b)-(c), and Gaussian in the bottom row (d)-(e)-(f).
		For different values of a hyper-parameter $\alpha$, (a)-(d) $PSQR$, (b)-(e) $LSQR$ and (c)-(f) $IS$  are reported, along with their empirical standard deviation. As observed, when $\probs{\dsuff}$ is Uniform, PSO-LDE fails to learn the data density due to large \emph{relative support mismatch}.
		On the other hand, when $\probs{\dsuff}$ is Gaussian, the target surface is accurately approximated.
	}
	\label{fig:TrColsRes1}
\end{figure}

\paragraph{Uniform \boldmath$\probs{\dsuff}$}

First, we evaluate PSO-LDE for different values of $\alpha$ on the density inference task. The applied model is BD with 6 layers, number of blocks $N_B = 50$ and block size $S_B = 64$. The dataset size is $N_{DT} = 10^8$ and $\probs{\dsuff}$ is Uniform. In Figures \ref{fig:TrColsRes1-a}-\ref{fig:TrColsRes1-b}-\ref{fig:TrColsRes1-c} we can see the corresponding errors for learned models. The errors are huge, implying that the inference task failed. The reason for this is the \emph{relative support mismatch}
between Uniform and \emph{Transformed Columns} densities. After transformation by matrix $A$ the samples from $\probs{TrClmns}$ are more widely spread out within the space $\RR^{20}$. The range of samples along each dimension is now around $[-10, 10]$ instead of the corresponding range $[-2.3, 2.3]$ in \emph{Columns} distribution. Yet, the samples are mostly located in a small subspace of \emph{hyperrectangle} $\mathcal{R} = [-10, 10]^{20}$. When we choose $\probs{\dsuff}$ to be Uniform with $\mathcal{R}$ as its support, most of the samples $X^{\usuff}$ and $X^{\dsuff}$ are located in different areas of this huge \emph{hyperrectangle} and cannot balance each other to reach PSO equilibrium. Such \emph{relative support mismatch} prevents  proper learning of the data density function.

\paragraph{Gaussian \boldmath$\probs{\dsuff}$}

Next, instead of Uniform we used Gaussian distribution $\mathcal{N}(\mu, \Sigma)$ as our \down density $\probs{\dsuff}$. The mean vector $\mu$ is equal to the mean of samples from $\probs{\usuff}$; the $\Sigma$ is a diagonal matrix whose non-zero values are empirical variances for each dimension of available \up samples. In Figures \ref{fig:TrColsRes1-d}-\ref{fig:TrColsRes1-e}-\ref{fig:TrColsRes1-f} we can observe that overall achieved accuracy is high, yet it is worse than the results for \emph{Columns} distribution. Such difference can be again explained by a mismatch between $\probs{\usuff}$ and $\probs{\dsuff}$ densities. While \emph{Columns} and Uniform densities in Section \ref{sec:ColumnsEst} are relatively aligned to each other, the \emph{Transformed Columns} and Gaussian distributions have a bounded ratio $\frac{\probi{\usuff}{X}}{\probi{\dsuff}{X}}$ only around their mean point $\mu$. In far away regions such ratio becomes too big/small, causing PSO inaccuracies. Hence, we argue that a better choice of $\probs{\dsuff}$ would further improve the accuracy.

\begin{table}

	\centering
	\begin{tabular}{llll}
		\toprule
		Method     & $PSQR$ & $LSQR$     & $IS$ \\
		\midrule
		PSO-LDE, & $1.6 \cdot 10^{-21} \pm  3.17 \cdot 10^{-22}$ &
		$0.442 \pm 0.095$  &
		$26.44 \pm 0.05$  \\
		averaged over all $\alpha$ \\
		\midrule
		IS & $9.3 \cdot 10^{-21} \pm 3.4 \cdot 10^{-24}$ &
		$26.63 \pm 0.86$  &
		$31.3 \pm 0.04$  \\
		\midrule
		PSO-MAX & $2.1 \cdot 10^{-21} \pm 2.96 \cdot 10^{-22}$ &
		$0.54 \pm 0.088$  &
		$26.52 \pm 0.03$  \\
		
		\bottomrule
	\end{tabular}
	
	\caption{Performance comparison between various PSO instances for \emph{Transformed Columns} density}
	\label{tbl:Perf3}
\end{table}

Additionally, we can see that in case of \emph{Transformed Columns} PSO-LDE with $\alpha = \frac{1}{5}$ achieves the smallest error, with its $LSQR$ being $0.32 \pm 0.02$. Additionally, we evaluated the Importance Sampling (IS) method from Table \ref{tbl:PSOInstances1} and PSO-MAX defined in Eq.~(\ref{eq:PSOMaxLogEstNormalizedMU})-(\ref{eq:PSOMaxLogEstNormalizedMD}). In Table \ref{tbl:Perf3} we see that the performance of PSO-MAX is slightly worse than of PSO-LDE, similarly to what was observed for \emph{Columns}. Moreover, the IS fails entirely, producing a very large error. Furthermore, in order to stabilize its learning process we were required to reduce the learning rate from $0.0035$ to $0.0001$. Hence, here we can see again the superiority of bounded \mfs over not bounded.

\begin{figure}[tb]
	\centering
	
	\begin{tabular}{c}
		
		\subfloat{\includegraphics[width=0.45\textwidth]{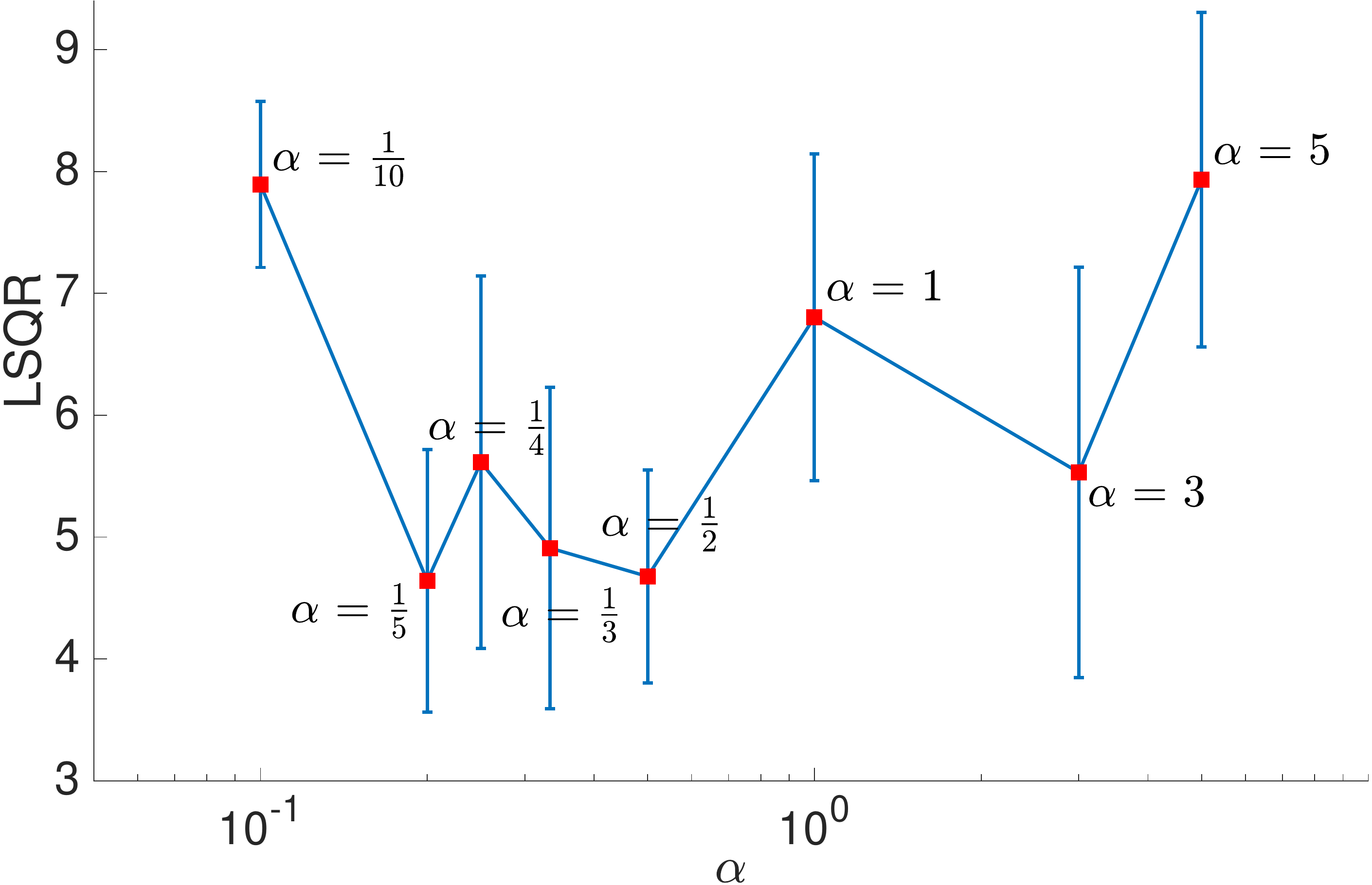}}
		
	\end{tabular}
	
	\protect
	\caption[Evaluation of PSO-LDE for estimation of \emph{Transformed Columns} distribution, where NN architecture is FC.]{Evaluation of PSO-LDE for estimation of \emph{Transformed Columns} distribution, where NN architecture is FC with 4 layers of size 1024 (see Section \ref{sec:BDLayers}). For different values of a hyper-parameter $\alpha$, $LSQR$ error is reported, along with its empirical standard deviation. Again, the small values of $\alpha$ (around $\frac{1}{5}$) have a lower error.
	}
	\label{fig:TrColsRes2}
\end{figure}

\paragraph{Various NN Architectures}

Additionally, we evaluated several different NN architectures for \emph{Transformed Columns} distribution, with $\probs{\dsuff}$ being Gaussian. In Figure \ref{fig:TrColsRes2} we report the performance for FC networks. As observed, the FC architecture has a higher error w.r.t. BD architecture in Figure \ref{fig:TrColsRes1-e}. Moreover, PSO-LDE with $\alpha$ around $\frac{1}{5}$ performs better.

\begin{figure}[tb]
	\centering
	
	\begin{tabular}{c}

		\subfloat{\includegraphics[width=0.98\textwidth]{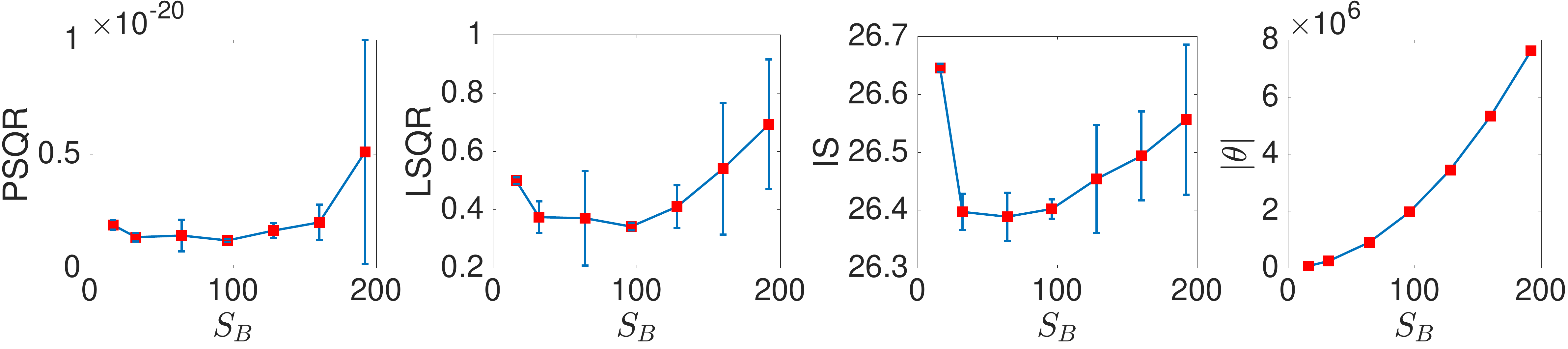}}
		
	\end{tabular}
	
	\protect
	\caption[Evaluation of PSO-LDE for estimation of \emph{Transformed Columns} distribution, where NN architecture is BD with various block sizes.]{Evaluation of PSO-LDE for estimation of \emph{Transformed Columns} distribution, where NN architecture is BD with 6 layers and number of blocks $N_B = 50$ (see Section \ref{sec:BDLayers}). The block size $S_B$ is taking values $\{16, 32, 64, 96, 128, 160, 192 \}$. The applied loss is PSO-LDE with $\alpha = \frac{1}{5}$.
		For different values of $S_B$ we report $PSQR$, $LSQR$ and $IS$, and their empirical standard deviation. Additionally, in the last column we depict the size of $\theta$ for each value of $S_B$.
	}
	\label{fig:TrColsRes3}
\end{figure}

Further, in Figure \ref{fig:TrColsRes3} we experiment with BD architecture for different values of the block size $S_B$. Similarly to what was observed in Section \ref{sec:ColumnsEstNNAME} for \emph{Columns} distribution, also here a too small/large block size has worsen accuracy.

\begin{figure}[tb]
	\centering
	
	\begin{tabular}{cc}

		\subfloat[\label{fig:TrColsRes4-a}]{\includegraphics[width=0.4\textwidth]{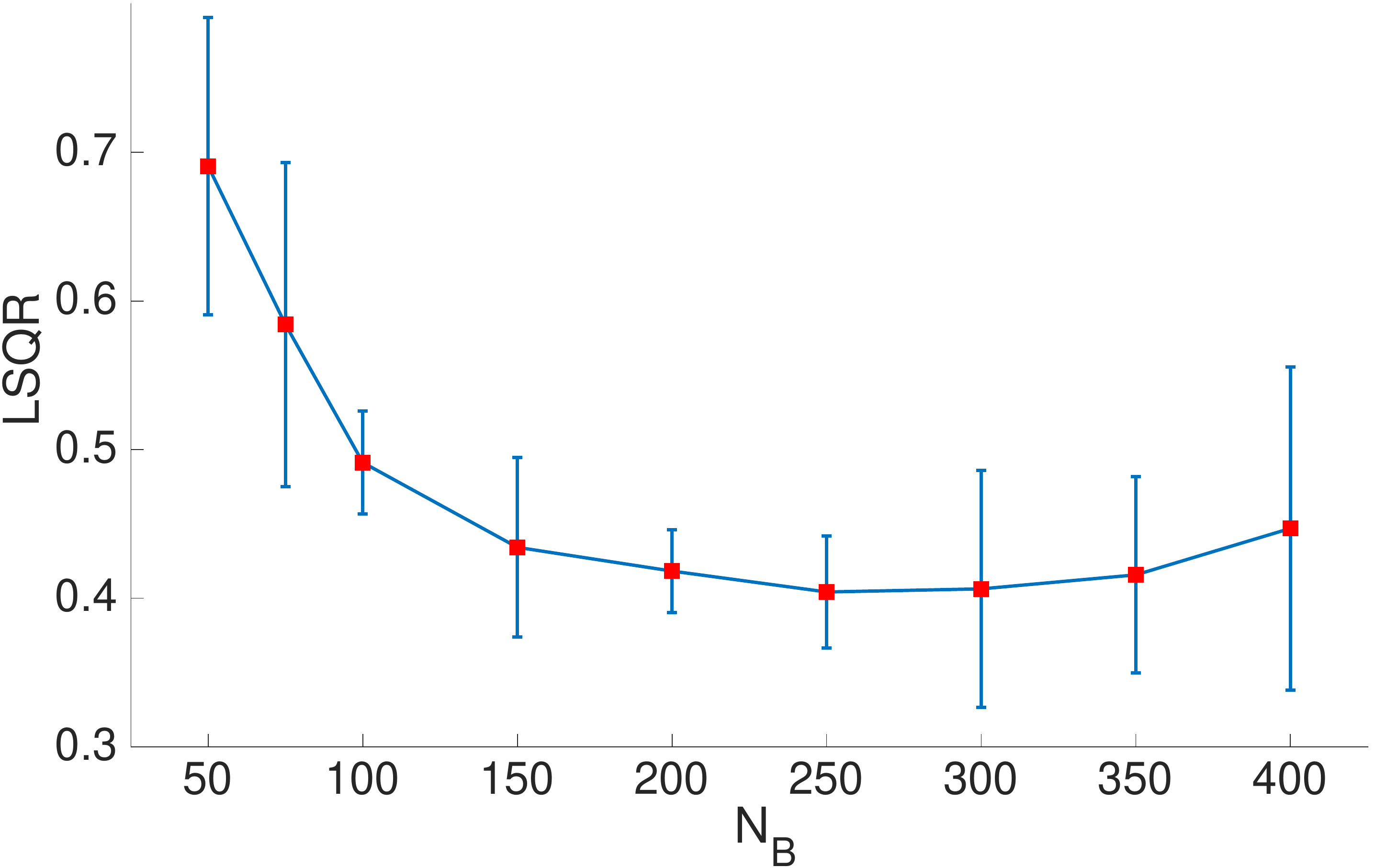}}
		&
		\subfloat[\label{fig:TrColsRes4-b}]{\includegraphics[width=0.4\textwidth]{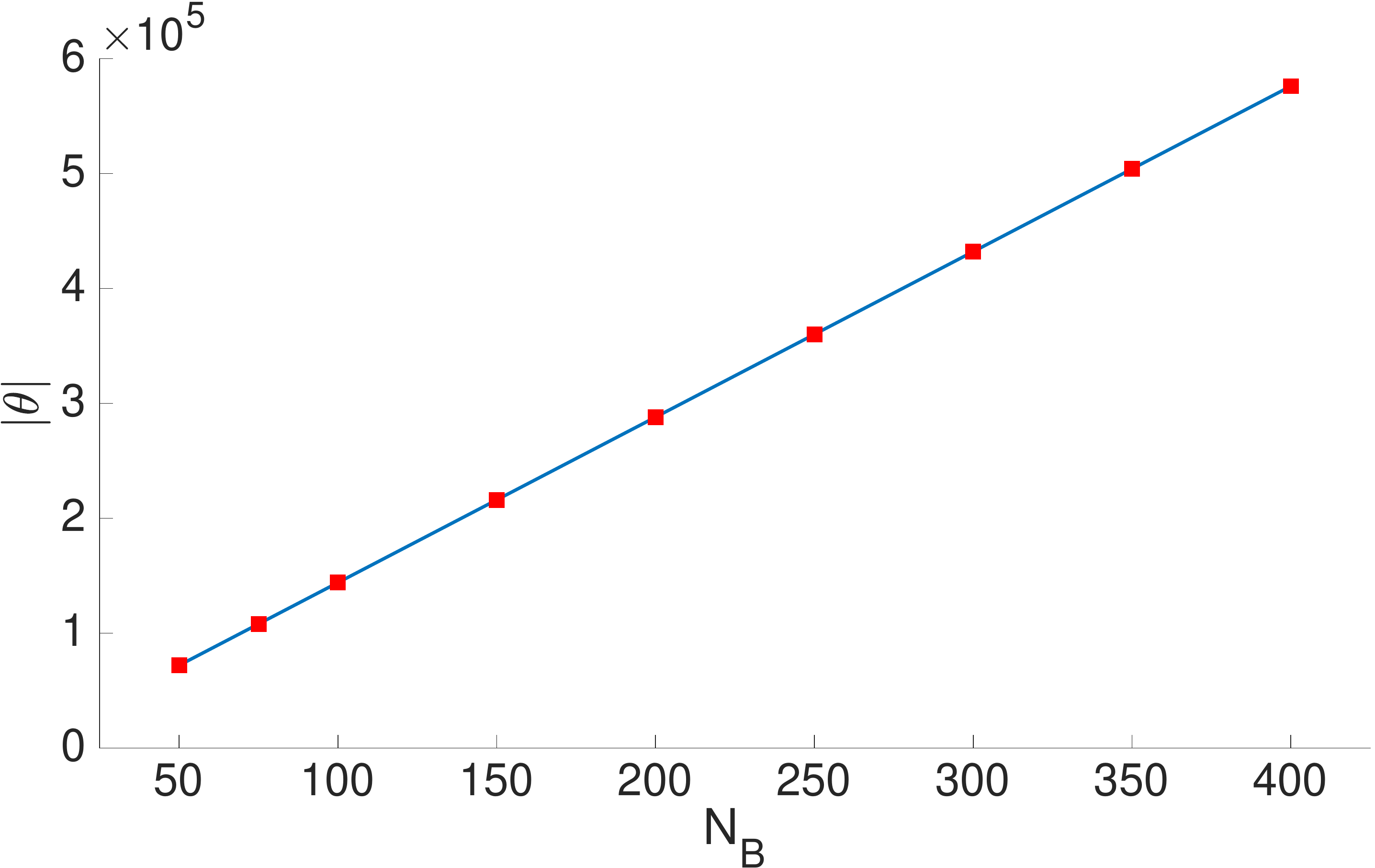}}
		
		\\
		
	\end{tabular}
	
	\protect
	\caption[Evaluation of PSO-LDE for estimation of \emph{Transformed Columns} distribution, where NN architecture is BD with various numbers of blocks.]{Evaluation of PSO-LDE for estimation of \emph{Transformed Columns} distribution, where NN architecture is BD with 6 layers and block size $S_B = 16$ (see Section \ref{sec:BDLayers}). The number of blocks $N_B$ is changing. The applied loss is PSO-LDE with $\alpha = \frac{1}{5}$.
		For different values of $N_B$, (a) $LSQR$ error and (b) the number of weights $\left| \theta \right|$ are reported.
	}
	\label{fig:TrColsRes4}
\end{figure}

We also perform additional experiments for BD architecture where this time we use small blocks, with $S_B = 16$. In Figure \ref{fig:TrColsRes4} we see that for such networks the bigger number of blocks per layer $N_B$, that is the bigger number of transformation channels, yields better performance. However, we also observe the drop in an accuracy when the number of these channels grows considerably (above 250). Furthermore, no matter how big is $N_B$, the models with small blocks ($S_B = 16$) in Figure \ref{fig:TrColsRes4} produced inferior results w.r.t. model with big blocks ($S_B = 64$, see Figure \ref{fig:TrColsRes1-e} for $\alpha = \frac{1}{5}$).

\begin{figure}[!tb]
	\centering
	
	\begin{tabular}{c}

		\subfloat[\label{fig:TrColsRes5-a}]{\includegraphics[width=0.98\textwidth]{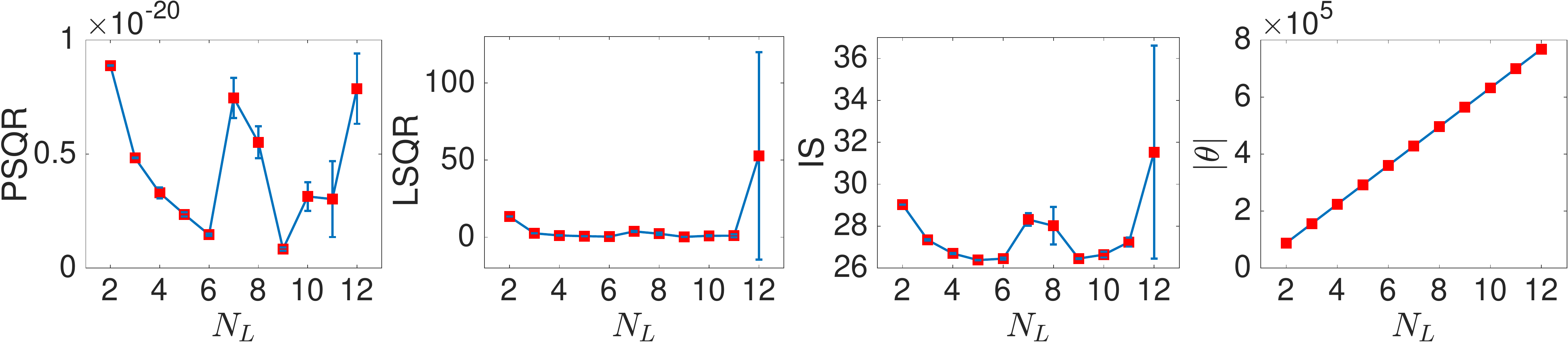}}
		\\
		
		\subfloat[\label{fig:TrColsRes5-b}]{\includegraphics[width=0.9\textwidth]{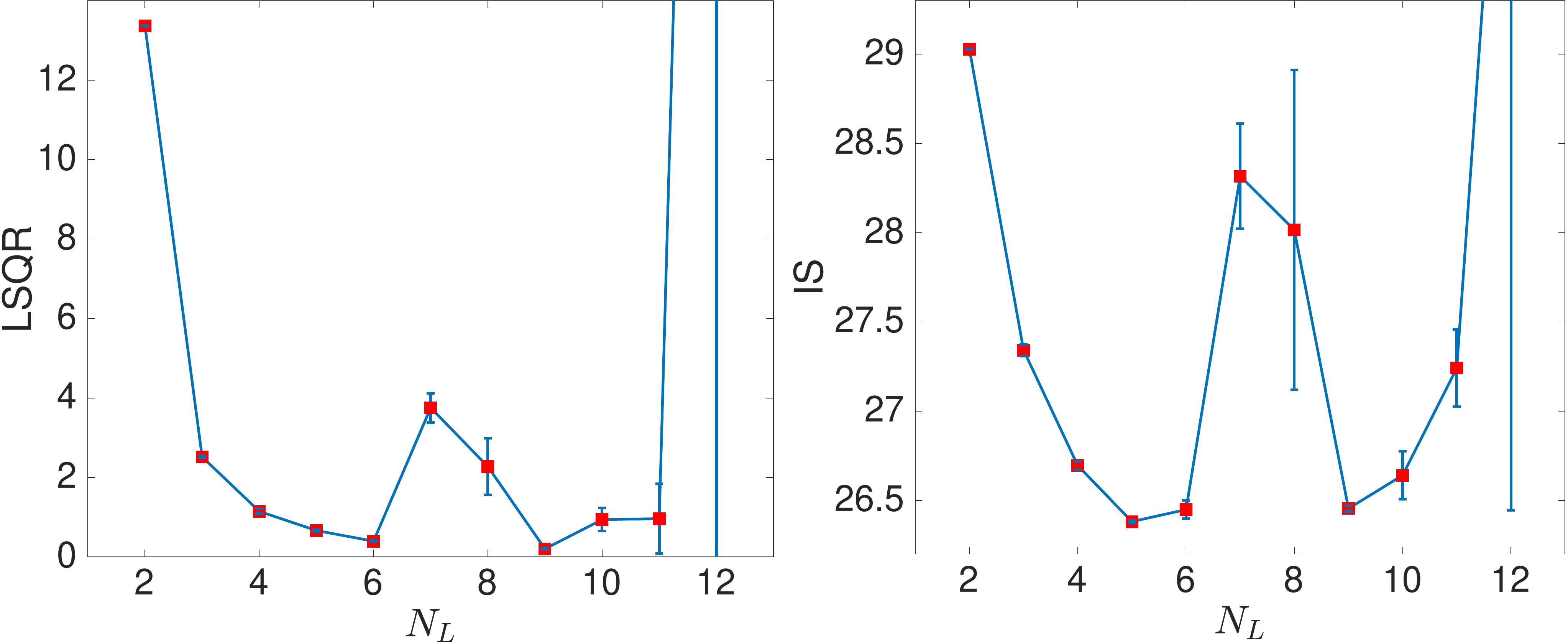}}
		
	\end{tabular}
	
	\protect
	\caption[Evaluation of PSO-LDE for estimation of \emph{Transformed Columns} distribution, where NN architecture is BD with various numbers of layers.]{Evaluation of PSO-LDE for estimation of \emph{Transformed Columns} distribution, where NN architecture is BD with number of blocks $N_B = 250$ and block size $S_B = 16$ (see Section \ref{sec:BDLayers}). The number of layers $N_L$ is changing from 2 to 12. The applied loss is PSO-LDE with $\alpha = \frac{1}{5}$.
		(a) For different values of $N_L$ we report $PSQR$, $LSQR$ and $IS$, and their empirical standard deviation. Additionally, in the last column we depict the size of $\theta$ for each value of $N_L$.
		(b) Zoom of (a).
	}
	\label{fig:TrColsRes5}
\end{figure}

Further, we evaluated the same small-block network for a different number of layers $N_L$, with $S_B = 16$ and $N_B = 250$. In Figure \ref{fig:TrColsRes5} we can see that when $N_L$ grows from 2 to 6, the overall performance is getting better. Yet, for a larger number of layers the performance trend is inconsistent. When $N_L$ is between 7 and 11, some values of $N_L$ are better than other, with no evident improvement pattern for large $N_L$. Moreover, for $N_L = 12$ the error grows significantly. More so, we empirically observed that 2 out of 5 runs of this setting totally failed due to the zero loss gradient. Thus, the most likely conclusion for this setting is that for a too deep networks the signal from input fails to reach its end, which is a known issue in DL domain. Also, from our experiments it follows that  learning still can succeed, depending on the initialization of network weights.

Furthermore, along with the above inconsistency that can occur for too deep networks, for $N_L = 9$ in Figure \ref{fig:TrColsRes5} we still received our best results on \emph{Transformed Columns} dataset, with the mean $LSQR$ being 0.204. This leads to the conclusion that even when BD network uses blocks of a small size ($S_B = 16$), it still can produce a superior performance if it is deep enough. Besides, the zero-gradient issue may be tackled by adding shortcut connections between various layers \citep{He16cvpr}.

Overall, in our experiments we saw that when $\probs{\usuff}$ and $\probs{\dsuff}$ are not properly aligned (i.e.~far from each other), PSO fails entirely. 
Further, PSO-LDE with values of $\alpha$ around $\frac{1}{5}$ was observed to perform better, which is an another superiority evidence for small values of $\alpha$.
Moreover, IS with unbounded \mfs could not be applied at all for faraway densities. BD architecture showed again a significantly higher accuracy over FC networks. Block size $S_B = 96$ produced a better inference. Finally, for BD networks with small blocks ($S_B = 16$) the bigger number of blocks $N_B$ and the bigger number of layers $N_L$ improve an accuracy. However, at some point the increase in both can cause the performance drop.

\subsection{PDF Estimation via PSO - 3D \emph{Image-based} Densities}
\label{sec:ImageEst}

\begin{figure}[tbp]
	\centering
	
	\begin{tabular}{cc}
		
		\subfloat[\label{fig:ImagePDF1-a}]{\includegraphics[width=0.46\textwidth]{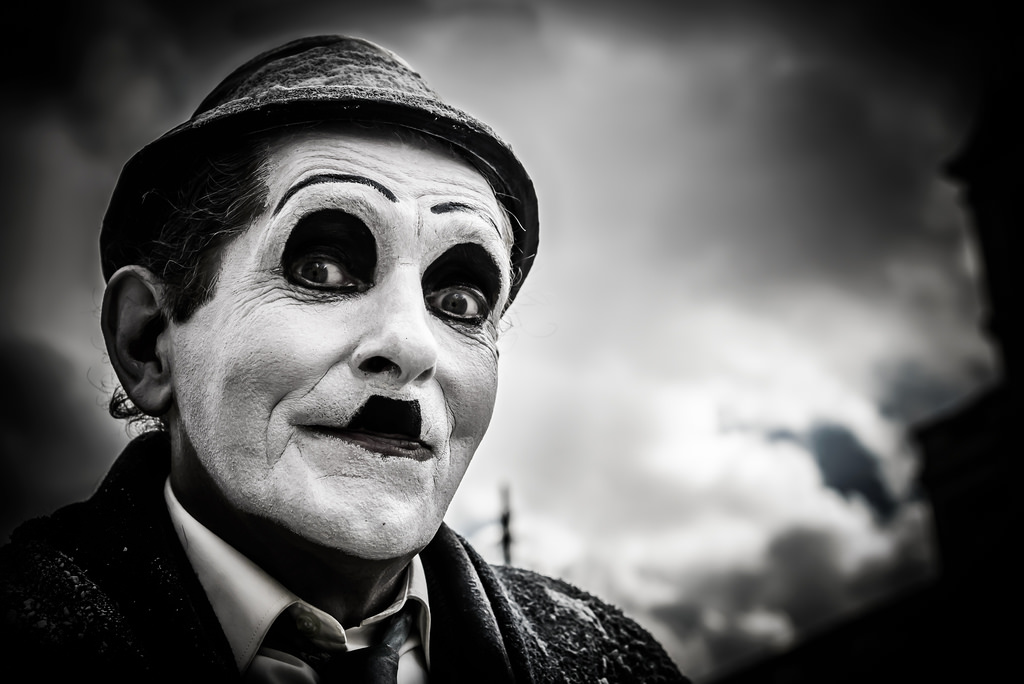}}
		&
		
		\subfloat[\label{fig:ImagePDF1-b}]{\includegraphics[width=0.46\textwidth]{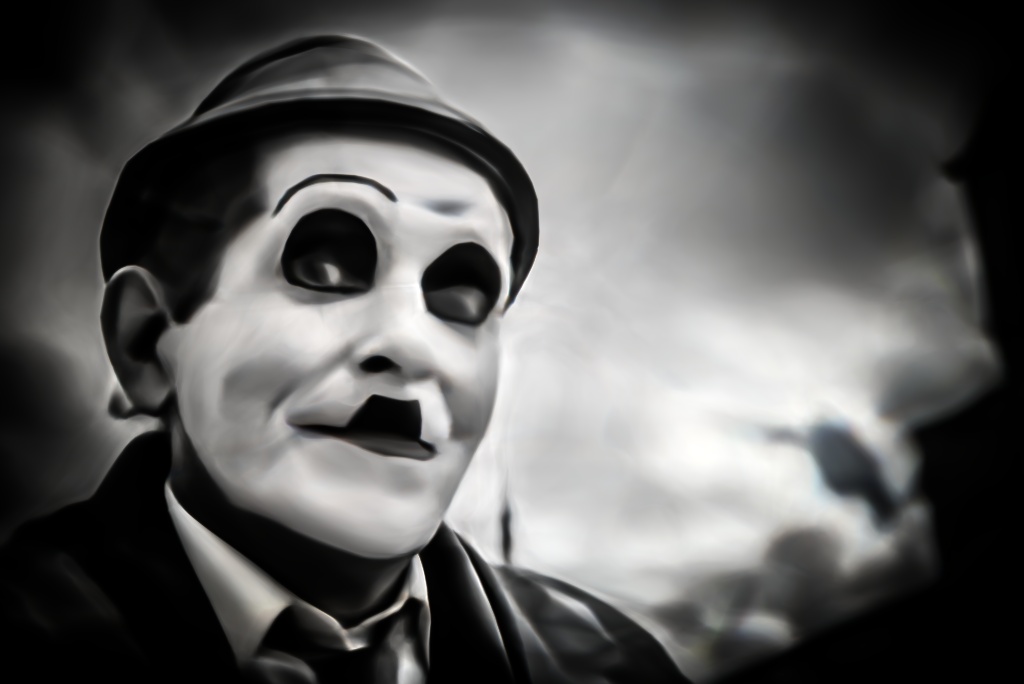}}
		\\	
		
		\subfloat[\label{fig:ImagePDF1-c}]{\includegraphics[width=0.46\textwidth]{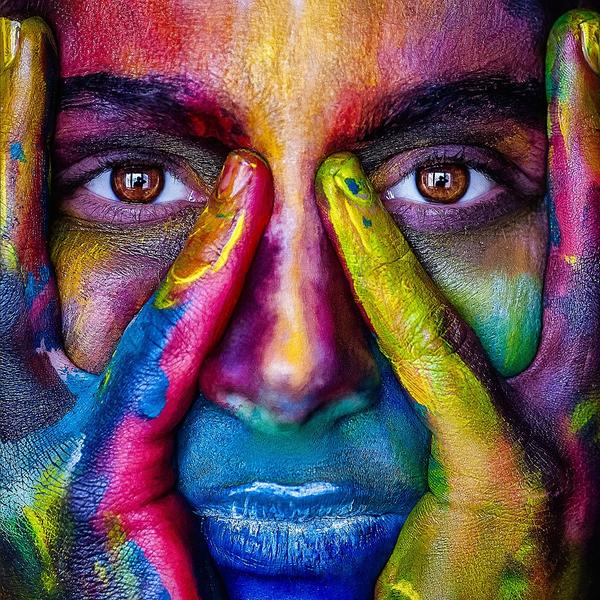}}
		&
		
		\subfloat[\label{fig:ImagePDF1-d}]{\includegraphics[width=0.46\textwidth]{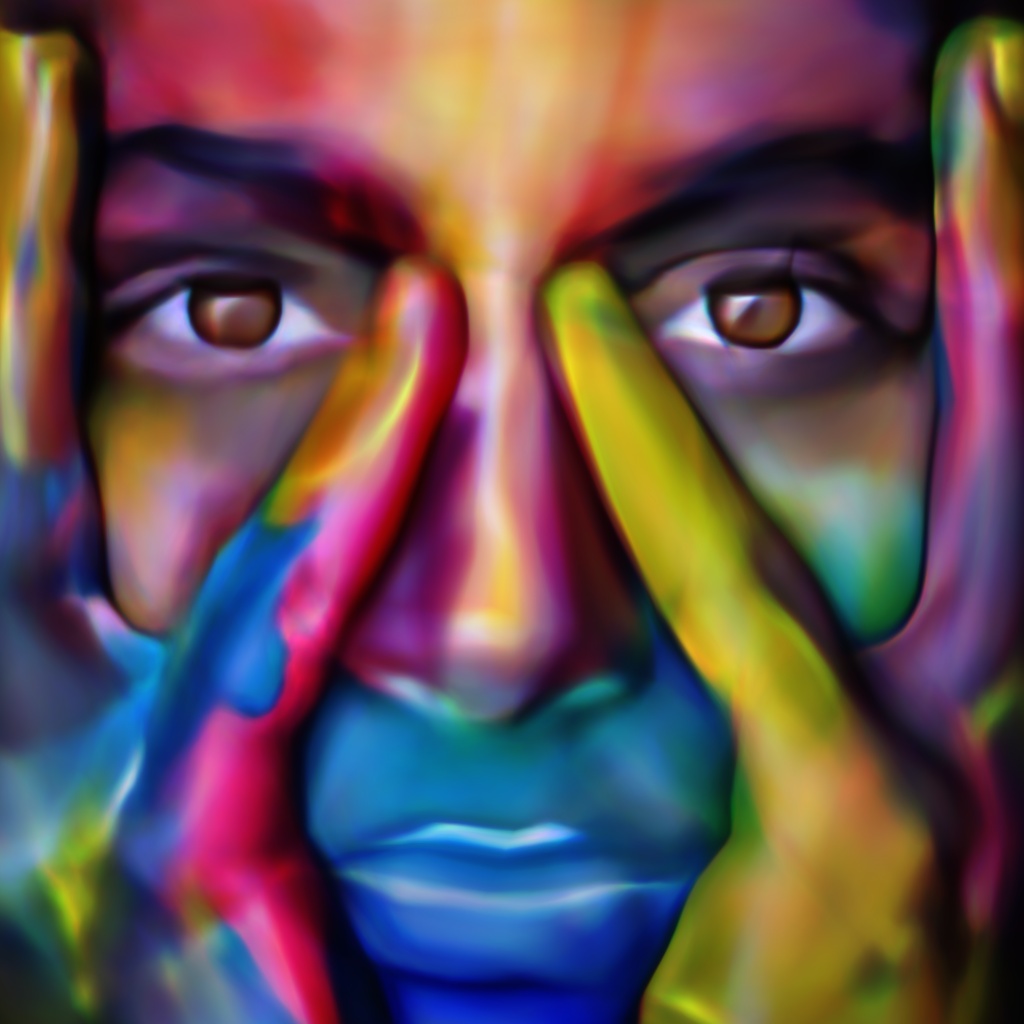}}
		\\

		\subfloat[\label{fig:ImagePDF1-e}]{\includegraphics[width=0.46\textwidth]{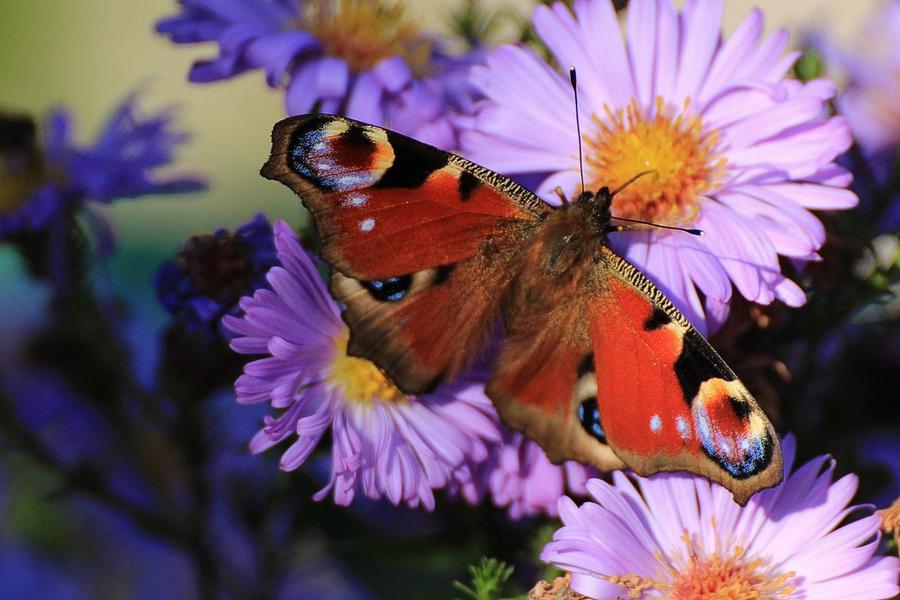}}
		&
		
		\subfloat[\label{fig:ImagePDF1-f}]{\includegraphics[width=0.46\textwidth]{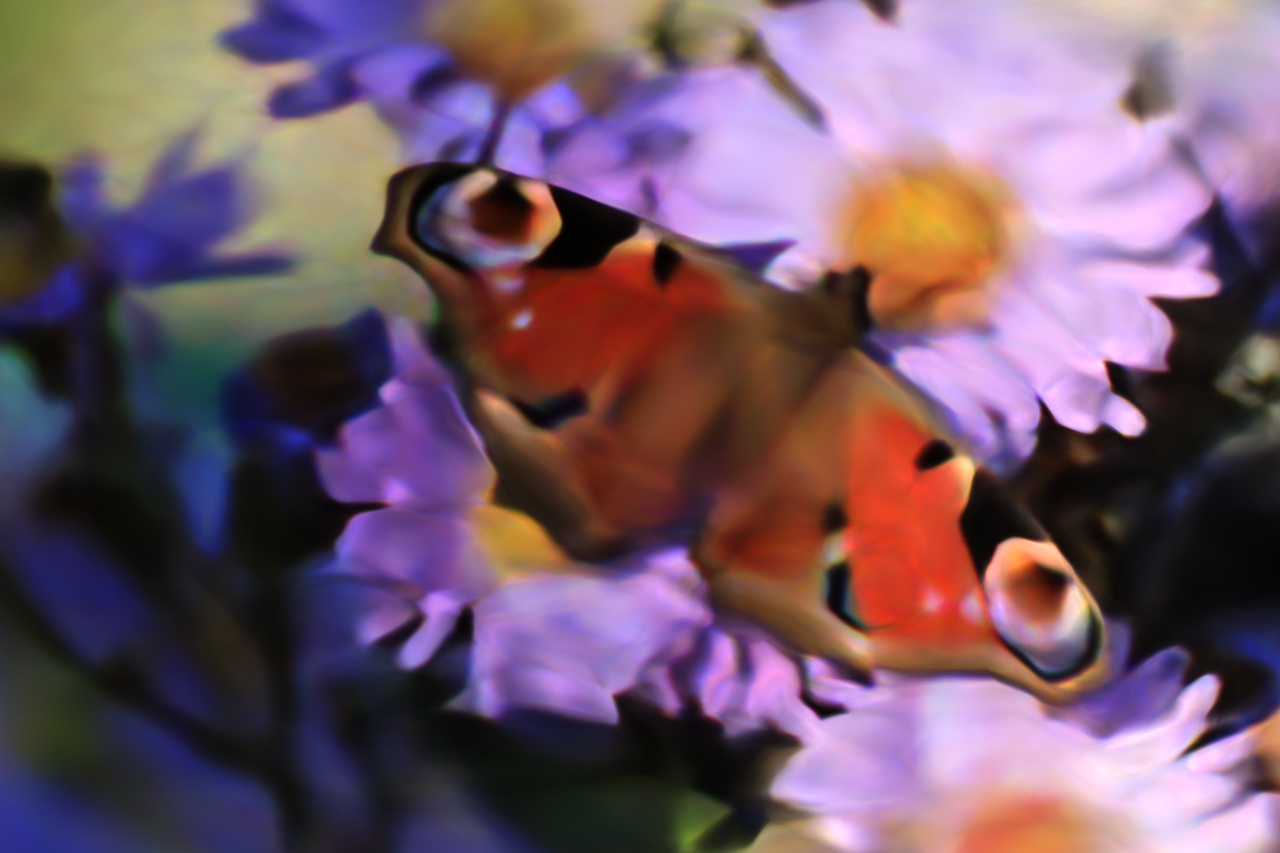}}
		
	\end{tabular}
	
	\protect
	\caption{(a),(c),(e) \emph{Image-based} densities and (b),(d),(f) their approximations - Part 1.
	}
	\label{fig:ImagePDF1}
\end{figure}

\begin{figure}
	\centering
	
	\newcommand{\width}[0] {0.41}
	
	\begin{tabular}{cc}

		\subfloat[\label{fig:ImagePDF2-a}]{\includegraphics[width=\width\textwidth]{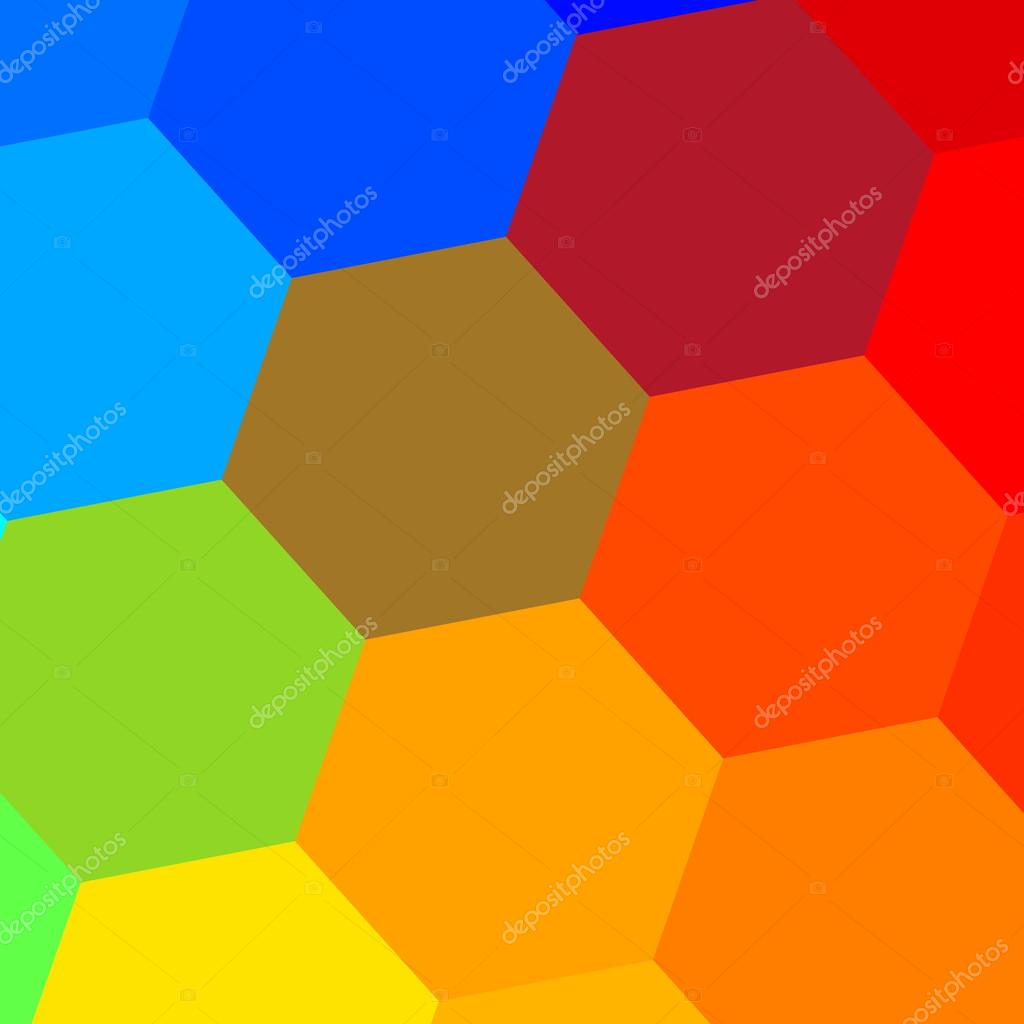}}
		&
		
		\subfloat[\label{fig:ImagePDF2-b}]{\includegraphics[width=\width\textwidth]{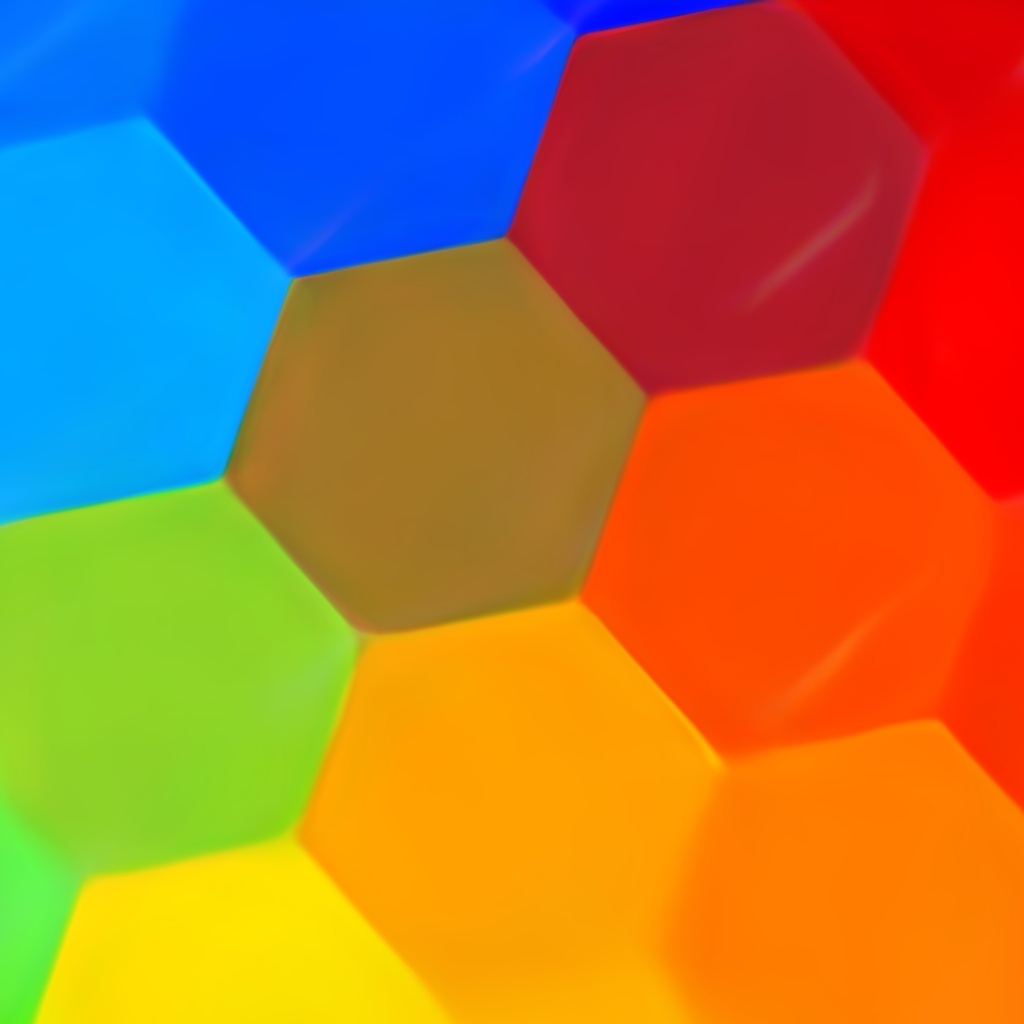}}
		\\

		\subfloat[\label{fig:ImagePDF2-c}]{\includegraphics[width=\width\textwidth]{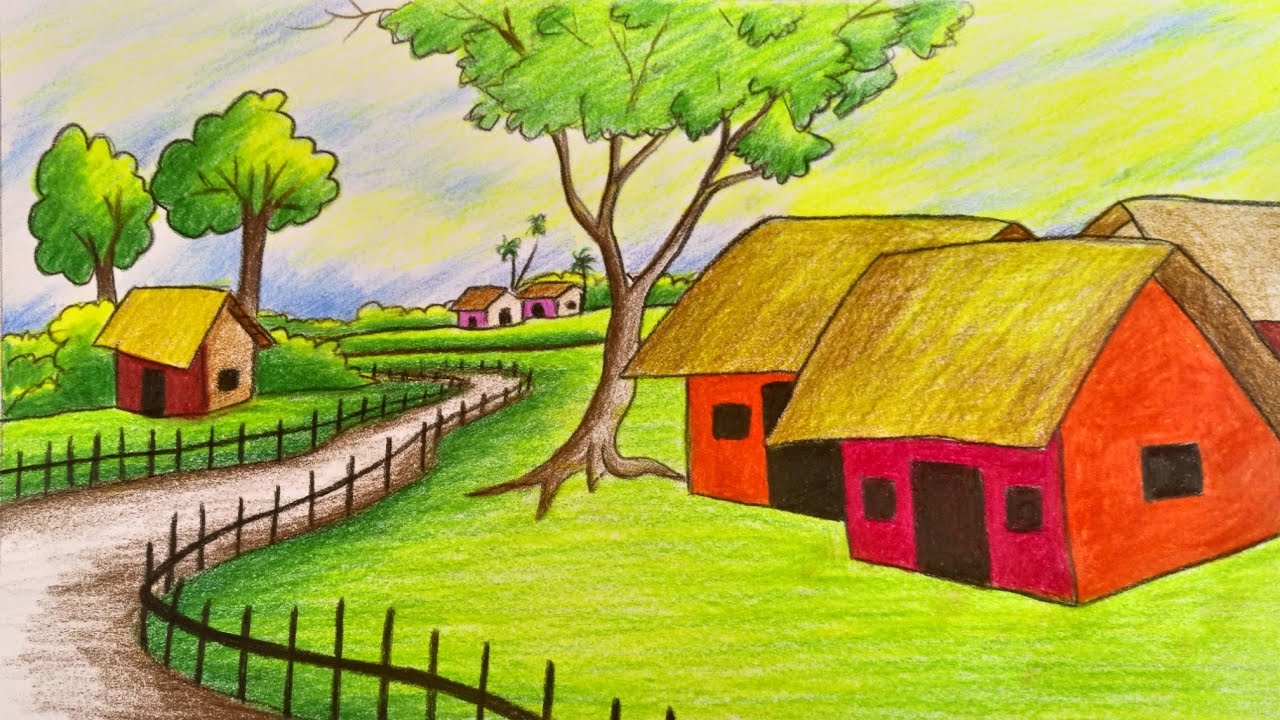}}
		&
		
		\subfloat[\label{fig:ImagePDF2-d}]{\includegraphics[width=\width\textwidth]{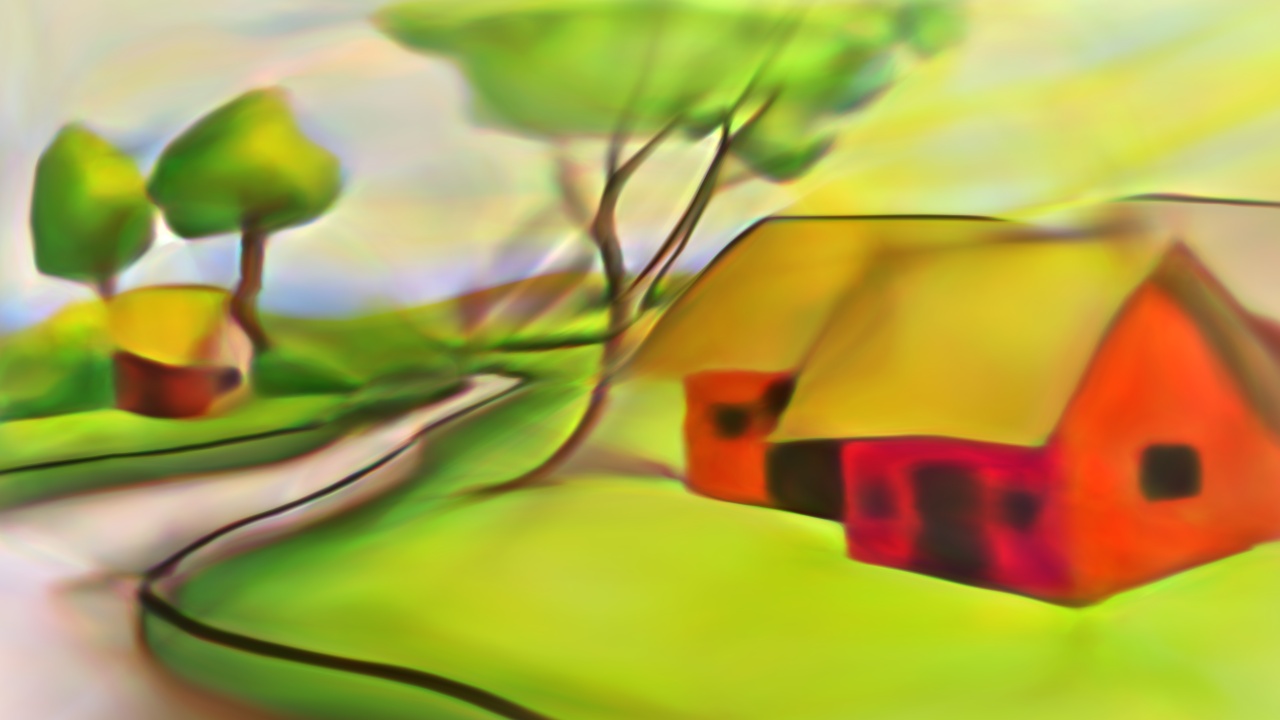}}
		\\

		\subfloat[\label{fig:ImagePDF2-e}]{\includegraphics[width=\width\textwidth]{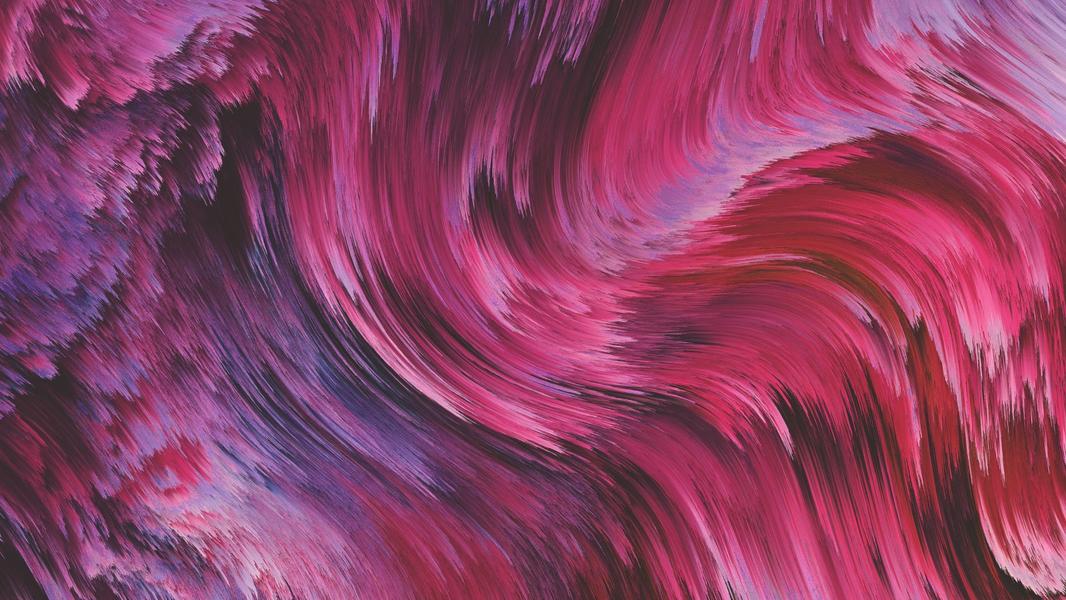}}
		&
		
		\subfloat[\label{fig:ImagePDF2-f}]{\includegraphics[width=\width\textwidth]{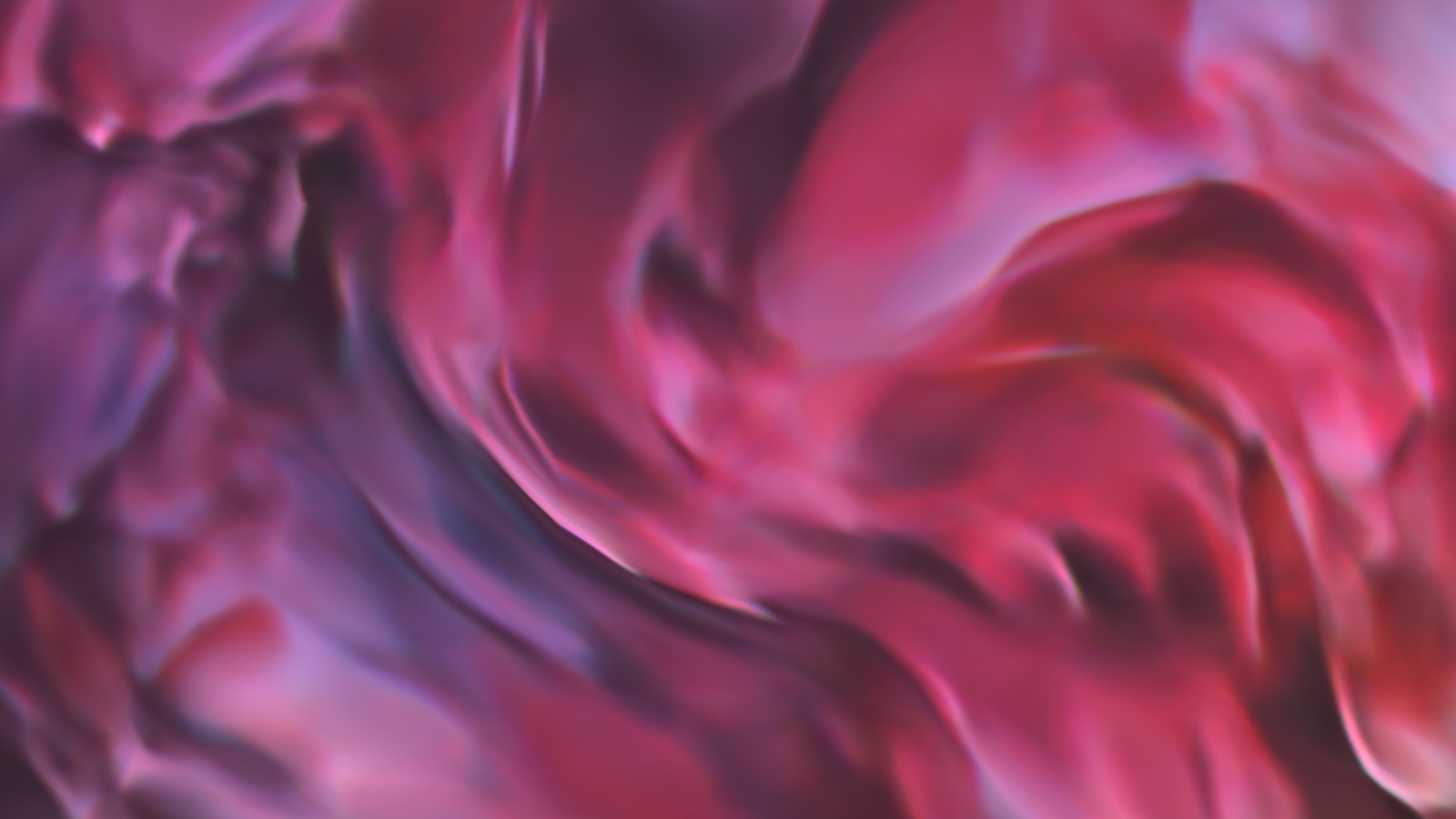}}

		%
		%
		%
		
	\end{tabular}
	
	\protect
	\caption{(a),(c),(e) \emph{Image-based} densities and (b),(d),(f) their approximations - Part 2.
	}
	\label{fig:ImagePDF2}
\end{figure}

In order to further evaluate the presented herein PSO-LDE density estimation approach, we use intricate 3D densities that are based on image surfaces. More specifically, we consider a given RGB image $I$ as a function $F(x,y,c)$ from $\RR^3$ to $\RR$ where $x$, $y$ and $c$ represent width, height and color channel of $I$ respectively. For  simplicity we define the range for each input scalar variable from $[x, y, c]$ to be $[0, 1]$, with $F: [0, 1]^3 \rightarrow [0, 1]$. We use grid of values from an image $I$ to appropriately interpolate outputs of the function $F(X)$ at any input point $X \in [0, 1]^3 \subseteq \RR^3$. In our experiments we used a linear interpolation.

Further, we use $F$ as pdf function which we sample to create a dataset for our PSO-LDE experiments. Yet, the function $F$, interpolated from image $I$, is not a valid pdf function since its integral can be any positive number. Thus, we normalize it by its total integral to get a normalized function $\bar{F}$ which we use as an intricate 3D pdf function for further density estimation evaluation. 
Next, we sample the density $\bar{F}$ via rejection sampling method and gather a dataset of size $10^7$ per each \emph{image-based} density. 
Furthermore, we approximate the target surface $\log \bar{F}(X)$ via PSO-LDE only from these sampled points.

Note that $\log \bar{F}$ has a very sophisticated structure since used images have a very high contrast and  nearby pixels typically have significantly different values. This makes $\log \bar{F}$ a highly non-linear function, which cannot be easily approximated by typical parametric density estimation techniques. Yet, as we will see below, due to the approximation power of DL, which is exploited in full by PSO, and due to a high flexibility of the proposed BD architecture, the non-parametric PSO-LDE allows us to accurately estimate even such complex distributions. Importantly, we emphasize that the evaluated distributions in this section have their support in $\RR^3$, and are \textbf{not} some very high-dimensional densities over image data that can be often encountered in DL domain.

We use PSO-LDE with $\alpha = \frac{1}{4}$ in order to learn the target $\log \bar{F}$. The applied NN architecture is block-diagonal with a number of blocks $N_B = 15$, a block size $S_B = 64$ and a number of layers $N_L = 14$ (see Section \ref{sec:BDLayers}). In order to tackle the problem of vanishing gradients in such a deep model, we introduced shortcut connections into our network, where each layer has a form $u_{i} = h_{i}(u_{i - 1}; \theta_{i}) + u_{i - 1}$ with $u_i$ and $h_{i}(\cdot; \theta_{i})$ being respectively the output and the applied transformation function of $i$-th BD layer within the BD network. Note that in this paper we use such shortcuts only for the experiments of this section.
Furthermore, after training is finished, we convert the inferred $\bar{F}$ back into an image format, producing an \emph{inferred image} $I'$.

In Figures \ref{fig:ImagePDF1} and \ref{fig:ImagePDF2}
we show \emph{inferred images} for several input images. 
As can be seen, there is high resemblance between both input $I$ and inferred $I'$.
That is, PSO-LDE succeeded to accurately infer densities even with a very complicated surface.
Note that each \emph{image-based} density was inferred by using identical hyper-parameters, that are the same as in the rest of our experiments (except for NN structure where shortcuts were applied). Additional parameter-tuning per a specific input image will most probably improve the produced herein results.

\section{Conclusions and Future Work}
\label{sec:Concl}

In this paper we contributed a new algorithm family, \emph{Probabilistic Surface Optimization} (PSO), that allows to learn numerous different statistical functions of given data, including (conditional) density estimation and ratios between two unknown pdfs of two given datasets. In our work we found a new perspective to view a model as a representation of a virtual physical surface, which is pushed by the PSO algorithm \up and \down via gradient descent (GD) optimization updates. Further, the equilibrium at each point, that is, when \up and \down forces are point-wise equal, ensures that the converged surface satisfies PSO \bp, where the ratio of the frequency components is equal to the opposite ratio of the analytical components. In Section \ref{sec:PSOInst} we saw that such formulation yields infinitely many estimation approaches to learn almost any function of the density ratio. Moreover,
it generalizes numerous existing works, like \emph{energy} and \emph{unnormalized} models as also critics of GAN approaches. Likewise, we showed that $f$-divergence and Bregman divergence based techniques (e.g. the cross-entropy loss from the image classification domain) are also instances of PSO, applying the same physical forces over the model surface.

We provided a thorough analysis of the \psofunc implicitly employed during the optimization, describing its equilibrium for a wide diapason of settings. Furthermore, we derived the sufficient conditions over PSO \mfs under which the equilibrium is stable. We likewise related PSO to Legendre-Fenchel transform, demonstrating that its convergence is an inverse of the \mgn ratio, with their \ps being convex-conjugate of each other. This resembles the relationship between Langrangian and Hamiltonian mechanics, opening interesting future directions to connect control and learning theories.

Furthermore, we systematically modulated the set of all PSO instances into various subgroups, providing a useful terminology for a future PSO study. Additionally, along this paper we described several possible parameterizations of PSO family, with each having its own benefits. Concretely, PSO can be represented/parametrized via a pair of \ms \pair which leads to the geometrical/physical force perspective. Such angle brings many insights and is the central focus of this work. Likewise, we can parametrize PSO by \widepair that leads to the \psofunc, which may be viewed as an energy of the optimized physical system. Further, the polar parametrization $\{ c_{r}, c_{\angle} \}$ in Section \ref{sec:ConVerComp} permits for an easier feasibility verification. Lastly, $\{ \phi^c, G \}$ described in Section \ref{sec:fDivv} allows to connect PSO methods with $f$-divergence between \up and \down densities.

Moreover, the \textbf{main} goal behind this work is to introduce a novel universal way in forging new statistical techniques and corresponding objective functions.
Due to simplicity and intuitiveness of the presented PSO principles, this new framework allows for an easy derivation of new statistical approaches, which, in turn, is highly useful in many different domains. Depending on the target function required by a specific application, a data analyst can select suitable \mfs according to the PSO \bp, and simply employ them inside the general PSO loss. Along this paper we demonstrated a step-by-step derivation of several such new approaches.

Likewise, herein we investigated the reason for high resemblance between the statistical model inference and physics over virtual surfaces. We showed that during the optimization, a change of model output at any point (the height change of the virtual surface at the point) is equal to the model kernel (a.k.a. NTK, \cite{Jacot18nips}) between this point and the optimized training point. Following from this, the optimization can be viewed as pushes at training points performed via some employed sticks, whose shape is described by the kernel. We analyzed this kernel's properties (e.g. the shape of the pushing sticks) and their impact over the convergence of PSO algorithm. Specifically, we showed that its bandwidth corresponds to the flexibility/expressiveness of the model - with a narrower kernel it is possible to push the surface towards various target forms, making it more elastic.

Further, the bandwidth of the model kernel can be viewed as a hyper-parameter that controls the estimation bias-variance tradeoff.
We empirically investigated both \underfit and \overfit scenarios that can occur in PSO.
In our experiments we showed that the wide bandwidth is correlated with a sub-optimal optimization performance in case of a large training dataset. Moreover, if it is too narrow and in case of a small training dataset, the surface converges to peaks around the training points and also produce a poor target approximation. Thus, the optimal kernel bandwidth depends on the number of available training points, which agrees with existing analysis of KDE methods.

Furthermore, we applied  PSO to learn data log-density, proposing several new PSO instances for this purpose, including PSO \emph{log density estimators} (PSO-LDE).
Additionally, we presented a new NN block-diagonal architecture that allowed us to significantly reduce the bandwidth of the model kernel and to extremely increase an approximation accuracy. In our experiments we showed how the above methods can be used to perform precise pdf inference of multi-modal 20D data, getting a superior accuracy over other state-of-the-art baselines. Importantly, in an infinite dataset setting we also empirically revealed a connection between the point-wise error and gradient norm at the point, which in theory can be used for measuring a model uncertainty.

Along this paper we remarked many possible research directions to further enhance PSO estimation techniques. The current solution is still very new and many of its aspects require additional attention and further study. First of all, it is important to understand the precise relation between an accuracy of PSO estimators and various properties of the model kernel.
This can lead to a better explanation of the PSO convergence properties, and may guide us to better NN architectures and new methods to control the bias-variance tradeoff, as also new techniques for quantifying model uncertainty in a small dataset setting.

Moreover, it is still unclear how to pick the most optimal PSO instance for any considered target function. Particularly in the context of density estimation, currently we learn multiple models for different values of PSO-LDE hyper-parameter $\alpha$ and choose the one with the highest performance metric. Yet, such brute-force procedure is computationally very expensive. It is important to understand the exact connection between $\alpha$ and the produced log-pdf estimation, and also to provide a more intelligent way to choose $\alpha$ based on properties of the given data. Furthermore, since PSO framework allows us to generate an infinite number of various PSO instances to approximate a specific target function,
a more thorough exploration of all PSO instances for log-pdf inference is required.
Herein, the goal is to categorize different instances by their statistic robustness properties and to find more optimal density estimation techniques.

We plan to address all of the above in our future work.

\section{Acknowledgments}

This work was supported in part by the Israel Ministry of Science \& Technology (MOST) and Intel Corporation. We gratefully acknowledge the support of NVIDIA Corporation with the donation of the Titan Xp GPU, which, among other GPUs, was used for this research.

\renewcommand{\theHsection}{A\arabic{section}}

\appendix

\section{Proof of Lemmas \ref{lmm:PSO_second_der} and \ref{lmm:PSO_infor}}
\label{sec:PSO_derivatives_proof}

\subsection{Lemma \ref{lmm:PSO_second_der}}

Consider the setting of Section \ref{sec:Consist}. Further, notate the model Hessian as $\HH_{\theta}(X) \equiv \nabla_{\theta\theta}
f_{\theta}(X)
$. Using the gradient of $L_{PSO}(f_{\theta})$ defined in Eq.~(\ref{eq:GeneralPSOLossFrml_Limit}), the second derivative of $L_{PSO}(f_{\theta})$ w.r.t. $\theta$ is:
\begin{multline}
\nabla_{\theta\theta}
L_{PSO}(f_{\theta})
=
-
\E_{X \sim \probs{\usuff}}
\left[
M^{\usuff}{'}
\left[
X,
f_{\theta}(X)
\right]
\cdot
\II_{\theta}(X, X)
+
M^{\usuff}
\left[
X,
f_{\theta}(X)
\right]
\cdot
\HH_{\theta}(X)
\right]
+\\
+
\E_{X \sim \probs{\dsuff}}
\left[
M^{\dsuff}{'}
\left[
X,
f_{\theta}(X)
\right]
\cdot
\II_{\theta}(X, X)
+
M^{\dsuff}
\left[
X,
f_{\theta}(X)
\right]
\cdot
\HH_{\theta}(X)
\right]
=\\
=
-
\E_{X \sim \probs{\usuff}}
M^{\usuff}{'}
\left[
X,
f_{\theta}(X)
\right]
\cdot
\II_{\theta}(X, X)
-
\E_{X \sim \probs{\usuff}}
M^{\usuff}
\left[
X,
f_{\theta}(X)
\right]
\cdot
\HH_{\theta}(X)
+\\
+
\E_{X \sim \probs{\dsuff}}
M^{\dsuff}{'}
\left[
X,
f_{\theta}(X)
\right]
\cdot
\II_{\theta}(X, X)
+
\E_{X \sim \probs{\dsuff}}
M^{\dsuff}
\left[
X,
f_{\theta}(X)
\right]
\cdot
\HH_{\theta}(X)
.
\label{eq:GeneralPSOLossFrml_Hessss}
\end{multline}
Likewise, at $\theta^*$ we also have:
\begin{multline}
-
\E_{X \sim \probs{\usuff}}
M^{\usuff}
\left[
X,
f_{\theta^*}(X)
\right]
\cdot
\HH_{\theta^*}(X)
+
\E_{X \sim \probs{\dsuff}}
M^{\dsuff}
\left[
X,
f_{\theta^*}(X)
\right]
\cdot
\HH_{\theta^*}(X)
=\\
=
-
\E_{X \sim \probs{\usuff}}
M^{\usuff}
\left[
X,
f^*(X)
\right]
\cdot
\HH_{\theta^*}(X)
+
\E_{X \sim \probs{\dsuff}}
M^{\dsuff}
\left[
X,
f^*(X)
\right]
\cdot
\HH_{\theta^*}(X)
=\\
=
\int 
\left[
-
\probi{\usuff}{X}
\cdot
M^{\usuff}
\left[
X,
f^*(X)
\right]
+
\probi{\dsuff}{X}
\cdot
M^{\dsuff}
\left[
X,
f^*(X)
\right]
\right]
\cdot
\HH_{\theta^*}(X)
dX
=
0
,
\label{eq:GeneralPSOLossFrml_Hessss2}
\end{multline}
where the last row is true because $f^*$ satisfies PSO \bp.

Therefore, we have
\begin{multline}
\hess \equiv 
\nabla_{\theta\theta}
L_{PSO}(f_{\theta^*})
=\\
=
-
\E_{X \sim \probs{\usuff}}
M^{\usuff}{'}
\left[
X,
f_{\theta^*}(X)
\right]
\cdot
\II_{\theta^{*}}(X, X)
+
\E_{X \sim \probs{\dsuff}}
M^{\dsuff}{'}
\left[
X,
f_{\theta^*}(X)
\right]
\cdot
\II_{\theta^{*}}(X, X)
=\\
=
-
\E_{X \sim \probs{\usuff}}
M^{\usuff}{'}
\left[
X,
f^{*}(X)
\right]
\cdot
\II_{\theta^{*}}(X, X)
+
\E_{X \sim \probs{\dsuff}}
M^{\dsuff}{'}
\left[
X,
f^{*}(X)
\right]
\cdot
\II_{\theta^{*}}(X, X)
.
\label{eq:GeneralPSOLossHessian_proof}
\end{multline}
Observe also that only $\II_{\theta^{*}}$ terms depend on the parameter vector $\theta^{*}$.

Additionally, $\hess$ has an another form:
\begin{equation}
\hess
=
\E_{X \sim \probs{\dsuff}}
\frac{M^{\usuff}
	\left[
	X,
	f^{*}(X)
	\right]
	}{T'\left[X, \frac{\probi{\usuff}{X}}{\probi{\dsuff}{X}}\right]}
\cdot
\II_{\theta^{*}}(X, X)
,
\label{eq:GeneralPSOLossHessian_proof2}
\end{equation}
where $T'(X, z) \triangleq \frac{\partial T(X, z)}{\partial z}$ is a first derivative of the considered PSO convergence $T(X, z)$. This can be derived as follows. 

First, since $T$ and $R \equiv \frac{M^{\dsuff}}{M^{\usuff}}$ are inverse functions, derivative of $R$ can be computed via derivative of $T$ as:
\begin{equation}
\frac{\partial R(X, s)}{\partial s}
=
\frac{1}{T'(X, R(X, s))}
.
\label{eq:Deriv_of_ratio}
\end{equation}
Observe that due to PSO \bp $R(X, f^{*}(X)) = \frac{\probi{\usuff}{X}}{\probi{\dsuff}{X}}$ we have
\begin{equation}
\frac{\partial R(X, s)}{\partial s}
\Bigg|_{s = f^{*}(X)}
=
\frac{1}{T'(X, R(X, f^{*}(X)))}
=
\frac{1}{T'(X, \frac{\probi{\usuff}{X}}{\probi{\dsuff}{X}})}
.
\label{eq:Deriv_of_ratio_optim}
\end{equation}

Next, the expression inside integral of Eq.~(\ref{eq:GeneralPSOLossHessian_proof}) is:
\begin{multline}
-
\probi{\usuff}{X}
\cdot
M^{\usuff}{'}
\left[
X,
f^{*}(X)
\right]
+
\probi{\dsuff}{X}
\cdot
M^{\dsuff}{'}
\left[
X,
f^{*}(X)
\right]
=\\
=
\probi{\dsuff}{X}
\cdot
M^{\usuff}{'}
\left[
X,
f^{*}(X)
\right]
\cdot
\left[
- \frac{\probi{\usuff}{X}}{\probi{\dsuff}{X}}
+
\frac{M^{\dsuff}{'}
	\left[
	X,
	f^{*}(X)
	\right]
	}{M^{\usuff}{'}
	\left[
	X,
	f^{*}(X)
	\right]
	}
\right]
=\\
=
\probi{\dsuff}{X}
\cdot
M^{\usuff}{'}
\left[
X,
f^{*}(X)
\right]
\cdot
\left[
- 
\frac{M^{\dsuff}
	\left[
	X,
	f^{*}(X)
	\right]
}{M^{\usuff}
\left[
X,
f^{*}(X)
\right]
}
+
\frac{M^{\dsuff}{'}
	\left[
	X,
	f^{*}(X)
	\right]
}{M^{\usuff}{'}
\left[
X,
f^{*}(X)
\right]
}
\right]
=\\
=
\probi{\dsuff}{X}
\cdot
M^{\usuff}
\left[
X,
f^{*}(X)
\right]
\cdot
\frac{
M^{\dsuff}{'}
	\left[
	X,
	f^{*}(X)
	\right]
\cdot
M^{\usuff}
\left[
X,
f^{*}(X)
\right]
-
}
{M^{\usuff}
	\left[
	X,
	f^{*}(X)
	\right]^2
	}
\\
\frac{-M^{\dsuff}
	\left[
	X,
	f^{*}(X)
	\right]
	\cdot
	M^{\usuff}{'}
	\left[
	X,
	f^{*}(X)
	\right]}{}
=\\
=
\probi{\dsuff}{X}
\cdot
M^{\usuff}
\left[
X,
f^{*}(X)
\right]
\cdot
\frac{\partial \frac{M^{\dsuff}}{M^{\usuff}}(X, s)}{\partial s}
|_{s = f^{*}(X)}
=\\
=
\probi{\dsuff}{X}
\cdot
M^{\usuff}
\left[
X,
f^{*}(X)
\right]
\cdot
\frac{\partial R(X, s)}{\partial s}
|_{s = f^{*}(X)}
=
\frac{\probi{\dsuff}{X}
	\cdot
	M^{\usuff}
	\left[
	X,
	f^{*}(X)
	\right]
	}{T'(X, \frac{\probi{\usuff}{X}}{\probi{\dsuff}{X}})}
.
\label{eq:GeneralPSOLossHessian_proof3}
\end{multline}
Hence:
\begin{equation}
\hess
=
\int
\frac{\probi{\dsuff}{X}
	\cdot
	M^{\usuff}
	\left[
	X,
	f^{*}(X)
	\right]
}{T'(X, \frac{\probi{\usuff}{X}}{\probi{\dsuff}{X}})}
\cdot
\II_{\theta^{*}}(X, X)
dX
,
\label{eq:GeneralPSOLossHessian_proof4}
\end{equation}
from which Eq.~(\ref{eq:GeneralPSOLossHessian_proof2}) follows.

\hfill $\blacksquare$

\subsection{Lemma \ref{lmm:PSO_infor}}

Consider the empirical PSO gradient $\nabla_{\theta}
\hat{L}_{PSO}^{N^{\usuff},N^{\dsuff}}(f_{\theta})$ as defined in Eq.~(\ref{eq:GeneralPSOLossFrml}).
Its uncentered variance is then:
\begin{multline}
\E
\left[
\nabla_{\theta}
\hat{L}_{PSO}^{N^{\usuff},N^{\dsuff}}(f_{\theta})
\cdot
\nabla_{\theta}
\hat{L}_{PSO}^{N^{\usuff},N^{\dsuff}}(f_{\theta})^T
\right]
=\\
=
\frac{1}{(N^{\usuff})^2}
\sum_{i,j = 1}^{N^{\usuff}}
\E
\left[
M^{\usuff}
\left[
X^{\usuff}_{i},
f_{\theta}(X^{\usuff}_{i})
\right]
\cdot
M^{\usuff}
\left[
X^{\usuff}_{j},
f_{\theta}(X^{\usuff}_{j})
\right]
\cdot
\II_{\theta}(X^{\usuff}_{i}, X^{\usuff}_{j})
\right]
+\\
+
\frac{1}{(N^{\dsuff})^2}
\sum_{i,j = 1}^{N^{\dsuff}}
\E
\left[
M^{\dsuff}
\left[
X^{\dsuff}_{i},
f_{\theta}(X^{\dsuff}_{i})
\right]
\cdot
M^{\dsuff}
\left[
X^{\dsuff}_j,
f_{\theta}(X^{\dsuff}_j)
\right]
\cdot
\II_{\theta}(X^{\dsuff}_{i}, X^{\dsuff}_j)
\right]
-\\
-
\frac{1}{N^{\usuff} N^{\dsuff}}
\sum_{i = 1}^{N^{\usuff}}
\sum_{j = 1}^{N^{\dsuff}}
\E
\left[
M^{\usuff}
\left[
X^{\usuff}_{i},
f_{\theta}(X^{\usuff}_{i})
\right]
\cdot
M^{\dsuff}
\left[
X^{\dsuff}_j,
f_{\theta}(X^{\dsuff}_j)
\right]
\cdot
\II_{\theta}(X^{\usuff}_{i}, X^{\dsuff}_j)
\right]
-\\
-
\frac{1}{N^{\usuff} N^{\dsuff}}
\sum_{i = 1}^{N^{\usuff}}
\sum_{j = 1}^{N^{\dsuff}}
\E
\left[
M^{\usuff}
\left[
X^{\usuff}_{i},
f_{\theta}(X^{\usuff}_{i})
\right]
\cdot
M^{\dsuff}
\left[
X^{\dsuff}_j,
f_{\theta}(X^{\dsuff}_j)
\right]
\cdot
\II_{\theta}(X^{\dsuff}_j, X^{\usuff}_{i})
\right]
=\\
=
\frac{1}{(N^{\usuff})^2}
\sum_{i = 1}^{N^{\usuff}}
\E
\left[
M^{\usuff}
\left[
X^{\usuff}_{i},
f_{\theta}(X^{\usuff}_{i})
\right]
\cdot
M^{\usuff}
\left[
X^{\usuff}_{i},
f_{\theta}(X^{\usuff}_{i})
\right]
\cdot
\II_{\theta}(X^{\usuff}_{i}, X^{\usuff}_{i})
\right]
+\\
+
\frac{1}{(N^{\dsuff})^2}
\sum_{i = 1}^{N^{\dsuff}}
\E
\left[
M^{\dsuff}
\left[
X^{\dsuff}_{i},
f_{\theta}(X^{\dsuff}_{i})
\right]
\cdot
M^{\dsuff}
\left[
X^{\dsuff}_{i},
f_{\theta}(X^{\dsuff}_{i})
\right]
\cdot
\II_{\theta}(X^{\dsuff}_{i}, X^{\dsuff}_{i})
\right]
+\\
+
\frac{1}{(N^{\usuff})^2}
\sum_{\substack{i, j = 1 \\ i \neq j}}^{N^{\usuff}}
\E
\left[
M^{\usuff}
\left[
X^{\usuff}_{i},
f_{\theta}(X^{\usuff}_{i})
\right]
\cdot
M^{\usuff}
\left[
X^{\usuff}_{j},
f_{\theta}(X^{\usuff}_{j})
\right]
\cdot
\II_{\theta}(X^{\usuff}_{i}, X^{\usuff}_{j})
\right]
+\\
+
\frac{1}{(N^{\dsuff})^2}
\sum_{\substack{i, j = 1 \\ i \neq j}}^{N^{\dsuff}}
\E
\left[
M^{\dsuff}
\left[
X^{\dsuff}_{i},
f_{\theta}(X^{\dsuff}_{i})
\right]
\cdot
M^{\dsuff}
\left[
X^{\dsuff}_j,
f_{\theta}(X^{\dsuff}_j)
\right]
\cdot
\II_{\theta}(X^{\dsuff}_{i}, X^{\dsuff}_j)
\right]
-\\
-
\frac{1}{N^{\usuff} N^{\dsuff}}
\sum_{i = 1}^{N^{\usuff}}
\sum_{j = 1}^{N^{\dsuff}}
\E
\left[
M^{\usuff}
\left[
X^{\usuff}_{i},
f_{\theta}(X^{\usuff}_{i})
\right]
\cdot
M^{\dsuff}
\left[
X^{\dsuff}_j,
f_{\theta}(X^{\dsuff}_j)
\right]
\cdot
\II_{\theta}(X^{\usuff}_{i}, X^{\dsuff}_j)
\right]
-\\
-
\frac{1}{N^{\usuff} N^{\dsuff}}
\sum_{i = 1}^{N^{\usuff}}
\sum_{j = 1}^{N^{\dsuff}}
\E
\left[
M^{\usuff}
\left[
X^{\usuff}_{i},
f_{\theta}(X^{\usuff}_{i})
\right]
\cdot
M^{\dsuff}
\left[
X^{\dsuff}_j,
f_{\theta}(X^{\dsuff}_j)
\right]
\cdot
\II_{\theta}(X^{\dsuff}_j, X^{\usuff}_{i})
\right]
.
\label{eq:Grad_empir_proof2}
\end{multline}

Denote:
\begin{equation}
\bar{\mu}_{\theta}^{\usuff} \triangleq \E_{X \sim \probs{\usuff}}
M^{\usuff}
\left[
X,
f_{\theta}(X)
\right]
\cdot
\nabla_{\theta} f_{\theta}(X)
,
\quad
\bar{\mu}_{\theta}^{\dsuff} \triangleq \E_{X \sim \probs{\dsuff}}
M^{\dsuff}
\left[
X,
f_{\theta}(X)
\right]
\cdot
\nabla_{\theta} f_{\theta}(X)
.
\label{eq:nottt_proof}
\end{equation}
According to Eq.~(\ref{eq:Grad_empir_proof2}), we have:
\begin{multline}
\E
\left[
\nabla_{\theta}
\hat{L}_{PSO}^{N^{\usuff},N^{\dsuff}}(f_{\theta})
\cdot
\nabla_{\theta}
\hat{L}_{PSO}^{N^{\usuff},N^{\dsuff}}(f_{\theta})^T
\right]
=\\
=
\frac{1}{N^{\usuff}}
\E_{X \sim \probs{\usuff}}
\left[
M^{\usuff}
\left[
X,
f_{\theta}(X)
\right]^2
\cdot
\II_{\theta}(X, X)
\right]
+
\frac{1}{N^{\dsuff}}
\E_{X \sim \probs{\dsuff}}
\left[
M^{\dsuff}
\left[
X,
f_{\theta}(X)
\right]^2
\cdot
\II_{\theta}(X, X)
\right]
+\\
+
\frac{N^{\usuff} - 1}{N^{\usuff}}
\bar{\mu}_{\theta}^{\usuff} \cdot (\bar{\mu}_{\theta}^{\usuff})^T
+
\frac{N^{\dsuff} - 1}{N^{\dsuff}}
\bar{\mu}_{\theta}^{\dsuff} \cdot (\bar{\mu}_{\theta}^{\dsuff})^T
-
\bar{\mu}_{\theta}^{\usuff} \cdot (\bar{\mu}_{\theta}^{\dsuff})^T
-
\bar{\mu}_{\theta}^{\dsuff} \cdot (\bar{\mu}_{\theta}^{\usuff})^T
=\\
=
\frac{1}{N^{\usuff}}
\left[
\E_{X \sim \probs{\usuff}}
\left[
M^{\usuff}
\left[
X,
f_{\theta}(X)
\right]^2
\cdot
\II_{\theta}(X, X)
\right]
-
\bar{\mu}_{\theta}^{\usuff} \cdot (\bar{\mu}_{\theta}^{\usuff})^T
\right]
+\\
+
\frac{1}{N^{\dsuff}}
\left[
\E_{X \sim \probs{\dsuff}}
\left[
M^{\dsuff}
\left[
X,
f_{\theta}(X)
\right]^2
\cdot
\II_{\theta}(X, X)
\right]
-
\bar{\mu}_{\theta}^{\dsuff} \cdot (\bar{\mu}_{\theta}^{\dsuff})^T
\right]
+\\
+
\bar{\mu}_{\theta}^{\usuff} \cdot (\bar{\mu}_{\theta}^{\usuff})^T
+
\bar{\mu}_{\theta}^{\dsuff} \cdot (\bar{\mu}_{\theta}^{\dsuff})^T
-
\bar{\mu}_{\theta}^{\usuff} \cdot (\bar{\mu}_{\theta}^{\dsuff})^T
-
\bar{\mu}_{\theta}^{\dsuff} \cdot (\bar{\mu}_{\theta}^{\usuff})^T
.
\label{eq:Grad_empir_proof3}
\end{multline}

Further, the outer product of $\nabla_{\theta}
\hat{L}_{PSO}^{N^{\usuff},N^{\dsuff}}(f_{\theta})$'s expected value $\E
\left[
\nabla_{\theta}
\hat{L}_{PSO}^{N^{\usuff},N^{\dsuff}}(f_{\theta})
\right]
=
- \bar{\mu}_{\theta}^{\usuff} + \bar{\mu}_{\theta}^{\dsuff}
$ is:
\begin{equation}
\E
\left[
\nabla_{\theta}
\hat{L}_{PSO}^{N^{\usuff},N^{\dsuff}}(f_{\theta})
\right]
\cdot
\E
\left[
\nabla_{\theta}
\hat{L}_{PSO}^{N^{\usuff},N^{\dsuff}}(f_{\theta})
\right]^T
=
\bar{\mu}_{\theta}^{\usuff} \cdot (\bar{\mu}_{\theta}^{\usuff})^T
+
\bar{\mu}_{\theta}^{\dsuff} \cdot (\bar{\mu}_{\theta}^{\dsuff})^T
-
\bar{\mu}_{\theta}^{\usuff} \cdot (\bar{\mu}_{\theta}^{\dsuff})^T
-
\bar{\mu}_{\theta}^{\dsuff} \cdot (\bar{\mu}_{\theta}^{\usuff})^T
,
\label{eq:Grad_empir_exp_value}
\end{equation}
and hence:
\begin{multline}
\variance
\left[
\nabla_{\theta}
\hat{L}_{PSO}^{N^{\usuff},N^{\dsuff}}(f_{\theta})
\right]
=
\frac{1}{N^{\usuff}}
\left[
\E_{X \sim \probs{\usuff}}
\left[
M^{\usuff}
\left[
X,
f_{\theta}(X)
\right]^2
\cdot
\II_{\theta}(X, X)
\right]
-
\bar{\mu}_{\theta}^{\usuff} \cdot (\bar{\mu}_{\theta}^{\usuff})^T
\right]
+\\
+
\frac{1}{N^{\dsuff}}
\left[
\E_{X \sim \probs{\dsuff}}
\left[
M^{\dsuff}
\left[
X,
f_{\theta}(X)
\right]^2
\cdot
\II_{\theta}(X, X)
\right]
-
\bar{\mu}_{\theta}^{\dsuff} \cdot (\bar{\mu}_{\theta}^{\dsuff})^T
\right]
.
\label{eq:Grad_empir_var}
\end{multline}
Next, using relations $N^{\usuff} = \frac{\tau}{\tau + 1} N$ and $N^{\dsuff} = \frac{1}{\tau + 1} N$, we can write the above variance as:
\begin{multline}
\variance
\left[
\nabla_{\theta}
\hat{L}_{PSO}^{N^{\usuff},N^{\dsuff}}(f_{\theta})
\right]
=
\frac{1}{N}
\bigg[
\frac{\tau + 1}{\tau}
\E_{X \sim \probs{\usuff}}
\left[
M^{\usuff}
\left[
X,
f_{\theta}(X)
\right]^2
\cdot
\II_{\theta}(X, X)
\right]
+\\
+
\left[
\tau + 1
\right]
\E_{X \sim \probs{\dsuff}}
\left[
M^{\dsuff}
\left[
X,
f_{\theta}(X)
\right]^2
\cdot
\II_{\theta}(X, X)
\right]
-
\frac{\tau + 1}{\tau}
\bar{\mu}_{\theta}^{\usuff} \cdot (\bar{\mu}_{\theta}^{\usuff})^T
-
\left[
\tau + 1
\right]
\bar{\mu}_{\theta}^{\dsuff} \cdot (\bar{\mu}_{\theta}^{\dsuff})^T
\bigg]
.
\label{eq:Grad_empir_var2}
\end{multline}

Further, due to the identical support assumption $\spp^{\usuff} \equiv \spp^{\dsuff}$ we also have $\bar{\mu}_{\theta^*}^{\usuff} = \bar{\mu}_{\theta^*}^{\dsuff}$:
\begin{multline}
\E
\left[
\nabla_{\theta}
\hat{L}_{PSO}^{N^{\usuff},N^{\dsuff}}(f_{\theta^*})
\right]
=
- \bar{\mu}_{\theta^*}^{\usuff} + \bar{\mu}_{\theta^*}^{\dsuff}
=\\
=
-
\E_{X \sim \probs{\usuff}}
M^{\usuff}
\left[
X,
f^*(X)
\right]
\cdot
\nabla_{\theta} f_{\theta^*}(X)
+
\E_{X \sim \probs{\dsuff}}
M^{\dsuff}
\left[
X,
f^*(X)
\right]
\cdot
\nabla_{\theta} f_{\theta^*}(X)
=\\
=
\int 
\left[
-
\probi{\usuff}{X}
\cdot
M^{\usuff}
\left[
X,
f^*(X)
\right]
+
\probi{\dsuff}{X}
\cdot
M^{\dsuff}
\left[
X,
f^*(X)
\right]
\right]
\cdot
\nabla_{\theta} f_{\theta^*}(X)
dX
=
0
,
\label{eq:Grad_empir_proof4}
\end{multline}
where the last row is true because $f^*$ satisfies PSO \bp. Hence, the following is also true:
\begin{equation}
\bar{\mu}_{\theta^*}^{\usuff} \cdot (\bar{\mu}_{\theta^*}^{\usuff})^T = \bar{\mu}_{\theta^*}^{\dsuff} \cdot (\bar{\mu}_{\theta^*}^{\dsuff})^T = \bar{\mu}_{\theta^*}^{\usuff} \cdot (\bar{\mu}_{\theta^*}^{\dsuff})^T =
\bar{\mu}_{\theta^*}^{\dsuff} \cdot (\bar{\mu}_{\theta^*}^{\usuff})^T
.
\label{eq:Grad_empir_proof5555}
\end{equation}

Therefore, $\variance
\left[
\nabla_{\theta}
\hat{L}_{PSO}^{N^{\usuff},N^{\dsuff}}(f_{\theta})
\right]$ (Eq.~(\ref{eq:Grad_empir_var2})) at $\theta^*$ is $\variance
\left[
\nabla_{\theta}
\hat{L}_{PSO}^{N^{\usuff},N^{\dsuff}}(f_{\theta^*})
\right]
= \frac{1}{N} \JJ$ with:
\begin{multline}
\JJ
=
\frac{\tau + 1}{\tau}
\E_{X \sim \probs{\usuff}}
\left[
M^{\usuff}
\left[
X,
f^*(X)
\right]^2
\cdot
\II_{\theta^*}(X, X)
\right]
+\\
+
(\tau + 1)
\E_{X \sim \probs{\dsuff}}
\left[
M^{\dsuff}
\left[
X,
f^*(X)
\right]^2
\cdot
\II_{\theta^*}(X, X)
\right]
-\\
-
\frac{(\tau + 1)^2}{\tau}
\E_{\substack{X \sim \probs{\usuff}\\ X' \sim \probs{\dsuff}}}
M^{\usuff}
\left[
X,
f^{*}(X)
\right]
\cdot
M^{\dsuff}
\left[
X',
f^{*}(X')
\right]
\cdot
\II_{\theta^{*}}(X, X')
,
\label{eq:Grad_empir_proof7}
\end{multline}
where we applied identity from Eq.~(\ref{eq:Grad_empir_proof5555}).
Observe that only $\II_{\theta^{*}}$ terms depend on the parameter vector $\theta^{*}$. Likewise, the term next to $\frac{(\tau + 1)^2}{\tau}$ is actually $\bar{\mu}_{\theta^*}^{\usuff} \cdot (\bar{\mu}_{\theta^*}^{\dsuff})^T$, and it can be substituted by either of $\{\bar{\mu}_{\theta^*}^{\usuff} \cdot (\bar{\mu}_{\theta^*}^{\usuff})^T; \bar{\mu}_{\theta^*}^{\dsuff} \cdot (\bar{\mu}_{\theta^*}^{\dsuff})^T; \bar{\mu}_{\theta^*}^{\dsuff} \cdot (\bar{\mu}_{\theta^*}^{\usuff})^T  \}$.

\hfill $\blacksquare$

\section{Proof of Theorem \ref{thrm:PSO_asymp_norm}}
\label{sec:AsympNormalProof}

First, we prove the stepping stone lemmas.

\subsection{Lemmata}
\label{sec:AsympNormalGradProof}

\begin{lemma}
	\label{lmm:PSO_grad_conv} 
Denote $N \triangleq N^{\usuff} + N^{\dsuff}$ and $\tau \triangleq \frac{N^{\usuff}}{N^{\dsuff}}$, and assume $\tau$ to be a strictly positive, finite and constant scalar. Then
	$\sqrt{N}
	\cdot
	\nabla_{\theta}
	\hat{L}_{PSO}^{N^{\usuff},N^{\dsuff}}(f_{\theta^*})
	\dconv
	\NN(0, \JJ)
	$, convergence in distribution along with $N \rightarrow \infty$, where $\JJ$ is defined by Lemma \ref{lmm:PSO_infor}.
\end{lemma}

\begin{proof}	

Define:
\begin{equation}
\Sigma_{\theta}^{\usuff} \triangleq 
\variance_{X \sim \probs{\usuff}}
M^{\usuff}
\left[
X,
f_{\theta}(X)
\right]
\cdot
\nabla_{\theta} f_{\theta}(X)
=
\E_{X \sim \probs{\usuff}}
\left[
M^{\usuff}
\left[
X,
f_{\theta}(X)
\right]^2
\cdot
\II_{\theta}(X, X)
\right]
-
\bar{\mu}_{\theta}^{\usuff} \cdot (\bar{\mu}_{\theta}^{\usuff})^T
,
\label{eq:nottt_proof_cov1}
\end{equation}
\begin{equation}
\Sigma_{\theta}^{\dsuff} \triangleq 
\variance_{X \sim \probs{\dsuff}}
M^{\dsuff}
\left[
X,
f_{\theta}(X)
\right]
\cdot
\nabla_{\theta} f_{\theta}(X)
=
\E_{X \sim \probs{\dsuff}}
\left[
M^{\dsuff}
\left[
X,
f_{\theta}(X)
\right]^2
\cdot
\II_{\theta}(X, X)
\right]
-
\bar{\mu}_{\theta}^{\dsuff} \cdot (\bar{\mu}_{\theta}^{\dsuff})^T
,
\label{eq:nottt_proof_cov2}
\end{equation}
where $\bar{\mu}_{\theta}^{\usuff}$ and $\bar{\mu}_{\theta}^{\dsuff}$ are defined in Eq.~(\ref{eq:nottt_proof}).

Consider the average $\frac{1}{N^{\usuff}}
\sum_{i = 1}^{N^{\usuff}}
M^{\usuff}
\left[
X^{\usuff}_{i},
f_{\theta^*}(X^{\usuff}_{i})
\right]
\cdot
\nabla_{\theta} f_{\theta^*}(X^{\usuff}_{i})
$.
It contains i.i.d. random vectors with mean $\bar{\mu}_{\theta^*}^{\usuff}$ and variance $\Sigma_{\theta^*}^{\usuff}$.
Using a multivariate Central Limit Theorem (CLT) on the considered average, we have:
\begin{equation}
\bar{a}
\equiv
\sqrt{N^{\usuff}}
\left[
\frac{1}{N^{\usuff}}
\sum_{i = 1}^{N^{\usuff}}
M^{\usuff}
\left[
X^{\usuff}_{i},
f_{\theta^*}(X^{\usuff}_{i})
\right]
\cdot
\nabla_{\theta} f_{\theta^*}(X^{\usuff}_{i})
-
\bar{\mu}_{\theta^*}^{\usuff}
\right]
\dconv
\NN(0, \Sigma_{\theta^*}^{\usuff})
.
\label{eq:U_sum_conv2}
\end{equation}
In similar manner, we also have:
\begin{equation}
\bar{b}
\equiv
\sqrt{N^{\dsuff}}
\left[
\frac{1}{N^{\dsuff}}
\sum_{i = 1}^{N^{\dsuff}}
M^{\dsuff}
\left[
X^{\dsuff}_{i},
f_{\theta^*}(X^{\dsuff}_{i})
\right]
\cdot
\nabla_{\theta} f_{\theta^*}(X^{\dsuff}_{i})
-
\bar{\mu}_{\theta^*}^{\dsuff}
\right]
\dconv
\NN(0, \Sigma_{\theta^*}^{\dsuff})
.
\label{eq:U_sum_conv3}
\end{equation}

Therefore, the linear combination also has a convergence in distribution:
\begin{equation}
-
\sqrt{\frac{\tau + 1}{\tau}}
\cdot
\bar{a}
+
\sqrt{\tau + 1}
\cdot
\bar{b}
\dconv
\NN(0, \varSigma)
,
\quad
\varSigma \triangleq 
\frac{\tau + 1}{\tau} \cdot \Sigma_{\theta^*}^{\usuff} + (\tau + 1) \cdot \Sigma_{\theta^*}^{\dsuff}
,
\label{eq:U_sum_conv4}
\end{equation}
where $\sqrt{\frac{\tau + 1}{\tau}}$ and $\sqrt{\tau + 1}$ are finite constant scalars, and where the convergence happens along with $N^{\usuff} \rightarrow \infty$ and $N^{\dsuff} \rightarrow \infty$. Note that this implies the convergence in $N \rightarrow \infty$, since the latter leads to $\{ N^{\usuff} \rightarrow \infty, N^{\dsuff} \rightarrow \infty, \min(N^{\usuff}, N^{\dsuff}) \rightarrow \infty \}$ due to $\tau$ being fixed and finite.

Further, using relations $N^{\usuff} = \frac{\tau}{\tau + 1} N$ and $N^{\dsuff} = \frac{1}{\tau + 1} N$, the above linear combination is equal to: 
\begin{multline}
-
\sqrt{\frac{\tau + 1}{\tau}}
\cdot
\bar{a}
+
\sqrt{\tau + 1}
\cdot
\bar{b}
=
\sqrt{N}
\Big[
-
\frac{1}{N^{\usuff}}
\sum_{i = 1}^{N^{\usuff}}
M^{\usuff}
\left[
X^{\usuff}_{i},
f_{\theta^*}(X^{\usuff}_{i})
\right]
\cdot
\nabla_{\theta} f_{\theta^*}(X^{\usuff}_{i})
+
\bar{\mu}_{\theta^*}^{\usuff}
+\\
+
\frac{1}{N^{\dsuff}}
\sum_{i = 1}^{N^{\dsuff}}
M^{\dsuff}
\left[
X^{\dsuff}_{i},
f_{\theta^*}(X^{\dsuff}_{i})
\right]
\cdot
\nabla_{\theta} f_{\theta^*}(X^{\dsuff}_{i})
-
\bar{\mu}_{\theta^*}^{\dsuff}
\Big]
=
\sqrt{N}
\cdot
\nabla_{\theta}
\hat{L}_{PSO}^{N^{\usuff},N^{\dsuff}}(f_{\theta^*})
,
\label{eq:U_sum_conv5}
\end{multline}
where we used $\bar{\mu}_{\theta^*}^{\usuff} = \bar{\mu}_{\theta^*}^{\dsuff}$ from Eq.~(\ref{eq:Grad_empir_proof4}). Thus, we have $\sqrt{N}
\cdot
\nabla_{\theta}
\hat{L}_{PSO}^{N^{\usuff},N^{\dsuff}}(f_{\theta^*}) \dconv
\NN(0, \varSigma)$.

Finally, we have:
\begin{multline}
\varSigma = 
\frac{\tau + 1}{\tau} \cdot
\left[
\E_{X \sim \probs{\usuff}}
\left[
M^{\usuff}
\left[
X,
f_{\theta^*}(X)
\right]^2
\cdot
\II_{\theta^*}(X, X)
\right]
-
\bar{\mu}_{\theta^*}^{\usuff} \cdot (\bar{\mu}_{\theta^*}^{\usuff})^T
\right]
+\\
+ (\tau + 1) \cdot 
\left[
\E_{X \sim \probs{\dsuff}}
\left[
M^{\dsuff}
\left[
X,
f_{\theta^*}(X)
\right]^2
\cdot
\II_{\theta^*}(X, X)
\right]
-
\bar{\mu}_{\theta^*}^{\dsuff} \cdot (\bar{\mu}_{\theta^*}^{\dsuff})^T
\right]
.
\label{eq:U_sum_conv6}
\end{multline}
Using Eq.~(\ref{eq:Grad_empir_proof5555}) we conclude $\JJ \equiv \varSigma$.

\end{proof}

\subsection{Proof of Theorem}

Define $\hat{\theta}_{N^{\usuff},N^{\dsuff}} = \argmin_{\theta \in \Theta} \hat{L}_{PSO}^{N^{\usuff},N^{\dsuff}}(f_{\theta})$ and $\theta^* = \argmin_{\theta \in \Theta} L_{PSO}(f_{\theta})$. Since assumptions of Theorem \ref{thrm:PSO_consist} are also the assumptions of Theorem \ref{thrm:PSO_asymp_norm}, $f_{\theta^*}$ satisfies PSO \bp and that $\hat{\theta}_{N^{\usuff},N^{\dsuff}} \pconv \theta^*$ when $\min(N^{\usuff}, N^{\dsuff}) \rightarrow \infty$.

Further, note that first order conditions (FOCs) are also satisfied $\nabla_{\theta}
\hat{L}_{PSO}^{N^{\usuff},N^{\dsuff}}(f_{\hat{\theta}_{N^{\usuff},N^{\dsuff}}}) = 0$.
Assuming that $\hat{L}_{PSO}^{N^{\usuff},N^{\dsuff}}(f_{\theta})$ is continuously differentiable w.r.t. $\theta$, we can apply mean-value theorem on FOCs:
\begin{equation}
\nabla_{\theta}
\hat{L}_{PSO}^{N^{\usuff},N^{\dsuff}}(f_{\hat{\theta}_{N^{\usuff},N^{\dsuff}}})
=
\nabla_{\theta}
\hat{L}_{PSO}^{N^{\usuff},N^{\dsuff}}(f_{\theta^*})
+
\nabla_{\theta\theta}
\hat{L}_{PSO}^{N^{\usuff},N^{\dsuff}}(f_{\bar{\theta}})
\cdot
\left[
\hat{\theta}_{N^{\usuff},N^{\dsuff}}
-
\theta^*
\right]
 = 0,
\label{eq:asymp_normal_proof1}
\end{equation}
where $\bar{\theta}$ is located on a line segment between $\hat{\theta}_{N^{\usuff},N^{\dsuff}}$ and $\theta^*$.
Given the estimation consistency $\hat{\theta}_{N^{\usuff},N^{\dsuff}} \pconv \theta^*$, the definition of $\bar{\theta}$ implies $\bar{\theta} \pconv \theta^*$.

The identity in Eq.~(\ref{eq:asymp_normal_proof1}) can be rewritten as:
\begin{equation}
\sqrt{N}
\cdot
\left[
\hat{\theta}_{N^{\usuff},N^{\dsuff}}
-
\theta^*
\right]
=
-
\left[
\nabla_{\theta\theta}
\hat{L}_{PSO}^{N^{\usuff},N^{\dsuff}}(f_{\bar{\theta}})
\right]^{-1}
\cdot
\sqrt{N}
\cdot
\nabla_{\theta}
\hat{L}_{PSO}^{N^{\usuff},N^{\dsuff}}(f_{\theta^*})
,
\label{eq:asymp_normal_proof2}
\end{equation}
where we used a notation $N \triangleq N^{\usuff} + N^{\dsuff}$. Using CLT in Lemma \ref{lmm:PSO_grad_conv} we have:
\begin{equation}
\sqrt{N}
\cdot
\nabla_{\theta}
\hat{L}_{PSO}^{N^{\usuff},N^{\dsuff}}(f_{\theta^*})
\dconv
\NN(0, \JJ)
.
\label{eq:asymp_normal_proof3}
\end{equation}

Next, $\nabla_{\theta\theta}
\hat{L}_{PSO}^{N^{\usuff},N^{\dsuff}}(f_{\bar{\theta}})
$ has a form:
\begin{equation}
\nabla_{\theta\theta}
\hat{L}_{PSO}^{N^{\usuff},N^{\dsuff}}(f_{\bar{\theta}})
=
-
\frac{1}{N^{\usuff}}
\sum_{i = 1}^{N^{\usuff}}
\nabla_{\theta\theta}
\widetilde{M}^{\usuff}
\left[
X^{\usuff}_{i},
f_{\bar{\theta}}(X^{\usuff}_{i})
\right]
+
\frac{1}{N^{\dsuff}}
\sum_{i = 1}^{N^{\dsuff}}
\nabla_{\theta\theta}
\widetilde{M}^{\dsuff}
\left[
X^{\dsuff}_{i},
f_{\bar{\theta}}(X^{\dsuff}_{i})
\right]
.
\label{eq:asymp_normal_proof4}
\end{equation}
Using the uniform law of large numbers (LLN) and $\bar{\theta} \pconv \theta^*$, we get:
\begin{equation}
\frac{1}{N^{\usuff}}
\sum_{i = 1}^{N^{\usuff}}
\nabla_{\theta\theta}
\widetilde{M}^{\usuff}
\left[
X^{\usuff}_{i},
f_{\bar{\theta}}(X^{\usuff}_{i})
\right]
\pconv
\E_{X \sim \probs{\usuff}}
\nabla_{\theta\theta}
\widetilde{M}^{\usuff}
\left[
X,
f_{\theta^*}(X)
\right]
,
\label{eq:asymp_normal_proof5}
\end{equation}
\begin{equation}
\frac{1}{N^{\dsuff}}
\sum_{i = 1}^{N^{\dsuff}}
\nabla_{\theta\theta}
\widetilde{M}^{\dsuff}
\left[
X^{\dsuff}_{i},
f_{\bar{\theta}}(X^{\dsuff}_{i})
\right]
\pconv
\E_{X \sim \probs{\dsuff}}
\nabla_{\theta\theta}
\widetilde{M}^{\dsuff}
\left[
X,
f_{\theta^*}(X)
\right]
,
\label{eq:asymp_normal_proof6}
\end{equation}
and hence
\begin{multline}
\nabla_{\theta\theta}
\hat{L}_{PSO}^{N^{\usuff},N^{\dsuff}}(f_{\bar{\theta}})
\pconv
-
\E_{X \sim \probs{\usuff}}
\nabla_{\theta\theta}
\widetilde{M}^{\usuff}
\left[
X,
f_{\theta^*}(X)
\right]
+
\E_{X \sim \probs{\dsuff}}
\nabla_{\theta\theta}
\widetilde{M}^{\dsuff}
\left[
X,
f_{\theta^*}(X)
\right]
=
\\
=
\nabla_{\theta\theta}
\left[
-
\E_{X \sim \probs{\usuff}}
\widetilde{M}^{\usuff}
\left[
X,
f_{\theta^*}(X)
\right]
+
\E_{X \sim \probs{\dsuff}}
\widetilde{M}^{\dsuff}
\left[
X,
f_{\theta^*}(X)
\right]
\right]
=
\nabla_{\theta\theta}
L_{PSO}(f_{\theta^*})
=
\hess
,
\label{eq:asymp_normal_proof7}
\end{multline}
with $\hess$ being defined by Lemma \ref{lmm:PSO_second_der}.
Applying the Continuous Mapping Theorem, we get:
\begin{equation}
\left[
\nabla_{\theta\theta}
\hat{L}_{PSO}^{N^{\usuff},N^{\dsuff}}(f_{\bar{\theta}})
\right]^{-1}
\pconv
\hess^{-1}
.
\label{eq:asymp_normal_proof8}
\end{equation}

Further, we apply Slutzky theorem on Eqs.~(\ref{eq:asymp_normal_proof3})-(\ref{eq:asymp_normal_proof8}) to get:
\begin{equation}
-
\left[
\nabla_{\theta\theta}
\hat{L}_{PSO}^{N^{\usuff},N^{\dsuff}}(f_{\bar{\theta}})
\right]^{-1}
\cdot
\sqrt{N}
\cdot
\nabla_{\theta}
\hat{L}_{PSO}^{N^{\usuff},N^{\dsuff}}(f_{\theta^*})
\dconv
-
\hess^{-1}
\NN(0, \JJ)
=
\NN(0, \hess^{-1} \JJ \hess^{-1})
.
\label{eq:asymp_normal_proof9}
\end{equation}
Note that $\min(N^{\usuff}, N^{\dsuff}) \rightarrow \infty$ is required for the convergence (see Lemma \ref{lmm:PSO_grad_conv}). This limit is identical to $N \rightarrow \infty$ due to assumption that $\tau$ is fixed and constant.

Therefore, we get:
\begin{equation}
\sqrt{N}
\cdot
\left[
\hat{\theta}_{N^{\usuff},N^{\dsuff}}
-
\theta^*
\right]
\dconv
\NN(0, \hess^{-1} \JJ \hess^{-1})
,
\label{eq:asymp_normal_proof10}
\end{equation}
where the convergence in distribution is achieved along with $N \rightarrow \infty$. 
Further, all the regulatory assumptions of the theorem are required for application of CLT, LLN, and satisfaction of FOCs. See theorem 3.1 in \citep{Newey94book} for the more technical exposition.

\hfill $\blacksquare$

\section{Proof of Theorem \ref{thrm:PSO_diff_stats_props}}
\label{sec:PSO_diff_stats_props_proof}

The proof for $df_{\theta}(X)$'s expected value is trivial. Its covariance is derived as following:
\begin{multline}
\E
\left[
df_{\theta}(X)
\cdot
df_{\theta}(X')
\right]
=
\delta^2 \cdot 
\E
\left[
\nabla_{\theta} f_{\theta}(X)^T
\cdot
\nabla_{\theta}
\hat{L}_{PSO}^{N^{\usuff},N^{\dsuff}}(f_{\theta})
\cdot
\nabla_{\theta}
\hat{L}_{PSO}^{N^{\usuff},N^{\dsuff}}(f_{\theta})^T
\cdot
\nabla_{\theta} f_{\theta}(X')
\right]
=\\
=
\delta^2 \cdot 
\nabla_{\theta} f_{\theta}(X)^T
\cdot
\E
\left[
\nabla_{\theta}
\hat{L}_{PSO}^{N^{\usuff},N^{\dsuff}}(f_{\theta})
\cdot
\nabla_{\theta}
\hat{L}_{PSO}^{N^{\usuff},N^{\dsuff}}(f_{\theta})^T
\right]
\cdot
\nabla_{\theta} f_{\theta}(X')
,
\label{eq:PSO_diff_stats_proof_2}
\end{multline}
\begin{equation}
\E
\left[
df_{\theta}(X)
\right]
=
-
\delta \cdot 
\nabla_{\theta} f_{\theta}(X)^T
\cdot
\E
\left[
\nabla_{\theta}
\hat{L}_{PSO}^{N^{\usuff},N^{\dsuff}}(f_{\theta})
\right]
,
\label{eq:PSO_diff_stats_proof_3}
\end{equation}
\begin{multline}
\cov
\left[
df_{\theta}(X),
df_{\theta}(X')
\right]
=
\E
\left[
df_{\theta}(X)
\cdot
df_{\theta}(X')
\right]
-
\E
\left[
df_{\theta}(X)
\right]
\cdot
\E
\left[
df_{\theta}(X')
\right]
=\\
=
\delta^2 \cdot 
\nabla_{\theta} f_{\theta}(X)^T
\cdot
\E
\left[
\nabla_{\theta}
\hat{L}_{PSO}^{N^{\usuff},N^{\dsuff}}(f_{\theta})
\cdot
\nabla_{\theta}
\hat{L}_{PSO}^{N^{\usuff},N^{\dsuff}}(f_{\theta})^T
\right]
\cdot
\nabla_{\theta} f_{\theta}(X')
-\\
-
\delta^2 \cdot 
\nabla_{\theta} f_{\theta}(X)^T
\cdot
\E
\left[
\nabla_{\theta}
\hat{L}_{PSO}^{N^{\usuff},N^{\dsuff}}(f_{\theta})
\right]
\cdot
\E
\left[
\nabla_{\theta}
\hat{L}_{PSO}^{N^{\usuff},N^{\dsuff}}(f_{\theta})
\right]^T
\cdot
\nabla_{\theta} f_{\theta}(X')
=\\
=
\delta^2 \cdot 
\nabla_{\theta} f_{\theta}(X)^T
\cdot
\Bigg[
\E
\left[
\nabla_{\theta}
\hat{L}_{PSO}^{N^{\usuff},N^{\dsuff}}(f_{\theta})
\cdot
\nabla_{\theta}
\hat{L}_{PSO}^{N^{\usuff},N^{\dsuff}}(f_{\theta})^T
\right]
-\\
-
\E
\left[
\nabla_{\theta}
\hat{L}_{PSO}^{N^{\usuff},N^{\dsuff}}(f_{\theta})
\right]
\cdot
\E
\left[
\nabla_{\theta}
\hat{L}_{PSO}^{N^{\usuff},N^{\dsuff}}(f_{\theta})
\right]^T
\Bigg]
\cdot
\nabla_{\theta} f_{\theta}(X')
=\\
=
\delta^2 \cdot 
\nabla_{\theta} f_{\theta}(X)^T
\cdot
\variance
\left[
\nabla_{\theta}
\hat{L}_{PSO}^{N^{\usuff},N^{\dsuff}}(f_{\theta})
\right]
\cdot
\nabla_{\theta} f_{\theta}(X')
,
\label{eq:PSO_diff_stats_proof_1}
\end{multline}
where $\variance
\left[
\nabla_{\theta}
\hat{L}_{PSO}^{N^{\usuff},N^{\dsuff}}(f_{\theta})
\right]
$ was proven to have a form in Eq.~(\ref{eq:Grad_empir_var2}). Observe that it is proportional to $\frac{1}{N}$ where $N = N^{\usuff} + N^{\dsuff}$.

\hfill $\blacksquare$

\section{Proof of Theorem \ref{thrm:PSO_peak_convergence}}
\label{sec:PSO_peak_convergence_proof}

Assume that first order conditions (FOCs) were satisfied, $\nabla_{\theta}
\hat{L}_{PSO}^{N^{\usuff},N^{\dsuff}}(f_{\theta}) = 0$. Then we have:
\begin{equation}
\frac{1}{N^{\usuff}}
\sum_{i = 1}^{N^{\usuff}}
M^{\usuff}
\left[
X^{\usuff}_{i},
f_{\theta}(X^{\usuff}_{i})
\right]
\cdot
\nabla_{\theta} f_{\theta}(X^{\usuff}_{i})
=
\frac{1}{N^{\dsuff}}
\sum_{i = 1}^{N^{\dsuff}}
M^{\dsuff}
\left[
X^{\dsuff}_{i},
f_{\theta}(X^{\dsuff}_{i})
\right]
\cdot
\nabla_{\theta} f_{\theta}(X^{\dsuff}_{i})
.
\label{eq:peak_proof_1}
\end{equation}

Consider a specific $\probs{\usuff}$'s training sample $X \equiv X^{\usuff}_{j}$ with $j \in \{ 1, \ldots, N^{\usuff} \}$. Multiplying the above expression by $\nabla_{\theta} f_{\theta}(X)^T$, we get
\begin{multline}
M^{\usuff}
\left[
X,
f_{\theta}(X)
\right]
\cdot
g_{\theta}(X, X)
+
\sum_{i,i \neq j}^{N^{\usuff}}
M^{\usuff}
\left[
X^{\usuff}_{i},
f_{\theta}(X^{\usuff}_{i})
\right]
\cdot
g_{\theta}(X, X^{\usuff}_{i})
=\\
=
\frac{N^{\usuff}}{N^{\dsuff}}
\sum_{i = 1}^{N^{\dsuff}}
M^{\dsuff}
\left[
X^{\dsuff}_{i},
f_{\theta}(X^{\dsuff}_{i})
\right]
\cdot
g_{\theta}(X, X^{\dsuff}_{i})
.
\label{eq:peak_proof_2}
\end{multline}

Kernel $g_{\theta}$ is non-negative due to boundedness assumed in Eq.~(\ref{eq:RelKernelBounds}).
Since $M^{\usuff}$ is likewise assumed to be non-negative, we obtain an inequality:
\begin{equation}
M^{\usuff}
\left[
X,
f_{\theta}(X)
\right]
\cdot
g_{\theta}(X, X)
\leq
\frac{N^{\usuff}}{N^{\dsuff}}
\sum_{i = 1}^{N^{\dsuff}}
M^{\dsuff}
\left[
X^{\dsuff}_{i},
f_{\theta}(X^{\dsuff}_{i})
\right]
\cdot
g_{\theta}(X, X^{\dsuff}_{i})
.
\label{eq:peak_proof_3}
\end{equation}
Next, we divide by $g_{\theta}(X, X)$:
\begin{multline}
M^{\usuff}
\left[
X,
f_{\theta}(X)
\right]
\leq
\frac{N^{\usuff}}{N^{\dsuff}}
\sum_{i = 1}^{N^{\dsuff}}
M^{\dsuff}
\left[
X^{\dsuff}_{i},
f_{\theta}(X^{\dsuff}_{i})
\right]
\cdot
r_{\theta}(X, X^{\dsuff}_{i})
\leq\\
\leq
\frac{N^{\usuff}}{N^{\dsuff}}
\sum_{i = 1}^{N^{\dsuff}}
M^{\dsuff}
\left[
X^{\dsuff}_{i},
f_{\theta}(X^{\dsuff}_{i})
\right]
\cdot
\exp
\left[
- \frac{d(X, X^{\dsuff}_{i})}{h_{max}}
\right]
\equiv
\alpha
,
\label{eq:peak_proof_4}
\end{multline}
where in the last part we applied the assumed bounds over $r_{\theta}$.

Denote the inverse function of $M^{\usuff} \left[X, s\right]$ by $(M^{\usuff})^{-1} \left[X, z\right]$. Since $M^{\usuff}$ is assumed to be strictly decreasing, then so is its inverse $(M^{\usuff})^{-1}$. Further, apply $(M^{\usuff})^{-1}$ on both sides of Eq.~(\ref{eq:peak_proof_4}):
\begin{equation}
(M^{\usuff})^{-1}
\left[
X,
M^{\usuff}
\left[
X,
f_{\theta}(X)
\right]
\right]
=
f_{\theta}(X)
\geq
(M^{\usuff})^{-1}
\left[
X,
\alpha
\right]
,
\label{eq:peak_proof_5}
\end{equation}
where we reversed the inequality since the applied function is strictly decreasing.

Next, observe that $0 \leq \alpha \leq \infty$ due to the assumed non-negativity of $M^{\dsuff}$. Further, for $h_{max} \rightarrow 0$ we also have $\alpha \rightarrow 0$ - for zero bandwidth $\alpha$ goes also to zero. Moreover, due to its properties $(M^{\usuff})^{-1}\left[
X,
\alpha
\right]
$ is strictly decreasing for $\alpha \in [0, \infty]$. Hence, $(M^{\usuff})^{-1}
\left[
X,
\alpha
\right]
\rightarrow 
\max_{\alpha' \in [0, \infty]}
(M^{\usuff})^{-1}
\left[
X,
\alpha'
\right]
$ along with $\alpha \rightarrow 0$.

Further note that the range of $(M^{\usuff})^{-1}$ is the subset of values within $\RR$ that $f_{\theta}(X)$ can have. That is, $(M^{\usuff})^{-1}$ and $f_{\theta}$ share their range. Assuming that this range is entire $\RR$ or its positive part $\RRpos$, extended to contain $\infty$,
we will have $\max_{\alpha' \in [0, \infty]}
(M^{\usuff})^{-1}
\left[
X,
\alpha'
\right]
=
\infty
$.

To conclude, we have that $f_{\theta}(X)
\geq
(M^{\usuff})^{-1}
\left[
X,
\alpha
\right]$, where for $\alpha \rightarrow 0$ this lower bound behaves as $(M^{\usuff})^{-1}
\left[
X,
\alpha
\right]
\rightarrow 
\infty
$.

\hfill $\blacksquare$

\section{Proof of Theorem \ref{thrm:PSO_diff_diff}}
\label{sec:PSO_diff_diff_Proof}

First, we prove the stepping stone lemmas.

\subsection{Lemmata}

\begin{lemma}
	\label{lmm:Rel_Kernel_Diff} 
	Consider the relative model kernel $r_{\theta}$ defined in Eq.~(\ref{eq:RelKernel}), and assume it to be bounded as in Eq.~(\ref{eq:RelKernelBounds}). Then $| r_{\theta}(X_1, X) - r_{\theta}(X_2, X) | \leq \epsilon\left[ X_1, X_2, X \right]$ with
	$\epsilon\left[ X_1, X_2, X \right]
	\triangleq
	 1 -
	\exp \left[
	-
	\frac{1}{h_{min}}
	d(X_1, X_2)
	\right]
	\cdot
	\exp \left[
	-
	\frac{1}{h_{min}}
	\max
	\left[
	d(X_1, X),
	d(X_2, X)
	\right]
	\right]
	$.
\end{lemma}

\begin{proof}	
	
Consider two scenarios: $r_{\theta}(X_1, X) \geq r_{\theta}(X_2, X)$ and $r_{\theta}(X_1, X) < r_{\theta}(X_2, X)$.
In the first case we have:
\begin{multline}
\left|
r_{\theta}(X_1, X) - r_{\theta}(X_2, X)
\right|
=
r_{\theta}(X_1, X) - r_{\theta}(X_2, X)
\leq
1 -
\exp \left[
-
\frac{1}{h_{min}}
d(X_2, X)
\right]
\leq\\
\leq
1 -
\exp \left[
-
\frac{1}{h_{min}}
d(X_1, X_2)
\right]
\cdot
\exp \left[
-
\frac{1}{h_{min}}
d(X_1, X)
\right]
\triangleq
c_1
,
\label{eq:diff_diff_proofff1}
\end{multline}
where in the second row we used a triangle inequality $d(X_2, X) \leq d(X_1, X_2) + d(X_1, X)$.

Similarly, in the second case $r_{\theta}(X_1, X) < r_{\theta}(X_2, X)$ we will obtain:
\begin{multline}
\left|
r_{\theta}(X_1, X) - r_{\theta}(X_2, X)
\right|
=
r_{\theta}(X_2, X) - r_{\theta}(X_1, X)
\leq\\
\leq
1 -
\exp \left[
-
\frac{1}{h_{min}}
d(X_1, X_2)
\right]
\cdot
\exp \left[
-
\frac{1}{h_{min}}
d(X_2, X)
\right]
\triangleq
c_2
.
\label{eq:diff_diff_proofff2}
\end{multline}

Next, we combine the two cases:
\begin{multline}
\left|
r_{\theta}(X_1, X) - r_{\theta}(X_2, X)
\right|
\leq
\max
(c_1, c_2)
=\\
=
1 -
\exp \left[
-
\frac{1}{h_{min}}
d(X_1, X_2)
\right]
\cdot
\exp \left[
-
\frac{1}{h_{min}}
\max
\left[
d(X_1, X),
d(X_2, X)
\right]
\right]
.
\label{eq:diff_diff_proofff3}
\end{multline}

\end{proof}

\begin{lemma}
	\label{lmm:PSO_Diff_rel_Diff} 
	Consider the relative model kernel $r_{\theta}$ defined in Eq.~(\ref{eq:RelKernel}), and assume it to be bounded as in Eq.~(\ref{eq:RelKernelBounds}). Then $\left|\frac{df_{\theta}(X_1)}{g_{\theta}(X_1, X_1)} - \frac{df_{\theta}(X_2)}{g_{\theta}(X_2, X_2)}\right| \leq \varepsilon_1(X_1, X_2)$ with:
	\begin{multline}
	\varepsilon_1(X_1, X_2)
	=
	\delta \cdot 
	\Bigg[
	\frac{1}{N^{\usuff}}
	\sum_{i = 1}^{N^{\usuff}}
	\abs{M^{\usuff}\left[X^{\usuff}_{i},f_{\theta}(X^{\usuff}_{i})\right]}
	\cdot
	\epsilon\left[ X_1, X_2, X^{\usuff}_{i} \right]
	+\\
	+
	\frac{1}{N^{\dsuff}}
	\sum_{i = 1}^{N^{\dsuff}}
	\abs{M^{\dsuff}\left[X^{\dsuff}_{i},f_{\theta}(X^{\dsuff}_{i})\right]}
	\cdot
	\epsilon\left[ X_1, X_2, X^{\dsuff}_{i} \right]
	\Bigg]
	.
	\label{eq:diff_diff_proofff4}
	\end{multline}
\end{lemma}

\begin{proof}

According to Eq.~(\ref{eq:PSODffrntl}) we have:
\begin{multline}
\left|\frac{df_{\theta}(X_1)}{g_{\theta}(X_1, X_1)} - \frac{df_{\theta}(X_2)}{g_{\theta}(X_2, X_2)}\right|
=
\delta \cdot 
\Bigg|
\frac{1}{N^{\usuff}}
\sum_{i = 1}^{N^{\usuff}}
M^{\usuff}\left[X^{\usuff}_{i},f_{\theta}(X^{\usuff}_{i})\right]
\cdot
\left[
r_{\theta}(X_1, X^{\usuff}_{i}) - r_{\theta}(X_2, X^{\usuff}_{i})
\right]
-\\
-
\frac{1}{N^{\dsuff}}
\sum_{i = 1}^{N^{\dsuff}}
M^{\dsuff}\left[X^{\dsuff}_{i},f_{\theta}(X^{\dsuff}_{i})\right]
\cdot
\left[
r_{\theta}(X_1, X^{\dsuff}_{i}) - r_{\theta}(X_2, X^{\dsuff}_{i})
\right]
\Bigg|
\leq\\
\leq
\delta \cdot 
\frac{1}{N^{\usuff}}
\sum_{i = 1}^{N^{\usuff}}
\abs{M^{\usuff}\left[X^{\usuff}_{i},f_{\theta}(X^{\usuff}_{i})\right]}
\cdot
\left|
r_{\theta}(X_1, X^{\usuff}_{i}) - r_{\theta}(X_2, X^{\usuff}_{i})
\right|
+\\
+
\delta \cdot 
\frac{1}{N^{\dsuff}}
\sum_{i = 1}^{N^{\dsuff}}
\abs{M^{\dsuff}\left[X^{\dsuff}_{i},f_{\theta}(X^{\dsuff}_{i})\right]}
\cdot
\left|
r_{\theta}(X_1, X^{\dsuff}_{i}) - r_{\theta}(X_2, X^{\dsuff}_{i})
\right|
\leq\\
\leq
\delta \cdot 
\frac{1}{N^{\usuff}}
\sum_{i = 1}^{N^{\usuff}}
\abs{M^{\usuff}\left[X^{\usuff}_{i},f_{\theta}(X^{\usuff}_{i})\right]}
\cdot
\epsilon\left[ X_1, X_2, X^{\usuff}_{i} \right]
+\\
+
\delta \cdot 
\frac{1}{N^{\dsuff}}
\sum_{i = 1}^{N^{\dsuff}}
\abs{M^{\dsuff}\left[X^{\dsuff}_{i},f_{\theta}(X^{\dsuff}_{i})\right]}
\cdot
\epsilon\left[ X_1, X_2, X^{\dsuff}_{i} \right]
,
\label{eq:diff_diff_proofff5}
\end{multline}
where in the last part we applied Lemma \ref{lmm:Rel_Kernel_Diff}.

\end{proof}

\begin{lemma}
	\label{lmm:PSO_Diff_rel_magn} 
	Consider the relative model kernel $r_{\theta}$ defined in Eq.~(\ref{eq:RelKernel}), and assume it to be bounded as in Eq.~(\ref{eq:RelKernelBounds}). Then $\left|\frac{df_{\theta}(X)}{g_{\theta}(X, X)} \right| \leq \varepsilon_2$ with:
	\begin{equation}
	\varepsilon_2
	=
	\delta \cdot 
	\Bigg[
	\frac{1}{N^{\usuff}}
	\sum_{i = 1}^{N^{\usuff}}
	\abs{M^{\usuff}\left[X^{\usuff}_{i},f_{\theta}(X^{\usuff}_{i})\right]}
	+
	\frac{1}{N^{\dsuff}}
	\sum_{i = 1}^{N^{\dsuff}}
	\abs{M^{\dsuff}\left[X^{\dsuff}_{i},f_{\theta}(X^{\dsuff}_{i})\right]}
	\Bigg]
	.
	\label{eq:diff_diff_proofff41}
	\end{equation}
\end{lemma}

\begin{proof}	
\begin{multline}
\left|\frac{df_{\theta}(X)}{g_{\theta}(X, X)}\right|
=\\
=
\delta \cdot 
\Bigg|
\frac{1}{N^{\usuff}}
\sum_{i = 1}^{N^{\usuff}}
M^{\usuff}\left[X^{\usuff}_{i},f_{\theta}(X^{\usuff}_{i})\right]
\cdot
r_{\theta}(X, X^{\usuff}_{i})
-
\frac{1}{N^{\dsuff}}
\sum_{i = 1}^{N^{\dsuff}}
M^{\dsuff}\left[X^{\dsuff}_{i},f_{\theta}(X^{\dsuff}_{i})\right]
\cdot
r_{\theta}(X, X^{\dsuff}_{i})
\Bigg|
\leq\\
\leq
\delta \cdot 
\frac{1}{N^{\usuff}}
\sum_{i = 1}^{N^{\usuff}}
\abs{M^{\usuff}\left[X^{\usuff}_{i},f_{\theta}(X^{\usuff}_{i})\right]}
\cdot
r_{\theta}(X, X^{\usuff}_{i})
+
\delta \cdot 
\frac{1}{N^{\dsuff}}
\sum_{i = 1}^{N^{\dsuff}}
\abs{M^{\dsuff}\left[X^{\dsuff}_{i},f_{\theta}(X^{\dsuff}_{i})\right]}
\cdot
r_{\theta}(X, X^{\dsuff}_{i})
\leq\\
\leq
\delta \cdot 
\frac{1}{N^{\usuff}}
\sum_{i = 1}^{N^{\usuff}}
\abs{M^{\usuff}\left[X^{\usuff}_{i},f_{\theta}(X^{\usuff}_{i})\right]}
+
\delta \cdot 
\frac{1}{N^{\dsuff}}
\sum_{i = 1}^{N^{\dsuff}}
\abs{M^{\dsuff}\left[X^{\dsuff}_{i},f_{\theta}(X^{\dsuff}_{i})\right]}
,
\label{eq:diff_diff_proofff42}
\end{multline}
\end{proof}	
where in the last part we used $r_{\theta}(X, X') \leq 1$.

\subsection{Proof of Theorem}

Observe that:
\begin{multline}
\left|\frac{df_{\theta}(X_1)}{g_{\theta}(X_1, X_1)} - \frac{df_{\theta}(X_2)}{g_{\theta}(X_2, X_2)}\right|
=\\
=
\frac{1}{g_{\theta}(X_1, X_1)}
\cdot
\left| df_{\theta}(X_1) - df_{\theta}(X_2) - \frac{g_{\theta}(X_1, X_1) - g_{\theta}(X_2, X_2)}{g_{\theta}(X_2, X_2)} \cdot df_{\theta}(X_2) \right|
.
\label{eq:diff_diff_proofff6}
\end{multline}
Applying Lemma \ref{lmm:PSO_Diff_rel_Diff}, we have:
\begin{equation}
\left| df_{\theta}(X_1) - df_{\theta}(X_2) - \frac{g_{\theta}(X_1, X_1) - g_{\theta}(X_2, X_2)}{g_{\theta}(X_2, X_2)} \cdot df_{\theta}(X_2) \right|
\leq
g_{\theta}(X_1, X_1)
\cdot
\varepsilon_1(X_1, X_2)
.
\label{eq:diff_diff_proofff7}
\end{equation}

Using the reverse triangle inequality $||x| - |y|| \leq |x - y|$ on left part of the above equation, we get:
\begin{equation}
\left| 
\left|
df_{\theta}(X_1) - df_{\theta}(X_2)
\right|
- 
\left|
\frac{g_{\theta}(X_1, X_1) - g_{\theta}(X_2, X_2)}{g_{\theta}(X_2, X_2)} \cdot df_{\theta}(X_2) 
\right|
\right|
\leq
g_{\theta}(X_1, X_1)
\cdot
\varepsilon_1(X_1, X_2)
.
\label{eq:diff_diff_proofff8}
\end{equation}

Next, we check the above inequality under two possible scenarios:

\paragraph{1) $\left|
	df_{\theta}(X_1) - df_{\theta}(X_2)
	\right|
	\geq 
	\left|
	\frac{g_{\theta}(X_1, X_1) - g_{\theta}(X_2, X_2)}{g_{\theta}(X_2, X_2)} \cdot df_{\theta}(X_2) 
	\right|
	$:}
Here we have:
\begin{equation}
\left|
df_{\theta}(X_1) - df_{\theta}(X_2)
\right|
\leq
g_{\theta}(X_1, X_1)
\cdot
\varepsilon_1(X_1, X_2)
+
\left|
\frac{g_{\theta}(X_1, X_1) - g_{\theta}(X_2, X_2)}{g_{\theta}(X_2, X_2)} \cdot df_{\theta}(X_2) 
\right|
.
\label{eq:diff_diff_proofff9}
\end{equation}

\paragraph{2) $\left|
	df_{\theta}(X_1) - df_{\theta}(X_2)
	\right|
	< 
	\left|
	\frac{g_{\theta}(X_1, X_1) - g_{\theta}(X_2, X_2)}{g_{\theta}(X_2, X_2)} \cdot df_{\theta}(X_2) 
	\right|
	$:}
In such case Eq.~(\ref{eq:diff_diff_proofff9}) is trivially satisfied.

Hence, Eq.~(\ref{eq:diff_diff_proofff9}) is satisfied always and therefore:
\begin{multline}
\left|
df_{\theta}(X_1) - df_{\theta}(X_2)
\right|
\leq
g_{\theta}(X_1, X_1)
\cdot
\varepsilon_1(X_1, X_2)
+
\left|
g_{\theta}(X_1, X_1) - g_{\theta}(X_2, X_2)
\right|
\cdot
\left|
\frac{df_{\theta}(X_2)}{g_{\theta}(X_2, X_2)}
\right|
\leq\\
\leq
g_{\theta}(X_1, X_1)
\cdot
\varepsilon_1(X_1, X_2)
+
\left|
g_{\theta}(X_1, X_1) - g_{\theta}(X_2, X_2)
\right|
\cdot
\varepsilon_2
\equiv
\varepsilon_3(X_1, X_2)
,
\label{eq:diff_diff_proofff10}
\end{multline}
where we applied Lemma \ref{lmm:PSO_Diff_rel_magn}.

Further, using definitions of $\varepsilon_1$ and $\varepsilon_2$ we get:
\begin{multline}
\varepsilon_3(X_1, X_2)
=
g_{\theta}(X_1, X_1)
\cdot
	\delta \cdot 
	\Bigg[
	\frac{1}{N^{\usuff}}
	\sum_{i = 1}^{N^{\usuff}}
	\abs{M^{\usuff}\left[X^{\usuff}_{i},f_{\theta}(X^{\usuff}_{i})\right]}
	\cdot
	\epsilon\left[ X_1, X_2, X^{\usuff}_{i} \right]
	+\\
	+
	\frac{1}{N^{\dsuff}}
	\sum_{i = 1}^{N^{\dsuff}}
	\abs{M^{\dsuff}\left[X^{\dsuff}_{i},f_{\theta}(X^{\dsuff}_{i})\right]}
	\cdot
	\epsilon\left[ X_1, X_2, X^{\dsuff}_{i} \right]
	\Bigg]
+\\
+
\left|
g_{\theta}(X_1, X_1) - g_{\theta}(X_2, X_2)
\right|
\cdot
	\delta \cdot 
	\Bigg[
	\frac{1}{N^{\usuff}}
	\sum_{i = 1}^{N^{\usuff}}
	\abs{M^{\usuff}\left[X^{\usuff}_{i},f_{\theta}(X^{\usuff}_{i})\right]}
	+
	\frac{1}{N^{\dsuff}}
	\sum_{i = 1}^{N^{\dsuff}}
	\abs{M^{\dsuff}\left[X^{\dsuff}_{i},f_{\theta}(X^{\dsuff}_{i})\right]}
	\Bigg]
=\\
=
	\delta \cdot 
g_{\theta}(X_1, X_1)
\cdot
	\Bigg[
	\frac{1}{N^{\usuff}}
	\sum_{i = 1}^{N^{\usuff}}
	\abs{M^{\usuff}\left[X^{\usuff}_{i},f_{\theta}(X^{\usuff}_{i})\right]}
\cdot
\left[
\epsilon\left[ X_1, X_2, X^{\usuff}_{i} \right]
+
\frac{\left|
	g_{\theta}(X_1, X_1) - g_{\theta}(X_2, X_2)
	\right|
	}{g_{\theta}(X_1, X_1)}
\right]
	+\\
	+
	\frac{1}{N^{\dsuff}}
	\sum_{i = 1}^{N^{\dsuff}}
	\abs{M^{\dsuff}\left[X^{\dsuff}_{i},f_{\theta}(X^{\dsuff}_{i})\right]}
\cdot
\left[
\epsilon\left[ X_1, X_2, X^{\dsuff}_{i} \right]
+
\frac{\left|
	g_{\theta}(X_1, X_1) - g_{\theta}(X_2, X_2)
	\right|
	}{g_{\theta}(X_1, X_1)}
\right]
	\Bigg]
.
\label{eq:diff_diff_proofff11}
\end{multline}

Define $\nu_{\theta}(X_1, X_2, X) \triangleq 
\epsilon\left[ X_1, X_2, X \right]
+
\frac{\left|
	g_{\theta}(X_1, X_1) - g_{\theta}(X_2, X_2)
	\right|
}{g_{\theta}(X_1, X_1)}
$. Then, the combination of Eq.~(\ref{eq:diff_diff_proofff10}) and Eq.~(\ref{eq:diff_diff_proofff11}) will lead to:
\begin{multline}
\left|
df_{\theta}(X_1) - df_{\theta}(X_2)
\right|
\leq
	\delta \cdot 
g_{\theta}(X_1, X_1)
\cdot
	\Bigg[
	\frac{1}{N^{\usuff}}
	\sum_{i = 1}^{N^{\usuff}}
	\abs{M^{\usuff}\left[X^{\usuff}_{i},f_{\theta}(X^{\usuff}_{i})\right]}
	\cdot
\nu_{\theta}(X_1, X_2, X^{\usuff}_{i})
	+\\
	+
	\frac{1}{N^{\dsuff}}
	\sum_{i = 1}^{N^{\dsuff}}
	\abs{M^{\dsuff}\left[X^{\dsuff}_{i},f_{\theta}(X^{\dsuff}_{i})\right]}
	\cdot
\nu_{\theta}(X_1, X_2, X^{\dsuff}_{i})
	\Bigg]
.
\label{eq:diff_diff_proofff12}
\end{multline}

\hfill $\blacksquare$

\section{Proof of Softmax Cross-Entropy being Instance of PSO}
\label{sec:App9}

Here we will derive the cross-entropy loss combined with a $Softmax$ layer, typically used in the image classification domain, via PSO principles, showing it to be another instance of PSO. For this we define our training dataset as a set of pairs $\{ X_i, Y_i \}_{i = 1}^{N}$ where $X_i \in \RR^{n}$ is a data point of an arbitrary dimension $n$ (e.g. image) and $Y_i$ is its label - a discrete number that takes values from $\{1, \ldots, C\}$ with $C$ being the number of classes. 
Number of samples for each class is denoted by $\{N_1, \ldots, N_C\}$, with $\sum_{j = 1}^{C} N_j = N$. 
For the classification task we assume that each sample pair is i.i.d. sampled from an unknown density $\PP(X, Y) = \PP(X) \cdot \PP(Y | X)$. Our goal is to enforce the output of $Softmax$ layer to converge to the unknown conditional $\PP(Y | X)$. To this end, define a model $f_{\theta}$ that returns $C$ dimensional output $f_{\theta}(X) \in \RR^{C}$, with its $j$-th entry denoted by $f_{\theta,j}(X)$. Further, $Softmax$ transformation $h_{\theta}(X)$ is defined as:
\begin{equation}
h_{\theta,j}(X) = 
\frac{\exp f_{\theta,j}(X)}
{\sum_{k = 1}^{C} \exp f_{\theta,k}(X)}
=
\frac{\exp f_{\theta,j}(X)}
{\norm{\exp f_{\theta}(X)}_1}
,
\label{eq:SoftTransf}
\end{equation}
which yields properties $h_{\theta,j}(X) \geq 0$ and $\sum_{k} h_{\theta,k}(X) = 1$. We aim for $h_{\theta,j}(X)$ to converge to $\PP(Y = j | X)$ - the probability of $X$'s label to be $j$. Each $f_{\theta,j}(X)$ will be considered as an independent surface in PSO framework, which we will push to the equilibrium where
\begin{equation}
\frac{\exp f_{\theta,j}(X)}
{\norm{\exp f_{\theta}(X)}_1}
= \PP(Y = j | X),
\label{eq:Softmax_conv}
\end{equation}
by optimizing the corresponding loss $L_{PSO}^{j}(f_{\theta,j})$, with total minimized loss being defined as $L_{PSO}(f_{\theta}) = \sum_{j = 1}^{C} L_{PSO}^{j}(f_{\theta,j})$. That is, the described below minimization of $L_{PSO}(f_{\theta})$ will consist of solving $C$ PSO problems in parallel.

\paragraph{PSO over $f_{\theta,j}(X)$ via $L_{PSO}^{j}(f_{\theta,j})$:}

Consider a typical PSO estimation, where $\PP(X | Y = j)$ serves as \up density $\probs{\usuff}$ and $\PP(X | Y \neq j)$ - as \down density $\probs{\dsuff}$. Sample batch from $\PP(X | Y = j)$ is obtained by fetching samples with label $Y = j$; data points from $\PP(X | Y \neq j)$ will be the rest of samples. Note also that the identity $\frac{\PP(X | Y = j)}{\PP(X | Y \neq j)} = \frac{\PP(Y \neq j)}{\PP(Y = j)}
\cdot
\frac{\PP(Y = j | X)}{1 - \PP(Y = j | X)}
$ holds, due to below derivation:
\begin{multline}
\frac{\PP(X | Y \neq j)}{\PP(X | Y = j)}
=
\frac{\PP(X, Y \neq j) \cdot \PP(Y = j)}{\PP(X, Y = j) \cdot \PP(Y \neq j)}
=
\frac{\PP(Y = j)}{\PP(Y \neq j)}
\cdot
\frac{\sum_{k \neq j} \PP(X, Y = k)}{\PP(X, Y = j)}
=\\
=
\frac{\PP(Y = j)}{\PP(Y \neq j)}
\cdot
\frac{\left[ \sum_{k = 1}^{C} \PP(X, Y = k) \right] - \PP(X, Y = j)}{\PP(X, Y = j)}
=
\frac{\PP(Y = j)}{\PP(Y \neq j)}
\cdot
\left[
\frac{\PP(X)}{\PP(X, Y = j)}
- 1
\right]
=\\
=
\frac{\PP(Y = j)}{\PP(Y \neq j)}
\cdot
\frac{1 - \PP(Y = j | X)}{\PP(Y = j | X)}
.
\label{eq:BaesClassif}
\end{multline}
Considering PSO \bp, we are looking for a pair of \ms $\{ M^{\usuff}_{j}, M^{\dsuff}_{j} \}$ that for the below system:
\begin{equation}
\frac{M^{\dsuff}_{j}
	\left[
	X,
	f_{\theta}(X)
	\right]}{M^{\usuff}_{j}
	\left[
	X,
	f_{\theta}(X)
	\right]}
=
\frac{\PP(X | Y = j)}{\PP(X | Y \neq j)},
\label{eq:PSO_balance_softmax}
\end{equation}
will produce a solution at Eq.~(\ref{eq:Softmax_conv}). That is, denoting $\frac{M^{\dsuff}_{j}
	\left[
	X,
	s
	\right]}{M^{\usuff}_{j}
	\left[
	X,
	s
	\right]}$ by $R(X, s): \RR^n \times \RR^{C} \rightarrow \RR$ and using the identity in Eq.~(\ref{eq:BaesClassif}), we are looking for the transformation $R$ s.t. the solution $f_{\theta}(X)$ of:
\begin{equation}
R
	\left[
	X,
	f_{\theta}(X)
	\right]
=
\frac{\PP(Y \neq j)}{\PP(Y = j)}
\cdot
\frac{\PP(Y = j | X)}{1 - \PP(Y = j | X)}
,
\label{eq:PSO_balance_R}
\end{equation}
will satisfy Eq.~(\ref{eq:Softmax_conv}). Assuming that $R$ has a form $R
\left[
X,
f_{\theta}(X)
\right]
=
\bar{R}\left[
\frac{\exp f_{\theta,j}(X)}
{\norm{\exp f_{\theta}(X)}_1}
\right]
$, the above is equivalent to find the transformation $\bar{R}(s): \RR \rightarrow \RR$ s.t. the solution $s$ of a system:
\begin{equation}
\bar{R}(s)
=
\frac{\PP(Y \neq j)}{\PP(Y = j)}
\cdot
\frac{\PP(Y = j | X)}{1 - \PP(Y = j | X)}
\label{eq:PSO_balance_RR}
\end{equation}
is $\PP(Y = j | X)$. Thus, it can be easily identified as $\bar{R}(s) = \frac{\PP(Y \neq j)}{\PP(Y = j)}
\cdot
\frac{s}{1 - s}$. From this we conclude that $R
\left[
X,
f_{\theta}(X)
\right] = 
\frac{\PP(Y \neq j)}{\PP(Y = j)}
\cdot
\frac{\exp f_{\theta,j}(X)}
{\norm{\exp f_{\theta}(X)}_1 - \exp f_{\theta,j}(X)}$ and that \ms must satisfy:
\begin{equation}
\frac{M^{\dsuff}_{j}
	\left[
	X,
	f_{\theta}(X)
	\right]}{M^{\usuff}_{j}
	\left[
	X,
	f_{\theta}(X)
	\right]}
=
\frac{\PP(Y \neq j)}{\PP(Y = j)}
\cdot
\frac{\exp f_{\theta,j}(X)}
{\norm{\exp f_{\theta}(X)}_1 - \exp f_{\theta,j}(X)}.
\label{eq:PSO_balance_softmax2}
\end{equation}
Specifically, we will choose them to be:
\begin{equation}
M^{\usuff}_{j}
\left[
X,
f_{\theta}(X)
\right]
=
\PP(Y = j)
\cdot
\frac{\norm{\exp f_{\theta}(X)}_1 - \exp f_{\theta,j}(X)}
{\norm{\exp f_{\theta}(X)}_1}
=
\PP(Y = j)
\cdot
\frac{\sum_{k = 1, k \neq j}^{C} \exp f_{\theta,k}(X)}
{\norm{\exp f_{\theta}(X)}_1}
,
\label{eq:CEMagnFuncsMU11}
\end{equation}
\begin{equation}
M^{\dsuff}_{j}
\left[
X,
f_{\theta}(X)
\right]
=
\PP(Y \neq j)
\cdot
\frac{\exp f_{\theta,j}(X)}
{\norm{\exp f_{\theta}(X)}_1}
\label{eq:CEMagnFuncsMD11}
\end{equation}
where the denominator $\norm{\exp f_{\theta}(X)}_1$ serves as a normalization factor that enforces $\{ M^{\usuff}_{j}, M^{\dsuff}_{j} \}$ to be between 0 and 1. Such normalization is only one from many possible, yet this choice will eventually yield the popular softmax cross-entropy loss.

Using the above setting to define $L_{PSO}^{j}(f_{\theta,j})$, its gradient can be written as 
\begin{multline}
\nabla_{\theta}
L_{PSO}^{j}(f_{\theta,j})
=
-
\E_{X \sim \PP(X | Y = j)}
M^{\usuff}_{j}
\left[
X,
f_{\theta}(X)
\right]
\cdot
\nabla_{\theta} f_{\theta,j}(X)
+
\\
+
\E_{X \sim \PP(X | Y \neq j)}
M^{\dsuff}_{j}
\left[
X,
f_{\theta}(X)
\right]
\cdot
\nabla_{\theta} f_{\theta,j}(X)
.
\label{eq:PSOLoss_grad_mult_class}
\end{multline}
Gradient-based optimization via the above expression will lead to Eq.~(\ref{eq:Softmax_conv}). Also, in practice $\PP(Y = j)$ inside $M^{\usuff}_{j}$ can be approximated as $\frac{N_j}{N}$, and $\PP(Y \neq j)$ inside $M^{\dsuff}_{j}$ - as $\frac{N - N_j}{N}$.

\paragraph{PSO over multiple surfaces via $L_{PSO}(f_{\theta})$:}

 Further, combining all losses together into $L_{PSO}(f_{\theta}) = \sum_{j = 1}^{C} L_{PSO}^{j}(f_{\theta,j})$ will produce $\nabla_{\theta} L_{PSO}(f_{\theta}) = \sum_{j = 1}^{C} \nabla_{\theta} L_{PSO}^{j}(f_{\theta,j})$:
\begin{multline}
\nabla_{\theta}
L_{PSO}(f_{\theta})
=
\sum_{j = 1}^{C}
\bigg[
-
\E_{X \sim \PP(X | Y = j)}
M^{\usuff}_{j}
\left[
X,
f_{\theta}(X)
\right]
\cdot
\nabla_{\theta} f_{\theta,j}(X)
+
\\
+
\E_{X \sim \PP(X | Y \neq j)}
M^{\dsuff}_{j}
\left[
X,
f_{\theta}(X)
\right]
\cdot
\nabla_{\theta} f_{\theta,j}(X)
\bigg]
.
\label{eq:PSOLoss_grad_mult_class_all}
\end{multline}
The above expression is also the gradient of softmax cross-entropy, which can be shown as follows. First, note that the second term can be rewritten as:
\begin{multline}
\E_{X \sim \PP(X | Y \neq j)}
M^{\dsuff}_{j}
\left[
X,
f_{\theta}(X)
\right]
\cdot
\nabla_{\theta} f_{\theta,j}(X)
=\\
=
\int
\PP(X, Y \neq j)
\cdot
\frac{\exp f_{\theta,j}(X)}
{\norm{\exp f_{\theta}(X)}_1}
\cdot
\nabla_{\theta} f_{\theta,j}(X)
dX
=\\
=
\int
\left[
\sum_{k = 1, k \neq j}^{C} \PP(X, Y = k)
\right]
\cdot
\frac{\exp f_{\theta,j}(X)}
{\norm{\exp f_{\theta}(X)}_1}
\cdot
\nabla_{\theta} f_{\theta,j}(X)
dX
=\\
=
\sum_{k = 1, k \neq j}^{C}
\PP(Y = k)
\cdot
\E_{X \sim \PP(X | Y = k)}
\frac{\exp f_{\theta,j}(X)}
{\norm{\exp f_{\theta}(X)}_1}
\cdot
\nabla_{\theta} f_{\theta,j}(X)
.
\label{eq:PSOLoss_grad_mult_class_all2}
\end{multline}
Then, $\nabla_{\theta}
L_{PSO}(f_{\theta})
$ is equal to:
\begin{multline}
\nabla_{\theta}
L_{PSO}(f_{\theta})
=
\sum_{j = 1}^{C}
\bigg[
-
\PP(Y = j)
\cdot
\E_{X \sim \PP(X | Y = j)}
\frac{\sum_{k = 1, k \neq j}^{C} \exp f_{\theta,k}(X)}
{\norm{\exp f_{\theta}(X)}_1}
\cdot
\nabla_{\theta} f_{\theta,j}(X)
+
\\
+
\sum_{k = 1, k \neq j}^{C}
\PP(Y = k)
\cdot
\E_{X \sim \PP(X | Y = k)}
\frac{\exp f_{\theta,j}(X)}
{\norm{\exp f_{\theta}(X)}_1}
\cdot
\nabla_{\theta} f_{\theta,j}(X)
\bigg]
=\\
=
-
\sum_{j = 1}^{C}
\PP(Y = j)
\cdot
\E_{X \sim \PP(X | Y = j)}
\frac{\sum_{k = 1, k \neq j}^{C} \exp f_{\theta,k}(X)}
{\norm{\exp f_{\theta}(X)}_1}
\cdot
\nabla_{\theta} f_{\theta,j}(X)
+
\\
+
\sum_{j = 1}^{C}
\PP(Y = j)
\cdot
\E_{X \sim \PP(X | Y = j)}
\sum_{k = 1, k \neq j}^{C}
\frac{\exp f_{\theta,k}(X)}
{\norm{\exp f_{\theta}(X)}_1}
\cdot
\nabla_{\theta} f_{\theta,k}(X)
,
\label{eq:PSOLoss_grad_mult_class_all3333}
\end{multline}
where the second equality can be verified by examining coefficients of each $\nabla_{\theta} f_{\theta,j}(X)$ before and after "=". Further:
\begin{multline}
\nabla_{\theta}
L_{PSO}(f_{\theta})
=
-
\sum_{j = 1}^{C}
\PP(Y = j)
\cdot
\E_{X \sim \PP(X | Y = j)}
\bigg[
\frac{\norm{\exp f_{\theta}(X)}_1 - \exp f_{\theta,j}(X)}
{\norm{\exp f_{\theta}(X)}_1}
\cdot
\nabla_{\theta} f_{\theta,j}(X)
-\\
-
\sum_{k = 1, k \neq j}^{C}
\frac{\exp f_{\theta,k}(X)}
{\norm{\exp f_{\theta}(X)}_1}
\cdot
\nabla_{\theta} f_{\theta,k}(X)
\bigg]
,
\label{eq:PSOLoss_grad_mult_class_all4444}
\end{multline}
with the expression in brackets being derivative of $\log \frac{\exp f_{\theta,j}(X)}
{\norm{\exp f_{\theta}(X)}_1}
$ w.r.t. $\theta$.

Concluding from above, $L_{PSO}(f_{\theta})$ with the above gradient can be written as:
\begin{multline}
L_{PSO}(f_{\theta})
=
-
\sum_{j = 1}^{C}
\PP(Y = j)
\cdot
\E_{X \sim \PP(X | Y = j)}
\log \frac{\exp f_{\theta,j}(X)}
{\norm{\exp f_{\theta}(X)}_1}
=\\
=
-
\int
\sum_{j = 1}^{C}
\left[
\PP(X, Y = j)
\cdot
\log \frac{\exp f_{\theta,j}(X)}
{\norm{\exp f_{\theta}(X)}_1}
\right]
dX
=\\
=
-
\E_{X, Y \sim \PP(X, Y)}
\log \frac{\exp f_{\theta,Y}(X)}
{\norm{\exp f_{\theta}(X)}_1}
.
\label{eq:PSOLoss_grad_mult_class_all5555}
\end{multline}
with its empirical version being:
\begin{equation}
L_{PSO}(f_{\theta})
\approx
-
\frac{1}{N}
\sum_{i = 1}^{N}
\log \frac{\exp f_{\theta,Y_i}(X_i)}
{\norm{\exp f_{\theta}(X_i)}_1}
.
\label{eq:CESoftmaxLoss_PSO_emp}
\end{equation}

The above loss is known in Machine Learning community as softmax cross-entropy loss.
Therefore, we can conclude that PSO over multiple surfaces $\{ f_{\theta,j}(X) \}_{j = 1}^{C}$ with \ms in Eqs.~(\ref{eq:CEMagnFuncsMU11})-(\ref{eq:CEMagnFuncsMD11}) corresponds to cross-entropy when $\PP(X | Y = j)$ and $\PP(X | Y \neq j)$ serve as \up and \down densities respectively.
Yet, according to the PSO principles the \ms in Eqs.~(\ref{eq:CEMagnFuncsMU11})-(\ref{eq:CEMagnFuncsMD11}) are not the only choice for such convergence. In fact, we can change the norm within the denominator of $M^{\usuff}(\cdot)$ and $M^{\dsuff}(\cdot)$ to any $L$-p norm, since the denominator term is eventually canceled out and since its actual role is to bound outputs of \mfs. Similarly to what we observed in our experiments about the PSO-LDE (see Sections \ref{sec:DeepLogPDF} and \ref{sec:Exper}), different norms (the $\alpha$ value in context of PSO-LDE) can have smoother dynamics and produce a smaller approximation error.

\hfill $\blacksquare$

\section{Differential approximation}
\label{sec:DiffApr}

\begin{figure}[tb]
	\centering
	
	\newcommand{\width}[0] {0.4}
	\newcommand{\height}[0] {0.15}
	
	\begin{tabular}{cccc}
		
		\subfloat[\label{fig:DifferntlRes-a}]{\includegraphics[height=\height\textheight,width=\width\textwidth]{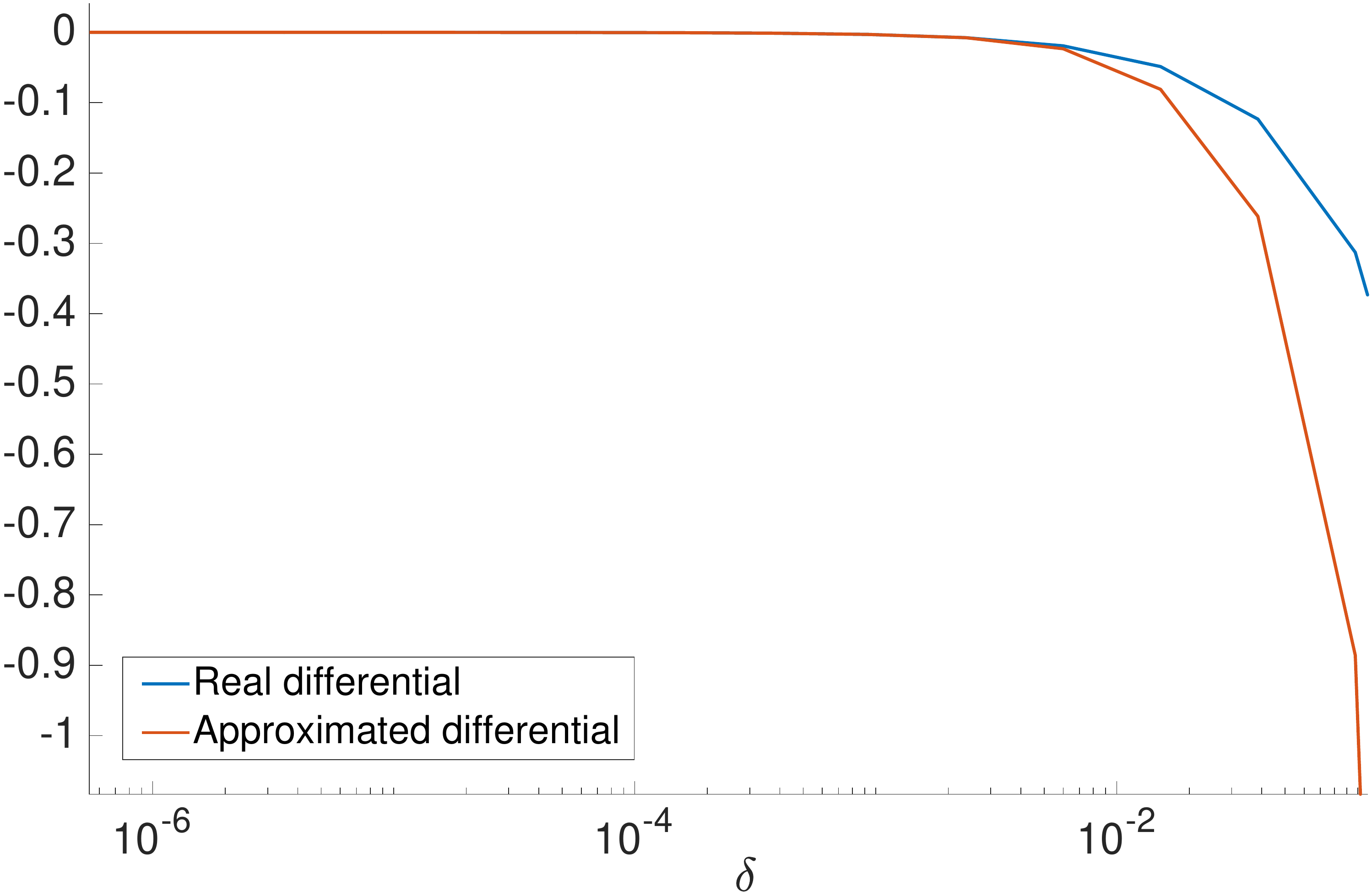}}
		&
		
		\subfloat[\label{fig:DifferntlRes-b}]{\includegraphics[height=\height\textheight,width=\width\textwidth]{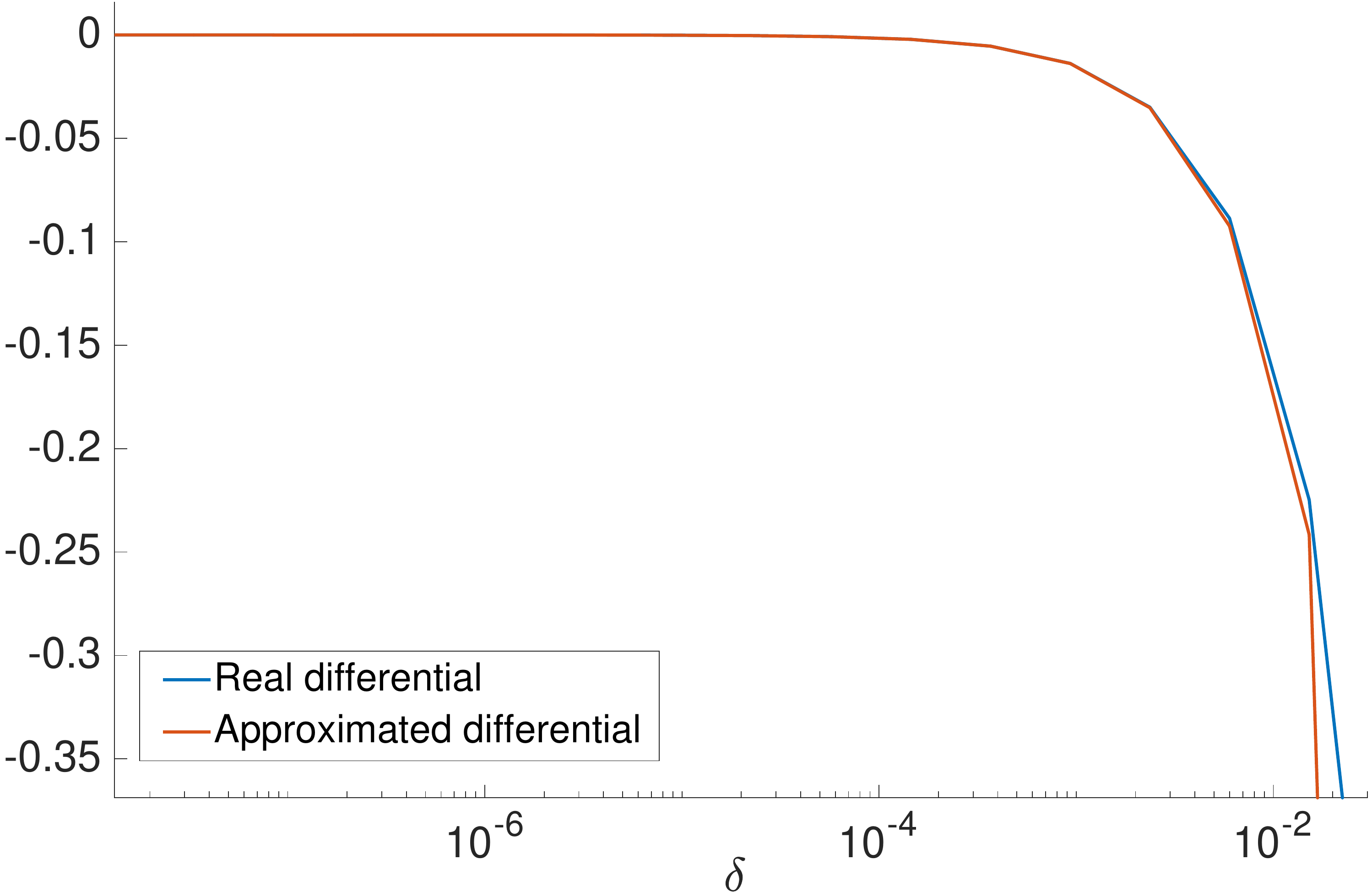}}
		\\
		\subfloat[\label{fig:DifferntlRes-c}]{\includegraphics[height=\height\textheight,width=\width\textwidth]{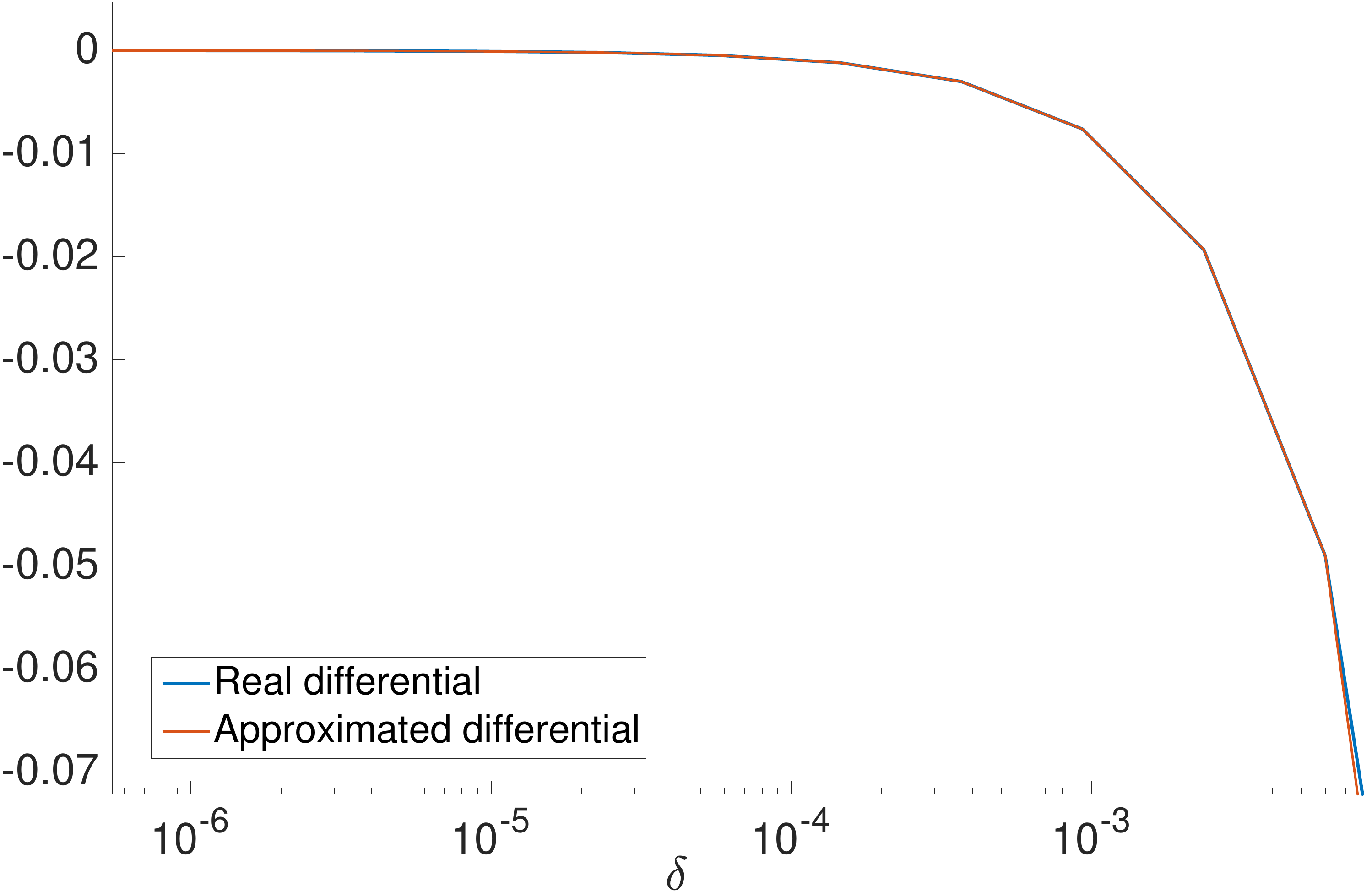}}
		&
		
		\subfloat[\label{fig:DifferntlRes-d}]{\includegraphics[height=\height\textheight,width=\width\textwidth]{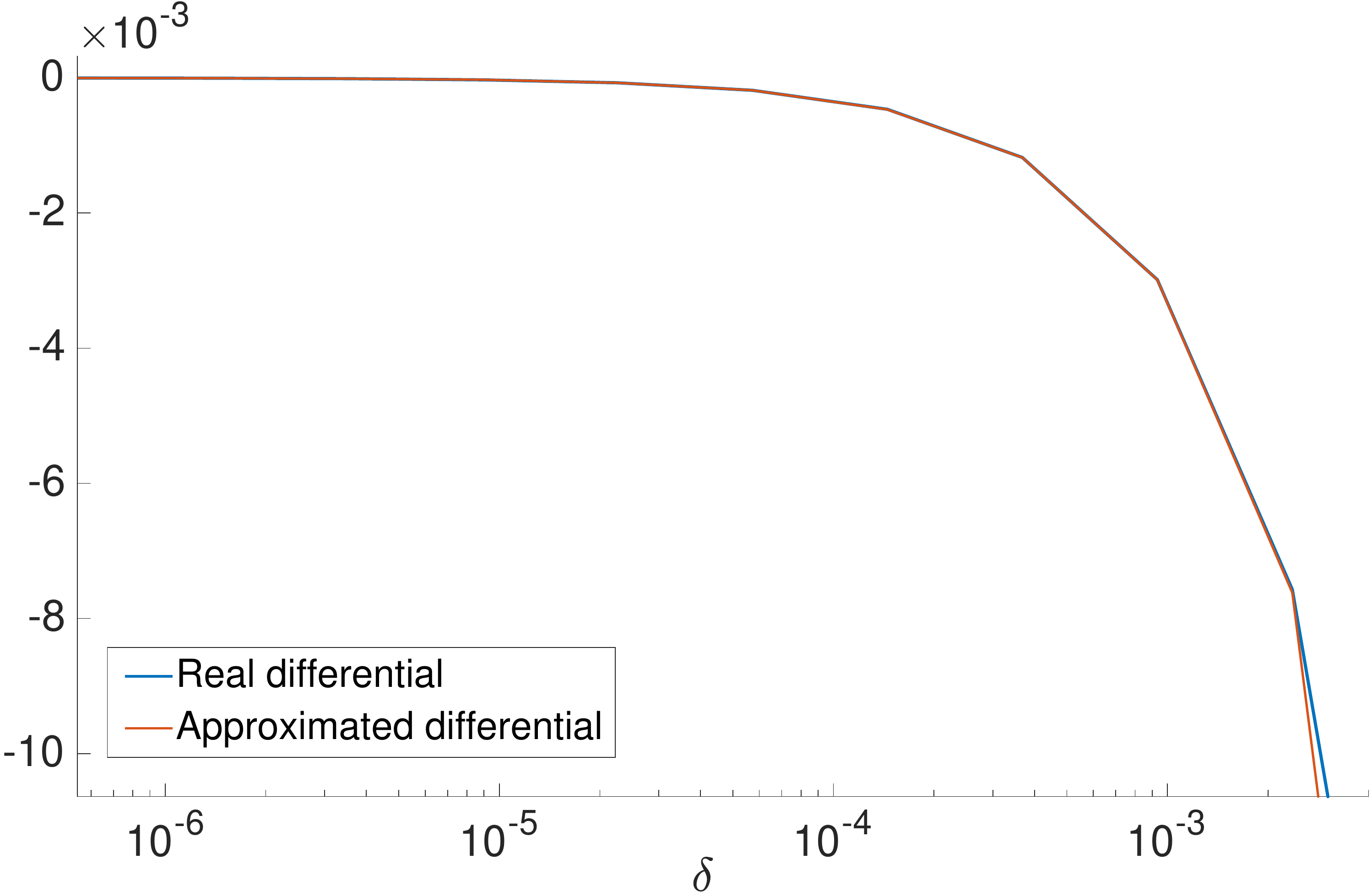}}
		\\
		\subfloat[\label{fig:DifferntlRes-e}]{\includegraphics[height=\height\textheight,width=\width\textwidth]{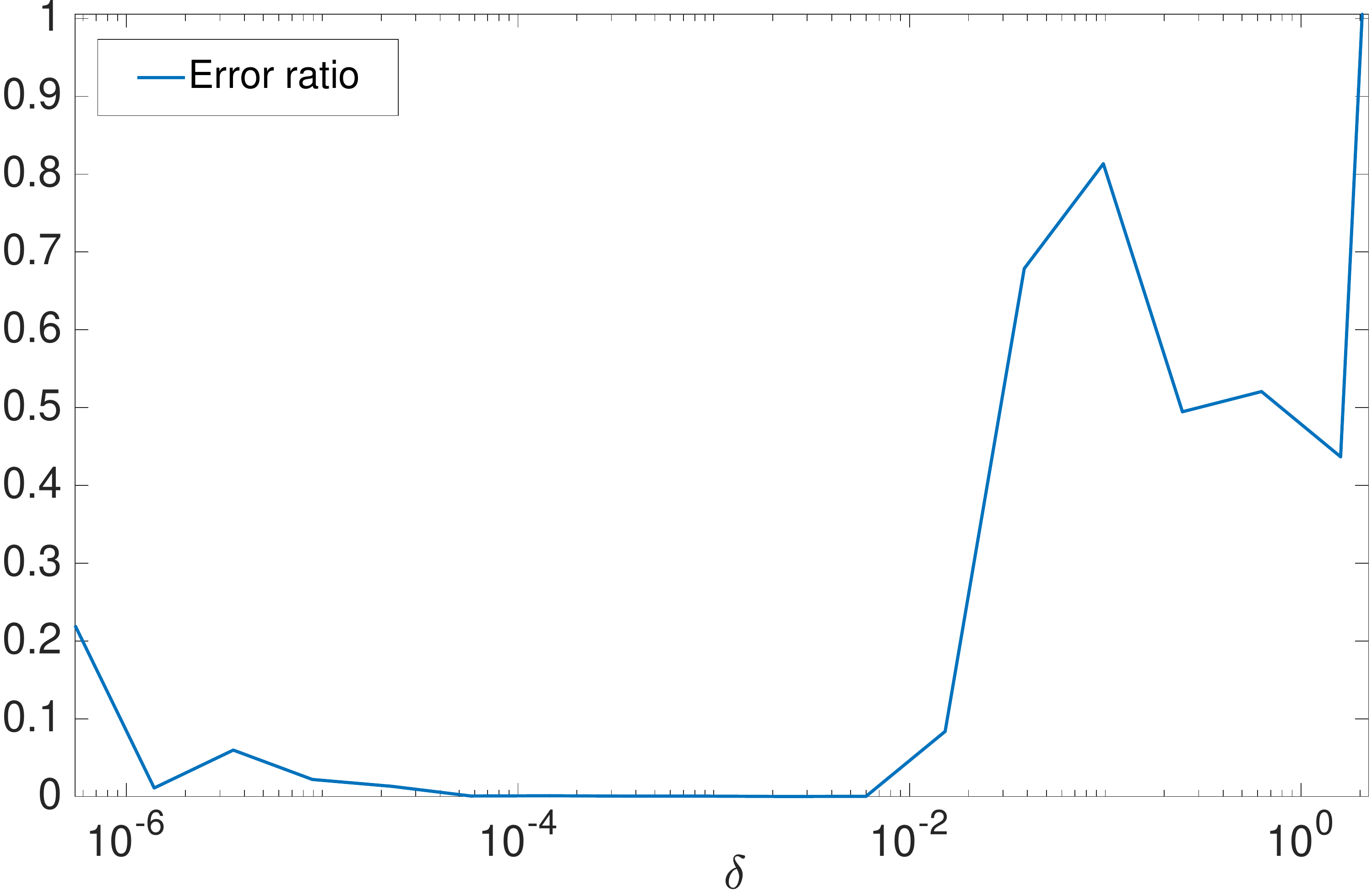}}
		&
		
		\subfloat[\label{fig:DifferntlRes-f}]{\includegraphics[height=\height\textheight,width=\width\textwidth]{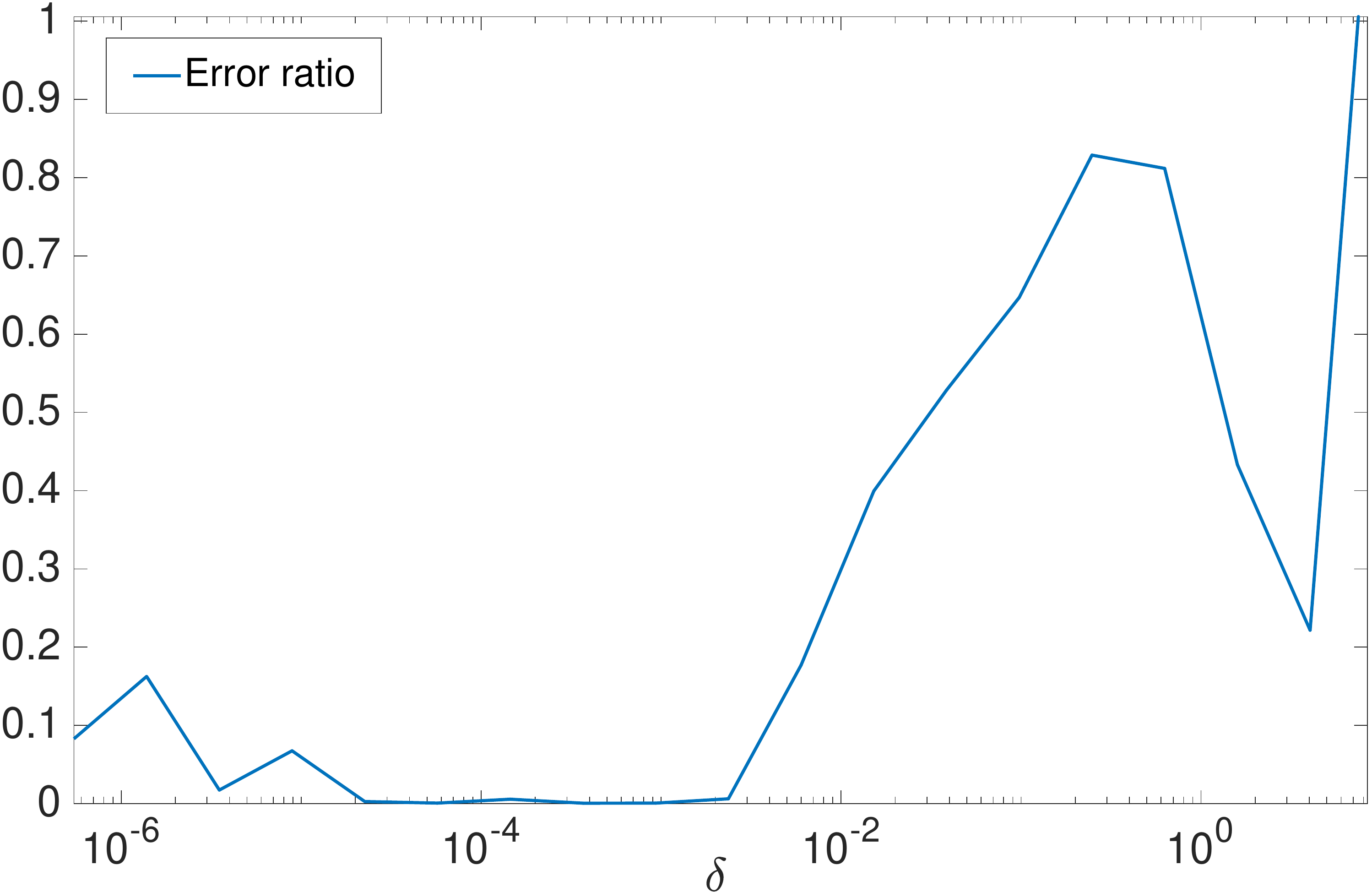}}
		
	\end{tabular}
	
	\protect
	\caption[Real and approximated \dfs at specified points.]{Real and approximated \dfs for the training point $X_{train}$ (a)-(c)-(e) and the testing point $X_{test}$ (b)-(d)-(f). (a)-(b) Real \df (blue line) vs approximated \df (red line), as a function of the learning rate $\delta$;
		(c)-(d) Zoom-in of (a)-(b);
		(e)-(f) Error ratio, defined in Eq.~(\ref{eq:RatioDef}).
	}
	\label{fig:DifferntlRes}
\end{figure}

In this section we will empirically justify our approximation in Eq.~(\ref{eq:PSODffrntl}), where we assumed that the surface \df, caused by GD update of weights $\theta$, can be approximated via its first-order Taylor expansion.

For this purpose we performed a single iteration of GD optimization and measured the real and the estimated \dfs at train and test points as following. First, points $D = \{ X_i \}_{i=1}^{2000}$ were sampled from $\probs{\usuff}$ density, which is \emph{Columns} distribution from Section \ref{sec:ColumnsEst}, where $X_i \in \RR^{n}$ with $n = 20$. Further, we performed a single GD iteration of the following loss:
\begin{equation}
L(\theta, D) = 
-
\frac{1}{1000}
\sum_{i = 1}^{1000}
f_{\theta}(X_i)
+
\frac{1}{1000}
\sum_{i = 1001}^{2000}
f_{\theta}(X_i),
\label{eq:ExprLoss}
\end{equation}
where $f_{\theta}(X)$ is a FC network depicted in Figure \ref{fig:NNLayers-a}, with overall 4 layers of size 1024 each. Next,
we measured the surface height $f_{\theta}(X)$ at two points $X_{train}$ and $X_{test}$ before and after GD update, where $X_{train} \in D$ and $X_{test} \notin D$. We performed this procedure for a range of learning rate values and thus obtained the real \df $df_{\theta}(X)$ at $X_{train}$ and $X_{test}$ as a function of $\delta$ (see Figure \ref{fig:DifferntlRes}).
Likewise, we calculated the \emph{approximated} \df $\bar{df}(X)$ at $X_{train}$ and $X_{test}$ using first-order Taylor expansion as:
\begin{multline}
\bar{df}(X) = 
\frac{\delta}{1000}
\cdot
\nabla_{\theta} 
f_{\theta}(X)^T
\cdot
\Big[
\sum_{i = 1}^{1000}
\nabla_{\theta} 
f_{\theta}(X_i)
-
\sum_{i = 1001}^{2000}
\nabla_{\theta} 
f_{\theta}(X_i)
\Big]
=\\
=
\frac{\delta}{1000}
\cdot
\Big[
\sum_{i = 1}^{1000}
g_{\theta}(X, X_i)
-
\sum_{i = 1001}^{2000}
g_{\theta}(X, X_i)
\Big],
\label{eq:ExprLossDiff}
\end{multline}
where $\theta$ is taken at time before GD update.

In Figures \ref{fig:DifferntlRes-a} and \ref{fig:DifferntlRes-b} we can see the calculated \dfs for both $X_{train}$ and $X_{test}$, respectively. In both figures the real \df (blue line) and the estimated \df (red line) become very close to each other for $\delta < 0.01$. Note that for the most part of a typical optimization process $\delta$ satisfies this criteria.

Further, in Figures \ref{fig:DifferntlRes-e} (for $X_{train}$) and \ref{fig:DifferntlRes-f} (for $X_{test}$) we can see the ratio:
\begin{equation}
ratio = 
\left|df_{\theta}(X) - \bar{df}(X)\right|
/
\left|df_{\theta}(X)\right|
,
\label{eq:RatioDef}
\end{equation}
which expresses an error $\left|df_{\theta}(X) - \bar{df}(X)\right|$ as the percentage from the real \df. As can be seen, for both $X_{train}$ and $X_{test}$ the error ratio is very low for $\delta < 0.01$ (under $10 \%$ for most part). Additionally, the error ratio slightly increases for a very small $\delta$ (around $10^{-6}$). We speculate this to be a precision artifact, since the calculation of an approximated \df in Eq.~(\ref{eq:ExprLossDiff}) was done in a single-precision floating-point format (float32) and involved the multiplication by a very small number $\delta$.

\begin{figure}[tb]
	\centering
	
	\newcommand{\width}[0] {0.4}
	\newcommand{\height}[0] {0.15}
	
	\begin{tabular}{cccc}
		
		\subfloat[\label{fig:DifferntlRes2-a}]{\includegraphics[height=\height\textheight,width=\width\textwidth]{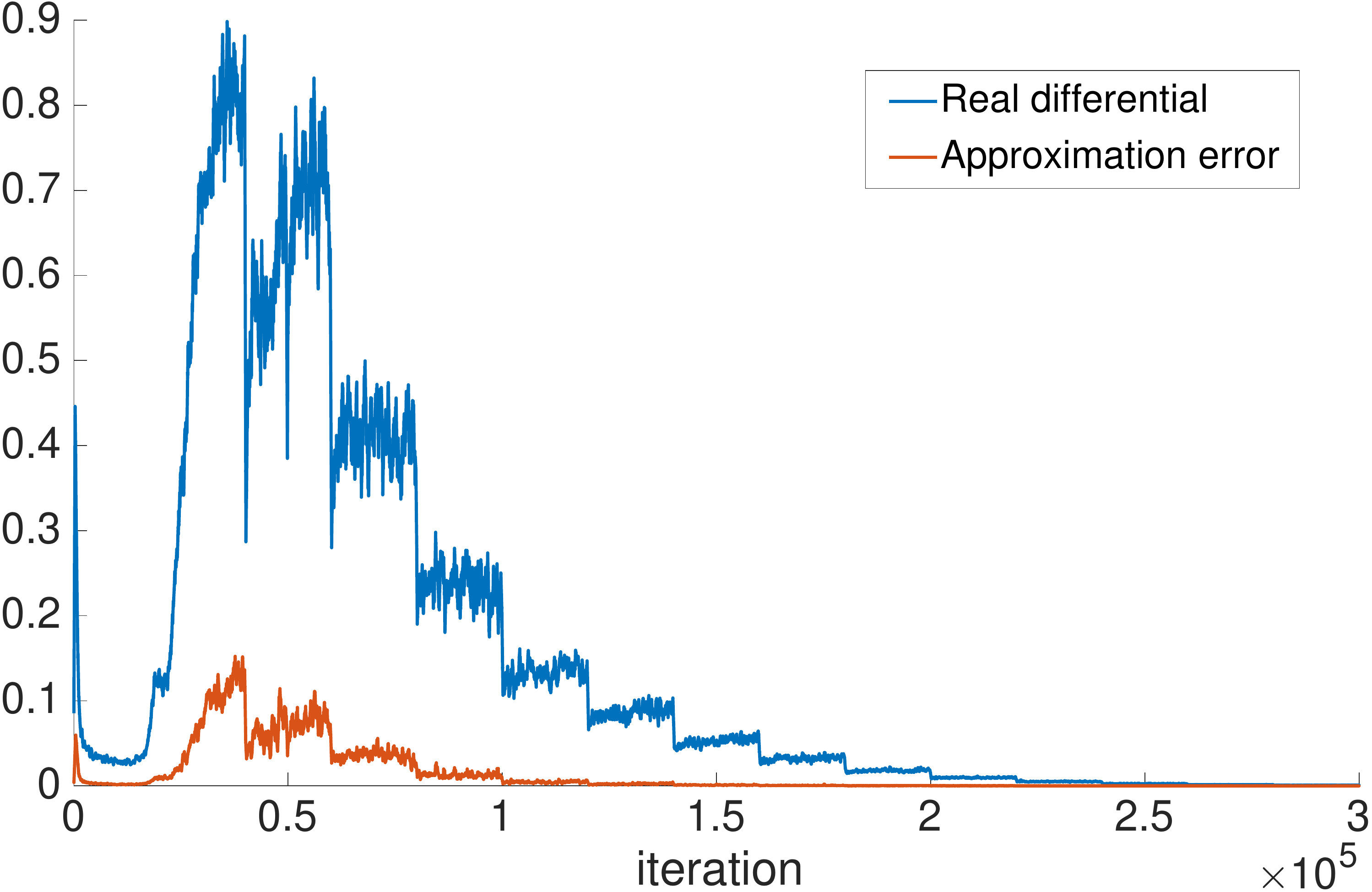}}
		&
		
		\subfloat[\label{fig:DifferntlRes2-b}]{\includegraphics[height=\height\textheight,width=\width\textwidth]{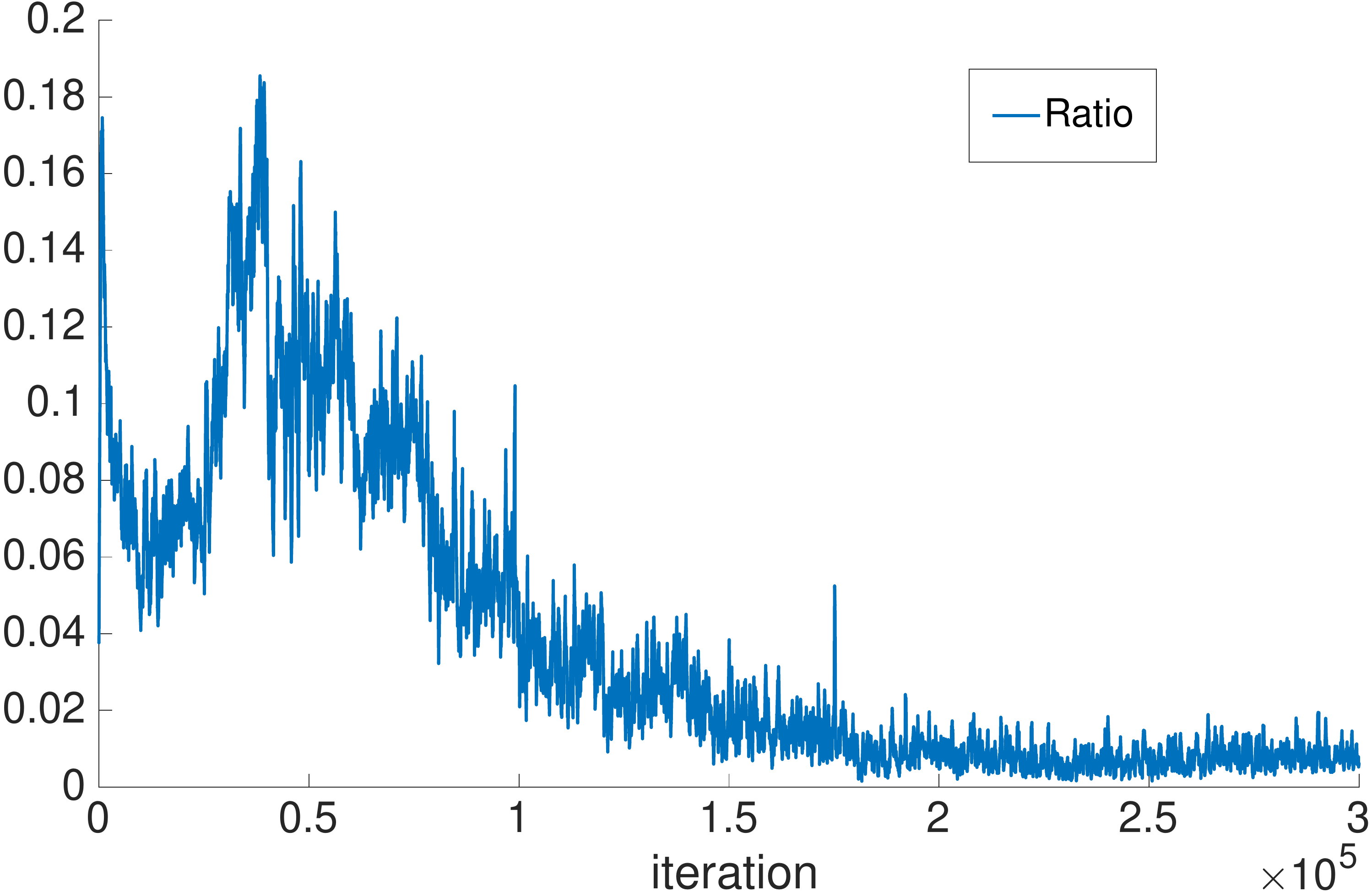}}
		
		\\
		
		\subfloat[\label{fig:DifferntlRes2-c}]{\includegraphics[height=\height\textheight,width=\width\textwidth]{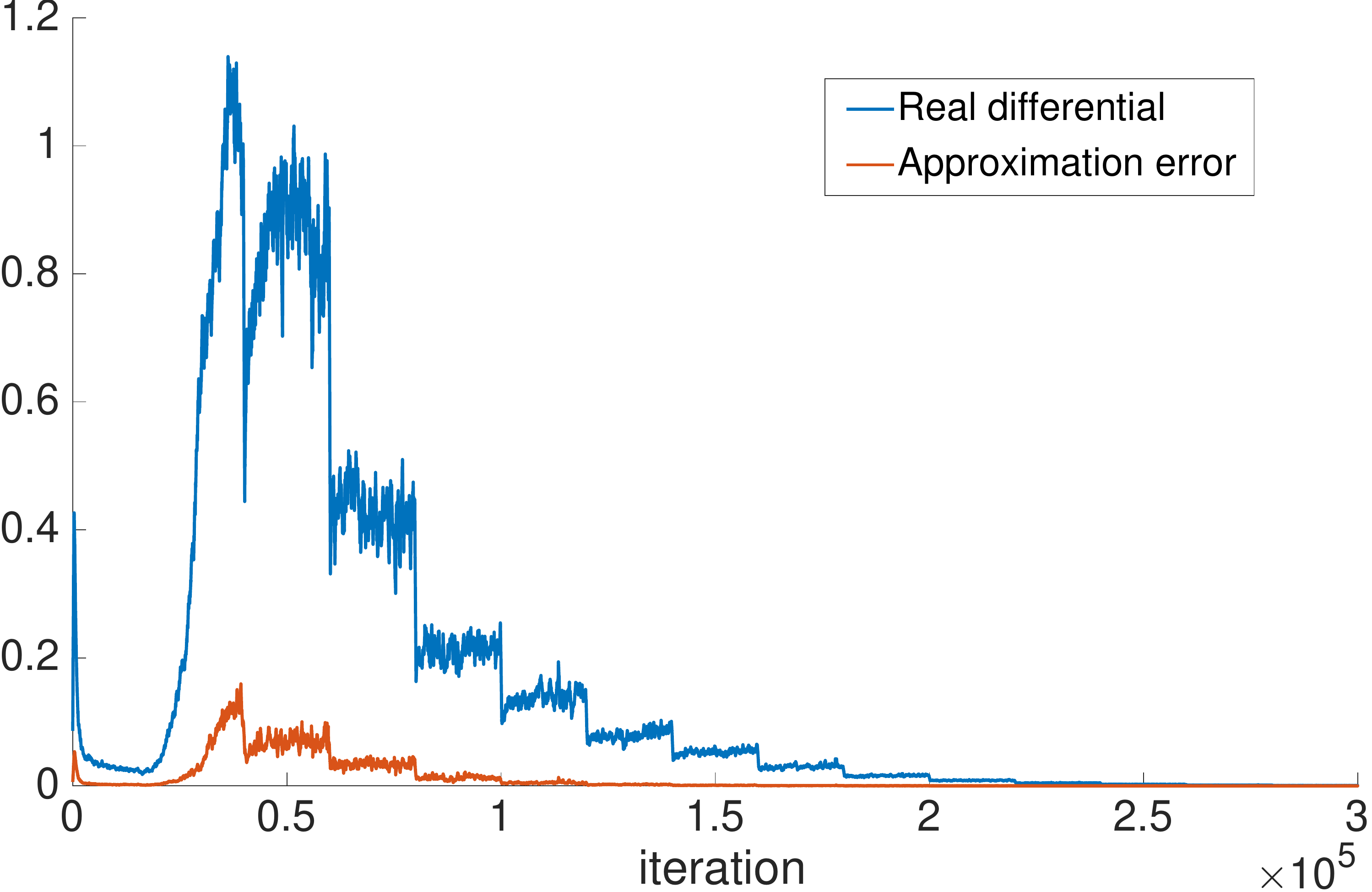}}
		&
		
		\subfloat[\label{fig:DifferntlRes2-d}]{\includegraphics[height=\height\textheight,width=\width\textwidth]{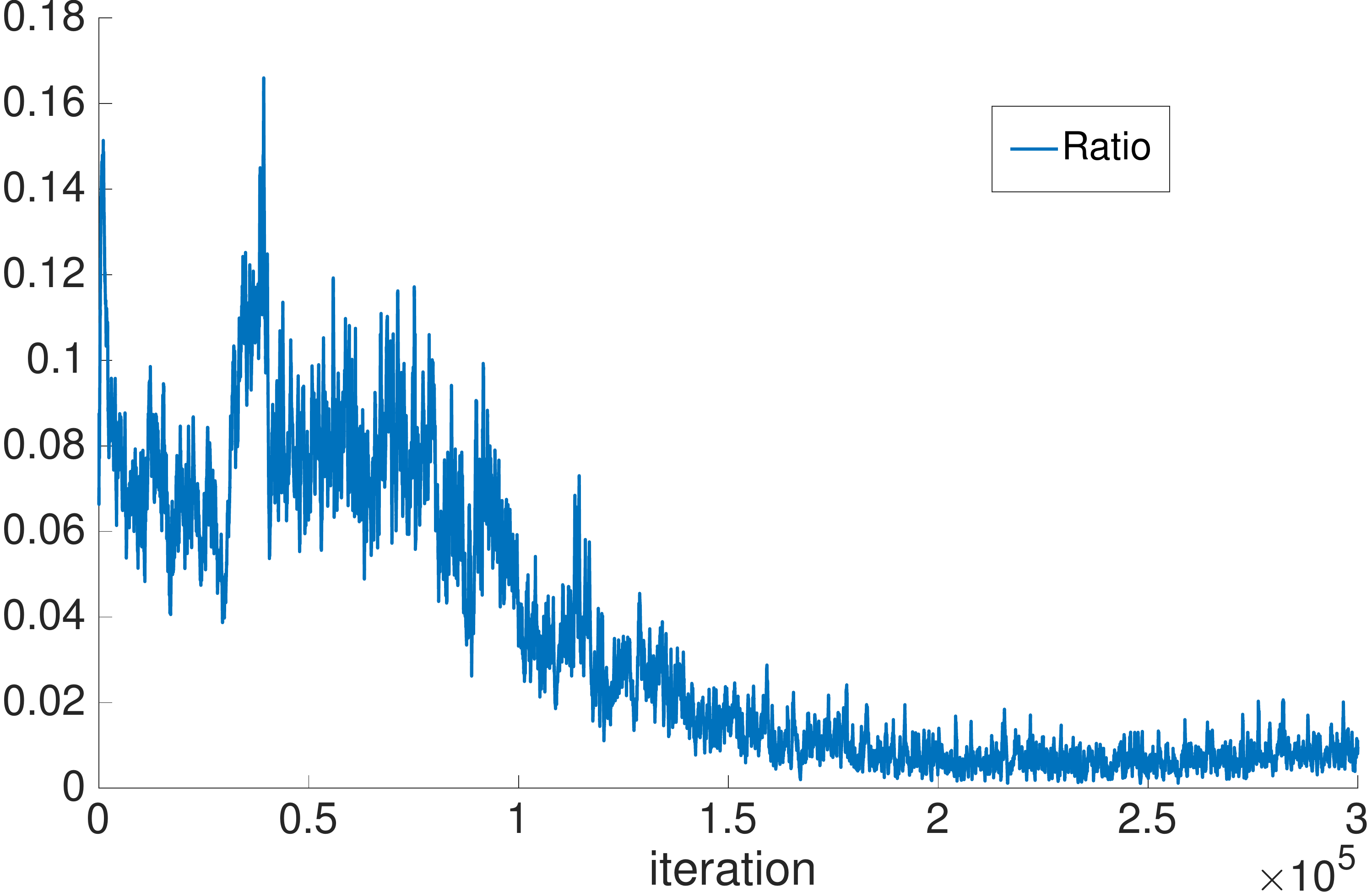}}

	\end{tabular}
	
	\protect
	\caption[The real \df and the approximation error during the training of PSO-LDE for two different testing points.]{The real \df and the approximation error during the training of PSO-LDE for two different testing points $X_1$ and $X_2$.
		(a)-(b) Results for $X_1$.
		(c)-(d) Results for $X_2$.
		(a)-(c) Blue line is an absolute value of the real \df at a specific point for each iteration time, smoothed via a moving mean with the window size 300; red line is the absolute value of a difference between the real \df and the approximated one, smoothed via a moving mean of the same window size.
		(b)-(d) Ratio between two lines in (a)-(c), can be seen as a moving mean version of Eq.~(\ref{eq:RatioDef}) - an error as the percentage of the real \df.
	}
	\label{fig:DifferntlRes2}
\end{figure}

Additionally, we calculated the real and approximated \dfs along the entire GD optimization process of PSO-LDE, where the same NN architecture was used as in the first experiment, and where the pdf inference was applied to \emph{Columns} distribution from Section \ref{sec:ColumnsEst}. Particularly, we trained a NN for 300000 iterations, while during each iteration we computed the real and the approximated \dfs for a specific test point $X$. We performed such simulation twice, for two different points and plotted their \dfs in Figure \ref{fig:DifferntlRes2}. In the left column, the blue line is the absolute value of the real \df, $\left|df_{\theta}(X)\right|$, and the red line is error $\left|df_{\theta}(X) - \bar{df}(X)\right|$, both smoothed via moving mean with window size 300. The right column shows the ratio between smoothed $\left|df_{\theta}(X) - \bar{df}(X)\right|$ and smoothed $\left|df_{\theta}(X)\right|$, $\left|df_{\theta}(X) - \bar{df}(X)\right|/\left|df_{\theta}(X)\right|$, which can be seen as the error percentage from the real \df. As shown in Figures \ref{fig:DifferntlRes2-b} and \ref{fig:DifferntlRes2-d},  this error percentage is less than $15\%$ and for most part of the training is even lower. This trend is the same for both verified points, suggesting that the real \df indeed can be approximated very closely by the first-order Taylor expansion. 
In overall, above we showed that most of the surface change can be explained by the \emph{gradient similarity} $g_{\theta}(X, X')$ in Eq.~(\ref{eq:PSODffrntl}).

\section{Weights Uncorrelation and Gradient Similarity Space}
\label{sec:UncorRes}

In this appendix we empirically demonstrate the relation between \emph{gradient similarity} $g_{\theta}(X, X') = \nabla_{\theta} 
f_{\theta}(X)^T \cdot \nabla_{\theta} 
f_{\theta}(X')$ and Euclidean distance $d(X, X')$, and show how this relation changes along the optimization over NNs. Particularly, we observe empirically that during first several thousand iterations of a typical optimization the trend is achieved where high values of $g_{\theta}(X, X')$ are correlated with small values of $d(X, X')$ - the model kernel of NN obtains a local-support structure. Further, this trend is preserved during the rest part of the optimization. This behavior can be seen as an another motivation for the kernel bandwidth analysis made in Section \ref{sec:ExprrKernelBnd} - the shape of $r_{\theta}(X, X') = \frac{g_{\theta}(X, X')}{g_{\theta}(X, X)}$ has some implicit particular bandwidth.

\begin{figure}[!tbp]
	\centering
	
	\begin{tabular}{cc}

		\subfloat[\label{fig:ColsRes6.1-a}]{\includegraphics[width=0.45\textwidth]{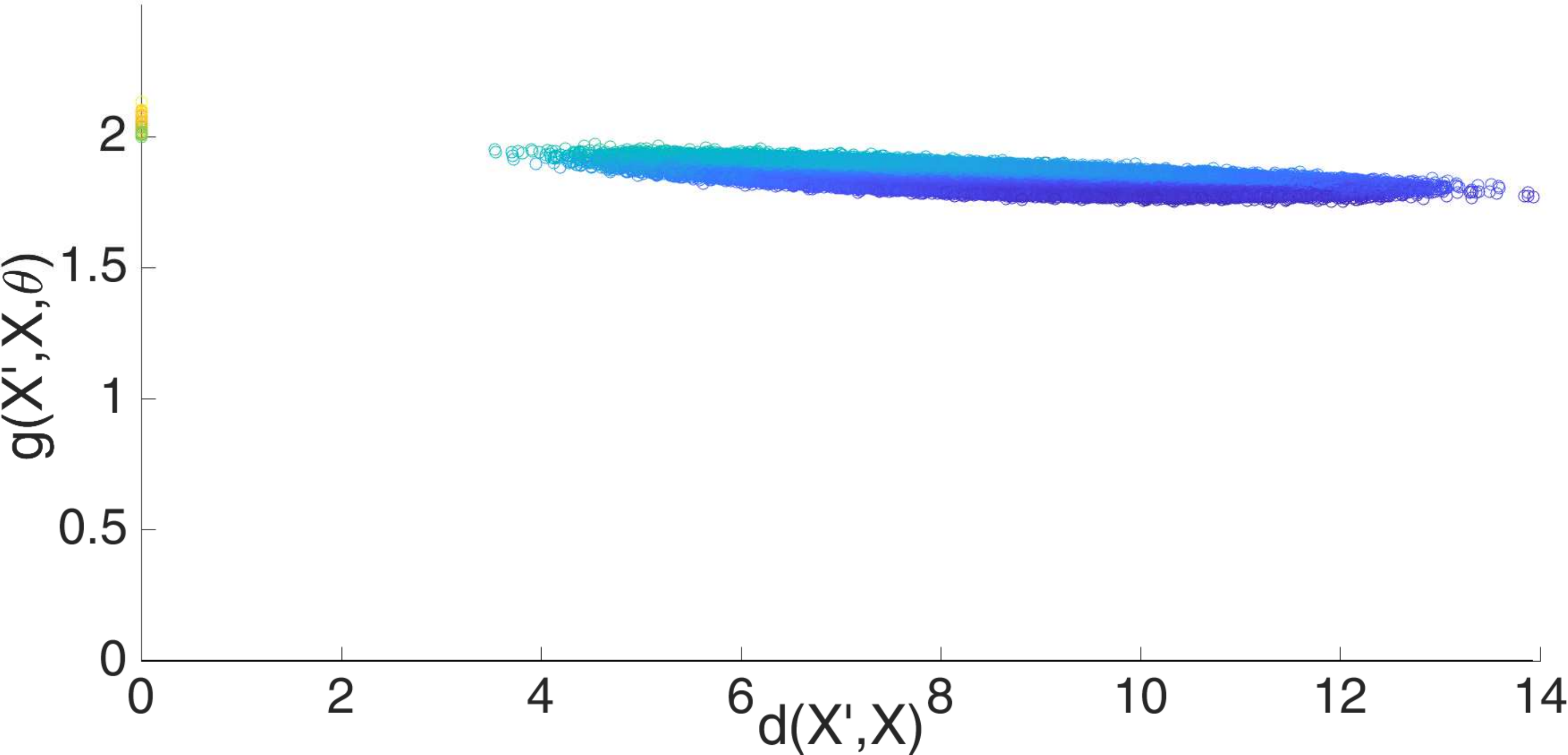}}
		&
		
		\subfloat[\label{fig:ColsRes6.1-b}]{\includegraphics[width=0.45\textwidth]{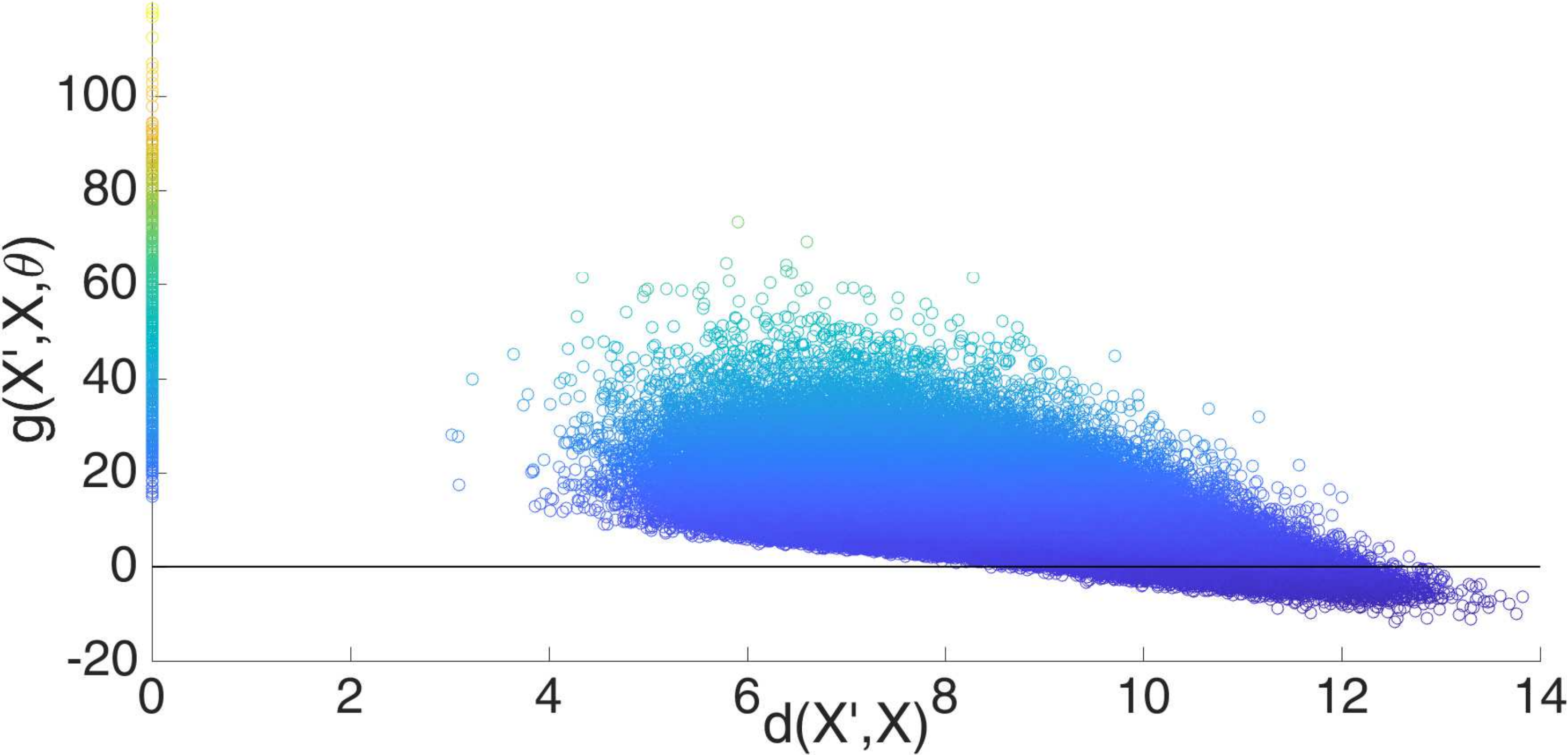}}
		\\
		\subfloat[\label{fig:ColsRes6.1-c}]{\includegraphics[width=0.45\textwidth]{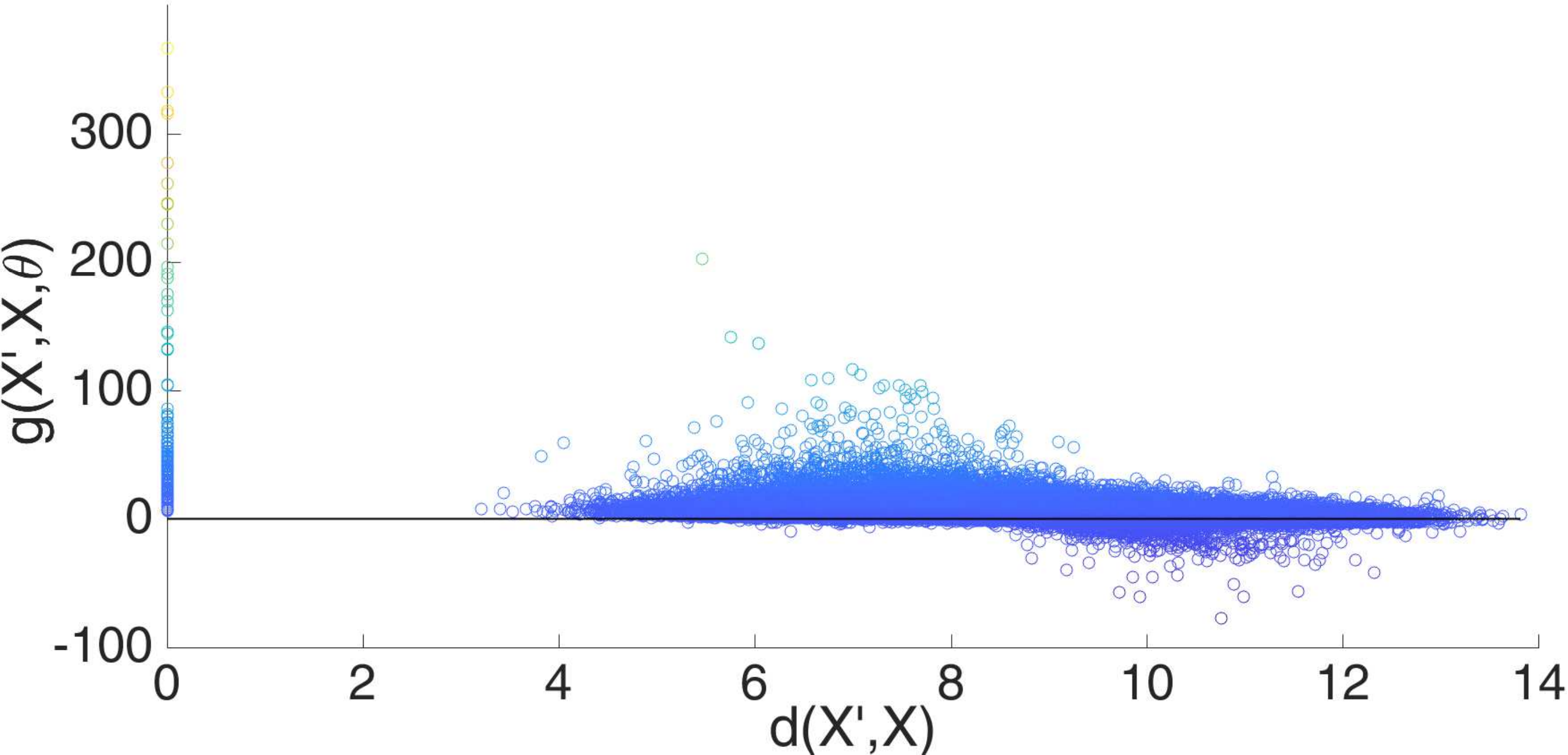}}
		&
		\subfloat[\label{fig:ColsRes6.1-d}]{\includegraphics[width=0.45\textwidth]{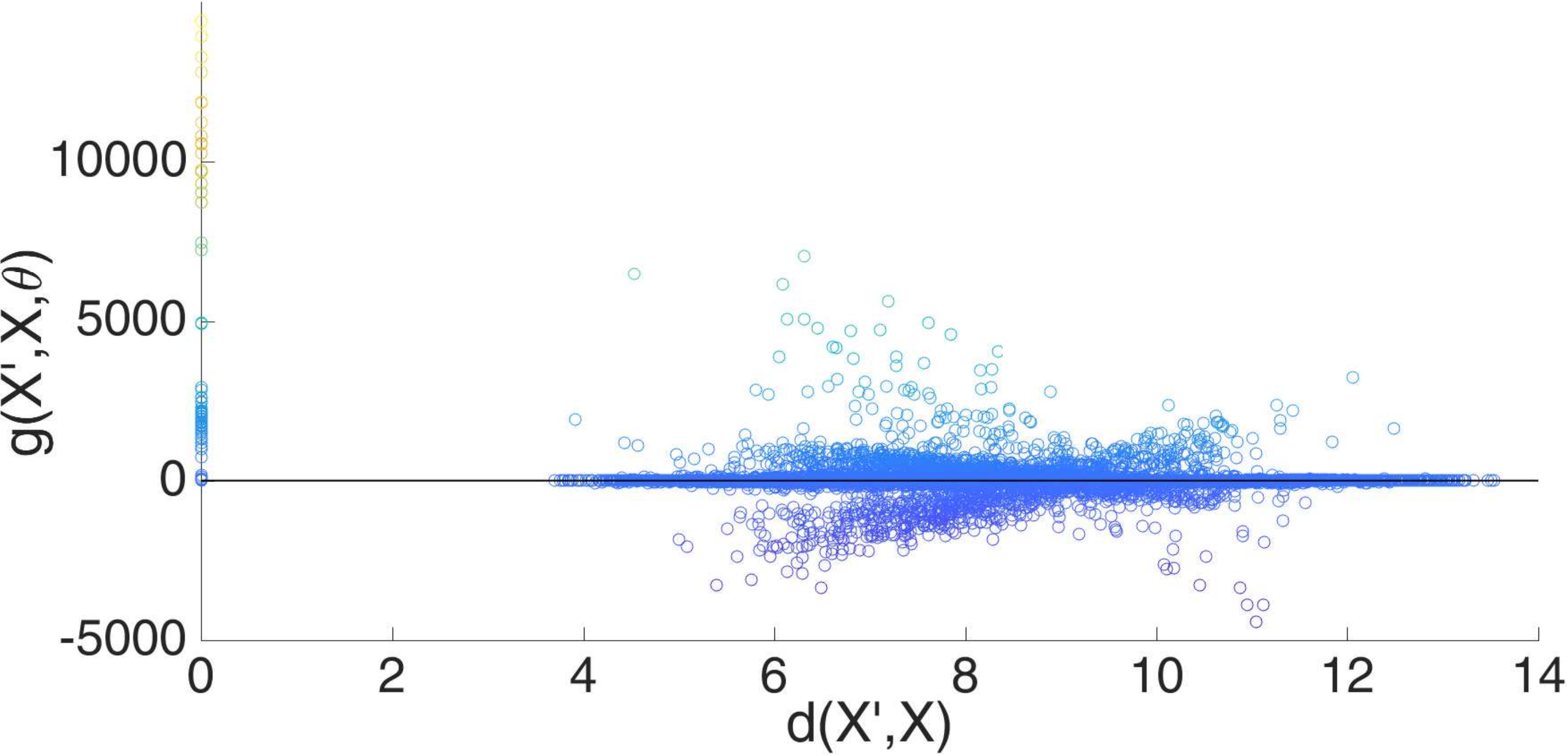}}
		
		\\
		\subfloat[\label{fig:ColsRes6.1-e}]{\includegraphics[width=0.45\textwidth]{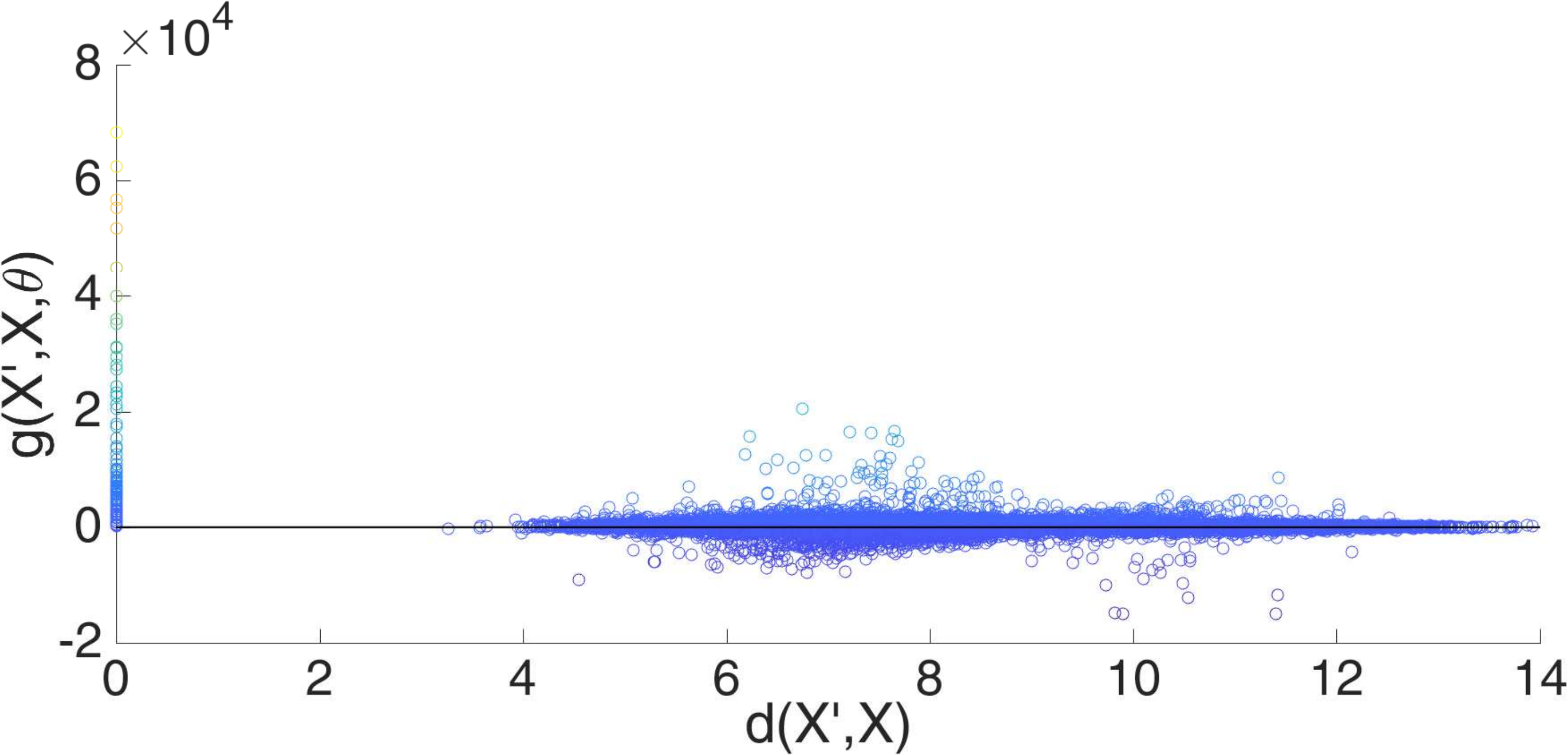}}
		&
		\subfloat[\label{fig:ColsRes6.1-f}]{\includegraphics[width=0.45\textwidth]{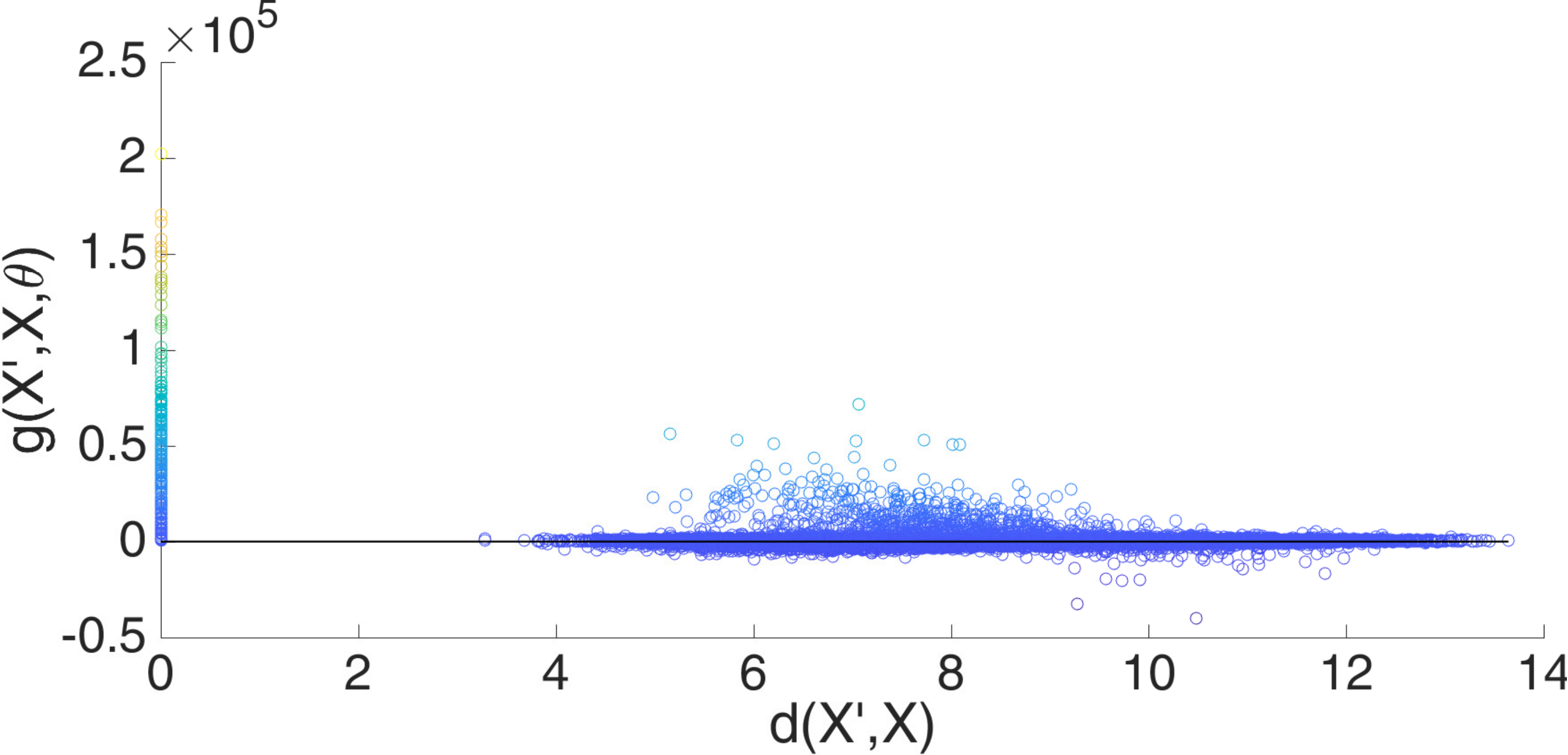}}
		
		\\
		\subfloat[\label{fig:ColsRes6.1-g}]{\includegraphics[width=0.45\textwidth]{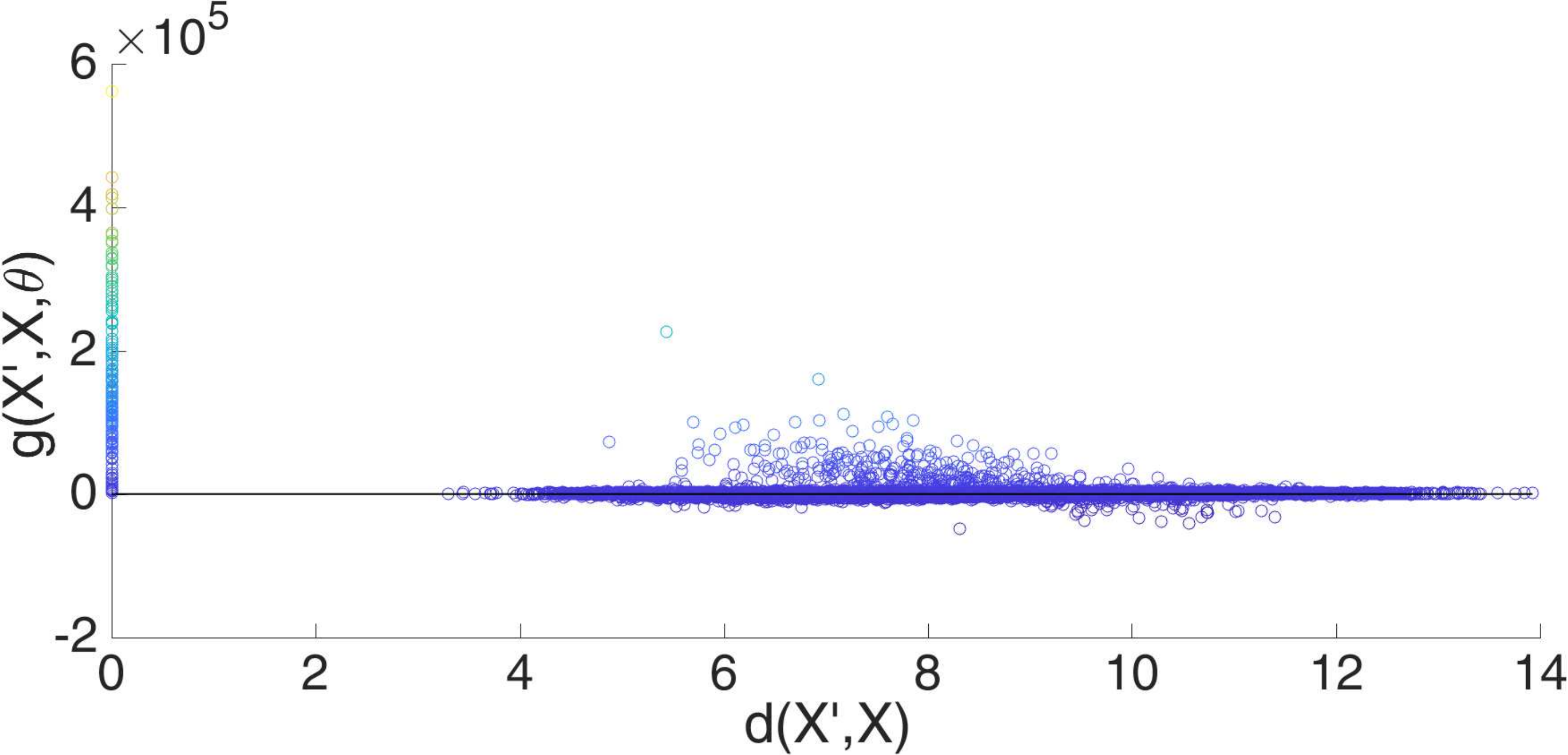}}
		&
		\subfloat[\label{fig:ColsRes6.1-h}]{\includegraphics[width=0.45\textwidth]{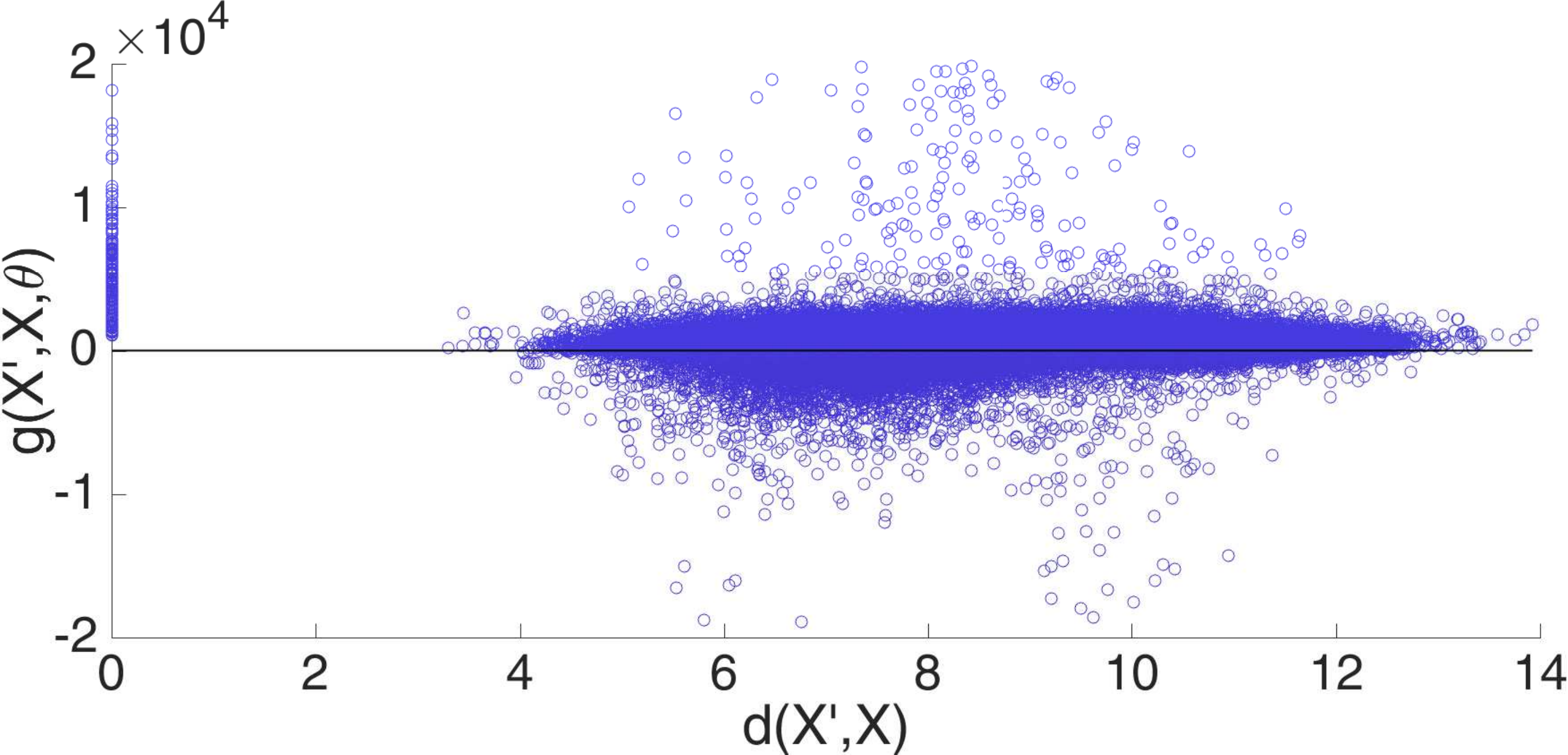}}

	\end{tabular}
	
	\protect
	\caption[Relation between values of \emph{gradient similarity} $g_{\theta}(X', X)$ and values of Euclidean distance $d(X', X)$, along the optimization time $t$.]{
		Relation between values of \emph{gradient similarity} $g_{\theta}(X', X)$ and values of Euclidean distance $d(X', X)$, along the optimization time $t$ (iteration index).
		PSO-LDE with $\alpha = \frac{1}{4}$ is applied, where NN architecture is block-diagonal with 6 layers, number of blocks $N_B = 50$ and block size $S_B = 64$ (see Section \ref{sec:BDLayers}).
		Each plot is constructed similarly to Figure \ref{fig:NNCovariance-b}, see the main text for more details.
		Outputs from both  $g_{\theta}(X', X)$ and $d(X', X)$ are demonstrated at different times; (a) $t = 0$, (b) $t = 100$, (c) $t = 3200$, (d) $t = 3400$, (e) $t = 6000$, (f) $t = 100000$ and (g) $t = 200000$. (h) Zoom-in of (g). As can be seen, self similarities $g_{\theta}(X, X)$, depicted at $d(X', X) = 0$, are high and increase during the optimization. The \emph{side} similarities $g_{\theta}(X', X)$ for $X' \neq X$, depicted at $d(X', X) > 0$, are centered around zero at $t = 200000$ and are significantly lower than self similarities.
	}
	\label{fig:ColsRes6.1}
\end{figure}

\begin{figure}[tb]
	\centering
	
	\newcommand{\width}[0] {0.235}
	\newcommand{\height}[0] {0.15}
	\setlength{\tabcolsep}{0pt}
	\renewcommand{\arraystretch}{0}
	
	\begin{tabular}{cccc}

		\subfloat[\label{fig:ColsRes6.12-a}]{
			\includegraphics[height=\height\textheight,width=\width\textwidth]{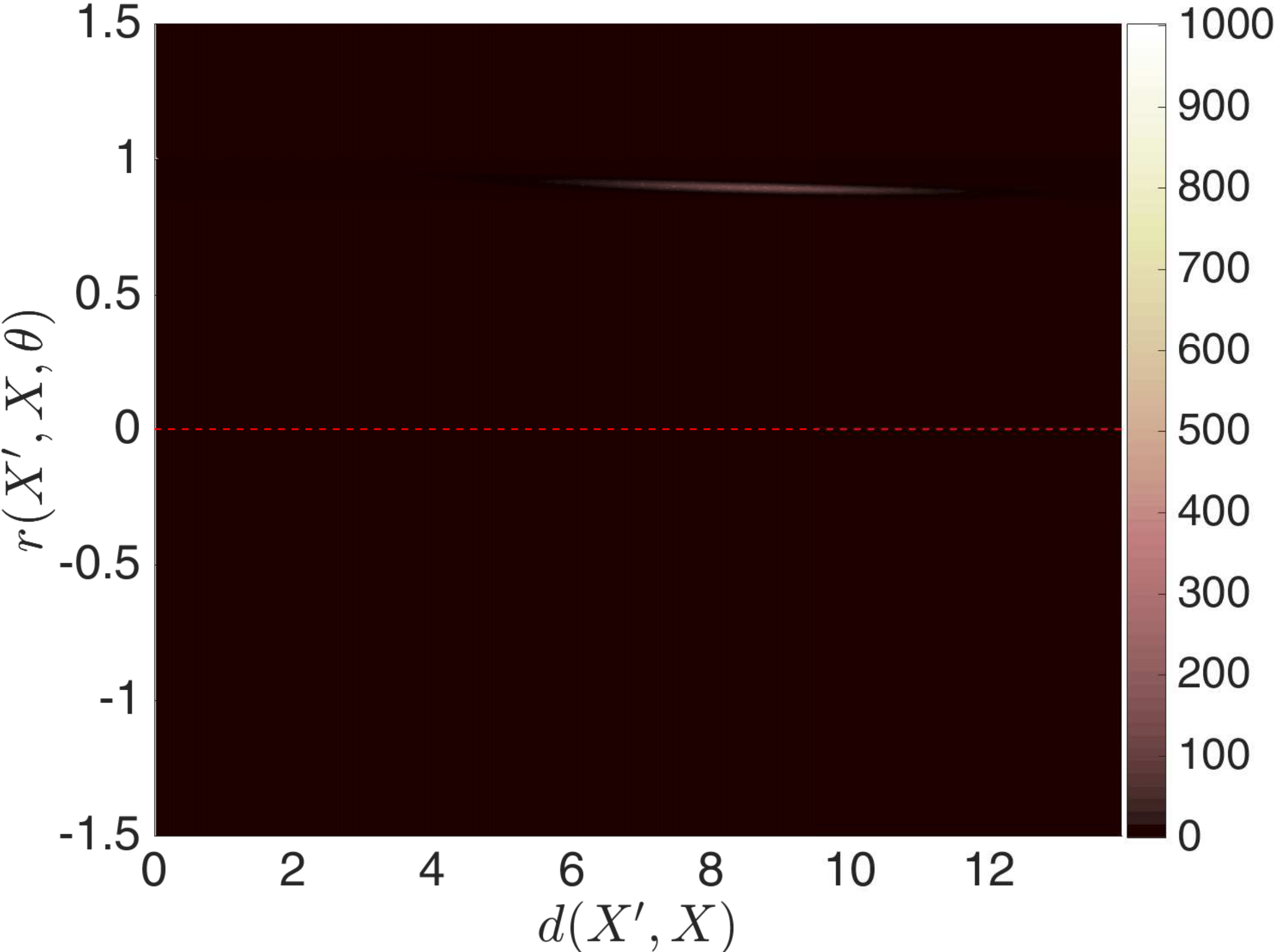}
		}
		
		&
		
		\subfloat[\label{fig:ColsRes6.12-b}]{
			\includegraphics[height=\height\textheight,width=\width\textwidth]{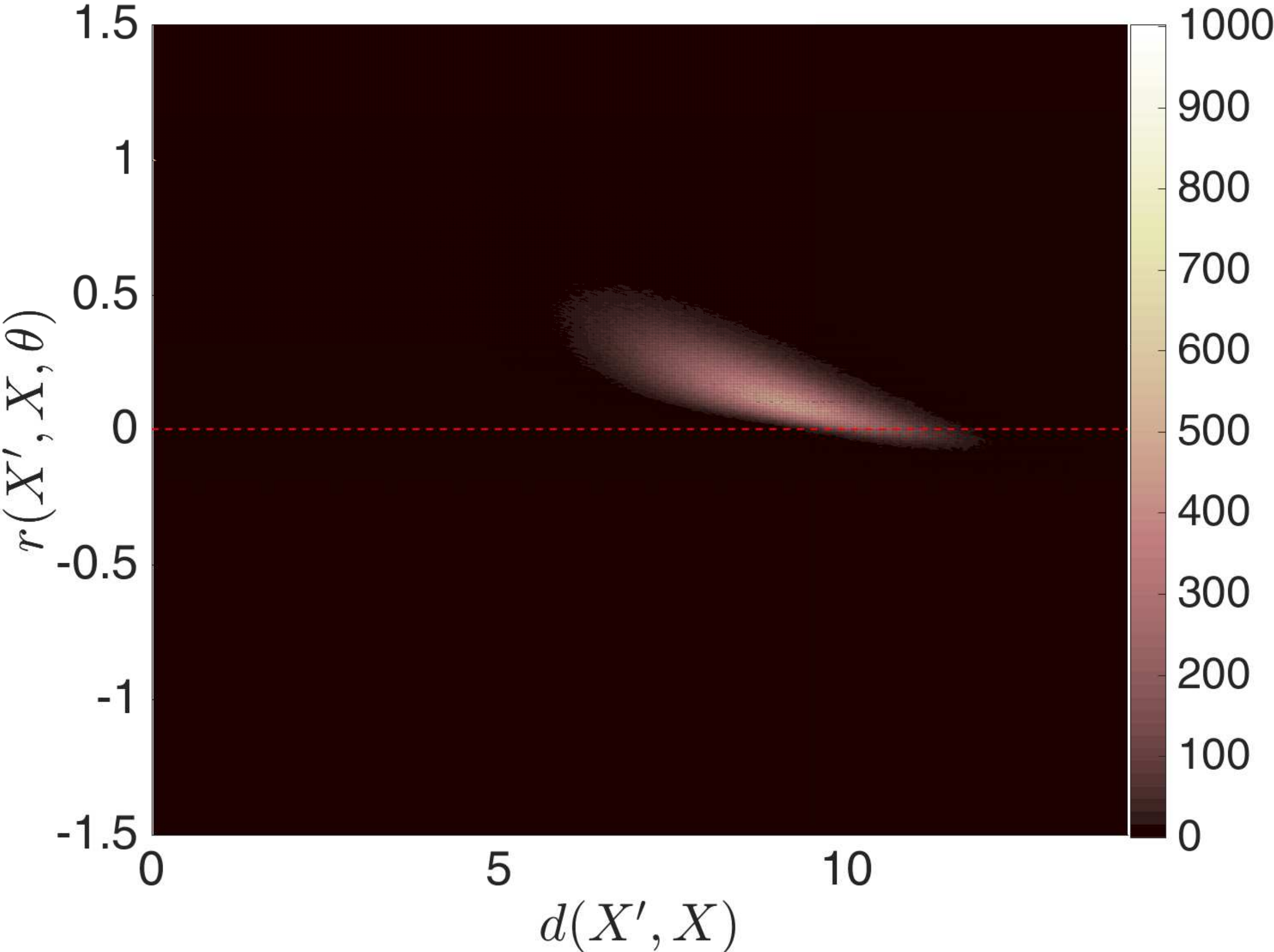}
		}
		
		&

		\subfloat[\label{fig:ColsRes6.12-c}]{
			\includegraphics[height=\height\textheight,width=\width\textwidth]{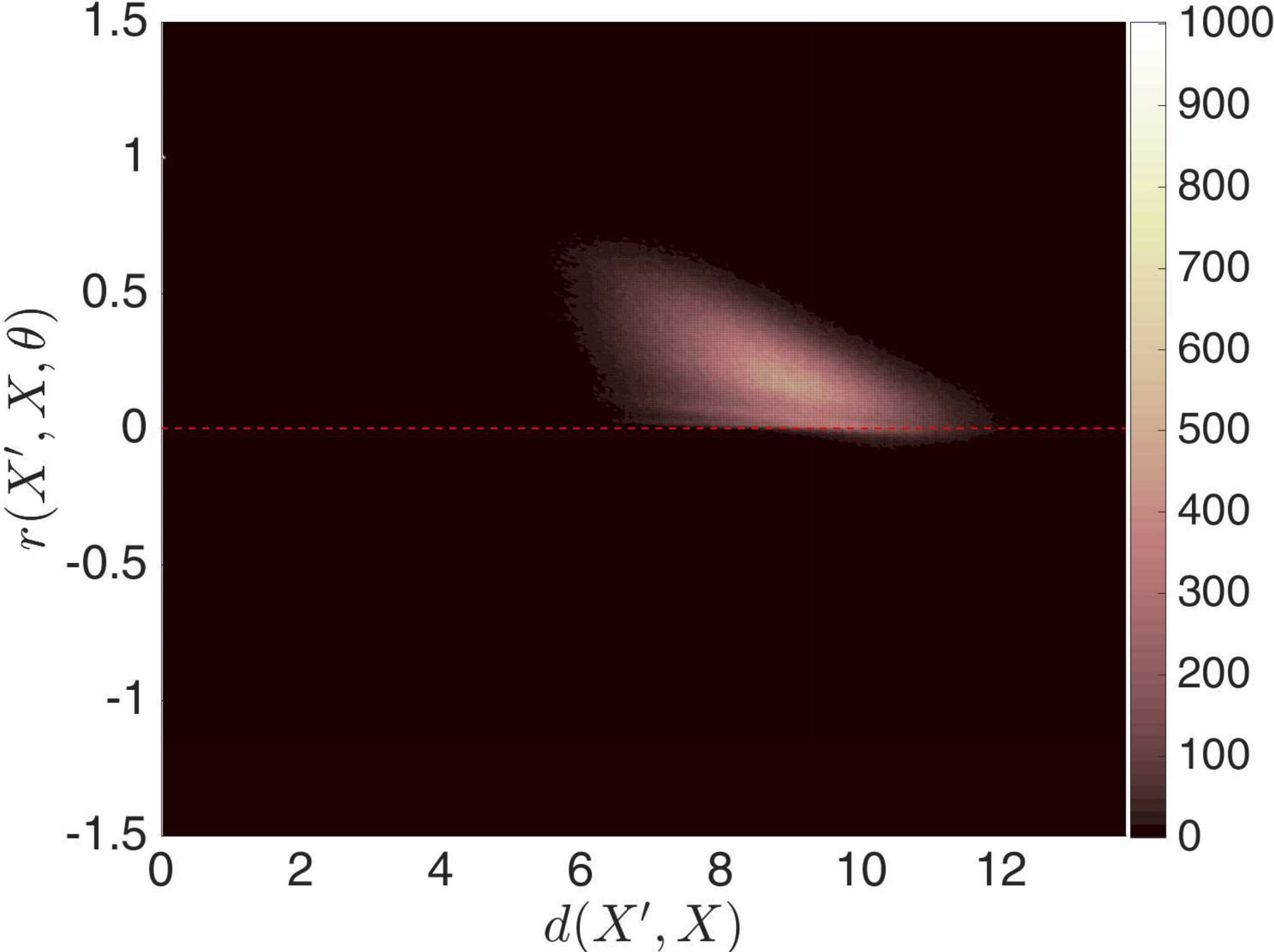}
		}
		
		&
		
		\subfloat[\label{fig:ColsRes6.12-d}]{
			\includegraphics[height=\height\textheight,width=\width\textwidth]{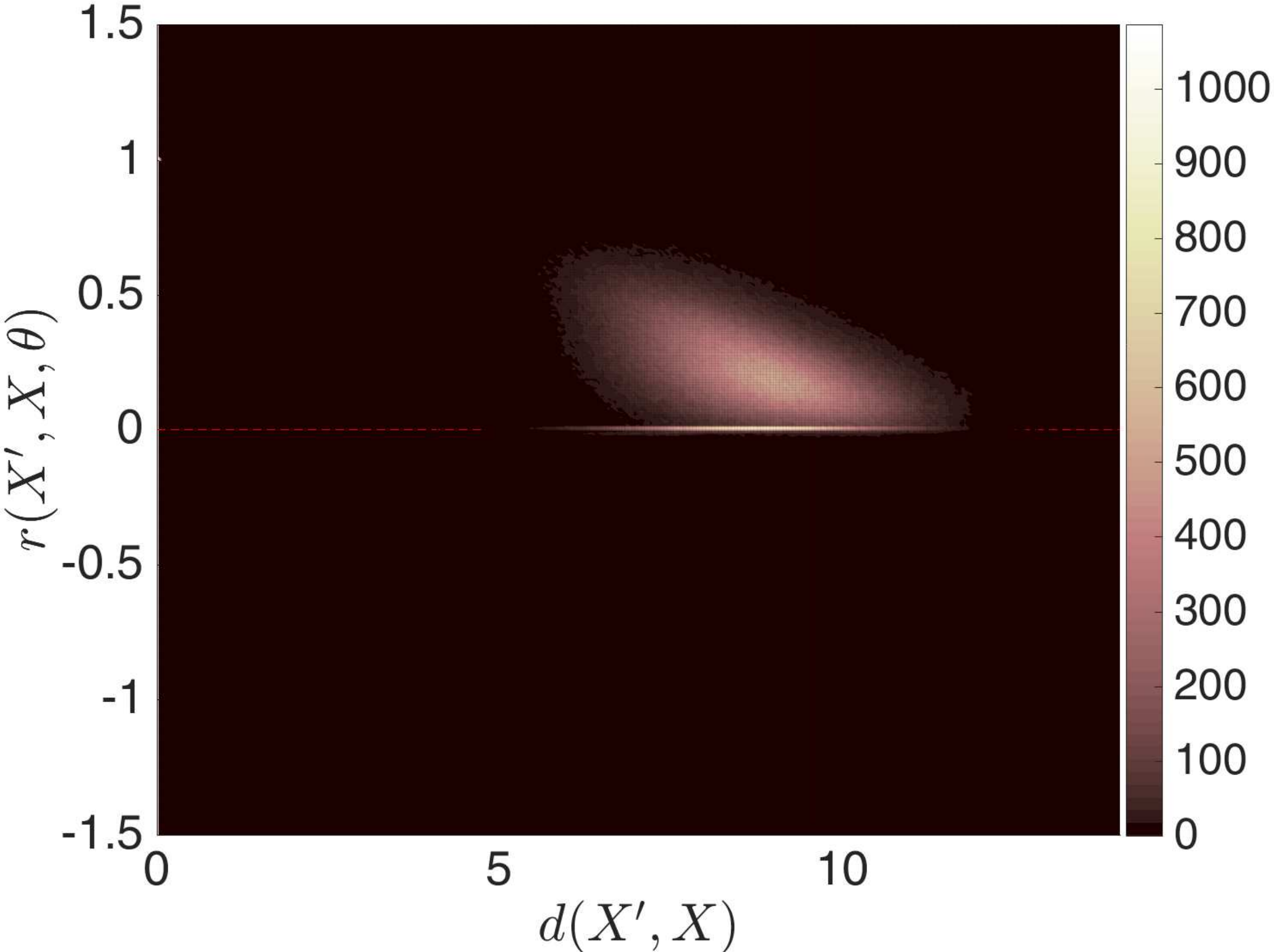}
		}
		
		\\

		\subfloat[\label{fig:ColsRes6.12-e}]{
			\includegraphics[height=\height\textheight,width=\width\textwidth]{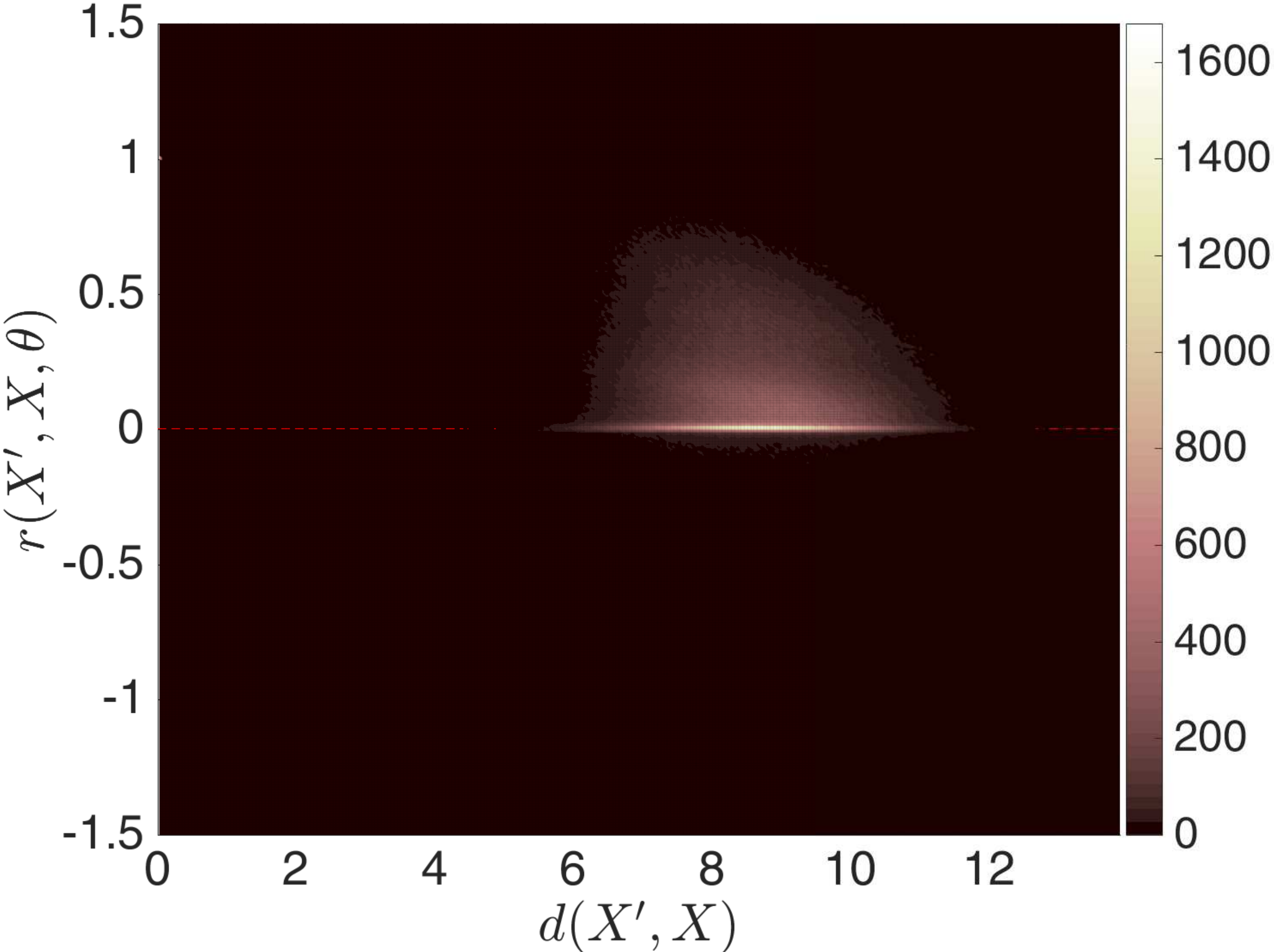}
		}
		
		&
		
		\subfloat[\label{fig:ColsRes6.12-f}]{
			\includegraphics[height=\height\textheight,width=\width\textwidth]{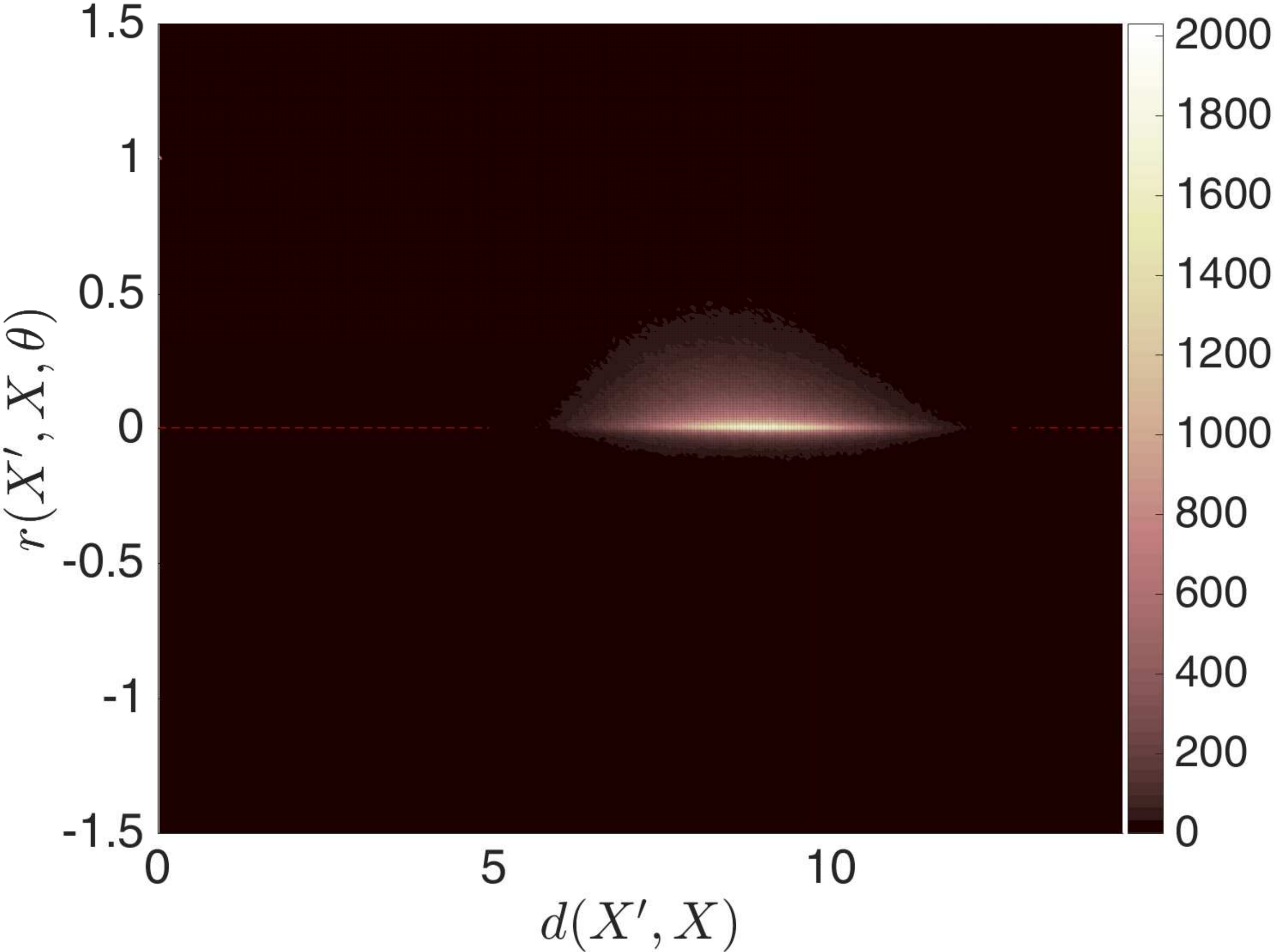}
		}
		
		&
		
		\subfloat[\label{fig:ColsRes6.12-g}]{
			\includegraphics[height=\height\textheight,width=\width\textwidth]{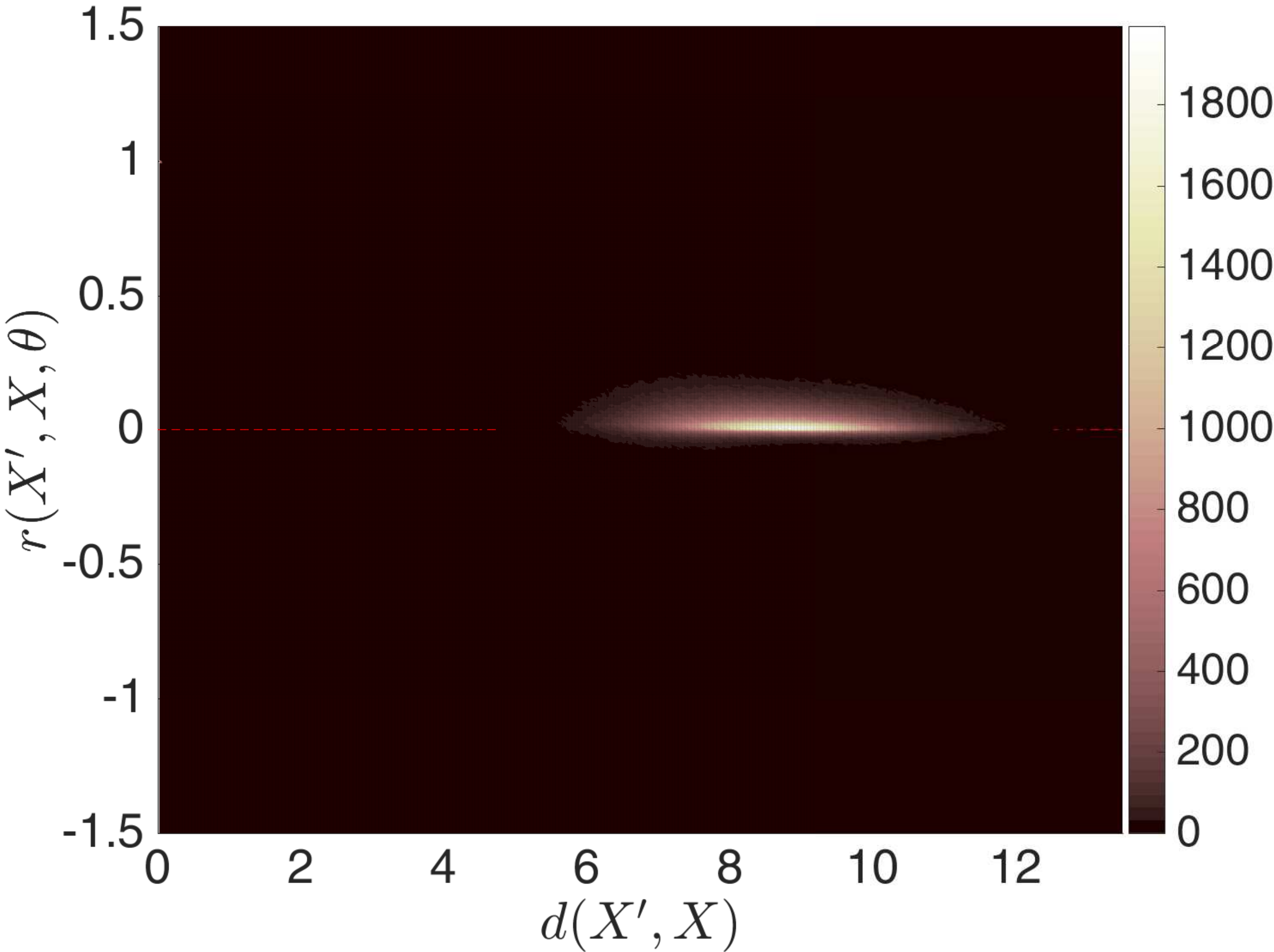}
		}
		
		&
		
		\subfloat[\label{fig:ColsRes6.12-h}]{
			\includegraphics[height=\height\textheight,width=\width\textwidth]{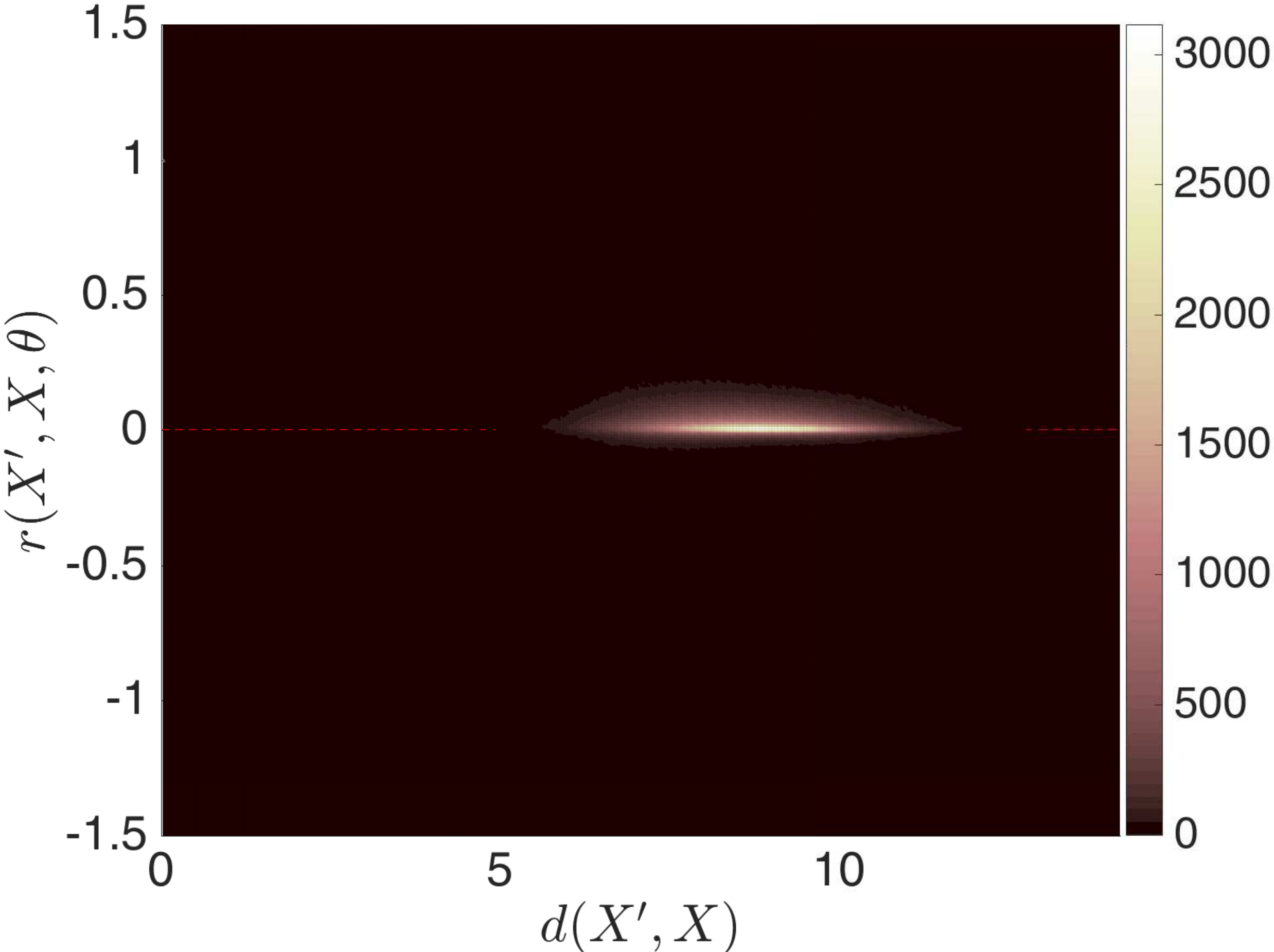}
		}
		
		\\
		
		\subfloat[\label{fig:ColsRes6.12-i}]{
			\includegraphics[height=\height\textheight,width=\width\textwidth]{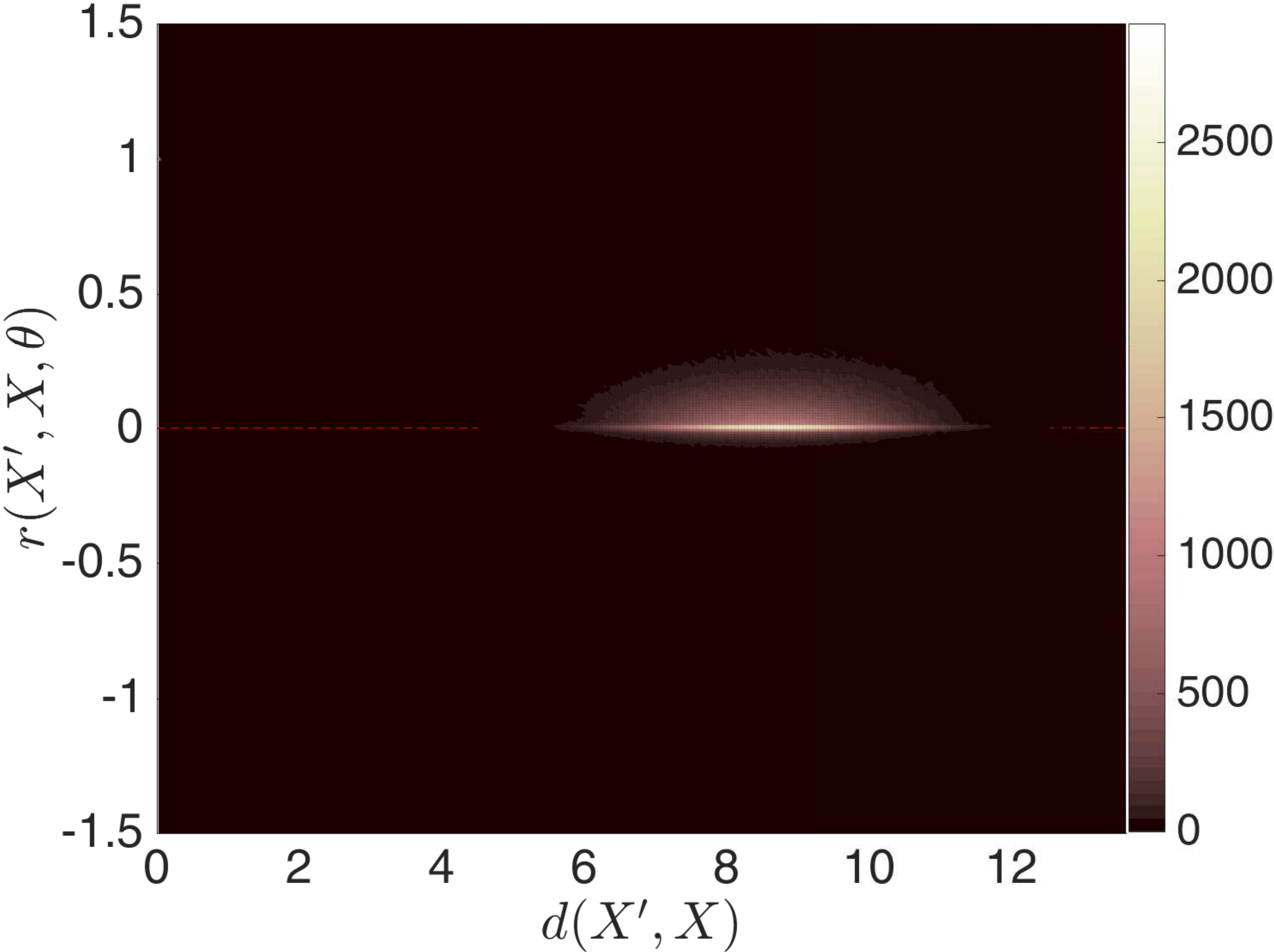}
		}
		
		&
		
		\subfloat[\label{fig:ColsRes6.12-j}]{
			\includegraphics[height=\height\textheight,width=\width\textwidth]{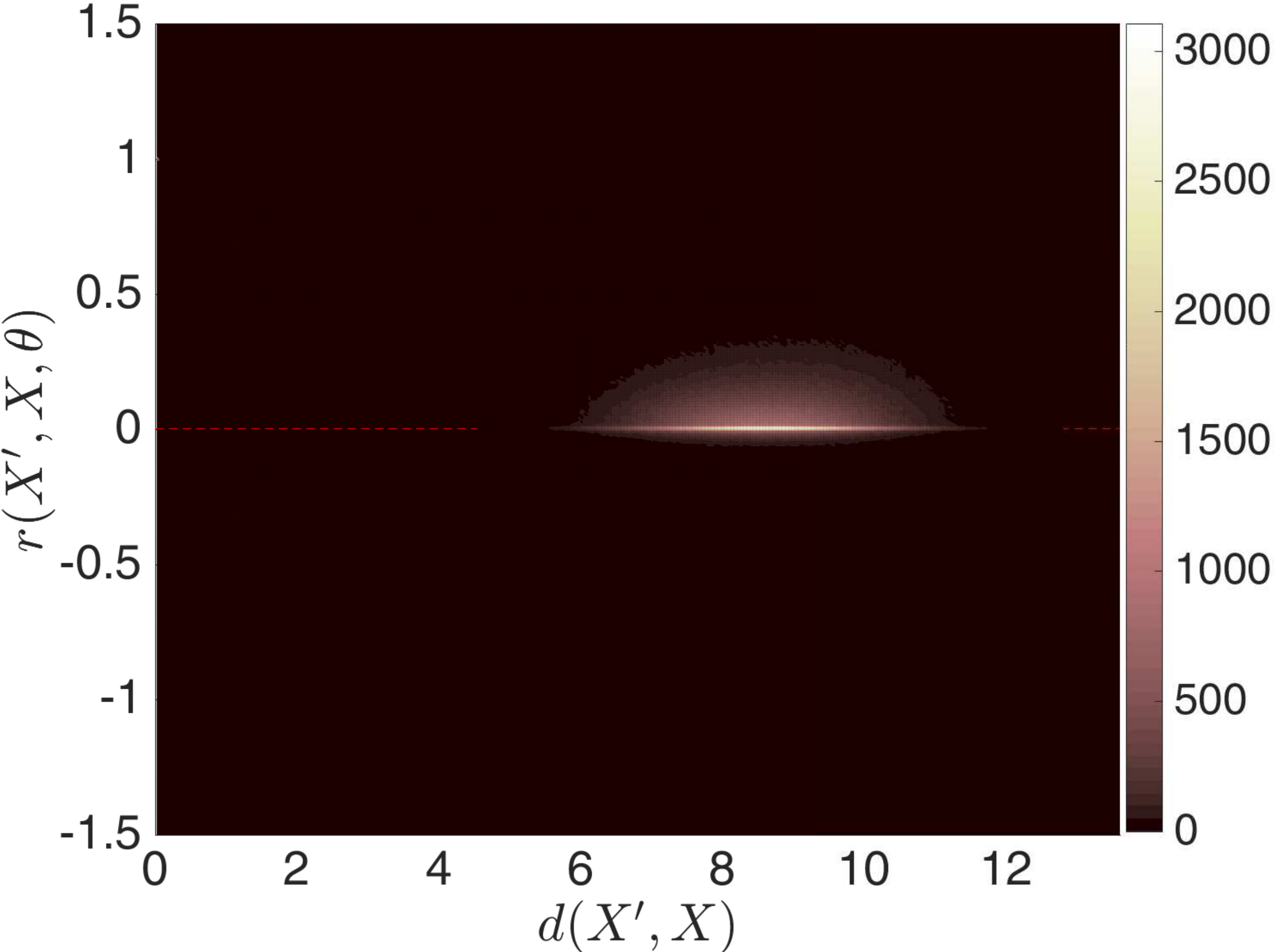}
		}
		
		&
		
		\subfloat[\label{fig:ColsRes6.12-k}]{
			\includegraphics[height=\height\textheight,width=\width\textwidth]{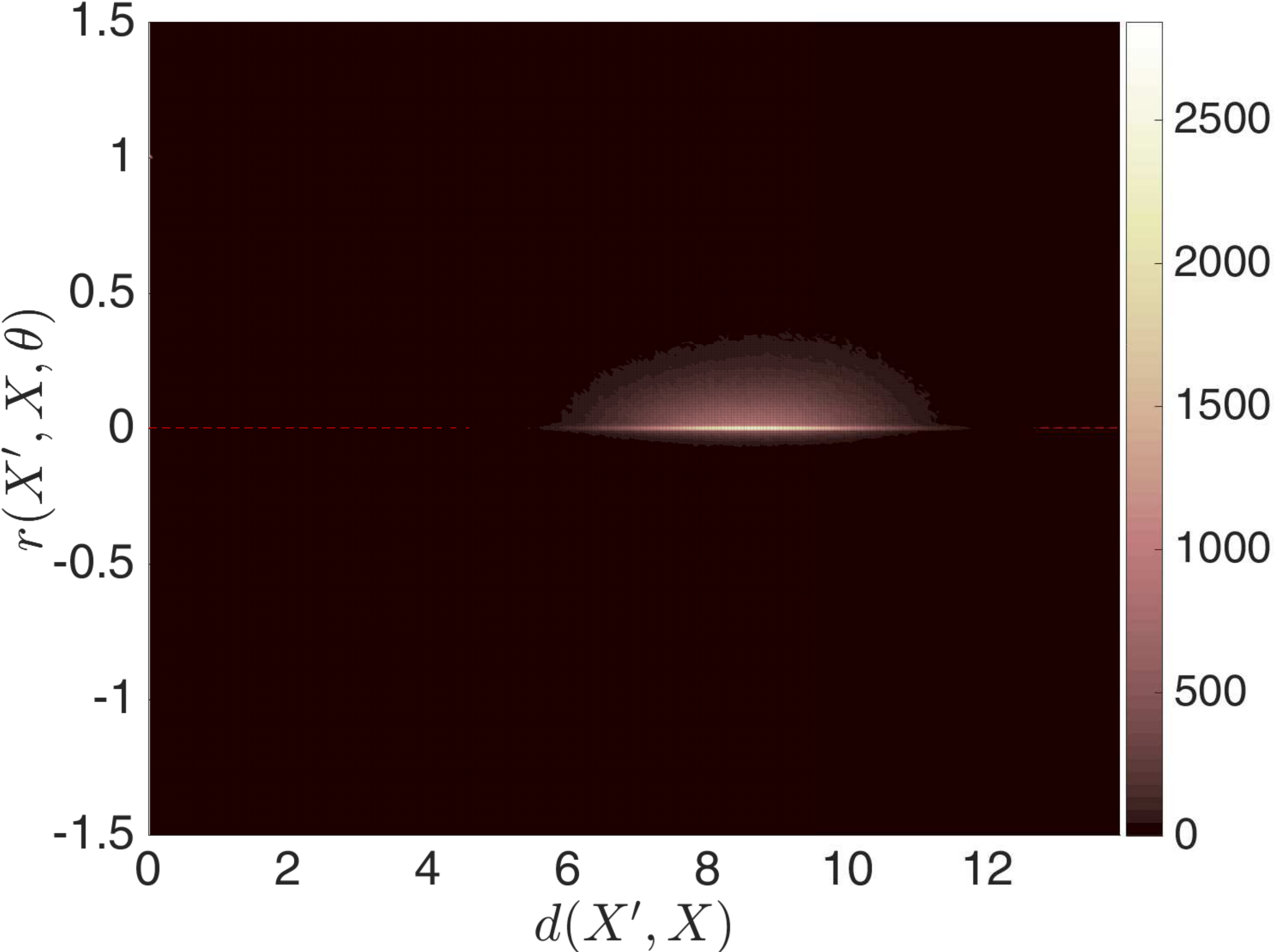}
		}
		
		&
		
		\subfloat[\label{fig:ColsRes6.12-l}]{
			\includegraphics[height=\height\textheight,width=\width\textwidth]{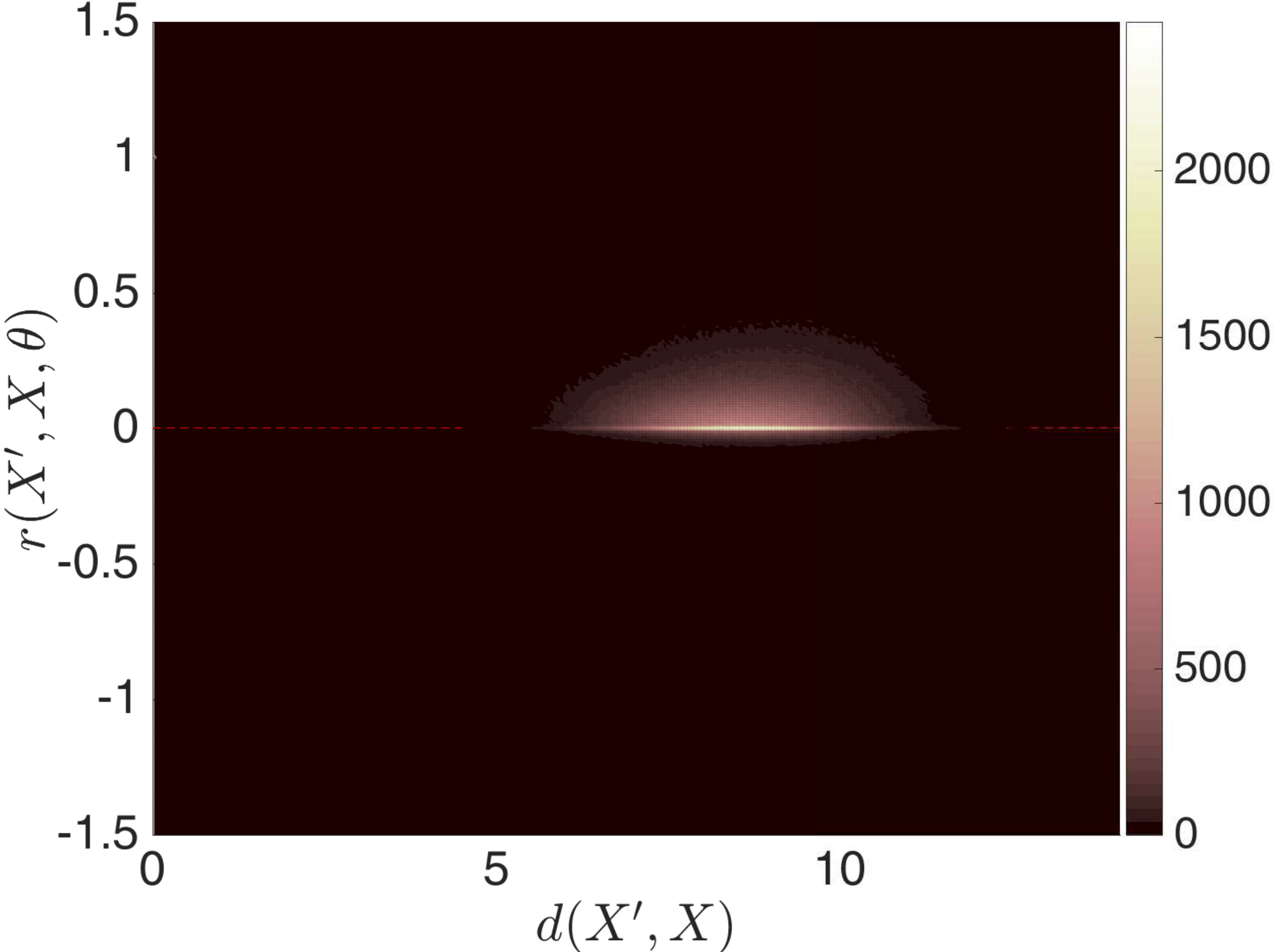}
		}
		
	\end{tabular}
	
	\protect
	\caption[Histograms of the relative \gs $r_{\theta}(X', X)$ and the Euclidean distance $d(X', X)$, for the experiment in Figure \ref{fig:ColsRes6.1}.]{Histograms of the relative \gs $r_{\theta}(X', X)$ and the Euclidean distance $d(X', X)$ at different optimization times $t$, for the experiment in Figure \ref{fig:ColsRes6.1}.
		At each time $t$ we calculate a \emph{relative} side-influence $r_{\theta}(X_i, X_j)$ and a Euclidean distance $d(X_i, X_j)$ for $10^6$ point pairs and depict a histogram of obtained $\{ r_{\theta}(X_i, X_j) \}$ and $\{ d(X_i, X_j) \}$.
		The optimization time is (a) $t = 0$, (b) $t = 100$, (c) $t = 3200$, (d) $t = 3400$, (e) $t = 4000$, (f) $t = 6000$, (g) $t = 10000$, (h) $t = 29000$, (i) $t = 100000$, (j) $t = 150000$, (k) $t = 200000$ and (l) $t = 300000$. As observed, after $t = 10000$ the relative \gs between far away regions is much smaller than 1, implying that there is a insignificant side-influence over height $f_{\theta}(X)$ at point $X$ from other points that are far away from $X$.
	}
	\label{fig:ColsRes6.12}
\end{figure}

\paragraph{Global Evaluation}

We apply PSO-LDE with $\alpha = \frac{1}{4}$ on a BD model for the inference of \emph{Columns} distribution defined in Eq.~(\ref{eq:ColumnsDef}), where at different optimization iterations we plot output pairs of $g_{\theta}(X, X')$ and $d(X, X')$. The plots are constructed similarly to Figure \ref{fig:NNCovariance-b}. Specifically, we sample 500 points $D^{\usuff} = \{ X^{\usuff}_{i} \}$ and 500 points $D^{\dsuff} = \{ X^{\dsuff}_{i} \}$ from $\probs{\usuff}$ and $\probs{\dsuff}$ respectively. For each sample from $D = D^{\usuff} \cup D^{\dsuff}$ we calculate the gradient $\nabla_{\theta} 
f_{\theta}(X)$. Further we compute Euclidean distance and the \gs between every two points within $D$, producing $\frac{1000 \cdot 1001}{2}$ pairs of distance and similarity values. These values are plotted in Figure \ref{fig:ColsRes6.1}.

Also, we compute a \emph{relative} side-influence $r_{\theta}(X, X')$, defined in Eq.~(\ref{eq:RelKernel}), for each pair of points in $D$. In Figure \ref{fig:ColsRes6.12} we construct a histogram of $10^6$ obtained pairs $\{ r_{\theta}(X_i, X_j), d(X_i, X_j)  \}$.

As seen from the above figures, during first iterations the \gs obtains a form where its values are monotonically decreasing with bigger Euclidean distance between the points. During next optimization iterations the self similarity $g_{\theta}(X, X)$ is growing by several orders of magnitudes. At the same time the side-similarity $g_{\theta}(X', X)$ for $X' \neq X$ is growing significantly slower and mostly stays centered around zero. In overall values of $g_{\theta}(X, X)$ are much higher than values of $g_{\theta}(X', X)$ for $X' \neq X$, implying that the model kernel has mostly a local influence/impact. 
Likewise, from Figure \ref{fig:ColsRes6.12} it is also clear that $r_{\theta}(X, X')$ for faraway points $X$ and $X'$ is near-zero during the most part of the optimization process (i.e. after 10000 iterations in this experiment). Thus, the corresponding bandwidth of $r_{\theta}$ can be bounded similarly to Eq.~(\ref{eq:RelKernelBounds}).

\begin{figure}[!tbp]
	\centering
	
	\newcommand{\width}[0] {0.9}
	\newcommand{\height}[0] {0.13}

	\begin{tabular}{c}
		\subfloat[\label{fig:ColsRes6.2-a}]{\includegraphics[width=0.8\textwidth]{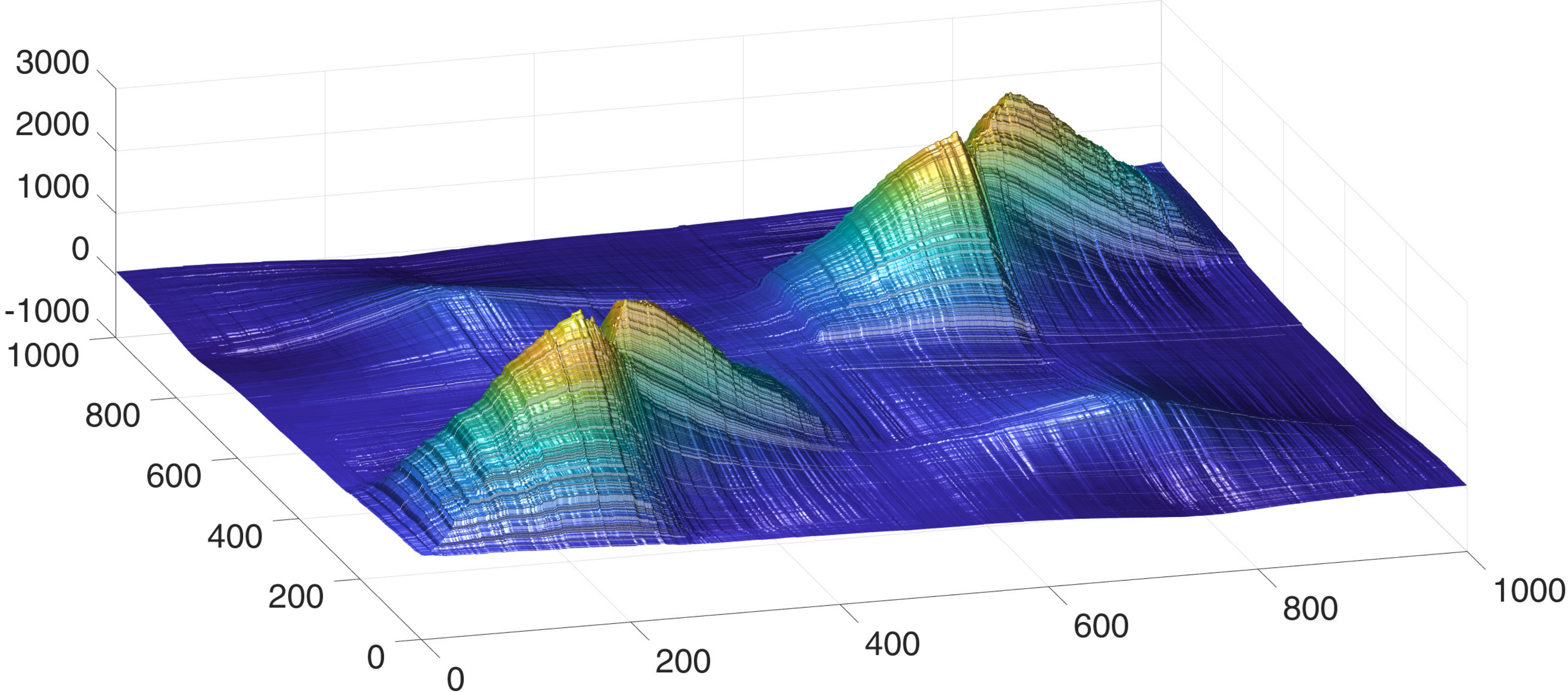}}
		\\
		\subfloat[\label{fig:ColsRes6.2-b}]{\includegraphics[height=\height\textheight,width=\width\textwidth]{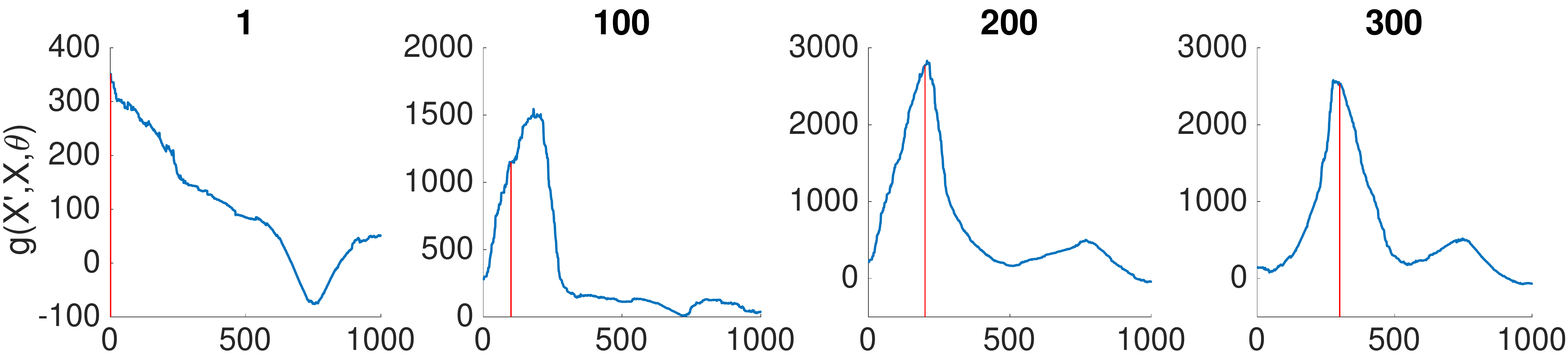}}
		\\
		\subfloat[\label{fig:ColsRes6.2-c}]{\includegraphics[height=\height\textheight,width=\width\textwidth]{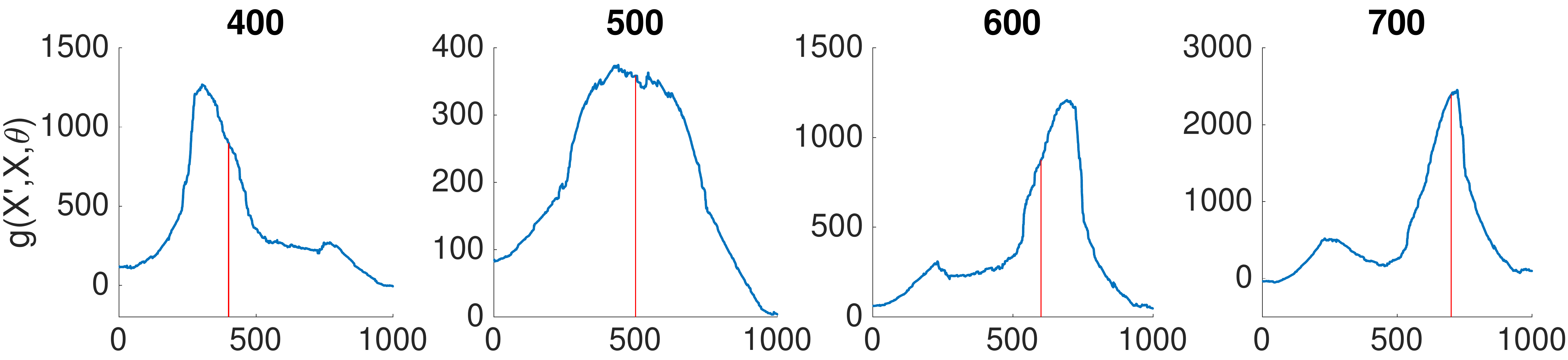}}
		\\
		\subfloat[\label{fig:ColsRes6.2-d}]{\includegraphics[height=0.12\textheight,width=0.7\textwidth]{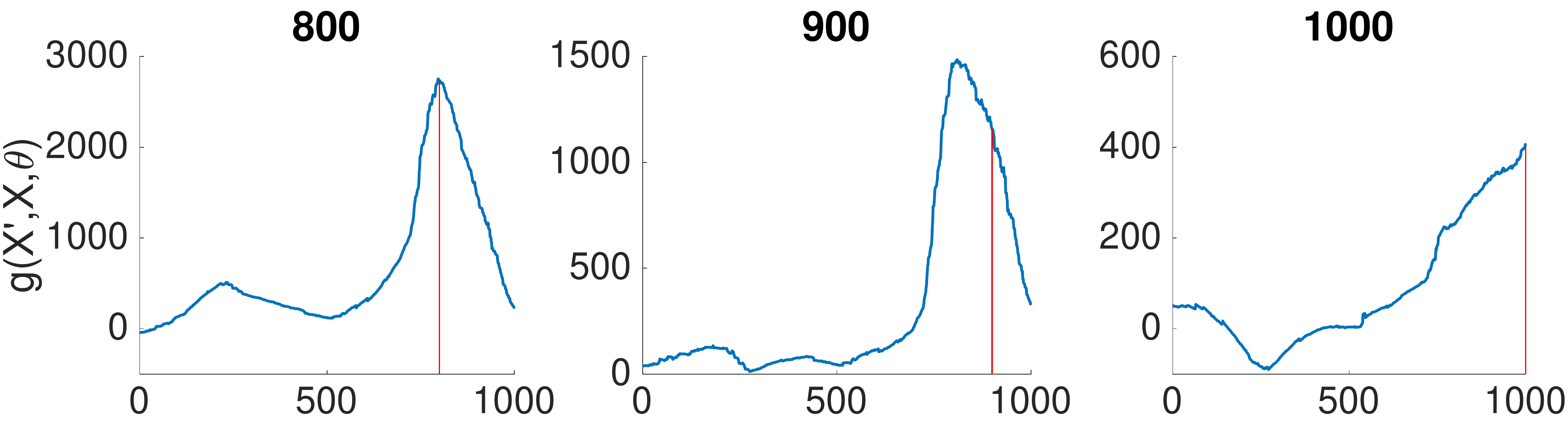}}
		
	\end{tabular}
	
	\protect
	\caption[Local relation between values of \emph{gradient similarity} $g_{\theta}(X', X)$ and values of Euclidean distance $d(X', X)$ within BD architecture.]{
		Local relation between values of \emph{gradient similarity} $g_{\theta}(X', X)$ and values of Euclidean distance $d(X', X)$ within BD architecture. The applied PSO method and model architecture are same as in Figure \ref{fig:ColsRes6.1}.
		After convergence, we calculate gradient $\nabla_{\theta} 
		f_{\theta}(X)$ along path within input space, $[-1, \ldots, -1] \longrightarrow[1, \ldots, 1]$, which is uniformly discretized via 1000 middle points. Afterwards, (a) the \emph{Gramian} matrix $G$ is constructed, with $G_{ij} = g_{\theta}(X_i, X_j)$ where $X_i$ and $X_j$ are $i$-th and $j$-th points along the path. Note that the index difference $\left|i - j\right|$ also represents Euclidean distance between points $X_i$ and $X_j$.
		Further, $i$-th row of $G$ contains similarities $g_{\theta}(X_i, \cdot)$ between point $i$ and rest of points. (a)-(c) 1-th, 100-th, \ldots, 900-th and 1000-th rows of $G$ are shown. Red line indicates $i$-th entry of $i$-th row where self similarity $g_{\theta}(X_i, X_i)$ is plotted. See more details in the main text.
	}
	\label{fig:ColsRes6.2}
\end{figure}

\paragraph{Local Evaluation}

Additionally, we performed a local evaluation of the above relation between \emph{gradient similarity} and Euclidean distance. Particularly, after  convergence we consider a path within a 20D input space, that starts at $S = [-1, \ldots, -1]$ and ends at $E = [1, \ldots, 1]$. We evenly discretized this path into 1000 middle points with which we form an ordered point set $D = \{S, \ldots, E\}$, with $\left| D \right| = 1000$. Afterwards, we calculate gradients $\nabla_{\theta} 
f_{\theta}(\cdot)$ at each point in $D$ and construct the \emph{Gramian} matrix $G$, with $G_{ij} = g_{\theta}(X_i, X_j)$. Note that the index of each point expresses also its location within the chosen point path, and the index difference $\left|i - j\right|$ also represents Euclidean distance between points $X_i$ and $X_j$. In Figure \ref{fig:ColsRes6.2-a} this matrix $G$ is depicted, and here it can be observed that $G$'s diagonal is very prominent. This again reasserts that the \emph{gradient similarity} kernel has some local support induced by the implicit bandwidth.

Moreover, each row $r_i$ inside $G$ represents $g_{\theta}(X_i, \cdot)$ - side similarity between the point $X_i$ and the rest of points in $D$ (i.e. the chosen path). Note that the $i$-th entry of $r_i$ represents self-similarity $g_{\theta}(X_i, X_i)$, while other entries represent the side-similarity from path points around $X_i$. Further, indexes of these other entries are related to the distance between the points and $X_i$, via the index difference $\left|i - j\right|$. Hence, first $i - 1$ entries of $r_i$ represent first $i - 1$ points within the chosen path $D$ before  the $i$-th point, whereas the last $1000 - i$ entries of $r_i$ represent points at the end of the path.
In Figures \ref{fig:ColsRes6.2-b}-\ref{fig:ColsRes6.2-d} 11 different rows of $G$ are depicted, where each row demonstrates \emph{gradient similarity} around some path point $X_i$ as a function of the second point index (and thus the distance between two points).
Here we can see that $g_{\theta}(X_i, X_j)$ typically has a peak at $X_j = X_i$, and further smoothly decreases as we walk away from point $X_i$ in any of the two path directions (towards $S$ or $E$). Thus, here we see that \emph{gradient similarity} has a bell-like behavior, returning high similarity for the same point and diminishing as the distance between the points grows. Yet, these "bells" are not centered, with some rows (e.g. $400$-th in Figure \ref{fig:ColsRes6.2-c}) having peaks outside of the $i$-th entry. Nevertheless, in context of PSO such local behavior allows us to conclude that when any training point $X$ is pushed by PSO loss, the force impact on the model surface is local, centered around the pushed $X$.

\begin{figure}[!tbp]
	\centering
	
	\newcommand{\width}[0] {0.9}
	\newcommand{\height}[0] {0.13}

	\begin{tabular}{c}
		\subfloat[\label{fig:ColsRes6.3-a}]{\includegraphics[width=0.8\textwidth]{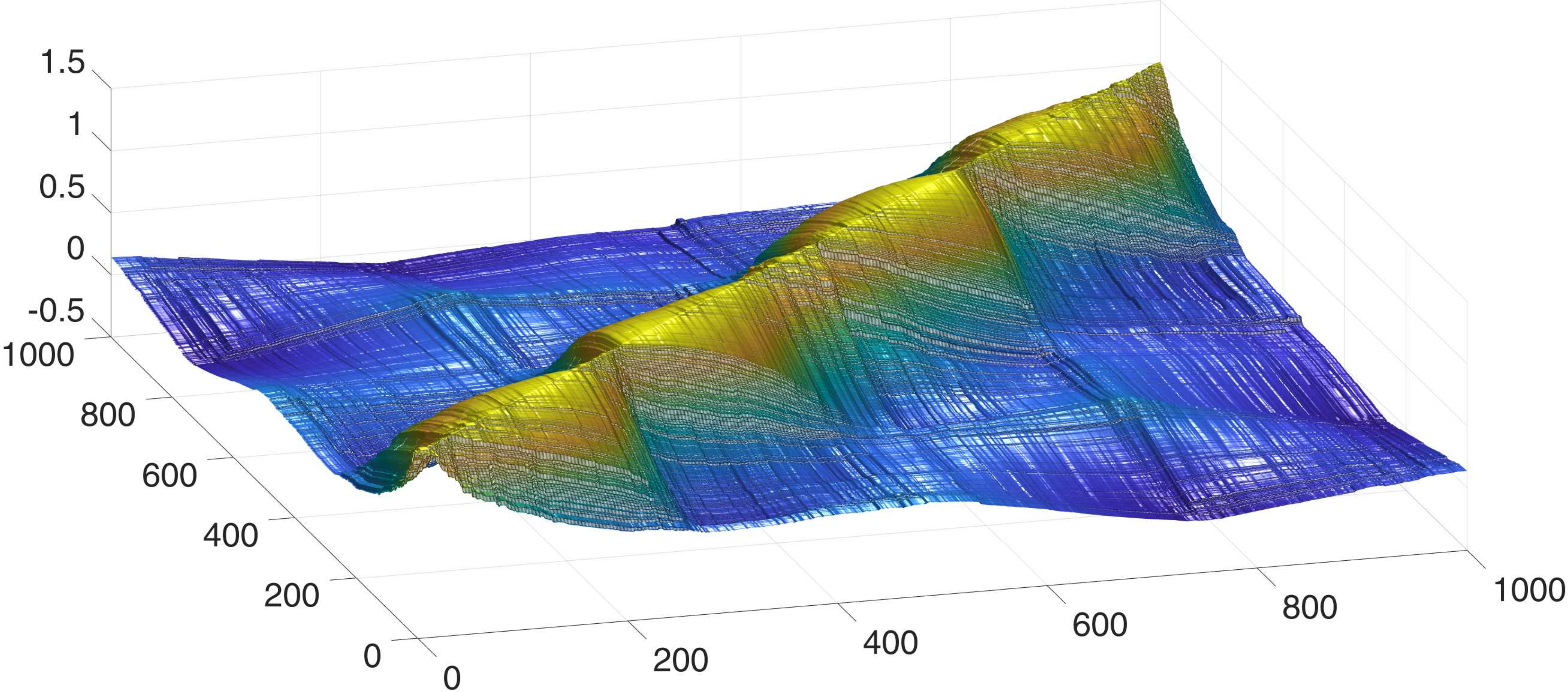}}
		\\
		\subfloat[\label{fig:ColsRes6.3-b}]{\includegraphics[height=\height\textheight,width=\width\textwidth]{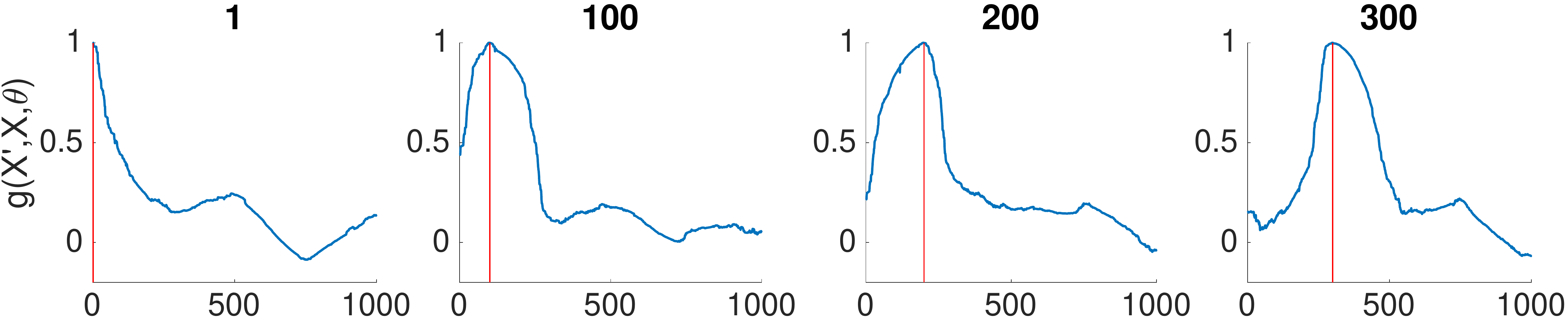}}
		\\
		\subfloat[\label{fig:ColsRes6.3-c}]{\includegraphics[height=\height\textheight,width=\width\textwidth]{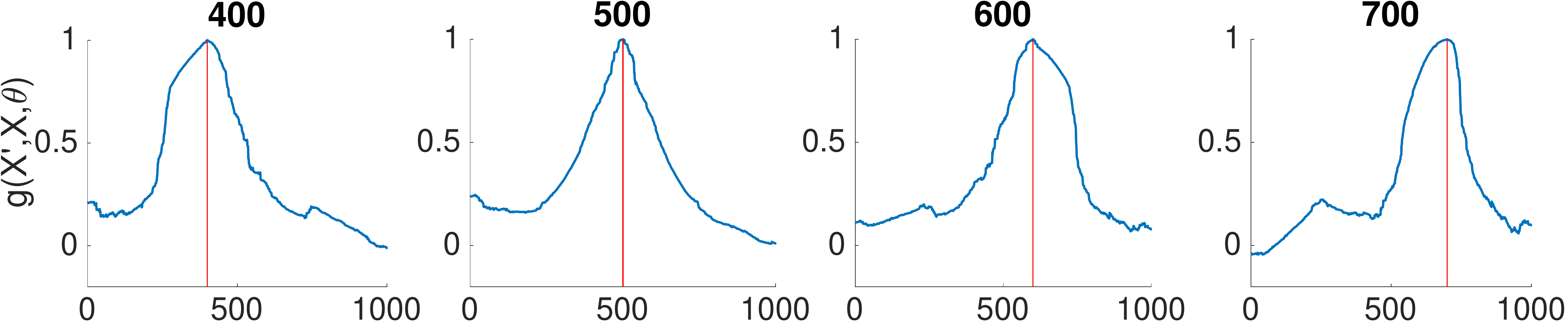}}
		\\
		\subfloat[\label{fig:ColsRes6.3-d}]{\includegraphics[height=0.14\textheight,width=0.72\textwidth]{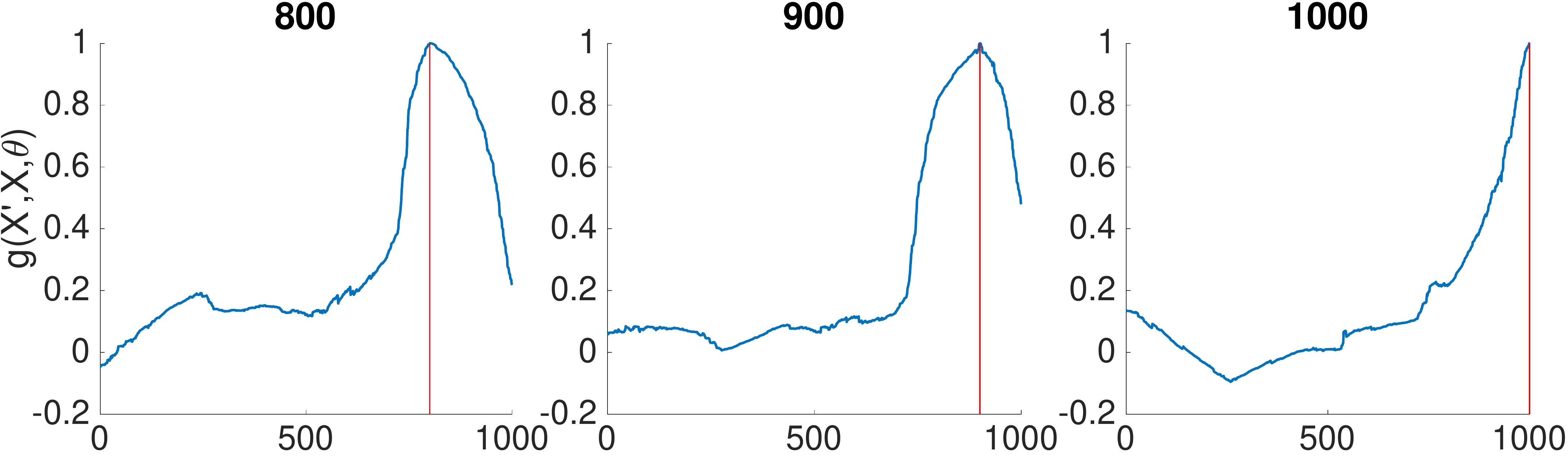}}
		
	\end{tabular}
	
	\protect
	\caption[Normalized gradient results within BD architecture.]{Normalized gradient results within BD architecture: same results from Figure \ref{fig:ColsRes6.2}, with \emph{Gramian} matrix being normalized to $\bar{G}: \bar{G}_{ij} = \bar{g}_{\theta}(X_i, X_j) =  \frac{\nabla_{\theta} f_{\theta}(X_i)^T \cdot \nabla_{\theta} f_{\theta}(X_j)}{\norm{\nabla_{\theta} f_{\theta}(X_i)} \cdot \norm{\nabla_{\theta} f_{\theta}(X_j)}}$. As observed, normalized \gs $\bar{g}_{\theta}(X_i, X_j)$ is more symmetrical along both directions of a chosen path $D$ compared with the regular \gs $g_{\theta}(X_i, X_j)$ in Figure \ref{fig:ColsRes6.2}.
	}
	\label{fig:ColsRes6.3}
\end{figure}

Further, in Figure \ref{fig:ColsRes6.3} we also plot the normalized \gs $\bar{g}_{\theta}(X_i, X_j) \triangleq  \frac{\nabla_{\theta} f_{\theta}(X_i)^T \cdot \nabla_{\theta} f_{\theta}(X_j)}{\norm{\nabla_{\theta} f_{\theta}(X_i)} \cdot \norm{\nabla_{\theta} f_{\theta}(X_j)}} = \cos \left[ \angle \left[ \nabla_{\theta} f_{\theta}(X_i), \nabla_{\theta} f(X_j) \right] \right]$ for the same setting as in Figure \ref{fig:ColsRes6.2}. Here, we can see that $\bar{g}_{\theta}(X_i, X_j)$, which is a cosine of the angle between $\theta$ gradients at points $X_i$ and $X_j$, has a more symmetrical and centered behavior compared with $g_{\theta}(X_i, X_j)$. That is, $\bar{g}_{\theta}(X_i, \cdot)$ has a peak when the second argument is equal to $X_i$, and it gradually decreases as the distance between the second argument and $X_i$ increases. Moreover, the asymmetry that we observed in case of $g_{\theta}(X_i, X_j)$ ($400$-th and $600$-th rows in Figure \ref{fig:ColsRes6.2-c}) is actually caused by the difference in a gradient norm at different points $X_i$ and $X_j$. Specifically, in Figure \ref{fig:ColsRes6.2-c} we can see that peak of $400$-th row is pushed towards the beginning of the path $D$, where $X_{400}$ has neighbors with higher norm $\norm{\nabla_{\theta} f_{\theta}(X_j)}$. In $\bar{g}_{\theta}(X_i, X_j)$ each gradient is normalized to have a unit norm, which eliminates the above asymmetry as observed in Figure \ref{fig:ColsRes6.3-c}.
Hence, nearby points with a large gradient norm may affect the \gs at a specific point and make it less symmetric/centered.

Note that in case the applied optimizer is GD, during the actual optimization the NN surface is pushed according to the $g_{\theta}(X_i, X_j)$ and not $\bar{g}_{\theta}(X_i, X_j)$. Therefore, various local asymmetries inside $g_{\theta}(X_i, X_j)$ may affect the optimization optimality.	However, Adam optimizer \citep{Kingma14arxiv}, used in most of our experiments, with its adaptive moment estimation and normalization per each weight implicitly transforms the actual "pushing" kernel of the optimization. We speculate that in this case the actual information kernel is more similar to the normalized gradient similarity. We shall leave a detailed investigation of the model kernel under Adam optimization update rule for future work.

Overall, our experiments show that NN weights undergo some \emph{uncorrelation} process during the first few thousands of iterations, after which the \emph{gradient similarity} obtains properties \textbf{approximately} similar to a local-support kernel function. Specifically, an opposite relation is formed between the \emph{gradient similarity} and Euclidean distance, where a higher distance is associated with a smaller similarity. This \emph{uncorrelation} can also be viewed as gradients (w.r.t. $\theta$) at different training points are becoming more and more linearly independent along the optimization, which as a result increases angles between gradient vectors. Hence, these gradients point to different directions inside the parameter space $\RR^{\left| \theta \right|}$, decreasing the side-influence between the training points. Furthermore, this \emph{uncorrelation} process was observed almost in each experiment, yet the radius of local support for the kernel $g_{\theta}(X', X)$ (that is, how fast \gs is decreasing w.r.t. $d(X, X')$) is changing depending on the applied NN architecture (e.g. FC vs BD) and on the specific inferred density $\probs{\usuff}$.

\section{$LSQR$ Divergence}
\label{sec:LSQRDivSec}

Here we will prove that $LSQR$ evaluation metric, considered in this paper, is an actual statistical divergence. First, the log-pdf squared error (LSQR) divergence is defined as:
\begin{equation}
LSQR(\PP, \QQ) = \int \probi{}{X} \cdot \left[ \log \probi{}{X} - \log \QQ(X) \right]^2 dX
,
\label{eq:LSQRDiv}
\end{equation}
where $\PP$ is the pdf over a compact support $\Omega \subset \RR^n$; $\QQ$ is normalized or unnormalized model whose support is also $\Omega$.

According to the definition of statistical divergences, LSQR must satisfy $\forall \PP, \QQ: LSQR(\PP, \QQ) \geq 0$ and $LSQR(\PP, \QQ) = 0 \Leftrightarrow \PP = \QQ$. Obviously, these two conditions are satisfied by Eq.~(\ref{eq:LSQRDiv}). Hence, $LSQR(\PP, \QQ)$ is the statistical divergence.

Importantly, we emphasize that $LSQR(\PP, \QQ)$ measures a discrepancy between pdf $\PP$ and model $\QQ$ where the latter is allowed to be unnormalized. Therefore, it can be used to evaluate PSO-based methods since they are only approximately normalized. However, such evaluation is only possible when $\PP$ is known analytically.

Further, in this paper $LSQR$ is measured between the target density, defined as $\PP^{\usuff}$ in the paper, and the pdf estimator $\bar{\PP}_{\theta}$ produced by some method in the following way:
\begin{equation}
LSQR(\PP^{\usuff}, \bar{\PP}_{\theta}) = \frac{1}{N} \sum_{i = 1}^{N} \left[ \log \probi{\usuff}{X^{\usuff}_{i}} - \log \bar{\PP}_{\theta}(X^{\usuff}_{i}) \right]^2
,
\label{eq:LSQRDivEst}
\end{equation}
where
$\{X^{\usuff}_{i}\}_{i = 1}^{N}$ are testing points, sampled from $\probs{\usuff}$, that were not involved in the estimation process of $\bar{\PP}_{\theta}$.

\section{Matrix $A$ from definition of \emph{Transformed Columns} Distribution}
\label{sec:App8}

Matrix $A$ was randomly generated under the constraint of having a determinant 1, to keep the volume of sampled points the same. Its generated entries are:
\begin{equation}
A
=
\scalebox{.25}{
	\begin{blockarray}{cccccccccccccccccccc}
	\begin{block}{(cccccccccccccccccccc)}
	0.190704135 & -0.103706818 & 0.287080085 & -0.224115607 & -0.0296220322 & -0.200017067 & 0.107420472 & 0.145939799 & -0.21151047 & 0.290260037 & -0.109926743 & -0.138214861 & 0.0739138407 & -0.173910764 & -0.158279581 & -0.138856972 & -0.512741096 & -0.0894244111 & -0.465075432 & 0.138203328 &  \\ 0.254157669 & 0.00972120677 & -0.425381258 & -0.165311223 & -0.0732519109 & 0.316785766 & -0.0651314216 & -0.153534853 & -0.294111694 & -0.29775682 & -0.285308807 & 0.12228138 & -0.11072477 & -0.0955035066 & 0.00942833152 & -0.252106498 & -0.40791915 & 0.14810246 & 0.219848116 & 0.0170452831 &  \\ 0.213837698 & 0.435577026 & -0.0250319656 & 0.297552176 & 0.14030724 & 0.17703815 & 0.179253182 & -0.0710653564 & 0.0507340489 & 0.235684257 & 0.33391508 & -0.40609493 & -0.197930666 & -0.322136609 & -0.146204612 & -0.164195869 & -0.0672063429 & 0.138970011 & 0.121959537 & -0.150499483 &  \\ 0.229971825 & 0.235126465 & 0.158420112 & -0.0223857003 & 0.28726862 & 0.133325519 & -0.352231148 & 0.403451114 & -0.0161522847 & -0.0193998935 & 0.0455239109 & 0.162348247 & -0.108553457 & -0.126312424 & 0.354695129 & -0.18792775 & 0.101384726 & -0.360466001 & 0.0797850581 & 0.338497968 &  \\ 0.168709235 & 0.0770914534 & -0.178165495 & -0.0661545928 & 0.323673824 & -0.216202087 & 0.475022027 & 0.129138946 & -0.0173273685 & -0.305472906 & -0.187205709 & -0.0359566427 & -0.216262061 & -0.0428909211 & -0.354871726 & 0.299579441 & 0.157803981 & -0.190573488 & 0.0280581785 & 0.281767192 &  \\ 0.238265575 & 0.14326338 & 0.323225051 & 0.101182056 & 0.222068216 & -0.42470829 & -0.132128709 & -0.203895612 & -0.38640015 & -0.193494524 & 0.15340899 & 0.0919709633 & -0.150512414 & 0.257608882 & 0.182536036 & 0.223159059 & -0.177186352 & 0.198776911 & 0.130192805 & -0.201342323 &  \\ 0.279579926 & -0.162797928 & -0.0586375005 & -0.211398563 & -0.178520507 & -0.0154862203 & -0.371463145 & 0.187233788 & -0.19506691 & 0.119455231 & -0.202696444 & -0.581504491 & -0.160227753 & 0.135365515 & -0.104123883 & 0.128700751 & 0.341221005 & 0.04776047 & 0.0956523695 & -0.105994479 &  \\ 0.220148935 & -0.2672238 & 0.259200965 & -0.348342982 & 0.155930129 & 0.0194560055 & 0.136887538 & -0.264686829 & 0.243602026 & 0.117285157 & -0.002458813 & 0.241121126 & -0.423370778 & 0.01867544 & -0.0604651676 & -0.397344315 & 0.273889032 & 0.122122216 & 0.00146222648 & -0.125711392 &  \\ 0.274834233 & 0.105359473 & 0.135585987 & -0.19681974 & -0.0573374634 & 0.272574082 & 0.0741237415 & 0.0122961299 & 0.331281502 & -0.33605766 & 0.271171768 & -0.205618232 & 0.213806318 & 0.535961439 & -0.103513592 & -0.0402908492 & -0.199964863 & -0.203164108 & -0.0657867077 & -0.0601941311 &  \\ 0.191457811 & -0.234983092 & 0.180777146 & 0.293340161 & -0.0365900702 & 0.214595475 & 0.236285794 & 0.256954426 & 0.152660465 & 0.0502990912 & -0.34518931 & 0.0711801608 & -0.128073093 & -0.0785011303 & 0.222291071 & 0.252921604 & -0.174804068 & -0.217786875 & 0.114727637 & -0.492946359 &  \\ 0.223304218 & -0.0257257131 & -0.413566391 & 0.172073687 & 0.326658323 & -0.251109479 & -0.188027609 & 0.334468985 & 0.161164411 & 0.150742708 & -0.0165671389 & 0.232902107 & 0.21926384 & 0.194760411 & -0.235966926 & -0.210659524 & -0.00475631985 & 0.172782694 & -0.148866241 & -0.300869538 &  \\ 0.243388344 & 0.243683991 & -0.121589184 & -0.187437781 & 0.0578409285 & -0.27767391 & 0.118032949 & -0.346625601 & 0.292226942 & 0.243065097 & -0.348108205 & -0.141543891 & 0.37738745 & -0.0330943339 & 0.390702777 & 0.02656679 & 0.0292203222 & -0.0520266353 & 0.165078206 & 0.0596683021 &  \\ 0.230141811 & 0.0247039796 & -0.160146528 & 0.276454478 & -0.576178718 & -0.329822403 & 0.25548108 & 0.14013065 & -0.0103114951 & -0.162178682 & 0.0992219866 & 0.012231234 & -0.189306474 & 0.0930990418 & 0.282716476 & -0.306277128 & 0.142988577 & -0.00844652753 & -0.179308874 & 0.106622196 &  \\ 0.232700446 & -0.203358735 & -0.0404096623 & 0.0651155788 & -0.0876307509 & -0.191362905 & -0.350703084 & -0.292134625 & 0.189041139 & -0.344463408 & 0.182873652 & 0.00843615229 & 0.103973847 & -0.52808934 & -0.153969081 & 0.0920142117 & 0.0274541623 & -0.30781733 & -0.126023284 & -0.157587751 &  \\ 0.199527977 & 0.204898606 & -0.0450516851 & -0.270830109 & -0.122669508 & 0.203103159 & 0.237732868 & 0.012741777 & -0.462819922 & 0.114896626 & 0.187819585 & 0.289174359 & 0.310962024 & -0.0974490092 & -0.00676534358 & 0.104883625 & 0.361388813 & -0.151993161 & -0.133748076 & -0.307757495 &  \\ 0.175957009 & 0.114449497 & 0.0205901599 & -0.309807805 & -0.235656127 & 0.00557129199 & -0.0424724202 & 0.307742364 & 0.323261613 & -0.00822297356 & 0.171003077 & 0.20675015 & -0.0593725006 & -0.232062167 & 0.0481299972 & 0.415888833 & -0.102561506 & 0.524594092 & 0.0346260903 & 0.106458345 &  \\ 0.223633992 & -0.408554226 & -0.309022911 & 0.131333656 & 0.253238692 & 0.245874156 & 0.0563174345 & -0.203397977 & -0.0592004594 & 0.179237363 & 0.275909961 & -0.0735710289 & -0.0681467794 & 0.0974804318 & 0.340180897 & 0.282145806 & -0.0050587548 & 0.0683927742 & -0.329991983 & 0.246896636 &  \\ 0.215583387 & 0.212554612 & 0.289443614 & 0.366636519 & -0.034606719 & 0.274097732 & -0.127438501 & -0.179674304 & 0.0238811214 & -0.126780371 & -0.419167616 & 0.0954798128 & 0.126317847 & 0.0330201935 & -0.166833728 & 0.039242297 & 0.231007699 & 0.29173484 & -0.381428263 & 0.182896273 &  \\ 0.210225414 & -0.42839288 & 0.239002756 & 0.174593297 & 0.0372987005 & -0.000156767089 & 0.183766314 & 0.132302659 & -0.124351715 & -0.0656208547 & 0.126841957 & -0.0486788409 & 0.471440044 & -0.112599645 & -0.0828678102 & -0.173846575 & 0.0871950404 & 0.210646003 & 0.459392924 & 0.242375288 &  \\ 0.220671036 & 0.0270987729 & -0.0325243953 & 0.176295746 & -0.299408603 & -0.00284534072 & -0.133474574 & -0.186385408 & 0.0205827025 & 0.45209226 & 0.0583508015 & 0.321753008 & -0.120986377 & 0.217766867 & -0.359227919 & 0.180578462 & -0.115873733 & -0.281505687 & 0.296303031 & 0.245563455 &  \\ 
	\end{block}
	\end{blockarray}
}
.
\label{eq:AMatDef}
\end{equation}

\end{document}